\documentclass[letterpaper, 11pt]{article}

\usepackage{
  amsmath, amsthm, amssymb, mathtools, dsfont, units,   
   listings, color, inconsolata, pythonhighlight,        
   fancyhdr, sectsty, hyperref, enumerate, enumitem }   	
\usepackage{amsthm,amssymb}
\usepackage{mathtools,thmtools}
\usepackage{bm,esint}
\usepackage{color}
\usepackage{float,graphicx,subcaption}
\usepackage{cancel}
\usepackage{mathrsfs}  
\usepackage{bbm}
\usepackage{euscript}
\usepackage{hyperref}
\usepackage{cleveref}
\usepackage{upgreek}
\usepackage[overload]{empheq}
\usepackage{booktabs} 
\usepackage{makecell}
\usepackage{multirow}
\usepackage{tikz}

\usepackage{lipsum}

\usepackage{authblk}

\usepackage[left=1.15in, right=1.15in, top=1.2in, bottom=1in, headsep=.25in]{geometry}

\usepackage[bottom]{footmisc}

\allowdisplaybreaks

\sectionfont{\fontsize{13}{15}\selectfont}
\subsectionfont{\fontsize{11}{15}\selectfont}

\hypersetup{colorlinks=true, linkcolor=blue, citecolor=blue}

\renewenvironment{abstract}
{\small\begin{quote}\noindent \par{\sc \abstractname.}}
{\noindent\end{quote}}

\definecolor{greenText}{rgb}{0.5, 0.7, 0.5}
\definecolor{greyText}{rgb}{0.5, 0.5, 0.5}
\definecolor{codeFrame}{rgb}{0.5, 0.7, 0.5}

\lstdefinestyle{code} {
  frame=single, rulecolor=\color{codeFrame},            
  numbers=left,                                         
  numbersep=8pt,                                        
  numberstyle=\tiny\color{greyText},                    
  commentstyle=\color{greenText},                       
  basicstyle=\linespread{1.1}\ttfamily\footnotesize,    
  keywordstyle=\ttfamily\footnotesize,                  
  showstringspaces=false,                               
  xleftmargin=1.95em,                                   
  framexleftmargin=1.6em,                               
  breaklines=true,                                      
  postbreak=\mbox{\textcolor{greenText}{$\hookrightarrow$}\space} 
}

\newcommand{\R}{\mathbb{R}}   
\newcommand{\C}{\mathbb{C}}   
\newcommand{\Z}{\mathbb{Z}}   
\newcommand{\N}{\mathbb{N}}   

\newcommand{\E}{\mathbb{E}}     

\newcommand{\1}{\mathds{1}}		

\DeclareMathOperator*{\argmin}{argmin}    

\renewcommand{\footrulewidth}{0pt}
\renewcommand{\headrulewidth}{0pt}

\usepackage[utf8]{inputenc} 
\usepackage[T1]{fontenc}    
\usepackage{hyperref}       
\usepackage{url}            
\usepackage{booktabs}       
\usepackage{amsfonts}       
\usepackage{nicefrac}       
\usepackage{microtype}      
\usepackage{xcolor}         
\usepackage{tikz-cd}
\usepackage{minitoc}

\usepackage{amsmath,amsfonts,bm}
\usepackage{euscript}

\def\1{\bm{1}}

\DeclareMathAlphabet{\mathsfit}{\encodingdefault}{\sfdefault}{m}{sl}
\SetMathAlphabet{\mathsfit}{bold}{\encodingdefault}{\sfdefault}{bx}{n}

\DeclareFontFamily{U}{wncy}{}
\DeclareFontShape{U}{wncy}{m}{n}{<->wncyr10}{}
\DeclareSymbolFont{mcy}{U}{wncy}{m}{n}
\DeclareMathSymbol{\Sha}{\mathord}{mcy}{"58}

\usepackage{amsthm,amssymb}
\usepackage{mathtools,thmtools}
\usepackage{bm,esint}
\usepackage{color}
\usepackage{float,graphicx,subcaption}
\usepackage{cancel}

\usepackage{todonotes}

\newcommand{\pa} {\partial}

\renewcommand{\i}{\ifmmode\mathit{\mathchar"7010 }\else\char"10 \fi}
\renewcommand{\j}{\ifmmode\mathit{\mathchar"7011 }\else\char"11 \fi}

\renewcommand{\i}{\ifmmode\mathit{\mathchar"7010 }\else\char"10 \fi}
\renewcommand{\j}{\ifmmode\mathit{\mathchar"7011 }\else\char"11 \fi}

\newcommand{\norm}[1]{\left\|#1\right\|}

\newcommand{\rulesep}{\unskip\ \vrule\ }
\newcommand{\hrulesep}{\unskip\ \hrule\ }

\newcommand{\map}{\mathcal{L}}

\newcommand{\sol}{\EuScript{S}}
\newcommand{\bX}{\EuScript{X}}

\newcommand{\bY}{\EuScript{Y}}

\newcommand{\Lip}{\mathrm{Lip}}
\newcommand{\opt}{\mathrm{opt}}
\newcommand{\set}[2]{{\left\{ #1 \,\middle|\, #2 \right\}}}
\newcommand{\slot}{{\,\cdot\,}}

\DeclareMathOperator{\atantwo}{atan2}

\makeatletter
\renewcommand{\paragraph}{%
  \@startsection{paragraph}{4}%
  {\z@}{.8ex \@plus 1ex \@minus .2ex}{-1em}%
  {\normalfont\normalsize\bfseries}%
}
\makeatother

\usepackage{color}
\definecolor{darkred}{rgb}{.6,0,0}
\definecolor{darkblue}{rgb}{0,0,.7}
\definecolor{darkgreen}{rgb}{0,.7,0}
\definecolor{darkbrown}{rgb}{0.8,0.4,0.4}

\newcommand{\sam}[1]{#1}

\newtheoremstyle{named}{}{}{\itshape}{}{\bfseries}{.}{.5em}{\thmnote{#3}}
\theoremstyle{named}

\renewcommand{\j}{{j}}
\renewcommand{\k}{{k}}

\newcommand{\id}{\mathrm{Id}}
\newcommand{\bb}{\,|\,}
\newcommand{\prior}{p_\mathrm{prior}}

\DeclareMathOperator\supp{supp}

\declaretheorem[style=deflt,numberwithin=section]{theorem}
\declaretheorem[style=deflt,sibling=theorem]{lemma}
\declaretheorem[style=deflt,sibling=theorem]{proposition}
\declaretheorem[style=deflt,sibling=theorem]{corollary}
\declaretheorem[style=deflt,sibling=theorem]{hypothesis}

\declaretheorem[style=deflt,sibling=theorem]{remark}

\usepackage{subcaption}
\usepackage{amsmath}

\usepackage{amsfonts}
\usepackage{color}
\usepackage[labelfont=bf]{caption}
\usepackage{subcaption} 

\usepackage[utf8]{inputenc} 
\usepackage[T1]{fontenc}    
\usepackage{hyperref}       
\usepackage{url}            
\usepackage{booktabs}       
\usepackage{amsfonts}       
\usepackage{nicefrac}       
\usepackage{microtype}      
\usepackage{xcolor}         
\usepackage{minitoc}
\usepackage{bbm}

\usepackage{authblk}

\usepackage{amsthm,amssymb}
\usepackage{float,leqno,amssymb,epsf,graphicx}
\usepackage{euscript}
\usepackage{mathtools,thmtools}

\usepackage{rotating}

\usepackage{pgffor}
\usepackage{tikz}

\usepackage{colortbl}
\usepackage{arydshln,xcolor,array}

\newcommand{\T}{\mathbb{T}}

\newcommand{\cA}{\mathcal{A}}
\newcommand{\cB}{\mathcal{B}}

\newcommand{\cE}{\mathcal{E}}
\newcommand{\cF}{\mathcal{F}}
\newcommand{\cG}{\mathcal{G}}

\newcommand{\cI}{\mathcal{I}}
\newcommand{\cJ}{\mathcal{J}}

\newcommand{\cL}{\mathcal{L}}

\newcommand{\cN}{\mathcal{N}}

\newcommand{\cP}{\mathcal{P}}

\newcommand{\cR}{\mathcal{R}}

\newcommand{\cU}{\mathcal{U}}

\renewcommand{\i}{\ifmmode\mathit{\mathchar"7010 }\else\char"10 \fi}
\renewcommand{\j}{\ifmmode\mathit{\mathchar"7011 }\else\char"11 \fi}

\renewcommand{\i}{\ifmmode\mathit{\mathchar"7010 }\else\char"10 \fi}
\renewcommand{\j}{\ifmmode\mathit{\mathchar"7011 }\else\char"11 \fi}

\makeatletter
\renewcommand{\paragraph}{%
  \@startsection{paragraph}{4}%
  {\z@}{.8ex \@plus 1ex \@minus .2ex}{-1em}%
  {\normalfont\normalsize\bfseries}%
}
\makeatother

\usepackage{color}
\definecolor{darkred}{rgb}{.6,0,0}
\definecolor{darkblue}{rgb}{0,0,.7}
\definecolor{darkgreen}{rgb}{0,.7,0}
\definecolor{darkbrown}{rgb}{0.8,0.4,0.4}

\usepackage{adjustbox}

\newcommand{\infobox}[1]{\makebox[2em][r]{
    \raisebox{1.5em}{
        \rotatebox{75}{\footnotesize\parbox[c]{6em}{\centering #1}}
    }
}}
\newcommand{\kbox}[1]{\makebox[6em][l]{
    \raisebox{8em}{{#1}}
}}

\newtheoremstyle{named}{}{}{\itshape}{}{\bfseries}{.}{.5em}{\thmnote{#3}}
\theoremstyle{named}

\renewcommand{\j}{{j}}
\renewcommand{\k}{{k}}

\begin{document}

\doparttoc 
\faketableofcontents 


\title{\LARGE{\bfseries 
Generative AI for fast and accurate \\ statistical computation of fluids}}

\author{
Roberto Molinaro $^{1,3,\ast}$, Samuel Lanthaler $^{2,\ast}$, Bogdan  Raoni\'c $^{1,4,\ast}$, Tobias Rohner $^{1,\ast}$, Victor Armegioiu $^{1}$, Stephan Simonis $^{5}$, Dana Grund $^{1,6}$, Yannick Ramic $^{1}$,
Zhong Yi Wan $^{7}$, 
Fei Sha $^{7}$, Siddhartha Mishra $^{1,4,\dagger}$, Leonardo Zepeda-N\'u\~nez $^{7,\dagger}$\\ 
\small$^{1}$ Seminar for Applied Mathematics, D-MATH, ETH Zurich, Switzerland, \\
\small$^{2}$ University of Vienna, Vienna, Austria, \\
\small$^{3}$ Jua.ai, Zurich, Switzerland, \\
\small$^{4}$ ETH AI Center, Zurich, Switzerland, \\
\small$^{5}$ Karlsruhe Institute of Technology (KIT), 76131 Karlsruhe, Germany, \\
\small$^{6}$ D-USYS, ETH Zurich, Switzerland,\\ 
\small$^{7}$ Google Research, Mountain View, CA 94043, USA, \\
\small$^\ast$ Equal contribution, \small$^\dagger$ Co-corresponding authors.
}

\date{}
\maketitle
\vspace{-1cm}


\begin{abstract} \boldmath
We present a generative AI algorithm for addressing the pressing task of fast, accurate, and robust statistical computation of three-dimensional turbulent fluid flows. Our algorithm, termed as \emph{GenCFD}, is based on an end-to-end conditional score-based diffusion model. Through extensive numerical experimentation with a set of challenging fluid flows, we demonstrate that GenCFD provides an accurate approximation of relevant statistical quantities of interest while also efficiently generating high-quality realistic samples of turbulent fluid flows and ensuring excellent spectral resolution. 
In contrast, ensembles of deterministic ML algorithms, trained to minimize mean square errors, regress to the mean flow. We present rigorous theoretical results uncovering the surprising mechanisms through which diffusion models accurately generate fluid flows. These mechanisms are illustrated with solvable toy models that exhibit the mathematically relevant features of turbulent fluid flows while being amenable to explicit analytical formulae. Our codes are publicly available at \url{https://github.com/camlab-ethz/GenCFD}.
\end{abstract}
\section{Introduction.}

Fluids are ubiquitous in nature and in engineering \cite{FRI}, encompassing phenomena as diverse as atmospheric and oceanic flows in climate modeling, waves and tsunamis in hydrology, flows of gases in astrophysics, sub-surface flows in mineral reservoirs and in the Earth’s mantle, blood flow in the human body to flows past vehicles such as cars and airplanes. As such, understanding, predicting, and controlling fluids is indispensable for scientific discovery and engineering design.

However, the study of fluid flows is very challenging as they span a vast range of spatio-temporal scales and encompass a rich phenomenology of states. In particular, flows at high \emph{Reynolds numbers} (${\rm Re}$) can evolve chaotically into states containing energetic eddies that span a very large range of scales~\cite{FRI}. This exhibition of multi-scale complexity and sensitive dependence on inputs is often attributed to \emph{turbulence}~\cite{FRI}, considered by Richard Feynman as the \emph{most important unsolved problem of classical physics}~\cite{feynman2015feynman}.

Fluids are mathematically modeled by (variants of) the famous \emph{Navier--Stokes} equations. In the absence of analytical solution formulae for these nonlinear systems of partial differential equations (PDEs), \emph{simulating} fluids in silico with numerical algorithms such as finite difference~\cite{BCG1}, finite element~\cite{Hest1}, finite volume~\cite{LEV1} and spectral methods~\cite{SV} etc., have emerged as the dominant paradigm for predicting fluid flows. Although highly successful in many contexts, this field of computational fluid dynamics (CFD) suffers from an intrinsic \emph{curse of computational complexity} as the underlying computational cost scales as ${\rm Re}^3$, where ${\rm Re}$ is large for many flows of interest~\cite{TBbook}. Consequently, direct numerical simulations (\emph{DNS}) of fluid flows are \emph{prohibitive} in practice, and a variety of turbulence models~\cite{TBbook} have been proposed as alternatives to DNS. 

However, at best, these models represent incomplete approximations with ad hoc closures, often including undetermined and uncertain parameters. Even so, for several downstream applications like atmospheric flows, even high-quality models, such as large eddy simulations (\emph{LES})~\cite{LESbook} entail a heavy computational burden.

Moreover, given the very high sensitivity of fluid flows to small perturbations in inputs such as initial and boundary conditions (see Fig.~\ref{fig:1}(A)), deterministic simulations, whether DNS or LES, have limited predictive power \cite{FMTacta,UQbook}. Fortunately, the computation of statistical quantities of interest is much more stable to perturbations \cite{FMTacta,UQbook} (see also Fig.~\ref{fig:1} (A)), making \emph{statistical computation}, often referred to as forward \emph{uncertainty quantification} (UQ), imperative in computational fluid dynamics as well as the preferred paradigm for design and optimization in engineering applications \cite{UQbook}. 

Alas, statistical computation of fluid flows is extremely challenging: to compute the desired statistical quantities, one typically requires an \emph{ensemble} of inputs sampled from an underlying probability distribution, where each member of this ensemble is numerically evolved with an already computationally expensive DNS or LES, resulting in an ensemble of trajectories from which the target statistics are estimated.
Although the computational cost grows linearly in the number of ensemble members, due to the slow (square-root) convergence of random sampling, one needs a large ensemble for accurate statistical computation~\cite{UQbook,FLMW1,rohner2024efficient}, making the overall pipeline virtually intractable.
This renders the design of algorithms for the fast and accurate statistical computation of fluid flows a grand challenge of modern computational science \cite{UQbook}.    

Given their success at providing fast and accurate surrogates for solutions of many PDEs, machine learning (ML) algorithms, such as PINNs \cite{KAR1,KAR4}, neural operators \cite{DeepONet,FNO,CNO}, graph neural networks \cite{Bat1} and transformers \cite{herde2024poseidon}, are promising candidates for fast statistical computation of fluid flows, by replacing the expensive numerical solver with these much faster ML-based surrogates.
Unfortunately, these deterministic neural networks, which are trained to minimize the mean square prediction errors, are observed to fail at accurate statistical computation of complex multiscale physical systems~\cite{Gencast,AIspect}. As demonstrated in Fig.~\ref{fig:1} (D), the ensembles predicted by these algorithms collapse to the mean instead of learning the underlying probability distribution of the fluid flow. 

Given this context, our main goal is to address the outstanding challenge of designing a fast and accurate framework for the statistical computation of fluid flows. To this end, we tailor the so-called \emph{score-based diffusion models}, see Fig.~\ref{fig:1} (B, C), which are particular examples of generative AI and were developed for and are widely used in image and video generation~\cite{ramesh2021zero,saharia2022photorealistic,rombach2022high,ho2022video,sora2024,bar2024lumiere}, to the disparate task of computing the statistics of fluid flows. As Fig.~\ref{fig:1} (D) already shows, we demonstrate, through extensive numerical experiments, that our method, termed as GenCFD, yields accurate approximations of statistical quantities of interest, while also producing very high-quality realizations of a variety of challenging fluid flows (see \ref{mm}). At the same time, GenCFD is several orders of magnitude faster than CFD solvers (see SI Table~15). To be more specific, GenCFD takes approximately $1$ second to generate a complex three-dimensional turbulent fluid flow. We also provide rigorous mathematical arguments and analytically tractable toy models to \emph{explain} the success of GenCFD in the statistical computation of complex physical systems, \emph{uncovering} the precise mechanisms through which a diffusion model, such as ours, can provide accurate statistical computation for complex dynamical systems such as turbulent fluid flows. Thus, with GenCFD, we present a generative AI algorithm which can transform the simulation of fluid flows and has the potential to significantly impact a large number of downstream tasks in physics, climate science, and engineering.  
\section{Problem formulation and setup} 

Fluid flows are modeled by (variants of) the Navier--Stokes equations (defined in {\bf SI} Sec.~\ref{mm}), which can be written as an abstract nonlinear PDE of the form ${\mathcal{L}}_{\bar{u}}[u]=0$, with a differential operator ${\mathcal{L}}$, $\bar{u} \in {\mathcal {X}}$ representing inputs to the PDE (such as initial and boundary conditions for the Navier--Stokes equations) and $u \in {\mathcal{Y}}$ being the solution and ${\mathcal{X,Y}}$ suitable function spaces. The resulting \emph{Solution Operator} ${\mathcal S} : {\mathcal X} \mapsto {\mathcal Y}$ maps the inputs $\bar{u}$ to the solution $u$. Given a distribution $\mu \in {\rm Prob}({\mathcal X})$, statistical computation (or forward UQ) entails the calculation of the so-called \emph{push-forward measure} ${\mathcal S}_\#\mu \in {\rm Prob}({\mathcal Y})$, which describes how uncertainties in the inputs $\bar{u}$ are transformed by the solution operator of a PDE~\cite{UQbook}. This abstract problem is very challenging in view of the intrinsic infinite-dimensionality of the underlying function spaces. Hence, in {\bf SI} Sec.~\ref{mm}, we derive how this problem of \emph{statistical computation} for PDEs can be (approximately) recast in terms of computing a \emph{conditional probability distribution} given by the (generalized) probability density $p(u|\bar{u})$, conditioned on inputs $\bar{u} \sim \bar{p}(\bar{u})$ drawn from an input distribution with density $\bar{p}$, which is an approximation to $\mu$.

In Fig.~\ref{fig:1} (A), we illustrate how this conditional probability distribution is approximated in current UQ algorithms for CFD \cite{UQbook,FLMW1,LMP1,rohner2024efficient,simonis2024computing}. In a first step, an ensemble of inputs (for instance, initial data) is drawn from the distribution $\bar{p}$. Each ensemble member is then evolved with a CFD solver that approximates the solution operator ${\mathcal S}$. Thereafter, the empirical measure of these evolved samples approximates the target distribution $p(u|\bar{u})$ and statistical quantities such as mean and variance can be readily computed. However, this process is prohibitively expensive as a large number of ensemble members need to be evolved with already computationally expensive CFD solvers.

ML algorithms for computing the target conditional distribution work by replacing the CFD solver by a neural network $\Psi_\theta \approx {\mathcal S}$ in the afore-sketched UQ algorithm, where the parameters $\theta$ are determined by minimizing the mismatch between $\Psi_{\theta}$ and ${\mathcal S}$ in the mean-square (or absolute) norm. However, as seen in Fig.~\ref{fig:1} (D) and discussed previously in \cite{Gencast,AIspect}, these ML ensembles are observed to regress to the mean of the underlying distribution and are not able to generate the variance of that distribution. These observations underscore the urgent need for the design of alternative AI approaches for the accurate statistical computation of fluids. 

To this end, we propose a paradigm shift: instead of developing fast surrogates for ensemble based computations, we seek to learn the underlying distribution \emph{directly}. In particular, we propose
a \emph{conditional diffusion} model to \emph{generate} the probability distribution $p(u|\bar{u})$. As illustrated in Fig.~\ref{fig:1} (B), a conditional diffusion model \cite{DMRev1,batzolis2021conditional} approximates the target conditional probability distribution with a two-step process. In the first \emph{forward} step, given a pair of samples, $\bar{u} \sim \bar{p}$ and $u \sim p(u|\bar{u})$, \emph{noise} is iteratively added to $u_0=u$ in order to transform it to a sample $u_K$ that follows a known distribution such as an \emph{isotropic Gaussian} of the form $p_K(u_K|\bar{u}) \sim \mathcal{N}(u_K; 0, \sigma_K^2 I)$, with zero mean and a prescribed variance $\sigma_K^2 I$. In general, this iterative process is implemented by solving a suitable stochastic differential equation (SDE, see {\bf SI} Sec.~\ref{mm} for details) forward in time \cite{karras2022elucidating}. Next, the key step is the so-called \emph{reverse step} (Fig.~\ref{fig:1} (B)) where given $\bar{u}$ and a \emph{noisy sample} $u_K \sim p_K(u_K|\bar{u})$, the \emph{reverse SDE}
\begin{equation}
\label{eq:rsde}
du_\tau = - 2\dot{\sigma}_\tau\sigma_\tau \nabla_{u_\tau}\log p_\tau(u_\tau|\bar{u}) d\tau +  \sqrt{2\dot{\sigma}_\tau\sigma_\tau} d\widehat{W}_\tau
\end{equation}
is solved backward in \emph{pseudo-time} $\tau \in [0,K]$ with a terminal  distribution $p_K$ and $\widehat{W}_\tau$ is the Brownnian motion in backward time. While postponing detailed notation for this SDE to {\bf SI} Sec.~\ref{mm}, we would like to emphasize that solving it from $\tau=K$ to $\tau=0$ recovers the target conditional distribution as $p_0(u|\bar{u}) = p(u|\bar{u})$ \cite{karras2022elucidating}.  

However, solving the SDE \eqref{eq:rsde} requires the explicit form of the so-called \emph{score-function} $\log p_{\tau}(u_{\tau}|\bar{u})$ of the underlying distribution at each $\tau \in [0,K]$, which is not available. Instead, we follow score-based diffusion models \cite{karras2022elucidating,batzolis2021conditional} and approximate the score-function in terms of the infamous Tweedie's formula by
\begin{equation}
\label{eq:tweedie1}
    \nabla_{u} \log p_\tau(u_\tau|\bar{u}) \approx \frac{D_\theta(u_\tau(\bar{u}), \bar{u}, \sigma_\tau) - u_\tau}{\sigma_\tau^2}.
\end{equation}
Here, the so-called \emph{denoiser} $D_\theta$, a neural network with trainable parameters $\theta$, takes the condition $\bar{u} \sim \bar{p}$, the \emph{noisy sample} $u_\tau(\bar{u})$ (drawn from a Gaussian $\mathcal{N}(\cdot; u, \sigma_\tau^2I)$) and the noise level $\sigma_\tau$ as inputs in order to output the \emph{clean} underlying sample. Hence, as illustrated in Fig.~\ref{fig:1} (C), we need to \emph{train} the denoiser $D_\theta$ to remove noise from the \emph{noisy sample} $u_\tau(\bar{u}) = u+\eta$, $\eta\sim \mathcal{N}(0,\sigma_\tau^2I)$ and output the clean underlying sample $u$. This is achieved by training the denoiser to minimize the \emph{denoiser training objective or diffusion loss} 
\begin{align}
\label{eq:J}
{\mathcal J}(D_\theta)
= 
{\mathbb E}_{\bar{u}\sim \bar{p}} {\mathbb E}_{u|\bar{u}} {\mathbb E}_{\eta \sim {\mathcal N}(0,\sigma_\tau^2I)}
\| D_\theta(u + \eta; \bar{u}, \sigma_\tau) - u \|^2.
\end{align}
At inference, the reverse SDE \eqref{eq:rsde}, with its score-function replaced by the trained denoiser, is (numerically) integrated backward in time to generate samples from the target distribution $p(u|\bar{u})$, given the input condition $\bar{u}$ and isotropic Gaussian noise $u_K$. 

We chose a specific neural network architecture for the denoiser in our generative AI algorithm. As detailed in {\bf SI} Sec.~\ref{mm} and illustrated in Fig.~\ref{fig:1} (C), it is a UViT \cite{saharia2022photorealistic} type neural operator specifically adapted for multiscale information processing. 

Finally, it is essential to point out that several novel elements were incorporated into conditional score-based diffusion models in order to deal with the fact that our target distributions are push-forwards of the solution operators of time-dependent PDEs, rather than distributions over static data such as natural images. These include lead-time conditioning, all-to-all training \cite{herde2024poseidon} and special variance-capturing loss functions; for details, see {\bf SI} Sec.~\ref{mm}.  

We tested our proposed conditional score-based diffusion model, GenCFD, on a suite of five challenging fluid flows (see {\bf SI} Sec.~\ref{mm} for detailed description of datasets). To provide context to our results, we also tested ML baselines on the same suite of problems. To this end, we considered three state-of-the-art neural operators as baselines (defined in {\bf SI} Sec.~\ref{mm}): the UViT model, which is also the architecture of the model underpinning GenCFD, the popular Fourier Neural Operator (FNO) \cite{FNO} and a novel variant of it that adds \emph{local convolutional} layers to the Fourier layers, which we term as C-FNO (see {\bf SI} Sec.~\ref{mm}). All the baselines are neural networks $\Psi_\theta$ that are trained to minimize the mean square error ${\mathbb E}_{\bar{u}\sim \bar{p}} \|\Psi_\theta(\bar{u})- {\mathcal S}(\bar{u})\|^2$. Statistical computation is performed by generating ensembles of the form $(\Psi_\theta)_{\#} \bar{p}$, see also {\bf SI} Sec.~\ref{mm}.  

All the models are trained with data drawn from specific distributions as outlined in {\bf SI} Sec.~\ref{mm}. However, at test time, we focus on input distributions $\bar{p} \approx \delta_{\bar{u}^\ast}$, for some $\bar{u}^\ast \in {\mathcal X}$ (see Fig.~\ref{fig:1} (A) for an illustration). We do this as i) it allows us to evaluate the ability of the models to generalize \emph{out of distribution} and ii) it is well known that, even if the initial distribution is (approximately) a Dirac measure, the intrinsic chaotic evolution of turbulent fluids \emph{spreads out the measure} \cite{FMTacta,LMP1,FLMW1} (see also Fig.~\ref{fig:1} (A) and {\bf SI} Sec.~\ref{mm}).  

\begin{figure}
	\centering
	\includegraphics[width=15.75cm]{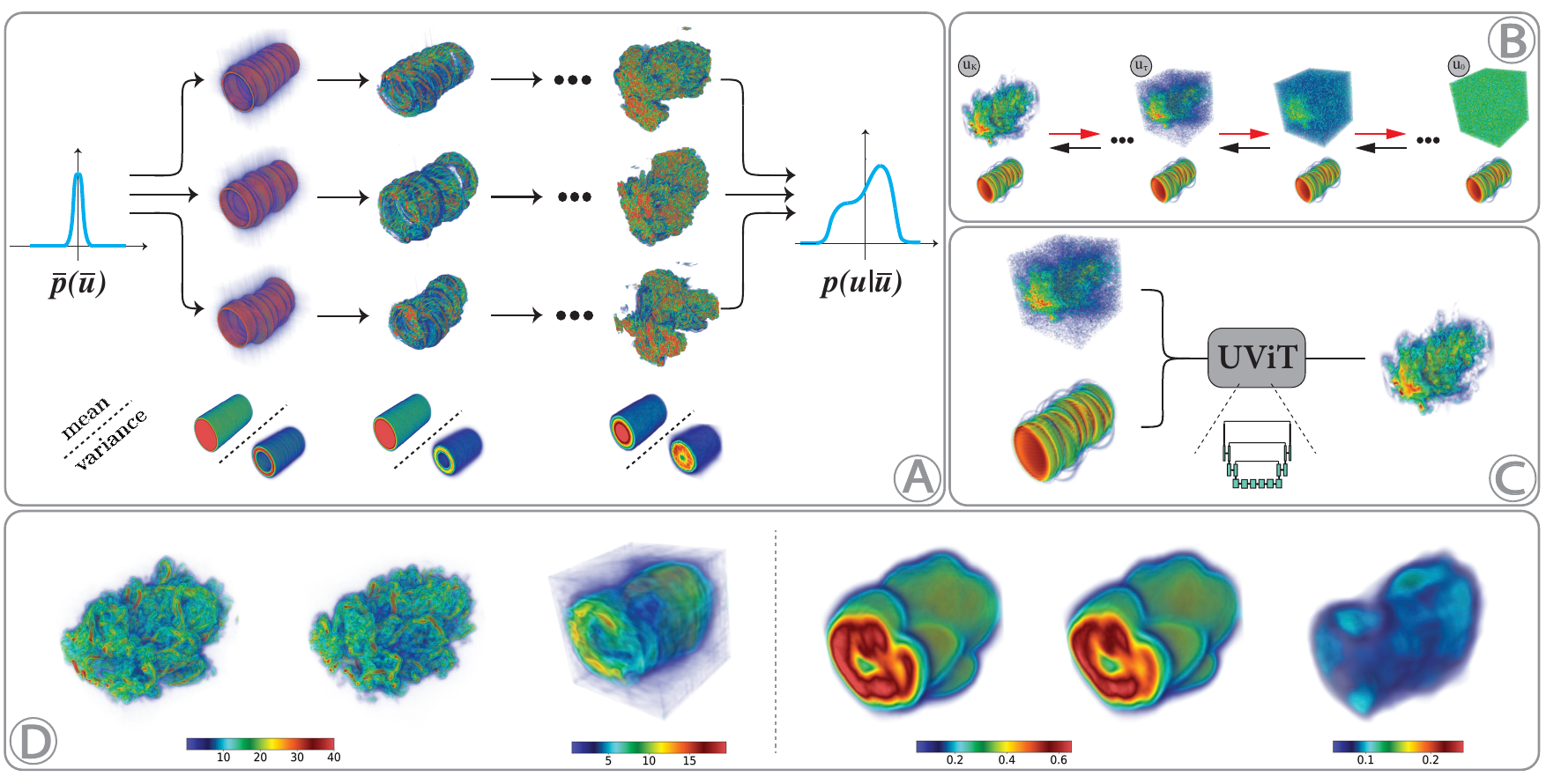} 
	\caption{\textbf{Visual summary of this article.} (\textbf{A}): Our goal here is the statistical computation of a fluid flow, i.e., computing the push-forward of the distribution $\bar{p}(\bar{u})$ on the inputs (initial and boundary conditions) with respect to the solution operator ${\mathcal S}$ of a PDE to provide the target distribution $p(u|\bar{u})$, with $u = {\mathcal S}(\bar{u})$ (at any desired time) as its output. In UQ for CFD, one draws samples $\bar{u} \sim \bar{p}$, evolves them in time with a CFD solver and computes statistics such as mean and variance from samples $u \sim p(\cdot|\bar{u})$. ML algorithms simply replace the CFD solver with a neural network surrogate, trained to minimize the mean square error with respect to each $u$. On the other hand, our method GenCFD is based on (\textbf{B}): A conditional diffusion model, which at inference (black arrows) generates $u \sim  p(\cdot|\bar{u})$, given input $\bar{u} \sim \bar{p}$ and isotropic Gaussian noise $u_K$, by solving the reserve SDE \eqref{eq:rsde} backward in time. During training of the diffusion model, noise is added iteratively (red arrows) to transform any $u \sim p$ to a noisy sample and (\textbf{C}): The denoiser (UViT for GenCFD) is a neural network that is trained to output a clean sample of the solution $u$, given input $\bar{u}$ and noise. The denoiser replaces the score function in the reverse SDE \eqref{eq:rsde}. (\textbf{D}). Results for the cylindrical shear flow dataset: individual Realizations (Left sub-panel) of the vorticity intensity and standard deviation (Right sub-panel) of the pointwise kinetic energy at $T=1$ for the ground truth (Left), GenCFD (Center) and an ML baseline C-FNO (Right). The ground truth is generated by a DNS with a spectral hyperviscosity method. Please note the different ranges of the colorbars in (D). }
	\label{fig:1}
\end{figure}
\newpage
\section{Experimental results} 

Fig.~\ref{fig:2} summarizes our  experimental observations on the \emph{Taylor--Green vortex} for the incompressible Navier--Stokes equations (see {\bf SI} Sec.~\ref{mm}), a prototypical benchmark for three-dimensional turbulent fluid flows \cite{FRI} which is widely used for validating CFD solvers as well as turbulence models.

As seen from Fig.~\ref{fig:2} (A, B, F), the flow is highly intricate with a wide range of small scales. Our task is to (approximate) the underlying distribution at future times, conditioned on a Taylor--Green initial datum. From Fig.~\ref{fig:2} (A) where we plot the pointwise kinetic energy at time $T=2$, we see that GenCFD is able to generate \emph{high-quality realistic} samples of the underlying fluid flow. In fact, it is not possible to visually distinguish between the ground truth and GenCFD-generated samples. This high sample quality of GenCFD is further demonstrated in Fig.~\ref{fig:2} (B) where we plot the (pointwise) intensity of the fluid vorticity, computed by taking the curl of the generated velocity fields. Again, it is not possible to visually distinguish between the quality of the GenCFD-generated vorticity and the ground truth. This is particularly impressive as the model has never been trained on vorticity profiles. Nonetheless, GenCFD is able to accurately approximate the \emph{multivariate} structures of the velocity profiles so that the derivatives give rise to the accurate vorticity profiles.

On the other hand, the samples generated by all the baselines (see Fig.~\ref{fig:2} (A, B) for C-FNO which is the strongest ML baseline on this dataset) are of poor quality and do not capture the small scales of the flow. In particular, high-intensity vortex tubes are completely smeared out. 

Moreover, GenCFD excels at approximating statistical quantities such as the mean (of the pointwise kinetic energy shown in Fig.~\ref{fig:2} (C)), the variance (of the kinetic energy shown in Fig.~\ref{fig:2} (D)) and even point PDFs (of the \(x\)-velocity component shown in Fig.~\ref{fig:2} (E)). In particular, the variance of the flow is very hard for deterministic neural operators to approximate as the initial condition is (nearly) a Dirac measure. Yet, GenCFD provides an excellent approximation of this statistical quantity, especially when compared to the baselines. Similarly, the point PDFs are well spread out by this time ($T=2$) and GenCFD still approximates them very well. On the other hand, the baselines completely fail at capturing the variance and the generated PDF collapses to a single point. These observations are reinforced by the quantitative results presented in Fig.~\ref{fig:2} (G) and  {\bf SI} Table~6. In Fig.~\ref{fig:2} (G), we present the $L^1$-errors in the mean and the standard deviation of the \(x\)-component of the velocity, as well as the spatially integrated $1$-Wasserstein distance between the target distribution and the conditional distribution generated by GenCFD (see {\bf SI} Sec.~\ref{mm} for definitions). The quantitative results show how well GenCFD approximates the mean, the variance, and the underlying probability distribution. In particular, it is one order of magnitude more accurate at capturing the variance and approximating ground-truth in terms of Wasserstein distance when compared to the baselines.

These qualitative and quantitative results amply demonstrate the ability of GenCFD to accurately capture the statistics of the Taylor--Green vortex. 

Capturing the correct spectral behavior is fundamental in the study of turbulent fluids as energy cascades down to the smaller scales via a power law decay of the spectrum \cite{FRI}. In Fig.~\ref{fig:2} (F) we plot the energy spectrum for the ground truth, GenCFD, and the best-performing baseline (C-FNO), from which we observe that GenCFD approximates the energy spectrum (and its power law decay) accurately, all the way down to the smallest resolved scales. On the other hand, the spectrum generated by C-FNO and other baselines is highly inaccurate and decays way too fast (exponentially) to represent a turbulent flow. This is consistent with the observation of the lack of small-scale structures in the baselines presented here as well as in the literature \cite{Gencast,AIspect}.

GenCFD's superior performance in generating accurate turbulent fluid flows extends to the other four flows considered here. For example, for the three-dimensional \emph{cylindrical shear flow} \cite{rohner2024efficient}, we observe exactly the same qualitative and quantitative results obtained for the Taylor--Green vortex, as shown in Fig.~\ref{fig:1} (D) where we present a sample of the vorticity intensity and the computed variance at time $T=1$, for the ground truth, GenCFD, and C-FNO  for a given initial conditions (see also  {\bf SI}Figures 2-6 and Table~7 for further results for this benchmark). 

What does it cost for GenCFD to generate these samples and statistics of fluid flows ? In terms of compute, we see from  {\bf SI} Table 15 that is takes approximately $0.45$ seconds for GenCFD to generate a single sample of Taylor-Green vortex or Shear flow on a GPU. In terms of \emph{sample complexity}, we recall that the test task is out-of-distribution and GenCFD has only seen  \emph{one set of input-output pairs} per input condition during training. Nevertheless, it is able to generate \emph{a large diversity in samples}, for the same input condition, as shown in  {\bf SI} Figures 2-3 and 24-26 for the Shear flow benchmark (see also Fig.~\ref{fig:2} (A, B) for the Taylor--Green vortex), showcasing the very low sample-complexity of GenCFD in generating high-quality fluid flows.

In Figure \ref{fig:3}, we present a representative glimpse of the experimental results for the other three very challenging datasets exhibiting different physics, boundary conditions and with different downstream applications. We start with a three-dimensional \emph{nozzle flow} at a Reynolds number of up to $Re=20000$ for the Navier--Stokes equations (see {\bf SI} Sec.~\ref{mm}). We consider this flow as a prototypical example of \emph{turbulent jet flows} that are widely studied in engineering \cite{TBbook}. The simulated flow field differs from the other datasets in i) being both wall-bounded as well as having a freestream leading to ii) non-trivial wall boundary conditions in place of the previously considered periodic boundary conditions and iii) \emph{the entire flow needs to be generated from a single scalar input, i.e., the injection velocity}, which is in the form of a boundary condition rather than the initial condition as in the Taylor--Green and the Cylindrical Shear Flow examples in Fig.~\ref{fig:2} and Fig.~\ref{fig:1} respectively. As visualized with a sample of (pointwise) vorticity intensity (see Fig.~\ref{fig:3} (A)), the flow, with the ground truth generated by an LES (see {\bf SI} Sec.~\ref{mm}), consists of an energetic jet emanating from the inlet and evolving in a turbulent manner to an intricate collection of multi-scale whirls and eddies further downstream. From Fig.~\ref{fig:3} (A) (see also  {\bf SI} Figs.~10-11), we observe that GenCFD is able to generate samples of this complex flow realistically while the best-performing baseline (UViT in this case) fails completely in generating the small-scale features in the vorticity and collapses onto a (laminar) jet in the middle of the flow. Similarly, statistics of this complex flow are accurately approximated by GenCFD while the baselines fail to account for the variance (Fig.~\ref{fig:3} (A)). Further qualitative and quantitative results in  {\bf SI} Figs.~10-12 and 19, and Table~9, show that GenCFD can accurately generate this complex flow from just a single scalar input vastly outperforming the baselines that are, at best, only able to generate the mean behavior.

In Fig.~\ref{fig:3} (B), we consider the three-dimensional \emph{cloud-shock interaction problem}, which is a well-established benchmark for \emph{compressible fluid flows} \cite{LEV1} (see {\bf SI} Sec.~\ref{mm}). As visualized with the density profile in Fig.~\ref{fig:3} (B), an incoming supersonic shock wave hits a high-density cloud and leads to the excitation of shock waves while creating a zone of turbulent mixing in their wake. Even though the underlying equations (compressible Euler vs.\ incompressible Navier--Stokes, see {\bf SI} Sec.~\ref{mm}) and the flow dynamics (presence of discontinuous shock waves) are very different from the previously considered examples, GenCFD is able to generate realistic flow samples, while also yielding highly accurate approximations of statistical quantities of interest (Fig.~\ref{fig:3} (B) and {\bf SI} Figs.~7-9, 19 and Table~8). On the other hand, baselines such as C-FNO fail to capture the turbulent mixing zone, although the strong shock wave is accurately computed. 

Finally, in Fig.~\ref{fig:3} (C), we present results for the \emph{dry convective planetary boundary layer}, a well-known benchmark in the atmospheric sciences \cite{sullivan2011}, heavily used for understanding the statistics of boundary layer flows and validating and calibrating turbulence models in meteorology. This atmospheric flow corresponds to the dynamics of air under the effect of a surface heat flux (modeling radiative heating through a summer day) and a weak large-scale geostrophic wind which induces shear at the surface, leading to a complex combination of updraft and downdraft plumes driving (vertically) anisotropic turbulent motion (see Fig.~\ref{fig:3} (C) for a visualization of the $x$-component of velocity). Not only are the underlying PDEs (anelastic flow equations, (see {\bf SI} Sec.~\ref{mm}) different from the previous datasets, but the physics of this flow are far richer, due to the presence of heat transfer. Nevertheless, GenCFD generates high-quality samples of this flow and accurately approximates the variance, greatly outperforming the baselines (Fig.~\ref{fig:3} (C)). We provide a detailed qualitative and quantitative analysis of this benchmark in the {\bf SI} Figs.~14-17, 19 and Table~10, particularly for (horizontally averaged) statistics, which further showcase the excellent performance of GenCFD. 

As mentioned earlier, the (approximately) Dirac test distribution for all the datasets is different from the underlying training distribution, highlighting the ability of GenCFD to generalize. 

In fact, GenCFD is able to robustly generalize to unseen test distributions, either with no additional training (\emph{zero shot}) or when fine-tuned with a few downstream samples (\emph{few shot}), as shown in the  {\bf SI} (see {\bf SI} Fig.~27 and Table~12)

Another avenue where GenCFD shines is its ability to generate the complex temporal dynamics of turbulent fluid flows, including transitions from laminar to turbulent regimes as shown in the SI. in Figs.~20-23 and Table~11. In those figures and tables, we present samples and statistics for the Taylor--Green vortex at time $T=0.8$, when the flow is still laminar and not yet turbulent, showcasing that GenCFD provides accurate samples and approximations to time-varying statistical quantities. We attribute this ability to our \emph{lead time conditioning} and an \emph{all-to-all} training procedure, which leverages the semi-group property of the underlying solution operator (see {\bf SI} Sec.~\ref{mm}).

Last but not least, the main premise for the design of surrogates for CFD solvers is their computational speed. To this end, in {\bf SI} Table~15, we compare the computational cost of generating the ground truth with state-of-the-art CFD solvers (on GPUs and CPUs) and the sample generation time of GenCFD at inference to find that GenCFD can provide anywhere between \emph{one to five orders of magnitude speedup} over traditional solvers in generating fluid flows (see {\bf SI} Table 15), depending on the dataset, underlying solver and hardware (GPU vs.\ CPU) used for CFD. In particular, GenCFD can generate a fluid flow in $1$ to $4$ seconds while it takes a standard CFD solver anywhere between minutes (on GPUs) to hours (on CPUs). 

This massive speedup, coupled with its statistical accuracy, renders GenCFD particularly attractive for widespread downstream applications. 

\section{Theory}

Why does GenCFD work so well in generating realistic fluid flows and accurately approximating their statistical and spectral behavior when baselines completely fail to do so? To address this question, we present theoretical arguments, based on rigorous mathematical analysis in the SI, with a heuristic summary here. Following {\bf SI} Sec.~\ref{mm} and for the setting considered here, the goal of statistical computation for a PDE with the (approximate) solution operator ${\mathcal S}$ and any initial datum $\bar{u}^{\ast} \in {\mathcal X}$, is to compute the conditional distribution corresponding to the \emph{Law} of the random variable ${\rm Law}_{\delta \bar{u}} {\mathcal S}(\bar{u}^\ast + \delta \bar{u})$, over randomly chosen \emph{very small} perturbations $\delta \bar{u} $, with  $\|\delta \bar{u}\|_{{\mathcal X}} \approx 0$. 

Given the sensitive dependence of turbulent fluid flows to inputs, it is reasonable to hypothesize that there exists a \emph{sensitivity scale} $\bar{\epsilon} \ll 1$, such that for perturbations $\|\delta \bar{u}\|_{{\mathcal X}} \sim \bar{\epsilon}$, the outputs are well-separated, i.e., $\|{\mathcal S}(\bar{u}^\ast + \delta \bar{u}) - {\mathcal S}(\bar{u}^\ast)\|_{{\mathcal Y}} \gg 1.$ 

On the other hand, we argue in the {\bf SI} Sec.~\ref{thr} that several factors including the empirically observed and theoretically argued fact that trained neural networks \emph{operate at the edge of chaos} \cite{feng2020optimalmachineintelligenceedge}, the well-known \emph{spectral bias of neural networks} \cite{Rah1} and the need for bounded gradients for training neural networks with gradient descent imply the \emph{insensitivity of neural networks to small-scale perturbations}, i.e., $\Psi_{\theta}(\bar{u}^\ast + \delta \bar{u}) \approx \Psi_\theta(\bar{u}^\ast)$, whenever $\|\delta \bar{u}\|_{{\mathcal X}} < \bar{\epsilon}$. Consequently, training such a neural network $\Psi_\theta$ to learn the target conditional distribution using the \emph{ensemble perturbation approach} amounts to minimizing 
\newcommand{\highlight}[1]{{\color{red!65!black}#1}}
\begin{align*}
&{\mathbb E}_{\delta \bar{u}} \|\Psi_\theta(\highlight{\bar{u}^\ast + \delta \bar{u}})
- {\mathcal S}(\bar{u}^\ast + \delta \bar{u})\|^2 \approx {\mathbb E}_{\delta \bar{u}} \|\Psi_\theta(\highlight{\bar{u}^\ast})
- {\mathcal S}(\bar{u}^\ast + \delta \bar{u})\|^2  \quad ({\text{insensitivity~hypothesis}}) \\
&= \|\Psi_\theta(\bar{u}^\ast) - {\mathbb E}_{\delta \bar{u}}  {\mathcal S}(\bar{u}^\ast + \delta \bar{u})\|^2 + {\rm Var}_{\delta \bar{u}}
[{\mathcal S}(\bar{u}^\ast + \delta \bar{u})] . \quad\qquad ({\text{bias-variance~decomposition}})
\end{align*}
Note that we cannot replace ${\mathcal S}(\bar{u}^\ast + \delta \bar{u})$ with ${\mathcal S}(\bar{u}^\ast)$ above due to the sensitive dependence of ${\mathcal S}$ to inputs. As the second term in the above sum is independent of $\theta$, the optimal neural network is given by $\Psi_\theta(\bar{u}^\ast)={\mathbb E}_{\delta \bar{u}}  {\mathcal S}(\bar{u}^\ast + \delta \bar{u})$, which is precisely the \emph{mean} of the ensemble. This simple argument clearly explains the extensive empirical observation, both here and in the literature \cite{AIspect,Gencast}, of why ensembles of neural networks trained to minimize least-square errors for learning turbulent fluids regress to the mean and fail to generate sufficient variance. 

However, the same insensitivity hypothesis also implies that the denoisers $D_\theta$ in a diffusion model such as GenCFD, being neural networks, will be \emph{as insensitive to small input perturbations}, i.e., $D_{\theta}(\bar{u}^\ast + \delta \bar{u}) \approx D_\theta(\bar{u}^\ast)$, when $\|\delta \bar{u}\|_{{\mathcal X}} < \bar{\epsilon}$. Then, how are diffusion models such as GenCFD, based on the same neural networks, able to approximate the target distribution? The answer to this lies in the nature of the \emph{loss function} \eqref{eq:J} in training diffusion models as the following calculation shows. Starting with the specific form of the diffusion training loss \eqref{eq:J} in our context as, 
\begin{align*}
{\mathcal J}(D_{\theta}) &= {\mathbb E}_{\delta \bar{u}} {\mathbb E}_{\eta} \left[ 
\Vert 
D_{\theta}({\mathcal S}(\bar{u}^\ast+\delta \bar{u}) + \eta;
\highlight{\bar{u}^\ast + \delta \bar{u}}
,\sigma) 
- 
{\mathcal S}(\bar{u}^\ast+\delta \bar{u})
\Vert^2
\right] \\
&\approx
{\mathbb E}_{\delta \bar{u}} {\mathbb E}_{\eta} \left[ 
\Vert 
D_{\theta}({\mathcal S}(\bar{u}^\ast+\delta \bar{u}) + \eta;
\highlight{\bar{u}^\ast}
,\sigma) 
- 
{\mathcal S}(\bar{u}^\ast+\delta \bar{u})
\Vert^2
\right] \quad ({\text{insensitivity~hypothesis}})
\\
&=
{\mathbb E}_{u} {\mathbb E}_{\eta} \left[ 
\Vert 
D_{\theta}(u + \eta;\bar{u}^\ast,\sigma) 
- 
u
\Vert^2
\right],
\end{align*}
where the last line follows by a change of variables to $u={\mathcal S}(\bar{u}^\ast+\delta \bar{u})$. Thus, ${\mathcal J}$ is the \emph{denoiser training objective or diffusion loss} for recovering the distribution corresponding to the ${\rm Law}_{\delta \bar{u}} {\mathcal S}(\bar{u}^\ast + \delta \bar{u})$, conditioned on the input $\bar{u}^*$, which is precisely the goal of our statistical computation. This formal argument reveals the \emph{surprising mechanism} through which a diffusion model can leverage the highly unstable nature of sensitive maps such as solution operators of fluid flows, to accurately approximate the conditional distributions, even when the underlying neural networks themselves are not sensitive to small perturbations, justifying the empirically observed performance of GenCFD. 

Given the formidable mathematical challenge of analytically characterizing the solution operators of the Navier--Stokes equations, we present \emph{solvable toy models}, which capture essential features of turbulent fluid flows while still being analytically tractable. To this end, in {\bf SI} Sec.~\ref{mm} (See Fig.~\ref{fig:4} (A) for visualization), we consider a sequence of simple maps between unit intervals, indexed by a small parameter $\Delta$ that encodes input sensitivity. These maps contain oscillations on progressively finer and finer scales as $\Delta \rightarrow 0$. Consequently, the asymptotic limit of these maps can only be described in terms of (pointwise) statistics which are explicitly computed in the {\bf SI} Sec.~\ref{thr}. As their Lipschitz constant blows up when $\Delta \rightarrow 0$, these maps are clearly very sensitive to small (spatial) input perturbations. On the other hand, the spectral bias of neural networks has been widely explored in the context of one-dimensional oscillatory maps \cite{Rah1} and it is well established that they fail to approximate high frequencies, making them \emph{insensitive} to perturbations at such small scales. As expected from the theory presented above, Fig.~\ref{fig:4} (C) shows how a multilayer perceptron (MLP) trained to minimize the mean square error between its prediction and the underlying map behaves as $\Delta \rightarrow 0$ and increasingly higher frequency oscillations are introduced. We see that for relatively large $\Delta$, the ML model provides an accurate approximation, at least after a lot of training steps (see {\bf SI} Fig.~28). However, when $\Delta \ll 1$, as predicted by the theory, the model fails to approximate the fine-scale oscillations and regresses to the mean. On the other hand, as shown in Fig.\ref{fig:4}(B), a diffusion model (the same MLP but trained with the diffusion loss \eqref{eq:J}) is able to approximate the underlying map for all values of $\Delta$, including when $\Delta \ll 1$, where it predicts the correct limit distribution. Both these observations are rigorously proved in the {\bf SI} Sec.~\ref{thr} by deriving explicit formulae for the optimal denoisers, putting our proposed theory on a firm mathematical footing for this toy problem. Moreover, in the {\bf SI} Sec.~\ref{thr}, we also present and rigorously analyze a second toy problem which mimics the spectral behavior of fluid flows, reproducing energy spectra similar to Fig.~\ref{fig:2} (F).

\begin{figure}
	\centering
	\includegraphics[width=15.75cm]{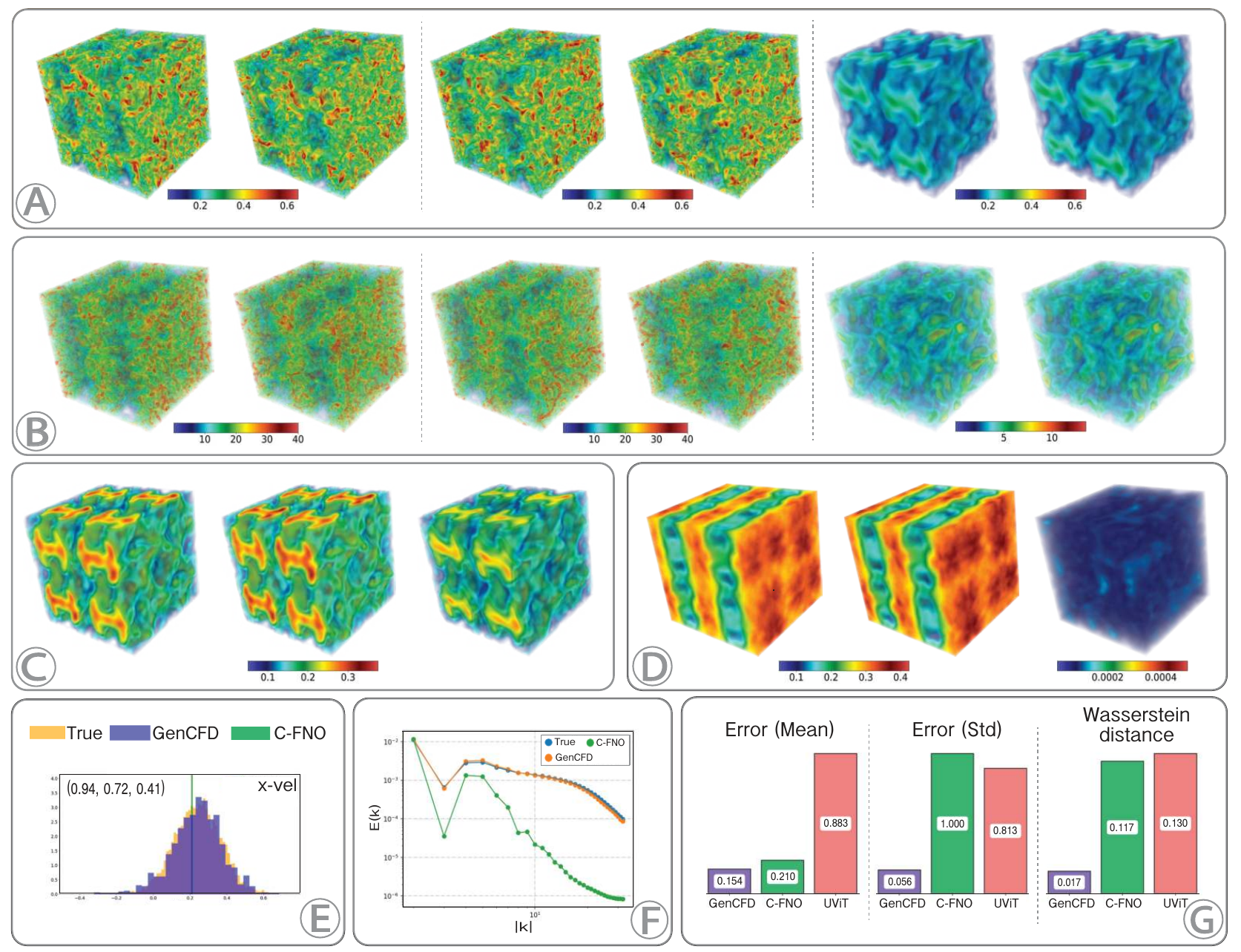} 
	\caption{\textbf{Results for the Taylor--Green vortex dataset.} (\textbf{A} and \textbf{B}): Two Samples of the flow at time $T=2$ for the same initial condition with ground truth (Left), GenCFD (center) and C-FNO baseline (Right) for the pointwise kinetic energy (A) and vorticity intensity (B). (\textbf{C}): Mean and (\textbf{D}): Standard deviation, of the pointwise kinetic energy at time $T=2$ with ground truth (Left), GenCFD (Center) and C-FNO (Right). (\textbf{E}): PDF of the $x$-velocity component at the spatial point $(0.94,0.72,0.41)$ and $T=2$ and (\textbf{F}): Energy spectrum at $T=2$. (\textbf{G}): Errors in predicting the mean (Left), standard deviation (Center) and $1$-Wasserstein distance (Right) at time $T=2$ with GenCFD, C-FNO and UVIT. Ground truth is generated by a DNS with a spectral hyperviscosity method. Please note the different ranges of the colorbars in (B) and (D).}
	\label{fig:2}
\end{figure}

\newpage
\begin{figure}
	\centering
	\includegraphics[width=15.75cm]{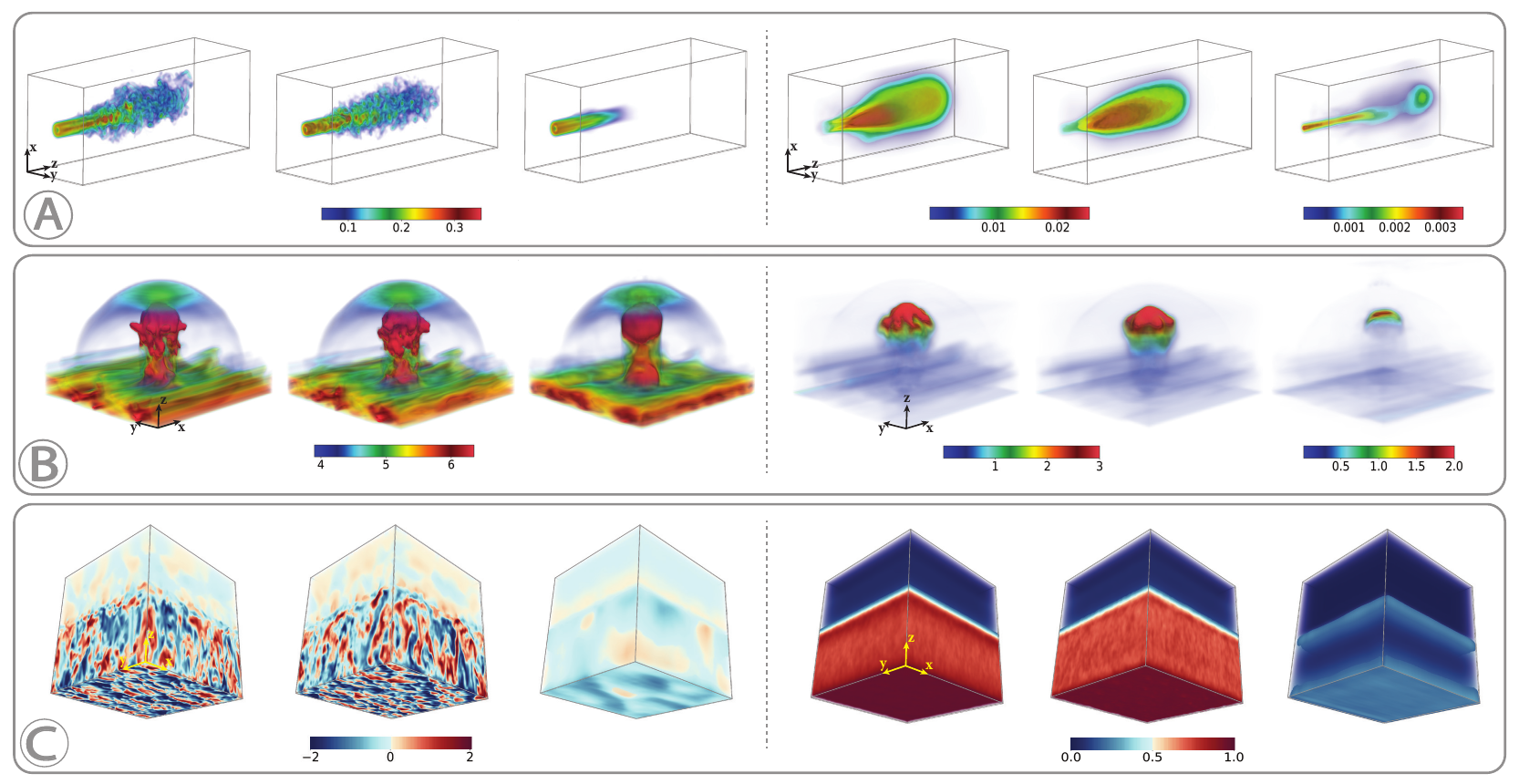} 
	\caption{\textbf{Results for the other flows: noozle flow, cloud-shock interaction and dry convective boundary layer.} (\textbf{A}): A sample of the vorticity intensity (Left sub panel) and the standard deviation of the (pointwise) kinetic energy (Right sub panel) for the nozzle flow at time $T=1$ with ground truth (left), GenCFD (center) and UViT (right). (\textbf{B}): A sample of the density (Left sub panel) and standard deviation of the density (Right sub panel) at time $T=0.6$ for the cloud-shock interaction problem with the compressible Euler equations with ground truth (left), GenCFD (center) and C-FNO (right). (\textbf{C}): A sample of the $x$-component of the velocity (Left sub panel) and standard deviation of the $x$-velocity at time $T=2.4$ for the dry convective planetary boundary layer dataset with ground truth (Left), GenCFD (Center) and UViT (Right). The ground truth is generated by an LES with a lattice Boltzmann method (nozzle flow), a DNS with a high-resolution finite volume method (cloud-shock interaction) and an LES with a WENO finite difference method (convective planetary boundary layer). Please note the different ranges of the colorbars in Panels A and B (Right subpanels).}
	\label{fig:3}
\end{figure}

\begin{figure}
	\centering
	\includegraphics[width=15.75cm]{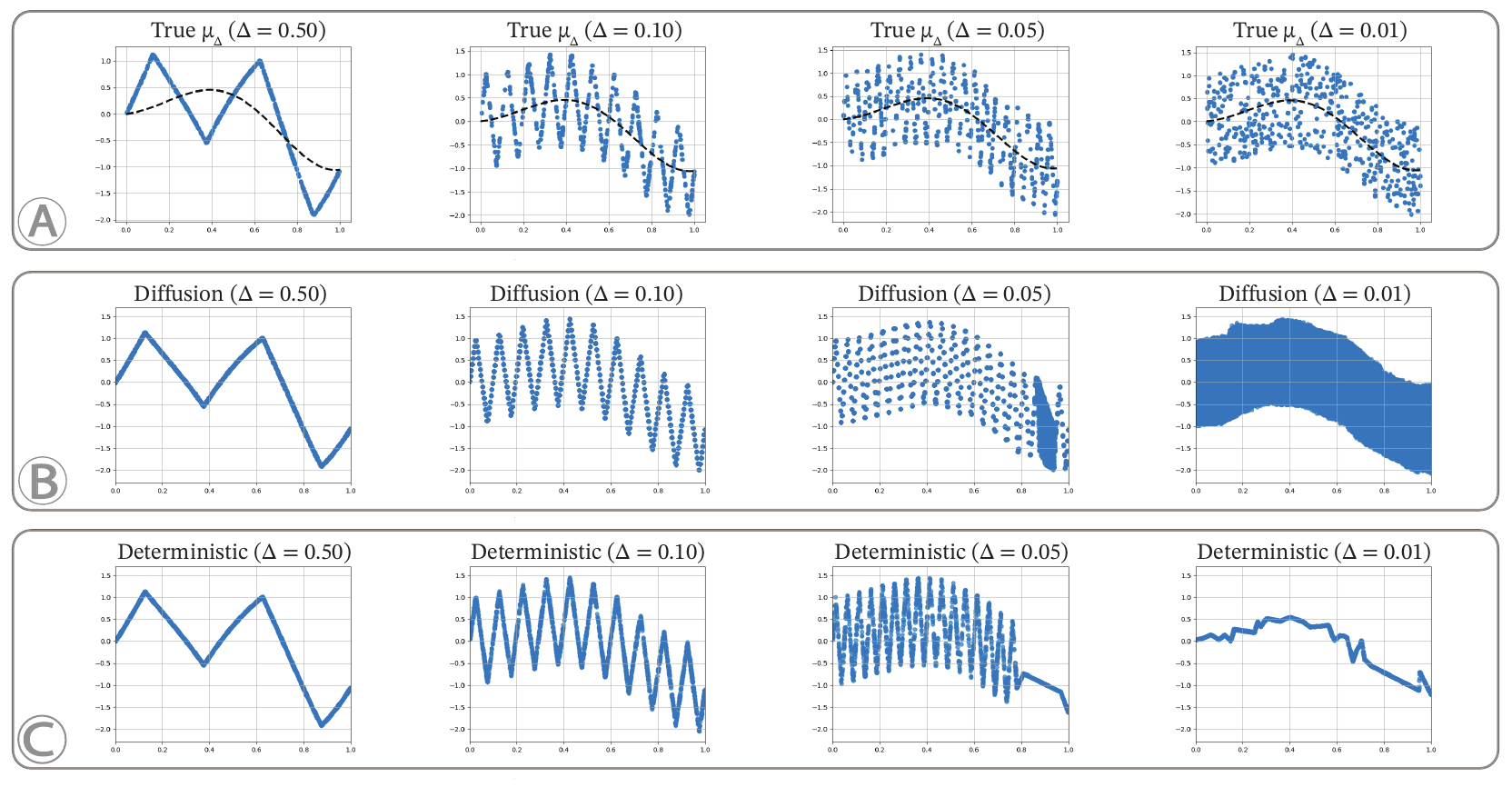} 
	\caption{\textbf{Results for Toy Model $\# 1$.} This solvable toy model captures relevant aspects of turbulent fluids while being analytically tractable (for a detailed problem formulation, see {\bf SI} Sec.~\ref{mm}). (\textbf{A}): Visualization of the underlying maps ${\mathcal S}^\Delta$ (as in {\bf SI} Sec.~\ref{mm}) for different values of $\Delta$ representing oscillations at higher and higher frequencies (inversely proportional to $\Delta$) being added to a mean function (in black), with the $\Delta \rightarrow 0$ limit of ${\mathcal S}^\Delta$ being uniform probability distributions at each point. (\textbf{B}): Results for different values of $\Delta$ with a score-based diffusion model (with an MLP as the denoiser) after $10000$ epochs of training to minimize the denoiser training objective or diffusion loss \eqref{eq:J}. (\textbf{C}): Results for different values of $\Delta$, with an MLP after $10000$ epochs of training to minimize the mean square loss. Panel (B) clearly shows how the diffusion model can accurately approximate the underlying function for large $\Delta$, while approximating the underlying probability distribution for $\Delta \approx 0$. On the other hand, Panel (C) shows that the same neural network (MLP) trained to minimize the mean square loss is able recover the underlying map for large $\Delta$ but collapses to the mean as $\Delta \approx 0$. These observations are rigorously proved in the SI.}
\label{fig:4}
\end{figure}
\clearpage
\newpage

\newpage

\newpage
\section{Concluding remarks}
We address the outstanding challenge of designing very fast and accurate algorithms for the statistical computation of fluid flows by proposing GenCFD, a \emph{conditional score-based diffusion model}. With extensive experimentation on a variety of three-dimensional datasets with comprehensive evaluation metrics, we demonstrate the capability of GenCFD to generate realistic and statistically accurate flows, while being several orders of magnitude faster in runtime than the optimized CFD solvers we tested (see {\bf SI} Table~15). These empirical results hold equally well for popular academic benchmarks, such as the Taylor--Green vortex, as well as engineering flows, like turbulent round jets, and atmospheric flows, like the dry convective planetary boundary layer, showcasing the widespread applicability of GenCFD. Moreover, we provide theoretical arguments to support the strong performance of GenCFD: given the insensitivity of neural networks to very small perturbations, training the neural network to match trajectories of sensitive dynamical systems, such as the PDEs governing fluid flow, leads to a regression to the mean. However, we demonstrate that the same neural networks when trained to estimate underlying probability distributions in a Diffusion model are able to recover the statistical behavior of these sensitive dynamical systems.  

Compared to traditional CFD solvers, the main advantage of GenCFD is its speed, which can be several orders of magnitude faster in accurately computing flow statistics ({\bf SI} Table~15). Moreover, GenCFD is completely \emph{data-driven} and
agnostic to the underlying PDE, whereas a CFD solver needs explicit information about the latter.
Compared to deterministic ML algorithms such as neural operators, the main advantage of GenCFD is its statistical accuracy, whereas ensembles based on these ML methods, albeit fast, are not statistically accurate as we have shown that they regress to the mean and fail to capture the underlying variance.

In the context of (turbulent) fluid flows,\cite{gao2024bayesian,gao2024generative,gao2024generative2,yang2023denoising,oommen2024integrating} do already consider diffusion models.
However, these articles either focus on two-dimensional flows or on learning coarse-grained dynamics by using a diffusion model either for the decoding stage or for generating a sequence of coarse variables or by combining neural operators (for coarse graining) and diffusion models (for finer scales).
Moreover, they mostly provide average metrics that do not demonstrate their ability to compute the entire target distribution, nor do they present any theoretical justification of why their methods work. In contrast, GenCFD is an end-to-end conditional diffusion model that can generate statistically accurate snapshots (of the whole trajectory) of turbulent flows given underlying inputs.

The algorithmic pipeline of GenCFD can be seamlessly modified to provide fast and accurate generation of turbulent flows to other high-impact applications, such weather and climate modeling \cite{Tapio-ESM}, where atmospheric and oceanic flows are the cornerstone of earth system modeling. As GenCFD can accurately generate convective boundary layers, it is natural to extend it to moist flows such as clouds, which can have a large impact on cloud-resolving weather and climate simulations \cite{Tapio-ESM}. Along the same lines, the theoretical underpinning of GenCFD can also help understand the strong performance of diffusion model-based probabilistic weather emulators, such as \cite{Gencast}, particularly when compared to ensembles generated by deterministic ML weather emulators, which are accurate in short-term weather modeling, but they do not capture the true probabilistic nature of even medium-term, let alone long-term weather modeling \cite{Gencast,AIspect}.

The general framework of GenCFD is very versatile, and it can be readily extended to other fluid flows, particularly around obstacles by either masking the obstacle \cite{CNO} or adding graph neural network-based encoders and decoders \cite{Gencast}. Similarly, extensions to plasma flows, governed by magnetohydrodynamics equations, are relatively straightforward.
Although diffusion models have been recently used in several applications such as statistical downscaling~\cite{mardani2024residual,Downscaling:Wan_2023}, physical inverse problems~\cite{dasgupta2024conditional,zhang2024back}, ensemble augmentation~\cite{li2024generative}, and data assimilation~\cite{rozet2023score};  
the theoretical analysis presented here provides further impetus for their adoption in an even wider variety of multiscale physical and engineering systems whose outputs depend sensitively on inputs. These include (but are not limited to): i) Computing invariant measures for the long-time limit of chaotic dynamical systems such as the Lorenz system; ii) Bayesian inverse problems \cite{STU1}, particularly for fluid flows, where the inverse operator is known to be unstable~\cite{LMW1} even when the forward problem is well-posed and the goal of statistical computation is to sample from the posterior measure; iii) Non-convex variational problems in materials science \cite{Mat1} which model phase transitions in crystalline materials; iv) Homogenization of multi-scale materials, particularly in the modeling of composites~\cite{Hom1}.
\newpage
\begin{center}
{Supplementary Information for:} \\ 
\bfseries \boldmath Generative AI for fast and accurate \\ statistical computation of fluids \\
\end{center}
\addcontentsline{toc}{section}{} 
\part{} 
\parttoc 

\clearpage
\newpage

\section{Methods}\label{mm}

\subsection{Problem Formulation}

\paragraph{Governing Equations.} The PDEs governing fluid flows are special cases of the generic time-dependent PDE
\begin{equation}
\label{eq:pde}
\begin{aligned}
\pa_t u(x,t) + \map\left(u,\nabla_x u, \nabla^2_x u , \ldots \right) &= 0, \quad \forall  x \in D \subset \R^d, t \in (0,T), \\`
\cB(u) &= u_b, \quad \forall (x,t) \in \pa D \times (0,T), \\
u(x,0) &= \bar{u}(x), \quad x \in D, 
\end{aligned}
\end{equation}
where, $d$ is the spatial dimension, $T$ is the time-horizon, $u \in C(\bX;[0,T])$ is the solution of \eqref{eq:pde}, for a function space $\bX \subset L^p(D;\R^{n})$ for some $1 \leq p < \infty$, $\bar{u}\in \bX$ is the initial datum, $u_b$ is the boundary datum, and $\map$, $\cB$ are the underlying differential and boundary operators, respectively. 

Concrete examples of \eqref{eq:pde} are given by the well-known Navier--Stokes equations~\cite{MB1} for incompressible fluid flows, which take the form 
\begin{equation}
    \label{eq:NS}
    \begin{aligned}
        \pa_t u +(u\cdot \nabla)u + \nabla p &=\nu \Delta u, \\
        {\rm div}~u &=0, 
    \end{aligned}
\end{equation}
in a domain $D\subset \R^d$ with suitable boundary conditions. Here, $u:[0,T] \times D \to \R^d$ is the velocity field and $p:[0,T] \times D \to \R$ is the pressure. The parameter $\nu$ is the so-called \emph{kinematic viscosity} of the fluid and is inversely proportional to the \emph{Reynolds number} ($\mathrm{Re}$). 

Similarly, compressible fluid flow is modeled by the \emph{compressible Navier--Stokes} equations~\cite{LL1}. Again, most compressible fluids of interest have $\mathrm{Re} \gg 1$. Consequently, one is interested in the corresponding infinite Reynolds number limit which yields the \emph{compressible Euler equations}~\cite{LL1}. These nonlinear PDEs are special cases of so-called \emph{hyperbolic systems of conservation laws}: a large class of PDEs of the generic form~\cite{dafermos2005hyperbolic}
\begin{align}
  \label{eq:conslaw}
  \partial_t u(x,t) + \nabla \cdot F(u(x,t)) = 0 .
\end{align}
Here $u: D \times [0,T] \to \R^m$ is the physical state with $m$ components. The function $F: \R^m \to \R^{d\times m}$ is the physical flux, which describes how the physical state variables are transported through the system, and $\nabla_x \cdot F(u(x,t))$ is the (spatial) divergence of the vector field $F \circ u: D\times [0,T] \to \R^{d\times m}$, $(x,t) \mapsto F(u(x,t))$, with components $[\nabla_x \cdot F(u(x,t))]_j = \sum_{k=1}^d \partial_{x_k}(F_{k,j}(u(x,t)))$.  

In the specific example of the compressible Euler equations, the state variables are $u=[\rho,\rho v, E]$, with density $\rho$, velocity $v=[v_{x_1},v_{x_2},\cdots,v_{x_d}]$, pressure $p$ and total energy $E$ related by the ideal gas equation of state
\begin{equation}
    E = \frac{1}{2}\rho |u|^2 + \frac{p}{\gamma - 1},
\end{equation}
with gas constant $\gamma$. The corresponding flux function is given by,
\begin{equation}
    F=[\rho v, \rho v\otimes v+p{\bf I},(E+p)v] .
\end{equation}

\paragraph{Statistical Computation.} For simplicity of the exposition, we assume that the boundary conditions in the PDE \eqref{eq:pde} are fixed. Then, the solutions of the time-dependent PDE \eqref{eq:pde} are given in terms of the underlying \emph{solution operator} $\sol:[0,T]\times \bX\mapsto \bX$ such that $u(t)=\sol^t (\bar{u}) = \sol(t,\bar{u})$ is the solution of \eqref{eq:pde} at any time $t \in [0,T]$. 

As mentioned in the Main Text, statistical computation of \eqref{eq:pde}, also termed as forward uncertainty quantification (UQ), refers to the computation of the \emph{push-forward measure} $\hat{\mu}_t = \sol^t_{\#}\bar{\mu}$ of some input measure $\bar{u} \sim \bar{\mu} \in {\rm Prob} (\bX)$ by the solution operator $\sol^t$ of \eqref{eq:pde}.

Unfortunately, this solution operator may not be necessarily well-defined~\cite{FMTacta} (particularly when $d=3$), and 
even when $\sol^t$ is well-defined, it can be very sensitive to initial conditions, i.e., $\|\sol^t(\bar{u}) - \sol^t(\tilde{u})\|_\bX$ can grow exponentially in time, even when $\|\bar{u} - \tilde{u}\|_\bX \ll 1$~\cite{MB1}. Then, the question arises:

\begin{center}
    \emph{How can we even define the push-forward measure $\hat{\mu}_t = \sol^t_{\#}\bar{\mu}$?}
\end{center}

To answer this question, we observe that, in practice, one cannot access the solution operator $\sol^t$ explicitly. Instead, one approximates $\sol^t$  with numerical simulations (or analytical approximations) resulting in an operator $\sol^{t,\Delta} \approx \sol^t$ (in a suitable sense), for small enough values of a \emph{discretization parameter} $\Delta$. As $\sol^{t,\Delta}:\bX^{\Delta} \to \bX^{\Delta}$ maps between finite dimensional spaces $\bX^{\Delta} \cong \R^N$ for $N \gg 1$, the push-forward $\hat{\mu}_t^{\Delta}=\sol^{t,\Delta}_{\#}\bar{\mu}$ is always well-defined. Given this, we consider the limit 
\begin{equation}
    \label{eq:ss1}
    \hat{\mu}_t =  \lim\limits_{\Delta \rightarrow 0} \hat{\mu}^{\Delta}_t = \lim\limits_{\Delta \rightarrow 0} \sol^{t,\Delta}_{\#}\bar{\mu},
\end{equation}
in a suitable topology. Clearly, if $\sol^{t,\Delta} \rightarrow \sol^t$ as $\Delta \rightarrow 0$, the above limit is simply $\hat{\mu}_t = \sol^t_{\#}\bar{\mu}$. The interesting case materializes when the deterministic approximation $\sol^{t,\Delta}$ does not converge to a well-defined limit as $\Delta \rightarrow 0$ as in the case of computing unstable and turbulent fluid flows~\cite{FKMT1,FMTacta,LMP1,FLMW1}. Nevertheless, the limit \eqref{eq:ss1} can still be well-defined and one can even observe strong convergence to the limit, in the sense that $W_p(\hat{\mu}^{\Delta}_t,\hat{\mu}_t) \to 0$ as $\Delta \to 0$, for the appropriate $p$-Wasserstein metric on measures~\cite{LMP1,FLMW1}. These observations are formalized under the rubric of the theory of \emph{statistical solutions} of PDEs \cite{FTbook,FLM1,FMW1} and the limit measure $\hat{\mu}_t$ is termed as the statistical solution of the PDE \eqref{eq:pde}. 

Given the above discussion, the goal of statistical computation for fluid flows can be reformulated as follows: we are given initial data $\bar{\mu} \in {\rm Prob}(\bX)$, which we approximate by a finite-dimensional projection $\bar{\mu}^{\Delta}$ for a given discretization parameter $\Delta > 0$. With the approximate solution operator $\sol^{t,\Delta}$ given by a suitable finite-dimensional discretization of \eqref{eq:pde}, and the map $\left(\sol^{t,\Delta} \times \text{ID} \right): \bX \rightarrow \bX \times \bX$, such that $\left(\sol^{t,\Delta} \times \text{ID} \right)(\bar{u}) = (\sol^{t,\Delta}(\bar{u}), \bar{u})$,
we consider the distribution, 
$\mu^{\Delta}_t := \left(\sol^{t,\Delta} \times \text{ID} \right)_{\#} \bar{\mu}^{\Delta}$ in order to explicitly condition the evolution in terms of the initial data. The corresponding limit distribution is given by
\begin{equation}
    \label{eq:ss2}
    \mu_t =  \lim\limits_{\Delta \rightarrow 0} \mu^{\Delta}_t = \lim\limits_{\Delta \rightarrow 0} \left(\sol^{t,\Delta} \times \text{ID}\right)_{\#} \bar{\mu}^{\Delta}.
\end{equation}
Following~\cite{FMW1,LMP1,FLMW1}, the limit can be well-defined under suitable hypotheses and $W_p(\mu^{\Delta}_t,\mu_t) \to 0$, even when the solution operator $\sol^{t,\Delta}$ does not converge. 
In general, this limit measure admits a \emph{conditional representation} 
\begin{equation}
    \label{eq:cond1}
    \mu_t(du,d\bar{u}) = P_t(du\bb\bar{u}) \; \bar{\mu}(d\bar{u}),
\end{equation}
where $P_t(du\bb \bar{u})$ represents the conditional probability distribution of 
$u(t)$ given $\bar{u}$.
If $\bar{\mu} = \delta_{\bar{u}}$ and $\sol^{t,\Delta}$ converge to the solution operator $\sol^t$ of \eqref{eq:pde}, then $ P_t(du\bb\bar{u})= \delta_{\sol^t(\bar{u})}$. However, the interesting case of unstable and turbulent fluid flows corresponds to a \emph{non-Dirac spread-out} conditional probability measure, even when the initial measure is concentrated on a function (See Main Text Fig.~1 (A) for an illustration). 

Given the empirical fact that $\mu_t^\Delta \to \mu_t$ as $\Delta \to 0$, we choose $\Delta$ sufficiently small and rely on a disintegration property similar to \eqref{eq:cond1} to realize the conditional probability measure $P^{\Delta}_t(du\bb\bar{u})$ with
\begin{equation}
    \label{eq:cond2}
    \mu^{\Delta}_t(du,d\bar{u}) = P^{\Delta}_t(du\bb\bar{u}) \; \bar{\mu}^{\Delta}(d\bar{u}).
\end{equation}
This brings us to the goal for our statistical computation of fluid flows:
\begin{center}
    \emph{For discretization parameter $\Delta \ll 1$ and given initial measure $\bar{\mu}^{\Delta}$, compute the conditional probability measure $P^{\Delta}_t(du\bb\bar{u})$ \eqref{eq:cond2} that characterizes uncertainty in a fluid flow.}
\end{center}
For notational simplicity, we fix $t$ and $\Delta$ and suppress the dependence on $t$ and $\Delta$ in \eqref{eq:cond2}. We also observe that all the measures in \eqref{eq:cond2} are supported in finite dimensions. 

We will usually represent $\bar{\mu}$ and $P(du|\bar{u})$ by their (generalized) densities $\bar{p}(\bar{u})$ and $p(u\bb\bar{u})$, respectively. Hence, as stated in the Main Text, the goal of statistical computation boils down to approximating the conditional distribution  $p(u\bb\bar{u})$, given the initial measure $\bar{p}(\bar{u})$ (Main Text Fig.~1 (A)).

\subsection{Score-based Diffusion Models}
As mentioned in the Main Text, we will adapt a specific generative AI algorithm, namely score-based diffusion models~\cite{ho2020denoising,vincent2011connection,song2019_score_based,song2020score} for computing the conditional distribution  $p(u\bb\bar{u})$, conditioned on the initial (\emph{prior}) measure $\bar{p}(\bar{u})$. We recall from the Main Text that diffusion models,~\cite{DMRev1} and references therein, learn probability distributions based on a very simple idea, realized in terms of a process with two steps (see Main Text Fig.~1 (B) for an illustration). In a forward step, \emph{noise} is iteratively added to samples drawn from the target distribution in order to transform it to a known noisy distribution, typically of the Gaussian type. The key \emph{reverse} step is based on \emph{denoising}. In it, noise is iteratively removed from samples drawn from the known noisy distribution and they are transformed into samples that follow the target distribution. Different diffusion models differ in how the denoising step is performed in practice. Here, we adapt the widely used \emph{score-based diffusion models},~\cite{karras2022elucidating} and references therein. 

\paragraph{Learning Unconditional Distributions with Score-based Diffusion Models.} For the ease of exposition, we will first consider the case of a target distribution $p\in {\rm Prob}(\R^N)$ that we wish to learn from data, the \emph{forward step} in a score-based diffusion model consists of adding noise to samples drawn from $p$ by solving the \emph{stochastic differential equation} (SDE)~\cite{karras2022elucidating}
\begin{equation}
    \label{eq:fsde}
    du_{\tau} = \frac{\dot{s}_\tau}{s_\tau} u_\tau \, d\tau + s_\tau \sqrt{2\dot{\sigma}_\tau\sigma_\tau} \, dW_\tau, \quad {\rm for}~\tau \in [0,K],
\end{equation}
with time index $\tau$, which stands for the time variable in the \emph{diffusion process} and is not related to the time $t$ used to express the physical time evolution in the PDE \eqref{eq:pde}. The drift and diffusion coefficients are given in terms of the \emph{shape function} $s_\tau$ and noise function $\sigma_{\tau}$, respectively, and $W_\tau$ is the $N$-dimensional standard Wiener process. The shape and noise functions $s_\tau,\sigma_\tau$ are chosen such that setting $s_0=1,\sigma_0=0$ results in aligning the marginal distribution $p_0=p$ with the target distribution $p$.

Solving the SDE \eqref{eq:fsde} forward in time $\tau$ results in the addition of noise to the samples $u_0 \sim p_0 = p$, transforming them to samples drawn from a so-called \emph{Gaussian Perturbation Kernel}
\begin{equation}
    \label{eq:iso_gauss}
    p_K(u_K) \sim \mathcal{N}(u_K; 0, s_K^2 \sigma_K^2 I),
\end{equation}
leading to a terminal distribution which is indistinguishable from an \emph{isotropic Gaussian with zero mean} at time~$\tau = K$. 

The \emph{reverse step} in a score-based diffusion model consists of solving the so-called \emph{reverse SDE}
\begin{equation}
    \label{eq:si_rsde}
    du_\tau = \left[  \frac{\dot{s}_\tau}{s_\tau}u_\tau - 2\dot{\sigma}_\tau\sigma_\tau s_\tau^2\nabla_u \log p_\tau(u_\tau) \right] d\tau +  s_\tau \sqrt{2\dot{\sigma}_\tau\sigma_\tau} d\widehat{W}_\tau 
\end{equation}
\emph{backward in time} with terminal distribution $p_K$ (as defined in \eqref{eq:iso_gauss}), while $p_\tau$ is the underlying distribution at any time $\tau \in [0,K]$. This reverse process yields the desired target distribution $p_0 = p$ as the initial distribution at $\tau = 0$. 

While the forward SDE \eqref{eq:fsde} is straightforward to simulate, once the so-called \emph{diffusion schedule} $(s_\tau,\sigma_\tau)$ is given, solving the reverse-SDE \eqref{eq:si_rsde} needs the (approximate) knowledge of the so-called \emph{score function} $\nabla_u \log p_\tau(u)$. The approximation of this score function lies at the heart of any diffusion model. 

For our work, we will adopt the widely-used framework of~\cite{karras2022elucidating} and approximate the score function in \eqref{eq:si_rsde} via a \emph{denoiser} $D_{\theta}(u + \epsilon_\tau, \sigma_\tau)$, which is a parametric function with parameters $\theta \in \Theta \subset \R^M$. Given a sample $u \sim p$, drawn from the target distribution $p$ and the given noise level $\epsilon_\tau = \epsilon \sigma_\tau$, for the noise function $\sigma_\tau$ and a parameter $\epsilon$, the parameters $\theta$ of the denoiser are learned by minimizing the error in predicting the underlying \emph{clean} sample $u$. The remarkable Tweedie's formula~\cite{DMRev1} then relates the score-function in \eqref{eq:si_rsde} as 
\begin{equation}
    \label{eq:tweddie}
    \nabla_u \log p_\tau(u_\tau) \approx \frac{D_{\theta}(\hat{u}_\tau, \sigma_\tau) - \hat{u}_\tau}{s_\tau \sigma_\tau^2}, \,\, \text{with } \hat{u}_{\tau} : = \frac{u_{\tau}}{s_\tau}, 
\end{equation}
enabling the solution of the reverse SDE \eqref{eq:si_rsde}. Thus, one needs to specify the diffusion schedule and the denoiser architecture in order to characterize a diffusion model. 
\paragraph{Learning Conditional Distributions.} Given that our goal of statistical computation of fluid flows entails computing conditional distributions $p(u\bb \bar{u})$, we need to adapt the score-based diffusion model presented above. To this end, we follow the approach of~\cite{batzolis2021conditional} and modify the denoiser in \eqref{eq:tweddie} to take the form 
\begin{equation}
    \label{eq:dnos}
    D_\theta(u_\tau, \sigma_\tau) \to D_\theta(u_\tau(\bar{u}), \bar{u}, \sigma_\tau),
\end{equation}
with noise $\sigma_\tau$ now added to samples $u(\bar{u})$ drawn from the underlying conditional distribution $p(u\bb\bar{u})$. Moreover, samples drawn from the prior distribution $\bar{u} \sim \bar{p}$ are explicit inputs to the denoiser in \eqref{eq:dnos}. Theorem 1 of~\cite{batzolis2021conditional} shows that Tweedie's formula \eqref{eq:tweddie} can be readily modified in this case to yield Formula (2) of the Main Text:
\begin{equation}
    \label{eq:si_tweedie1}
    \nabla_{u} \log p_\tau(u_\tau|\bar{u}) \approx \frac{D_\theta(u_\tau(\bar{u}), \bar{u}, \sigma_\tau) - \hat{u}_\tau}{s_\tau \sigma_\tau^2}.
\end{equation}
Samples from the target conditional distribution $p(u\bb \bar{u})$ are now drawn by simulating the reverse SDE, Equation (1) of the Main Text, with  the score function $\nabla_u \log p_\tau(u_\tau \bb \bar{u})$ being approximated by the denoiser \eqref{eq:si_tweedie1}. 

\paragraph{Learning Time-Conditioned Distributions.} Given the time-dependent nature of our underlying PDE \eqref{eq:pde}, we need to learn probability distributions $p(u\bb(t,\bar{u}))$, with $t \in [0,T]$ being the time and $\bar{u} \sim \bar{p}$, the (finite-dimensional approximation of) initial data. Hence, the denoiser \eqref{eq:dnos} has to be further modified to condition it on the time variable $t$ so that the entire trajectory of the distribution can be generated. Moreover, given the results of~\cite{herde2024poseidon} where a novel \emph{all-to-all} training procedure was proposed for learning solution operators of time-dependent PDEs, we will similarly exploit the \emph{semi-group} property of the solution operator of \eqref{eq:pde} to \emph{condition the denoiser on lead times}. To this end, we further modify the denoiser \eqref{eq:dnos} to
\begin{equation}
    \label{eq:dnos1}
     D_\theta(u_\tau(\bar{u}), \bar{u}, \sigma_\tau) \to D_\theta(t_n - t_\ell,u_\tau(t_n,\bar{u}), u(t_\ell,\bar{u}),\sigma_\tau),
\end{equation}
with times $0 \leq t_\ell \leq t_n \leq T$, initial data $\bar{u} \sim \bar{p}$ and $u(t_\ell,\bar{u})$ being the state of the system at a previous time step $t_\ell$ and noise $\sigma_\tau$ added to the current state $u(t_n,\bar{u})$ of the system at time $t_n$. Consequently, these intermediate conditional distributions can be chained together to learn the target conditional distribution by
\begin{equation}
    \label{eq:pdf-rollout}
    p(u\bb\bar{u}) = \prod_{\ell=1}^L p(u(t_\ell, \bar{u})\bb u(t_{\ell-1}, \bar{u})),
\end{equation}
with $0 \leq t_0 < t_1 < \ldots < t_\ell < \ldots < t_L=T$ being a set of monotonically increasing lead times and $ p(u\bb \bar{u})$ being the conditional distribution at the final time T, given the initial condition $\bar{u} \sim \bar{p}$.

\subsection{The Denoiser}
The main remaining step in specifying our conditional score-based diffusion model is to choose the architecture and training process for the denoiser in \eqref{eq:dnos1}. 

\subsubsection{The Denoiser Architecture} 
We choose a UViT~\cite{saharia2022photorealistic} as the model for our denoiser \eqref{eq:dnos1}, see Fig.~\ref{fig:s1} for a schematic. As seen in this figure, the model takes the lead time, the noisy sample at the current time, the underlying sample at a previous time, and the noise level to output a \emph{denoised} or clean sample. The inputs are lifted into a latent space and processed through a set of \emph{convolutional hidden layers}, which are stacked together in an \emph{encoder-decoder} form as suggested in the very popular U-Net architecture of~\cite{unet} in order to enable \emph{multi-scale} information processing. In contrast to the standard U-Net, UViT replaces a convolutional layer at the bottleneck (base of the U-Net) with a \emph{global attention layer} such that global mixing can take place in latent space. In three space dimensions, \emph{axial attention} blocks~\cite{ho2019axial} replace the global attention layer for computational efficiency. Residual skip connections are added to transfer information between the encoder and decoder at all hidden layers. Finally, the noise level and lead time are conditioned into the model at all levels by incorporating them inside the conditional layer norms of the model. All these steps are further detailed below. 

As illustrated in Fig.~\ref{fig:s1}, the operations of the denoiser start with the input data \(\bar{u} \in\R^{(v\times h\times w)}\) being projected into an embedding space of dimension \(C_0\) through a convolutional layer. The data is then sequentially downsampled \(n\) times, reducing its resolution in each dimension by a factor of \(2^n\). Each downsampling step is followed by a residual block composed of \(n_{res}\) layers, where each layer is structured as a sequence of a group normalization layer (denoted as GN in the figure), a non-linear activation function \(f\), and a convolutional layer \(\mathcal{C}\). This sequence is repeated twice. Meanwhile, the noise level \(\sigma_\tau\) and lead time \(t_n\) are embedded using two independent Fourier projections, concatenated, and processed by a multi-layer perceptron (MLP) comprised of two linear layers \(L\) and a non-linear activation function \(f\). The MLP generates scale \(a\) and shift \(b\) parameters, which are used to condition the second group normalization in the residual blocks. Following the residual blocks, multi-head attention mechanisms are applied in each layer. To address the computational cost of global attention in 3D data, axial attention is implemented instead
\cite{ax_attn}. 
On the other hand, data upsampling is performed through a \textit{depth-to-space} operation. In other words, a linear transformation is first performed to increase the number of channels by a factor of 4, transforming the input tensor of shape $(h,w, c)$ to $(h, w, 4c)$ and subsequently reshaped as $(c, 2h, 2w)$. Similarly to the downsampling stack, in the upsampling stack, each upsampling layer is followed by residual and attention blocks.
Skip connections are used between downsampling and upsampling stacks.

Below, we rigorously formulate the building pieces of the main blocks of the denoiser architecture. 

\paragraph{Affine Transformation.}
A linear layer in neural networks performs an affine transformation on the input data $x \in \R^{ h \times w \times c}$, defined as:
\begin{equation}
    L: \R^{c} \to \mathbb{R}^{\hat{c}}, \quad  L(x) = xW + b,
\end{equation}
where \( W \in \R^{c \times \hat{c}} \) is the weight matrix containing the learnable parameters, and \( b \in \R^{\hat{c}} \) is the bias vector, which also consists of learnable parameters. 

\paragraph{Convolution.} The discrete, multi-channel convolution with an $s$-stride of the input $x$ is defined as follows:
\begin{align}
    \mathcal{C}: \mathbb{R}^{h\times w \times c} \to \mathbb{R}^{\hat{h}\times \hat{w} \times \hat{c}}, \quad
    \mathcal{C}(x) = (x \star K_w)[i,j,\hat{l}] := \sum_{m,n=0}^{k-1} \sum_{l=1}^{c}  K_w[m,n,l,\hat{l}] \cdot x [is + m, js + n, l],
\end{align}
where $l$, $\hat{l}$ correspond to the indices of the input and output channels, respectively, and $i = 0, \dots, h-1$, $j = 0, \dots, w-1$, $\hat{l} = 1, \dots, \hat{c}$. 
Downsampling in the architecture in Fig.~\ref{fig:s1} is performed with a convolution operation with kernel size $3$ and stride $2$.

\paragraph{Group Normalization.} Group normalization (GN) is a technique used to normalize the features of an input tensor \( x \in \mathbb{R}^{h \times w \times c} \). Unlike batch normalization, which normalizes across the batch dimension, group normalization divides the channels into groups and normalizes the features within each group. Specifically, the channels are split into \( G \) groups, each containing \( \frac{c}{G} \) channels. For each group \( g \), the mean \(\mu_g\) and variance \(\sigma_g^2\) are computed as
\begin{equation}
    \mu_g = \frac{1}{m} \sum_{i \in g} x_i, \quad \sigma_g^2 = \frac{1}{m} \sum_{i \in g} (x_i - \mu_g)^2,
\end{equation}
respectively, where \( m = \frac{h w c}{G} \) is the number of elements in each group. The normalized output is then given by
\begin{equation}
    \hat{x}_i = \frac{x_i - \mu_g}{\sqrt{\sigma_g^2 + \epsilon}},
\end{equation}
where \( \epsilon \) is a small constant for numerical stability. Finally, a learnable scale \( \gamma \) and shift \( \beta \) are applied to each normalized group, yielding the final output
\begin{equation}
    \tilde{x}_i = \gamma \hat{x}_i + \beta.
\end{equation}

\paragraph{Fourier Embedding and Adaptive Scale.} Given an input tensor \( x \in \mathbb{R}^{h \times w \times 1} \), where the last dimension typically represents the temporal lead time $t$ or the diffusion noise $\sigma$, Fourier embeddings transform these coordinates into a higher-dimensional space as
\begin{equation}
    \gamma(x) = \left[ \sin(2\pi Bx), \cos(2\pi Bx) \right],
\end{equation}
where \( B \in \mathbb{R}^{d_f \times 2} \) is a matrix of frequencies, chosen from a fixed grid, and \( d_f \) is the dimensionality of the Fourier feature space. The resulting embedding \( \gamma(x) \) has a shape of \( h \times w \times 2d_f \), capturing both the sine and cosine components at multiple frequencies. Two independent Fourier embeddings are used for the lead time and diffusion noise. The embeddings \( \gamma_t(x) \) and \( \gamma_\sigma(x) \) are then concatenated together to
\begin{equation}
    \gamma(x)=\text{Concat}(\gamma_t(x), \gamma_\sigma(x)
\end{equation}
and transformed through an MLP $\mathcal{M}$, defined as 
\begin{equation}
    \mathcal{M}(\gamma(x)) = \text{GeLU}({\gamma(x)}{ W}_1 + {b}_1){W}_2 + {b}_2 ,
\end{equation}
with ${ W}_1 \in \R^{2d_f \times C_0 }$, ${b}_1 \in \R^{C_0}$, ${W}_2 \in \R^{C_0 \times 2C_0 }$, ${ b}_2 \in \R^{2C_0}$.
The output of the MLP is then split into a scale $a \in \R^{C_0}$ and shift $b \in \R^{C_0}$ used to adjust suitable group normalization in the residual blocks of the UVit as
\begin{equation}
    \tilde{x} = (a+1)GN(x) + b .
\end{equation}

\paragraph{Multi-Head Attention.}
Given an input tensor \( x \in \mathbb{R}^{h \times w \times c} \) (\( x \in \mathbb{R}^{h \times w \times d \times c} \) for 3D problems), following common practice in vision, the tensor is first reshaped into a new one of shape \( (hw, c) \) (\( (hwd, c) \) for 3D problems). To preserve spatial information, a positional embedding \( p \in \mathbb{R}^{hw \times c} \), learned during training, is added to the reshaped tensor, yielding \( x' = x + p \). 
The multi-head attention is then defined as
\begin{equation}
    \text{MHA}(x') = \text{Concat}(\mathcal{A}_1(x'), \dots, \mathcal{A}_h(x')) W^O,
\end{equation}
where \( W^O \in \mathbb{R}^{c \times c} \) is the output projection matrix, and \( \mathcal{A}_l(x') \), for \( l = 1, \dots, h \), represents the outputs of each attention head. Each head \( \mathcal{A}_k \) computes scaled dot-product attention as follows:
\begin{equation}
    \mathcal{A}_k(x') = \text{Attention}(Q_k, K_k, V_k) = \text{softmax}\left(\frac{Q_k K_k^\top}{\sqrt{d_k}}\right) V_k,
\end{equation}
where \( Q_k = x' W_k^Q \), \( K_k = x' W_k^K \), and \( V_k = x' W_k^V \) are the query, key, and value projections for the \( k \)-th head. \( d_k \) is the dimension of each head's subspace (typically \( d_k = \frac{c}{h} \) for \( h \) heads), and \( W_k^{Q,K,V} \in \mathbb{R}^{d_k \times d_k} \) are the respective projection matrices.
The attention block in the UVit architecture is defined as
\begin{equation}
    \tilde{x}= x + \text{MHA}(\text{GN}(x)) , 
\end{equation}
where $\text{GN}$ is a group norm with 32 groups.

\paragraph{Axial Attention.} The memory requirements of standard attention scale quadratically with the sequence length. In the case of multidimensional problems, it becomes prohibitively expensive also if performed at the bottleneck of a U-structure like the ones used in UVit. To circumvent this problem, axial attention has been proposed~\cite{ax_attn}.
Axial attention performs attention sequentially over each axis $i$ of the tensor $x$, mixing information only along the axis $i$. It is implemented  by simply transposing all axes
except $i$ to the batch axis, and performing $\text{MHA}$ as described above. In the case of 3D problems, given the input tensor $x \in \mathbb{R}^{h \times w \times d \times c}$, the axial attention block along the axis $i$ can be defined as
\begin{equation}
    \tilde{x} = x + \text{GN}_{2} (\text{MHAA}_i(\text{GN}_{1} (x))), 
\end{equation}
where $\text{MHAA}_i$ denotes multi-head axial attention along the $i-th$ axis, i.e.\
\begin{equation}
    \text{MHAA}_i(x) = \text{MHA}(\text{Transpose}_i(x)) . 
\end{equation}
Here, $\text{Transpose}_i$ is the operation transposing all axes except $i$ to the batch axis. 
The block is repeated for each axis.

\subsubsection{Denoiser Training} 
The parameters $\theta \in \Theta \subset \R^M$ that define the denoiser \eqref{eq:dnos1} need to be determined from training data. To this end, we consider training data in the form of \emph{trajectories} $S_i = \{u(t_\ell, \bar{u}^i)\}_{\ell=1}^L$, for $1 \leq i \leq I$, with all $\bar{u}^i \sim \bar{p}$ and $0 = t_0 < t_1 < \ldots < t_\ell < \ldots < t_L = T$. Then, the denoiser parameters are given as the (local) minimizers for $\theta \in \Theta$ for the \emph{denoising loss function}:
\begin{equation}
    \label{eq:dnostr}
    L(D_\theta,\sigma) = \mathbb{E}_{\bar{u}^i} \mathbb{E}_{(u(t_\ell, \bar{u}^i), u(t_n, \bar{u}^i)) } \mathbb{E}_{\eta} \left\| D_{\theta}(t_n - t_\ell,u(t_n, \bar{u}^i) + \eta, u(t_\ell, \bar{u}^i), \sigma) - u(t_n, \bar{u}^i) \right\|^2 .
\end{equation}
Here, $\bar{u}^i \sim \bar{p}$, $(u(t_\ell, \bar{u}^i), u(t_n, \bar{u}^i)) \sim S_i$, $\eta \sim \mathcal{N}(0, \sigma^2 I)$,  $t_n > t_\ell$, and $0 < t_n \leq T$. Thus, for any noise level $\sigma$, the denoiser is trained in order to remove the noise from a noisy sample and output the \emph{clean} sample (see Main Text Fig.~1 (C) for an illustration). 

Our training of the denoiser-based diffusion model largely follows the methodology proposed in~\cite{karras2022elucidating}. For ease of notation, in the rest we will set $\sigma_\tau=\sigma$, $u_\tau=u$, and focus without loss of generality and for the ease of presentation on the unconditional case. In this case, the loss function for training the denoiser $D_{\theta}$ simplifies to 
\begin{equation}
    \mathcal{L}(D_\theta,\sigma) = \mathbb{E}_{u \sim p_{\text{data}}} \mathbb{E}_{\eta \sim \mathcal{N}(0,\sigma^2 I)} \| D_{\theta}(u+\eta, \sigma)-u\|^2
\end{equation}
and
\begin{equation}
    \label{eq:loss1}
    \mathcal{L}(D_\theta) = \mathbb{E}_{\sigma \sim p_{\text{train}}} \left[\lambda(\sigma) L(D_\theta,\sigma)\right], 
\end{equation} 
where $\lambda$ is a weight dependent on the noise level $\sigma$ and $p_{train}=p_{\text{train}}(\ln{\sigma})=\mathcal{U}(\ln{(\sigma_{\text{min}}), \ln{(\sigma_{\text{max}})}})$, $\sigma_{\text{min}} = 10^{-3}, \sigma_{\text{max}}=80$.

Moreover, in score-based diffusion models, it is common~\cite{karras2022elucidating} to precondition the denoiser predictions as
\begin{equation}
    D_{\theta}(u+\eta, \sigma) = c_{\text{skip}}(\theta)(u+\eta) + c_{\text{out}}(\theta) F_{\theta}\left(c_{in}(\sigma)(u+\eta); c_{\text{noise}}(\sigma)\right)
\end{equation}
where $F_{\theta}$ is the raw U-Net model.
Upon defining
\begin{equation}
    F_{\text{target}} = \frac{1}{c_{\text{out}}(\sigma)} (u - c_{\text{skip}}(\sigma)(u+\eta)) , 
\end{equation}
the loss function can be rewritten as
\begin{equation}
    \mathcal{L}(D_\theta) = \mathbb{E}_{\sigma, u, n} \left[w(\sigma) \| F_{\theta}\left(c_{\text{in}}(\sigma)(u+\eta), c_{\text{noise}}(\sigma)\right) - F_{\text{target}} \|^2_2\right] . 
\end{equation}
Here, $c_{\text{noise}}$ is chosen to be $c_{\text{noise}}(\sigma)= \frac{1}{4} \log(\sigma)$. 
On the other hand, $c_{\text{in}}$, $c_{\text{out}}$, $c_{\text{skip}}$ are derived by imposing the following requirements:
\begin{align}
    \text{Var}[c_{\text{in}}(\sigma)(u+\eta)] = 1 \quad 
    & \Rightarrow 
    \quad c_{\text{in}}(\sigma) = \frac{1}{\sqrt{\sigma^2 + \sigma_{\text{data}}^2}} , \\
    \text{Var}[F_{\text{target}}] = 1 \quad 
    & \Rightarrow 
    \quad c_{\text{out}}(\sigma)  = \sigma \cdot \sqrt{\sigma^2 + \sigma_{\text{data}}^2} , \\
    c_{\text{skip}}(\sigma) = \text{argmin}_{c_{\text{skip}}^2} c_{\text{skip}}(\sigma) \quad 
    & \Rightarrow  \quad c_{\text{skip}}(\sigma) = \frac{\sigma_{\text{data}}}{\sigma^2 + \sigma_{\text{data}}^2} .
\end{align}
Finally, the weight $\lambda$ is obtained by requiring the effective weight $w(\sigma)$ to be uniform across noise levels, i.e.\
\begin{equation}
    w(\sigma) = 1 \quad \Rightarrow \quad \lambda(\sigma)=\frac{\sigma^2 + \sigma_{\text{data}}^2}{(\sigma\sigma_{\text{data}})^2} .
\end{equation}
Details can be found in Appendix B.6 of~\cite{karras2022elucidating}.
In all the experiments addressed in this paper, $\sigma_{data}=0.5$.

\paragraph{Variance Capturing Loss.}
In order to improve the approximation of the standard deviation of the generated samples, the following term can be added to the loss function:
\begin{equation}
    \label{eq:square-loss}
    \mathcal{L}_{sq}(D_\theta) = \mathbb{E}_{\sigma, u, \eta} \left[w(\sigma) \| D_{\theta}(u+\eta, \sigma)^2 - u^2 \|^2_2\right] .
\end{equation}
This is motivated by the fact that the variance of a random variable $x$ is defined as
\begin{equation}
    \mathbb{V}[x] = \mathbb{E}[x^2] -  \mathbb{E}[x]^2.
\end{equation}
If  we assume that the denoiser prediction $\tilde{u}= D_{\theta}(u+\eta, \sigma)$ is a Gaussian perturbation of the noise-free data $u$, i.e.\ $\tilde{u}= u +\psi$ (for instance, for large $\sigma$, the denoiser predictions are still affected by noise), then 
\begin{equation}
    \mathbb{V}[\Tilde{u}]=\mathbb{V}[u+\psi] = \mathbb{E}[u^2] + \mathbb{E}[\psi^2]  + 2\mathbb{E}[u\psi] -  \mathbb{E}[u]^2
\end{equation}
and
\begin{equation}
    \mathbb{V}[\Tilde{u}] - \mathbb{V}[u] = \mathbb{E}[\psi^2]  + 2\mathbb{E}[u\psi] . 
\end{equation}
Therefore, the error in the variance learned samples can be reduced by minimizing \eqref{eq:square-loss}, weighted by a factor $\lambda_{sq}$, together with \eqref{eq:loss1}, i.e.\
\begin{equation}
    L(D_\theta) = \mathcal{L}(D_\theta) + \lambda_{sq}\mathcal{L}_{sq}(D_\theta) .
\end{equation}

\subsubsection{Inference of the Diffusion Model}
At the time of inference, samples can be generated by solving the following SDE:
\begin{equation}
    \label{eq:sde}
    du = 2\left( \frac{\dot{\sigma}}{\sigma} + \frac{\dot{s}}{s}\right)d\tau - 2s\frac{\dot{\sigma}}{\sigma}D_\theta\left(\frac{u}{s}, \sigma\right)d\tau + s\sqrt{2\dot{\sigma}\sigma}dw .
\end{equation}
Details can be found in~\cite{karras2022elucidating}, Appendix B.
More specifically, we use the Euler--Maruyama method, where the time steps are defined according to a sequence of noise levels $\{\sigma_i\}$, $\tau_i = \sigma^{-1}(\sigma_i)$, $i=1,...,N$, $N=128$ 
and 
\begin{equation}
    \sigma_i = \left( \sigma_{\text{max}}^{\frac{1}{\rho}} + \frac{i}{N-1}(\sigma_{\text{min}}^{\frac{1}{\rho}} - \sigma_{\text{max}}^{\frac{1}{\rho}}) \right)^{\rho} . 
\end{equation}
At this point, it only remains to define the $\sigma$ and $s$ scheduling.

\paragraph{Sigma Scheduler.}
Two most common choices of the noise scheduler \(\sigma\) are
\begin{itemize}
    \item exponential: $\sigma_\tau = \mathcal{O}(\exp{(\tau)})$ , 
    \item tangent: $\sigma_\tau = \mathcal{O}(\tanh{(\tau)})$ . 
\end{itemize}
In all experiments, the exponential noise scheduler is used, as it was empirically determined to be a more effective choice for our models.

\paragraph{Diffusion Scheme.}
We employ the variance-exploding (VE) schedule, which sets the forward scheduler $s_\tau = 1$, $\forall t$.

\subsection{Baselines}
In addition to our conditional score-based diffusion model, GenCFD, we consider the following ML baselines for all our experiments. Note that all these baselines are trained to minimize the mismatch between their predictions and the ground truth in the mean square or $L^2$-norm. 

\subsection{FNO}

A \textit{Fourier neural operator} ({FNO}) $\cG$~\cite{FNO} is a composition
\begin{equation}
    \label{eq:fno}
    \cG: \bX \to \bY, \quad 
    \cG = Q \circ \cL_L \circ \dots \circ \cL_1 \circ R.
\end{equation}
It has a \emph{lifting operator} ${u}(x) \mapsto R (u(x),x)$, where $R$ is represented by a linear function $R: \R^{d_u} \to \R^{d_v}$ where $d_u$ is the number of components of the input function and $d_v$ is the \emph{lifting dimension}. The operator $Q$ is a non-linear projection, instantiated by a shallow neural network with a single hidden layer, $128$ neurons and $\text{GeLU}$ activation function, such that $v^{L+1}(x) \mapsto \cG( u )(x) = Q \left( v^{L+1}(x)\right)$. 

Each \emph{hidden layer} $\cL_\ell: v^\ell(x) \mapsto v^{\ell+1}(x)$ is of the form 
\begin{equation}
    v^{\ell+1}(x)  = \sigma\left(
    W_\ell \cdot v^{\ell}(x) + \left( K_\ell v^{\ell} \right) (x) 
    \right),
\end{equation}
with $W_\ell \in \R^{d_v\times d_v}$ a trainable weight matrix (residual connection), $\sigma$ an activation function, corresponding to $\text{GeLU}$, and the \emph{non-local Fourier layer} 
\begin{equation}
    K_\ell v^\ell = \cF_N^{-1} \left(P_\ell(k) \cdot \cF_N v^\ell(k) \right),
\end{equation}
where $\cF_N v^\ell (k)$ denotes the (truncated)-Fourier coefficients of the discrete Fourier transform (DFT) of $v^\ell(x)$, computed based on the given $s$ grid values in each direction. The maximal number of modes is set to $M$. Note that $P_\ell(k) \in \C^{d_v \times d_v}$ is a complex Fourier multiplication matrix indexed by $k\in \Z^d$, and $\cF_N^{-1}$ denotes the inverse DFT.

For time-dependent problems, the lead time is conditioned into the FNO model at all levels by incorporating it inside the conditional instance norms of the model. Those normalization layers are applied before \textit{each} Fourier layer.

\subsubsection{C-FNO} 

We significantly enhanced the performance of the FNO model by incorporating \textit{local convolution operations} in addition to the global convolutions performed in the original Fourier layer. Let $\mathcal{C}_{m, d_v}$ be a discrete, local convolution operator applied with a kernel of size $d_v \times d_v \times m^d$, acting on the space of functions of $v^\ell$. The \textit{modified} Fourier layer is then defined by
\begin{align} 
    \Tilde{v}^{\ell}(x) &=  \mathcal{C}_{3, d_v} \circ \sigma \circ  \mathcal{C}_{5, d_v} \circ v^\ell(x) , \\ 
    v^{\ell+1}(x)  &= \sigma\left(
    W_\ell \cdot \Tilde{v}^{\ell}(x) + \left( K_\ell \Tilde{v}^{\ell} \right) (x) 
    \right).
\end{align}
Thus, we alternate between global Fourier layers and local convolutions in this architecture. 

\subsubsection{UViT}

For the UViT baseline, we use the same architecture as the UViT in the diffusion model. However, the learning objective differs since the baseline is deterministic. As a result, training this baseline is significantly different from training the diffusion model. Thus, the use of the UViT model can be considered as an \emph{ablation} for the role of the diffusion training objective \eqref{eq:dnostr} in the performance of GenCFD.

\subsection{Training and Test Protocols}
GenCFD and the above-mentioned baselines are trained on all the datasets by sampling data $\bar{u} \sim \mu_0$, where the underlying probability distributions for each dataset are described below. 

However, as mentioned in the Main Text, at \emph{test time}, our goal is the statistical computation of fluid flow \eqref{eq:pde} with \emph{Dirac} input distributions, i.e., the inputs (for simplicity, initial conditions) are given by
\begin{equation}
    \label{eq:dirac}
    \hat{\mu}_0 = \delta_{\hat{u}},
\end{equation}
which implies that $\bar{\mu}^{\Delta} = \delta_{\hat{u}^\Delta}$ in \eqref{eq:cond1}, with $\hat{u}^{\Delta} \approx \hat{u}$. Thus, we test the ability of our algorithm and the baselines to compute statistical solutions in the interesting case where the input is a Dirac and it is only the chaotic evolution of the flow that makes the conditional measure \eqref{eq:cond1} \emph{spread out}. To generate the ground truth with such Dirac initial conditions, we use an \emph{ensemble perturbation} approach outlined below. 

\paragraph{Ensemble Perturbation Approach for Ground Truth Generation.}
Given a distribution of inputs (say initial conditions) $\bar{u} \sim \bar{\mu}$, we approximate the ground truth distribution at time $T = t$ by the empirical measure $\hat{\mu}_t^{\Delta}$ induced by Monte Carlo samples on $\bar{\mu}^{\Delta}$ which are propagated in time by a classical numerical method (See Main Text Fig.~1 (A) for an illustration). This strategy is used to approximate the conditional probability measure $P_t^{\Delta}(du|\bar{u})$ as follows:
\begin{equation}
    P_t^{\Delta}(du|\bar{u}) = \hat{\mu}_{\bar{u},t}^{\Delta} \;\mbox{ with }\; \bar{\mu}_{\bar{u}}^{\Delta} \approx \bar{\mu}_{\bar{u}} = \delta_{\bar{u}}.
\end{equation}
Due to the deterministic nature of the classical simulation, the initial distribution $\bar{\mu}_{\bar{u}} = \delta_{\bar{u}}$ also needs to be approximated by another distribution $\bar{\mu}_{\bar{u}}^{\Delta} \approx \bar{\mu}_{\bar{u}}$ with $B_{\frac{\varepsilon}{2}}(\bar{u}) \subseteq \supp\bar{\mu}_{\bar{u}}^{\Delta} \subseteq B_{\varepsilon}(\bar{u})$ for some $\varepsilon > 0$. This slight modification causes trajectories to diverge and approximate the underlying distribution. Note that this approximation can be made as accurate as desired by choosing a small enough value for $\varepsilon$. In general, we take $\varepsilon \sim \Delta$ for the experiments performed in this paper.

Note that we now have two nested probability distributions. Firstly, the distribution of initial conditions $\bar{\mu}$, and secondly the approximations $\bar{\mu}_{\bar{u}}^{\Delta}$ of $\delta_{\bar{u}}$ for each concrete initial condition $\bar{u}$. The sampling is therefore also realized by two stages of Monte Carlo sampling. Note that $|\supp\bar{\mu}| \gg |\supp\bar{\mu}_{\bar{u}}^{\Delta}|$. This motivates the terminology of samples of $\bar{\mu}$ being called \emph{macro}-samples, while samples of $\bar{\mu}_{\bar{u}}^{\Delta}$ approximating the Dirac-delta distribution are called \emph{micro}-samples.

To summarize, we are interested in Monte Carlo samples of $P_t(du|\bar{u})$ where $\bar{u} \sim \bar{\mu}$. The algorithm to draw them is as follows:
\begin{enumerate}
    \item Draw \emph{macro} samples $\bar{u}^1, \bar{u}^2, \dots, \bar{u}^{M_{\text{macro}}} \sim \bar{\mu}$, 
    \item For each macro sample $\bar{u}^i$ draw \emph{micro} samples $\bar{u}_{\bar{u}^i}^1, \bar{u}_{\bar{u}^i}^2, \dots, \bar{u}_{\bar{u}^i}^{M_{\text{micro}}} \sim \bar{\mu}_{\bar{u}^i}^{\Delta}$, 
    \item Evolve each sample $\bar{u}_{\bar{u}^i}^j$ in time to get $u_{\bar{u}^i}^j(t)$  ,
    \item Approximate each $P_t(du|\bar{u}=\bar{u}^i)$ by
    \begin{equation}
        P_t(du|\bar{u}=\bar{u}^i) \approx \frac{1}{M_{\text{micro}}}\sum_{j=1}^{M_{\text{micro}}}\delta_{u_{\bar{u}^i}^j(t)}.
    \end{equation}
\end{enumerate}
This is the strategy that generates the ground truth distribution, with respect to which we test GenCFD and the baselines. 

\subsection{Datasets}\label{sec:datasets}
In the Main Text, we have presented results with five challenging three-dimensional flow datasets, which we describe in detail below. Datasets are available in a Google cloud storage bucket \url{gs://gencfd} under a CC BY 4.0 license.

\subsubsection{Taylor--Green Vortex (TG)}

We simulate the well-known Taylor--Green vortex~\cite{TaylorGreen1937} for the incompressible Navier--Stokes equations \eqref{eq:NS} in a probabilistic setting by adding small perturbations to the velocity field. The initial conditions are given by
\begin{equation}
    \begin{aligned}
        \bar{u}_x(x, y, z) &= A\cos(1\pi x)\sin(2\pi y)\sin(2\pi z) + \varepsilon_x(x,y,z), \\
        \bar{u}_y(x, y, z) &= B\sin(1\pi x)\cos(2\pi y)\sin(2\pi z) + \varepsilon_y(x,y,z), \\
        \bar{u}_z(x, y, z) &= C\sin(1\pi x)\sin(2\pi y)\cos(2\pi z) + \varepsilon_z(x,y,z),
    \end{aligned}
\end{equation}
where we choose $A = 1$, $B = -1$, $C = 0$ to fulfill the incompressibility constraint. The perturbations $\varepsilon_x$, $\varepsilon_y$, and $\varepsilon_z$ are defined to be
\begin{equation}
    \varepsilon_d(x, y, z) = \frac{1}{8} \sum_{(i,j,k) \in \{0,1\}^3} \delta_{d,i,j,k} \alpha_i(4\pi x) \alpha_j(4\pi y) \alpha_k(4\pi z),\quad \mbox{ where } \alpha_i(x) = \begin{cases}
        \sin(x) &\mbox{ if } i = 0, \\
        \cos(x) &\mbox{ if } i = 1,
    \end{cases}
\end{equation}
where $\delta_{d,i,j,k} \sim \cU_{[-0.025, 0.025]}$.

For simulating a Dirac distributed initial condition, we fix the values $\delta_{d,i,j,k}$ and add a second perturbation of a similar form as $\varepsilon_d$ to the flow field. The difference being that the frequency of the Fourier modes is doubled, and their amplitudes are chosen proportional to the mesh size.

Although the initial datum only contains a few large frequencies, the turbulent cascade into progressively smaller scales leads to the dynamic generation of higher frequencies, see Main Text Fig.~2 (A and B) for illustrations of the pointwise kinetic energy and vorticity intensity. Thus, the solution transitions from a laminar to a turbulent regime with time. We simulate this experiment with a spectral viscosity method implemented within the publicly available, state-of-the-art, highly optimized GPU-based {\bf Azeban} code~\cite{rohner2024efficient} at a spatial resolution of $128\times 128 \times 128$.  

Here, the underlying solution operator maps the initial data to (trajectories of) the solution at later times. In the experiments conducted, the total time duration was scaled to \( T = 2.0 \). We test the solution at \( T_{\text{test}} = 0.8 \) and \( T_{\text{test}} = 2.0 \). The all-to-all training was performed using snapshots corresponding to the time points $\{0, 0.4, 0.8, 1.2, 1.6, 2.0 \}$. There are $15$ input-output pairs per trajectory in the training set. 

\subsubsection{Cylindrical Shear Flow (CSF)}

The cylindrical shear flow for the incompressible Navier--Stokes equations~\eqref{eq:conslaw} is heavily inspired by the flat vortex sheet experiment in~\cite{lanthaler2021statistical} and is introduced as a 3D equivalent to the latter~\cite{rohner2024efficient}. The initial conditions are given by
\begin{equation}
    \begin{aligned}
        \bar{u}_x(x,y,z) &= \tanh\left(2\pi \frac{r - 0.25}{\rho}\right), \\
        \bar{u}_y(x,y,z) &= 0, \\
        \bar{u}_z(x,y,z) &= 0,
    \end{aligned}
\end{equation}
where $r^2 = (y - 0.5 + \sigma_y^j(x))^2 + (z - 0.5 + \sigma_z^j(x))^2$ and $\rho$ is the smoothness parameter. We define the perturbations $\sigma_y^j(x)$ and $\sigma_z^j(x)$ in the following way: Let $\alpha_k^y$ and $\alpha_k^z$ be i.i.d.\ uniformly distributed on $[0, 1]$ and let $\beta_k^y$ and $\beta_k^z$ be i.i.d.\ uniformly distributed on $[0, 2\pi]$. Then $\sigma_y^j(x)$ and $\sigma_z^j$ are given by
\begin{equation}
    \begin{aligned}
        \sigma_y^j(x) &= \delta \sum_{k=1}^{p} \alpha_k^y \sin(2\pi k x - \beta_k^y), \\
        \sigma_z^j(x) &= \delta \sum_{k=1}^{p} \alpha_k^z \sin(2\pi k x - \beta_k^z).
    \end{aligned}
\end{equation}
These initial conditions are well defined in the limit $\rho \rightarrow 0$ where the interface between the flow directions becomes discontinuous and are then equal to
\begin{equation}
    \begin{aligned}
        \bar{u}_x(x,y,z) &= \begin{cases}
            -1 & \text{for } r \leq 0.25, \\
            1 & \text{otherwise}, 
        \end{cases} \\
        \bar{u}_y(x,y,z) &= 0, \\
        \bar{u}_z(x,y,z) &= 0, 
    \end{aligned}
\end{equation}
where $r$ is defined as above.

For simulating a Dirac distributed initial condition, we fix the values of $\alpha_k^y$, $\alpha_k^z$, $\beta_k^y$, and $\beta_k^z$. Then we extend the perturbation by an additional three modes with amplitude proportional to the mesh size.

We choose $p = 10$ and $\delta = 0$ as our default configuration. The initial shear flow is then evolved by numerically solving the three-dimensional incompressible Navier--Stokes equations with the spectral viscosity based {\bf Azeban} code~\cite{rohner2024efficient}. The resulting approximate solutions follow a very complex temporal evolution and contain a large range of small-scale eddies as seen in Main Text Fig.~1 (D). The dataset is also generated at the resolution of $128^3$. 

Again, the underlying solution operator maps the initial velocity field to the velocity field at later times. In the experiments conducted, the total time duration was scaled to \( T = 1.0 \). The all-to-all training was performed using snapshots corresponding to the time points $\{0, 0.25, 0.5, 0.75, 1.0 \}$. There are $10$ input-output pairs per trajectory in the training set.

\subsubsection{Cloud-Shock Interaction (CSI)}

This dataset is the three-dimensional version of the well-known shock-bubble test case for compressible fluid flows~\cite{LEV1,MSS1,KLye1} where a supersonic shock wave hits a high-density cloud (bubble) and leads to the excitation of shock waves while creating a zone of turbulent mixing in their wake.

For this experiment, we subdivide the domain $[0, 1]^3$ into three subdomains. The initial condition is then initialized as a constant on each of these subdomains. We define
\begin{equation}
    \begin{aligned}
        S &= \{ (x, y) \in [0,1]^3 \mid x \leq 0.05+\sigma_x(y) \},  \\
        C &= \{ (x, y) \in [0,1]^3 \mid r \leq 0.13 + \sigma_r(\atantwo(y-0.5,z-0.5)) \},  \\
        E &= [0,1]^3 \setminus (S \cup C), 
    \end{aligned}
\end{equation}
where $r^2 = (x-0.25)^2 + (y-0.5)^2 + (z-0.5)^2$ and $\sigma_x$, and $\sigma_r$ are defined as
\begin{equation}
    \begin{aligned}
        \sigma_x(y) &= \frac{\delta}{\sum_{k=1}^p \alpha^x_k} \sum_{k=1}^p \alpha^x_k\cos(2\pi k (y + \beta^x_k)),  \\
        \sigma_r(\varphi) &= \frac{\delta}{\sum_{\k=1}^p \alpha^r_k} \sum_{k=1}^p \alpha^r_k\cos(2\pi k (\varphi + \beta^r_k)).
    \end{aligned}
\end{equation}
The parameters $\alpha^x_k$, $\alpha^r_k$, $\beta^x_k$, and $\beta^r_k$ are all uniformly distributed in $[0, 1]$. Furthermore, for the in-distribution data, we use $\delta = 0.06$ and $p = 10$. Combining all of this, the initial conditions for this experiment are given by
\begin{equation}
    (\bar{\rho}, \bar{u}_x, \bar{u}_y, \bar{p}) = \begin{cases}
        (0.386859, 11.2536, 0, 167.345) &\mbox{ if } (x, y, z) \in S , \\
        (10, 0, 0, 1) &\mbox{ if } (x, y, z) \in C ,\\
        (1, 0, 0, 1) &\mbox{ if } (x, y, z) \in E .
    \end{cases}
\end{equation}
For generating Dirac distributed initial conditions, the values of $\alpha_k^x$, $\alpha_k^r$, $\beta_k^x$, and $\beta_k^r$ are fixed. Then the perturbations are extended by the next three higher modes with their amplitudes set to be proportional to the mesh size.

The dataset is simulated with the high-resolution finite volume GPU-optimized {\bf Alsvinn} code of~\cite{KLye1} at a resolution of $256^3$.

For this dataset, the solution operator maps the initial conditions (density, momenta and energy) to the solution at later times. In the experiments we conducted, the total time duration was scaled to \( T = 1.0 \). The all-to-all training was performed using snapshots corresponding to the time points $\{0, 0.25, 0.5, 0.75, 1.0 \}$. There are $10$ input-output pairs per trajectory in the training set. 

\subsubsection{Nozzle Flow (NF)} 

This dataset is generated from a three-dimensional fluid flow through a nozzle geometry of diameter \(2l_{\text{c}}\) and length \(4l_{\text{c}}\) into a larger domain of diameter \(11l_{\text{c}}\) and length \(56l_{\text{c}}\) that is filled with the same medium. 
The injection with a randomized inflow velocity profile of maximal magnitude \(u_{\text{c}}\) generates a wall-bounded turbulent pipe flow at \(Re= 10000\) up to \(Re= 20000\) inside the nozzle, where \(Re\coloneqq u_{\text{c}}l_{\text{c}}/\nu_{\text{c}}\), \(u_{\text{c}}=1~\text{m}/\text{s}\) to \(u_{\text{c}}=2~\text{m}/\text{s}\), \(l_{\text{c}}=1~\text{m}\), and \(\nu_{\text{c}} = 1.0\times 10^{-4}~\text{m}^{2}/\text{s}\). 
Consequently, in the larger domain a turbulent round jet is formed, see Main Text Fig.~3 (A). 
Turbulent jet flows appear in many engineering applications such as fluid mixing, combustion or acoustic control. 
The vortex generation method in the computational setup from \cite{kummerlander2024optimization} is modified such that random perturbations 
\begin{equation}\label{eq:nozzleVortU}
    \begin{aligned}
        u_{\text{vort}}(x,y,z,t) 
        = 
        \text{frac}(t) \frac{1}{2\pi} \sum_{i=1}^{n_{\text{vort}}} 
            \Gamma_{i} 
            \frac{\left(1 - \text{exp} \left( -\frac{\left| [x,y,z] - p_{i}\right|^{2}}{2\sigma^{2}}\right) \right)}{|[x,y,z] - p_{i}|^{2}}   
            \left( \left( p_{i} - [x,y,z] \right) \times d_{\text{in}} \right)  
    \end{aligned}
\end{equation}
are added to the inflow mean velocity
\begin{equation}\label{eq:nozzleMeanU}
	    u_{\text{mean}}(x,y,z,t) = \text{frac}(t) [2.77 u_\text{c} \left( 1 - \left(\frac{1}{l_{\text{c}}}\left( \left(x-5.5l_{\text{c}}\right)^2 + \left( z - 5.5l_{\text{c}}\right)^{2}\right)^{\frac{1}{2}}\right)^{\frac{9}{8}}\right) , 0 , 0] , 
\end{equation}
where
\begin{equation}
    \Gamma_{i} = 4 s_{i} \left( \frac{\pi}{6 \ln (3) - 9 \ln(2)} \frac{k_i A_{\text{in}}}{n_{\text{vort}}} \right)^{\frac{1}{2}} 
\end{equation}
and
\begin{equation}   
	k_i = \frac{3}{2} \left( | u_{\text{mean}}(p_i)| I_{\text{turb}} \right)^2  
\end{equation}
denote the circulations per vortex, and the turbulent kinetic energies, respectively, for \(i=1,2,\ldots, n_{\text{vort}}\), and \(\text{frac}(t) \in [0,1]\) is polynomially increased to unity until \(t=0.5~\text{s}\). 
Here, \(n_{\text{vort}}=50\) is the number of vortices of size \(\sigma=0.1 l_{\text{c}}\) located at inlet positions
\begin{equation}
    [0,r_{i},\theta_{i}] = [0,l_{\text{c}} \left(\gamma^{(r)}_{i}\right)^{\frac{1}{2}}, 2 \pi \gamma^{(\theta)}_{i}]
\end{equation} 
with the respective Cartesian locations \(p_i\in \mathbb{R}^{3}\) and signs \(s_{i}\). 
The random parameters \(\gamma^{(r)}_{i}\) and \(\gamma^{(\theta)}_{i}\) are i.i.d.\ uniformly in \([0,1)\), and \(s_{i}\) are i.i.d.\ discrete uniformly on \(\{-1,1\}\). 
Moreover, \(I_{\text{turb}}=0.05\) is the turbulence intensity, \(A_{\text{in}}\) is the inlet area of the nozzle, and \(d_{\text{in}} \in \mathbb{R}^{d}\) is the normalized inflow direction.  
The mean velocity inlet profile \eqref{eq:nozzleMeanU} as well as the perturbation \eqref{eq:nozzleVortU} are Langevin-combined \cite{kummerlander2024optimization} to synthetically produce a turbulent inflow velocity 
\begin{equation}\label{eq:nozzleInU}
    \begin{aligned}
	   u_{\text{in}}(x,y,z,t_{n}) \coloneqq u_{\text{mean}} (x,y,z,t_{n}) + u_{\text{vort}} (x,y,z,t_{n}) - \frac{ u_{\text{vort}} (x,y,z,t_n) \cdot \nabla u_{\text{in}}(x,y,z,t_{n-1})}{\left|\nabla u_{\text{in}}(x,y,z,t_{n-1})\right|}d_{\text{in}} 
    \end{aligned}
\end{equation}
at a discrete timestep \(t_n \in N\), where \(u_{\text{in}}(x,y,z,0) = [0,0,0,0]\). 
We use a lattice Boltzmann method with an LES model to approximate the initial boundary value problem in the nozzle geometry that is based on the weakly compressible Navier--Stokes equations at a Mach number of \(Ma=4.7\times 10^{-3}\) with the velocity inlet condition \eqref{eq:nozzleInU}, a constant pressure boundary at the outlet, and no-slip boundary conditions at the remaining cylinder walls. 
We simulate this experiment with the open-source highly parallel C++ library \textbf{OpenLB} \cite{openlb,kummerlander2024olb17} that scales efficiently on hundreds of GPGPU nodes \cite{kummerlander2022advances}. 
For the NF experiment, the domain is discretized with \(10.19\times 10^{6}\) grid points at a resolution of \(124 \times 124\times 663\) and we compute a time horizon until \(t=130~\text{s}\) with a step size \(\Delta t = 2.479\cdot 10^{-4}\). 

Consequently, in this dataset, the solution operator maps the initial conditions and the boundary conditions (inflow velocity) to the velocity field at later times. In the experiments we conducted, the time horizon was rescaled and non-dimensionalized to \( T = 1.3 \), with the testing time set to \( T_{\text{test}} = 1.0 \). Note that the all-to-all training was performed using snapshots corresponding to the time points \( \{0, 0.2, 0.4, 0.6, 0.8, 1.0, 1.2\} \). Thus, there are $21$ input-output pairs per trajectory in the training set.

\subsubsection{Dry Convective Planetary Boundary Layer (CBL)}

This test case describes the growth of a CBL as encountered during a summer day~\cite{sullivan2011}. Forced at the surface by solar radiative heating and weak geostrophic winds~\cite{moeng_comparison_1994}, warm plumes ascend to the top of the boundary layer (\emph{boundary layer height} $z_i\approx 1$ km), balanced by larger scale downdrafts, resulting in turbulent dynamics spanning a wide range of scales.
CBL dynamics are crucial to understanding the fundamental properties and sensitivities of the (strato) cumulus clouds that can form above a CBL as moisture is added to the simulation, which are in turn a major source of uncertainty in current climate projections~\cite{schneider_climate_2017}.

In order to allow for larger time steps, an anelastic approximation to the fully compressible Navier--Stokes equations is adopted~\cite{pauluis_thermodynamic_2008,pressel_large-eddy_2015}, eliminating sound waves from the system. 
Given a hydrostatic reference state
\begin{equation}
    \alpha_0 \frac{\partial p_0}{\partial x_3} = -g,
\end{equation}
with reference profiles $\alpha_0(x_3)=\rho_0(x_3)^{-1}$ for specific volume and $p_0(x_3)$ for pressure, and gravity constant $g=9.80665$ ms$^{-2}$, density changes are neglected in thermodynamic and continuity equations. This leaves the dynamic pressure perturbation $p'$ to be diagnosed by solving an elliptic equation at each time step. 

The anelastic equations with velocities $u_1,u_2,u_3$ and entropy $s$ as prognostic equations are then given by
\begin{equation}
    \begin{aligned}
    \frac{\partial u_i}{\partial t} + \frac{1}{\rho_0} \frac{\partial (\rho_0 u_i u_j)}{\partial x_i} 
    & =
    -\frac{1}{\rho_0} \frac{\partial (\rho_0 \tau_{ij})}{\partial x_j}
    -\frac{\partial \alpha_0 p'}{\partial x_i} + b\delta_{13} 
    - \epsilon_{ijk}\delta_{j3} f(u_k-U_{g,k}),
    \\
    \frac{\partial s}{\partial t} + \frac{1}{\rho_0} \frac{\partial (\rho_0 u_i s)}{\partial x_i}
    &=
    -\frac{1}{\rho_0} \frac{\partial (\rho_0 \gamma_i)}{\partial x_i},
    \\
    \frac{\partial \rho_0 u_i}{\partial x_i} &= 0,    
    \end{aligned}
\end{equation}
within the domain $D$ and the time span $[0,T]$,
where we use standard conventions for summing, the Kronecker delta $\delta_{ij}$, and the Levi--Civita tensor $\epsilon_{ijk}$. The Coriolis parameter $f=0.376e-4$ s$^{-1}$ acts upon the difference to the large scale geostrophic wind $U_g$, and we refer to~\cite{pauluis_thermodynamic_2008} for the definition of the buoyancy term $b$.
As is characteristic for large eddy simulations (LES), which do not aim to resolve the smallest turbulent eddies, the effect of turbulent motion smaller than the grid size is modeled by the sub-grid scale (SGS) stresses $\tau_{ij}$ for velocity and $\gamma_i$ for entropy.
The radiative forcing is prescribed as horizontally constant heat flux $Q_*$ at the lower boundary.

We solve the anelastic equations with the open source Python and Cython based \textbf{PyCLES} library~\cite{pressel_large-eddy_2015} at a resolution of $128^3$ and a grid spacing of $40~\text{m}$ and $16~\text{m}$ horizontally and vertically, respectively. In agreement with previous findings~\cite{pressel_numerics_2017}, we omit the modeling of the SGS terms ($\gamma_{s,i}=\tau_{ij}=0$) and rely on the favorable dissipative properties of the WENO scheme which provide an \emph{implicit large eddy simulation}. 

The simulation is initialized with zero velocities and a horizontally constant profile of potential temperature
\begin{equation}
    \theta = T \left(\frac{p_0}{p}\right)^{\frac{R}{c_p}},
\end{equation}
given by a characteristic profile of the shape
\begin{equation}
    \theta(x_3) =
    \begin{cases}
        T_g, & 0<x_3<z_a, \\
        T_g + (x_3 - z_a) (\partial_3\theta)_i, & z_a<x_3<z_b, \\
        T_g + (\partial_3\theta)_i (z_a-z_b) + (x_3 - z_b) (\partial_3\theta)_e,
    \end{cases}
\end{equation}
with dry constants $R=287.1$ JK$^{-1}$kg$^{-1}$ and $c_p=1004.0$ JK$^{-1}$kg$^{-1}$, and reference pressure $p_0$ at ground level. We fix the environmental profile, given by the lapse rate $(\partial_3\theta)_e=0.003~\text{K}\,\text{m}^{-1}$ and a fixed intercept, to match the original test case formulation~\cite{sullivan2011}, and define the boundary layer height to be $z_i = (z_b-z_a)/2$. This leaves as free parameters of the initial condition the boundary layer temperature $T_g$, the initial boundary layer height $z_i$, and the inversion lapse rate (sharpness) $(\partial_3\theta)_i$. Together with the geostrophic wind $U_g$ and the heat forcing $Q_*$, each realization of the CBL simulation is hence parameterized by five free parameters which we sample uniformly in the training distribution, as summarized in Table~\ref{tab:cbl_parameters}.

Based on this deterministic setup, convection is initialized by random perturbations at the lowest model levels. For the same initial profile of $\theta_0$, different realizations of the CBL are hence obtained by varying the random seed. For the experiments in this paper, all training samples were generated with the same random seed (and varying input parameters), since there is no dependence on the exact initial perturbation after the model spinup phase. In order to sample a Dirac measure, however, the random seed is varied while keeping the input parameters fixed.

The boundary layer height $z_i$ is computed as in~\cite{sullivan2011} by the \emph{maximum gradient method} as the horizontal mean of the vertical location with the largest temperature gradient. Due to the constant $Q_*$ forcing, $z_i$ grows linearly in time. 
In Figure~\ref{fig:cbl4}, we evaluate GenCFD against the numerical truth on horizontal statistics, including statistical moments of prognostic variables, the vertical temperature flux, and turbulent kinetic energy (TKE).

In the experiments conducted, the total duration of $7200~\text{s}$ in model time was scaled to \( T = 2.4 \). We test the solution at \( T_{\text{test}} = 2.4 \). The all-to-all training was performed using snapshots corresponding to the time points \( \{1.2, 1.4, 1.6, 1.8, 2.0, 2.2, 2.4\} \). There are $21$ input-output pairs per trajectory in the training set. During the testing phase, the model is conditioned on the time step \( t = 1.2 \), with the assumption that the spin-up phase has been completed by this point. Thus, the models aim to learn the statistics of the solution operator which maps initial conditions and parameters to solutions at later times.

\begin{table}[h!]
    \caption{\textbf{Parameters in the convective boundary layer experiment.} Default values are as given in~\cite{sullivan2011}, the lower and upper bounds are used for uniform sampling of the training data. For testing with a Dirac distribution, the default values are used.}
    \label{tab:cbl_parameters}
    \centering
    \begin{tabular}{lllccccc}
    \toprule
        Parameter & Unit         & Type              & Lower bound & Default value & Upper bound \\
    \midrule
    \midrule
        $Q_*$       & Kms$^{-1}$ & Forcing           & $0.1$        & $0.24$        & $0.3$ \\
        $U_g$       & ms$^{-1}$  & Forcing           & $0.0$        & $1.0$         & $5.0$ \\
        $T_g$       & K          & Initial condition & $297.5$      & $300.0$       & $302.5$  \\
        $z_i$       & m          & Initial condition & $974.0$      & $1024.0$      & $1074.0$ \\
        $(\partial_3\theta)_i$ & Km$^{-1}$ & Initial condition & $0.03$ & $0.08$    & $0.15$ \\
    \bottomrule
    \end{tabular}
\end{table}

\subsection{Details of Models and Hyperparameters}
Here, we describe the selection procedure for GenCFD and the baselines.

\paragraph{GenCFD.} In all experiments, axial attention was applied at every layer of the UViT denoiser. Each of the tested models consistently used 4 axial attention blocks and 8 attention heads. The UViT architecture (see Fig.~\ref{fig:s1}) consisted of 3 downsampling layers, each with a downsampling ratio of 2. The intermediate channel dimension for the input and output spaces was set to 128, and the Fourier embedding dimension for physical (PDE) time and noise levels was also 128. In the CSF, TG, CSI, and NF experiments, the number of channels per downsampling layer followed the configuration $(64, 128, 256)$. The resulting model is referred to as the \textit{base} architecture. For the CBL benchmark, a smaller architecture, referred to as the \textit{small} architecture, was used for improved memory efficiency and faster training. This version employed channels per layer in the sequence $(48, 96, 192)$. Table~\ref{tab:uvit_details} presents the sizes of the models used in our experiments. Note that the ground truth data is typically generated at a resolution that is higher than the resolution of the computational grid of the UViT denoiser (Table~\ref{tab:uvit_details}). Hence, the data is downsampled onto this computational grid using numerical downsampling.

\begin{table}[h!]
    \caption{{\textbf{Details of the GenCFD architectures used in the benchmarks.} The column \textit{In/Out Ch.} corresponds to the number of input/output channels of the models.}}
    \label{tab:uvit_details}
    \centering
    \begin{small}
     \begin{tabular}{ c c c c c}
        \toprule
            \makecell{Benchmark} & 
            \makecell{Model}& 
             \makecell{Size} & 
             \makecell{Resolution} & 
             \makecell{In/Out Ch.}  \\
        \midrule
        \midrule
        \multirow{1}{*}{\centering{CSF, TG}}
        & 
        \makecell{\textit{base}}  & 
        \makecell{$70.2\text{M}$}  & 
        \makecell{$64^3$}  & 
        \makecell{$3/3$}  \\
        \midrule
        \midrule
        \multirow{1}{*}{\centering{CSI}}
        & 
        \makecell{\textit{base}}  & 
        \makecell{$70.2\text{M}$}  & 
        \makecell{$64^3$}  & 
        \makecell{$4/4$}    \\
        \midrule
        \midrule
        
        \multirow{1}{*}{\centering{NF}} &
        \makecell{\textit{base}}  & 
        \makecell{$70.2\text{M}$}  & 
        \makecell{$64\times64\times192$}  &
        \makecell{$4/3$}   \\
        \midrule
        \midrule
        \multirow{1}{*}{\centering{CBL}} &
        \makecell{\textit{small}}  & 
        \makecell{$40.1\text{M}$}  & 
        \makecell{$128^3$}  & 
        \makecell{$5/4$}   \\
    
        \bottomrule
     \end{tabular}
    \end{small}
\end{table} 

Moreover, in Table~\ref{tab:gencfd_train}, we outline the experimental details for each dataset, including the number of training samples, batch sizes, gradient steps (and corresponding epochs), as well as the number of GPUs and their memory capacity used for training. Each model was trained for approximately 72 to 120 hours, except for the one used in the NF benchmark, which completed training in just 24 hours. It is worth noting that training times could be significantly reduced by utilizing multiple GPUs with larger memory.

\begin{table}[h!]
    \caption{{
    \textbf{Experimental details for the GenCFD models.} The first column indicates the benchmark, the second shows the number of trajectories (\emph{Num.\ Traj.}) used for training, and the third lists the total number of training samples for all-to-all training (\emph{Num.\ Sam.}). The fourth column specifies the batch size (\emph{B.\ S.}), the fifth details the number of gradient steps (\emph{Num.\ Grad.\ S.}) during training, and the sixth provides the approximate number of training epochs (\emph{Epoch}). Finally, the seventh column notes the number of GPUs used for training along with their memory (\emph{GPU (num:mem)}).}}
    \label{tab:gencfd_train}
    \centering
    \begin{small}
      \begin{tabular}{ c c c c c c c}
        \toprule
            \makecell{Benchmark} & 
            \makecell{Num.\ Traj.}& 
             \makecell{Num.\ Sam.} & 
             \makecell{B.\ S.} & 
             \makecell{Num.\ Grad.\ S.} & 
             \makecell{Epoch} &
             \makecell{GPU ($\text{num}:\text{mem}$)}\\
        \midrule
        \midrule
        \multirow{1}{*}{\centering{TG}} &
        \makecell{$6.6\text{K}$}  & 
        \makecell{$99\text{K}$}  & 
        \makecell{$5$}  &
        \makecell{$1.0\text{M}$} &
        \makecell{$50.5$} &
        \makecell{$1:24\text{GB}$} \\
        \midrule
        \midrule
        \multirow{1}{*}{\centering{CSF}} &
        \makecell{$9.9\text{K}$}  & 
        \makecell{$99\text{K}$}  & 
        \makecell{$5$}  &
        \makecell{$1.0\text{M}$} &
        \makecell{$50.5$} &
        \makecell{$1:24\text{GB}$} \\
        \midrule
        \midrule
        \multirow{1}{*}{\centering{CSI}} &
        \makecell{$9.9\text{K}$}  & 
        \makecell{$99\text{K}$}  & 
        \makecell{$2$}  &
        \makecell{$0.8\text{M}$} &
        \makecell{$16.2$} &
        \makecell{$1:24\text{GB}$} \\
        \midrule
        \midrule
        \multirow{1}{*}{\centering{NF}} &
        \makecell{$9\text{K}$}  & 
        \makecell{$189\text{K}$}  & 
        \makecell{$4$}  &
        \makecell{$0.1\text{M}$} &
        \makecell{$2.2$} &
        \makecell{$1:80\text{GB}$} \\
        \midrule
        \midrule
        \multirow{1}{*}{\centering{CBL}} &
        \makecell{$7.5\text{K}$}  & 
        \makecell{$157.5\text{K}$}  & 
        \makecell{$4$}  &
        \makecell{$0.25\text{M}$} &
        \makecell{$6.3$} &
        \makecell{$4:80\text{GB}$} \\
        \bottomrule
     \end{tabular}
    \end{small}
\end{table}

\paragraph{Baselines.}
The training and inference procedures differ between the baselines (FNO, C-FNO, and UViT) and the GenCFD. The GenCFD model was trained without using early stopping, whereas the deterministic models were trained until convergence using early stopping to prevent overfitting. In every benchmark, we observed that the training of deterministic models invariably collapsed to the mean of the output distribution. Even when the deterministic models were trained for extended durations and allowed to overfit the training set, the outcome remained unchanged. The deterministic models completed training in approximately 24 to 48 hours. Note that the same number of training samples was used as in the GenCFD trainings.

For each experiment and baseline (except for CBL, due to memory constraints), cross-validation was conducted using a random grid search over a range of hyperparameters -- typically 10 to 20 for the FNO models and 4 to 10 for UViT, respectively. Since the deterministic UViT model shares the same architecture as GenCFD across all benchmarks, only the learning rate was varied during the hyperparameter sweeps.

\begin{table}[h!]
    \caption{{\textbf{Details of the FNO architectures used in the benchmarks.} The $d_v$ corresponds to the lifting dimension, $L$ to the number of Fourier layers, $M$ number of modes used in the Fourier layer, $lr$ to the peak learning rate, \emph{B.\ S.} to the batch size. The architectures are obtained using a random grid search over a range of hyperparameters (typically 10 to 20 configurations).}}
    \label{tab:fno}
    \centering
    \begin{small}
      \begin{tabular}{ c c c c c c c c}
        \toprule
            \makecell{Benchmark} & 
            \makecell{$d_v$}& 
             \makecell{$L$} & 
             \makecell{$M$} & 
             \makecell{$lr$} & 
             \makecell{B.\ S.} &
             \makecell{Size}  \\
        \midrule
        \midrule
    
        \multirow{1}{*}{\centering{TG}} &
        \makecell{64}  & 
        \makecell{5}  & 
        \makecell{12}  &
        \makecell{0.0001}  & 
        \makecell{2}  & 
        \makecell{42.1M} \\
        \midrule
        \midrule
        \multirow{1}{*}{\centering{CSL}} &
        \makecell{64}  & 
        \makecell{5}  & 
        \makecell{16}  &
        \makecell{0.0001}  & 
        \makecell{5}  & 
        \makecell{95.2M} \\
        \midrule
        \midrule
        \multirow{1}{*}{\centering{CSI}} &
        \makecell{64}  & 
        \makecell{5}  & 
        \makecell{16}  &
        \makecell{0.0001}  & 
        \makecell{2}  & 
        \makecell{95.2M} \\
        \midrule
        \midrule
        \multirow{1}{*}{\centering{NF}} &
        \makecell{64}  & 
        \makecell{5}  & 
        \makecell{12}  &
        \makecell{0.0001}  & 
        \makecell{1}  & 
        \makecell{42.1M} \\
        \midrule
        \midrule
        \multirow{1}{*}{\centering{CBL}} &
        \makecell{48}  & 
        \makecell{4}  & 
        \makecell{16}  &
        \makecell{0.0001}  & 
        \makecell{1}  & 
        \makecell{43.1M} \\
        \bottomrule
     \end{tabular}
    \end{small}
\end{table}

\begin{table}[t!]
\caption{{\textbf{Details of the C-FNO architectures used in the benchmarks.} The $d_v$ corresponds to the lifting dimension, $L$ to the number of Fourier layers, $M$ number of modes used in the Fourier layer, $lr$ to the peak learning rate, \emph{B.\ S.} to the batch size. The architectures are obtained using a random grid search over a range of hyperparameters (typically 10 to 20 configurations).}}
\label{tab:cfno}
\centering
\begin{small}
  \begin{tabular}{ c c c c c c c c}
    \toprule
        \makecell{Benchmark} & 
        \makecell{$d_v$}& 
         \makecell{$L$} & 
         \makecell{$M$} & 
         \makecell{$lr$} & 
         \makecell{B.\ S.} &
         \makecell{Size}  \\
    \midrule
    \midrule

    \multirow{1}{*}{\centering{TG}} &
    \makecell{64}  & 
    \makecell{4}  & 
    \makecell{12}  &
    \makecell{0.0001}  & 
    \makecell{2}  & 
    \makecell{40.0M} \\
    \midrule
    \midrule
    \multirow{1}{*}{\centering{CSL}} &
    \makecell{64}  & 
    \makecell{4}  & 
    \makecell{16}  &
    \makecell{0.0001}  & 
    \makecell{3}  & 
    \makecell{82.4M} \\
    \midrule
    \midrule
    \multirow{1}{*}{\centering{CSI}} &
    \makecell{64}  & 
    \makecell{3}  & 
    \makecell{16}  &
    \makecell{0.0001}  & 
    \makecell{2}  & 
    \makecell{62.1M} \\
    \midrule
    \midrule
    \multirow{1}{*}{\centering{NF}} &
    \makecell{64}  & 
    \makecell{5}  & 
    \makecell{16}  &
    \makecell{0.0001}  & 
    \makecell{1}  & 
    \makecell{102.7M} \\
    \midrule
    \midrule
    \multirow{1}{*}{\centering{CBL}} &
    \makecell{48}  & 
    \makecell{4}  & 
    \makecell{16}  &
    \makecell{0.0001}  & 
    \makecell{1}  & 
    \makecell{46.9M} \\

\bottomrule
\end{tabular}
\end{small}
\end{table}

\subsection{Evaluation Metrics}
\label{sec:em}
Evaluation of the model's performance employs a suite of metrics to quantify the fidelity and diversity of the generated samples relative to the ground truth distribution:
\begin{itemize}
    \item \textbf{$L^2$-norm error between the mean} of the ground truth ($\mu_{\text{exact}}$) and the approximated distribution ($\mu$):
    \begin{equation}
        \label{eq:merr}
        e_{\mu} = \|\mu_{\text{exact}} - \mu\|_2 . 
    \end{equation}
    
    \item \textbf{$L^2$-norm error between the standard deviation} of the ground truth ($\sigma_{\text{exact}}$) and the approximated distribution ($\sigma$):
    \begin{equation}
        \label{eq:sderr}
        e_{\sigma} = \|\sigma_{\text{exact}} - \sigma\|_2 . 
    \end{equation}
    The standard deviation error is normalized with respect to the ground truth norm.
    
    \item \textbf{Average $1$-point Wasserstein distance} between the ground truth $p_{\text{exact}}$ and the approximated distribution $p$ (conditional and unconditional) computed over $M$ spatial points:
    \begin{equation}
        \label{eq:wass}
        W_s(p_{\text{exact}}, p) = \sum_{i=1}^M \left( \int_{0}^{1} \left| F_{exact}^{-1}(u(x_i)) - F^{-1}(u(x_i)) \right|^s \, dx \right)^{1/s} , 
    \end{equation}
    with $F$ being the CDFs.
    
    \item \textbf{Continuous ranked probability score (CRPS)}: Given an ensemble of predictions $U=\{u\}_{m=1}^M$, $x_m \sim p$, and a single observation $u_{\text{exact}}$ from the ground truth distributions $u_{\text{exact}} \sim p_{\text{exact}}$, the pointwise CRPS score is defined as
    \begin{equation}
        \label{eq:crps}
        \text{CRPS}[U, u_{\text{exact}}] = \frac{1}{M} \sum_{m=1}^{M} \| u_m - u_{\text{exact}} \|_2^2 - \frac{1}{2M^2} \sum_{m=1}^{M} \sum_{j=1}^{M} \| u_m - u_j \|_2^2 . 
    \end{equation}
    Given an ensemble of observations $U_{\text{exact}}= \{u_{\text{exact}}\}_{n=1}^N$, we can extend the definition of the CRPS as 
    \begin{equation}
        \text{CRPS}[U, U_{\text{exact}}] = \frac{1}{N} \sum_{n=1}^{N} \text{CRPS}[U, u_{\text{exact}, n}] . 
    \end{equation}
    It should be noted that $\text{CRPS}[U, U_{\text{exact}}]$ is a function of the spatial coordinates $x$. To get a single global indicator of the ensembles' similarities, the (relative) $L^2$-norm of $\text{CRPS}[U, U_{\text{exact}}]$ can be computed as
    \begin{equation}
        \text{CRPS}_G[U, U_{\text{exact}}] = \|\text{CRPS}[U, U_{\text{exact}}]\|_2^2 .
    \end{equation}
    Moreover, the CRPS~\eqref{eq:crps} is also normalized with respect to the $L^2$-norm of the true observation $u_{\text{exact}}$.

\end{itemize}

\paragraph{On Energy Spectra.} 
Given a flow field $u : \R^d \to \R^d$ where $d$ denotes the dimension of space, we denote its component-wise Fourier transform by $\hat{u} : \R^d \to \C^d$. The energy spectrum is then defined as
\begin{equation}
    E_k = \frac{1}{2}\int_{|\xi|=k}\! \norm{\hat{u}(\xi)}^2 \,\mathrm{d}\xi.
\end{equation}
As our solutions $u^{\Delta} : \T^d \to \R^d$ lie on the $d$-dimensional torus $\T^d$, we make use of their discrete Fourier transforms $\hat{u}^{\Delta}_k$ in order to compute the energy spectrum. Furthermore, we integrate over the ball in $L^1$, as that aligns with the computational grid. All in all, the energy spectra of our discrete solutions are computed as
\begin{equation}
    E_k = \frac{\Delta^d}{2}\sum_{\norm{\xi}_1=k} \norm{\hat{u}^{\Delta}_\xi}^2.
\end{equation}

\subsection{Details on Toy Model $\# 1$ of the Main Text} 
\label{sec:toy1-desc}

In the toy problem mentioned in the Main Text, we aim to mimic essential aspects of the behavior of turbulent flows, in particular the sensitive dependence of outputs on small perturbations of the inputs. To this end, we fix $\Delta = 1/N$ and consider a very simple one-dimensional model by setting $u,\bar{u} \in \R$ and introducing a sequence of one-dimensional mappings 
\begin{equation}
    \label{eq:tm1}
    \sol^\Delta(\bar{u}) 
    = 
    m(\bar{u}) + s_N(\bar{u}), \quad s_N(\bar{u}) := \Lambda(N\bar{u}), 
\end{equation}
with $m:\R \mapsto \R$ being any \emph{mean} function and $\Lambda$ being the $1$-periodic extension of the hat-function, with values $\Lambda(0) = \Lambda(1) = -1$, $\Lambda(1/2) = 1$. The parameter $\Delta=1/N$ allows us to control the input-sensitivity of the underlying map $\sol^\Delta$, increasing sensitivity as $\Delta \to 0$. Fix the initial measure $\bar{\mu} = \cU[0,1]$ to be the uniform measure on $[0,1]$. Let $\mu^\Delta = (\sol^\Delta \times \id)_\# \bar{\mu}$ denote the push-forward measure that we wish to approximate. 

We observe from the Main Text Fig.~4 (A) how $\sol^\Delta$ becomes more and more oscillatory as $\Delta \to 0$. In particular, it does not seem possible to realize a deterministic limit. It is also easy to show that $\Lip(\sol^\Delta) \sim 4N \to \infty$. Nevertheless, we will later show that the $\Delta \to 0$ limit is well-defined \emph{statistically} and the resulting conditional distribution is given by the uniform distribution $p(u\bb \bar{u}) = \cU [m(\bar{u})-1,m(\bar{u})+1]$, centered around the mean $m$. 

Thus, this toy problem mimics several relevant features of turbulent fluid flow such as i) no deterministic limit under mesh refinement, ii) unstable behavior of the numerical approximation operator $\sol^\Delta$ in the limit, iii) the limit under mesh refinement is well-defined statistically and iv) the limit measure is not a Dirac but is \emph{spread out}. 

\subsubsection{Numerical Results}
\label{sec:c43}

We implemented toy problem \#1, and trained both deterministic models and diffusion models for various values of the parameter $\Delta > 0$. 

Our theoretical considerations assumed a bounded Lipschitz constant $L^\ast$ as a mathematical proxy for the limitations in learning oscillatory functions. In practice, the approximation of neural networks is limited by (at least) three factors: (i) the available training data,
(ii) the model capacity (architecture),
(iii) the approximate optimization by stochastic gradient descent.

\paragraph{Training Data.} We train all models with a total of $N=2048$ training samples, sampled uniformly on the interval $[0,1]$. Since we focus on values of $\Delta\ge 0.02 \gg 1/N$, we expect the training samples to allow (in principle) for near-perfect interpolation of $\sol^\Delta$.

\paragraph{Architecture.} Our experiments are based on small models, all of which are chosen as vanilla dense, feedforward MLPs with ReLU activation. The deterministic model employs 2 hidden layers, and width 256, mapping a one-dimensional input (corresponding to $\bar{u}$) to a one-dimensional output,
\begin{equation}
    \bar{u} \mapsto \Psi_\mathrm{det}(\bar{u}).
\end{equation}
The diffusion model has depth 3 and width 512, mapping a three-dimensional input (corresponding to $(u;\bar{u},\sigma)$) to a one-dimensional output, i.e.\
\begin{equation}
    (\bar{u}, u, \sigma) \mapsto D(u;\bar{u},\sigma).
\end{equation}
From limited experimentation during the implementation, the qualitative results of the experiments are observed to be robust to changes in the hyperparameters of the networks. Our goal is to examine the qualitative behavior of the deterministic and diffusion models when $\Delta \to 0$, independently. No attempt is made to provide a quantitative comparison between the deterministic and diffusion models (which is probably meaningless for these toy problems).

\paragraph{Training.} All models are implemented and trained in \textbf{PyTorch}. We use the Adam optimizer with the learning rate set to $10^{-3}$. Training is performed for a fixed number of epochs, for a maximum of $10000$ epochs. The deterministic models are trained with MSE loss. For these one-dimensional, highly oscillatory toy problems, it is not always clear whether true stagnation of the training progress is observed. We therefore opt to illustrate not only the final results after $10000$ epochs, but also the training progress. The observed difference between the deterministic and diffusion models during training is another interesting outcome of these toy problems.

\paragraph{Illustration of Results.} To illustrate the trained deterministic models, we sample $2048$ point in $\bar{u}$, and show a scatter plot of $(\bar{u}, \Psi(\bar{u}))$. Similarly, for the denoising models, we show a scatter plot as follows: We first sample $300$ points uniformly in $\bar{u}$, and then generate $100$ samples from the learned conditional distribution $p(u \bb \bar{u})$ for each point in $\bar{u}$, giving a scatter plot of $30000$ samples in total. The noise process is chosen as a variance-preserving process, as in~\cite{ho2020denoising}. The backward denoising process is run with 200 timesteps and a cosine noise schedule~\cite{nichol2021improved}.

\clearpage
\newpage

\section{Detailed Theory}
\label{thr}
In this section, we present rigorous mathematical statements (and their proofs) that justify and expand on the theoretical discussion in the Main Text. Our aim is to explain, with rigorous mathematical analysis, the observations from our numerical experiments. In particular, we focus on explaining how i) diffusion models are able to learn the underlying probability distributions while deterministic ML baselines regress to mean fields and ii) diffusion models provide excellent spectral resolution (coverage) and are able to approximate small scales, right up to the smallest eddies in the data, while deterministic ML models have very poor spectral resolution.

\subsection{Main results}

\paragraph{Characterization of Optimal Denoisers.} As the time-conditioning in \eqref{eq:dnostr} is not relevant for this theoretical discussion, we omit it and consider the \emph{denoiser training objective} or Diffusion loss as 
\begin{align}
    \label{eq:Jdiff}
    \cJ(D_\theta)
    = 
    \E_{\bar{u}\sim \bar{p}} \E_{u|\bar{u}} \E_{\eta \sim \cN(0,\sigma^2)}
    \Vert D_\theta(u + \eta; \bar{u}, \sigma) - u \Vert^2.
\end{align}
Hence, our aim is to remove noise from the noisy sample $u_\sigma = u+\eta$, $\eta\sim \cN(0,\sigma^2)$, conditioned on the input $\bar{u}$,  during the training of the denoiser. It turns out (see Lemma~\ref{lem:Dformula}) that we can explicitly characterize the optimal denoiser $D_\opt = \argmin_\theta \cJ(D_\theta)$ for any noise level $\sigma > 0$ by $D_\opt(u_\sigma;\bar{u},\sigma) = \E[u\bb (\bar{u},u_\sigma)]$. Therefore, if the noise process has ended up in a location $u_\sigma$, the optimal denoiser considers the distribution of the conditional random variable $u$ given $(\bar{u},u_\sigma)$, i.e.\ all possible origins $u$ of the noise process conditioned on the input $\bar{u}$ and the noised sample $u_\sigma$, and selects the most likely origin as the expected value in this distribution. 

This key observation can be used to further characterize the optimal denoiser in the zero-noise ($\sigma \to 0$) limit (see Proposition~\ref{prop:1}). In particular, we prove that, in this limit, the optimal denoiser $D_\opt(w;\bar{u},\sigma=0)$ evaluated at a point $w$, is identified with the closest point $w^\ast$ in the support of $p(u\bb \bar{u})$, corresponding to a projection onto the data manifold. An immediate consequence of the identity $D_\opt(u_\sigma;\bar{u},\sigma) = \E[u\bb (\bar{u},u_\sigma)]$ for $\sigma > 0$ is the fact that
\begin{proposition}
    \label{prop:whatsthepoint}
    If 
    $u \bb \bar{u}$ is in fact deterministic, i.e.\ $u=\cF(\bar{u})$, then 
    $
    D_\opt(u_\sigma; \bar{u}, \sigma) \equiv \cF(\bar{u})$ for all $\bar{u}, \sigma$.
\end{proposition}
The above proposition makes it clear that, if a conditional distribution is generated by an underlying deterministic map $\cF$, then the optimal denoiser will simply collapse to this deterministic map. In practice, the diffusion model is trained on
\begin{align}
    \label{eq:Jdiff1}
    \cJ^\Delta(D_\theta) =  
    \E_{\bar{u}\sim \bar{p}} \E_{u^\Delta \bb \bar{u}} \E_{\eta\sim \cN(0,\sigma^2)} \Vert D_\theta(u^\Delta+\eta; \bar{u},\sigma) - u^\Delta \Vert^2,
\end{align}
with $u^\Delta \bb \bar{u}= \sol^\Delta(\bar{u})$ obtained from a numerical solver. Given that \eqref{eq:Jdiff1} stems from the \emph{deterministic} approximate solution operator $\sol^\Delta$, the optimal denoiser should be concentrated around $\sol^\Delta(\bar{u})$ for any $\Delta > 0$. Hence, there should be no difference between using a denoiser training objective \eqref{eq:Jdiff1} and a purely deterministic loss,
\begin{align}
    \label{eq:Jdet1}
    \cJ^\Delta_{\mathrm{det}}(\Psi_\theta) 
    =
    \E_{\bar{u}\sim \bar{p}} \Vert \Psi_\theta(\bar{u}) - \sol^\Delta(\bar{u}) \Vert^2.
\end{align}
 In other words, this result suggests \emph{there should be very little difference between the simulations carried out using the ML baselines and our proposed conditional-diffusion model, seemingly contradicting our experimental observations.} Therefore, we need to formulate a more refined analysis in order to resolve this apparent contradiction. 

\paragraph{Input Sensitivity.}
To this end, we recall that, for fluid flows, the approximate solution operator outputs solutions which contain energetic eddies across a very large range of scales, exhibiting chaotic dynamics. Consequently, the behavior of $\sol^\Delta$ asymptotically as $\Delta \to 0$ is very oscillatory and unstable~\cite{FLMW1,LMP1}. It is precisely this lack of stability in the $\Delta \rightarrow 0$ limit that could prevent us from realizing an \emph{unconstrained} optimal denoiser within the class of neural networks. Our refined theoretical analysis is based on the assumption of a \emph{mismatch in input sensitivity}: the underlying solution operator $\sol^\Delta$ is extremely sensitive to small perturbations $\delta \bar{u}$. Even perturbations of a small size $\Vert \delta \bar{u} \Vert_{\bX} \sim {\tilde{\epsilon}}$ can entail
\begin{align}
    \label{eq:sensitive}
    \Vert \sol^\Delta(\bar{u}+\delta \bar{u}) - \sol^\Delta(\bar{u}) \Vert_{\bY} 
    \gg 1.
\end{align}
This is precisely the \emph{sensitivity hypothesis} of the underlying operators of the main text. 

In contrast, we argue that our trained neural network model could not be able to match this input sensitivity; i.e.\ a sufficiently small input perturbation $\Vert \delta \bar{u}\Vert_{\bX} \sim {\tilde{\epsilon}} \ll 1$ only leads to a small output perturbation,
\begin{equation}
    \Vert \Psi_\theta(\bar{u}+\delta \bar{u}) - \Psi_\theta(\bar{u}) \Vert_{\bY} \ll 1
    \quad\text{and} \quad
    \Vert D_\theta(u_\sigma;\bar{u}+\delta \bar{u},\sigma) - D_\theta(u_\sigma;\bar{u},\sigma) \Vert_{\bY}\ll 1.
\end{equation}
This is the \emph{insensitivity hypothesis for neural networks}, that is discussed in the main text. Why does this hypothesis hold? Why are neural networks \emph{insensitive to very small perturbations in inputs}?  

A possible answer lies in the notion that \emph{neural networks learn and generalize well at the edge of chaos} \cite{feng2020optimalmachineintelligenceedge,zhang2021edgechaosguidingprinciple}. In this framework based on statistical physics, the forward pass of neural networks is viewed as a dynamical system. The underlying principle states that optimal computational capability of neural networks (or other systems, including the brain) emerges when the dynamical system is at a critical point between order and chaos. A relevant \emph{measure of chaos} is the Lyapunov exponent of the input-to-output map, defined as $\gamma \approx \frac1T \log(|\delta x_T|/|\delta x_0|)$ \cite[eq.\ (S27)]{feng2020optimalmachineintelligenceedge}, defined in terms of a temporal parameter $T$ (which can be the depth for deep networks or the number of rollout steps for autoregressive neural network predictions) and the quotient of the magnitude of output perturbations $|\delta x_T|$ versus the magnitude of (small) input perturbations $|\delta x_0|$. This quantity is equivalent to the Lipschitz constant $\Lip(\Psi_\theta) \approx |\delta x_T|/|\delta x_0|$. Based on both theoretical considerations and extensive empirical evidence, it has been demonstrated that neural networks maximize their performance and generalization capability when $\gamma \approx 0$ \cite{feng2020optimalmachineintelligenceedge}, i.e.\ when $\Lip(\Psi_\theta) \approx 1$. Thus, this \emph{edge of chaos} principle leads to an obvious tension (or indeed contradiction) between the opposing goals of keeping Lipschitz constants of neural networks \emph{bounded of order $1$} to ensure optimal performance, and the requirement of having \emph{exponentially large} Lipschitz constants as would be required to fit the exponential input-sensitivity of turbulent/chaotic flows. 

A second and related motivation for the assumption of insensitivity of neural networks is by viewing it as a high-dimensional analogue of the well-known spectral bias of neural networks \cite{Rah1}. As observed for function regression in one dimension, neural networks are biased against fitting high-order Fourier modes. In that context, the lack of regularity of an underlying mapping, i.e.\ its input sensitivity, is encoded by a slow decay of the Fourier spectrum; fitting a highly input-sensitive function by a neural network necessitates the accurate approximation of high-order Fourier modes, and hence overcoming this observed spectral bias.

Finally, neural networks are trained by variants of stochastic gradient descent algorithms that require that the underlying gradients are well-behaved, i.e, are bounded and of order one. Otherwise, the well-known \emph{exploding and vanishing gradient problem} \cite{evgp} will be encountered and will lead to a failure of neural network training. Clearly, these gradients are related to the Lipschitz constants of the neural network $\Psi_\theta$ with respect to the inputs. Hence, ensuring well-bounded gradients, necessary for neural network training, also adds weight to the insensitivity hypothesis that we propose here. 

Thus, we argue that our assumption of a mismatch in input-sensitivity is natural, given the underlying chaotic dynamics of fluid flows or any sensitive map. While leaving a more rigorous investigation of this assumption for future work, we posit it here as a postulate, based on which theoretical consequences are to be derived below. This will allow us to explain several of our empirical observations.

\paragraph{Lipschitz Continuity Quantifies Input Sensitivity.}
The informal inequality \eqref{eq:sensitive} implies that discretized solution operators of interest have very large Lipschitz constants, i.e.\ that $\Lip(\sol^\Delta) \gg 1$ for small $\Delta$. To capture this in our mathematical analysis, we argue that these Lipschitz constants are so large, that for all practical purposes the relevant regime is captured by assuming $\Lip(\sol^\Delta) \to \infty$, as $\Delta \to 0$; in fact, for PDEs such as the Navier--Stokes equations, whether $\Lip(\sol^\Delta)$ remains finite or not is a long-standing open problem, and the potential ill-posedness in the limit $\Delta \to 0$ is a realistic possibility.
In contrast, we surmise that \sam{training by stochastic gradient descent tends to bias neural networks towards stability and away from highly oscillatory multiscale mappings}. We mathematically formulate this hypothesis as follows:
\begin{hypothesis}
    \label{hyp}
    Practical minimization of the denoiser objective \eqref{eq:Jdiff} is only possible within a subclass of mappings $D_\theta$ satisfying a Lipschitz bound $\Lip(D_\theta) \le L^\ast$, for some cut-off $L^\ast \ge 1$.
\end{hypothesis}
In fact, the Lipschitz bound is one possible mathematical requirement for ruling out wild oscillations and can be replaced by other equivalent criteria such as bounds on total variation or conditions on band-limited approximations~\cite{reno}. Under this hypothesis, it is straightforward to show that \emph{constrained minimization} of the deterministic training objective \eqref{eq:Jdet1}, within the model class specified by hypothesis~\ref{hyp}, \emph{cannot approximate the true minimizer}. This merely reflects the fact that $\Lip(\sol^\Delta) \gg L^\ast$, as $\Delta \to 0$, and hence the optimal $D_\opt^\Delta = \sol^\Delta$ (cp. Proposition~\ref{prop:whatsthepoint}) has a divergent Lipschitz constant in this limit. In particular, the unconstrained minimizer is unstable and cannot be approximated under Hypothesis~\ref{hyp}, highlighting the \sam{potential} role of instabilities and multiscale structure of the underlying operators in hindering the success of deterministic approximations of the solution operator of turbulent flows by neural networks.

\paragraph{Statistical Computation with Diffusion Models is still Tractable.} On the other hand, how can training of the denoiser~\eqref{eq:Jdiff} within the model class of Hypothesis~\ref{hyp} lead to an accurate statistical computation of fluid flows? A positive answer is provided by the following proposition (see p.~\pageref{pf:denoiser-est} for a proof):
\begin{proposition}
    \label{prop:denoiser-est}
    Let $\mu, \mu^\Delta$ be two probability measures with bounded support on $\{|u|\le M\}$. Assume that the optimal conditional denoiser $D_{\mathrm{opt}}(u;\bar{u},\sigma)$ for $\mu$ is $L^\ast$-Lipschitz continuous for some $L^\ast \ge 1$. Let $D^\Delta$ be the optimal \emph{constrained} denoiser $D^\Delta$ for $\mu^\Delta$,
    \begin{equation}
        D^\Delta(u;\bar{u},\sigma) = \argmin_{\Lip(D_\theta)\le L^\ast} \cJ^\Delta(D_\theta,\sigma).
    \end{equation}
    Then, we have
    \begin{align}
        \label{eq:denoiser-est}
        \E_{(u,\bar{u})\sim \mu} \E_{\eta\sim \cN(0,\sigma^2)}
        \left\Vert
        D^\Delta(u+\eta;\bar{u},\sigma) 
        -
        D_{\mathrm{opt}}(u+\eta;\bar{u},\sigma)
        \right\Vert^2
        \le CL^\ast W_1(\mu^\Delta, \mu),
        \quad \forall \, \sigma> 0,
    \end{align}
    with constant $C$ depending on $M$, but otherwise independent of $\mu^\Delta$, $\mu$, and independent of $L^\ast$ and $\sigma$.
\end{proposition}

As the whole premise of statistical computation of fluids, backed up by the theory of statistical solutions \cite{FKMT1, LMP1}, rests on the observation that $ W_1(\mu^\Delta, \mu) \rightarrow 0$ as $\Delta \to 0$~\cite{LMP1,FLMW1,rohner2024efficient}, we see from \eqref{eq:denoiser-est} that the \emph{constrained} denoiser achieves an almost optimal loss for $\mu$, even though in this setting, we may have $\Lip(D^\Delta_\opt) = \Lip(\sol^\Delta) \to \infty$, and $D_\theta^\Delta$ will not be able to approximate the true optimizer $D^\Delta_\opt(u;\bar{u},\sigma) =  \sol^\Delta(\bar{u})$ at 
any training resolution $\Delta \ll 1$.

Hence, Proposition~\ref{prop:denoiser-est} reveals the \emph{surprising mechanism} through which a diffusion model can leverage the highly unstable and multiscale nature of the underlying operator (modeling fluid flow for instance) to enable accurate statistical computation even though deterministic approximation in this context is not tractable. 

The key assumption in Proposition~\ref{prop:denoiser-est} is the Lipschitz continuity of the optimal denoiser. Due to our current lack of mathematical understanding of the fine properties of statistical solutions, an end-to-end rigorous proof guaranteeing this property for fluid flows remains out of reach with existing mathematical tools. Furthermore, a highly technical proof would not necessarily shed further light on the fundamental mechanisms through which diffusion models work. 
Instead, we here choose another approach and present two \emph{solvable toy models} which mimic relevant aspects of the behavior of turbulent fluids while being analytically tractable. These problems shed further light on the fundamental difficulties encountered by deterministic models, and illustrate how such difficulties can be overcome by probabilistic diffusion models.

\subsection{Toy Model \#1: Illustrating the Consequences of Input Sensitivity Mismatch}
\label{sec:toy1}

We recall that toy model \#1 is a one-dimensional model which mimics essential aspects of the behavior of turbulent fluid flows (cp. Section \ref{sec:toy1-desc}). At the same time, it is analytically tractable and allows for a rigorous mathematical analysis that we describe here. Recall that we fix $\Delta = 1/N$ and that, for $u,\bar{u} \in \R$, we have introduced a sequence of one-dimensional mappings, 
\begin{equation*}
    \sol^\Delta(\bar{u}) 
    = 
    m(\bar{u}) + s_N(\bar{u}), \quad s_N(\bar{u}) := \Lambda(N\bar{u}), 
\end{equation*}
where $m:\R \mapsto \R$ is any \emph{mean} function and $\Lambda$ is a $1$-periodic hat-function. We fix the initial measure $\bar{\mu} = \cU([0,1])$ to be the uniform measure on $[0,1]$. 
We observe from Figure~\ref{fig:15} how $\sol^\Delta$ becomes more and more oscillatory as $\Delta \to 0$. In particular, it does not seem possible to realize a deterministic limit. It is also easy to show that $\Lip(\sol^\Delta) \sim 4N \to \infty$.

We now study $L^\ast$-Lipschitz minimizers of the deterministic loss
\begin{equation}
    \cJ^\Delta_\mathrm{det}(\Psi_\theta) 
    = 
    \E_{\bar{u}\sim \bar{\mu}} | \Psi_\theta(\bar{u}) - \sol^\Delta(\bar{u}) |^2,
\end{equation}
and the conditional diffusion training objective
\begin{equation}
\cJ^\Delta(D_\theta) = \E_{\bar{u}\sim \prior}\E_{u^\Delta \bb \bar{u}} \E_{\eta\sim \cN(0,\sigma^2)} | D_\theta(u^\Delta+\eta;\bar{u},\sigma) - u^\Delta |^2,
\end{equation}
where $u^\Delta \bb \bar{u} = \sol^\Delta(\bar{u})$. The sensitivity hypothesis in the Main Text suggests that, for some sensitivity scale ${\tilde{\epsilon}} > 0$, and $\delta \bar{u}\sim \cU([-{\tilde{\epsilon}},{\tilde{\epsilon}}])$, the Lipschitz optimizer of $\cJ^\Delta_{\mathrm{det}}$ satisfies $\Psi^\Delta(\bar{u}) 
\approx 
\E_{\delta \bar{u}}\left[
\sol^\Delta(\bar{u}+\delta \bar{u})
\right]$, while the constrained optimizer of $\cJ^\Delta$ is approximately equal to the optimal denoiser for $\mathrm{Law}_{\delta \bar{u}}\left[
\sol^\Delta(\bar{u}+\delta \bar{u})
\right]$. Our goal is to make this intuition rigorous, via asymptotic analysis as $\Delta \to 0$. 

\begin{remark}[Leading-order Analysis]
\label{rem:leading-order}
With the notation and assumptions above, we have to leading order in ${\tilde{\epsilon}}$,
\begin{align*}
\sol^\Delta(\bar{u}+\delta\bar{u})
&= 
\sol^\Delta(\bar{u}) + O(N{\tilde{\epsilon}}),
\\
m(\bar{u}+\delta \bar{u})
&=
m(\bar{u}) + O(\Lip(m) {\tilde{\epsilon}}), 
\\
\Psi_\theta(\bar{u}+\delta \bar{u})
&= 
\Psi_\theta(\bar{u}) + O(L^\ast {\tilde{\epsilon}}).
\end{align*}
We will be interested in the regime $\Lip(m) \sim L^\ast \ll N = \Delta^{-1}$, where the gap $L^\ast \ll \Delta^{-1}$ corresponds to a \emph{sensitivity mismatch} between $\Psi_\theta$ and $\sol^\Delta$. Our main observation is that for $\Delta \ll {\tilde{\epsilon}} \ll 1/L^\ast$, the remainder terms for $\Psi_\theta$ and $m$ can be neglected, but this is clearly not admissible for $\sol^\Delta$. As a consequence, $\Psi_\theta$ cannot accurately capture the variation of $\sol^\Delta$ at input scale ${\tilde{\epsilon}}$. To enable rigorous analysis while capturing this relevant regime, we will fix $L^\ast, {\tilde{\epsilon}}$ and consider the asymptotic limit $\Delta \to 0$, in the following.
\end{remark}

\paragraph{Deterministic Models Collapse to the Mean.} We denote the constrained optimizer of the deterministic problem formulation by,
\begin{equation}
    \Psi^\Delta := \argmin_{\Lip(\Psi_\theta)\le L^\ast} \cJ_{\mathrm{det}}^\Delta(\Psi_\theta).
\end{equation}
We note that the underlying map $\sol^\Delta$ has the smallest length-scale $\Delta$, whereas the smallest length scale of $\Psi^\Delta$ is uniformly bounded due to the imposed Lipschitz bound. Thus, in the limit $\Delta \to 0$, the scale separation between the approximated maps $\sol^\Delta$ and the approximants $\Psi^\Delta$ increases arbitrarily. As argued in Remark~\ref{rem:leading-order}, it is in this limit that we can expect our leading-order analysis of the Main Text to be rigorously justified. 

We now assume that the mean function $m(\bar{u})$ is $L^\ast$-Lipschitz. We fix a (arbitrary) constant ${\tilde{\epsilon}} > 0$. The discussion in the Main Text is based on the approximate identity $\Psi_\theta(\bar{u}+\delta \bar{u}) \approx \Psi_\theta(\bar{u})$ for $|\delta \bar{u}|\le {\tilde{\epsilon}}$. Since we assume that $\Psi_\theta(\bar{u})$ and $m(\bar{u})$ obey the same Lipschitz bound, we also expect that $m(\bar{u}+\delta \bar{u}) \approx m(\bar{u})$ to the same accuracy, and hence 
\[
\sol^\Delta(\bar{u}+\delta\bar{u})
=
m(\bar{u}+\delta\bar{u}) + s_N(\bar{u}+\delta \bar{u})
\approx
m(\bar{u}) + s_N(\bar{u}+\delta \bar{u}).
\]
This motivates the following definition:
\begin{equation}
    \tilde{\Psi}^\Delta(\bar{u}) := m(\bar{u}) + \E_{\delta \bar{u}} \left[ s_N(\bar{u}+\delta \bar{u}) \right],
    \quad
    \delta \bar{u} \sim \cU([-{\tilde{\epsilon}},{\tilde{\epsilon}}]).
\end{equation}
Thus, we have $\tilde{\Psi}^\Delta(\bar{u}) \approx \E_{\delta \bar{u}}\left[ \sol^\Delta(\bar{u}+\delta \bar{u})\right] \approx \Psi^\Delta(\bar{u})$, under the assumptions of the Main Text. We next confirm this intuition by showing that $\Psi^\Delta$ and $\tilde{\Psi}^\Delta$ are asymptotically equivalent, as $\Delta \to 0$.

\begin{proposition}
    \label{prop:toy1-det}
    With the definitions above, we have
    \begin{equation}
        \lim_{\Delta \to 0}
        \E_{\bar{u}\sim \bar{\mu}} |\Psi^\Delta(\bar{u}) - \tilde{\Psi}^\Delta(\bar{u})|^2 
        =
        0,
        \quad
        \text{and}
        \quad
        \lim_{\Delta \to 0}
        \E_{\bar{u}\sim \bar{\mu}} |\Psi^\Delta(\bar{u}) - m(\bar{u})|^2 
        = 0.
    \end{equation}
\end{proposition}

The last proposition rigorously justifies the approximate identity $\Psi^\Delta(\bar{u}) \approx \E_{\delta \bar{u}}\left[ \sol^\Delta(\bar{u}+\delta \bar{u})\right]$, provided ${\tilde{\epsilon}}$ is chosen sufficiently small so that $m(\bar{u}+\delta \bar{u}) \approx m(\bar{u})$, and shows that the optimal constrained model $\Psi^\Delta$ collapses to the mean, in the limit $\Delta \to 0$.  

\paragraph{Probabilistic Models Predict Uncertainty.} We next consider the probabilistic problem formulation of conditional diffusion models. We denote by
\begin{equation}
\label{eq:Ddelta}
    D^\Delta(u_\sigma;\bar{u},\sigma) := \argmin_{\Lip(D_\theta)\le L^\ast} \cJ^\Delta(D_\theta),
\end{equation}
the optimal constrained denoiser for $\mu^\Delta(du\bb\bar{u}) = \delta(u-\sol^\Delta(\bar{u}))$ and with $\bar{u}\sim \cU([0,1])$. We again assume that $m(\bar{u})$ is Lipschitz continuous. Again, under the assumptions of the Main Text, we then have 
$\sol^\Delta(\bar{u}+\delta \bar{u}) \approx m(\bar{u}) + s_N(\bar{u}+\delta \bar{u})$.
We thus consider the conditional probability,
\begin{equation}
    \nu^\Delta(du \bb \bar{u}) := \mathrm{Law}_{\delta \bar{u}}\left[
    m(\bar{u}) + s_N(\bar{u}+\delta \bar{u})
    \right], 
    \quad 
    \delta \bar{u} \sim \cU([-{\tilde{\epsilon}},{\tilde{\epsilon}}]),
\end{equation}
for which $\nu^\Delta(du \bb \bar{u}) \approx \mathrm{Law}_{\delta \bar{u}}\left[ \sol^\Delta (\bar{u}+\delta \bar{u})\right]$, consistent with the discussion in the Main Text. 
We denote the optimal (unconstrained) denoiser of $\nu^\Delta(du\bb \bar{u})$ by,
\begin{equation}
\label{eq:Dtilde}
      \tilde{D}^\Delta(u_\sigma;\bar{u},\sigma)
    :=
    \argmin_{D} \E_{\bar{u}\sim \bar{\mu}} \E_{u\sim \nu^\Delta(\slot\bb \bar{u})} \E_{\eta \sim \cN(0,\sigma^2)}
    \Vert 
    D(u+\eta;\bar{u},\sigma) - u
    \Vert^2.
\end{equation}
The difference between $D^\Delta$ and $\tilde{D}^\Delta$ is that $u$ is sampled from $\mu^\Delta(du\bb\bar{u})$ and $\nu^\Delta(du\bb \bar{u})$, respectively. In addition, $D^\Delta$ is a \emph{constrained} minimizer with $\Lip(D^\Delta)\le L^\ast$ imposed, whereas $\tilde{D}^\Delta$ is an \emph{unconstrained} minimizer.

It turns out that $\nu^\Delta$ is asymptotically equivalent to a simpler measure $\mu$, as $\Delta \to 0$: We thus finally define $\mu \in {\rm Prob}(\R\times [0,1] )$ as the uniform measure on 
\begin{equation}
    \cI(m) := 
    \set{(u,\bar{u}) \in \R\times [0,1]}{u \in [m(\bar{u})-1,m(\bar{u})+1]},
\end{equation}
so that $\mu(du\bb \bar{u}) = \cU([m(\bar{u})-1,m(\bar{u})+1])$. It will be shown in Lemma~\ref{lem:sandwich} that $\nu^\Delta \to \mu$.

The following result shows that the optimal constrained denoiser for $\mu^\Delta = \delta(u-\sol^\Delta(\bar{u}))$ is asymptotically equivalent to the optimal (unconstrained) denoiser for $\nu^\Delta(du \bb \bar{u}) = \mathrm{Law}_{\delta \bar{u}}\left[m(\bar{u}) + s_N(\bar{u}+\delta \bar{u}) \right] \approx \mathrm{Law}_{\delta \bar{u}}\left[ \sol^\Delta(\bar{u}+\delta \bar{u})\right]$. This rigorously justifies the conclusion drawn from the sensitivity hypothesis of the Main Text, in the regime $\Delta \ll 1/L^\ast$.
\begin{proposition}
    \label{prop:toy1-prob}
    With the definitions above, and for $L^\ast$ a constant sufficiently large depending only on the Lipschitz constant of $m(\bar{u})$, we have
    \begin{equation}
        \lim_{\Delta \to 0} \E_{(u,\bar{u})\sim \mu} \E_{\eta \sim \cN(0,\sigma^2)}
        \left\Vert
        D^\Delta(u+\eta;\bar{u},\sigma) - \tilde{D}^\Delta(u+\eta;\bar{u},\sigma)
        \right\Vert^2 
        =0,
    \end{equation}
    uniformly in $\sigma > 0$.
\end{proposition}


\subsection{Toy Model \#2: Illustrating Spectral Accuracy of Diffusion Models} 
\label{sec:toy2}
From the experimental results in the Main Text (Fig.~2 (F)), we observed that neural networks trained to minimize least square errors have a very small effective spectrum. The general ideas that went into the toy model of the last section can also be used to gain intuition regarding the success of diffusion models in reproducing correct energy spectra. The hypothesis is again that deterministically trained models cannot capture highly oscillatory behavior, causing them to collapse to the mean in the oscillatory limit. 

\subsubsection{Motivation}
\label{sec:c51}

We now consider a (translation-equivariant) equation like the Navier--Stokes equations. Assume we have an accurate approximation $\Psi \approx \sol$ of the corresponding solution operator (or of $\sol^\Delta$ for small $\Delta > 0$). For a given input $\bar{u}$, let $u = \sol(\bar{u})$ be the corresponding solution. We define a parametric path in the input function space, $h \mapsto \bar{u}_h := \bar{u}(\slot+h)$. We note that, by translation-equivariance of $\sol$, we have $\sol(\bar{u}_h) = u_h := u(\slot + h)$. If $\Psi$ is a good approximation of $\sol$, then by assumption, we have 
 \[
 	\Psi(\bar{u}_h) \approx u_h, \quad \forall h,
 \]
 and for any continuous linear functional $\ell: L^2(D) \to \R$, we also have
 \[
 	\langle \ell, \Psi(\bar{u}_h)\rangle \approx \langle \ell, u_h\rangle, \quad \forall h.
 \]
   But now, consider the Fourier expansion:
\[
u(x) = \sum_{k\in \Z^d} \hat{u}(k) e^{ikx}
\]
and define $\ell$ as the projection onto the $k$-th Fourier mode. Then 
\[
h \mapsto 
\langle \ell, \Psi(\bar{u}_h) \rangle 
\approx 
\langle \ell, \sol(\bar{u}_h) \rangle
=
\hat{u}(k) e^{ikh}.
\]
If $\Vert \Psi(\bar{u}_h) - \sol(\bar{u}_h) \Vert_{L^2} \le \epsilon \ll 1$, then clearly, we must also have
\[
\left|
\langle \ell, \Psi(\bar{u}_h)\rangle - \hat{u}(k) e^{ikh}
\right|
\le 
\Vert \Psi(\bar{u}_h) - \sol(\bar{u}_h) \Vert_{L^2} \le \epsilon.
\]
There are only two options for this upper bound to hold: Either $|\hat{u}(k)|\lesssim \epsilon$ is inherently small (in which case $\langle \ell, \Psi(\bar{u}_h)\rangle \approx 0$ would do), or $|\hat{u}(k)| \gg \epsilon$ is not small, in which case $h \mapsto \langle \ell, \Psi(\bar{u}_h) \rangle /  \hat{u}(k)$ must be a good approximation of the oscillatory function $h \mapsto  e^{ikh}$.
Just as in the previous section, if $\Psi$ is constrained to be non-oscillatory, then it is impossible to achieve a highly accurate approximation of $h \mapsto e^{ikh}$ for large $k$. Instead, we expect to see a collapse to the mean in the limit $|k|\to \infty$.

\subsubsection{Model} 
\label{sec:c52}

The discussion above shows that for relevant solution operators $\sol$ in fluid dynamics, with solutions exhibiting slowly decaying Fourier spectrum, the mapping 
\[
\bar{u}_h \mapsto \sol(\bar{u}_h),
\]
can be considered oscillatory in a sense related to Fourier analysis. The following toy model replaces the dependence on the input function $\bar{u}_h$ by a dependence on $h\in [0,1]$, resulting in the following oscillatory model with parameter $k\in \N$:
\[
h \mapsto \sol^{(k)}(h),
\quad
\sol^{(k)}(h) := (\cos(2\pi kh), \sin(2\pi kh)) \in \R^2.
\]
This toy problem captures a core difficulty in accurately reproducing energy spectra for fluid flows: namely, Fourier modes of the solution at high wavenumber ($k\gg 1$) are increasingly sensitive to small perturbations of the initial data. The initial data is here replaced by $h\in [0,1]$. In comparison to the previous toy problem, an additional feature of this toy model lies in the fact that the outputs are constrained to lie on the unit circle, no matter what value the wavenumber $k$ assumes. This \emph{toy constraint} is designed to mimic real physical laws (constraints) such as energy balance in fluid flows, and represents a stable statistical property (akin to the robustness of energy spectra in fluid flows).

For this toy problem, our analysis suggests that a deterministically trained model will collapse to $(0,0)$ for large $k$. In contrast, the analysis presented below suggests that a practically trained conditional diffusion model will instead produce a denoiser,
\[
D_\theta(u;h,\sigma) 
\approx
\begin{cases}
(\cos(2\pi kh),\sin(2\pi kh)), & (k \sim 1), \\
u/|u|, & (k\gg 1).
\end{cases}
\]
The conditional probability distribution corresponding to this denoiser is deterministic for small/moderate $k$, but it is non-deterministic for large $k$. The limiting denoiser as $k\to \infty$ pushes the noise distribution toward a uniform distribution on the circle, which is the correct statistical limit of the above oscillatory map.

\subsubsection{Theory}

We will study $L^\ast$-Lipschitz minimizers of the deterministic loss,
\[
\cJ^{(k)}_\mathrm{det}(\Psi_\theta) 
= 
\E_{h\sim \bar{\mu}} | \Psi_\theta(h) - \sol^{(k)}(h) |^2,
\]
with $\bar{\mu} = \cU([0,1])$ the uniform measure, and the conditional diffusion training objective,
\[
\cJ^{(k)}(D_\theta) = \E_{h\sim \bar{\mu}}\E_{u \bb h} \E_{\eta\sim \cN(0,\sigma^2)} | D_\theta(u+\eta;h,\sigma) - u |^2,
\]
where $u \bb h = \sol^{(k)}(h)$. The length scale hypothesis in the Main Text suggests that, for some length scale ${\tilde{\epsilon}} > 0$, and $\delta h\sim \cU([-{\tilde{\epsilon}},{\tilde{\epsilon}}])$, the Lipschitz optimizer of $\cJ^{(k)}_{\mathrm{det}}$ satisfies $\Psi_{\theta^\ast}(h) 
\approx 
\E_{\delta h}\left[
\sol^\Delta(h+\delta h)
\right]$, and that the constrained optimizer of $\cJ^{(k)}$ is approximately equal to the optimal denoiser for $\mathrm{Law}_{\delta h}\left[
\sol^\Delta(h+\delta h)
\right]$. Our goal is to make this intuition rigorous, via asymptotic analysis as $k\to \infty$. 

\paragraph{Deterministic Models Collapse to the Mean.} We denote the constraind optimizer of the deterministic problem formulation by,
\[
\Psi^{(k)} := \argmin_{\Lip(\Psi_\theta)\le L^\ast} \cJ_{\mathrm{det}}^{(k)}(\Psi_\theta).
\]
We note that the underlying map $\sol^{(k)}$ has length-scale $1/k$, whereas the smallest length scale of $\Psi^{(k)}$ is uniformly bounded due to the imposed Lipschitz bound. We now denote 
\[
\tilde{\Psi}^{(k)}(h) := \E_{\delta h}\left[ 
\sol^{(k)}(h + \delta h)
\right], 
\quad
\delta h \sim \cU([-{\tilde{\epsilon}},{\tilde{\epsilon}}]).
\]
We expect that $\Psi^{(k)}(h) \approx \tilde{\Psi}^{(k)}(h)$, under the assumptions of the Main Text. We next confirm this intuition, by showing that $\Psi^{(k)}$ and $\tilde{\Psi}^{(k)}$ are asymptotically equivalent, as $k\to \infty$.

\begin{proposition}
\label{prop:toy2-det}
With the definitions above, we have
\[
\lim_{k \to \infty}
\E_{h\sim \bar{\mu}} |\Psi^{(k)}(h) - \tilde{\Psi}^{(k)}(h)|^2 
=
0.
\]
In fact, both $\Psi^{(k)}, \tilde{\Psi}^{(k)} \to 0$ collapse to $0$ in $L^2(\bar{\mu})$ as $k\to \infty$.
\end{proposition}

The last proposition rigorously justifies the approximate identity $\Psi^{(k)}(h) \approx \E_{\delta h}\left[ \sol^{(k)}(h+\delta h)\right]$, and shows that the optimal constrained model $\Psi^{(k)}$ collapses to the mean (zero), in the limit $k\to \infty$.  

\paragraph{Probabilistic Models Predict Uncertainty.} We next consider the probabilistic problem formulation of conditional diffusion models. We denote by
\[
D^{(k)}(u_\sigma;h,\sigma) := \argmin_{\Lip(D_\theta)\le L^\ast} \cJ^{(k)}(D_\theta),
\]
the optimal constrained denoiser for $\mu^{(k)}(du\bb h) = \delta(u-\sol^{(k)}(h))$. Given the discussion in the Main Text, we consider the conditional probability,
\[
\nu^{(k)}(du \bb h) := \mathrm{Law}_{\delta h}\left[
\sol^{(k)}(h+\delta h)
\right], 
\quad 
\delta h \sim \cU([-{\tilde{\epsilon}},{\tilde{\epsilon}}]),
\]
for which $\nu^{(k)}(du \bb h) \approx \mathrm{Law}_{\delta h}\left[ \sol^{(k)} (h+\delta h)\right]$, consistent with the discussion in the Main Text. 
We define
\[
\tilde{D}^{(k)}(u_\sigma;h,\sigma)
:=
\argmin_{D} \E_{h\sim \bar{\mu}} \E_{u\sim \nu^{(k)}(\slot\bb h)} \E_{\eta \sim \cN(0,\sigma^2)}
\Vert 
D(u+\eta;h,\sigma) - u
\Vert^2.
\]
to be the optimal (unconstrained) denoiser of $\nu^{(k)}(du\bb h)$. Similar to toy model \#1, we will show that $\nu^{(k)}(du \bb h)$ is asymptotically equivalent to $\mu(du \bb h)$, the conditional distribution arising from the joint uniform probability $\mu = \cU(\mathbb{S}^1) \otimes \cU([0,1])$ on $\mathbb{S}^1\times [0,1]$, with $\mathbb{S}^1\subset \mathbb{R}^2$ denoting the unit circle.
The following result shows that the optimal constrained denoiser for $\mu^{(k)} = \delta(u-\sol^{(k)}(h))$ is asymptotically equivalent to the optimal (unconstrained) denoiser for $\nu^{(k)}(du \bb h) = \mathrm{Law}_{\delta h}\left[\sol^{(k)}(h+\delta h) \right]$, up to an error term that is exponentially small in the Lipschitz constant $L^\ast$:
\begin{proposition}
\label{prop:toy2-prob}
Let $\mu := \cU(\mathbb{S}^1)\otimes \cU([0,1])$.
With the definitions above, and for constant $L^\ast$ sufficiently large, we have
\[
\limsup_{k \to \infty} \E_{(u,h)\sim \mu} \E_{\eta \sim \cN(0,\sigma^2)}
\left\Vert
D^{(k)}(u+\eta;h,\sigma) - \tilde{D}^{(k)}(u+\eta;h,\sigma)
\right\Vert^2 
\le Ce^{-L^\ast/8C},
\]
with $C$ and $L^\ast$ independent of $\sigma > 0$.
\end{proposition}
 We note that the appearance of the exponentially small additional term is due to the fact that the limiting denoiser is not uniformly Lipschitz continuous at the origin. However, since the origin is far from the data manifold, this mismatch only leads to a very small error contribution. Thus, we argue that also in this case, this analysis can justify the conclusion drawn from the length scale hypothesis of the Main Text, in the asymptotic regime $k\to \infty$.

\subsection{Mathematical Derivation}

\subsubsection{Characterizing the Optimal Denoiser}
\label{sec:c2}

The simple form of the forward process 
\begin{align}
\label{eq:usigma}
u_\sigma  = u + \eta, \quad \eta \sim \cN(0,\sigma^2 I),
\end{align}
with $\eta$ independent of $u$ and $\bar{u}$,
 allows for explicit solution of the optimal denoiser, giving us insight into its mathematical properties. If $p(u \bb \bar{u})$ is the conditional distribution of $u$ given the initial data $\bar{u}$, then the diffusion process defines $u_\sigma \bb (u, \bar{u})$ as a Gaussian random variable. The denoiser and its gradient are closely related to the posterior distribution of $u \bb (u_\sigma, \bar{u})$:
\begin{lemma}
\label{lem:Dformula}
Assume that $u \sim p(\slot\bb \bar{u})$, and $u_\sigma$ is obtained by the forward process \eqref{eq:usigma}. Then the minimizer of $D_\opt = \argmin_D \cJ(D,\sigma)$ (cp. \eqref{eq:Jdiff}) is given by
\begin{align}
\label{eq:D1}
D_\opt(u_\sigma;\bar{u},\sigma) = \E[u \bb (u_\sigma, \bar{u})] .
\end{align}
The posterior distribution $u\bb (u_\sigma, \bar{u}) \sim q_\sigma(u;u_\sigma,\bar{u})$ is given by the following mathematical expression,
\begin{align}
\label{eq:posterior}
q_\sigma(u; w, \bar{u}) = \frac{1}{Z_\sigma} e^{-|u-w|^2/2\sigma^2} p(u \bb \bar{u}),
\qquad Z_\sigma = \int e^{-|u-w|^2/2\sigma^2} p(u\bb \bar{u} ) \, du.
\end{align}
\end{lemma}
For completeness, we include a proof of Lemma~\ref{lem:Dformula} after the statement of Proposition~\ref{prop:1}, below.
In words, the explicit formula in Lemma~\ref{lem:Dformula} tells us the following: Given that the noise process has ended up at location $u_\sigma$ and given the additional information about $\bar{u}$, the denoiser considers the distribution of all possible origins $u\bb (u_\sigma, \bar{u})$ of the noise process over the distribution $p(u \bb \bar{u})$ and it singles out the most likely origin as the expected value over this distribution, i.e.\ the value $D_\opt(u_\sigma;\bar{u},\sigma) = \E[u\bb (u_\sigma,\bar{u})]$. 

We also remark the following corollary, which is immediate from Lemma~\ref{lem:Dformula}:
\begin{corollary}
\label{cor:bounded}
If $p(\slot\bb \bar{u})$ is supported on a bounded set $\{|u|\le M\}$, then $\left| D_\opt(u_\sigma;\bar{u},\sigma)\right|\le M$ for all $u_\sigma$ and $\sigma$.
\end{corollary}

It is interesting to consider the limit $\sigma \to 0$ of the posterior \eqref{eq:posterior}. We fix $w$ independently of $\sigma$, and consider the limiting behavior as $\sigma \to 0$, conditioned on the event $u_\sigma = w$. In this limit, the enumerator and denominator individually (formally) converge to $w p(w \bb \bar{u})$ and $p(w\bb \bar{u})$, respectively. Thus if we evaluate the optimal denoiser \eqref{eq:D1} at $u_\sigma = w$, we expect $D_\opt(w;\bar{u}, \sigma) \to w$ as $\sigma \to 0$. This is true, if $p(w \bb \bar{u}) > 0$ and if e.g. $w \mapsto p(w\bb \bar{u})$ is continuous. However, in general, $p(u\bb \bar{u})$ may be $0$ in some locations, or may even be supported on a lower-dimensional data manifold. In this case, we may have $p(w\bb \bar{u}) = 0$ at $w$, and the behavior of $\lim_{\sigma \to 0} D_\opt(w;\bar{u},\sigma)$ is unclear, at first sight. We next show that $D_\opt(w;\bar{u},\sigma)$ converges to the closest point in the support of $p(\slot\bb \bar{u})$: 
\begin{proposition}
\label{prop:1}
Fix $w\in \R^d$. Assume that there exists a unique closest point $w^\ast\in \R^d$ in the support of $p(u \bb \bar{u})$, i.e.\ $w^\ast = \argmin_{u\in \supp(p(\slot \bb \bar{u}))} | w - u |$. Then,
\[
D_\opt(w;\bar{u},\sigma) = \E_{u\sim q_\sigma(\slot;w,\bar{u})}[u] \to w^\ast,
\qquad \text{as } \sigma \to 0,
\]
i.e.\ in this limit the optimal denoiser $D_\opt(w;\sigma=0)$ evaluated at $w$, points to the closest point $w^\ast$ in the support of $p(u\bb \bar{u})$.
\end{proposition}

We are often interested in comparing optimal denoisers between two probability measures $\mu$ and $\mu^\Delta$. We end this section by stating two results that allow us to relate the distance between optimal denoisers to the Wasserstein distance between $\mu$ and $\mu^\Delta$. A first result is provided by the previously stated Proposition~\ref{prop:denoiser-est}. 
Calculations on specific examples (e.g. toy model \#2) show that the optimal denoiser $D_{\mathrm{opt}}$ can be singular in the limit $\sigma \to 0$, in the sense that the local Lipschitz constant may blow up in certain locations. However, under suitable hypotheses on the data distribution, this only happens at a positive distance from the data distribution. The following Proposition generalizes Proposition~\ref{prop:denoiser-est} to allow for this possibility.

\begin{proposition}
\label{prop:denoiser-est2}
With the notation of Proposition~\ref{prop:denoiser-est}, assume that there exists a set $A \subset \R^d$ such that the restriction $D_{\mathrm{opt}}(\slot;\slot,\sigma)|_{A}$ is $L^\ast$-Lipschitz continuous for all $\sigma > 0$. Let $D^\Delta$ be defined as before. Then 
\[
\cE := \E_{(u,\bar{u})\sim \mu} \E_{\eta\sim \cN(0,\sigma^2)}
\left\Vert
D^\Delta(u+\eta;\bar{u},\sigma) 
-
D_{\mathrm{opt}}(u+\eta;\bar{u},\sigma)
\right\Vert^2,
\]
is upper bounded by 
\begin{align}
\label{eq:denoiser-est2}
\cE \le C\left\{L^\ast W_1(\mu^\Delta, \mu)
+ \mathrm{Prob}_{\mu_\sigma}\left[\R^d \setminus A\right]
+ \mathrm{Prob}_{\mu^\Delta_\sigma}\left[\R^d \setminus A\right]\right\},
\end{align}
where $C = C(M)> 0$ is a constant depending only on $M$.
\end{proposition}

\subsubsection{Proofs for Section~\ref{sec:c2}.}

\begin{proof}[Proof of Lemma~\ref{lem:Dformula}]
We recall that by definition, $D_\opt(u;\bar{u},\sigma)$ minimizes the functional 
\[
\cJ(D,\sigma)
= 
\E_{u\sim p(\slot \bb \bar{u})} \E_{\eta \sim \cN(0,\sigma^2)} 
\Vert D(u + \eta; \bar{u}, \sigma) - u \Vert^2.
\]
We now replace the expectation over $(u,\eta)\sim p(u \bb \bar{u})\otimes \cN(0,\sigma^2 I)$ by the expectation over $(u,u_\sigma)\bb \bar{u}$, where $u_\sigma \bb (u,\bar{u}) = u + \eta$ is obtained from the noise process and $u \bb \bar{u}\sim p(u\bb \bar{u})$. Then,
\[
\cJ(D,\sigma)
= 
\E_{(u,u_\sigma)\bb \bar{u}} \Vert D(u_\sigma; \bar{u}, \sigma) - u \Vert^2.
\]
Next, we note that $D(u_\sigma;\bar{u}, \sigma)$ depends on $u_\sigma$, but not on $u$. This motivates splitting the expectation up as $\E_{(u,u_\sigma)\bb \bar{u}} = \E_{u_\sigma \sim p_\sigma(\slot \bb \bar{u})} \E_{u|(u_\sigma,\bar{u})}$, to obtain,
\begin{align*}
\cJ(D,\sigma)
&= 
\E_{u_\sigma \sim p_\sigma(\slot \bb \bar{u})} \E_{u|(u_\sigma,\bar{u})} \Vert D(u_\sigma; \bar{u}, \sigma) - u \Vert^2
\\
&=
\E_{u_\sigma \sim p_\sigma(\slot \bb \bar{u})}  \Vert D(u_\sigma;\bar{u}, \sigma) - \E_{u|(u_\sigma,\bar{u})}[u] \Vert^2 
+ \E_{u_\sigma \sim p_\sigma(\slot \bb \bar{u})}\mathrm{Var}_{u\bb (u_\sigma,\bar{u})}[u],
\end{align*}
where the last identity follows from a simple bias-variance decomposition. Since the last term is independent of $D(u_\sigma;\bar{u},\sigma)$, it follows that $\cJ(D,\sigma)$ is minimized by the choice $D(u_\sigma;\bar{u},\sigma) = \E_{u\bb (u_\sigma,\bar{u})}[u]$. The formula for the posterior follows by a straightforward calculation from Bayes formula,
\[
p(u\bb u_\sigma, \bar{u}) \propto p(u_\sigma \bb u,\bar{u}) p(u\bb \bar{u}).
\]
\end{proof}

\begin{proof}[Proof of Proposition~\ref{prop:1}]
Let $w\in \R^d$ be given, and let $w^\ast$ denote the closest point to $w$ in the support of $p(\slot \bb \bar{u})$. Let $q_\sigma(u;w,\bar{u})$ denote the posterior measure \eqref{eq:posterior}. Since $D_\opt(w; \bar{u},\sigma) = \int u q_\sigma(u;w,\bar{u}) \, du$ and $w^\ast = \int w^\ast q_\sigma(u;w,\bar{u}) \, du$, we have
\begin{align*}
\left| D_\opt(w;\bar{u},\sigma) - w^\ast \right|
&\le
\int |u - w^\ast| q_\sigma(u;w,\bar{u}) \, du
\\
&=
\frac{\int |u-w^\ast| e^{-|u-w|^2/2\sigma^2} p(u\bb \bar{u}) \, du}{\int e^{-|u-w|^2/2\sigma^2} p(u\bb \bar{u}) \, du}. 
\end{align*}
Denote $r:= |w - w^\ast|$, and let $\epsilon > 0$ be given. Since $w^\ast$ is the unique closest point to $w$, in the support of $p(\slot \bb \bar{u})$, it follows that there exists $\delta > 0$, such that $|u - w| < r+\delta$ implies that $|u - w^\ast| < \epsilon$.\footnote{If not, then there exists a sequence $u_n\in \mathrm{supp}(p(\slot \bb \bar{u}))$, such that $|u_n - w| \le r+\frac1n$, while at the same time $|u_n - w^\ast|\ge \epsilon > 0$. This sequence must have a limit point $u^\ast$, necessarily belonging to the (closed) support of $p(\slot \bb \bar{u})$, $|u^\ast - w| \le \limsup_n |u_n - w|  = r$, and $|u^\ast - w^\ast| \ge \epsilon > 0$; thus, $u^\ast$ is as close to $w$ as $w^\ast$, contradicting the uniqueness of $w^\ast$.}
 Then,
\begin{align*}
\int |u-w^\ast| e^{-|u-w|^2/2\sigma^2} p(u\bb \bar{u}) \, du
&=
\int_{r\le |u-w|< r+\delta} |u-w^\ast| e^{-|u-w|^2/2\sigma^2} p(u\bb \bar{u}) \, du
\\
&\qquad +
\int_{|u-w|\ge r+\delta} |u-w^\ast| e^{-|u-w|^2/2\sigma^2} p(u\bb \bar{u}) \, du
\\
&\le \epsilon \int e^{-|u-w|^2/2\sigma^2} p(u\bb \bar{u}) \, du 
 + e^{-(r+\delta)^2/2\sigma^2} \E_{u\sim p(\slot\bb \bar{u})}[|u-w^\ast|],
\end{align*}
and 
\begin{align*}
\int e^{-|u-w|^2/2\sigma^2} p(u\bb \bar{u}) \, du
&\ge 
\int_{r\le |u-w| \le r+\delta/2} e^{-|u-w|^2/2\sigma^2} p(u\bb \bar{u}) \, du
\\
&\ge
e^{-(r+\delta/2)^2/2\sigma^2}
\int_{r\le |u-w| \le r+\delta/2} p(u\bb \bar{u}) \, du
\\
&\ge 
e^{-(r+\delta/2)^2/2\sigma^2}
\int_{|u-w^\ast| \le \delta/2} p(u\bb \bar{u}) \, du.
\end{align*}
We note that $\int_{|u-w^\ast| \le \delta/2} p(u\bb \bar{u}) \, du > 0$, since $w^\ast$ belongs to the support of $p$.
It follows that 
\begin{align*}
\frac{\int |u-w^\ast| e^{-|u-w|^2/2\sigma^2} p(u\bb \bar{u}) \, du}{\int e^{-|u-w|^2/2\sigma^2} p(u\bb \bar{u}) \, du}
&\le 
\frac{ \epsilon \int e^{-|u-w|^2/2\sigma^2} p(u\bb \bar{u}) \, du }{\int e^{-|u-w|^2/2\sigma^2} p(u\bb \bar{u}) \, du}
\\
&\qquad +
\frac{  e^{-(r+\delta)^2/2\sigma^2} \E_{u\sim p(\slot\bb \bar{u})}[|u-w^\ast|] }{\int e^{-|u-w|^2/2\sigma^2} p(u\bb \bar{u}) \, du}
\\
&\le \epsilon + \frac{e^{-(r+\delta)^2/2\sigma^2} \E_{u\sim p(\slot\bb \bar{u})}[|u-w^\ast|] }{e^{-(r+\delta/2)^2/2\sigma^2}
\int_{|u-w^\ast| \le \delta/2} p(u\bb \bar{u}) \, du}.
\end{align*}
Letting $\sigma \to 0$, the last term converges to $0$ on account of the fact that
\[
e^{-(r+\delta)^2/2\sigma^2} \ll e^{-(r+\delta/2)^2/2\sigma^2}.
\]
Thus, 
\[
\limsup_{\sigma \to 0} \frac{\int |u-w^\ast| e^{-|u-w|^2/2\sigma^2} p(u\bb \bar{u}) \, du}{\int e^{-|u-w|^2/2\sigma^2} p(u\bb \bar{u}) \, du}
\le \epsilon.
\]
Since $\epsilon > 0$ was arbitrary, and the left-hand side is independent of $\epsilon$, we conclude that 
\[
\lim_{\sigma\to 0} \left|D_\opt(w;\bar{u},\sigma) - w^\ast \right|
\le
\lim_{\sigma\to 0} \frac{\int |u-w^\ast| e^{-|u-w|^2/2\sigma^2} p(u\bb \bar{u}) \, du}{\int e^{-|u-w|^2/2\sigma^2} p(u\bb \bar{u}) \, du}
= 0.
\]
This concludes our proof.
\end{proof}

The following lemma will be used in the proof of Proposition~\ref{prop:denoiser-est} and \ref{prop:denoiser-est2}.

\begin{lemma}
\label{lem:convex-min}
Let $\cA \subset H$ be a convex set in a Hilbert space $H$. Let $\cJ(D) = \Vert D - F\Vert^2$ be a quadratic functional on $\cA$, where $F \in H$. If $D_{\mathrm{opt}} \in \argmin_{D\in \cA} \cJ(D)$, then 
\[
\cJ(D) - \cJ(D_{\mathrm{opt}}) \ge \Vert D - D_{\mathrm{opt}} \Vert^2, \quad \forall \, D \in \cA.
\]
\end{lemma}

\begin{proof}[Proof of Lemma~\ref{lem:convex-min}]
Fix $D\in \cA$ and let $D_\tau := (1-\tau) D_{\mathrm{opt}} + \tau D$. Since $\cA$ is convex, we have $D_\tau \in \cA$ for all $\tau\in [0,1]$. Since $D_{\mathrm{opt}}$ is a minimizer of $\cJ$, it follows that $\frac{d}{d\tau}|_{\tau=0} \cJ(D_\tau) \ge 0$. Evaluating the derivative, this implies, 
\[
\frac{d}{d\tau}\Big|_{\tau=0} \cJ(D_\tau)
=
2\langle \dot{D}_\tau, D_{\mathrm{opt}} - F \rangle
=
2\langle D - D_{\mathrm{opt}}, D_{\mathrm{opt}} - F \rangle \ge 0,
\quad \forall \, D\in \cA.
\]
Given $D\in \cA$, we now obtain
\begin{align*}
\cJ(D) - \cJ(D_{\mathrm{opt}}) 
&= 
\Vert D - F \Vert^2 - 
\Vert D_{\mathrm{opt}} - F \Vert^2
\\
&= 
\langle (D - F) - (D_{\mathrm{opt}}-F), (D-F) + (D_{\mathrm{opt}}-F) \rangle
\\
&= 
\langle D-D_{\mathrm{opt}}, D + D_{\mathrm{opt}} - 2F \rangle
\\
&= 
\langle D-D_{\mathrm{opt}}, D - D_{\mathrm{opt}} \rangle
+ \underbrace{2 \langle D - D_{\mathrm{opt}}, D_{\mathrm{opt}} - F \rangle}_{\ge 0}
\\
&\ge \Vert D - D_{\mathrm{opt}} \Vert^2.
\end{align*}
\end{proof}

The proof of Proposition~\ref{prop:denoiser-est} will also make use of the following:

\begin{lemma}
\label{lem:Lip2}
Assume $\phi: \bX \to \bX$ is a Lipschitz function, and $\mu, \nu \in {\rm Prob}(\bX)$ are probability measures. Then, 
\[
\left|
\E_{u\sim \mu} \Vert \phi(u) \Vert^2 
- 
\E_{u\sim \nu} \Vert \phi(u) \Vert^2
\right|
\le 
2\Vert \phi \Vert_{L^\infty} \Lip(\phi) W_1(\mu, \nu).
\]
\end{lemma}

\begin{proof}[Proof of Lemma~\ref{lem:Lip2}]

Let $\pi \in {\rm Prob}(\bX \times \bX)$ be an optimal $W_1$-coupling between $\mu$ and $\nu$. Then, 
\begin{align*}
\E_{u\sim \mu} \Vert \phi(u) \Vert^2 - \E_{u\sim \nu} \Vert \phi(u) \Vert^2
&= \int \Vert \phi(u)\Vert^2 \, d\mu(u) - \int \Vert \phi(v) \Vert^2 \, d\nu(v)
\\
&= \int 
\left\{
\Vert \phi(u) \Vert^2 - \Vert \phi(v) \Vert^2 
\right\}
\, d\pi(u,v)
\\
&= 
\int 
\left(
\Vert \phi(u) \Vert + \Vert \phi(v) \Vert
\right)
\left\{
\Vert \phi(u) \Vert - \Vert \phi(v) \Vert 
\right\}
\, d\pi(u,v)
\\
&\le 
\int 
2\Vert \phi \Vert_{L^\infty}
\left\{
\Vert \phi(u)  - \phi(v) \Vert 
\right\}
\, d\pi(u,v)
\\
&\le 
2\Vert \phi \Vert_{L^\infty} \Lip(\phi)
\int 
\left\{
\Vert u-v \Vert 
\right\}
\, d\pi(u,v)
\\
&= 2\Vert \phi \Vert_{L^\infty} \Lip(\phi) W_1(\mu, \nu).
\end{align*}
This proves the claimed bound if $\E_{u\sim \mu} \Vert \phi(u) \Vert^2 \ge \E_{u\sim\nu}\Vert \phi(u) \Vert^2$. For the reverse case, we can simply switch $\mu$ and $\nu$ in the above estimates. The claimed bound thus follows.
\end{proof}

\subsubsection{Proof of Proposition~\ref{prop:denoiser-est}}
\label{sec:denoiser-est}
We now come to the proof of Proposition~\ref{prop:denoiser-est}.

\begin{proof}[Proof of Proposition~\ref{prop:denoiser-est}]
\label{pf:denoiser-est}
The idea is to compare the optimal constrained denoiser $D^\Delta = \argmin_{\Lip(D_\theta)\le L^\ast} \cJ^\Delta(D_\theta)$ with the unconstrained denoiser $D_{\mathrm{opt}} = \argmin_{D} \cJ(D)$, for 
\[
\cJ(D) := \E_{(u,\bar{u})\sim \mu} \E_{\eta \sim \cN(0,\sigma^2)} 
\Vert D(u + \eta; \bar{u}, \sigma) - u \Vert^2.
\]
By the assumptions of this proposition, $D_{\mathrm{opt}}$ is $L^\ast$-Lipschitz continuous. We note that, for any $\sigma>0$,
\begin{align}
\cJ(D^\Delta,\sigma) 
&\le \cJ^\Delta(D^\Delta,\sigma) 
+  
\left|
\cJ(D^\Delta,\sigma) - \cJ^\Delta(D^\Delta,\sigma)
\right|
\notag
\\
&\le  
\cJ^\Delta(D_{\mathrm{opt}},\sigma) 
+  
\left|
\cJ(D^\Delta,\sigma) - \cJ^\Delta(D^\Delta,\sigma)
\right|
\notag
\\
&\le 
\cJ(D_{\mathrm{opt}},\sigma) 
+  
2 \max_{D=D_{\mathrm{opt}},D^\Delta}
\left|
\cJ(D,\sigma) - \cJ^\Delta(D,\sigma)
\right|.
\label{eq:dm4}
\end{align}
To prove the claim, it thus suffices to show that there exists $C>0$, independent of $\Delta$, $L^\ast$ and $\sigma$, such that
\[
2\max_{D=D_{\mathrm{opt}},D^\Delta}
\left|
\cJ(D,\sigma) - \cJ^\Delta(D,\sigma)
\right|
\le C L^\ast W_1(\mu^\Delta,\mu).
\]
To prove such an estimate, we first recall that 
\[
\cJ(D,\sigma) 
= 
\E_{(u,\bar{u})\sim \mu} \E_{\eta\sim \cN(0,\sigma^2)} 
\Vert D(u+\eta;\bar{u},\sigma) - u \Vert^2,
\]
and similarly for $\cJ^\Delta$, except that $(u,\bar{u})\sim \mu$ is replaced by $(u,\bar{u})\sim \mu^\Delta$. Let us now momentarily fix $\eta\in \bX$. Given a choice of either $D=D^\Delta$ or $D=D_{\mathrm{opt}}$, we introduce,
\[
\phi_\eta(u,\bar{u}) := D(u+\eta;\bar{u},\sigma) - u.
\]
The following estimate will hold for either choice of $D=D^\Delta,D_{\mathrm{opt}}$. By assumption on $D^\Delta, D_{\mathrm{opt}}$ being $L^\ast$-Lipschitz, it follows that $\Lip(\phi_\eta) \le \Lip(D)+1 \le L^\ast+1$. By Lemma~\ref{lem:Lip2}, it therefore follows that
\begin{align}
\left|
\cJ(D,\sigma)
-
\cJ^\Delta(D,\sigma)
\right|
&= 
\left|
\E_{(u,\bar{u})\sim \mu} \E_{\eta\sim \cN(0,\sigma^2)} 
\Vert \phi_\eta(u,\bar{u})\Vert^2
-
\E_{(u,\bar{u})\sim \mu^\Delta} \E_{\eta\sim \cN(0,\sigma^2)}
\Vert \phi_\eta(u,\bar{u})\Vert^2
\right|
\notag
\\
&\le 
\E_{\eta\sim \cN(0,\sigma^2)} 
\Big|
\E_{(u,\bar{u})\sim \mu}
\Vert \phi_\eta(u,\bar{u})\Vert^2
-
\E_{(u,\bar{u})\sim \mu^\Delta}
\Vert \phi_\eta(u,\bar{u})\Vert^2
\Big|
\notag
\\
&\le 
2\E_{\eta\sim \cN(0,\sigma^2)}\left[\Vert \phi_\eta \Vert_{L^\infty}\right] \Lip(\phi_\eta) W_1(\mu,\mu^\Delta)
\notag
\\
&\le 2(L^\ast+1)\, \E_{\eta\sim \cN(0,\sigma^2)}\left[\Vert \phi_\eta \Vert_{L^\infty}\right] \, W_1(\mu,\mu^\Delta).
\label{eq:dm3}
\end{align}
Comparing with \eqref{eq:denoiser-est}, we finally need to show that $\E_{\eta\sim \cN(0,\sigma^2)}\left[\Vert \phi_\eta \Vert_{L^\infty}\right]\le B$ is bounded by a constant $B$ independent of $\Delta$. Then \eqref{eq:denoiser-est} holds with constant $2(L^\ast+1)B \le 4L^\ast B =: C$, where we used $L^\ast\ge 1$ to get a simpler bound.

We note that by the explicit formula for $D_{\mathrm{opt}}(u_\sigma;\bar{u},\sigma) = \E[u\bb (u_\sigma,\bar{u})]$ and the assumption that $\mu$ is concentrated on $B_M = \{\Vert u \Vert \le M\}$, it is immediate that $\Vert D_{\mathrm{opt}}(u_\sigma;\bar{u},\sigma)\Vert \le M$ for any choice of $u_\sigma$. In particular, this implies that for $D = D_{\mathrm{opt}}$, we have
\begin{align}
\label{eq:dm2}
\E_{\eta\sim \cN(0,\sigma^2)}\left[
\Vert \phi_\eta \Vert_{L^\infty}\right]
=
\E_{\eta\sim \cN(0,\sigma^2)}\left[
\Vert D_{\mathrm{opt}}(u+\eta;\bar{u},\sigma) - u \Vert_{L^\infty}
\right]
\le 2M.
\end{align}

For $D^\Delta$, we can also show that $\Vert D^\Delta(u+\eta;\bar{u},\sigma)\Vert \le M$. To see this, let us introduce the $M$-truncated mapping,
\[
D^\Delta_{M}(u_\sigma;\bar{u},\sigma)
:=
\begin{cases}
D^\Delta(u_\sigma;\bar{u},\sigma), &\text{if } \Vert D^\Delta(u_\sigma;\bar{u},\sigma)\Vert \le M,
\\
\frac{M \, D^\Delta(u_\sigma;\bar{u},\sigma)}{\Vert D^\Delta(u_\sigma;\bar{u},\sigma)\Vert}
&\text{if } \Vert D^\Delta(u_\sigma;\bar{u},\sigma)\Vert > M.
\end{cases}
\]
Then $D^\Delta_{M}$ is still $L^\ast$-Lipschitz. However, it is easy to see that for any $\Vert u \Vert \le M$ and $\bar{u},\eta \in \bX$, we have
\[
\Vert 
D^\Delta_{M}(u+\eta;\bar{u},\sigma) - u
\Vert 
\le 
\Vert 
D^\Delta(u+\eta;\bar{u},\sigma) - u
\Vert.
\]
Upon taking expectations with respect to $u,\bar{u},\eta$, this in turn implies that $\cJ^\Delta(D^\Delta_{M},\sigma) \le \cJ^\Delta(D^\Delta,\sigma)$. However, $D^\Delta$ is by assumption the minimizer of the functional $\cJ^\Delta$, over the set of $L^\ast$-Lipschitz mappings. By the uniqueness of a minimizer over this (convex) set, and since $D^\Delta_{M}$ is still $L^\ast$-Lipschitz, it follows that $D^\Delta = D^\Delta_{M}$, i.e.\ $D^\Delta$ is uniformly bounded by $M$. Thus, also in this case, we have for $D =D^\Delta$:
\begin{align}
\label{eq:dm1}
\E_{\eta\sim \cN(0,\sigma^2)}\left[
\Vert \phi_\eta \Vert_{L^\infty}\right]
=
\E_{\eta\sim \cN(0,\sigma^2)}\left[
\Vert D^\Delta(u+\eta;\bar{u},\sigma) - u \Vert_{L^\infty}
\right]
\le 2M.
\end{align}

Combining \eqref{eq:dm1}, \eqref{eq:dm2}, \eqref{eq:dm3} and \eqref{eq:dm4}, we conclude that
\[
\cJ(D^\Delta,\sigma)
\le 
\cJ(D_{\mathrm{opt}},\sigma) 
+
C L^\ast W_1(\mu^\Delta, \mu),
\]
for $C = 8M$. Since $D_{\mathrm{opt}}$ is the optimizer of the quadratic functional $\cJ$, it follows from Lemma~\ref{lem:convex-min} that 
\[
\E_{(u,\bar{u})\sim \mu}\E_{\eta\sim \cN(0,\sigma^2)} \Vert D^\Delta(u+\eta;\bar{u},\sigma) - D_{\mathrm{opt}}(u+\eta;\bar{u},\sigma)\Vert^2
\le
\cJ(D^\Delta) - \cJ(D_{\mathrm{opt}}),
\]
and hence
\[
\E_{(u,\bar{u})\sim \mu}\E_{\eta\sim \cN(0,\sigma^2)} \Vert D^\Delta(u+\eta;\bar{u},\sigma) - D_{\mathrm{opt}}(u+\eta;\bar{u},\sigma)\Vert^2
\le CL^\ast W_1(\mu^\Delta, \mu),
\]
by the previous bound. This completes the proof of Proposition~\ref{prop:denoiser-est}.
\end{proof}

\subsubsection{Proof of Proposition~\ref{prop:denoiser-est2}}

\begin{proof}[Proof of Proposition~\ref{prop:denoiser-est2}]
We note that 
\begin{align*}
\cJ(D^\Delta) 
&\le \cJ^\Delta(D^\Delta) + \left| \cJ(D^\Delta) - \cJ^\Delta(D^\Delta) \right|
\\
&\le \cJ^\Delta(D_{\mathrm{opt}}) + \left| \cJ(D^\Delta) - \cJ^\Delta(D^\Delta) \right|
\\
&\le 
\cJ(D_{\mathrm{opt}}) + \left| \cJ(D^\Delta) - \cJ^\Delta(D^\Delta) \right|
+ \left| \cJ(D_{\mathrm{opt}}) - \cJ^\Delta(D_{\mathrm{opt}}) \right|.
\end{align*}
By assumption $D^\Delta$ is $L^\ast$-Lipschitz. Thus, \eqref{eq:dm3} and \eqref{eq:dm2} in the proof of Proposition~\ref{prop:denoiser-est} imply that 
\[
\left|
\cJ(D^\Delta) - \cJ^\Delta(D^\Delta)
\right|
\le 
4M(L^\ast+1) W_1(\mu,\mu^\Delta).
\]

By assumption, $D_{\mathrm{opt}}$ is $L^\ast$-Lipschitz when restricted to $A\subset \R^d$. By the Kirszbraun theorem, there therefore exists $D: \R^d \to \R^d$ with $\Lip(D) = \Lip(D_{\mathrm{opt}}|_{A})$, $\Vert D \Vert_{L^\infty} = \Vert D_{\mathrm{opt}} \Vert_{L^\infty}$ and $D|_{A} \equiv D_{\mathrm{opt}}|_{A}$. Given such a choice of $D$, we now bound
\[
|\cJ(D_{\mathrm{opt}}) - \cJ^\Delta(D_{\mathrm{opt}})|
\le 
|\cJ(D_{\mathrm{opt}}) - \cJ(D)|
+
|\cJ(D) - \cJ^\Delta(D)|
+
|\cJ^\Delta(D) - \cJ^\Delta(D_{\mathrm{opt}})|.
\]
Denote $A^c = \R^d \setminus A$. The first term can be bounded by observing that 
\begin{align*}
\left|\cJ(D_{\mathrm{opt}}) - \cJ(D)\right|
&=
\left|
\E_{u} \Vert D_{\mathrm{opt}}(u_\sigma;\bar{u},\sigma) - u \Vert^2 
-
\E_{u} \Vert D(u_\sigma;\bar{u},\sigma) - u \Vert^2
\right|
\\
&=
\left|
\E_{u}\left[
1_{A^c}(u_\sigma,\bar{u}) \Vert D_{\mathrm{opt}}(u_\sigma;\bar{u},\sigma) - u \Vert^2 
\right]
-
\E_{u}\left[
1_{A^c}(u_\sigma,\bar{u})\Vert D(u_\sigma;\bar{u},\sigma) - u \Vert^2
\right]
\right|
\\
&\le 
(\Vert D_{\mathrm{opt}} \Vert_{L^\infty}+1)^2 \mathrm{Prob}_{\mu_\sigma}\left[ A^c \right].
\end{align*}
Since $D_{\mathrm{opt}}$ is the optimal denoiser for $\mu$, it follows that $\Vert D_{\mathrm{opt}}\Vert_{L^\infty} \le M$, from Corollary~\ref{cor:bounded}. Thus,
\[
\left|\cJ(D_{\mathrm{opt}}) - \cJ(D)\right| \le (M+1)^2 \mathrm{Prob}_{\mu_\sigma}\left[ A^c \right].
\]
Similarly, we can show that
\begin{align*}
\left|\cJ^\Delta(D_{\mathrm{opt}}) - \cJ^\Delta(D)\right|
\le 
(M+1)^2 \mathrm{Prob}_{\mu^\Delta_\sigma}\left[ A^c \right].
\end{align*}
 Finally, \eqref{eq:dm3} and \eqref{eq:dm2} in the proof of Proposition~\ref{prop:denoiser-est} imply that 
\[
\left|
\cJ(D) - \cJ^\Delta(D)
\right|
\le 
4M(L^\ast+1) W_1(\mu,\mu^\Delta).
\]
Combining these estimates, it follows that 
\[
\cJ(D^\Delta) - \cJ(D_{\mathrm{opt}})
\le 
C \left\{ 
L^\ast W_1(\mu,\mu^\Delta) + \mathrm{Prob}_{\mu_\sigma}[A^c] + \mathrm{Prob}_{\mu^\Delta_\sigma}[A^c]
\right\},
\]
where $C = C(M)>0$ depends only on $M$. $D_{\mathrm{opt}}$ is the unconstrained optimizer of the quadratic functional $\cJ$. Thus, by Lemma~\ref{lem:convex-min}, it follows that 
\[
\E_{(u,\bar{u})\sim \mu}\E_{\eta\sim \cN(0,\sigma^2)} \Vert D^\Delta(u+\eta;\bar{u},\sigma) - D_{\mathrm{opt}}(u+\eta;\bar{u},\sigma)\Vert^2
\le
\cJ(D^\Delta) - \cJ(D_{\mathrm{opt}}).
\]
The claimed bound on the error thus follows.
\end{proof}

\subsubsection{Proofs for Section~\ref{sec:toy1}}

We here give the detailed proof of Proposition~\ref{prop:toy1-det} and Proposition~\ref{prop:toy1-prob}. 

\paragraph{Deterministic Setting.}
We start with the proof of Proposition~\ref{prop:toy1-det}. 

\begin{proof}[Proof of Proposition~\ref{prop:toy1-det}]
We first note that $\tilde{\Psi}^\Delta(\bar{u}) = m(\bar{u}) + \E_{\delta \bar{u}}\left[ s_N(\bar{u}+\delta \bar{u})\right]$ converges to $m(\bar{u})$ as $\Delta \to 0$. This follows easily from well-known facts about weak limits, which imply in particular that the rapidly oscillating function $\delta \bar{u} \mapsto s_N(\bar{u}+\delta \bar{u})$ satisfies,
\[
\E_{\delta \bar{u}}\left[ s_N(\bar{u}+\delta \bar{u})\right]
= 
\frac{1}{2{\tilde{\epsilon}}}\int_{-{\tilde{\epsilon}}}^{\tilde{\epsilon}} \Lambda(N\left(\bar{u}+ v\right)) \, dv
\to 
\int_0^1 \Lambda(\xi) \, d\xi = 0,
\]
where the last equality follows from our definition of $\Lambda$. Since this convergence holds pointwise for any fixed $\bar{u}$, by dominated convergence, it follows also in $L^2([0,1])$. Thus, we conclude that $\tilde{\Psi}^\Delta(\bar{u})$ and $m(\bar{u})$ are asymptotically equivalent, in the sense that
\[
\lim_{\Delta \to 0} \E_{\bar{u}\sim \bar{\mu}} \Vert \tilde{\Psi}^\Delta(\bar{u}) - m(\bar{u}) \Vert^2 
= 
0.
\]
It will thus suffice to show,
\[
\lim_{\Delta \to 0} \E_{\bar{u}\sim \bar{\mu}}
\Vert \Psi^\Delta(\bar{u}) - m(\bar{u}) \Vert^2 = 0.
\]
To this end, we first write
\begin{align*}
\E_{\bar{u}\sim \bar{\mu}}
\Vert \Psi^\Delta(\bar{u}) - m(\bar{u}) \Vert^2
&= \E_{\bar{u}\sim \bar{\mu}}
\Vert \Psi^\Delta(\bar{u}) - \sol^\Delta(\bar{u}) \Vert^2
\\
&\quad - 2 \E_{\bar{u}\sim \bar{\mu}}
\langle \Psi^\Delta(\bar{u}) - m(\bar{u}), s_N(\bar{u})\rangle
- \E_{\bar{u}\sim \bar{\mu}}
\Vert s_N(\bar{u}) \Vert^2.
\end{align*}
Since $\Psi^\Delta$ minimizes the first term over all $L^\ast$-Lipschitz functions, and since the mean function is $L^\ast$-Lipschitz by assumption, we obtain,
\begin{align*}
\E_{\bar{u}\sim \bar{\mu}}
\Vert \Psi^\Delta(\bar{u}) - m(\bar{u}) \Vert^2
&\le \E_{\bar{u}\sim \bar{\mu}}
\Vert m(\bar{u}) - \sol^\Delta(\bar{u}) \Vert^2
\\
&\quad - 2 \E_{\bar{u}\sim \bar{\mu}}
\langle \Psi^\Delta(\bar{u}) - m(\bar{u}), s_N(\bar{u})\rangle
- \E_{\bar{u}\sim \bar{\mu}}
\Vert s_N(\bar{u}) \Vert^2
\\
&= - 2 \E_{\bar{u}\sim \bar{\mu}}
\langle \Psi^\Delta(\bar{u}) - m(\bar{u}), s_N(\bar{u})\rangle.
\end{align*}
It is a textbook exercise in analysis to show that $\Psi(\bar{u}) := \Psi^\Delta(\bar{u}) - m(\bar{u})$ is $2L^\ast$-Lipschitz continuous, and that there exists a constant $C>0$, such that 
\[
\sup_{\Lip(\Psi)\le 2L^\ast} \E_{\bar{u}\sim \bar{\mu}} 
\langle \Psi(\bar{u}), s_N(\bar{u}) \rangle 
\le \frac{C}{N} = C\Delta.
\]
Thus, we conclude that $\lim_{\Delta \to 0} \E_{\bar{u}\sim \bar{\mu}} \Vert \Psi^\Delta(\bar{u}) - m(\bar{u}) \Vert^2 = 0$, as claimed.
\end{proof}

\paragraph{Probabilistic Setting.} We now consider the probabilistic setting of conditional diffusion models. Our asymptotic results will hold for any ${\tilde{\epsilon}} > 0$. We thus assume ${\tilde{\epsilon}}$ to be fixed (arbitrarily). We recall that,
\[
\nu^\Delta(du\bb \bar{u}) := \mathrm{Law}_{\delta \bar{u}}\left[
m(\bar{u}) + s_N(\bar{u}+\delta \bar{u})
\right], 
\quad
\delta \bar{u} \sim \cU([-{\tilde{\epsilon}},{\tilde{\epsilon}}]),
\]
and $\tilde{D}^\Delta(u_\sigma;\bar{u},\sigma)$ denotes the optimal unconstrained conditional denoiser for $\nu^\Delta$. Since the derivation of Proposition~\ref{prop:toy1-prob} is more involved, we will first give an overview of the essential ingredients, and leave their proof for later paragraphs.

 Our first result shows that $\nu^\Delta \approx \cU([m(\bar{u})-1,m(\bar{u})+1])$ is approximately equivalent to a uniform distribution, in a suitable sense:
\begin{lemma}
\label{lem:sandwich}
Let $\mu(du \bb \bar{u}) = \cU([m(\bar{u})-1,m(\bar{u})+1])$ be a uniform measure. There exists a constant $C>0$, independent of $\Delta$, such that 
\[
(1-C\Delta) \mu(du\bb \bar{u}) \le \nu^\Delta(du\bb \bar{u}) \le (1+C\Delta) \mu(du\bb \bar{u}).
\]
\end{lemma}

The result of the last lemma is important because it allows us to identify the limit $D_{\mathrm{opt}} = \lim_{\Delta \to 0} \tilde{D}^\Delta$, owing to the following result.

\begin{lemma}
\label{lem:D-sandwich}
Let $\mu,\nu$ be probability measures on $\R^d$, supported on a bounded set $\{|u|\le M\}$ and suppose that for some $\epsilon \in (0,1)$, we have 
\[
(1-\epsilon) \mu \le \nu \le (1+\epsilon) \mu.
\]
Let $D^\mu(u_\sigma;\sigma), D^\nu(u_\sigma;\sigma)$ denote the corresponding denoisers. Then
\[
\Vert D^\mu(\slot;\sigma) - D^\nu(\slot;\sigma) \Vert_{L^\infty(\R^d)} \le 2M \epsilon.
\]
\end{lemma}

Given the results of Lemma~\ref{lem:sandwich} and Lemma~\ref{lem:D-sandwich}, the following corollary is now immediate:

\begin{corollary}
\label{cor:tm1}
Let $\mu$ be the uniform measure on $\cI := \set{(u,\bar{u})\in \R\times [0,1]}{u \in [m(\bar{u})-1,m(\bar{u})+1]}$. Let $D_{\mathrm{opt}}$ denote the optimal (unconstrained) conditional denoiser for $\mu$. Then we have,
\[
\E_{(u,\bar{u}) \sim \mu} \E_{\eta \sim \cN(0,\sigma^2)}
\left\Vert 
\tilde{D}^\Delta(u+\eta;\bar{u},\sigma) 
- 
D_{\mathrm{opt}}(u+\eta;\bar{u},\sigma)
\right\Vert^2 
\le C \Delta.
\]
\end{corollary}

Due to the simplicity of $\mu$, the optimal denoiser $D_{\mathrm{opt}}$ can be characterized quite explicitly, as shown next:
\begin{lemma}
\label{lem:tm1}
Let $D_{\mathrm{opt}}(u;\bar{u},\sigma)$ denote the optimal denoiser for the uniform measure $\mu$ on $\cI$ introduced above. Then $D_{\mathrm{opt}}$ is $L^\ast$-Lipschitz continuous, uniformly as $\sigma \to 0$, for some constant $L^\ast >0$, and 
\[
\lim_{\sigma \to 0} D_{\mathrm{opt}}(u;\bar{u},\sigma)
= g(u-m(\bar{u})),
\]
where
\begin{align}
\label{eq:tm10}
g(u)
=
\begin{cases}
-1, &\text{if }u < - 1, \\
u, &\text{if } -1 \le u \le +1, \\
+1, &\text{if } u > +1. \\
\end{cases}
\end{align}
\end{lemma}

The proof of Lemma~\ref{lem:tm1} is given below. Given the above results, we can now finally come to the proof of Proposition~\ref{prop:toy1-prob}.

\begin{proof}[Proof of Proposition~\ref{prop:toy1-prob}]
We recall that our goal is to show that 
\[
\lim_{\Delta \to 0} \E_{(u,\bar{u})\sim \mu} \E_{\eta \sim \cN(0,\sigma^2)}
\left\Vert
D^\Delta(u+\eta;\bar{u},\sigma) - \tilde{D}^\Delta(u+\eta;\bar{u},\sigma)
\right\Vert^2 = 0.
\]
Corollary~\ref{cor:tm1} shows that $\tilde{D}^\Delta \to D_{\mathrm{opt}}$ with $D_{\mathrm{opt}}$ the conditional diffusion model for $\mu$. It will thus be enough to show that 
\[
\lim_{\Delta \to 0} \E_{(u,\bar{u})\sim \mu} \E_{\eta \sim \cN(0,\sigma^2)}
\left\Vert
D^\Delta(u+\eta;\bar{u},\sigma) - D_{\mathrm{opt}}(u+\eta;\bar{u},\sigma)
\right\Vert^2
= 0.
\]
Since $D_{\mathrm{opt}}$ is $L^\ast$-Lipschitz continuous by Lemma~\ref{lem:tm1}, it follows from Proposition~\ref{prop:denoiser-est} that 
\[
\E_{(u,\bar{u})\sim \mu} \E_{\eta \sim \cN(0,\sigma^2)}
\left\Vert
D^\Delta(u+\eta;\bar{u},\sigma) - D_{\mathrm{opt}}(u+\eta;\bar{u},\sigma)
\right\Vert^2
\le CL^\ast W_1(\mu,\mu^\Delta).
\]
Lemma~\ref{lem:1d-statlim} below shows that $W_1(\mu^\Delta,\mu) \to 0$, completing the proof.
\end{proof}

The following lemma identifies a robust statistical limit for this toy problem.
\begin{lemma}
\label{lem:1d-statlim}
Let $\mu \in {\rm Prob}(\R\times [0,1] )$ be given by the uniform measure on 
\[
\cI(m) := 
\set{(u,\bar{u}) \in \R\times [0,1]}{u \in [m(\bar{u})-1,m(\bar{u})+1]}.
\]
Then, 
\[
W_1(\mu^\Delta, \mu) = O(\Delta) \to 0, \quad \text{ as } \Delta \to 0.
\]
\end{lemma}

Interestingly, Lemma~\ref{lem:1d-statlim} shows that, even though $\sol^\Delta$ is highly oscillatory and cannot possess a limiting function $\sol^\Delta \not \to \sol$, the associated probability measure $\mu^\Delta$ nevertheless converges in a statistical sense to a well-defined limit $\mu$. Disintegration of this limit $\mu$ yields,
\[
\mu(du,d\bar{u}) = p(u\bb \bar{u}) \, du \,d\bar{u},
\]
where $p(u\bb \bar{u}) = \cU [m(\bar{u})-1,m(\bar{u})+1]$ is the uniform distribution on an interval centered around $m(\bar{u})$. In particular, \emph{the limit is not a Dirac $\delta$-distribution}.

\paragraph{Proofs of Lemma~\ref{lem:sandwich}, Lemma~\ref{lem:D-sandwich}, Lemma~\ref{lem:tm1} and Lemma~\ref{lem:1d-statlim}}

In the following, we detail the proofs of Lemma~\ref{lem:sandwich}, Lemma~\ref{lem:D-sandwich}, Lemma~\ref{lem:tm1} and Lemma~\ref{lem:1d-statlim}. 

\begin{proof}[Proof of Lemma~\ref{lem:sandwich}]
Note that $\nu^\Delta$, by definition, is the pushforward measure of $\cU([-{\tilde{\epsilon}},{\tilde{\epsilon}}])$ under the mapping $f_N: [-{\tilde{\epsilon}},{\tilde{\epsilon}}]\to \R$, $f_N(\xi) = m(\bar{u}) + \Lambda(N\bar{u} + N\xi)$. Since $\Lambda$ is a hat function mapping onto the range $[-1,1]$, and since its derivative $|\Lambda'| = 1$ has magnitude $1$ almost everywhere, it follows from the change of variables formula for pushforward measures that $\nu^\Delta(du\bb \bar{u})$ is probability measure on $[m(\bar{u})-1,m(\bar{u})+1]$ with a probability density $q(u)$ whose value at a given $u$ is proportional to the number of points in the pre-image of $u$, i.e.\ 
\[
q(u) = c \, \#\set{\xi\in [-{\tilde{\epsilon}},{\tilde{\epsilon}}]}{f_N(\xi) = u},
\]
holds for almost every $u \in [m(\bar{u})-1,m(\bar{u})+1]$, where $c$ is a normalization constant. Since $f_N$ is $1/N$-periodic, there are $\lfloor 2{\tilde{\epsilon}} N \rfloor$ completed periods over the interval $[-{\tilde{\epsilon}},{\tilde{\epsilon}}]$. On each completed period, the equation $f_N(\xi) = u$ has two solutions for almost every $u\in [m(\bar{u})-1,m(\bar{u})+1]$. Thus, we have
\[
2\lfloor 2{\tilde{\epsilon}} N \rfloor
\le 
\#\set{\xi\in [-{\tilde{\epsilon}},{\tilde{\epsilon}}]}{f_N(\xi) = u}
\le 
2\lfloor 2{\tilde{\epsilon}} N \rfloor + 1,
\]
implying that
\[
2c\lfloor 2{\tilde{\epsilon}} N \rfloor
\le 
q(u)
\le 
2c\lfloor 2{\tilde{\epsilon}} N \rfloor + c,
\quad 
\forall \, u \in [m(\bar{u})-1,m(\bar{u})+1].
\]
Integration over $u$, and using the fact that $\int q(u) \, du = 1$, then implies that $c \sim 1/4\lfloor {\tilde{\epsilon}} N \rfloor$, and hence,
\[
q(u) = \frac12 + O\left(\frac{1}{{\tilde{\epsilon}} N}\right)
= \frac{1}{2} + O({\tilde{\epsilon}}^{-1}\Delta),
\]
with an absolute implied constant in the big-O notation. Since $q_0(u) \equiv \frac12$ for the considered values of $u$ is the density of the uniform distribution $\mu = \cU([m(\bar{u})-1,m(\bar{u})+1])$, we conclude that there exists a constant $C = C({\tilde{\epsilon}})>0$, proportional to $1/{\tilde{\epsilon}}$, such that 
\[
(1-C\Delta) \mu \le \nu^\Delta \le (1+C\Delta) \mu.
\]
\end{proof}

\begin{proof}[Proof of Lemma~\ref{lem:D-sandwich}]
We recall that 
\[
D^\mu(w;\sigma) 
=
\frac{
\int u e^{-|w-u|^2/2\sigma^2} \mu(du)
}{
\int e^{-|w-u|^2/2\sigma^2} \mu(du)
},
\quad
D^\nu(w;\sigma) 
=
\frac{
\int u e^{-|w-u|^2/2\sigma^2} \nu(du)
}{
\int e^{-|w-u|^2/2\sigma^2} \nu(du)
}.
\]
We now compute
\begin{align*}
D^\mu(w;\sigma) - D^\nu(w;\sigma)
&=
\frac{
\int u e^{-|w-u|^2/2\sigma^2} [\mu(du) - \nu(du)]
}{
\int e^{-|w-u|^2/2\sigma^2} \mu(du)
}
+
\frac{
\int u e^{-|w-u|^2/2\sigma^2} \nu(du)
}{
\int e^{-|w-u|^2/2\sigma^2} \nu(du)
}
\left(
\frac{\int e^{-|w-u|^2/2\sigma^2} \nu(du)}{\int e^{-|w-u|^2/2\sigma^2} \mu(du)}
-
1
\right)
\\
&=
\frac{
\int u e^{-|w-u|^2/2\sigma^2} \left[1 - \frac{d\nu}{d\mu} \right] \mu(du)
}{
\int e^{-|w-u|^2/2\sigma^2} \mu(du)
}
+
D^\nu(w;\sigma)
\left(
\frac{\int e^{-|w-u|^2/2\sigma^2} \nu(du)}{\int e^{-|w-u|^2/2\sigma^2} \mu(du)}
-
1
\right).
\end{align*}
By assumption, we have $(1-\epsilon)\mu \le \nu \le (1+\epsilon)\mu$. This implies that the Radon-Nikodym derivative $d\nu/d\mu$ satisfies $-\epsilon \le \frac{d\nu}{d\mu} - 1 \le \epsilon$, implying that 
\[
\left\Vert 1 - \frac{d\nu}{d\mu} \right\Vert_{L^\infty(\mu)}
\le \epsilon.
\]
Furthermore, taking convolution with $e^{-|u|^2/2\sigma^2}$, the inequalities between $\mu$ and $\nu$ also imply,
\[
(1-\epsilon) \int e^{-|w-u|^2/2\sigma^2} \mu(du)
\le
\int e^{-|w-u|^2/2\sigma^2} \nu(du)
\le 
(1+\epsilon) \int e^{-|w-u|^2/2\sigma^2} \mu(du),
\]
and hence,
\[
\left|
\frac{\int e^{-|w-u|^2/2\sigma^2} \nu(du)}{\int e^{-|w-u|^2/2\sigma^2} \mu(du)}
-
1
\right|
\le \epsilon.
\]
Thus, we conclude that 
\begin{align*}
\left|\frac{
\int u e^{-|w-u|^2/2\sigma^2} \left[1 - \frac{d\nu}{d\mu} \right] \mu(du)
}{
\int e^{-|w-u|^2/2\sigma^2} \mu(du)
}\right|
&\le
\frac{
\int |u| e^{-|w-u|^2/2\sigma^2} \mu(du)
}{
\int e^{-|w-u|^2/2\sigma^2} \mu(du)
} 
\left\Vert 1 - \frac{d\nu}{d\mu} \right\Vert_{L^\infty(\mu)}
\le 
M \epsilon,
\end{align*}
and
\begin{align*}
\left|
D^\nu(w;\sigma)
\left(
\frac{\int e^{-|w-u|^2/2\sigma^2} \nu(du)}{\int e^{-|w-u|^2/2\sigma^2} \mu(du)}
-
1
\right)
\right|
&\le
|D^\nu(w;\sigma)| \epsilon
\le M \epsilon,
\end{align*}
where we have used the fact that $\supp(\mu), \supp(\nu) \subset \{|u|\le M\}$ in both estimates. Combining these estimates, we conclude that 
\[
\Vert D^\mu(\slot;\sigma) - D^\nu(\slot;\sigma) \Vert_{L^\infty(\R^d)}
\le 2M\epsilon,
\]
as claimed.
\end{proof}

\begin{proof}[Proof of Lemma~\ref{lem:tm1}]
Due to the problem setup, it is easy to see that the optimal conditional denoiser $D_{\mathrm{opt}}(u_\sigma;\bar{u},\sigma)$ must be a shift by $m(\bar{u})$ of the optimal denoiser for the uniform data distribution $p(u) = \cU[-1,1]$ over $[-1,1]$. It will therefore suffice to prove the statement for the optimal denoiser $D_{\mathrm{opt}}(u;\sigma)$ corresponding to this data distribution. We want to show:
\begin{enumerate}
\item $D_{\mathrm{opt}}(u;\sigma)$ is $L^\ast$-Lipschitz, uniformly as $\sigma \to 0$,
\item $\lim_{\sigma\to 0} D_{\mathrm{opt}}(u;\sigma) = g(u)$, given by \eqref{eq:tm10}.
\end{enumerate}

To prove property (2.) we simply note that, by Lemma~\ref{lem:Dformula}, the optimal denoiser converges to the closest point in the support of the data distribution $p(u) = \cU[-1,1]$. The formula \eqref{eq:tm10} is then immediate.

It remains to prove the uniform $L^\ast$-Lipschitz bound. To this end, we recall that, by Lemma~\ref{lem:Dformula}, $D_{\mathrm{opt}}(w;\sigma)$ is given by
\[
D_{\mathrm{opt}}(w;\sigma)
=
\frac{\int_{-1}^{+1} u e^{-(u-w)/2\sigma^2} \, du}{\int_{-1}^{+1} e^{-(u-w)/2\sigma^2} \, du}.
\]
If we denote $\underline{u} = D_{\mathrm{opt}}(w;\sigma)$ for simplicity, then a short calculation implies that 
\[
|D_{\mathrm{opt}}'(w;\sigma)|
=
\frac{
\int_{-1}^{+1} (u-\underline{u})^2 e^{-(u-w)^2/2\sigma^2} \, du
}{
\sigma^2 \int_{-1}^{+1} e^{-(u-w)^2/2\sigma^2} \, du
}.
\]
We will prove that $|D_\ast'(w;\sigma)| \le L^\ast$ is uniformly bounded in $\sigma$ and $w$. From the above formula, this is immediate for large $\sigma$; e.g. for $\sigma > 2$, we have 
\begin{align}
\label{eq:dl}
|D_{\mathrm{opt}}'(w;\sigma)|
&=
\frac{
\int_{-1}^{+1} (u-\underline{u})^2 e^{-(u-w)^2/2\sigma^2} \, du
}{
\sigma^2 \int_{-1}^{+1} e^{-(u-w)^2/2\sigma^2} \, du
}
\le 
\frac{
\int_{-1}^{+1} 4 e^{-(u-w)^2/2\sigma^2} \, du
}{
4\int_{-1}^{+1} e^{-(u-w)^2/2\sigma^2} \, du
}
=1, 
\qquad (\sigma >2).
\end{align}
 To establish an upper bound for $\sigma \in (0,2]$, we will distinguish between exterior points $\{w < -1\}, \{w>+1\}\subset [-1,1]^c$, and interior points $\{-1<w<+1\}\subset [-1,1]$.

\textbf{Exterior:} We first consider the exterior domain $\{w<-1\}$. Fix $\xi > 0$, and set $w = -1-\xi$. Our goal is to bound $D_{\mathrm{opt}}'(-1-\xi;\sigma)$ for all $\xi>0$ and $\sigma \in (0,2]$. Under the current assumptions, we have
\[
|D_{\mathrm{opt}}'(-1-\xi;\sigma)|
=
\frac{
\int_{-1}^{+1} (u-\underline{u})^2 e^{-(u+1+\xi)^2/2\sigma^2} \, du
}{
\sigma^2 \int_{-1}^{+1} e^{-(u+1+\xi)^2/2\sigma^2} \, du
}
\]
After a change of variables $u \to u-1$ and noting that $\underline{u}$ minimizes the quadratic variation, it follows that
\[
|D_{\mathrm{opt}}'(-1-\xi;\sigma)|
\le 
\frac{
\int_0^2 u^2 e^{-(u+\xi)^2/2\sigma^2} \, du
}{
\sigma^2 \int_0^2 e^{-(u+\xi)^2/2\sigma^2} \, du
}.
\]
Expanding $(u+\xi)^2 = u(u+2\xi) + \xi^2$, we can write 
\[
|D'_\ast(-1-\xi;\sigma)|
\le
\frac{
\int_0^2 u^2 e^{-u(u+2\xi)/2\sigma^2} \, du
}{
\sigma^2 \int_0^2 e^{-u(u+2\xi)/2\sigma^2} \, du
}
\le 
\frac{
\int_0^\infty u^2 e^{-u(u+2\xi)/2\sigma^2} \, du
}{
\sigma^2 \int_0^2 e^{-u(u+2\xi)/2\sigma^2} \, du
}
\]
It will be convenient to estimate the denominator in terms of an integration over $[0,\infty)$ instead of $[0,2]$. To this effect, we note that 
\[
\int_2^\infty e^{-u(u+2\xi)/2\sigma^2} \, du
=
\int_0^\infty e^{-(2+u)(2+u+2\xi)/2\sigma^2} \, du
\le 
e^{-2/\sigma^2} \int_0^\infty e^{-u(u+2\xi)/2\sigma^2} \, du.
\]
And thus, 
\[
\int_0^2 e^{-u(u+2\xi)/2\sigma^2} \, du
=
\int_0^\infty e^{-u(u+2\xi)/2\sigma^2} \, du
- \int_2^\infty e^{-u(u+2\xi)/2\sigma^2} \, du
\ge 
(1-e^{-2/\sigma^2}) \int_0^\infty e^{-u(u+2\xi)/2\sigma^2} \, du.
\]
It follows that 
\[
|D'_\ast(-1-\xi;\sigma)|
\le \frac{\cI}{1-e^{-2/\sigma^2}},
\qquad
\cI := \frac{
\int_0^\infty u^2 e^{-u(u+2\xi)/2\sigma^2} \, du
}{
\sigma^2 \int_0^\infty e^{-u(u+2\xi)/2\sigma^2} \, du
}.
\]
We note that $1-e^{-2/\sigma^2} \ge 1-e^{-1/2} > 0$ is uniformly lower bounded for all $\sigma \in (0,2]$. Therefore, to prove a uniform upper bound on $|D'_\ast(-1-\xi;\sigma)|$, it suffices to upper bound $\cI$. 

We now introduce $z := \xi/\sigma^2$, perform a change of variables $u \to \sigma u$, and write 
\[
\cI 
=
\frac{
\int_0^\infty u^2 e^{-u(u+2z)/2} \, du
}{
\int_0^\infty e^{-u(u+2z)/2} \, du
}.
\]
Let us introduce $x = u(u+2z)$, so that 
\[
u = \sqrt{x+z^2}-z = \frac{x}{\sqrt{x+z^2}+z},
\]
 and $dx = 2(u+z)du = 2\sqrt{x+z^2} \, du$. Then, 
\[
\cI 
=
\frac{
\int_0^\infty \frac{x^2 e^{-x/2} \, dx}{(\sqrt{x+z^2}+z)^2 \sqrt{x+z^2}} 
}{
\int_0^\infty \frac{e^{-x/2} \, dx}{\sqrt{x+z^2}} 
}.
\]
To bound this independently of $z\ge 0$, we first consider $z\in [0,1]$, and set $z=0$ in the numerator, $z=1$ in the denominator, to obtain the uniform bound:
\[
\cI|_{z\in [0,1]}
\le 
\frac{
\int_0^\infty \sqrt{x} e^{-x/2} \, dx 
}{
\int_0^\infty \frac{e^{-x/2} \, dx}{\sqrt{x+1}} 
}.
\]
For $z \ge 1$, we observe that,
\begin{align*}
\int_0^\infty \frac{x^2 e^{-x/2} \, dx}{(\sqrt{x+z^2}+z)^2 \sqrt{x+z^2}}
&\le z^{-3} \int_0^\infty x^2 e^{-x/2} \, dx,
\end{align*}
and
\begin{align*}
\int_0^\infty \frac{e^{-x/2} \, dx}{\sqrt{x+z^2}}
\ge \frac{1}{\sqrt{1+z^2}} \int_0^1 e^{-x/2} \, dx
\ge (2z)^{-1} \int_0^1 e^{-x/2} \, dx.
\end{align*}
It follows that 
\[
\cI|_{z\in [1,\infty)} \le \frac{2\int_0^\infty x^2 e^{-x/2} \, dx}{z^2 \int_0^1 e^{-x/2} \, dx}.
\]
The last term is uniformly bounded for $z\in [1,\infty)$, and $\lesssim z^{-2}$ as $z \to \infty$.
Recalling that $z = \xi/\sigma^2$ and $\xi = |w|-1$, these two estimates on $\cI$ imply an upper bound of the form,
\begin{align}
\label{eq:dext}
|D_{\mathrm{opt}}'(w,\sigma)| \le C \left(\frac{\sigma^2}{|w|-1+\sigma^2}\right)^{2},
\qquad
(|w|>1, \; \sigma \in (0,2]).
\end{align}
Technically, we have only proved the above bound for $w <-1$. However, the same upper bound also holds for $w>+1$, by symmetry. From \eqref{eq:dext}, we in fact observe that in the exterior domain, we have $D_{\mathrm{opt}}'(w;\sigma)\to 0$ except potentially at the boundary points $\{-1,+1\}$. This is consistent with the fact that $D_{\mathrm{opt}}(w;\sigma) \to \pm 1$ for $|w|>1$. 

\textbf{Interior:} Our final goal is to bound $D_{\mathrm{opt}}'(w;\sigma)$ in the interior, i.e.\ for all $w\in (-1,1)$. By symmetry about the origin, we may in fact assume that $w \in (-1,0)$. Under these assumptions, we have
\[
|D_{\mathrm{opt}}'(w;\sigma)|
\le 
\frac{
\int_{-1}^{+1} (u-w)^2 e^{-(u-w)^2/2\sigma^2} \, du
}{
\sigma^2 \int_{-1}^{+1} e^{-(u-w)^2/2\sigma^2} \, du
}.
\]
Making the change of variables $u \to u-w$ and noting the set inclusion $w+[0,1] \subset [-1,1]$, it follows that
\[
|D_{\mathrm{opt}}'(w;\sigma)|
\le 
\frac{
\int_{-\infty}^{\infty} u^2 e^{-u^2/2\sigma^2} \, du
}{
\sigma^2 \int_{0}^{+1} e^{-u^2/2\sigma^2} \, du
}
\]
Making the change of variables $\eta = u/\sigma$, and recalling that we consider $\sigma \in (0,2]$, we obtain
\begin{align}
\label{eq:dint}
|D_{\mathrm{opt}}'(w;\sigma)|
=
\frac{
\int_{-\infty}^{\infty} \eta^2 e^{-\eta^2/2} \, d\eta
}{
\int_{0}^{1/\sigma} e^{-\eta^2/2} \, d\eta
}
\le 
\frac{
\int_{-\infty}^{\infty} \eta^2 e^{-\eta^2/2} \, d\eta
}{
\int_{0}^{1/2} e^{-\eta^2/2} \, d\eta
},
\qquad (w\in [-1,1], \; \sigma \in (0,2]).
\end{align}
The right-hand side is independent of $w \in (-1,1)$ and $\sigma \in (0,2]$. Combining \eqref{eq:dl}, \eqref{eq:dext} and \eqref{eq:dint}, we have derived a unifom upper bound $|D_{\mathrm{opt}}'(w;\sigma)|\le L^\ast$, as desired.
\end{proof}

\begin{proof}[Proof of Lemma~\ref{lem:1d-statlim}]
Fix any $\phi \in \Lip_1(\R\times \R)$, such that $\phi(0,0) = 0$. We note that by Kantorovich duality, $W_1(\mu^\Delta, \mu)$ is the supremum of 
\[
\cR^\Delta(\phi)
=
\int \phi(u,\bar{u}) \, d\mu^\Delta - \int \phi(u,\bar{u}) \, d\mu,
\]
over all such $\phi$. By definition of $\mu^\Delta$ and $\mu$, we have
\begin{align*}
\cR^\Delta(\phi)
=
\int_0^1 \phi(\sol^\Delta(\bar{u}), \bar{u}) \, d\bar{u}
-
\int_0^1 \int_{-1}^1 \phi(m(\bar{u})+y, \bar{u}) \, dy \, d\bar{u}.
\end{align*}
To estimate $\cR^\Delta(\phi)$ from above, we first observe that for any bounded function $f: \R \to \R$,
\[
\int_0^1 f(\bar{u}) \, d\bar{u} = \int_0^1 \E_{\eta \sim \cU([0,\epsilon])}
\left[ f(\bar{u} + \eta) \right] \, d\bar{u} 
+ r(f,\epsilon),
\]
where the remainder $r(f;\epsilon)$ can be bounded by $|r(f;\epsilon)|\le 2  \epsilon\Vert f\Vert_{L^\infty}$. This follows from the fact that, for uniformly distributed $\bar{u}\in [0,1]$ and $\eta \in [0,\epsilon]$, the sum $\bar{u}+\eta$ has uniform density $\equiv 1$ over $[\epsilon,1]$, and is supported on $[0,1+\epsilon]$.

Thus, taking $\bar{u}\sim \cU[0,1]$, $\eta \sim \cU[0,\epsilon]$, it follows that 
\[
\int_0^1 \phi(\sol^\Delta(\bar{u}), \bar{u}) \, d\bar{u}
\le \E_{\bar{u}}\E_\eta \phi(\sol^\Delta(\bar{u}+\eta), \bar{u}+\eta)
+ 2 \Vert \phi \Vert_{L^\infty} \epsilon.
\]
We next recall that 
\[
\sol^\Delta(\bar{u}+\eta) = m(\bar{u}+\eta) + s_N(\bar{u}+\eta).
\]
The first term is $L^\ast$-Lipschitz continuous, implying that 
\[
\sol^\Delta(\bar{u}+\eta) = m(\bar{u}) + s_N(\bar{u}+\eta) + O_1(L^\ast \epsilon),
\]
with implied constant for the remainder term bounded by $1$. Let us introduce,
\[
\tilde{\phi}(\bar{u},y) := \phi(m(\bar{u}) + y, \bar{u}).
\]
The last bound combined with the $1$-Lipschitz continuity of $\phi$, then implies that 
\begin{align*}
\phi(
\sol^\Delta(\bar{u}+\eta),
\bar{u}+\eta
) 
&= \phi(m(\bar{u}) + s_N(\bar{u}+\eta), \bar{u}) + O(L^\ast \epsilon)
\\
&= \tilde{\phi}(\bar{u}, s_N(\bar{u}+\eta)) + O(L^\ast \epsilon),
\end{align*}
where the only dependence on $\phi$ of the last term is via the Lipschitz bound $L^\ast$.
The importance of this last expressions is that, if $\eta \sim \cN[0,\epsilon]$ and if $N\epsilon \in \N$ is integer, then the push-forward,
\[
y 
:= 
s_N(\bar{u}+\eta)
= 
\Lambda(N\bar{u} + N \eta),
\]
has uniform distribution $y \sim \cU[-1,1]$, independent of $\bar{u}$. Thus, we may choose $\epsilon := 1/N = \Delta$, for which it then follows that
\begin{align*}
\int_0^1 \phi(\sol^\Delta(\bar{u}), \bar{u}) \, d\bar{u}
&= \E_{\bar{u}}\E_\eta \phi(\sol^\Delta(\bar{u}+\eta), \bar{u}+\eta)
+ O\left(\Vert \phi \Vert_{L^\infty} \Delta\right)
\\
&= \E_{\bar{u}}\E_\eta \tilde\phi(\bar{u}, s_N(\bar{u}+\eta))
+ O\left(\left(L^\ast + \Vert \phi \Vert_{L^\infty}\right) \Delta\right)
\\
&= \E_{\bar{u}}\E_{y} \tilde\phi(\bar{u}, y)
+ O\left(\left(L^\ast + \Vert \phi \Vert_{L^\infty}\right) \Delta\right)
\\
&= \int_0^1 \int_{-1}^{1} \tilde\phi(\bar{u}, y)\, dy \, d\bar{u}
+ O\left(\left(L^\ast + \Vert \phi \Vert_{L^\infty}\right) \Delta\right).
\end{align*}
Recalling the definition of $\tilde{\phi}(\bar{u},y) = \phi(m(\bar{u})+y, \bar{u})$, it follows that
\begin{align*}
\cR^\Delta(\phi)
&= 
\int_0^1 \phi(\sol^\Delta(\bar{u}), \bar{u}) \, d\bar{u} 
-
 \int_0^1 \int_{-1}^{1} \phi(m(\bar{u})+y, \bar{u}) \, dy \, d\bar{u}
 \\
 &=
 O\left(\left(L^\ast + \Vert \phi \Vert_{L^\infty}\right) \Delta\right).
\end{align*}
Taking the supremum over all $\phi \in \Lip_1$ such that $\phi(0,0) = 0$, we conclude that 
\[
W_1(\mu^\Delta, \mu) 
= 
\sup_{\phi} \cR^\Delta(\phi)
\le C L^\ast \Delta.
\]
This proves the claim.
\end{proof}

\subsubsection{Proofs for Section~\ref{sec:toy2}}
We next provide the proof of Proposition~\ref{prop:toy2-det} and Proposition~\ref{prop:toy2-prob}. 

\paragraph{Deterministic Setting.}
We start with the proof of Proposition~\ref{prop:toy2-det}. 

\begin{proof}[Proof of Proposition~\ref{prop:toy2-det}]
The proof is very similar to the proof of Proposition~\ref{prop:toy1-det}. We first note that $\sol^{(k)}(h)$ is highly oscillatory and has mean zero, implying that
\[
\tilde{\Psi}^{(k)}(h) = 
\E_{\delta h}\left[
\sol^{(k)}(h+\delta h)
\right]
\to 0, 
\]
in $L^2([0,1])$. To complete the proof, it will thus suffice to show,
\[
\lim_{k \to \infty} \E_{h\sim \bar{\mu}}
\Vert \Psi^{(k)}(h) \Vert^2 = 0.
\]
To this end, we simply note that 
\begin{align*}
\E_{h\sim \bar{\mu}}
\Vert \Psi^{(k)}(h)\Vert^2
&= \E_{h\sim \bar{\mu}}
\Vert \Psi^{(k)}(h) - \sol^{(k)}(h) \Vert^2
+ 2 \E_{h\sim \bar{\mu}}
\langle \Psi^{(k)}(h), \sol^{(k)}(h) \rangle
- \E_{h\sim \bar{\mu}}
\Vert \sol^{(k)}(h) \Vert^2.
\end{align*}
Since $\Psi^\Delta$ minimizes the first term over all $L^\ast$-Lipschitz functions, we can compare with the $0$ function to obtain,
\begin{align*}
\E_{h\sim \bar{\mu}}
\Vert \Psi^{(k)}(h)\Vert^2
&\le \E_{h\sim \bar{\mu}}
\Vert 0 - \sol^{(k)}(h) \Vert^2
+ 2 \E_{h\sim \bar{\mu}}
\langle \Psi^{(k)}(h), \sol^{(k)}(h) \rangle
- \E_{h\sim \bar{\mu}}
\Vert \sol^{(k)}(h) \Vert^2
\\
&= 2 \E_{h\sim \bar{\mu}}
\langle \Psi^{(k)}(h), \sol^{(k)}(h) \rangle. 
\end{align*}
It is straight-forward to show that there exists a constant $C>0$, such that 
\[
\sup_{\Lip(\Psi)\le L^\ast} \E_{h\sim \bar{\mu}} 
\langle \Psi(h), \sol^{(k)}(h) \rangle 
\le \frac{C}{k}.
\]
Thus, we conclude that $\E_{h\sim \bar{\mu}} \Vert \Psi^{(k)}(h)\Vert^2 \le C/k \to 0$, as claimed.
\end{proof}

\paragraph{Probabilistic Setting.} We now consider the probabilistic setting of conditional diffusion models. Our asymptotic results will hold for any ${\tilde{\epsilon}} > 0$. We thus assume ${\tilde{\epsilon}}$ to be fixed (arbitrarily). We recall that,
\[
\nu^{(k)}(du\bb h) := \mathrm{Law}_{\delta h}\left[
\sol^{(k)}(h+\delta h)
\right], 
\quad
\delta h \sim \cU([-{\tilde{\epsilon}},{\tilde{\epsilon}}]),
\]
and $\tilde{D}^{(k)}(u_\sigma;h,\sigma)$ denotes the optimal unconstrained conditional denoiser for $\nu^{(k)}$. The derivation of Proposition~\ref{prop:toy2-prob} is more involved, so we will first give an overview of the essential ingredients, and leave proofs for the next subsection.

 Our first result shows that $\nu^{(k)} \approx \cU(\mathbb{S}^1)$ is approximately equivalent to a uniform distribution:
\begin{lemma}
\label{lem:sandwich2}
Let $\mu(du \bb \bar{u}) = \cU(\mathbb{S}^1)$ be a uniform measure. There exists a constant $C>0$, independent of $k$, such that 
\[
(1-Ck^{-1}) \mu(du\bb h) \le \nu^{(k)}(du\bb h) \le (1+Ck^{-1}) \mu(du\bb h).
\]
\end{lemma}

Since the proof is completely analogous to the proof of Lemma~\ref{lem:sandwich}, we will not discuss the details in this appendix. The result of the last lemma again allows us to easily identify the limit $D_{\mathrm{opt}} = \lim_{k \to \infty} \tilde{D}^{(k)}$, as an immediate consequence of Lemma~\ref{lem:D-sandwich} and Lemma~\ref{lem:sandwich2}:

\begin{lemma}
\label{lem:tm2}
Let $\mu = \cU(\mathbb{S}^1)\otimes \cU([0,1])$ be the uniform measure on $(u,h) \in \mathbb{S}^1\times [0,1]$. Let $D_{\mathrm{opt}}$ denote the optimal (unconstrained) conditional denoiser for $\mu$. Then we have,
\[
\E_{(u,h) \sim \mu} \E_{\eta \sim \cN(0,\sigma^2)}
\left\Vert 
\tilde{D}^{(k)}(u+\eta;h,\sigma) 
- 
D_{\mathrm{opt}}(u+\eta;h,\sigma)
\right\Vert^2 
\le C k^{-1}.
\]
\end{lemma}

Due to the simplicity of $\mu$, the optimal denoiser $D_{\mathrm{opt}}$ can be computed explicitly, as shown next:

\begin{lemma}
\label{lem:spectral-Dopt}
The optimal denoiser for $\mu=\cU(\mathbb{S}^1)\otimes \cU([0,1])$, is given by
\begin{align}
D_{\mathrm{opt}}(u;h,\sigma) = g_\sigma(|u|) \frac{u}{|u|},
\end{align}
where $g_\sigma: \R \to \R$ is given by $g_\sigma(t) = I_1(t/\sigma^2) / I_0(t/\sigma^2)$, with $I_\alpha(z)$ the modified Bessel function of the first kind and of order $\alpha$,
\[
I_\alpha(z) = \frac{1}{\pi} \int_0^\pi \cos(\alpha \theta) e^{z \cos(\theta)} \, d\theta
\sim \frac{e^{z}}{\sqrt{2\pi z}}, \quad \text{as } z\to \infty.
\]
 In particular, 
 \[
 \lim_{\sigma \to 0} D_{\mathrm{opt}}(u;h,\sigma) = \frac{u}{|u|}.
 \]
\end{lemma}

We include the details of the required calculation to prove Lemma~\ref{lem:spectral-Dopt} in the next subsection. As is clear from the limiting behavior as $\sigma \to 0$, we cannot have a uniform Lipschitz bound at the origin $u=0$ in this case. Thus, we need to better understand the (local) Lipschitz behavior of $D_{\mathrm{opt}}$ in this limit.

\begin{lemma}
\label{lem:spectra-Lip}
Assume $D: \R^d \to \R^d$ is of the form $D(u) = g(|u|) \frac{u}{|u|}$. Let $A_\delta := \{|u|\ge \delta \}$. Then for any $\delta > 0$,
\[
\Lip(D|_{A_\delta})
\le 
\Lip(g|_{[\delta,\infty)})
+
\frac{2\Vert g \Vert_{L^\infty}}{\delta}.
\]
\end{lemma}

Given the explicit form of $D_{\mathrm{opt}}$, we want to apply the last lemma with $g = g_\sigma$ to understand the Lipschitz behavior of $D_{\mathrm{opt}}$ near the origin. It remains to bound the Lipschitz constant of $g_\sigma$ on $[\delta,\infty)$. This is the subject of the next lemma.

\begin{lemma}
\label{lem:g-Lip}
Let $g_\sigma(t) = g(t/\sigma^2)$ where $g(z) := I_1(z) / I_0(z)$ is a quotient of modified Bessel functions of the first kind. There exists a constant $C>0$, such that for any $\delta > 0$,
\[
\Lip(g_\sigma|_{[\delta,\infty)}) \le \frac{C}{\delta}.
\]
\end{lemma}

The following corollary is now immediate:

\begin{corollary}
\label{cor:g-Lip}
Let $D_{\mathrm{opt}}(u;\sigma) = g_\sigma(|u|) \frac{u}{|u|}$ be the optimal denoiser for $\mu = \cU(\mathbb{S}^1)$. There exists a constant $C>0$, such that for any $\delta,\sigma > 0$, and with $A_\delta := \{|u|\ge \delta\}$:
\begin{align}
\Lip(D_{\mathrm{opt}}|_{A_\delta}) \le C \min\left( \frac{1}{\delta}, \frac{1}{\sigma^2} \right).
\end{align}
\end{corollary}

The last corollary is the first ingredient required to apply Proposition~\ref{prop:denoiser-est2}. The second ingredient is contained in the following lemma:

\begin{lemma}
\label{lem:circle-dense}
If $\mu\in \cP(\R^2)$ is a probability measure with $\supp \mu \subset \mathbb{S}^1$ and if $\delta \in (0,1/2]$, then 
\[
\mathrm{Prob}_{\mu_\sigma}\left[
B_\delta(0)
\right]
\le \frac{\delta^2}{2\sigma} e^{-1/8\sigma^2},
\]
where $\mu_\sigma = \mu \ast \cN(0,\sigma^2)$.
\end{lemma}

Given the above results, we can now finally come to the proof of Proposition~\ref{prop:toy2-prob}.

\begin{proof}[Proof of Proposition~\ref{prop:toy2-prob}]

We recall that our goal is to show that 
\[
\limsup_{k \to \infty} \E_{(u,h)\sim \mu} \E_{\eta \sim \cN(0,\sigma^2)}
\left\Vert
D^{(k)}(u+\eta;h,\sigma) - \tilde{D}^{(k)}(u+\eta;h,\sigma)
\right\Vert^2 \le C e^{-L^\ast/8C}.
\]
Lemma~\ref{lem:tm2} shows that $\tilde{D}^{(k)} \to D_{\mathrm{opt}}$ with $D_{\mathrm{opt}}$ the conditional diffusion model for $\mu$. It will thus be enough to show that 
\[
\limsup_{k \to \infty} \E_{(u,h)\sim \mu} \E_{\eta \sim \cN(0,\sigma^2)}
\left\Vert
D^{(k)}(u+\eta;h,\sigma) - D_{\mathrm{opt}}(u+\eta;h,\sigma)
\right\Vert^2
\le Ce^{-L^\ast/8}.
\]
Let 
\[
\cE := \limsup_{k \to \infty} \E_{(u,h)\sim \mu} \E_{\eta \sim \cN(0,\sigma^2)}
\left\Vert
D^{(k)}(u+\eta;h,\sigma) - D_{\mathrm{opt}}(u+\eta;h,\sigma)
\right\Vert^2.
\]
By Proposition~\ref{prop:denoiser-est2}, applied with $A := A_\delta$ and $\R^d \setminus A = B_\delta(0)$, we have the upper bound,
\[
\cE \le 
\limsup_{k\to \infty} 
C\left\{ L^\ast W_1(\mu^{(k)},\mu) 
+ \mathrm{Prob}_{\mu_\sigma}[B_\delta(0)]
+ \mathrm{Prob}_{\mu^{(k)}_\sigma}[B_\delta(0)]
\right\},
\]
under the constraint that $\Lip(D_{\mathrm{opt}}|_{A_\delta}) \le L^\ast$. By Corollary~\ref{cor:g-Lip}, this is the case provided that $C \min(\delta^{-1}, \sigma^{-2}) \le L^\ast$.
An argument completely analogous to the proof of Lemma~\ref{lem:1d-statlim} then shows that $\lim_{k \to \infty}W_1(\mu^{(k)},\mu) = 0$. And Lemma~\ref{lem:circle-dense} implies that
\[
\cE 
\le 
\frac{C\delta^2}{\sigma} e^{-(1-\delta)^2/2\sigma^2}.
\]
Again, we emphasize our constraint on $C \min(\delta^{-1}, \sigma^{-2}) \le L^\ast$. There are two cases, either (i) $\sigma^2 \ge C/L^2$ and we are free to let $\delta \to 0$, or (ii) $\sigma^2 < C/L^\ast$ and we will choose $\delta = C/L^\ast$. In the first case, we obtain
\[
\cE = 0.
\]
In the second case, we have $1/\sigma^2 \ge L^\ast/C$, and we obtain
\[
\cE
\le
\frac{C}{(L^\ast)^{3/2}} e^{-L^\ast/8C},
\]
for sufficiently large $L^\ast$. It follows that, independent of $\sigma$, we have the following upper bound,
\[
\cE
\le
C e^{-L^\ast/8C}.
\]
In particular, it follows that 
\[
\limsup_{k\to \infty}
\E_{h,u, \eta} \Vert D^{(k)}(u_\sigma; h,\sigma) - \tilde{D}^{(k)}(u_\sigma; h,\sigma) \Vert^2
\le
C e^{-L^\ast/8C}.
\]
\end{proof}

\paragraph{Proofs of Lemma~\ref{lem:spectral-Dopt}, Lemma~\ref{lem:spectra-Lip}, Lemma~\ref{lem:g-Lip} and Lemma~\ref{lem:circle-dense}.}

In the following we detail the proofs of Lemma~\ref{lem:spectral-Dopt}, Lemma~\ref{lem:spectra-Lip}, Lemma~\ref{lem:g-Lip} and Lemma~\ref{lem:circle-dense}. 

\begin{proof}[Proof of Lemma~\ref{lem:spectral-Dopt}]
By the explicit formula for $D_\opt$, we have 
\[
D_\opt(w;h,\sigma) = \frac{\int u \exp(-(u-w)^2/2\sigma^2) p(u\bb h) \, du}{\int \exp(-(u-w)^2/2\sigma^2) p(u \bb h) \, du}.
\]
Since the joint probability on $(h,u)$ is a product measure, the conditioning on $h$ is irrelevant, and $p(u \bb h) = \cU(\mathbb{S}^1)$. By the symmetries of the problem, the expectation perpendicular to $w$ vanishes, so that we can write 
\[
D_\opt(w;h,\sigma) = g_\sigma(|w|) \frac{w}{|w|}.
\]
For the computation of $g_\sigma(|w|)$, we may wlog assume that $w=(|w|,0)$ points in the first coordinate direction. We dot the above formula for $D_\opt$ by $w/|w|$, and parametrize $u = (\cos \theta, \sin \theta)$ by the angular variable $\theta \in [0,2\pi]$. It follows that 
\[
g_\sigma(|w|)
=
\frac{
\int_0^{2\pi} \cos \theta \exp\left( -\left\{(\cos \theta - |w|)^2 + \sin^2 \theta\right\}/2\sigma^2 \right) \, d\theta
}{
\int_0^{2\pi}  \exp\left( -\left\{(\cos \theta - |w|)^2 + \sin^2 \theta\right\}/2\sigma^2 \right) \, d\theta
}.
\]
After expansion of the squares and elementary simplifications, we obtain
\[
g_\sigma(|w|)
=
\frac{
\int_0^{\pi} \cos \theta e^{|w|\cos(\theta) / \sigma^2} \, d\theta
}{
\int_0^{\pi} e^{|w|\cos(\theta)/\sigma^2} \, d \theta
}
= \frac{I_1(|w|/\sigma^2)}{I_0(|w|/\sigma^2)}.
\]
This is the claimed formula.
\end{proof}

\begin{proof}[Proof of Lemma~\ref{lem:spectra-Lip}]
Given $u,v\in \R^d$, we have
\begin{align*}
|D(u) - D(v)|
&\le 
\left| g(|u|) \frac{u}{|u|} - g(|v|) \frac{u}{|u|} \right|
+
\left|
g(|v|) \left( \frac{u}{|u|} - \frac{v}{|v|} \right)
\right|
\\
&\le 
\Lip(g|_{[\delta,\infty)}) |u - v| 
+ 
\Vert g \Vert_{L^\infty} 
\left|
\frac{u}{|u|} - \frac{v}{|v|}
\right|.
\end{align*}
To estimate the last term, we note that
\begin{align*}
\left|
\frac{u}{|u|} - \frac{v}{|v|}
\right|
&\le
\left|
\frac{u}{|u|} - \frac{v}{|u|}
\right|
+
\left|
\frac{v}{|u|} - \frac{v}{|v|}
\right|
\\
&\le \frac{1}{|u|} |u-v| + |v| \left|
\frac{1}{|u|} - \frac{1}{|v|}
\right|
\\
&=
\frac{1}{|u|} |u-v| + \frac{1}{|u|} \left|
|v| - |u|
\right|.
\end{align*}
It thus follows that, for $|u|\ge \delta$,
\[
\left|
\frac{u}{|u|} - \frac{v}{|v|}
\right|
\le \frac{2}{\delta} |u-v|.
\]
Substitution in the first inequality now yields the claimed bound on $\Lip(D|_{A_\delta})$.
\end{proof}

\begin{proof}[Proof of Lemma~\ref{lem:g-Lip}]
We have 
\begin{align*}
\Lip(g_\sigma|_{[\delta,\infty)})
&=
\Vert g_\sigma' \Vert_{L^\infty([\delta,\infty))}
\\
&=
\frac{1}{\sigma^2} \Vert g' \Vert_{L^\infty([\delta/\sigma^2,\infty))}
\\
&= \frac{1}{\delta} \frac{\delta}{\sigma^2}  \Vert g' \Vert_{L^\infty([\delta/\sigma^2,\infty))}
\\
&\le 
\frac{1}{\delta} \Vert zg'(z) \Vert_{L^\infty([\delta/\sigma^2,\infty))}
\\
&\le 
\frac{1}{\delta} \Vert zg'(z) \Vert_{L^\infty([0,\infty))}
\end{align*}
The claim follows by observing that $C := \Vert zg'(z) \Vert_{L^\infty([0,\infty))} < \infty$. This last observation follows from the relationships $I_0'(z) = I_1(z)$, $I_1'(z) = I_0(z) - \frac{1}{z} I_1(z)$, so that 
\[
g'(z) 
=
\frac{I_1'(z) I_0(z) - I_1(z) I_0'(z)}{I_0(z)^2}
=
(1 - g(z)^2) - \frac{1}{z} g(z),
\]
and the asymptotics of $I_0(z)$, $I_1(z)$ as $z\to \infty$; Indeed, we have $zg'(z) = z(1-g(z)^2) - g(z)$, and the asymptotic expansion of $I_0(z)$, $I_1(z)$ as $z\to \infty$ implies that $g(z) \to 1$ and $1-g(z)^2 \sim C/ z$ for some constant $C$. Hence $zg'(z)$ remains uniformly upper bounded as $z\to \infty$. 
\end{proof}

\begin{proof}[Proof of Lemma~\ref{lem:circle-dense}]
For any $y \in \supp(\mu)$ and $x \in B_\delta(0)$, we clearly have $|x-y|^2 \ge (1-\delta)^2$. Hence,
\begin{align*}
\mathrm{Prob}_{\mu_\sigma}\left[ B_\delta(0)\right]
&= 
\int \int 1_{B_\delta}(x) \frac{e^{-|x-y|^2/2\sigma^2}}{2\pi \sigma} \mu(dy) \, dx
\\
&\le 
\frac{e^{-(1-\delta)^2/2\sigma^2}}{2\pi \sigma} |B_\delta|
= \frac{\delta^2}{2 \sigma} e^{-(1-\delta)^2/2\sigma^2}.
\end{align*}
\end{proof}

\clearpage
\newpage
\section{Further Experimental Results}
In this section, we expand on the discussion about the experimental results in the Main Text by contextualizing additional results and better highlighting the ones briefly discussed in the Main Text. 

\subsection{GenCFD Generates Very High-quality Samples of the Flow}
We start by recalling the ability of GenCFD to generate very high-quality samples, drawn from the conditional distribution, which was already discussed in the Main Text. To this end, in Fig.~2 (A) of the Main Text and Fig.~\ref{fig:s2} here, we present samples of the conditional kinetic energy (square of the norm of the velocity field), corresponding to the TG and CSF datasets, generated by GenCFD and compare them to ground-truth samples generated with the underlying CFD simulator. We observe from these figures that the pointwise kinetic energy samples are very realistic for both datasets, with very little to distinguish them visually from the ground truth samples. In particular, the small-scale eddies are captured very well in the generated velocity fields, and the diffusion model also provides a rich diversity of samples despite each of them corresponding to the \emph{same initial condition.} 

An even harder test of sample quality is investigated by computing the \emph{vorticity} from the generated velocity fields by taking a \emph{curl} and plotting the resulting pointwise vorticity intensity samples in Fig.~2 (B) of the Main Text and Fig.~\ref{fig:s3} here, for TG and CSF, respectively. When compared to the ground truth, we find the computed vorticity samples to be very accurate with a realistic rendition of small-scale features for both datasets, particularly of looping vortex tubes which are the characteristics of a turbulent fluid~\cite{MB1}. These realistic vorticity profiles, generated by our AI algorithm, are particularly impressive as the model itself has never been shown a vorticity field and has to infer it indirectly from the generated velocity field by taking its derivatives. This suggests that the covariate structures of the velocity fields are well captured by the AI algorithm.

In contrast to the high quality of samples generated by GenCFD, all the other ML baselines (see Figs.~2 (A and B) of the Main Text and Figs.~\ref{fig:s2} and \ref{fig:s3}, for C-FNO, which is the strongest baseline) only lead to very poor quality as well as very little diversity in terms of the generated samples, both for the kinetic energy and for the vorticity. In general, the samples generated by these baselines collapse to a field close to the mean velocity (and vorticity) field, rather than representing the statistical spread of the target conditional distribution. 

This observation of very high-quality sample generation with GenCFD is further reinforced in Figs.~\ref{fig:sm1} to \ref{fig:sm3} , where samples with other initial conditions for the CSF dataset are presented.

Moreover, in Fig.~\ref{fig:s7}, we present samples of the density for the CSI dataset, generated by GenCFD and compare it to ones obtained from the ground truth and C-FNO to observe that this high quality of sample generation by GenCFD is also present for compressible fluid flow. In particular, we observe that GenCFD is able to generate the leading shock wave, the rising turbulent plume in its wake, and also the small-scale eddies marking the turbulent mixing zone near the trailing shocks. On the other hand, C-FNO is able to approximate the leading shock wave accurately but smoothens out the rising plume while completely failing to generate the small-scale turbulent eddies in the wake.

In Figs.~\ref{fig:nozz1} and \ref{fig:nozz2}, we present the pointwise kinetic energy and vorticity intensity, respectively, of the nozzle flow experiment. Even for this experiment that is prototypical of real-world engineering flows, we see that GenCFD provides very high quality and diverse samples of the flow whereas the best performing baseline (UViT) regresses to the mean. In particular, the ability of GenCFD to generate this very complex flow, with a very complicated distribution of eddies that are both wall-bounded and yet have a freestream component, is noteworthy. 

Finally, in Fig.~\ref{fig:cbl1}, we visualize the x-component of the velocity field in the convective boundary layer (CBL) experiment to again find that GenCFD is able to generate realistic, diverse samples of the flow field, whereas the best performing baseline (UViT) smears out all the detailed flow features, especially the plumes going up and down due to convection. 

\subsection{GenCFD Accurately Approximates Statistical Quantities of Interest}

The high quality of AI-generated samples of fluid flow encourages us to test how well GenCFD approximates the statistical quantities of interest. We check this by computing the mean and the standard deviation of the conditional velocity field, generated by GenCFD and comparing them with the underlying ground truth and the statistics of the ML baselines. The results for all the datasets are presented in Main Text Fig.~2 (C and D) for TG, Figs.~\ref{fig:s4}, and \ref{fig:s5} for CSF, Fig.~\ref{fig:s8} for CSI, Fig.~\ref{fig:nozz3} and Main Text Fig.~3 (A) for NF and Fig.~\ref{fig:cbl2} and Main Text Fig.~3 (C) for CBL. We observe from these figures that GenCFD approximates the mean and variance very well. In contrast, the ML baselines (we present figures for the best performing baseline in each dataset) can approximate the mean fairly accurately but entirely fail at approximating the standard deviation. This (approximate) collapse to mean for the ML baselines is also seen when we plot the one-point PDFs in Main Text Fig.~2 (E) for TG, Fig.~\ref{fig:s6} for CSF, Fig.~\ref{fig:s9} for CSI, Fig.~\ref{fig:nozz4} for NF and Fig.~\ref{fig:cbl3} for CBL. In complete contrast, GenCFD very accurately and impressively approximates the underlying point PDFs. We would like to point out that this ability of GenCFD to accurately predict the PDFs is particularly noteworthy as the spread out PDFs have to be generated from inputs (initial conditions) that are (approximately) Dirac distributions.   

This accurate approximation of statistical quantities of interest with GenCFD is not just qualitative but also quantitative. We demonstrate this accuracy by presenting the errors in computing the mean, the standard deviation and as well as (the first marginal of) the 1-Wasserstein distance between the conditional distributions and the CRP Scores (CRPS), computed by GenCFD and other ML baselines, and the ground truth computed with the CFD solvers in Tables~\ref{tab:TG1}, \ref{tab:1}, \ref{tab:CS3D1}, \ref{tab:nozzle} and \ref{tab:dbl} for TG, CSF, CSI, NF and CBL, respectively. We see from these tables that GenCFD is \emph{significantly more accurate} than the ML baselines, particularly with respect to the standard deviation and the Wasserstein distance with gains ranging up to one order of magnitude for the Wasserstein distance and the standard deviation and demonstrating the ability of GenCFD for accurate statistical computation of turbulent fluid flows. 

\subsection{GenCFD Provides Excellent Spectral Resolution}

Energy spectra are key quantities of interest for the theoretical, experimental and computational study of turbulent fluid flows~\cite{FRI}. In particular, the famous K41 theory of turbulence is based on  predicting the decay of these spectra with respect to wave number. Hence, spectral resolutions are often used as proxies for judging the quality of physics-~\cite{LES} or AI-based~\cite{AIspect} simulators of turbulent fluid flows. We compute the energy spectra for the GenCFD and baseline generated fields for all the datasets and plot them in Main Text Fig.~2 (F) for TG and Fig.~\ref{fig:7} here for the rest of the datasets. We clearly observe from these figures that the spectral accuracy of GenCFD is excellent and the energy content,  up to the highest frequencies, is accurately generated. On the other hand, the deterministic ML baselines are only able to generate a small fraction of the spectrum accurately and lose spectral resolution very fast. This poor effective spectral resolution of deterministic ML models has also been observed in the context of weather and climate modeling, see for instance Fig.~1 of ~\cite{AIspect}.

\subsection{GenCFD Scales with Data}

A key attribute of modern AI models is their ability to scale with data~\cite{Kap1}. To test this, we compute the errors in standard deviation and with respect to the $1$-Wasserstein metric for GenCFD as the number of training samples for the CSF dataset varies and plot the results in Fig.~\ref{fig:8}. We observe from this figure that these test errors with GenCFD decrease as the amount of training data increases.

\subsection{The Statistical Computation with GenCFD is Robust}

We recall that the test task for GenCFD is \emph{out-of-distribution} as the test distribution is a Dirac measure whereas the training distribution is spread out. Yet, GenCFD computes the samples as well as the statistics accurately. We test this  \emph{generalization ability} further by choosing the functions, on which the test distribution is concentrated \eqref{eq:dirac}, from yet another distribution. To this end, we consider the CSF dataset and choose $p \in \{8, 9, 10, 11, 12\}$ uniformly at random for each sample, rather than constant and equal to 10. Additionally, the perturbation functions $\sigma_y^i$, and $\sigma_z^j$ are extended with another parameter $\xi_y, \xi_z \sim \cU_{[-0.0625, 0.0625]}$ by setting
\begin{equation*}
    \begin{aligned}
        \sigma_y^j(x) &= \delta \sum_{k=1}^{p} \alpha_k^y \sin(2\pi k x - \beta_k^y) - \xi_y \\
        \sigma_z^j(x) &= \delta \sum_{k=1}^{p} \alpha_k^z \sin(2\pi k x - \beta_k^z) - \xi_z.
    \end{aligned}
\end{equation*}
This leads to a distribution $\nu$, which is different from the training distribution. 

Zero-shot Evaluation is performed by sampling the initial condition $\bar{u} \sim \nu$ and feeding it directly into the pretrained models (GenCFD and baselines). No adjustments or modifications to the pretrained models are needed for this evaluation. 

For fine-tuning, the models are trained using the objectives \eqref{eq:dnostr} on 300 samples drawn from the distribution $\nu$. During this step, all pretrained model parameters are updated.

Even with no additional training (\emph{Zero-Shot mode}), GenCFD is able to generate high quality samples (see Fig.~\ref{fig:sm4}) as well as approximate mean, standard deviation and probability distribution quantitatively (Table~\ref{tab:4}) to reasonable accuracy. For instance, the 1-Wasserstein distance between the ground truth and the generated distribution increases, on average, by a factor of $2$ over the previously tested distribution (see Table~\ref{tab:1}), with this \emph{zero-shot} evaluation. By further \emph{fine-tuning} the model with merely $300$ trajectories generated from initial data, drawn from the new training distribution $\nu$, the error is reduced to the previous levels, see Table~\ref{tab:4} and compare with Table~\ref{tab:1} further showcasing the ability of GenCFD to handle \emph{distribution shifts.}

Another avenue for checking robustness arises when we check how the statistical errors with GenCFD increase over time. From Table~\ref{tab:5}, where we present the 1-Wasserstein distance between the ground truth and the GenCFD generated conditional distribution for the CSF dataset, we observe that after a modest increase in the beginning of evolution when the turbulence kicks in, the error is approximately constant for the time period when the turbulence is fully developed. 

This demonstration of  the robustness of our proposed algorithm to \emph{time evolution} is particularly pertinent for the approximation of the Taylor--Green vortex as the flow starts laminar (in fact smooth) and dynamically transitions to turbulence over time. For instance, the flow at time $T=0.8$ is still markedly laminar while it has become turbulent by time $T=2$. Can GenCFD still be robust with such transitions from laminar to turbulent flow? In particular, can it be accurate at also approximating deterministic flows? These questions are answered qualitatively in Figs.~\ref{fig:tg2}, \ref{fig:tgb2}, \ref{fig:tg_mean2} and \ref{fig:tg_std2}, from which we observe that GenCFD continues to provide realistic samples for both velocity and vorticity and approximates the mean and especially, the standard deviation, for the underlying laminar flow very accurately. This observation is further buttressed by the results in Table~\ref{tab:TG2} where we see that GenCFD has lower errors in every single metric when compared to the ML baselines. Thus, this experiment clearly showcases the ability of GenCFD to accurately approximate both deterministic and stochastic fluid flow.

\subsection{Statistical Computation with GenCFD is Fast}

The computational cost of inference with the GenCFD algorithm scales \emph{linearly} in the number of steps required for solving the reverse-SDE \eqref{eq:si_rsde}. In Table~\ref{tab:6}, we show how robust our algorithm is with respect to the number of steps in solving the reverse-SDE to find that as few as $16$ steps suffice in ensuring acceptable statistical accuracy, with $32$ steps being almost as good as $128$ steps. Given this observation, we can deploy our model with $32$ steps for the reverse SDE. The resulting inference time with GenCFD, in comparison to the underlying CFD solvers, is presented in Table~\ref{tab:10}. We observe from the figure that GenCFD requires approximately $1$ to $4$ seconds for a single inference run on a NVIDIA RTX4090 24GB GPU, for all the test cases that we have considered. These inference times will be even faster for more powerful GPUs. In contrast, the run times for CFD solvers vary considerably based on the underlying numerical method and on the hardware used to run them. A highly optimized code such as {\bf Azeban} can perform $128^3$ CFD simulations on a periodic domain with spectral viscosity method in approximately $10$ seconds on GPUs. However, this run time is significantly larger at approximately $20$ minutes, even on state-of-the-art H100 96GB GPUs for a more complicated test case like the nozzle flow, even when a highly scalable solver such as {\bf OpenLB} is employed. On the other hand, all the CFD simulations, which are performed on CPUs, required run times in the order of hours. This is indicative of the real world as most CFD codes run on CPUs. Hence, from Table~\ref{tab:10}, we see that GenCFD can provide a speedup, ranging from one to three orders of magnitude with respect to GPU-based CFD codes while providing an impressive three to five orders of magnitude speedup over CPU-based CFD codes. It is precisely this very high computational speed, coupled with accuracy, that makes GenCFD very attractive for applications in many areas of fluid dynamics.

\paragraph{Numerical Results with the Toy Models.} Recalling Toy Model \#1, which is given by the map $\sol^\Delta$ \eqref{eq:tm1}, we present numerical results with a diffusion model and an underlying deterministic ML baseline in Fig.~\ref{fig:15}, from where we observe that the deterministic ML approximation does accurately approximate $\sol^\Delta$, for moderate values of $\Delta \approx 0.1$. But for even lower values of $\Delta$, the deterministic approximation collapses to the mean as predicted by the theory presented earlier. On the other hand, the diffusion model is able to approximate $\sol^\Delta$ very well, for moderate values of $\Delta \approx 0.05$ when the mapping is not too oscillatory in a deterministic manner, while at the same time, being able to approximate the statistical limit for very small values of $\Delta \approx 0.002$. Consistent with our theory, this ability of diffusion models to be robust with respect to any value of $\Delta$ in this case is worth highlighting. It also matches the empirical observation that GenCFD was able to approximate the Taylor--Green vortex flow accurately for both the laminar and turbulent regimes.  

In Fig.~\ref{fig:15}, we also illustrate the different modes through which deterministic ML models and diffusion models \emph{learn during training.} A deterministic ML model first learns the mean and then adds oscillations as more and more gradient descent steps are taken, consistent with the well-known spectral bias of neural networks \cite{Rah1}. If the underlying map is too oscillatory, it simply does not add the oscillations and predicts the mean, which yields a (local) minimum for the $L^2$ loss. In contrast, diffusion models do the opposite. Already, very early in the training, they capture the statistical limit measure and as more gradient descent steps are taken, they start \emph{clearing out} the measure to reveal more of the underlying deterministic oscillations. If the underlying map is too oscillatory, the diffusion model sticks to the measure-valued output even when more training steps are taken, enabling it to capture both deterministic approximations as well as statistical information, as necessary. 

Finally in Fig.~\ref{fig:16}, we present results with a diffusion model and the underlying deterministic baseline on Toy Model \#2, which was described and rigorously analyzed in the previous section. We observe from this figure that while the deterministic model is accurate for low wave numbers (around $k \approx 10$), it collapses to the mean for $k \geq 30$, even with a lot of training steps. This also implies that the phase space approximation is very poor and the underlying constraint is violated at high wavenumber. In contrast, the diffusion model is able to learn the underlying map, both for small as well as large wavenumbers, even with a few training steps. The underlying constraint is satisfied for any of the tested wavenumbers and the contrast between the deterministic and diffusion models is nicely shown in the (synthetic) spectra plotted in Fig.~\ref{fig:7} (bottom row). The diffusion model is able to capture structures at all wave numbers whereas the deterministic model has a limited spectral resolution, explaining the spectral findings for fluid flows in Fig.~\ref{fig:7}. To obtain these results, we used the same model specifications and training procedure as described in Section~\ref{sec:c43}. In contrast to that section, the deterministic model now maps a one-dimensional input to a two-dimensional output, $h \mapsto \Psi_\mathrm{det}(h) \approx \sol^{(k)}(h)$. The diffusion model has four-dimensional input, $(u_\sigma,h,\sigma)$, where $u_\sigma\in \R^2$, $h\in [0,1]$, $\sigma \in \R$, and outputs a two-dimensional (denoised) vector $u$, $(u,h,\sigma) \mapsto D(u;h,\sigma)$. In each case, the model is trained on $2048$ samples, with $h\sim \cU([0,1])$ uniformly sampled over the input range $[0,1]$. All other implementation details are identical to the ones specified in Section~\ref{sec:c43}.

\clearpage
\newpage
\section{Supplementary Tables}
\begin{table}[h!]
\centering
\begin{small}
  \begin{tabular}{ c c c c c c}
    \toprule
        & 
         &
         \makecell{$e_\mu$} & \makecell{$e_\sigma$} & \makecell{$W_1$} & \makecell{$\text{CRPS}_G$}  \\
    \midrule
    \midrule
    \multirow{3}{*}{\centering{GenCFD}}
    & $u_x$ & 
    \makecell{{\bf 0.154}}  & \makecell{{\bf 0.056}}  & \makecell{{\bf 0.0165}}  & \makecell{{\bf 0.481}}  \\
    & $u_y$ & 
    \makecell{{\bf 0.155}}  & \makecell{{\bf 0.055}}  & \makecell{{\bf 0.0172}}  & \makecell{{\bf 0.479}}  \\
    & $u_z$ & 
    \makecell{{\bf 0.282}}  & \makecell{{\bf 0.053}}  & \makecell{{\bf 0.0145}}  & \makecell{{\bf 0.469}}  \\
    \midrule
    \midrule
    \multirow{3}{*}{\centering{UViT}}
    & $u_x$ & 
    \makecell{0.883}  & 
    \makecell{0.813}  & 
    \makecell{0.130}  & 
    \makecell{0.768}   \\
    & $u_y$ & 
    \makecell{0.944}  & 
    \makecell{0.829}  & 
    \makecell{0.138}  & 
    \makecell{0.802}   \\
    & $u_z$ & 
    \makecell{1.016}  & 
    \makecell{0.881}  & 
    \makecell{0.110}  & 
    \makecell{0.648}   \\
    \midrule
    \midrule
    \multirow{3}{*}{\centering{FNO}} 
    & $u_x$ & 
    \makecell{0.359}  & 
    \makecell{0.999}  & 
    \makecell{0.121}  & 
    \makecell{0.690}   \\
    & $u_y$ & 
    \makecell{0.362}  & 
    \makecell{0.999}  & 
    \makecell{0.123}  & 
    \makecell{0.690}   \\
    & $u_z$ & 
    \makecell{0.756}  & 
    \makecell{0.998}  & 
    \makecell{0.119}  & 
    \makecell{0.671}   \\

    \midrule
    \midrule
    \multirow{3}{*}{\centering{C-FNO}} 
    & $u_x$ & 
    \makecell{0.210}  & 
    \makecell{1.000}  & 
    \makecell{0.117}  & 
    \makecell{0.670}   \\
    & $u_y$ & 
    \makecell{0.210}  & 
    \makecell{1.000}  & 
    \makecell{0.118}  & 
    \makecell{0.668}   \\
    & $u_z$ & 
    \makecell{0.402}  & 
    \makecell{1.000}  & 
    \makecell{0.115}  & 
    \makecell{0.653}   \\
    
\bottomrule
\end{tabular}
  \end{small}

  \caption{{\textbf{Error metrics, defined in Materials and Methods, for different models for the Taylor--Green vortex experiment.} The metrics are defined in {\bf SM} \ref{sec:em} and computed at time $T=2.0$. Results for the best performing model are in bold.}}
\label{tab:TG1}
  \end{table}

\begin{table}[h!]
\centering
\begin{small}
  \begin{tabular}{ c c c c c c}
    \toprule
        & 
         &
         \makecell{$e_\mu$} & \makecell{$e_\sigma$} & \makecell{$W_1$} & \makecell{$\text{CRPS}_G$}  \\
    \midrule
    \midrule
    \multirow{3}{*}{\centering{GenCFD}}
    & $u_x$ 
    & \makecell{{\bf 0.088}}  & \makecell{{\bf 0.114}}  & \makecell{{\bf 0.034}}  & \makecell{{\bf 0.347}}  \\
    & $u_y$ 
    & \makecell{{\bf 0.271}}  & \makecell{{\bf 0.110}}  & \makecell{{\bf 0.030}}  & \makecell{{\bf 0.316}} \\
    & $u_z$ 
    & \makecell{{\bf 0.268}}  & \makecell{{\bf 0.113}} & \makecell{{\bf 0.032}}  & \makecell{{\bf 0.317}} \\
    \midrule
    \midrule
    \multirow{3}{*}{\centering{UViT}}
    & $u_x$ 
    & \makecell{0.604}  & \makecell{0.562}  & \makecell{0.253}  & \makecell{0.708}   \\
    & $u_y$ 
    & \makecell{1.096}  & \makecell{0.558}  & \makecell{0.147}  & \makecell{0.443}  \\
    & $u_z$ 
    & \makecell{1.038}  & \makecell{0.663}  & \makecell{0.150}  & \makecell{0.475}  \\
    \midrule
    \midrule
    \multirow{3}{*}{\centering{FNO}} 
    & $u_x$ 
    & \makecell{0.301}  & \makecell{0.957}  & \makecell{0.169}  & \makecell{0.547}  \\
    & $u_y$ 
    & \makecell{0.889}  & \makecell{0.959}  & \makecell{0.148}  & \makecell{0.486}  \\
    & $u_z$ 
    & \makecell{0.815}  & \makecell{0.962}  & \makecell{0.150}  & \makecell{0.485}  \\

    \midrule
    \midrule
    \multirow{3}{*}{\centering{C-FNO}} 
    & $u_x$ 
    & \makecell{0.217}  & \makecell{0.864}  & \makecell{0.133}  & \makecell{0.452}  \\
    & $u_y$ 
    & \makecell{0.622}  & \makecell{0.880}  & \makecell{0.120}  & \makecell{0.405}  \\
    & $u_z$ 
    & \makecell{0.641}  & \makecell{0.880}  & \makecell{0.124}  & \makecell{0.417}  \\

\bottomrule
\end{tabular}
  \end{small}
   \caption{{\textbf{Error metrics for different models at $T=1.0$ for the cylindrical shear flow experiment.} Results for the best performing model are in bold.}}
\label{tab:1}
  \end{table}

\begin{table}[h!]
\centering
\begin{small}
  \begin{tabular}{ c c c c c c}
    \toprule
        & 
         &
         \makecell{$e_\mu$} & \makecell{$e_\sigma$} & \makecell{$W_1$} & \makecell{$\text{CRPS}_G$}  \\
    \midrule
    \midrule
    \multirow{5}{*}{\centering{GenCFD}}
    & $\rho$ & 
    \makecell{\bf0.049}  & 
    \makecell{\bf0.381}  & 
    \makecell{\bf0.054}  & 
    \makecell{\bf0.0035}   \\
    & $u_x$ & 
    \makecell{\bf0.015}  & 
    \makecell{\bf0.332}  & 
    \makecell{\bf0.093}  & 
    \makecell{\bf0.0040}   \\
    & $u_y$ & 
    \makecell{\bf0.195}  & 
    \makecell{\bf0.203}  & 
    \makecell{\bf0.054}  & 
    \makecell{0.0033}   \\
    & $u_z$ & 
    \makecell{\bf0.171}  & 
    \makecell{\bf0.301}  & 
    \makecell{\bf0.021}  & 
    \makecell{\bf0.0012}   \\
    & $p$ & 
    \makecell{\bf0.015}  & 
    \makecell{\bf0.316}  & 
    \makecell{\bf1.525}  & 
    \makecell{\bf0.0436}   \\
    \midrule
    \midrule
    \multirow{5}{*}{\centering{UViT}}
    & $\rho$ & 
    \makecell{0.252}  & 
    \makecell{0.993}  & 
    \makecell{0.400}  & 
    \makecell{0.0151}   \\
    & $u_x$ & 
    \makecell{0.095}  & 
    \makecell{0.985}  & 
    \makecell{0.631}  & 
    \makecell{0.0151}   \\
    & $u_y$ & 
    \makecell{0.849}  & 
    \makecell{0.992}  & 
    \makecell{0.283}  & 
    \makecell{0.0065}   \\
    & $u_z$ & 
    \makecell{0.722}  & 
    \makecell{0.992}  & 
    \makecell{0.158}  & 
    \makecell{0.0046}   \\
    & $p$ & 
    \makecell{0.127}  & 
    \makecell{0.980}  & 
    \makecell{13.68}  & 
    \makecell{0.2834}   \\
    \midrule
    \midrule
    \multirow{5}{*}{\centering{FNO}} 
    & $\rho$ & 
    \makecell{0.138}  & 
    \makecell{0.921}  & 
    \makecell{0.242}  & 
    \makecell{0.0085}   \\
    & $u_x$ & 
    \makecell{0.071}  & 
    \makecell{0.940}  & 
    \makecell{0.393}  & 
    \makecell{0.0111}   \\
    & $u_y$ & 
    \makecell{0.680}  & 
    \makecell{0.973}  & 
    \makecell{0.223}  & 
    \makecell{0.0054}   \\
    & $u_z$ & 
    \makecell{0.424}  & 
    \makecell{0.929}  & 
    \makecell{0.100}  & 
    \makecell{0.0027}   \\
    & $p$ & 
    \makecell{0.081}  & 
    \makecell{0.892}  & 
    \makecell{8.328}  & 
    \makecell{0.1782}   \\

    \midrule
    \midrule
    \multirow{5}{*}{\centering{C-FNO}} 
    & $\rho$ & 
    \makecell{0.081}  & 
    \makecell{0.798}  & 
    \makecell{0.133}  & 
    \makecell{0.0052}   \\
    & $u_x$ & 
    \makecell{0.037}  & 
    \makecell{0.688}  & 
    \makecell{0.264}  & 
    \makecell{0.0058}   \\
    & $u_y$ & 
    \makecell{0.399}  & 
    \makecell{0.613}  & 
    \makecell{0.127}  & 
    \makecell{\bf 0.0032}   \\
    & $u_z$ & 
    \makecell{0.236}  & 
    \makecell{0.878}  & 
    \makecell{0.042}  & 
    \makecell{0.0017}   \\
    & $p$ & 
    \makecell{0.038}  & 
    \makecell{0.630}  & 
    \makecell{4.128}  & 
    \makecell{0.0811}   \\
\bottomrule
\end{tabular}
\end{small}
 \caption{{\textbf{Results for error metrics for different models at Time $T=1.0$ for the cloud-shock interaction experiment.} Results for the best performing model are in bold.}}
\label{tab:CS3D1}
\end{table}

\begin{table}[h!]
\centering
\begin{small}
  \begin{tabular}{ c c c c c c}
    \toprule
        & 
         &
         \makecell{$e_\mu$} & \makecell{$e_\sigma$} & \makecell{$W_1$} & \makecell{$\text{CRPS}_G$}  \\
    \midrule
    \midrule
    \multirow{4}{*}{\centering{GenCFD}}
    & $u_x$ & 
    \makecell{0.148}  & 
    \makecell{\bf 0.258}  & 
    \makecell{\bf 0.0073}  & 
    \makecell{\bf 0.268}   \\
    & $u_y$ & 
    \makecell{N/A}  & 
    \makecell{\bf0.207}  & 
    \makecell{\bf0.0049}  & 
    \makecell{\bf 0.508}   \\
    & $u_z$ & 
    \makecell{N/A}  & 
    \makecell{\bf0.218}  & 
    \makecell{\bf0.0050}  & 
    \makecell{\bf0.515}   \\
    & $E_k$ & 
    \makecell{\bf0.151}  & 
    \makecell{\bf0.240}  & 
    \makecell{\bf0.00065}  & 
    \makecell{\bf0.230}   \\
    \midrule
    \midrule
    \multirow{4}{*}{\centering{UViT}}
    & $u_x$ & 
    \makecell{\bf0.074}  & 
    \makecell{0.850}  & 
    \makecell{0.0109}  & 
    \makecell{0.324}   \\
    & $u_y$ & 
    \makecell{N/A}  & 
    \makecell{0.926}  & 
    \makecell{0.0098}  & 
    \makecell{0.657}   \\
    & $u_z$ & 
    \makecell{N/A}  & 
    \makecell{0.943}  & 
    \makecell{0.0100}  & 
    \makecell{0.669}   \\
    & $E_k$ & 
    \makecell{0.167}  & 
    \makecell{0.827}  & 
    \makecell{0.00128}  & 
    \makecell{0.276}   \\
    \midrule
    \midrule
    \multirow{4}{*}{\centering{FNO}} 
    & $u_x$ & 
    \makecell{0.176}  & 
    \makecell{0.864}  & 
    \makecell{0.0139}  & 
    \makecell{0.358}   \\
    & $u_y$ & 
    \makecell{N/A}  & 
    \makecell{0.934}  & 
    \makecell{0.0111}  & 
    \makecell{0.678}   \\
    & $u_z$ & 
    \makecell{N/A}  & 
    \makecell{0.956}  & 
    \makecell{0.0112}  & 
    \makecell{0.693}   \\
    & $E_k$ & 
    \makecell{0.237}  & 
    \makecell{0.848}  & 
    \makecell{0.00136}  & 
    \makecell{0.316}   \\

    \midrule
    \midrule
    \multirow{4}{*}{\centering{C-FNO}} 

    & $u_x$ & 
    \makecell{0.124}  & 
    \makecell{0.858}  & 
    \makecell{0.0116}  & 
    \makecell{0.333}   \\
    & $u_y$ & 
    \makecell{N/A}  & 
    \makecell{0.923}  & 
    \makecell{0.0100}  & 
    \makecell{0.661}   \\
    & $u_z$ & 
    \makecell{N/A}  & 
    \makecell{0.938}  & 
    \makecell{0.0102}  & 
    \makecell{0.673}   \\
    & $E_k$ & 
    \makecell{0.207}  & 
    \makecell{0.835}  & 
    \makecell{0.00134}  & 
    \makecell{0.287}   \\
\bottomrule
\end{tabular}
\end{small}
\caption{{\textbf{Results for Error metrics for different models at Time $T=1.0$ for the nozzle flow experiment.} Results for the best performing model are in bold. Note that the term N/A for the mean error $e_\mu$ for the $u_y,u_z$ components signifies the fact that the ground truth means of these velocity components are $0$ and the relative errors are not well-defined.}}
\label{tab:nozzle}
\end{table}

\begin{table}[h!]
\centering
\begin{small}
  \begin{tabular}{ c c c c c c}
    \toprule
        & 
         &
         \makecell{$e_\mu$} & \makecell{$e_\sigma$} & \makecell{$W_1$} & \makecell{$\text{CRPS}_G$}  \\
    \midrule
    \midrule
    \multirow{5}{*}{\centering{GenCFD}}
    & $u_x$ & 
    \makecell{\bf0.223}  & 
    \makecell{\bf0.072}  & 
    \makecell{\bf0.036}  & 
    \makecell{\bf0.557}   \\
    & $u_y$ & 
    \makecell{N/A}  & 
    \makecell{\bf0.094}  & 
    \makecell{\bf0.038}  & 
    \makecell{\bf0.567}   \\
    & $u_z$ & 
    \makecell{N/A}  & 
    \makecell{\bf0.059}  & 
    \makecell{\bf0.037}  & 
    \makecell{\bf0.553}   \\
    & $T$ & 
    \makecell{$10\cdot 10^{-5}$}  & 
    \makecell{\bf0.091}  & 
    \makecell{\bf0.025}  & 
    \makecell{\bf0.00060}   \\
    & $E_k$ & 
    \makecell{\bf0.109}  & 
    \makecell{\bf0.137}  & 
    \makecell{\bf0.072}  & 
    \makecell{}   \\
    \midrule
    \midrule
    \multirow{5}{*}{\centering{UViT}}
    & $u_x$ & 
    \makecell{0.235}  & 
    \makecell{0.807}  & 
    \makecell{0.305}  & 
    \makecell{0.692}   \\
    & $u_y$ & 
    \makecell{N/A}  & 
    \makecell{0.827}  & 
    \makecell{0.308}  & 
    \makecell{0.714}   \\
    & $u_z$ & 
    \makecell{N/A}  & 
    \makecell{0.825}  & 
    \makecell{0.403}  & 
    \makecell{0.704}   \\
    & $T$ & 
    \makecell{$\mathbf{6\cdot 10^{-5}}$}  & 
    \makecell{0.890}  & 
    \makecell{0.108}  & 
    \makecell{0.00086}   \\
    & $E_k$ & 
    \makecell{0.956}  & 
    \makecell{0.965}  & 
    \makecell{0.644}  & 
    \makecell{}   \\
    \midrule
    \midrule
    \multirow{5}{*}{\centering{FNO}} 
    & $u_x$ & 
    \makecell{0.345}  & 
    \makecell{0.936}  & 
    \makecell{0.343}  & 
    \makecell{0.753}   \\
    & $u_y$ & 
    \makecell{N/A}  & 
    \makecell{0.849}  & 
    \makecell{0.315}  & 
    \makecell{0.730}   \\
    & $u_z$ & 
    \makecell{N/A}  & 
    \makecell{0.973}  & 
    \makecell{0.470}  & 
    \makecell{0.779}   \\
    & $T$ & 
    \makecell{$18\cdot 10^{-5}$}  & 
    \makecell{0.612}  & 
    \makecell{0.080}  & 
    \makecell{0.00074}   \\
    & $E_k$ & 
    \makecell{0.978}  & 
    \makecell{0.986}  & 
    \makecell{0.657}  & 
    \makecell{}   \\

    \midrule
    \midrule
    \multirow{5}{*}{\centering{C-FNO}} 

    & $u_x$ & 
    \makecell{0.293}  & 
    \makecell{0.875}  & 
    \makecell{0.332}  & 
    \makecell{0.722}   \\
    & $u_y$ & 
    \makecell{N/A}  & 
    \makecell{0.922}  & 
    \makecell{0.360}  & 
    \makecell{0.762}   \\
    & $u_z$ & 
    \makecell{N/A}  & 
    \makecell{0.909}  & 
    \makecell{0.453}  & 
    \makecell{0.754}   \\
    & $T$ & 
    \makecell{$44\cdot 10^{-5}$}  & 
    \makecell{0.807}  & 
    \makecell{0.143}  & 
    \makecell{0.00100}   \\
    & $E_k$ & 
    \makecell{0.977}  & 
    \makecell{0.981}  & 
    \makecell{0.659}  & 
    \makecell{}   \\
\bottomrule
\end{tabular}
\end{small}
\caption{{\textbf{Results for Error metrics for different models at Time $T=1.0$ for the dry convective boundary layer experiment.} Results for the best performing model are in bold. Note that the term N/A for the mean error $e_\mu$ for the $u_y,u_z$ components signifies the fact that the ground truth means of these velocity components are $0$ and the relative errors are not well-defined.}}
\label{tab:dbl}
\end{table}

\begin{table}[t!]
\centering
\begin{small}
  \begin{tabular}{ c c c c c c}
    \toprule
        & 
         &
         \makecell{$e_\mu$} & \makecell{$e_\sigma$} & \makecell{$W_1$} & \makecell{$\text{CRPS}_G$}  \\
    \midrule
    \midrule
    \multirow{3}{*}{\centering{GenCFD}}
    & $u_x$ & 
    \makecell{{\bf 0.030}}  & \makecell{{\bf 0.228}}  & \makecell{{\bf 0.0077}}  & \makecell{{\bf 0.053}}  \\
    & $u_y$ & 
    \makecell{{\bf 0.030}}  & \makecell{{\bf 0.227}}  & \makecell{{\bf 0.0075}}  & \makecell{{\bf 0.050}}  \\
    & $u_z$ & 
    \makecell{{\bf 0.030}}  & \makecell{{\bf 0.251}}  & \makecell{{\bf 0.0061}}  & \makecell{{\bf 0.039}}  \\
    \midrule
    \midrule
    \multirow{3}{*}{\centering{UViT}}
    & $u_x$ & 
    \makecell{0.843}  & 
    \makecell{1.219}  & 
    \makecell{0.203}  & 
    \makecell{0.832}   \\
    & $u_y$ & 
    \makecell{0.880}  & 
    \makecell{1.328}  & 
    \makecell{0.207}  & 
    \makecell{0.869}   \\
    & $u_z$ & 
    \makecell{0.957}  & 
    \makecell{1.207}  & 
    \makecell{0.175}  & 
    \makecell{0.762}   \\
    \midrule
    \midrule
    \multirow{3}{*}{\centering{FNO}} 
    & $u_x$ & 
    \makecell{0.458}  & 
    \makecell{0.989}  & 
    \makecell{0.100}  & 
    \makecell{0.489}   \\
    & $u_y$ & 
    \makecell{0.459}  & 
    \makecell{0.987}  & 
    \makecell{0.102}  & 
    \makecell{0.489}   \\
    & $u_z$ & 
    \makecell{0.556}  & 
    \makecell{0.978}  & 
    \makecell{0.096}  & 
    \makecell{0.473}   \\

    \midrule
    \midrule
    \multirow{3}{*}{\centering{C-FNO}} 
    & $u_x$ & 
    \makecell{0.151}  & 
    \makecell{0.997}  & 
    \makecell{0.0389}  & 
    \makecell{0.171}   \\
    & $u_y$ & 
    \makecell{0.146}  & 
    \makecell{0.997}  & 
    \makecell{0.0367}  & 
    \makecell{0.166}   \\
    & $u_z$ & 
    \makecell{0.166}  & 
    \makecell{0.997}  & 
    \makecell{0.0317}  & 
    \makecell{0.147}   \\
    
\bottomrule
\end{tabular}
  \end{small}
  \caption{{\textbf{Error metrics for different models at $T=0.8$ for the Taylor--Green vortex experiment.}}}
\label{tab:TG2}
  \end{table}

\clearpage
\newpage

\begin{table}[t!]
\centering
\begin{small}
  \begin{tabular}{ c c c c c c}
    \toprule
        & 
         &
         \makecell{$e_\mu$} & \makecell{$e_\sigma$} & \makecell{$W_1$} & \makecell{$\text{CRPS}_G$}  \\
    \midrule
    \midrule
    \multirow{3}{*}{\centering{0-Shot}} 
    & $u_x$ 
    & \makecell{{ 0.230}}  & \makecell{{ 0.240}}  & \makecell{{ 0.100}}  & \makecell{{ 0.422}}  \\
    & $u_y$ 
    & \makecell{{ 0.507}}  & \makecell{{ 0.228}}  & \makecell{{ 0.051}}  & \makecell{{ 0.344}} \\
    & $u_z$ 
    & \makecell{{ 0.468}}  & \makecell{{ 0.227}} & 
    \makecell{{ 0.051}}  & \makecell{{ 0.343}} \\
    \midrule
    \midrule
    \multirow{3}{*}{\centering{Fine-Tuned}} 
    & $u_x$ 
    & \makecell{\bf{0.097}}  & \makecell{\bf{0.136}}  & \makecell{\bf{0.037}}  & \makecell{\bf{0.339}}   \\
    & $u_y$ 
    & \makecell{\bf{0.309}}  & \makecell{\bf{0.133}}  & \makecell{\bf{0.034}}  & \makecell{\bf{0.309}}  \\
    & $u_z$ 
    & \makecell{\bf{0.309}}  & \makecell{\bf{0.134}}  & \makecell{\bf{0.032}}  & \makecell{\bf{0.310}}  \\
    
\bottomrule
\end{tabular}
  \end{small}
  \caption{\textbf{Error metrics for GenCFD for the cylindrical shear flow experiment at $T=1$.} The tests were performed with data from a different distribution than the ones from Table~\ref{tab:1}.}
  \label{tab:4}
  \end{table}

\begin{table}[t!]
\centering
\begin{small}
  \begin{tabular}{ c c c c c c}
    \toprule
        & 
         &
         \makecell{$T=0.25$} & \makecell{$T=0.5$} & \makecell{$T=0.75$} & \makecell{$T=1$}  \\
    \midrule
    \midrule
    \multirow{3}{*}{\centering{GenCFD}} 
    & $u_x$ & \makecell{0.016}  & \makecell{0.022}  & \makecell{0.027}  & \makecell{0.034}   \\
    & $u_y$ & \makecell{0.013}  & \makecell{0.021}  & \makecell{0.025}  & \makecell{0.030}  \\
    & $u_z$ & \makecell{0.014}  & \makecell{0.022}  & \makecell{0.024}  & \makecell{0.032}  \\
    
\bottomrule
\end{tabular}
  \end{small}
  \caption{\textbf{Errors in Wasserstein metric for the distribution generated by GenCFD when compared to the ground truth distribution for the cylindrical shear flow experiment at different times in the evolution.}}
\label{tab:5}
  \end{table}

\begin{table}[t!]
\centering
\begin{small}
  \begin{tabular}{ c c c c c c c}
    \toprule
        & 
         &
         \makecell{$K=8$} & \makecell{$K=16$} & \makecell{$K=32$} & \makecell{$K=64$} &  \makecell{$K=128$}  \\
    \midrule
    \midrule
    \multirow{3}{*}{\centering{GenCFD}} 
    & $u_x$ & \makecell{1.736}  & \makecell{0.051}  & \makecell{0.035}  &  \makecell{0.034} & \makecell{0.034}   \\
    & $u_y$ & \makecell{2.311}  & \makecell{0.049}  & \makecell{0.031}   & \makecell{0.030} & \makecell{0.030}  \\
    & $u_z$ & \makecell{2.087}  & \makecell{0.045}  & \makecell{0.033}   & \makecell{0.032} & \makecell{0.032}  \\
    
\bottomrule
\end{tabular}
  \end{small}
   \caption{\textbf{Errors in Wasserstein metric for the distribution generated by GenCFD, with different number of steps of the reverse SDE \eqref{eq:si_rsde} when compared to the ground truth distribution for the cylindrical shear flow experiment at $T=1$.}}
\label{tab:6}
\end{table}

\begin{table}[h!]
\centering
\begin{small}
  \begin{tabular}{ c c c c c c }
    \toprule
        \makecell{Benchmark} & 
        \makecell{Ground truth \\ (GPU)}& 
         \makecell{Ground truth \\ (CPU)} & 
         \makecell{GenCFD \\ (GPU)} & 
         \makecell{\textbf{Speedup} \\ \textbf{(wrt GPU})} & 
         \makecell{\textbf{Speedup} \\ \textbf{(wrt CPU})} \\
    \midrule
    \midrule
    \multirow{1}{*}{\centering{TG, CSF}}
    & 
    \makecell{$1.07 \times 10^{1}$ secs}  & 
    \makecell{$0.72 \times 10^{3}$ secs}  & 
    \makecell{$0.450 \times 10^{0}$ secs}  & 
    \makecell{$\mathbf{2.37 \times 10^{1}}$} &
    \makecell{$\mathbf{1.60 \times 10^{3}}$}\\
    \midrule
    \midrule
    \multirow{1}{*}{\centering{CSI}}
    & 
    \makecell{$0.390 \times 10^{3}$ secs} & 
    \makecell{$1.80 \times 10^{4}$ secs}  & 
    \makecell{$0.450 \times 10^{0}$ secs} & 
    \makecell{$\mathbf{0.87 \times 10^{3}}$} &
    \makecell{$\mathbf{0.40 \times 10^{5}}$}   \\
    \midrule
    \midrule    
    \multirow{1}{*}{\centering{NF}} &
    \makecell{$1.20 \times 10^{3}$ secs}  & 
    \makecell{$1.17 \times 10^{4}$ secs}  & 
    \makecell{$1.45 \times 10^{0}$ secs}  & 
    \makecell{$\mathbf{0.83 \times 10^{3}}$} &
    \makecell{$\mathbf{0.81 \times 10^{4}}$}  \\
    \midrule
    \midrule
    \multirow{1}{*}{\centering{CBL}} &
    \makecell{N/A}                   & 
    \makecell{$0.48 \times 10^{5}$ secs}  & 
    \makecell{$0.38 \times 10^{1}$ secs} & 
    \makecell{N/A} &
    \makecell{$\mathbf{1.25 \times 10^{4}}$}     \\
\bottomrule
\end{tabular}
\end{small}
\caption{\textbf{Runtimes and speedup for generating a single sample with the CFD solvers and with GenCFD.} The CFD solvers were used to simulate the ground truth (see Section~\ref{sec:datasets} on which CFD solver is used for which experiment). The inference time to generate a single sample with GenCFD was measured on a NVIDIA RTX 4090 GPU with 24GB of memory. The term N/A implies that the corresponding GPU or CPU code is not available or has not been tested for the corresponding CFD solver. Note that in this comparison different machines have been used than for the sample generation. The computation time in seconds is rounded and includes I/O operations. For the TG, CSF, and CSI the ground truth (CPU) has been computed on a single Intel Core i7-9750H with 6 cores. The respective ground truth (GPU) was computed on an NVIDIA H100 GPU with 96GB of memory. For the NF the ground truth (CPU) has been computed on two Intel Xeon Gold 6326 with 16 cores each and the ground truth (GPU) has been computed on an NVIDIA H100 GPU, respectively. For CBL the ground truth data was generated on a cluster with diverse CPU hardware (mostly  AMD EPYC 7H12 and  AMD EPYC 7763 processors), and the mean runtime on a single core is reported.
}
\label{tab:10}
\end{table}

\clearpage
\newpage
\section{Supplementary Figures}

\begin{figure}[h!]
	\centering
	\includegraphics[width=15.5cm]{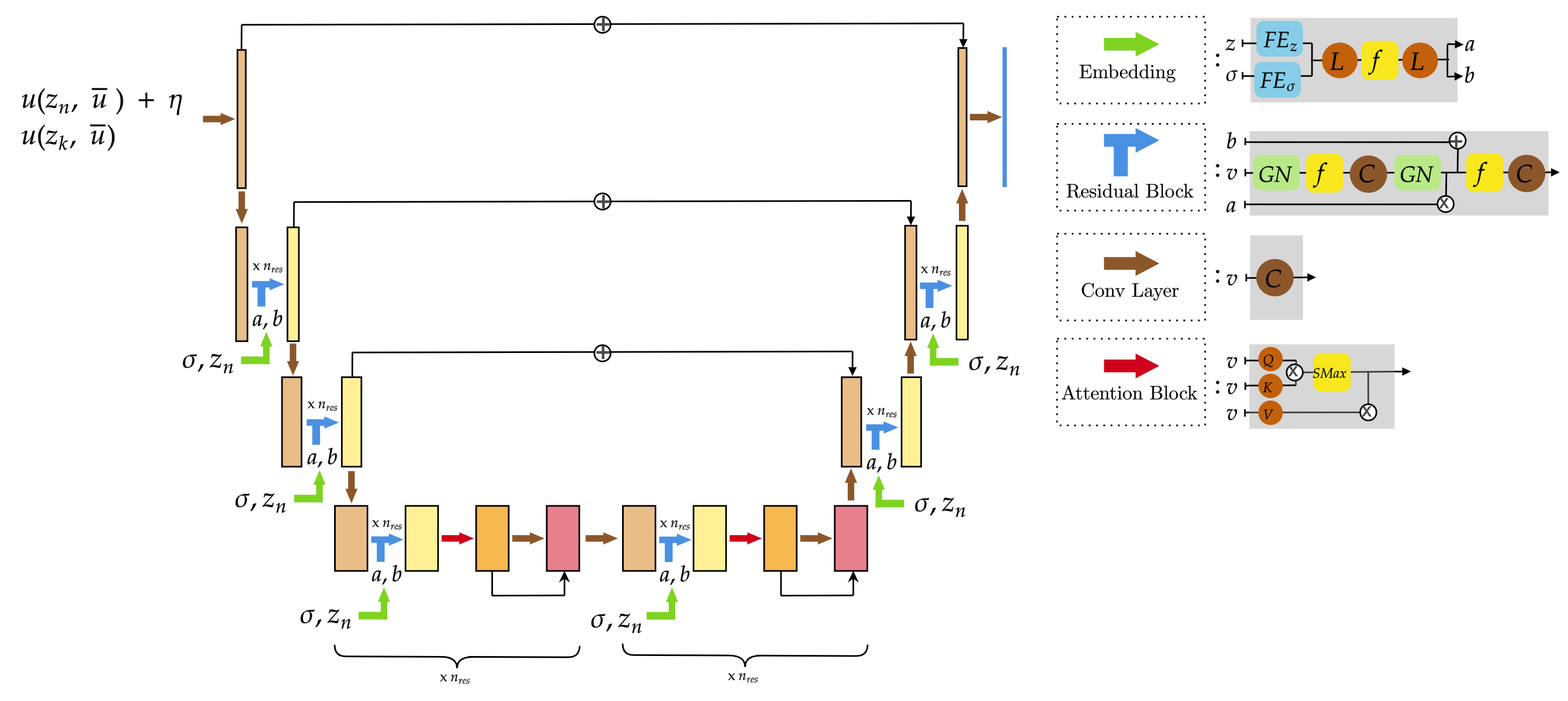} 
    \caption{\textbf{A schematic of the architecture of the UViT neural network which is used as the model for the denoiser in GenCFD.} A detailed description of the model is provided in Materials and Methods.}
    \label{fig:s1}
\end{figure}

\clearpage
\newpage

\begin{figure}[h!]
\minipage{\linewidth}
\minipage{0.33\textwidth}
\includegraphics[width=\linewidth, clip, trim=100 125 100 125, draft=false]{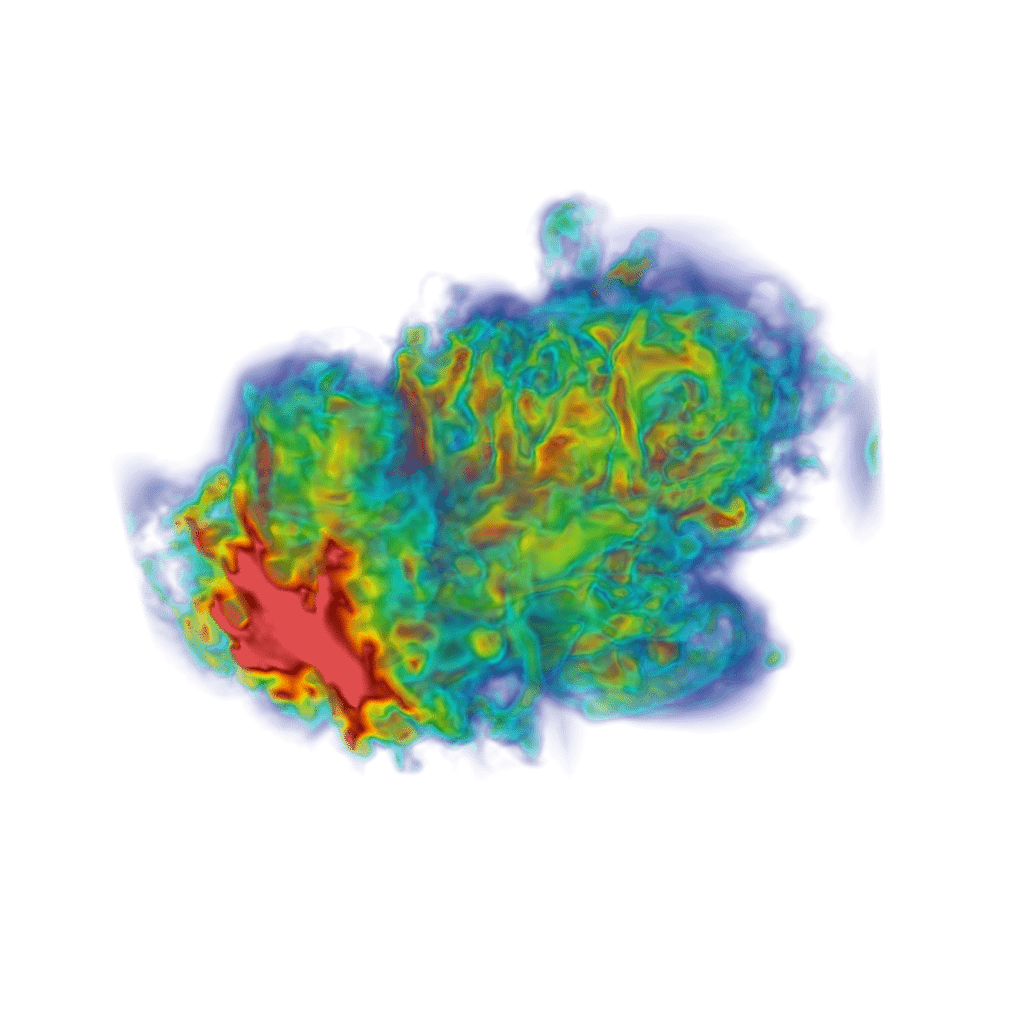}
\endminipage
\minipage{0.33\textwidth}
{\includegraphics[width=\linewidth, clip, trim=100 125 100 125, draft=false]{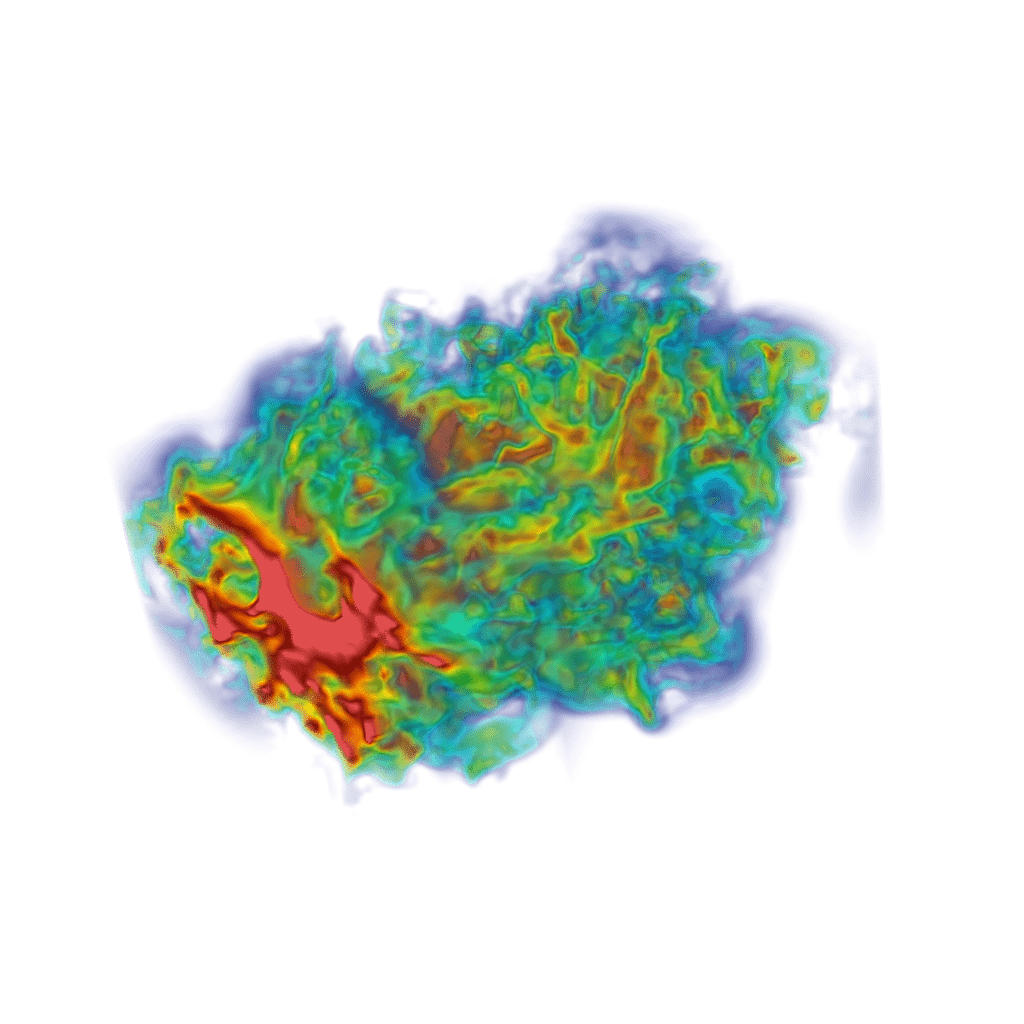}}
\endminipage
\minipage{0.33\textwidth}
{\includegraphics[width=\linewidth, clip, trim=100 125 100 125, draft=false]{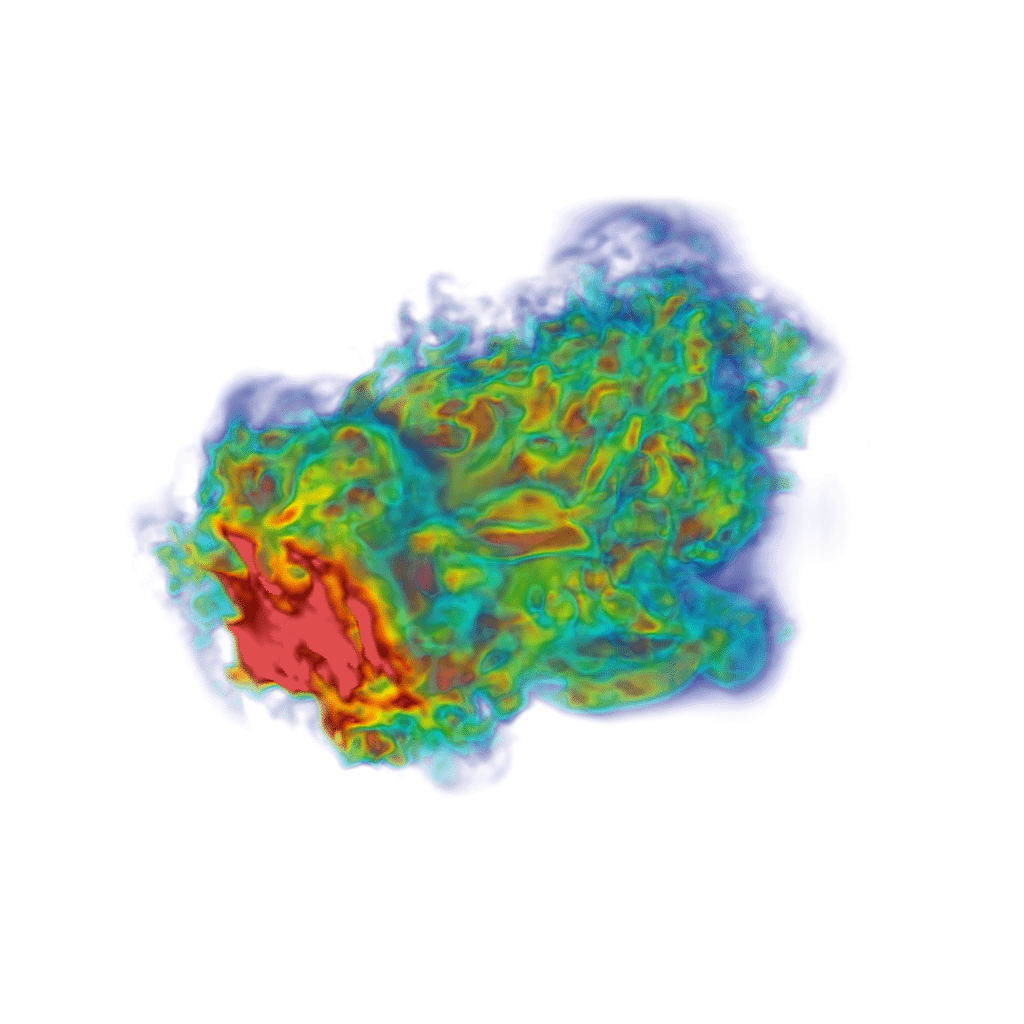}}
\endminipage
\endminipage

\minipage{\linewidth}
\minipage{0.33\textwidth}
\includegraphics[width=\linewidth, clip, trim=100 125 100 125, draft=false]{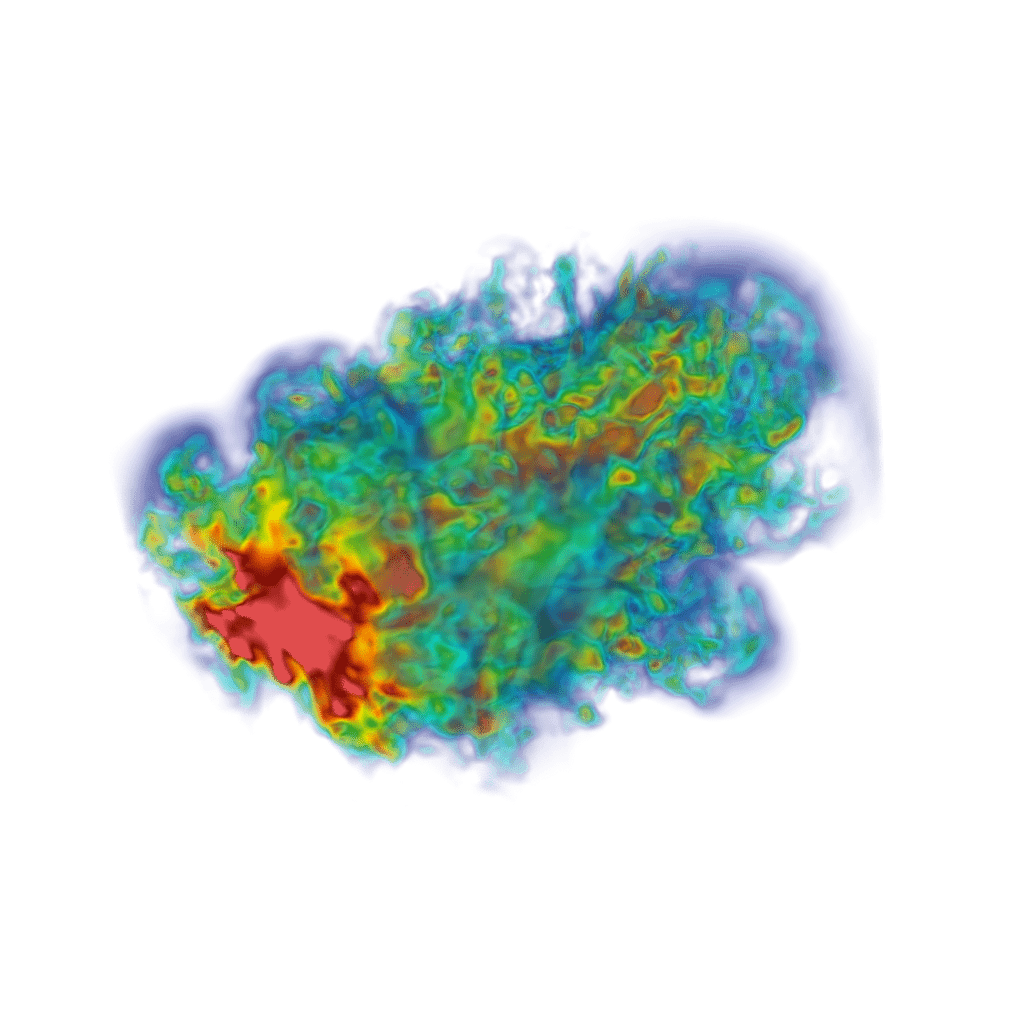}
\endminipage
\minipage{0.33\textwidth}
{\includegraphics[width=\linewidth, clip, trim=100 125 100 125, draft=false]{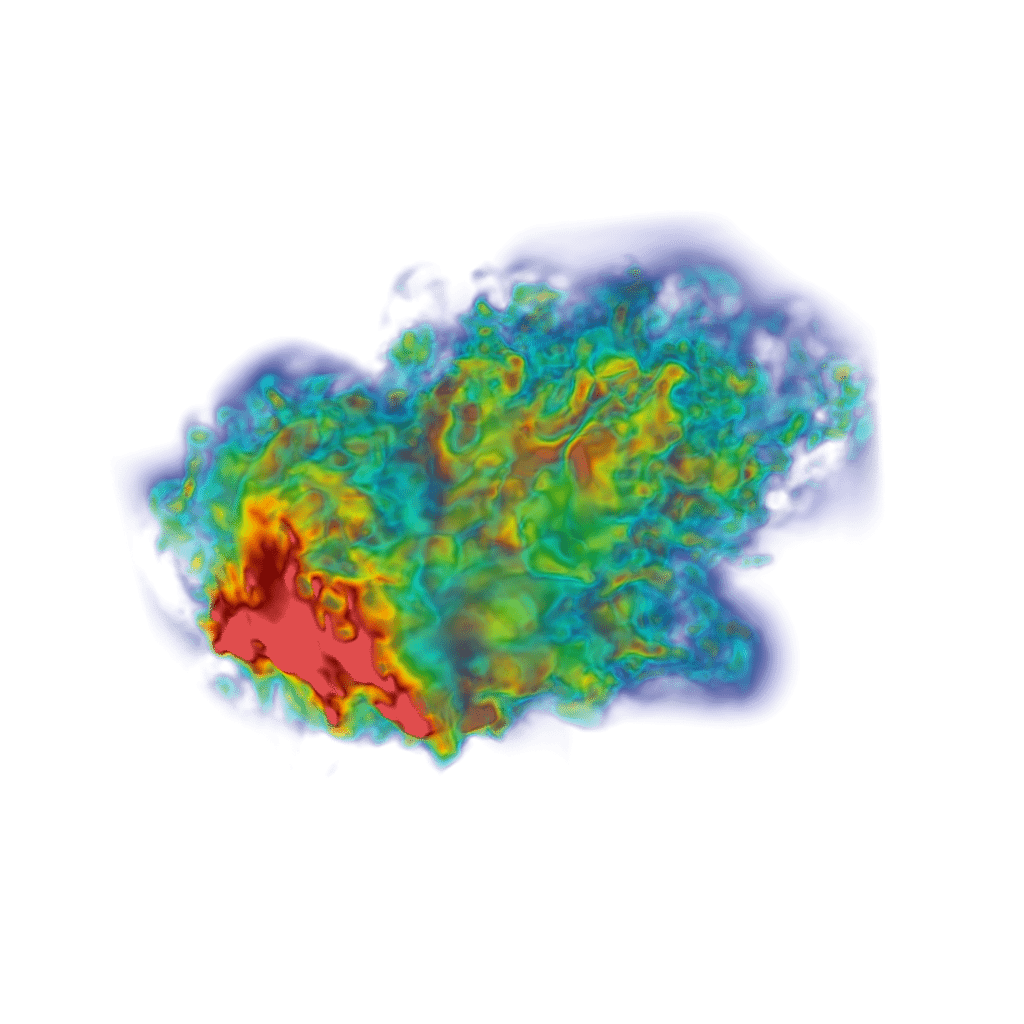}}
\endminipage
\minipage{0.33\textwidth}
{\includegraphics[width=\linewidth, clip, trim=100 125 100 125, draft=false]{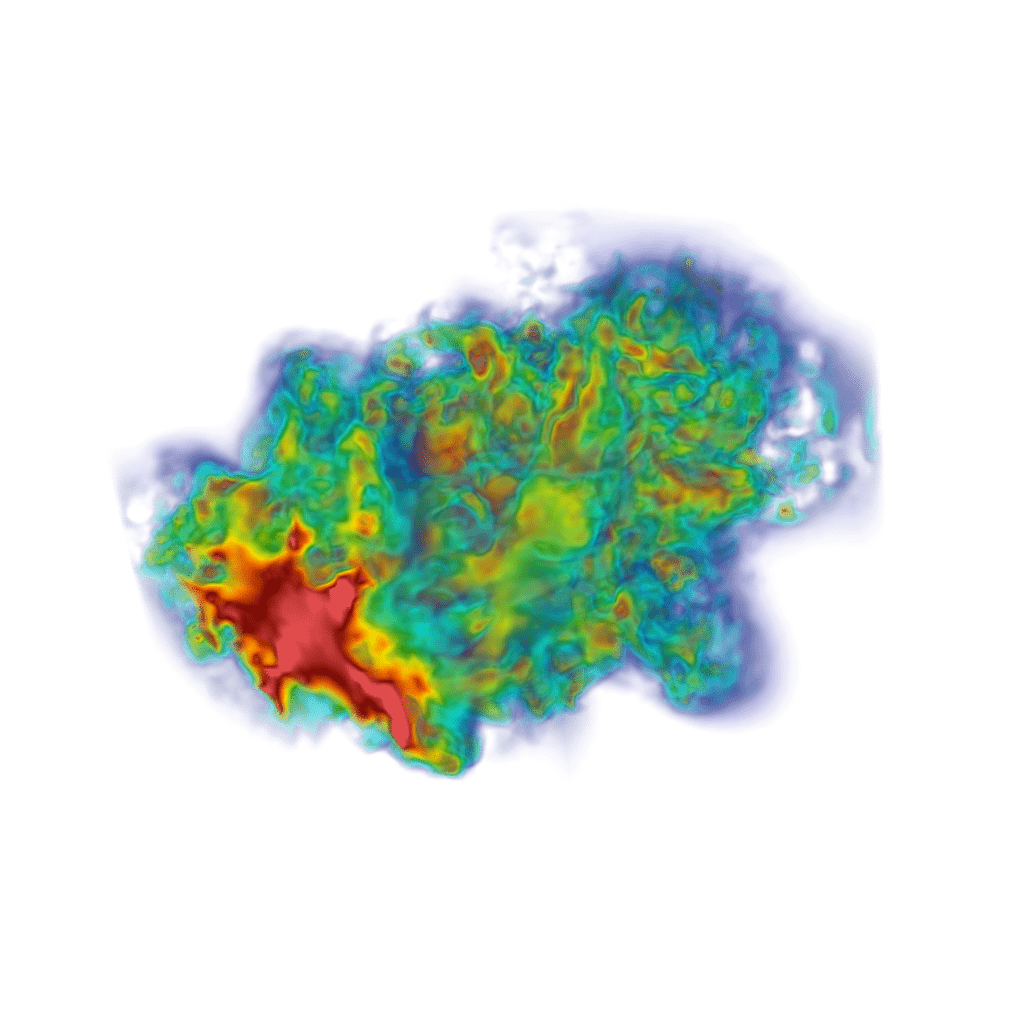}}
\endminipage
\endminipage

\minipage{\linewidth}
\minipage{0.33\textwidth}
\includegraphics[width=\linewidth, clip, trim=100 125 100 125, draft=false]{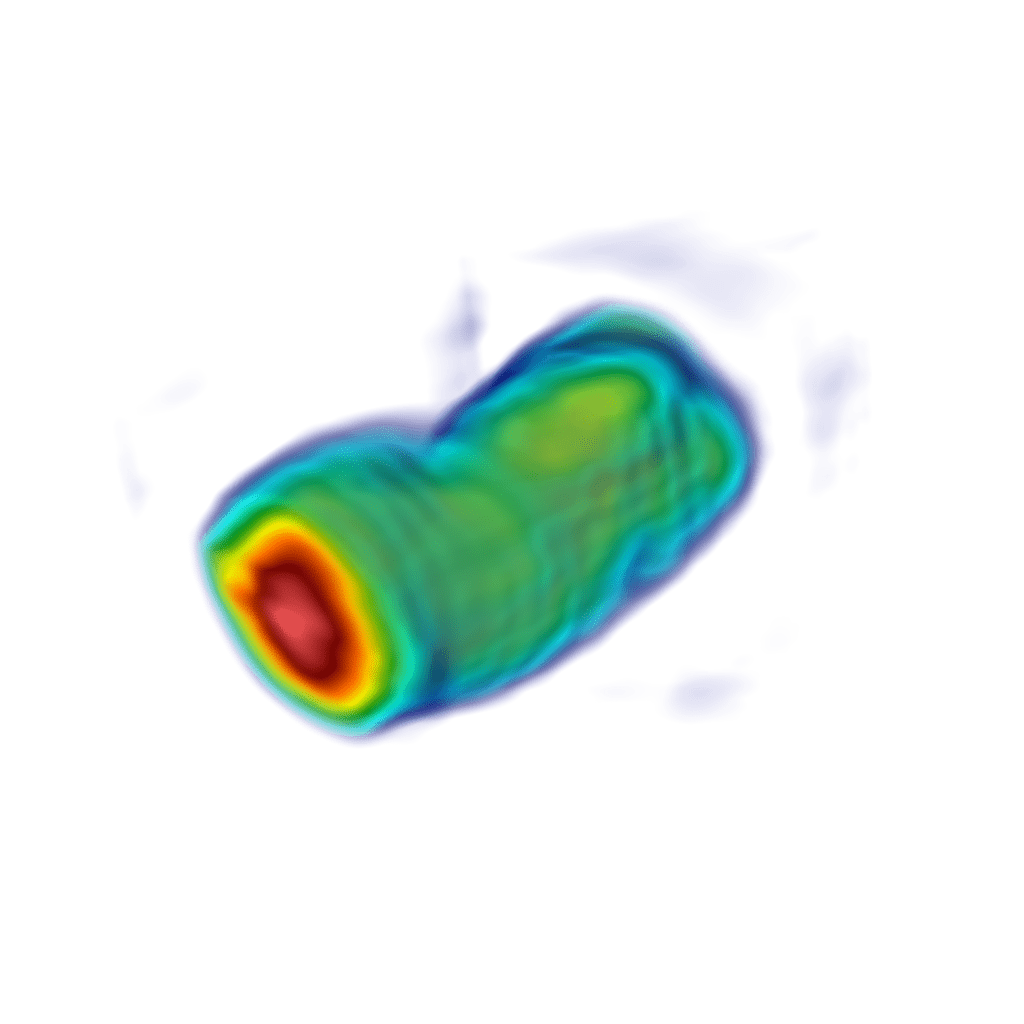}
\endminipage
\minipage{0.33\textwidth}
{\includegraphics[width=\linewidth, clip, trim=100 125 100 125, draft=false]{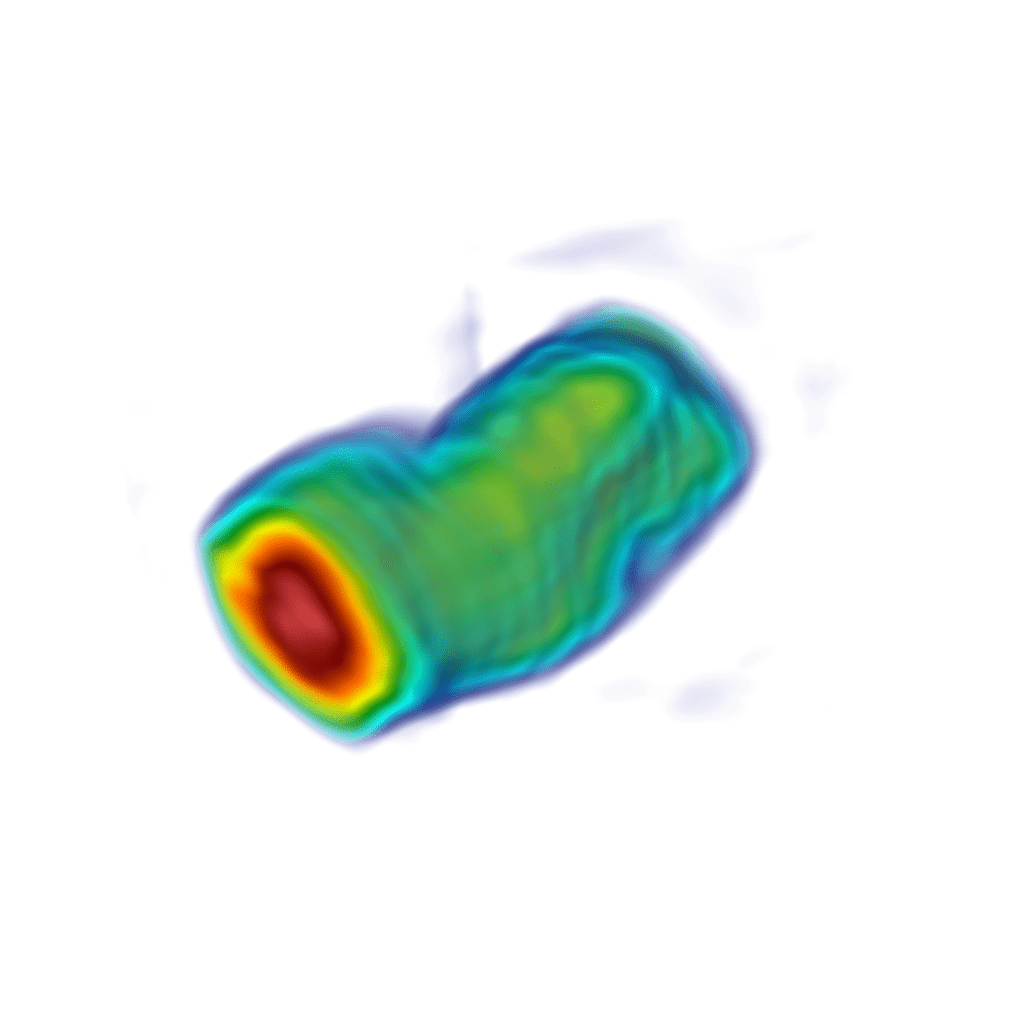}}
\endminipage
\minipage{0.33\textwidth}
{\includegraphics[width=\linewidth, clip, trim=100 125 100 125, draft=false]{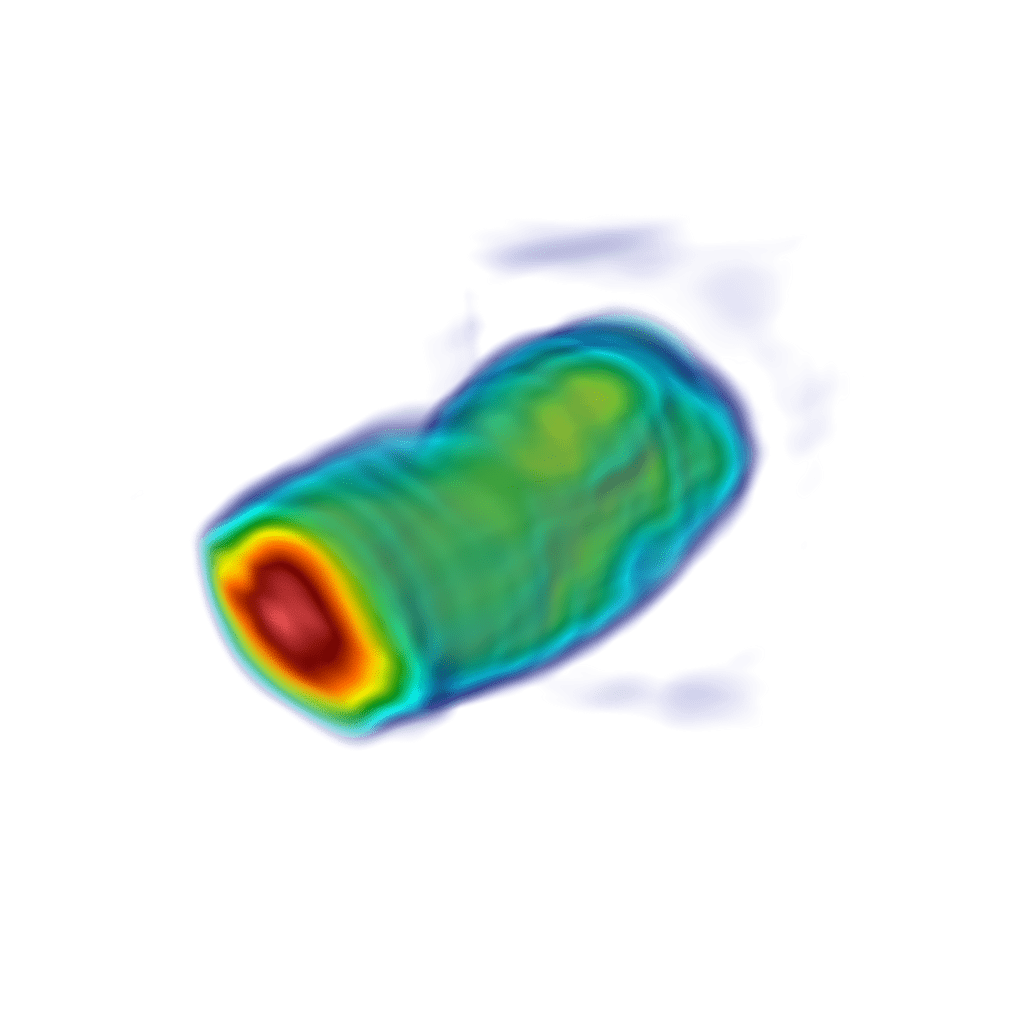}}
\endminipage
\endminipage
\caption{\textbf{Visualization of pointwise kinetic energy for 3 randomly generated samples for the three-dimensional cylindrical shear flow experiment at time $T=1$.} Ground truth (Top Row), GenCFD (Middle Row) and C-FNO (Bottom Row). The colormap for all the figures ranges from $0.6$ (dark blue) to $1.7$ (dark red).}
\label{fig:s2}
\end{figure}

\begin{figure}[h!]
\minipage{\linewidth}
\minipage{0.33\textwidth}
\includegraphics[width=\linewidth, clip, trim=100 125 100 125, draft=false]{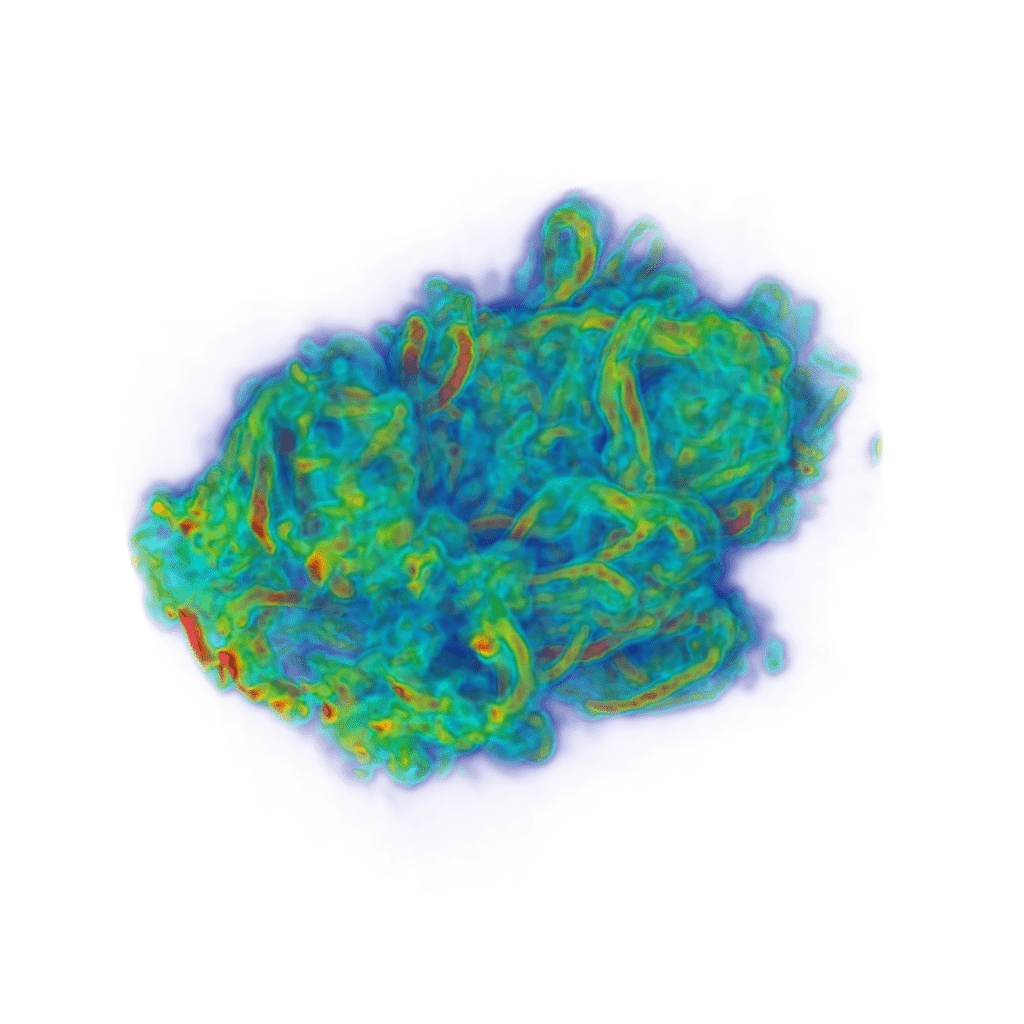}
\endminipage
\minipage{0.33\textwidth}
{\includegraphics[width=\linewidth, clip, trim=100 125 100 125, draft=false]{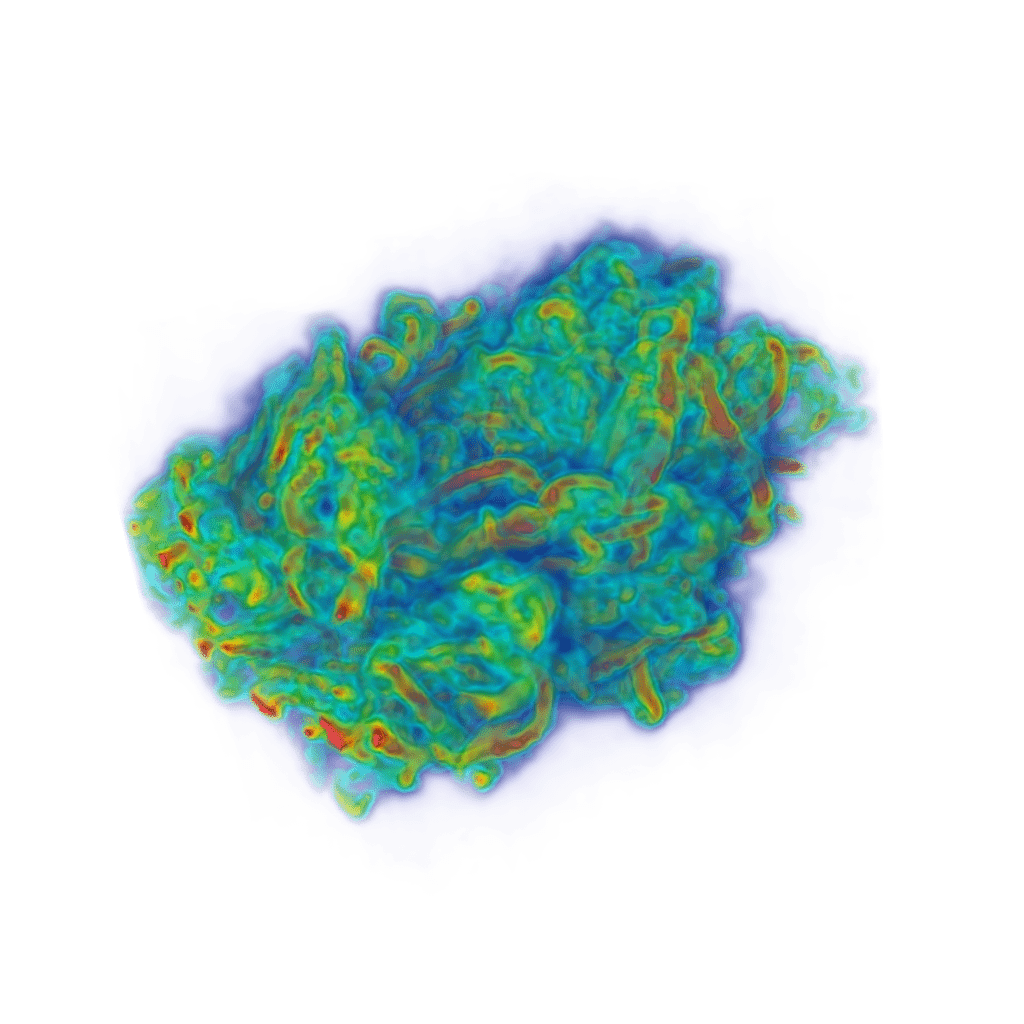}}
\endminipage
\minipage{0.33\textwidth}
{\includegraphics[width=\linewidth, clip, trim=100 125 100 125, draft=false]{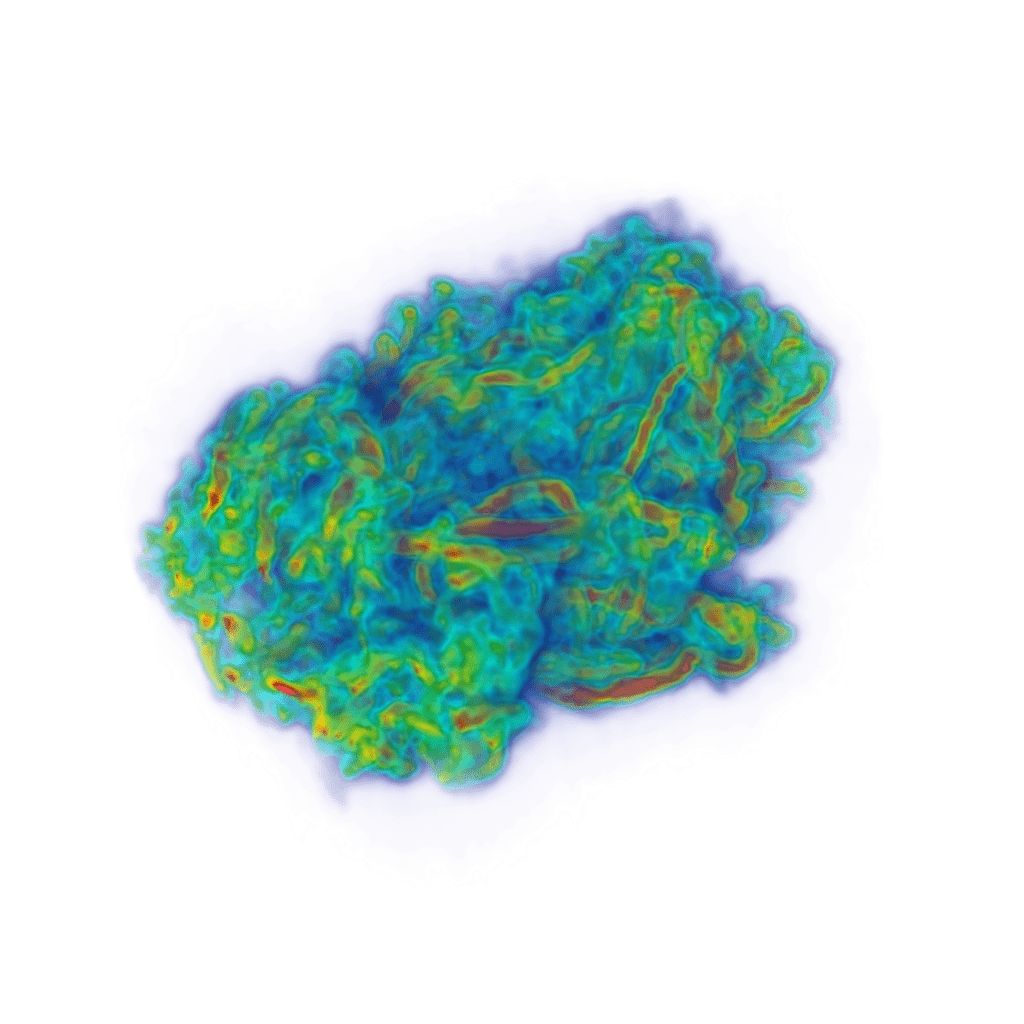}}
\endminipage
\endminipage

\minipage{\linewidth}
\minipage{0.33\textwidth}
\includegraphics[width=\linewidth, clip, trim=100 125 100 125, draft=false]{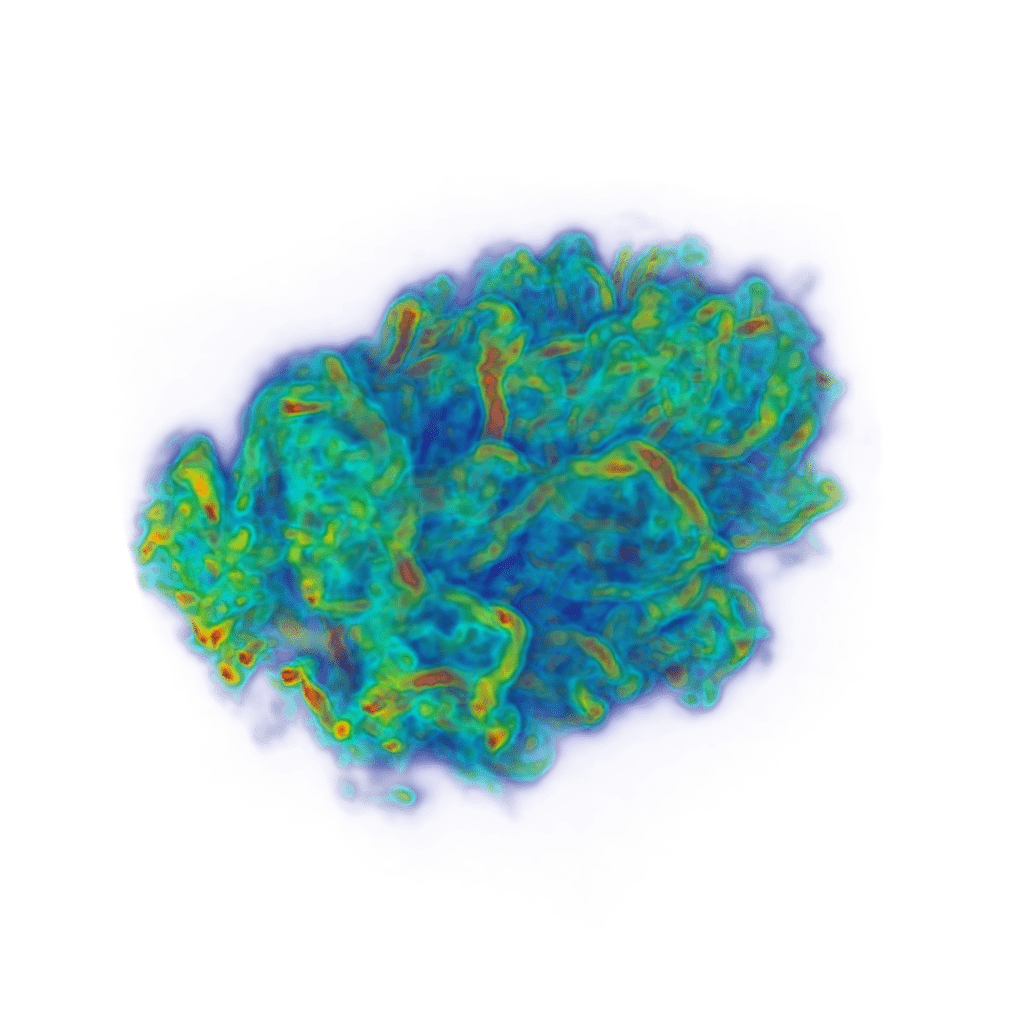}
\endminipage
\minipage{0.33\textwidth}
{\includegraphics[width=\linewidth, clip, trim=100 125 100 125, draft=false]{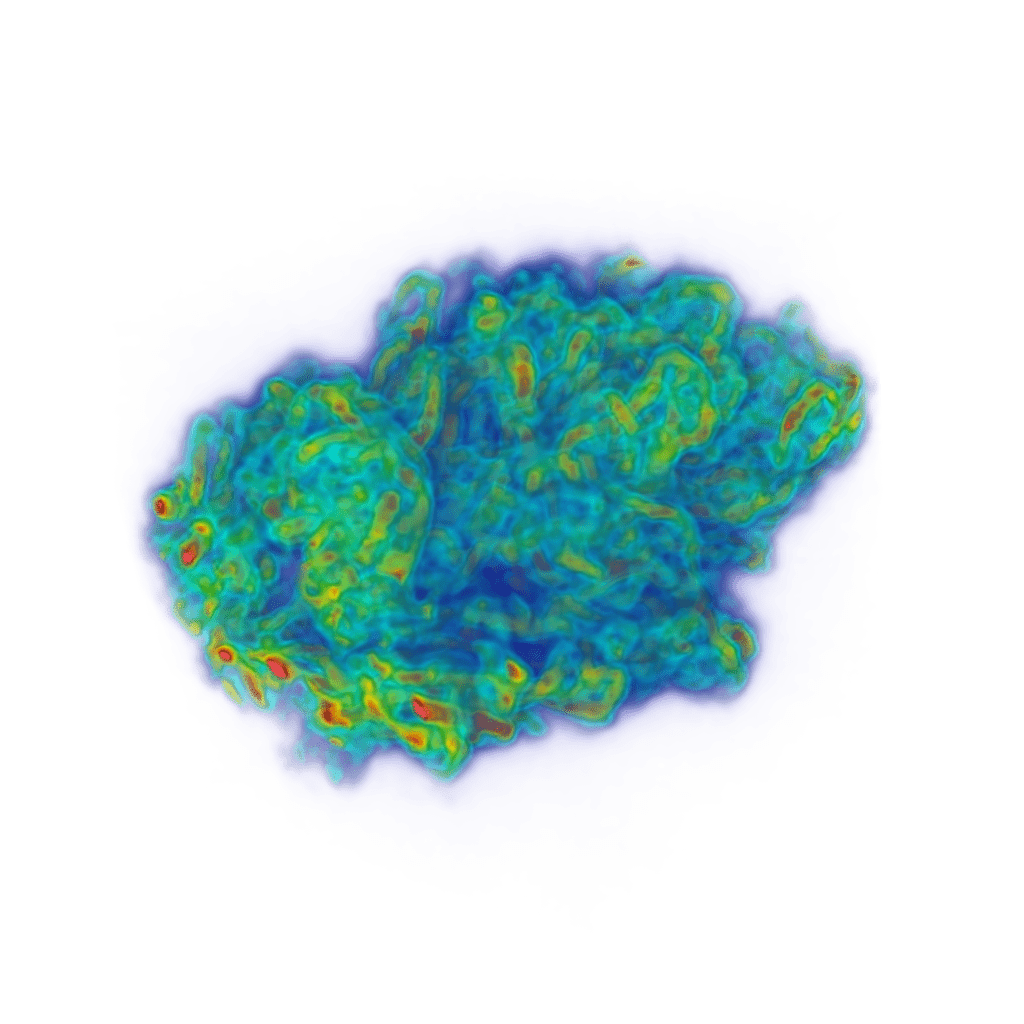}}
\endminipage
\minipage{0.33\textwidth}
{\includegraphics[width=\linewidth, clip, trim=100 125 100 125, draft=false]{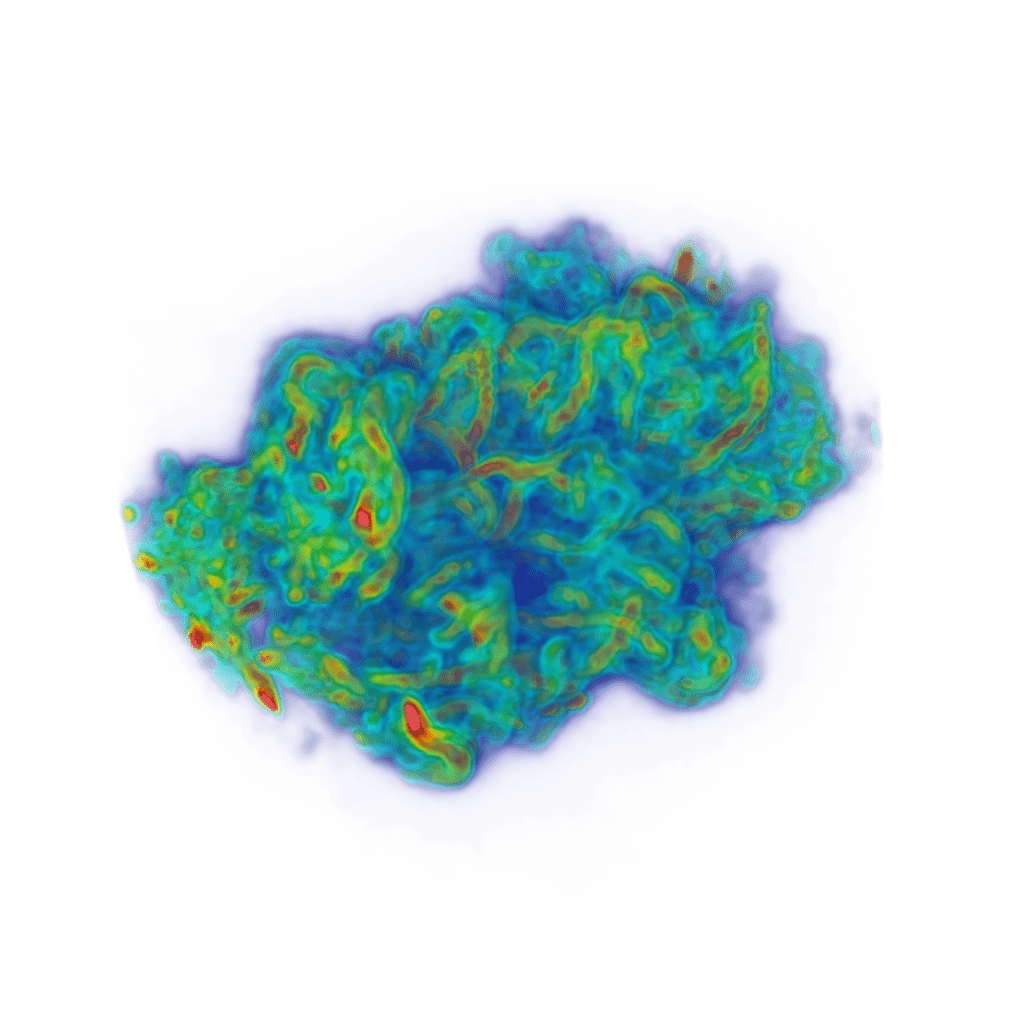}}
\endminipage
\endminipage

\minipage{\linewidth}
\minipage{0.33\textwidth}
\includegraphics[width=\linewidth, clip, trim=80 80 80 80, draft=false]{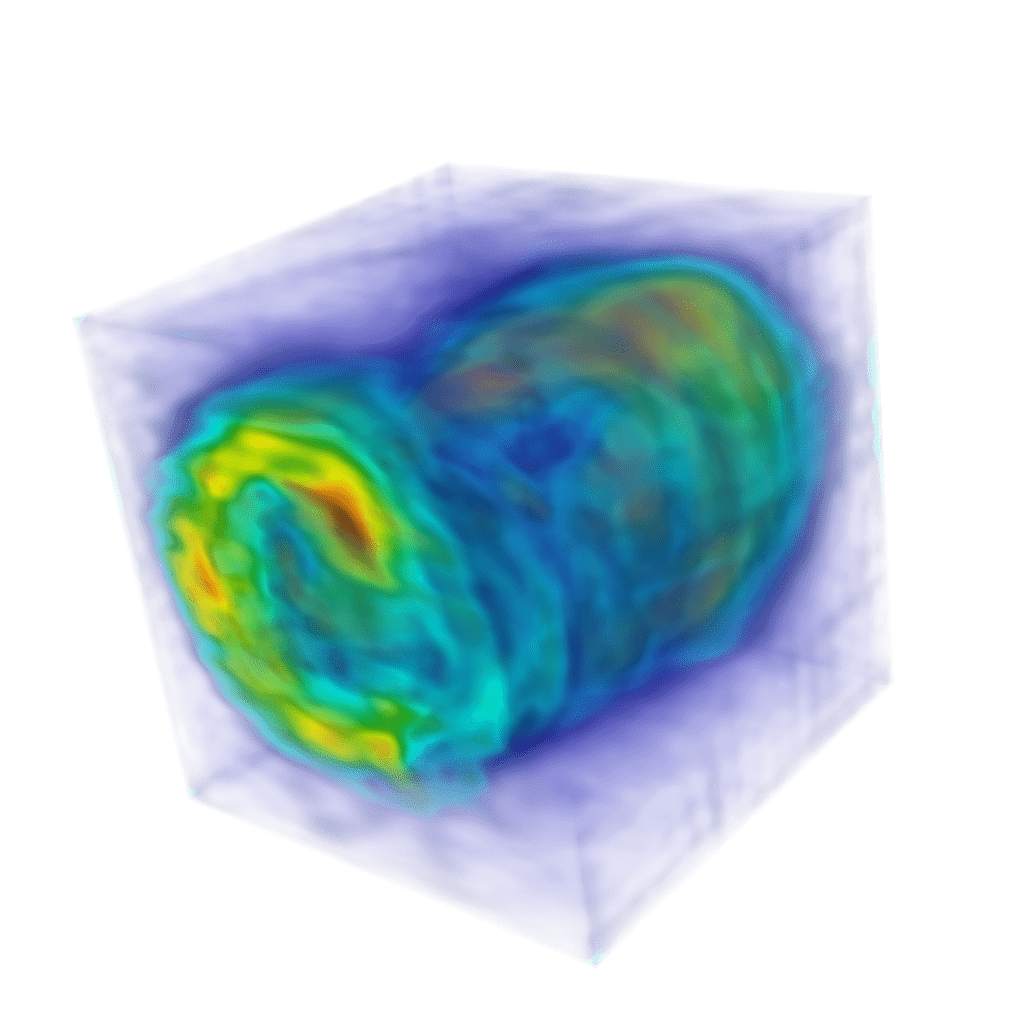}
\endminipage
\minipage{0.33\textwidth}
{\includegraphics[width=\linewidth, clip, trim=80 80 80 80, draft=false]{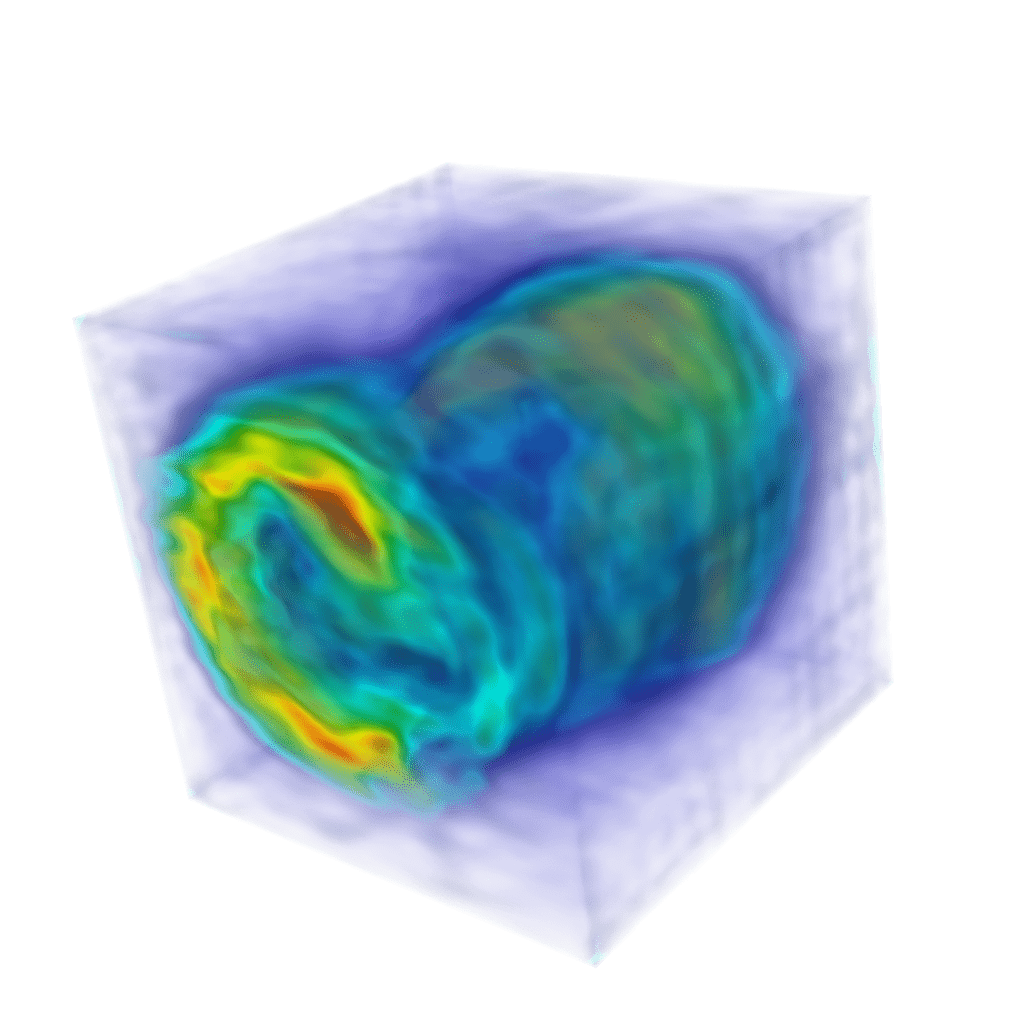}}
\endminipage
\minipage{0.33\textwidth}
{\includegraphics[width=\linewidth, clip, trim=80 80 80 80, draft=false]{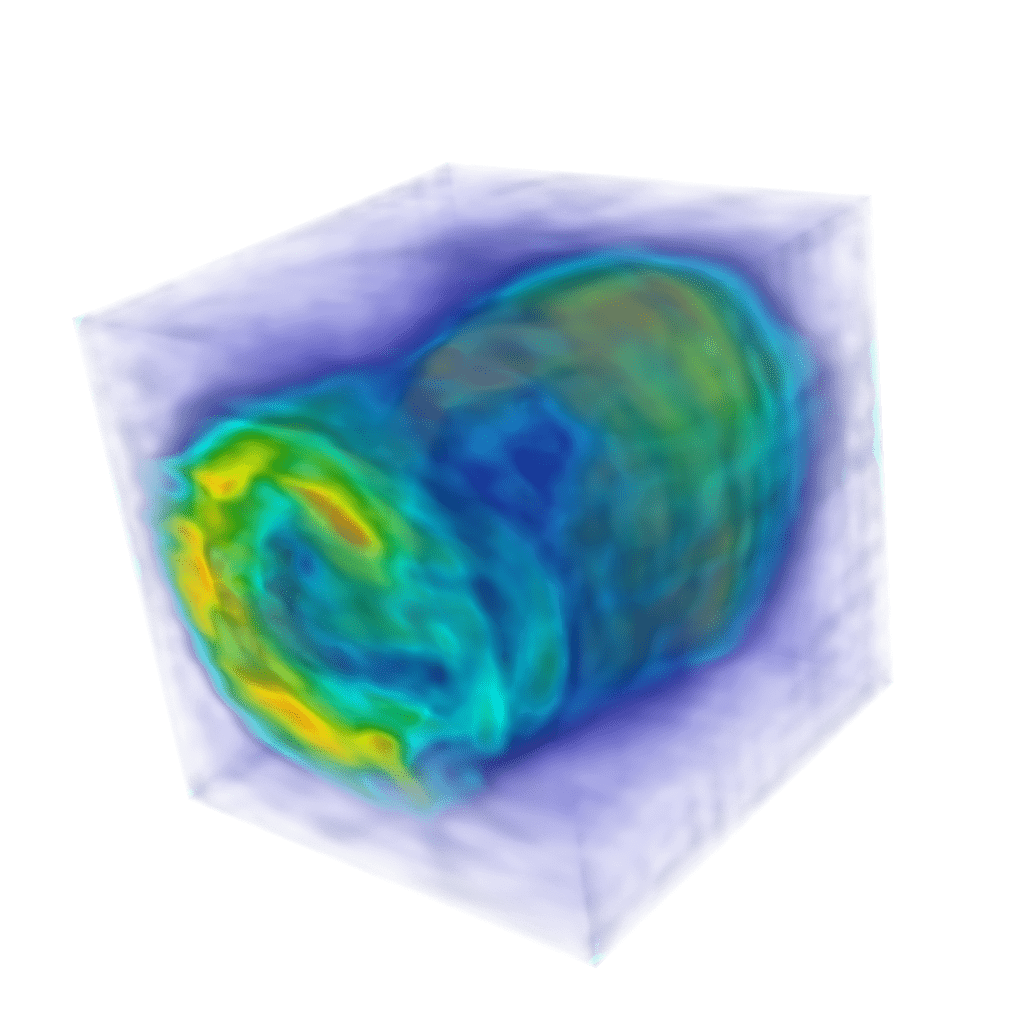}}
\endminipage
\endminipage
\caption{\textbf{Visualization of pointwise vorticity intensity for 3 randomly generated samples for the three-dimensional cylindrical shear flow experiment at time $T=1$.} Ground truth (Top Row), GenCFD (Middle Row) and C-FNO (Bottom Row). The colormap for the top and middle rows ranges from $10^{-4}$ (dark blue) to $40.0$ (dark red), whereas for the bottom row, it ranges from $0.5$ (filtering the low values) to $19.5$.}
\label{fig:s3}
\end{figure}

\begin{figure}[h!]
\minipage{\linewidth}
\minipage{0.25\textwidth}
\includegraphics[width=\linewidth, clip, trim=100 125 100 125]{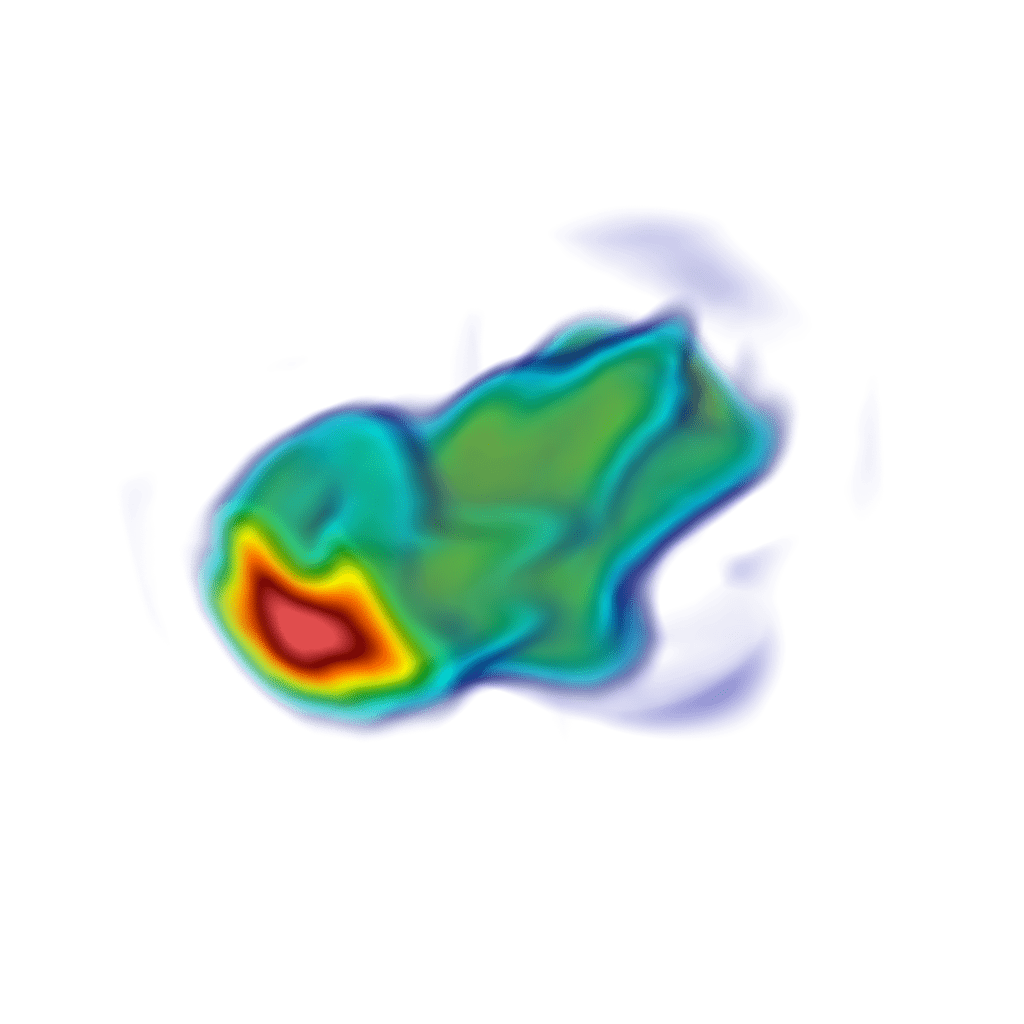}
\endminipage
\minipage{0.25\textwidth}
{\includegraphics[width=\linewidth, clip, trim=100 125 100 125]{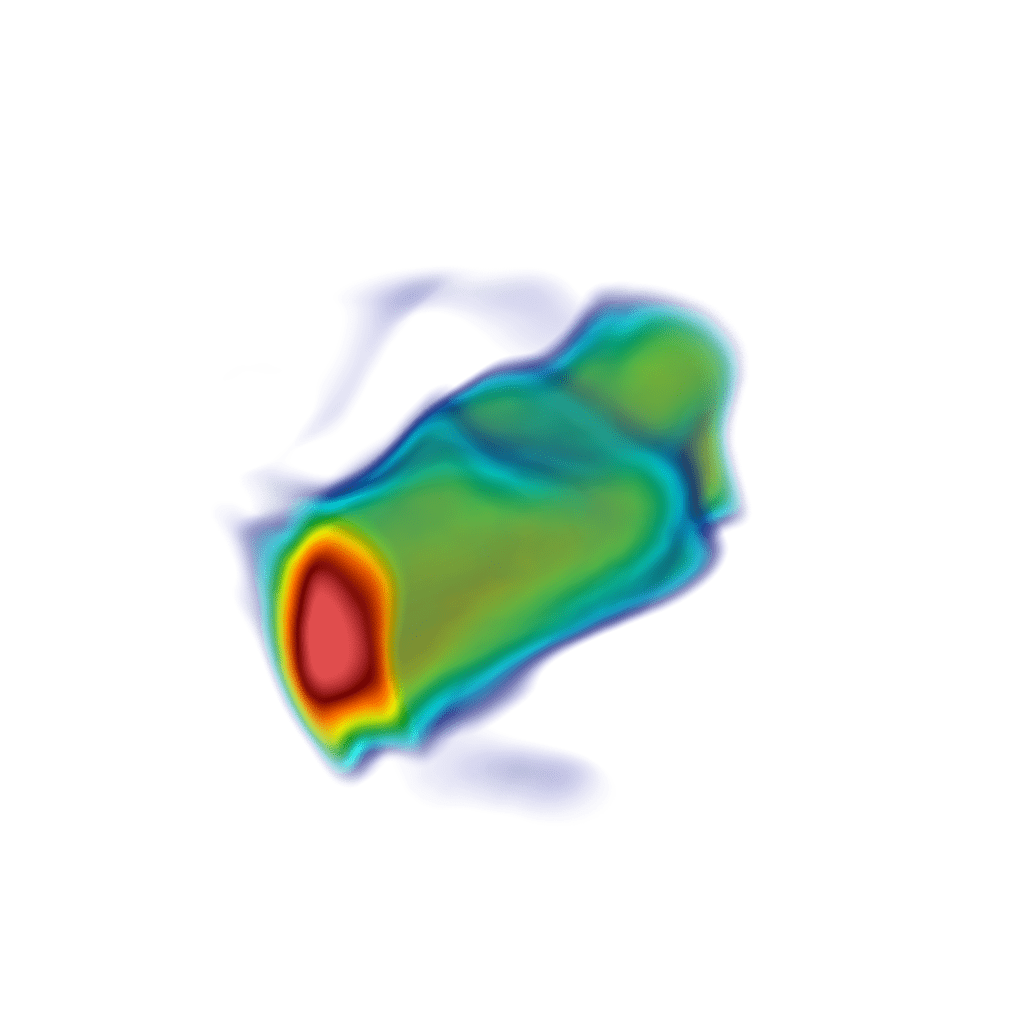}}
\endminipage
\minipage{0.25\textwidth}
{\includegraphics[width=\linewidth, clip, trim=100 125 100 125]{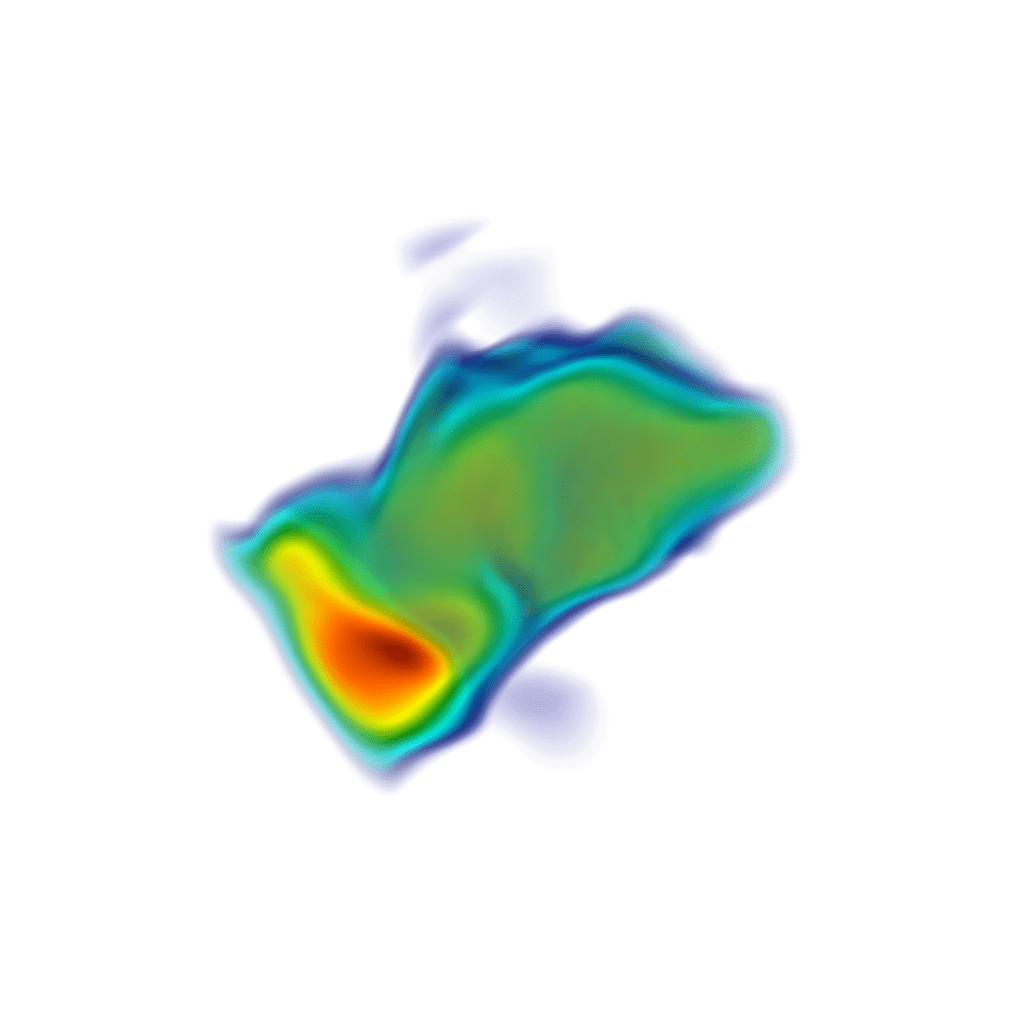}}
\endminipage
\minipage{0.25\textwidth}
{\includegraphics[width=\linewidth, clip, trim=100 125 100 125]{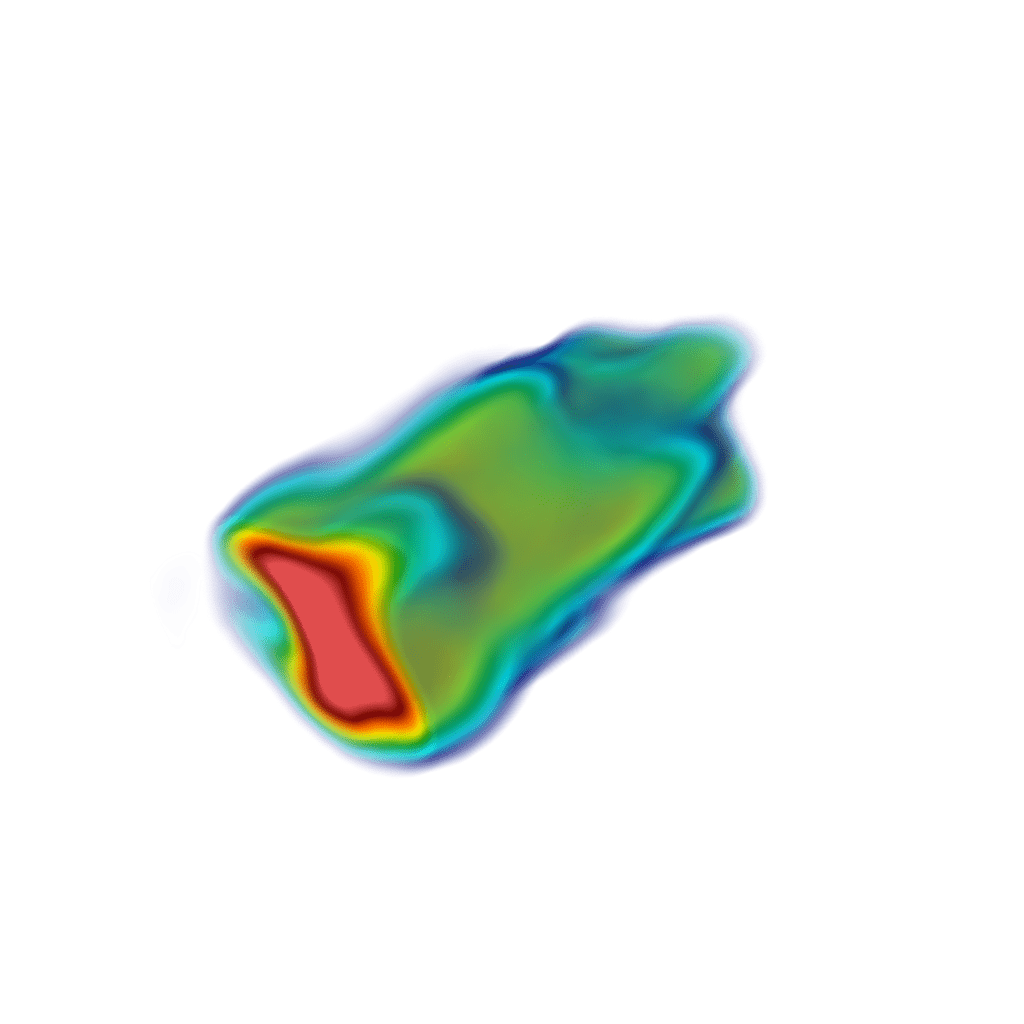}}
\endminipage
\endminipage

\minipage{\linewidth}
\minipage{0.25\textwidth}
\includegraphics[width=\linewidth, clip, trim=100 125 100 125]{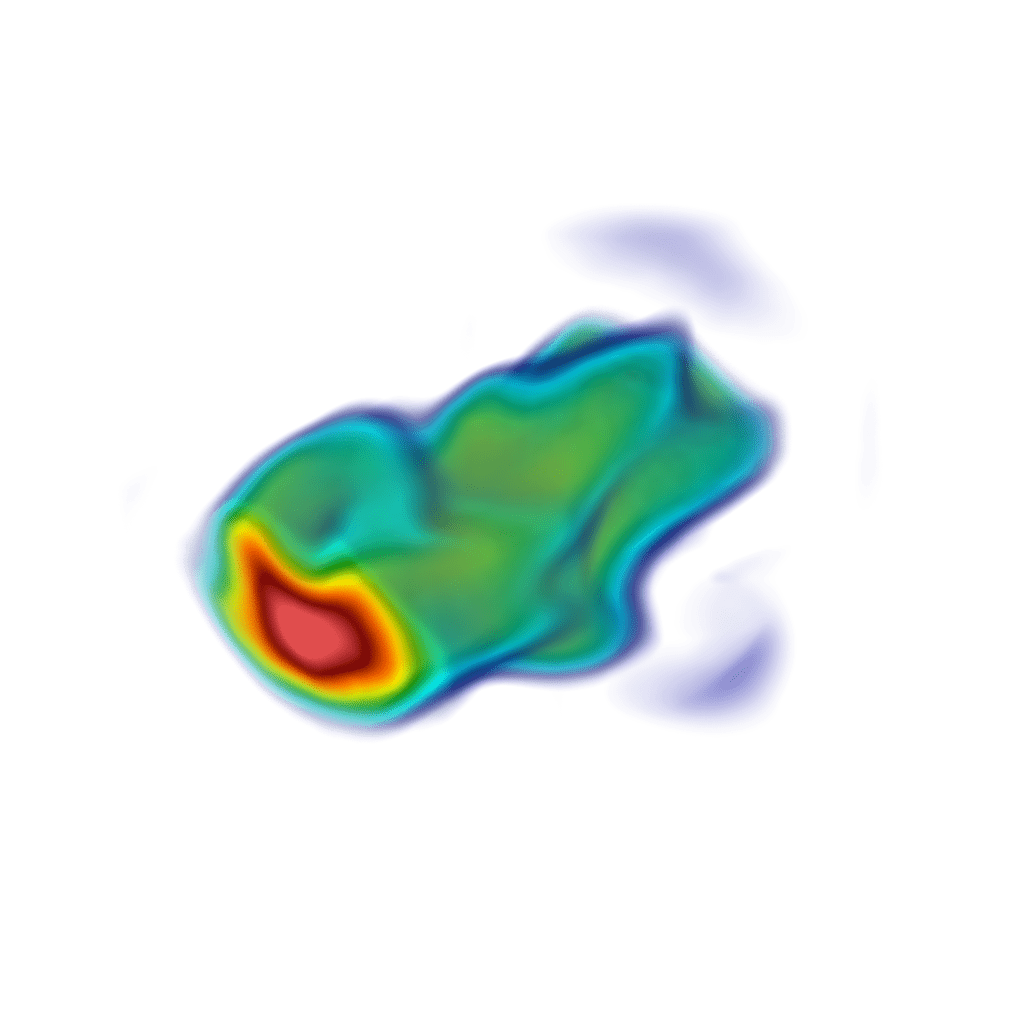}
\endminipage
\minipage{0.25\textwidth}
{\includegraphics[width=\linewidth, clip, trim=100 125 100 125]{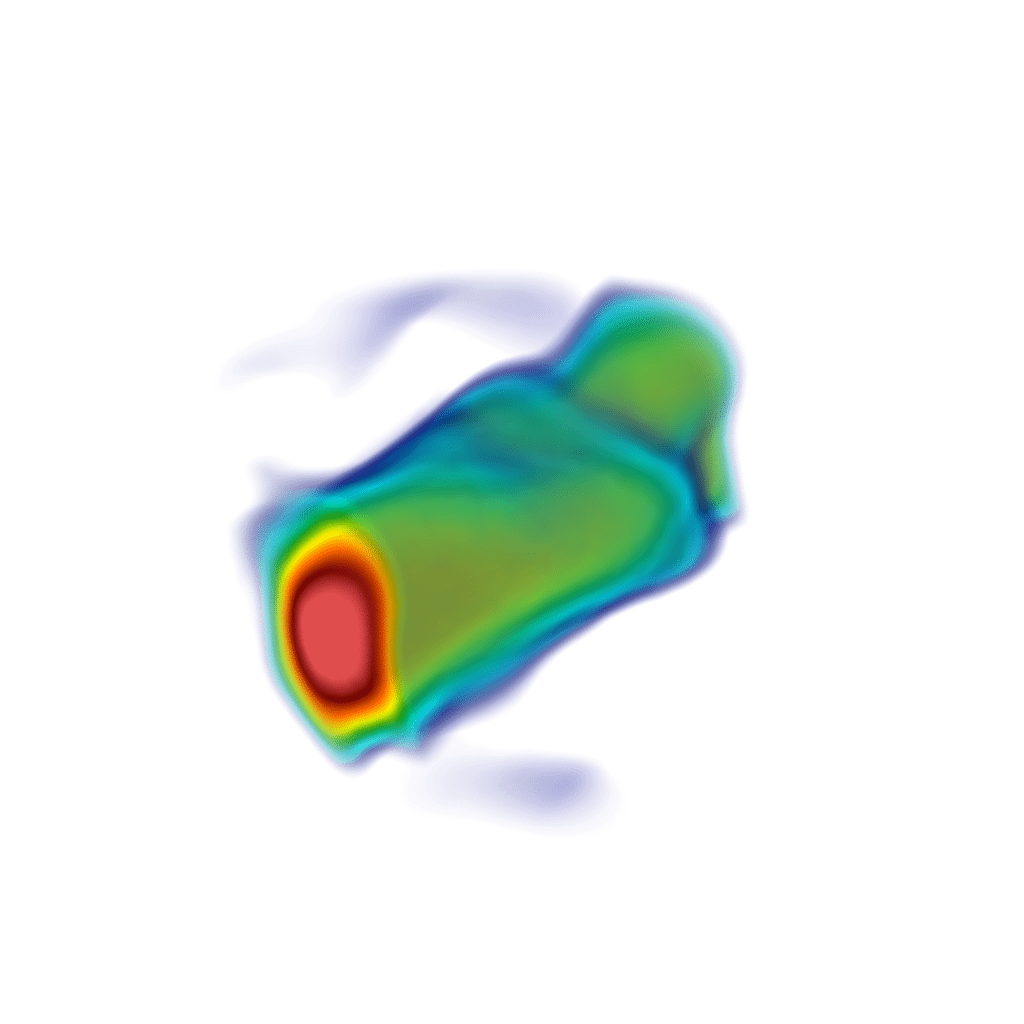}}
\endminipage
\minipage{0.25\textwidth}
{\includegraphics[width=\linewidth, clip, trim=100 125 100 125]{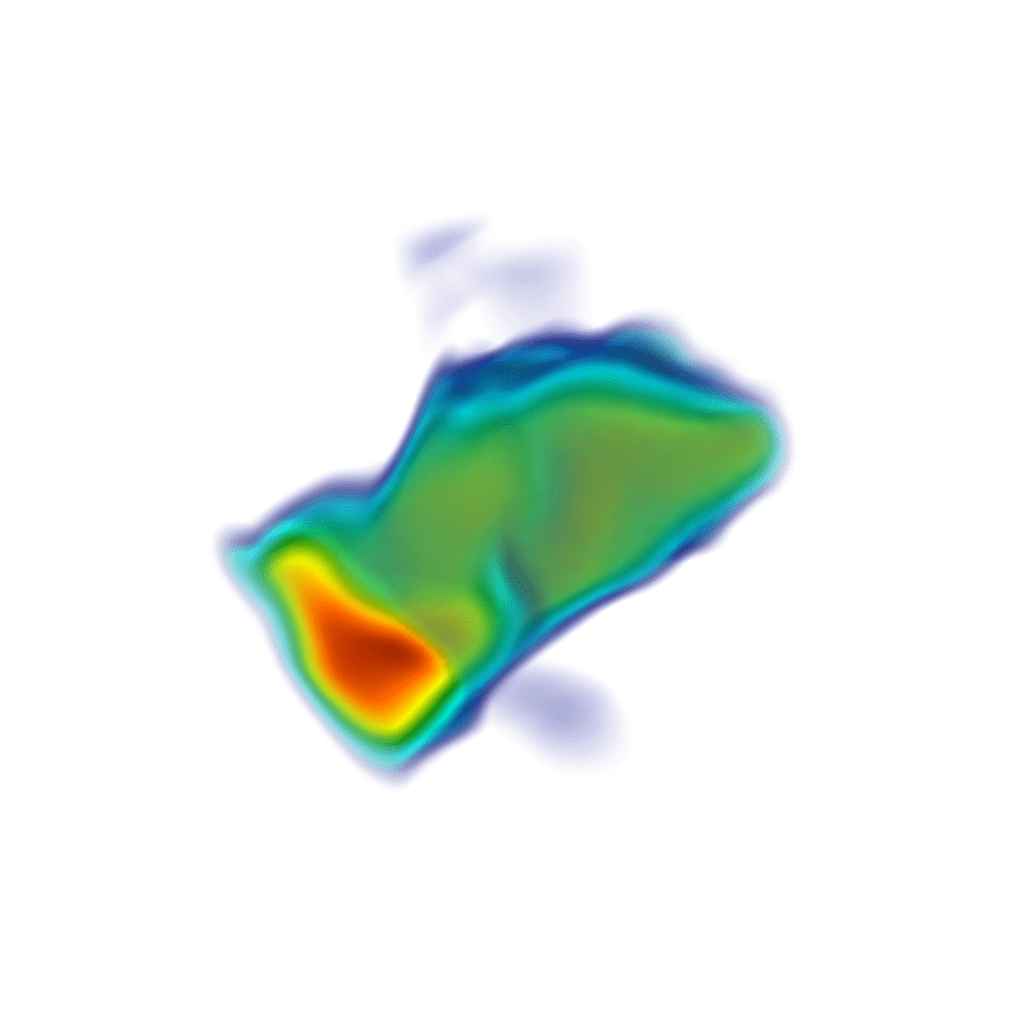}}
\endminipage
\minipage{0.25\textwidth}
{\includegraphics[width=\linewidth, clip, trim=100 125 100 125]{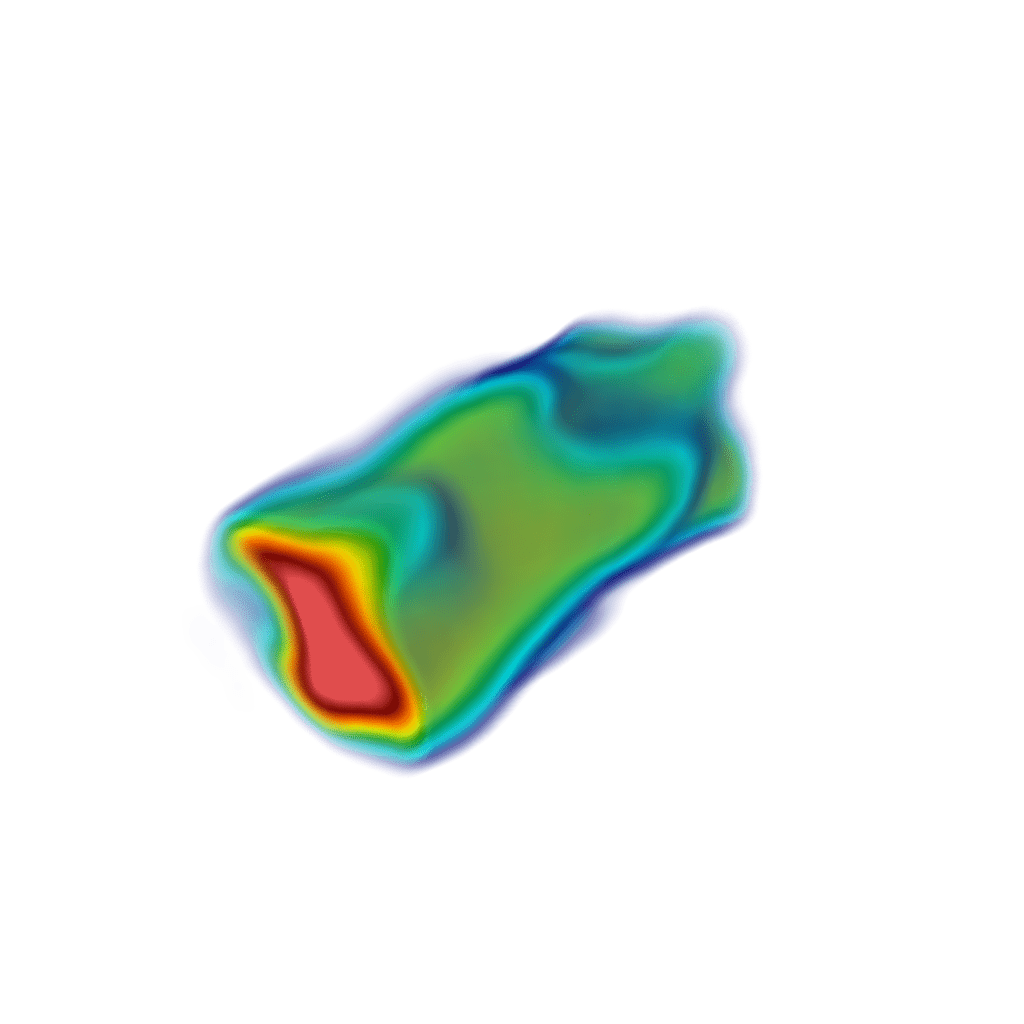}}
\endminipage
\endminipage

\minipage{\linewidth}
\minipage{0.25\textwidth}
\includegraphics[width=\linewidth, clip, trim=100 125 100 125]{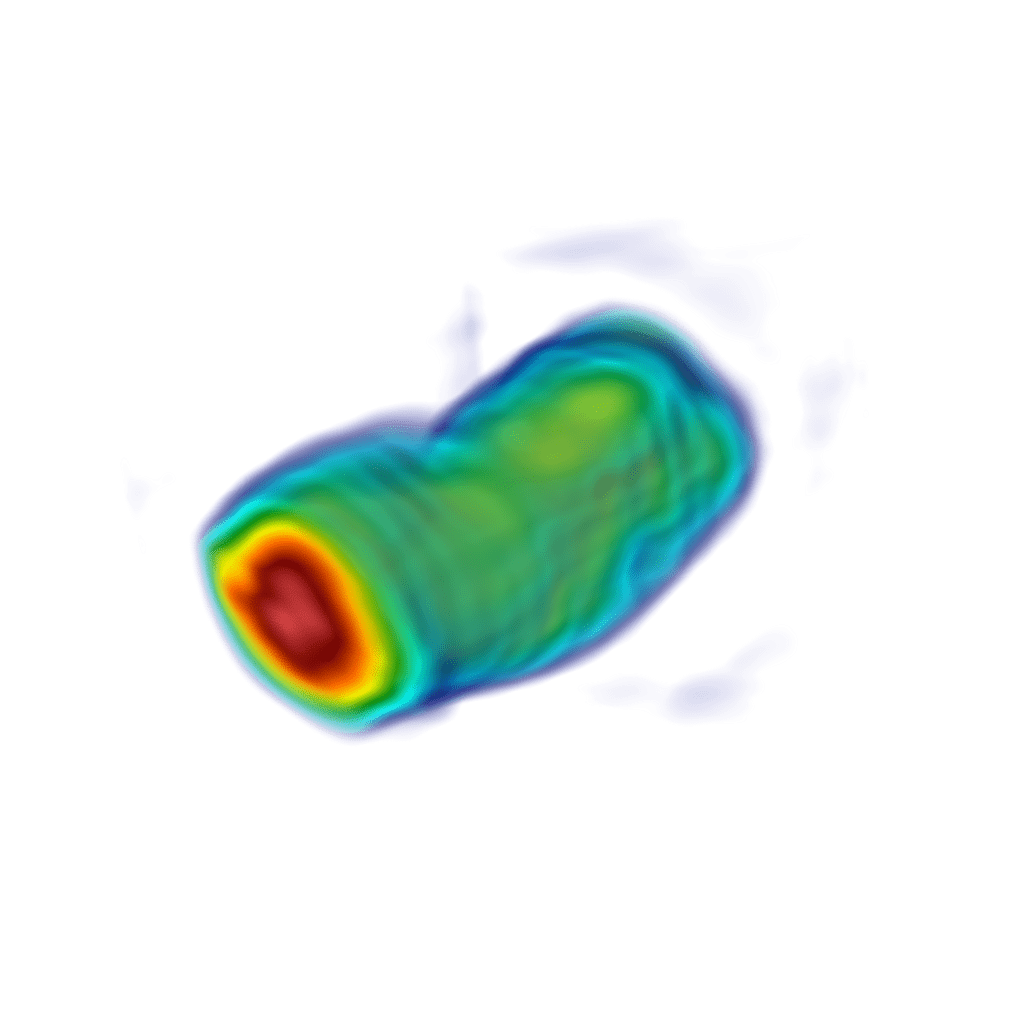}
\subcaption{$\bar{u}=\bar{u}^1$}
\endminipage
\minipage{0.25\textwidth}
{\includegraphics[width=\linewidth, clip, trim=100 125 100 125]{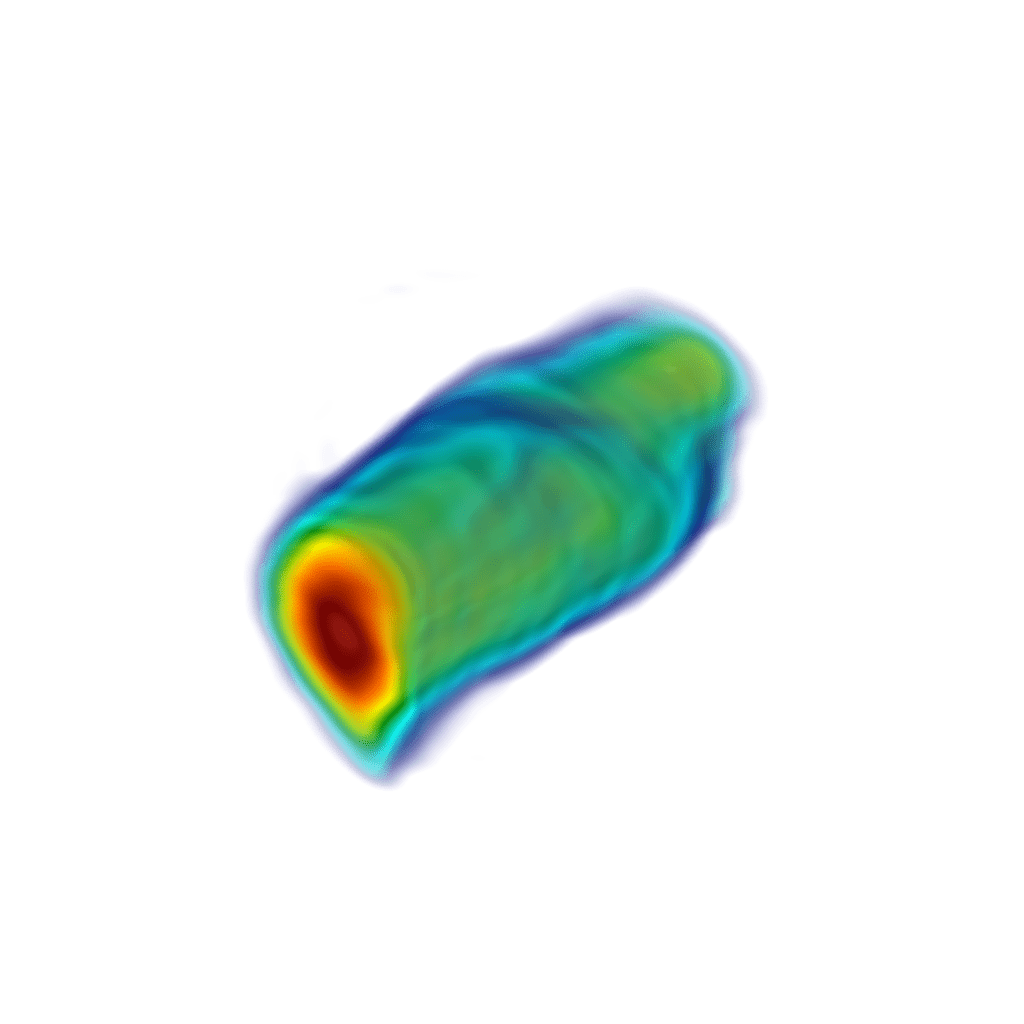}}
\subcaption{$\bar{u}=\bar{u}^2$}
\endminipage
\minipage{0.25\textwidth}
{\includegraphics[width=\linewidth, clip, trim=100 125 100 125]{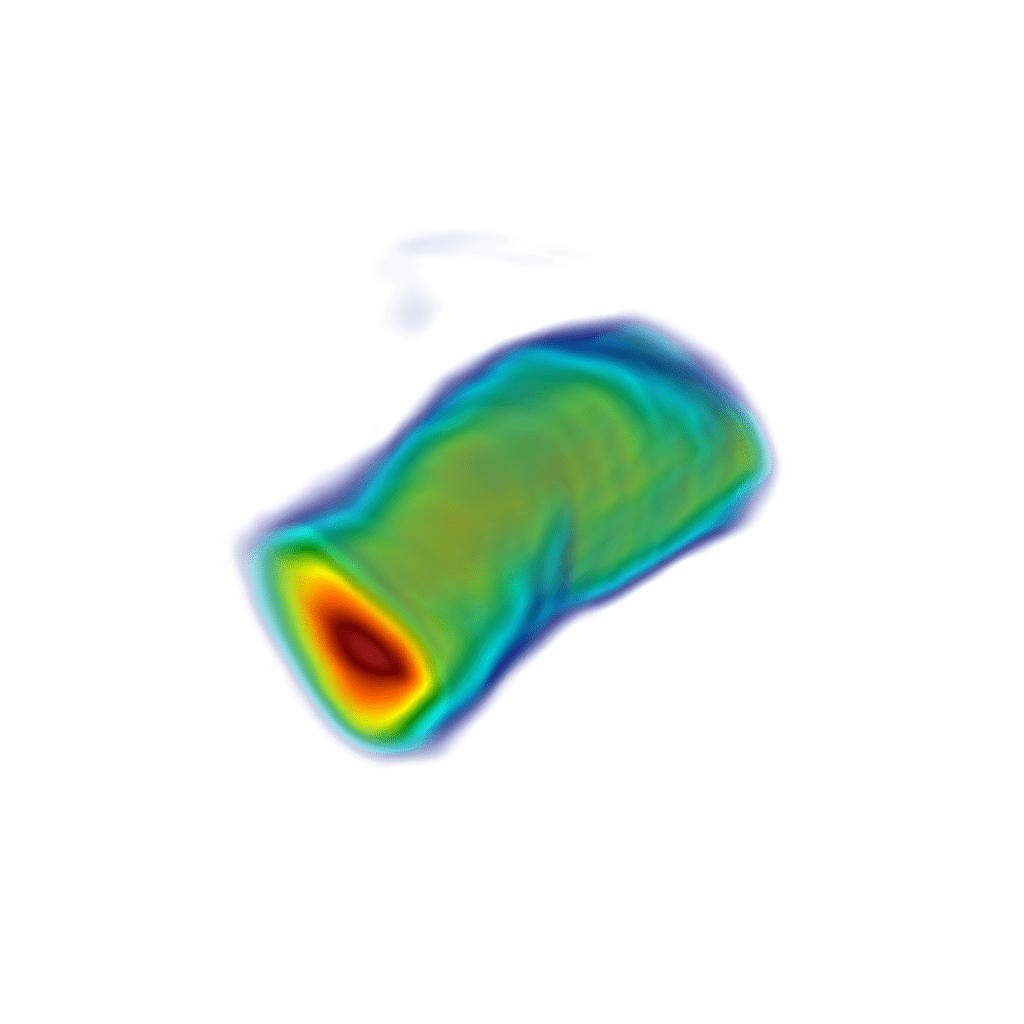}}
\subcaption{$\bar{u}=\bar{u}^3$}
\endminipage
\minipage{0.25\textwidth}
{\includegraphics[width=\linewidth, clip, trim=100 125 100 125]{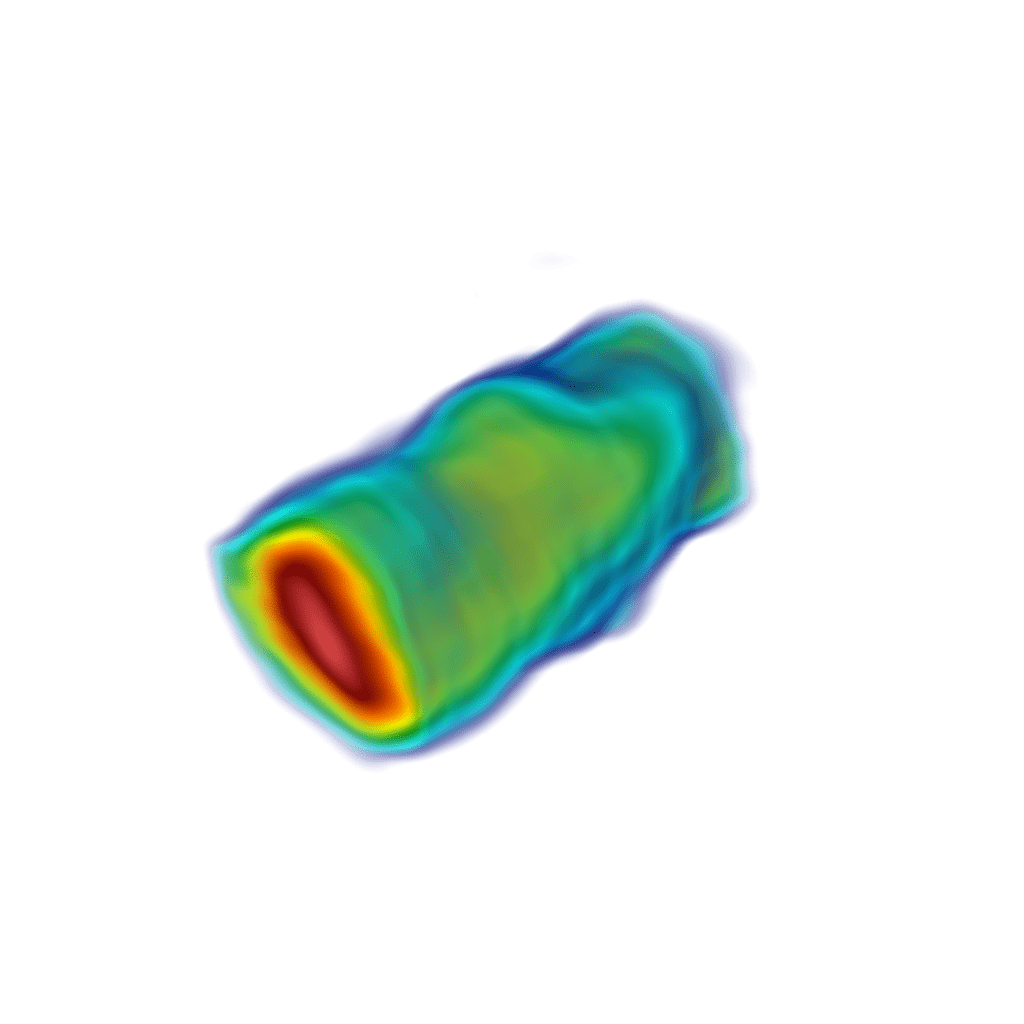}}
\subcaption{$\bar{u}=\bar{u}^4$}
\endminipage
\endminipage
\caption{\textbf{Visualization of the mean of the (pointwise) kinetic energy for the cylindrical shear flow experiment at time $T=1$, for four different initial distributions.} Data generated by the  ground truth (Top Row), GenCFD (Middle Row) and C-FNO (bottom Row). The colormap for all the figures ranges from $0.6$ (dark blue) to $1.7$ (dark red).}
\label{fig:s4}
\end{figure}

\begin{figure}[h!]
\minipage{\linewidth}
\minipage{0.25\textwidth}
\includegraphics[width=\linewidth, clip, trim=100 125 100 125]{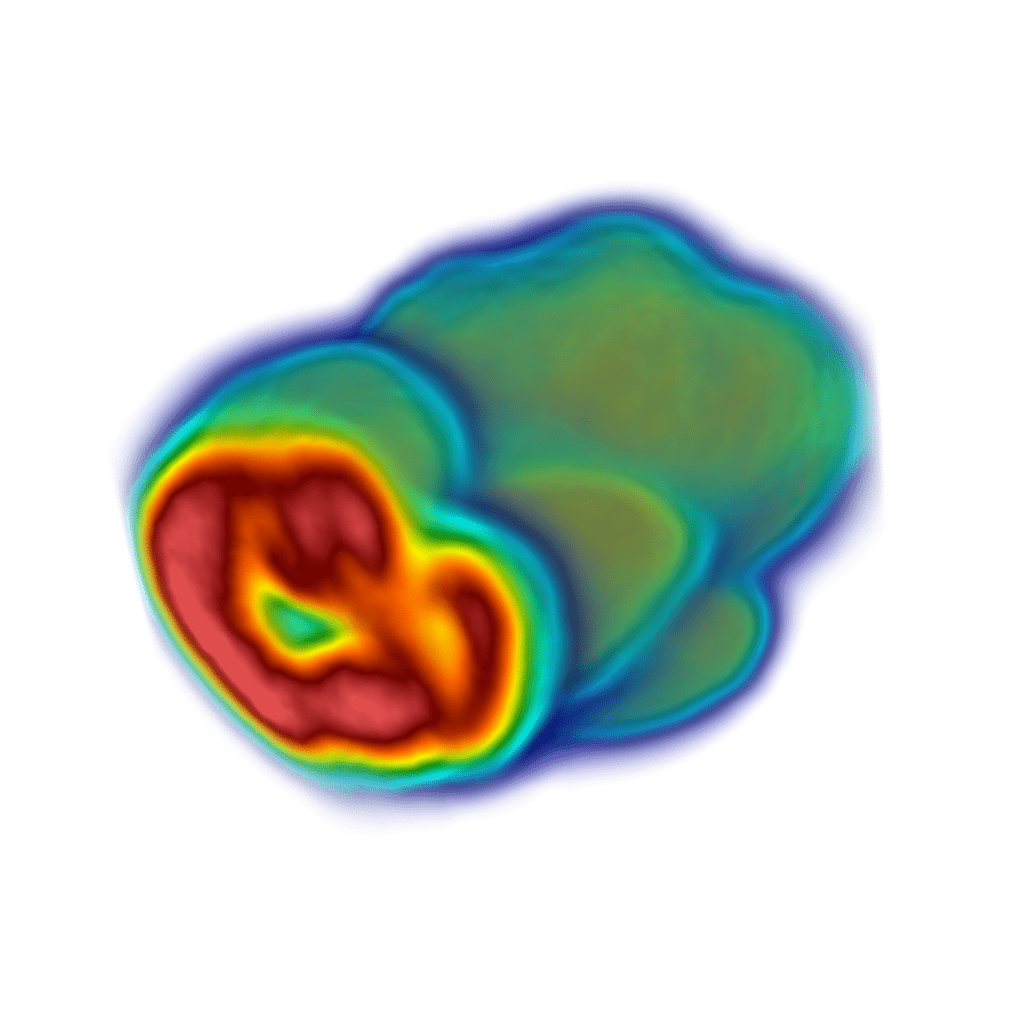}
\endminipage
\minipage{0.25\textwidth}
{\includegraphics[width=\linewidth, clip, trim=100 125 100 125]{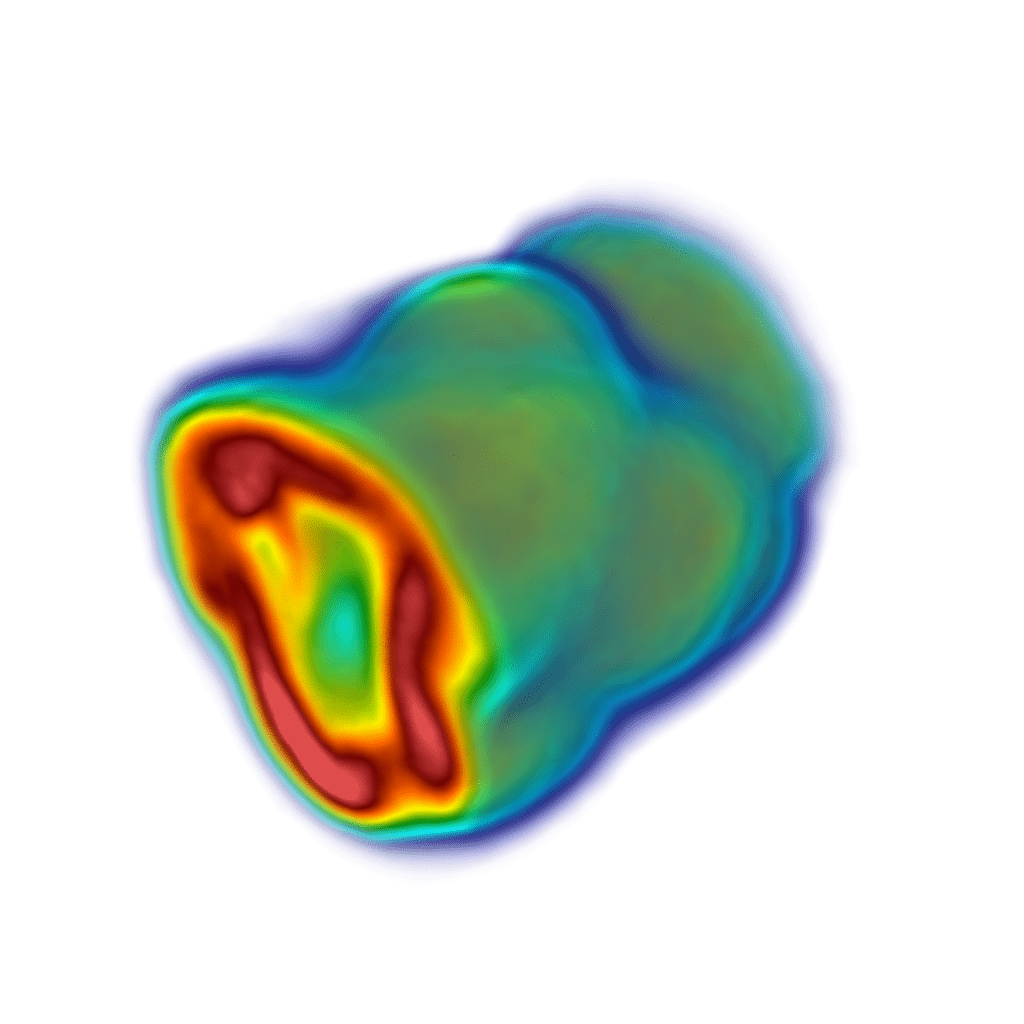}}
\endminipage
\minipage{0.25\textwidth}
{\includegraphics[width=\linewidth, clip, trim=100 125 100 125]{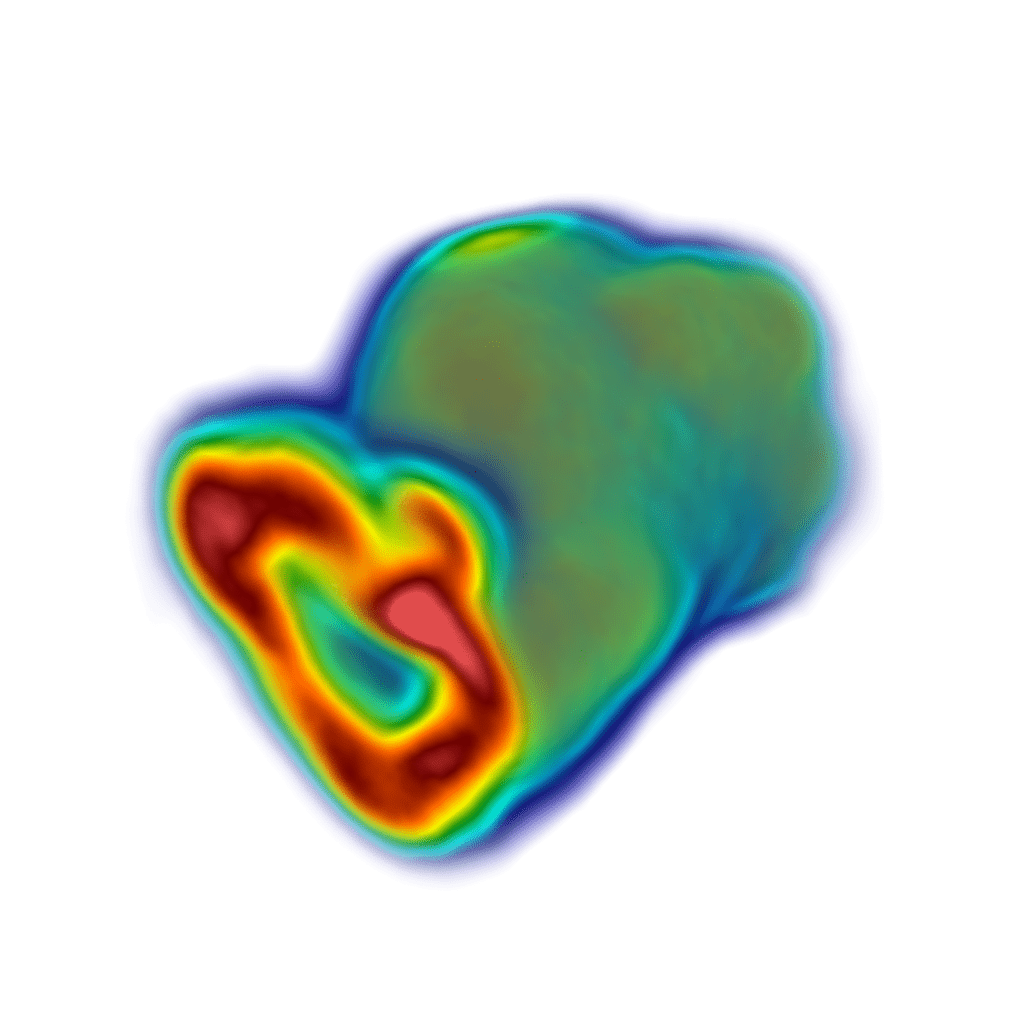}}
\endminipage
\minipage{0.25\textwidth}
{\includegraphics[width=\linewidth, clip, trim=100 125 100 125]{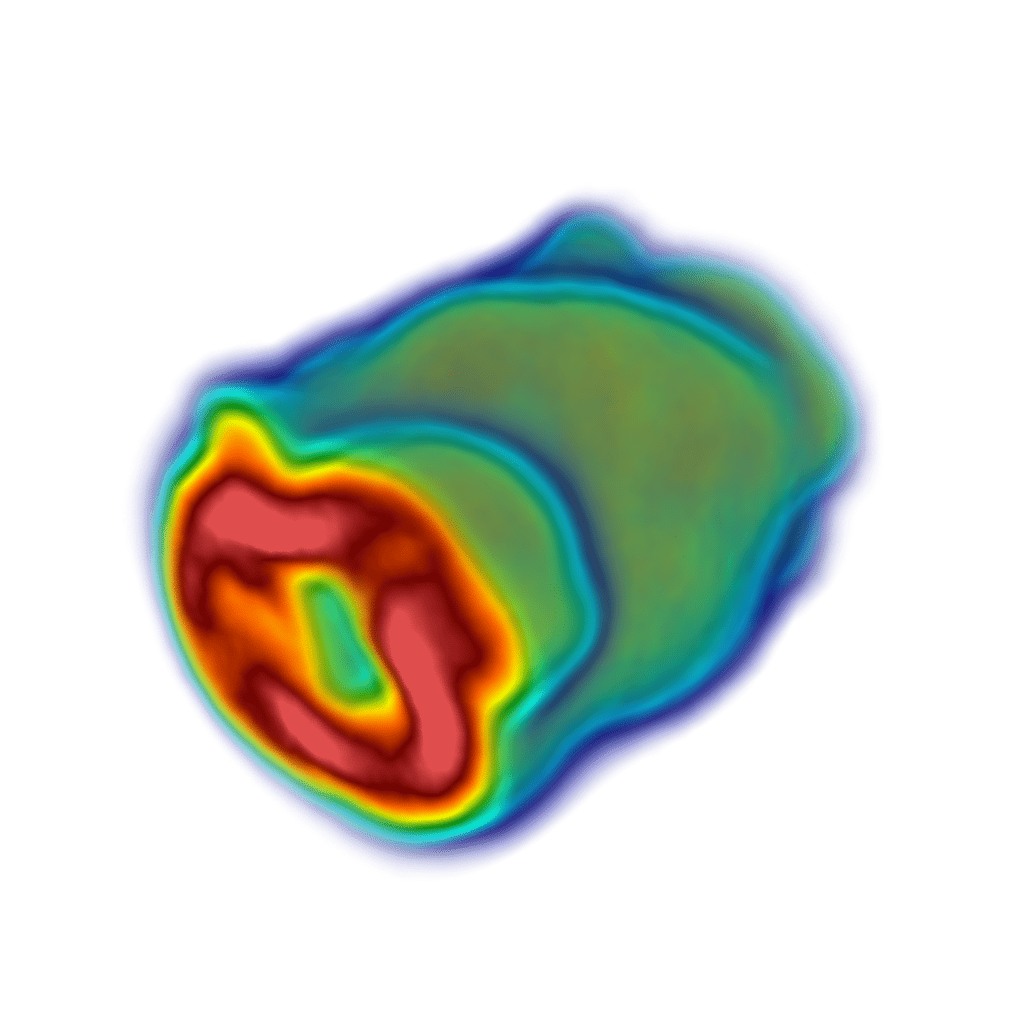}}
\endminipage
\endminipage

\minipage{\linewidth}
\minipage{0.25\textwidth}
\includegraphics[width=\linewidth, clip, trim=100 125 100 125]{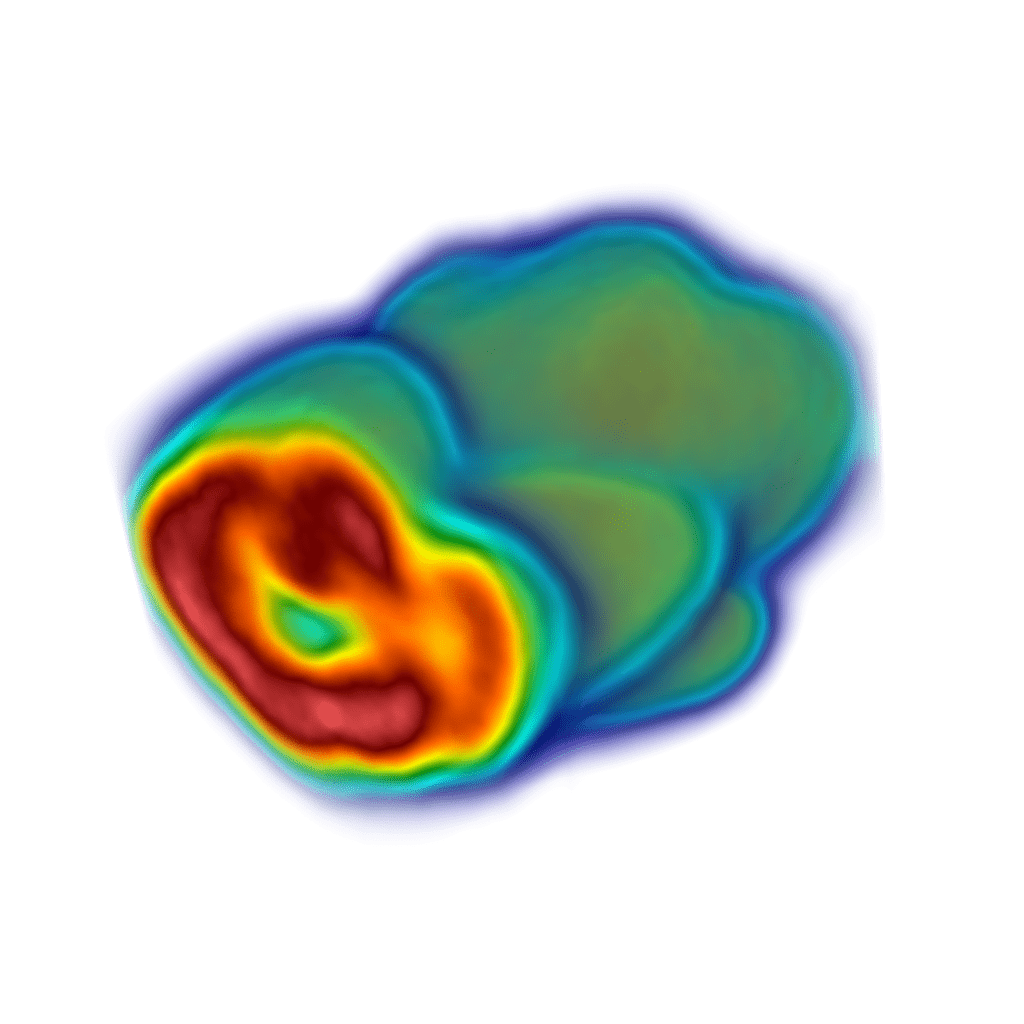}
\endminipage
\minipage{0.25\textwidth}
{\includegraphics[width=\linewidth, clip, trim=100 125 100 125]{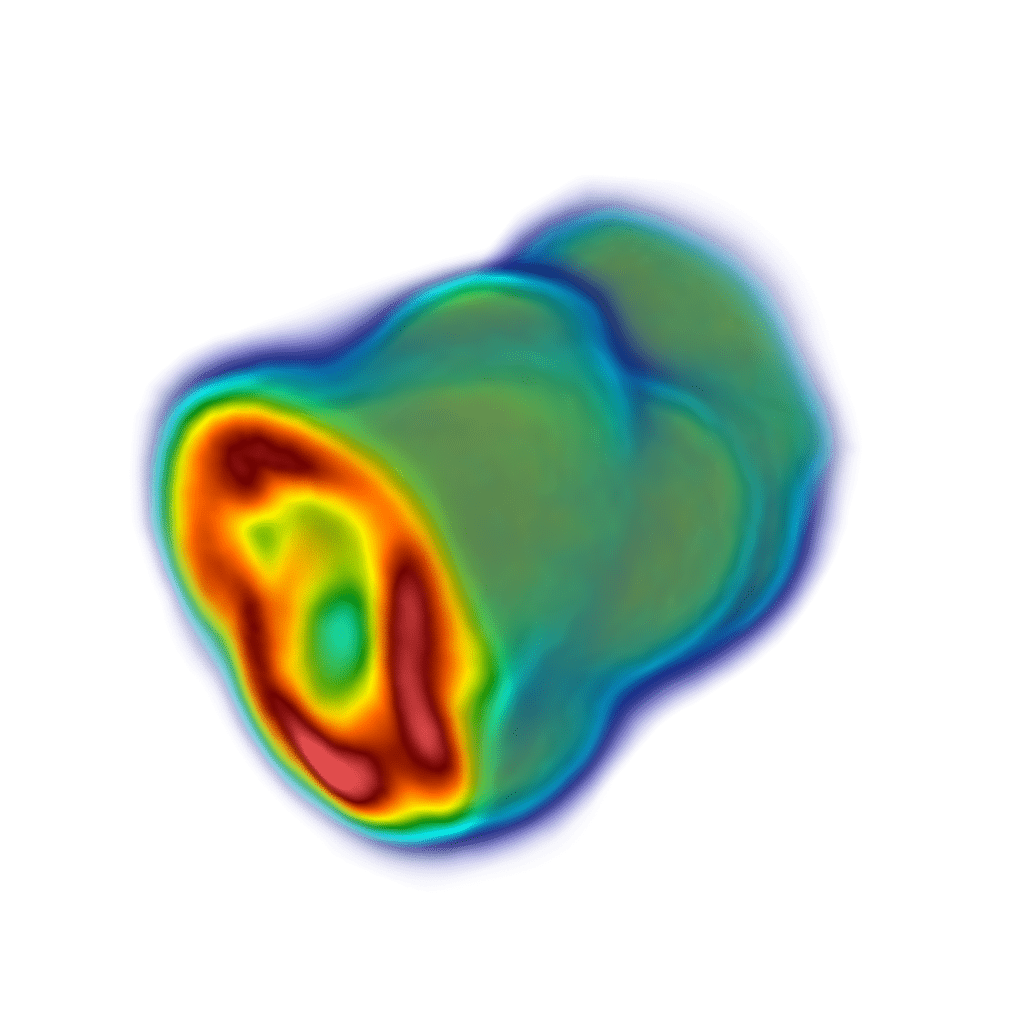}}
\endminipage
\minipage{0.25\textwidth}
{\includegraphics[width=\linewidth, clip, trim=100 125 100 125]{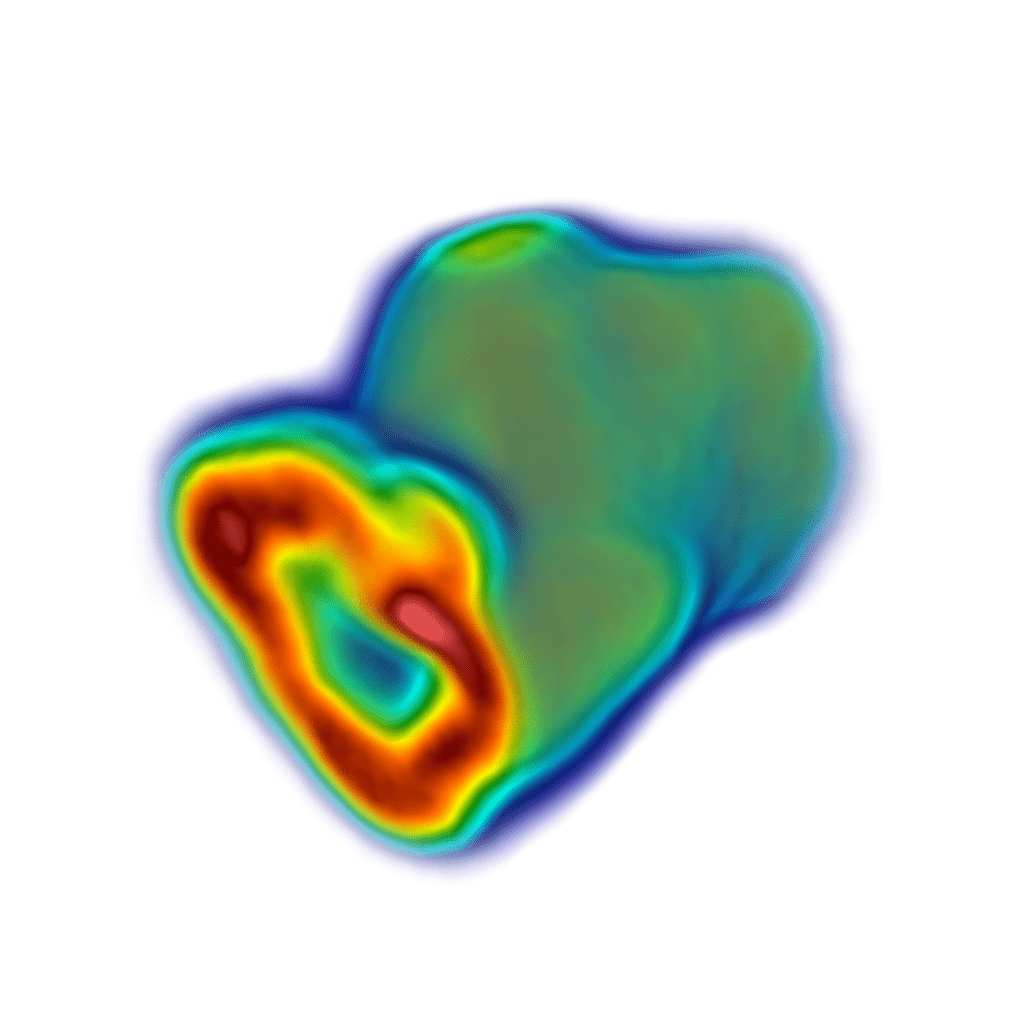}}
\endminipage
\minipage{0.25\textwidth}
{\includegraphics[width=\linewidth, clip, trim=100 125 100 125]{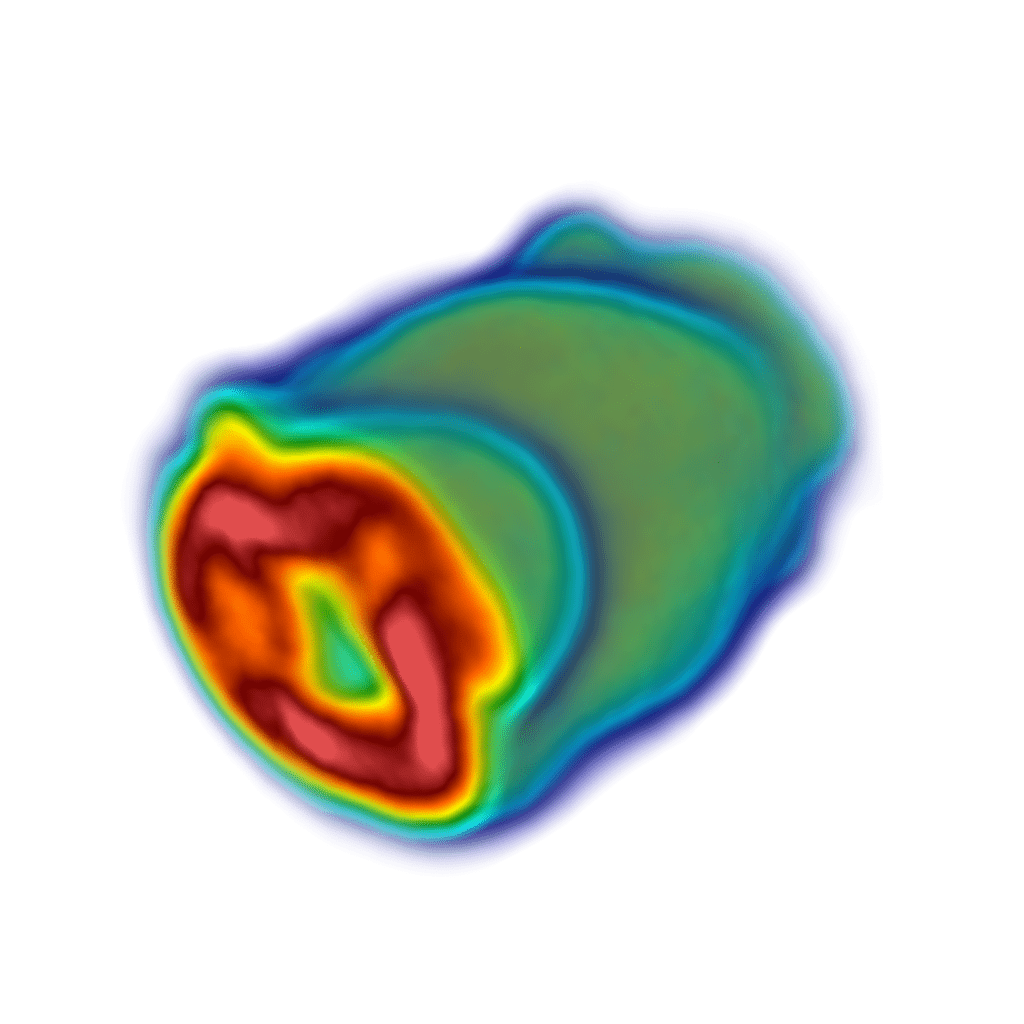}}
\endminipage
\endminipage

\minipage{\linewidth}
\minipage{0.25\textwidth}
\includegraphics[width=\linewidth, clip, trim=100 125 100 125]{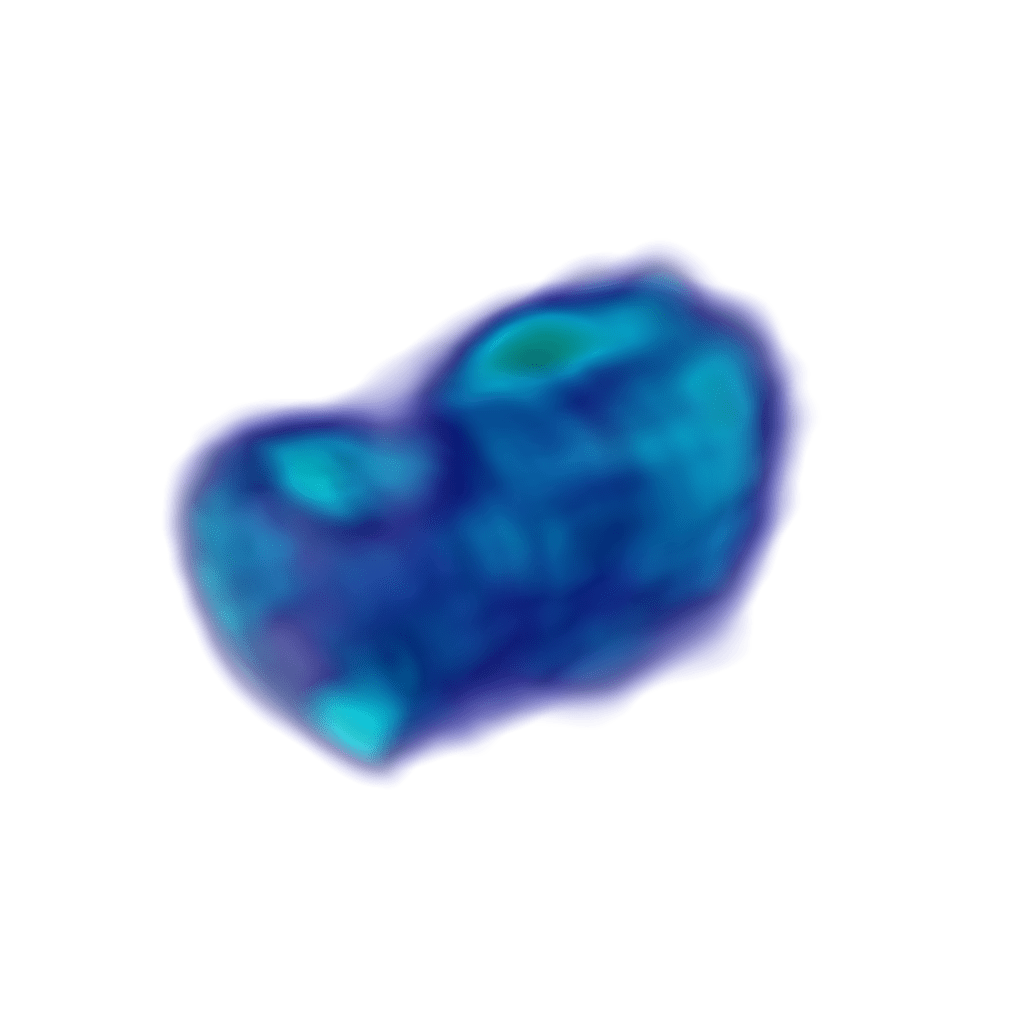}
\subcaption{$\bar{u}=\bar{u}^1$}
\endminipage
\minipage{0.25\textwidth}
{\includegraphics[width=\linewidth, clip, trim=100 125 100 125]{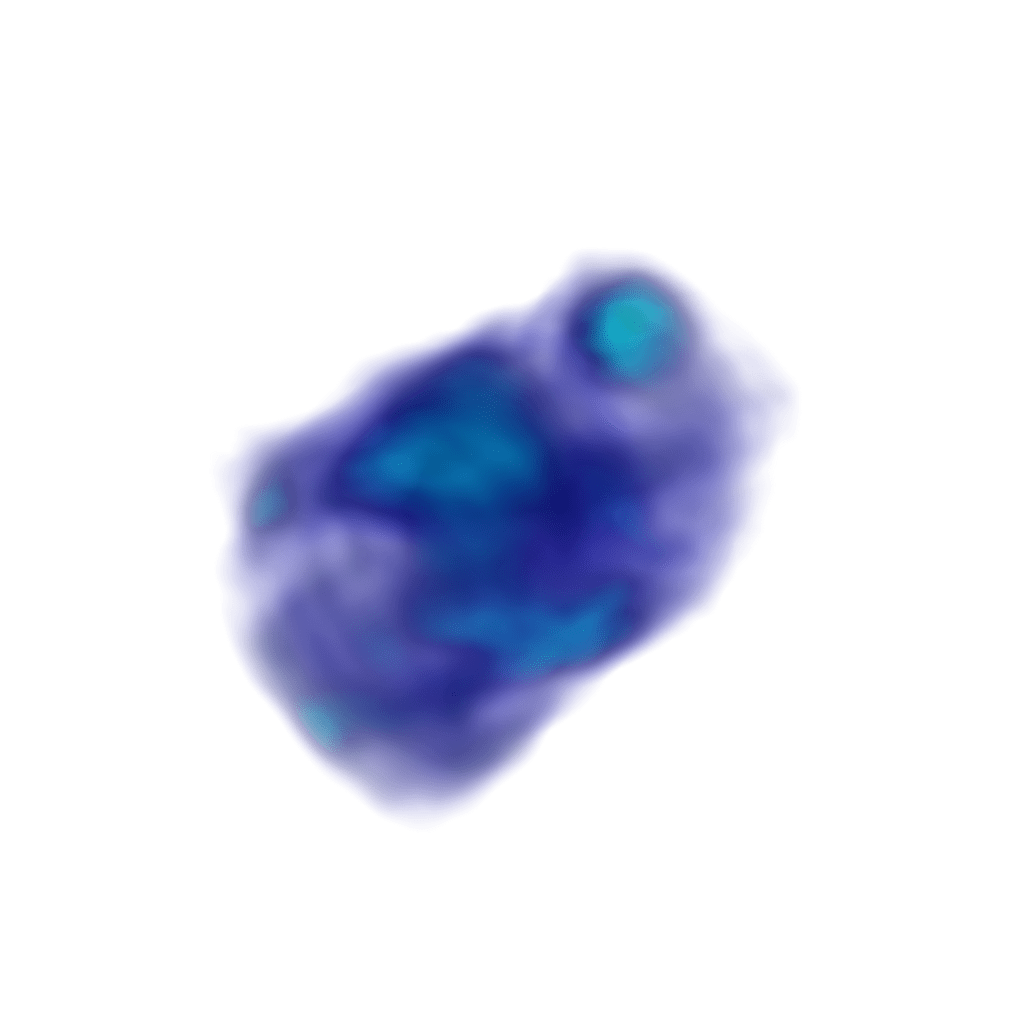}}
\subcaption{$\bar{u}=\bar{u}^2$}
\endminipage
\minipage{0.25\textwidth}
{\includegraphics[width=\linewidth, clip, trim=100 125 100 125]{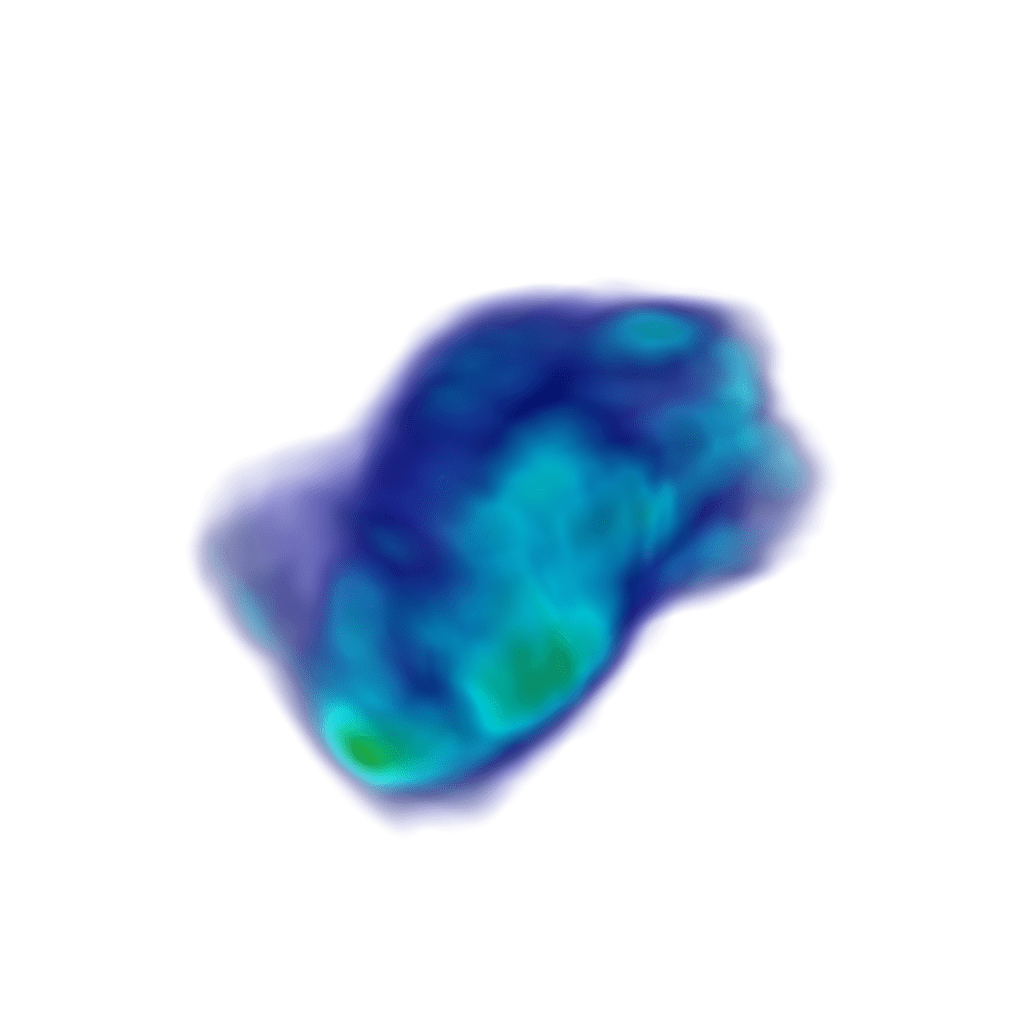}}
\subcaption{$\bar{u}=\bar{u}^3$}
\endminipage
\minipage{0.25\textwidth}
{\includegraphics[width=\linewidth, clip, trim=100 125 100 125]{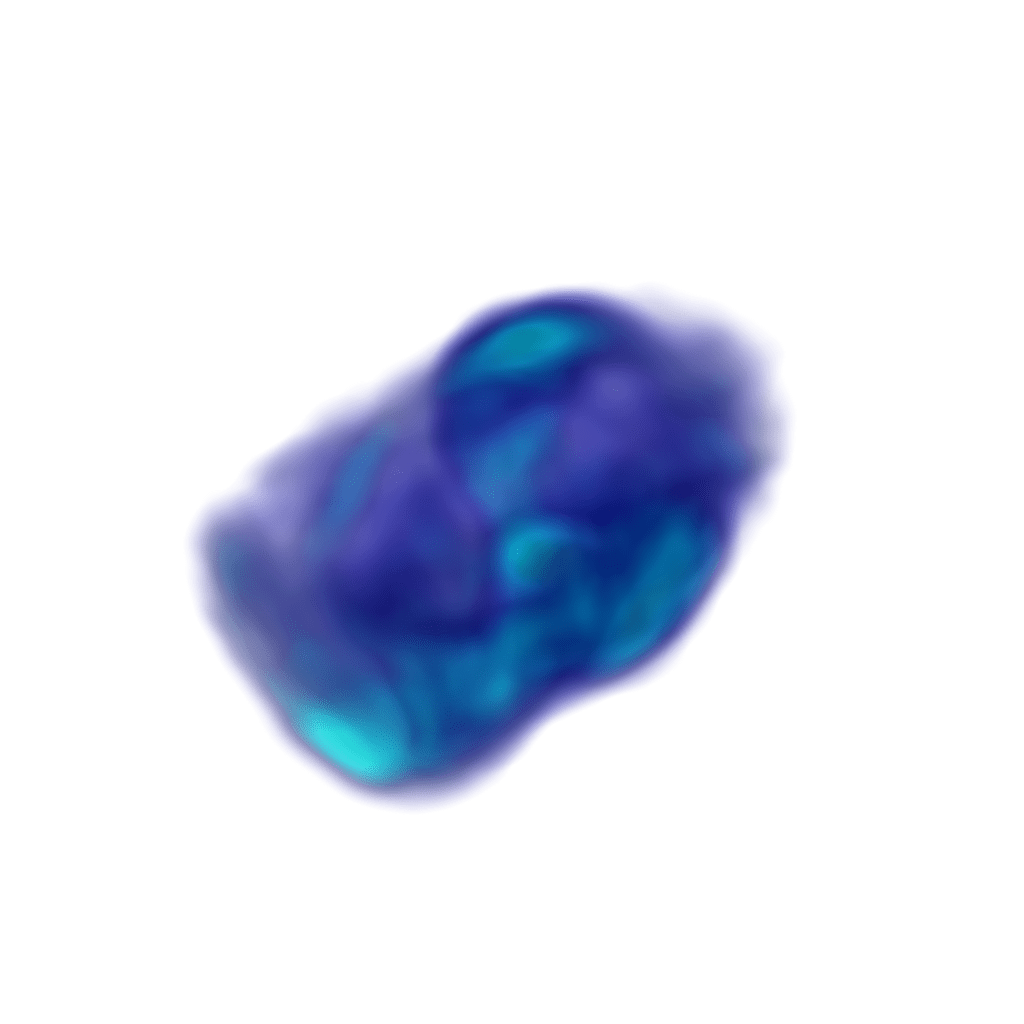}}
\subcaption{$\bar{u}=\bar{u}^4$}
\endminipage
\endminipage
\caption{\textbf{Visualization of the standard deviation of the kinetic energy for the cylindrical shear flow experiment at time $T=1$, for four different initial distributions.} Data generated by the  ground truth (Top Row), GenCFD (Middle Row) and C-FNO (Bottom Row). The colormap for the top and middle rows ranges from $0.05$ (dark blue) to $0.65$ (dark red), whereas for the bottom row, it ranges from $0.05$ to $0.25$.}
\label{fig:s5}
\end{figure}

\begin{figure}[h!]
\minipage{\linewidth}
\minipage{0.25\textwidth}
\includegraphics[width=\linewidth]{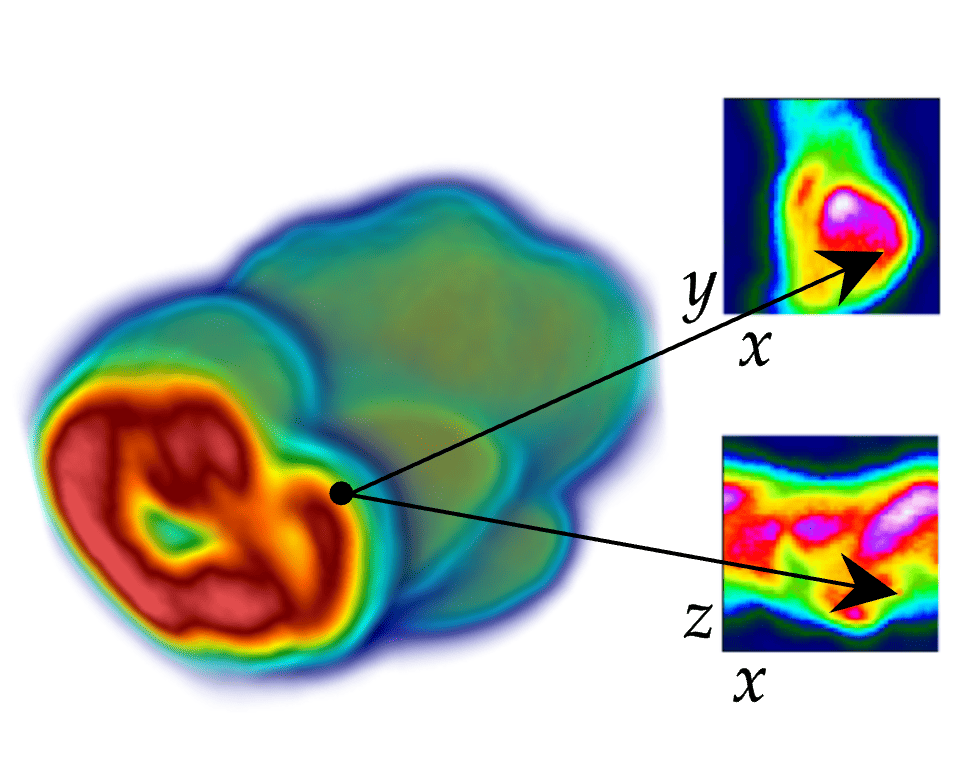}
\endminipage
\minipage{0.25\textwidth}
{\includegraphics[width=\linewidth]{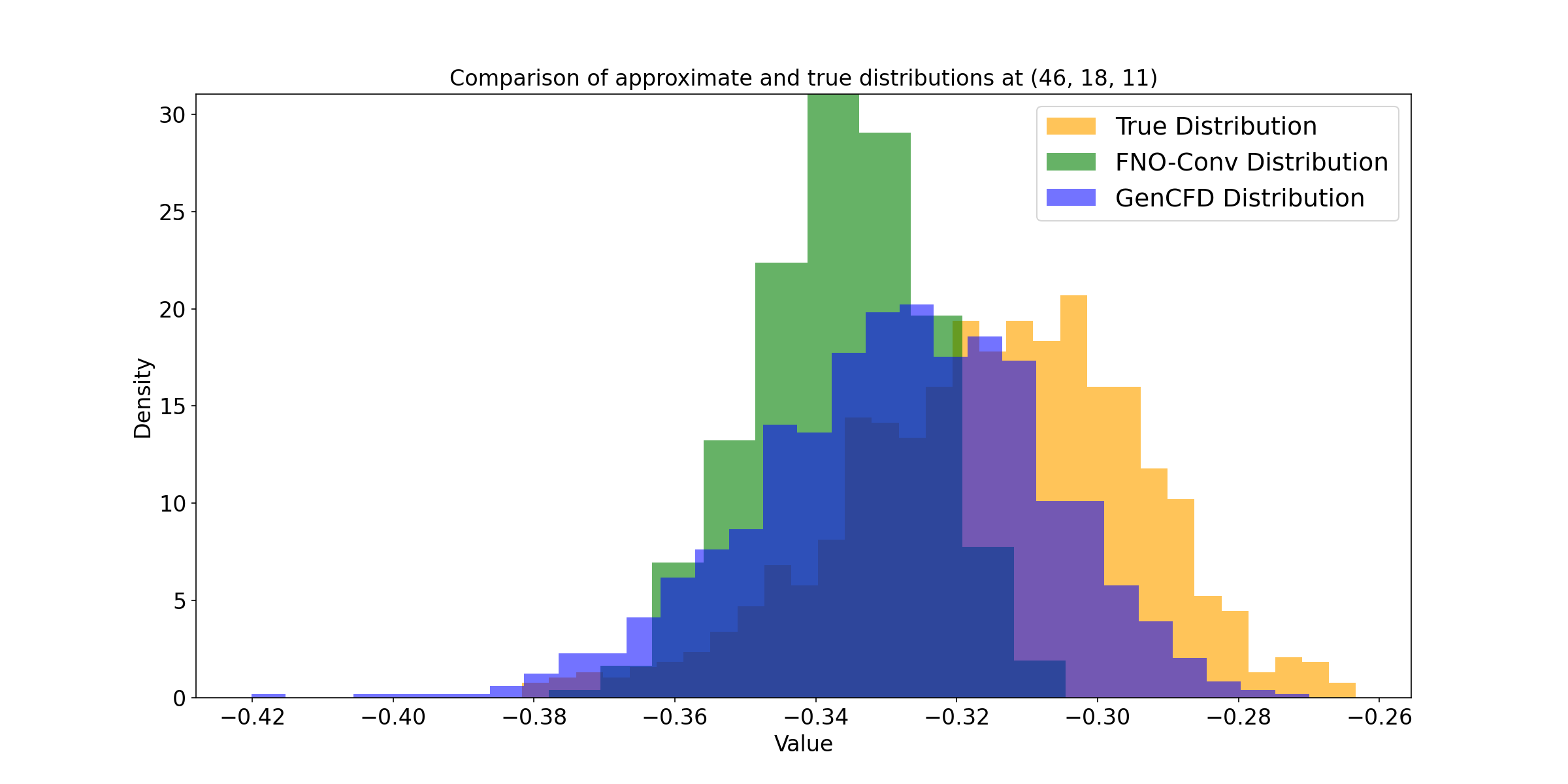}}
\endminipage
\minipage{0.25\textwidth}
{\includegraphics[width=\linewidth]{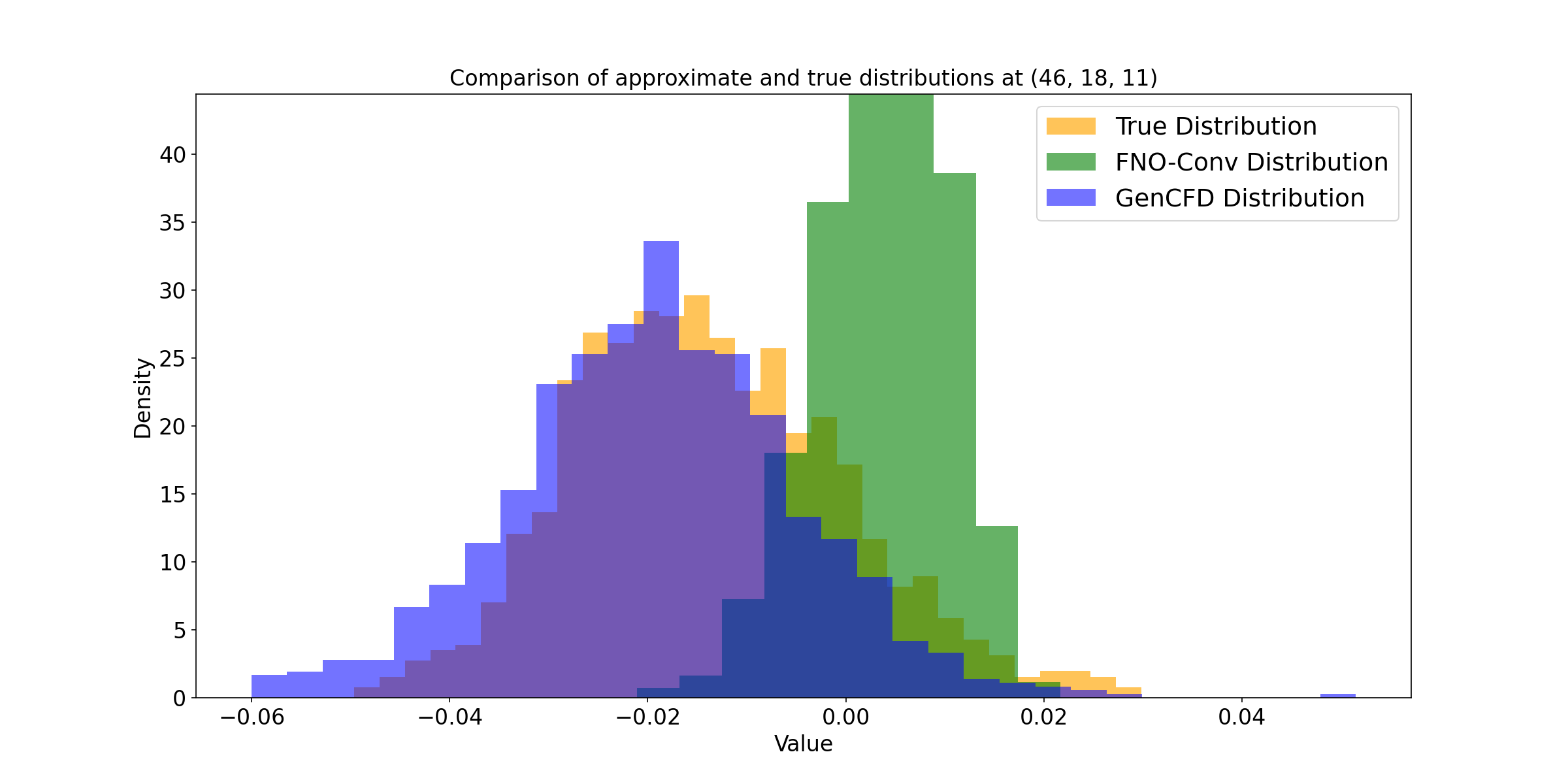}}
\endminipage
\minipage{0.25\textwidth}
{\includegraphics[width=\linewidth]{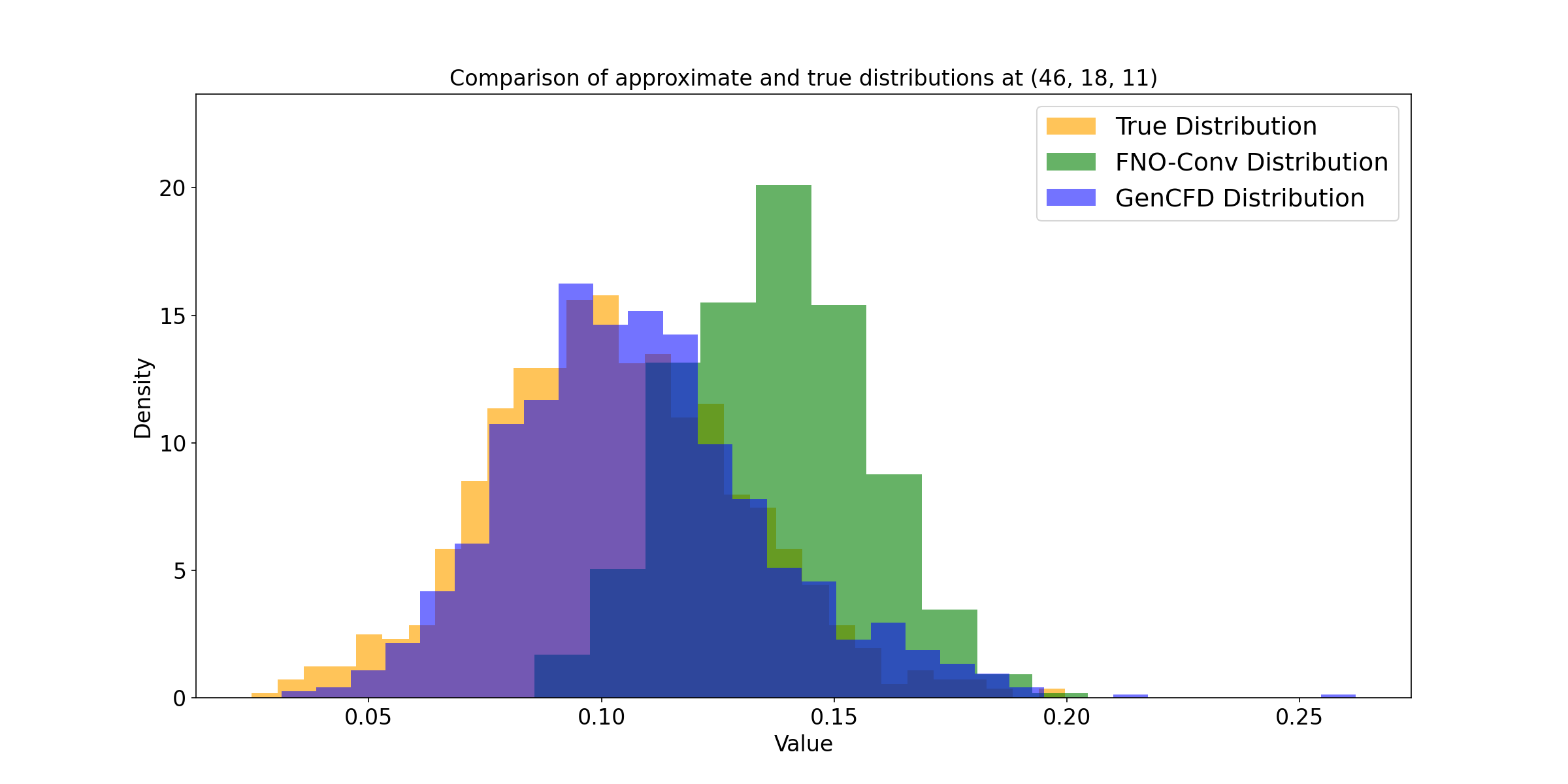}}
\endminipage
\endminipage

\minipage{\linewidth}
\minipage{0.25\textwidth}
\includegraphics[width=\linewidth]{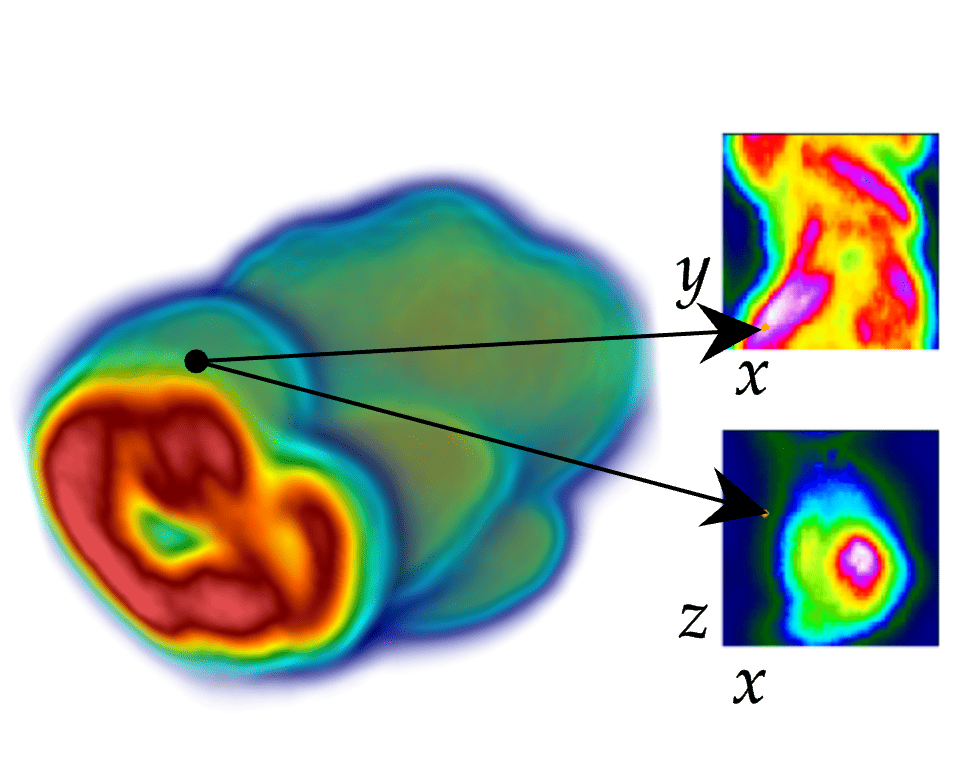}
\subcaption{Points}
\endminipage
\minipage{0.25\textwidth}
{\includegraphics[width=\linewidth]{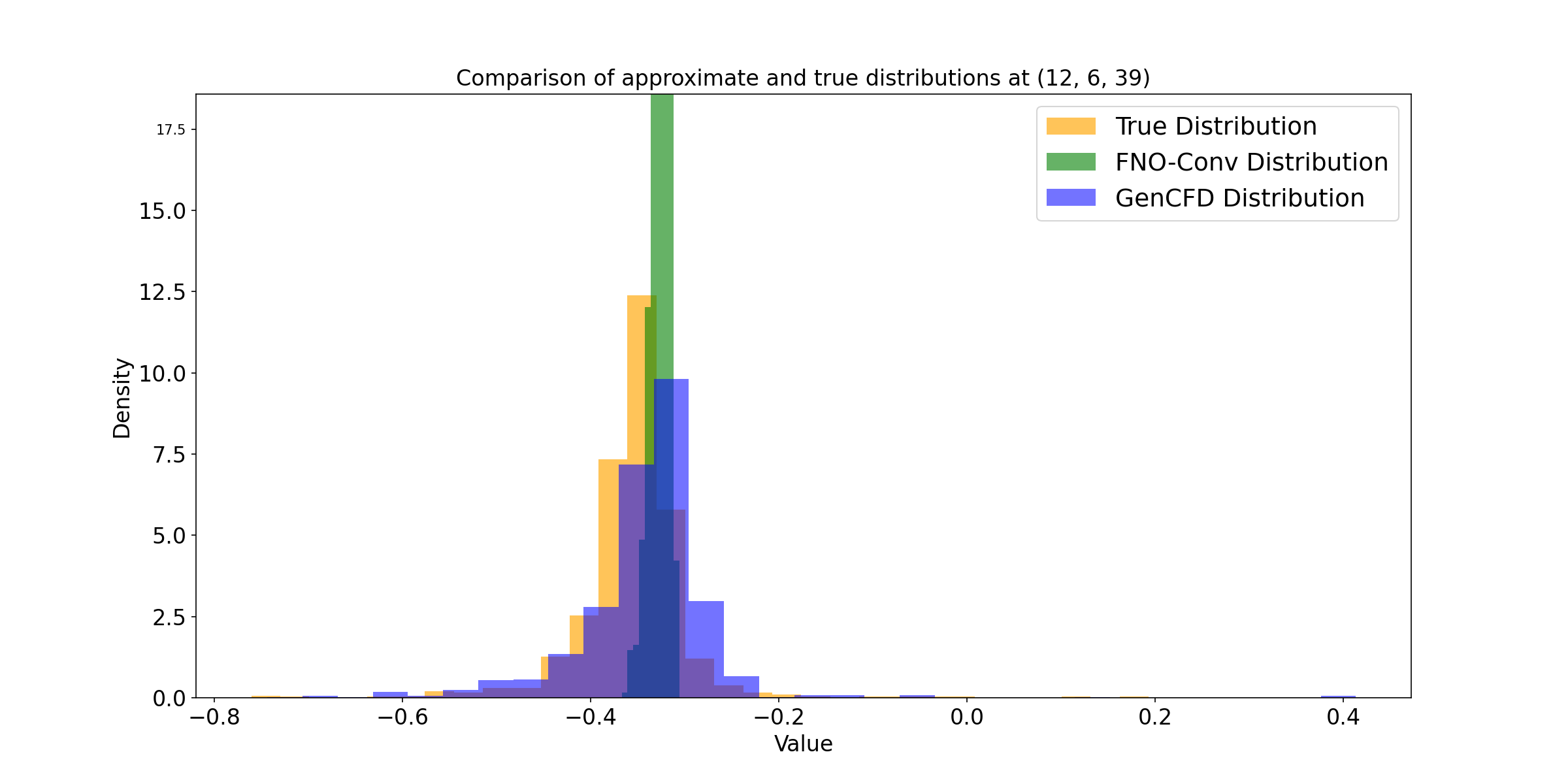}}
\subcaption{$u_x$}
\endminipage
\minipage{0.25\textwidth}
{\includegraphics[width=\linewidth]{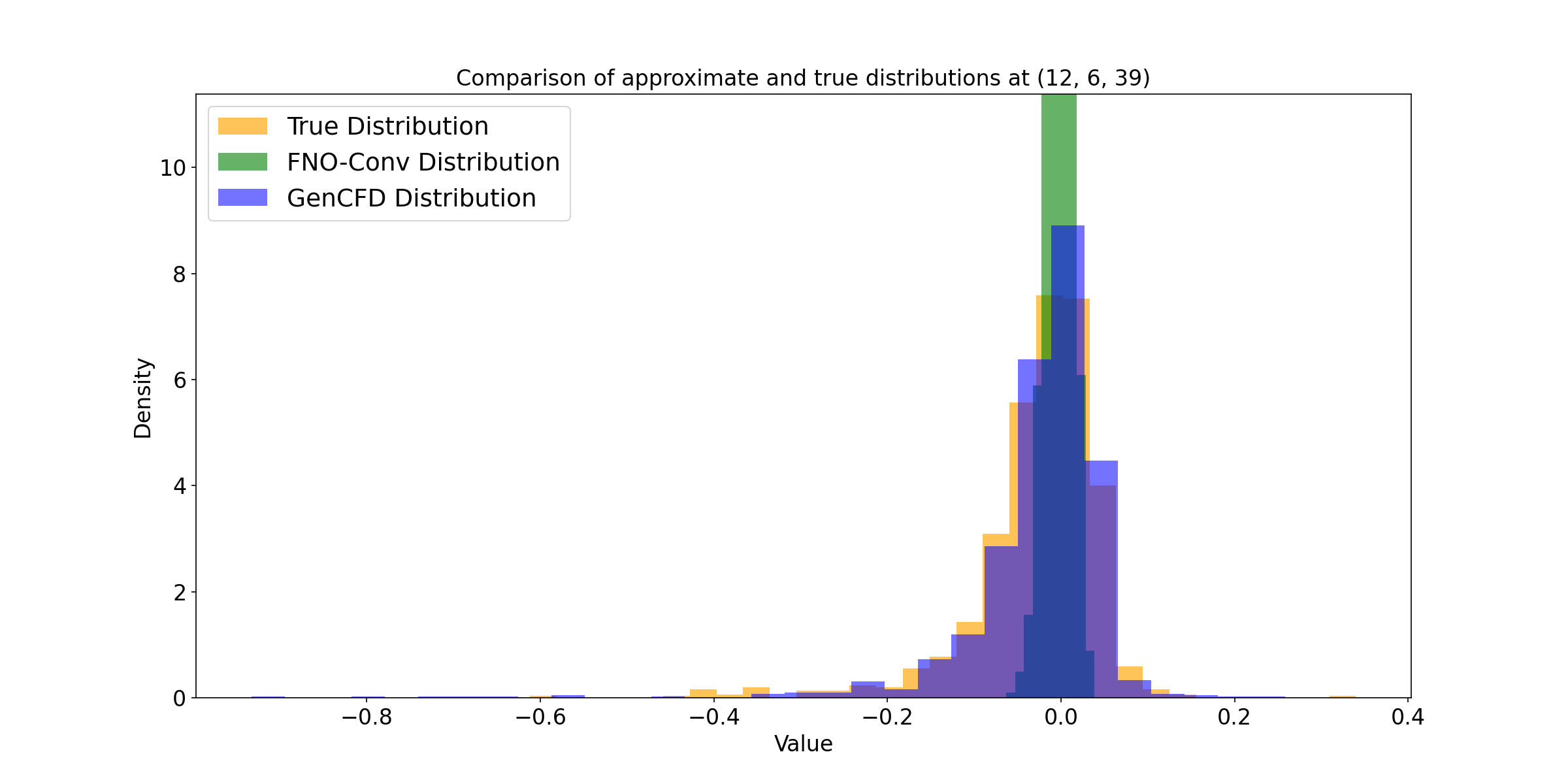}}
\subcaption{$u_y$}
\endminipage
\minipage{0.25\textwidth}
{\includegraphics[width=\linewidth]{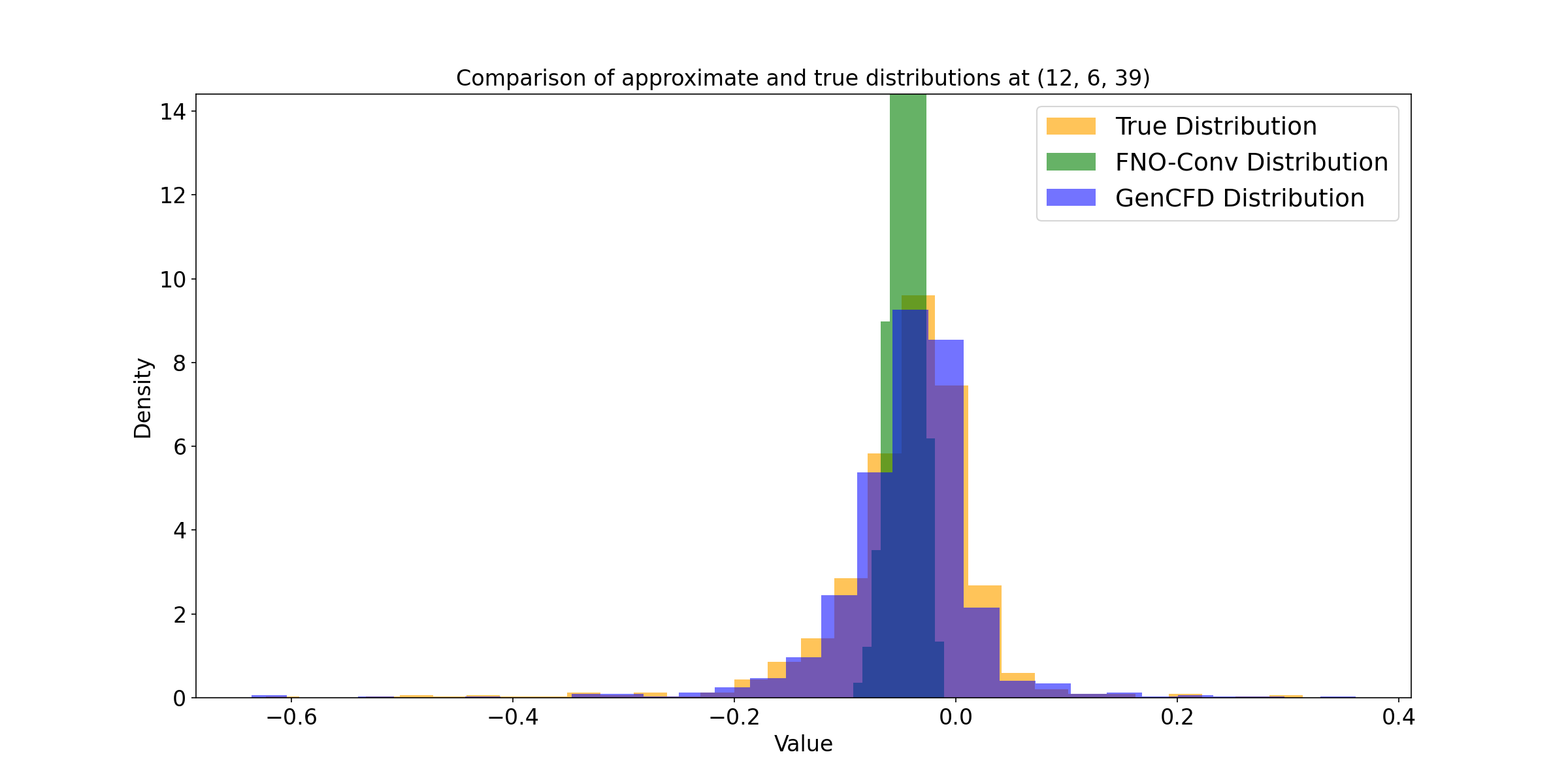}}
\subcaption{$u_z$}
\endminipage
\endminipage
\caption{\textbf{Visualization of the point PDFs, at two different points (marked in the left column), of all the velocity components for the three-dimensional cylindrical shear flow experiment at time $T=1$.} Results generated by the ground truth, GenCFD and C-FNO.}
\label{fig:s6}
\end{figure}

\begin{figure}[h!]
\minipage{\linewidth}
\minipage{0.24\textwidth}
\includegraphics[width=\linewidth, clip,trim={50 50 50 50}, draft=false]{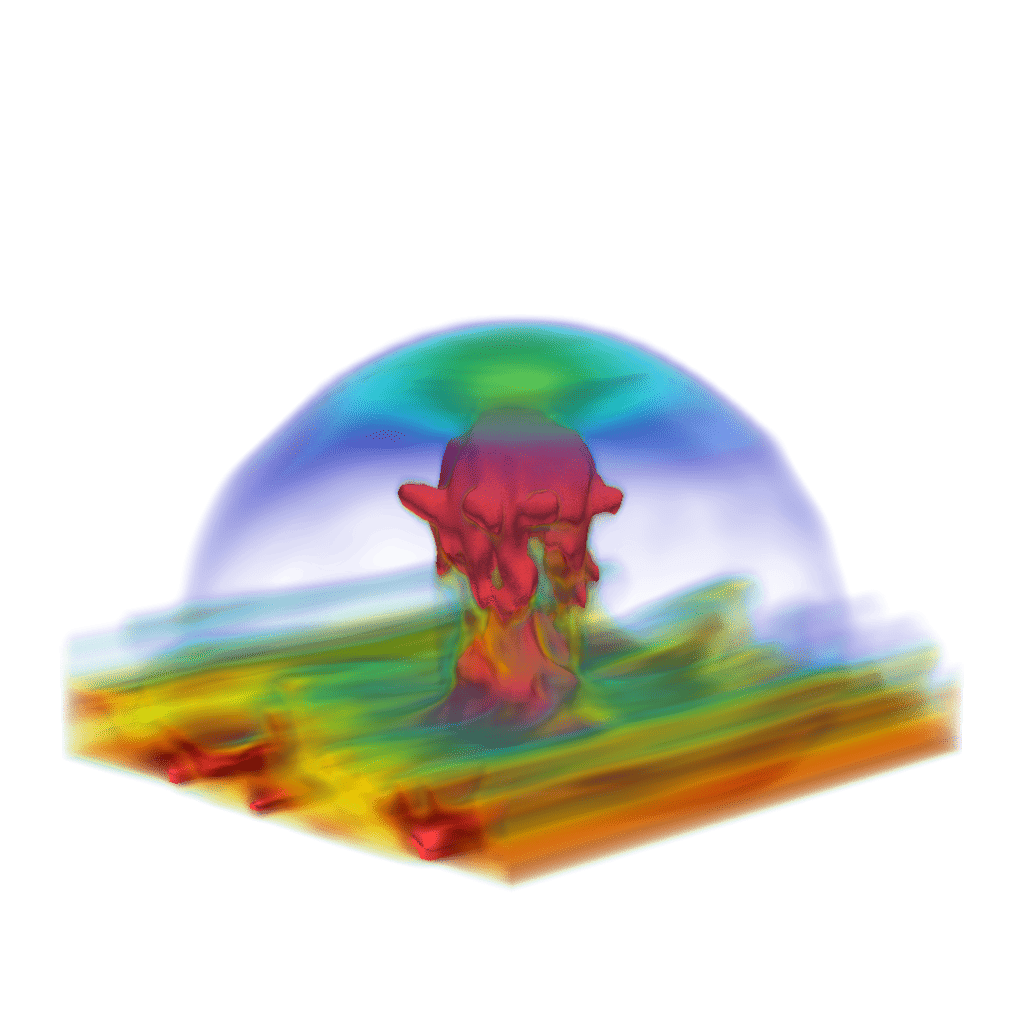}
\endminipage
\minipage{0.24\textwidth}
{\includegraphics[width=\linewidth, clip,trim={125 125 125 125}, draft=false]{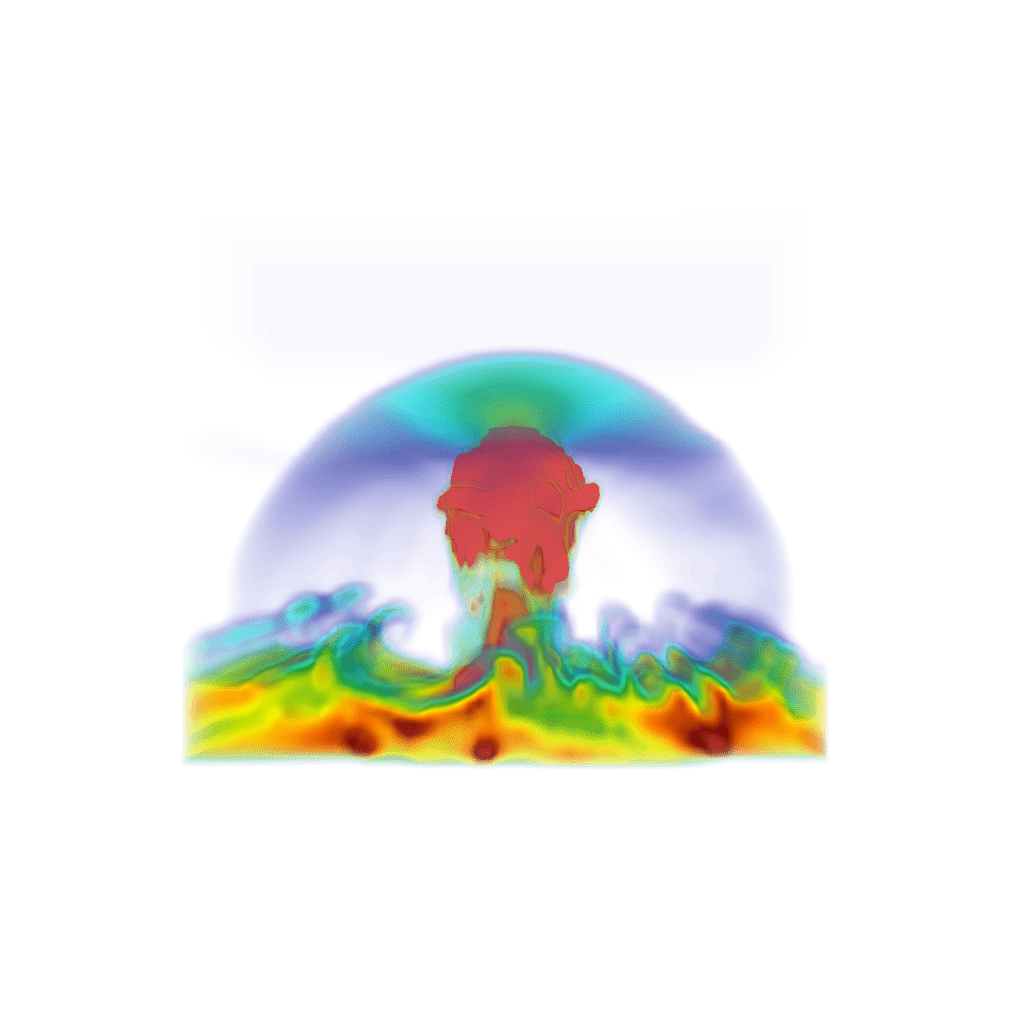}}
\endminipage
\rulesep
\minipage{0.24\textwidth}
{\includegraphics[width=\linewidth, trim={50 50 50 50}, clip, draft=false]{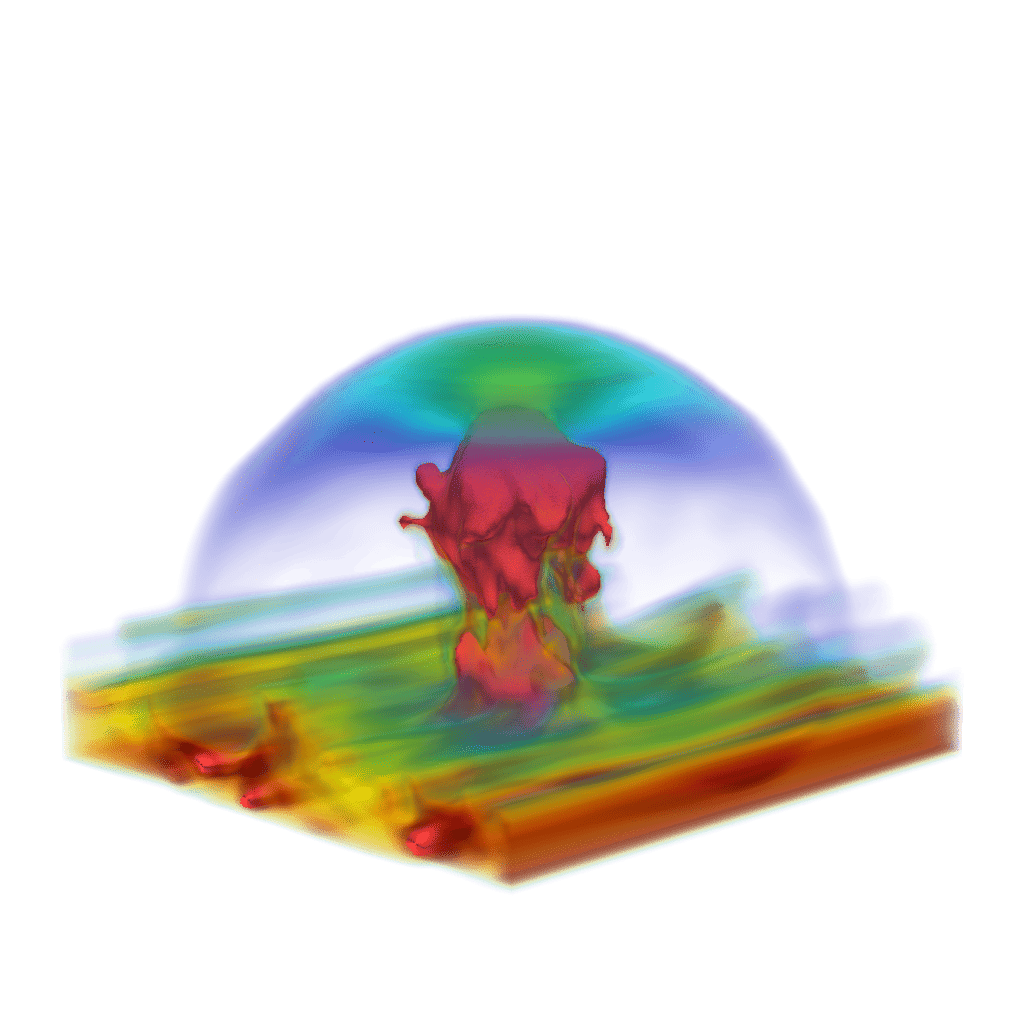}}
\endminipage
\minipage{0.24\textwidth}
{\includegraphics[width=\linewidth, trim={125 125 125 125}, clip, draft=false]{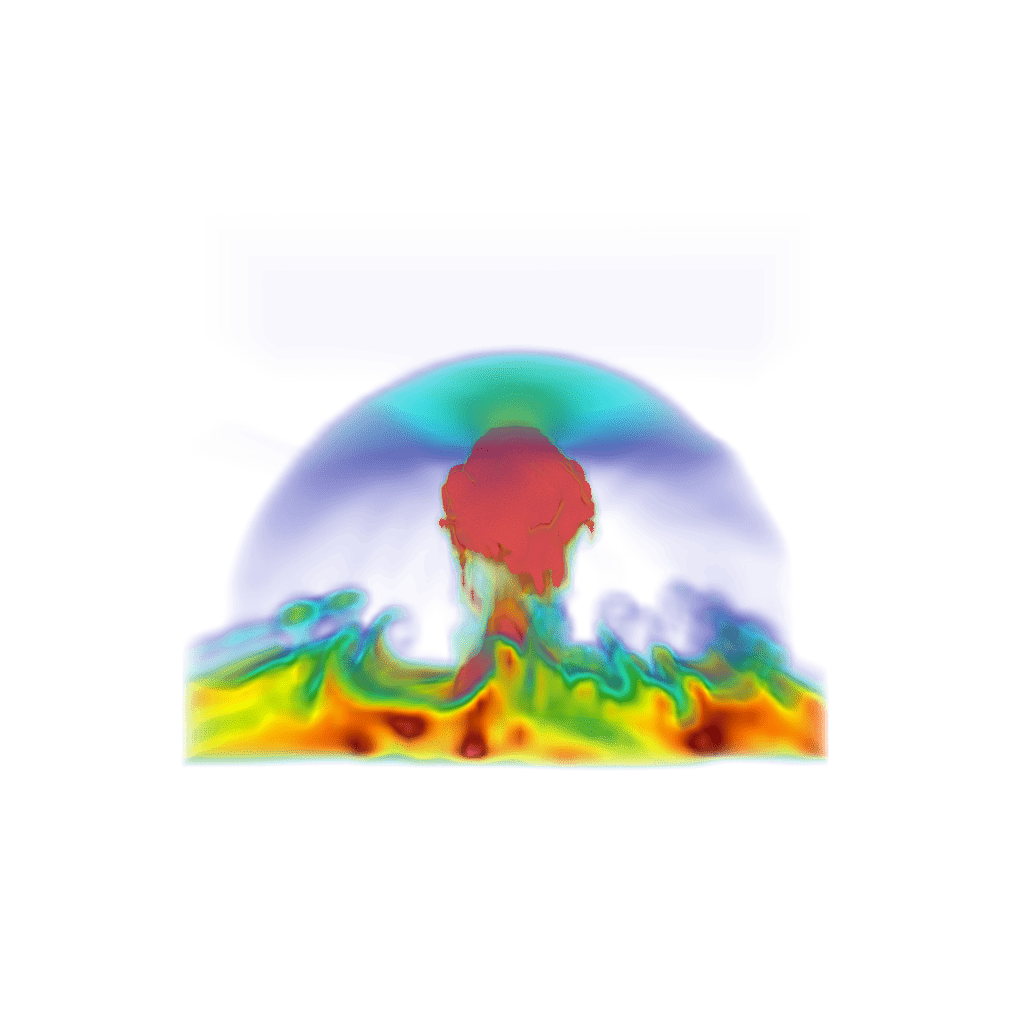}}
\endminipage
\hrulesep
\endminipage

\minipage{\linewidth}
\minipage{0.24\textwidth}
\includegraphics[width=\linewidth, clip,trim={50 50 50 50}, draft=false]{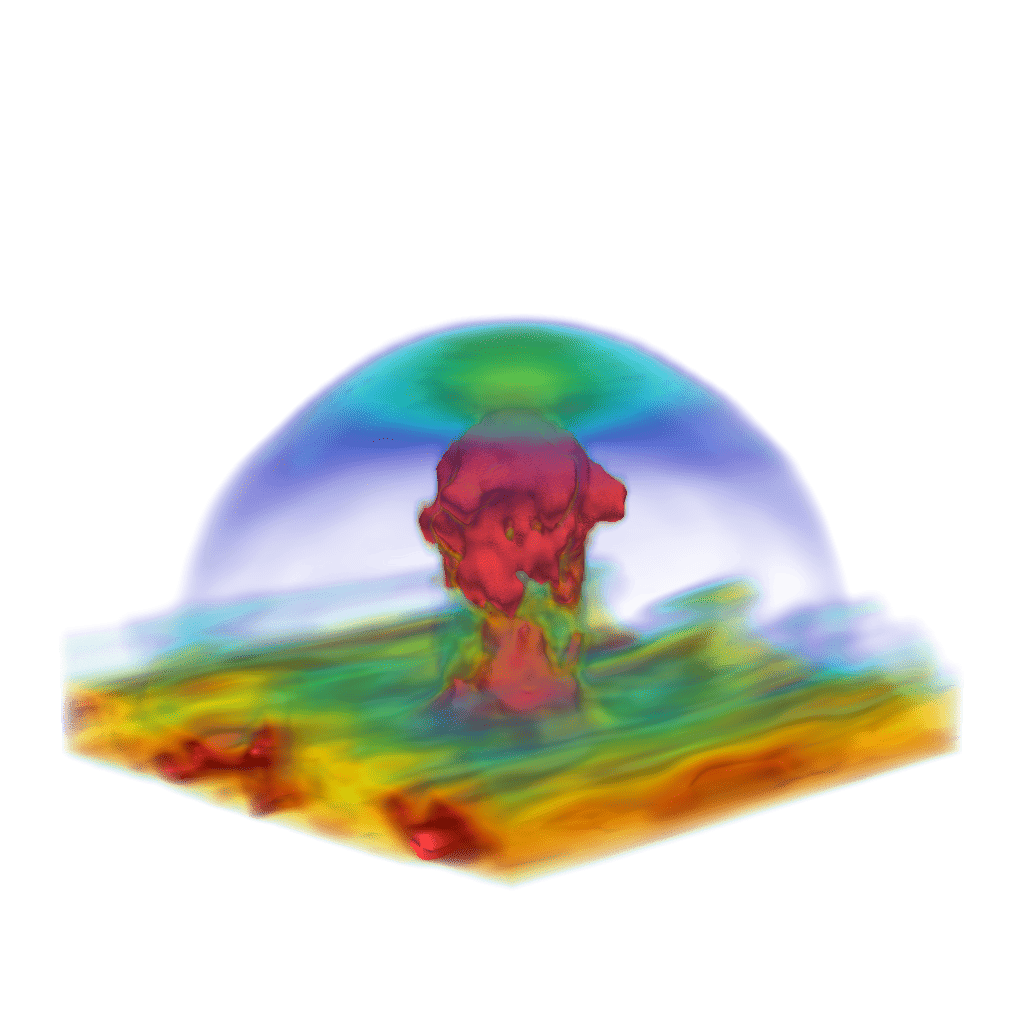}
\endminipage
\minipage{0.24\textwidth}
{\includegraphics[width=\linewidth, trim={125 125 125 125}, clip, draft=false]{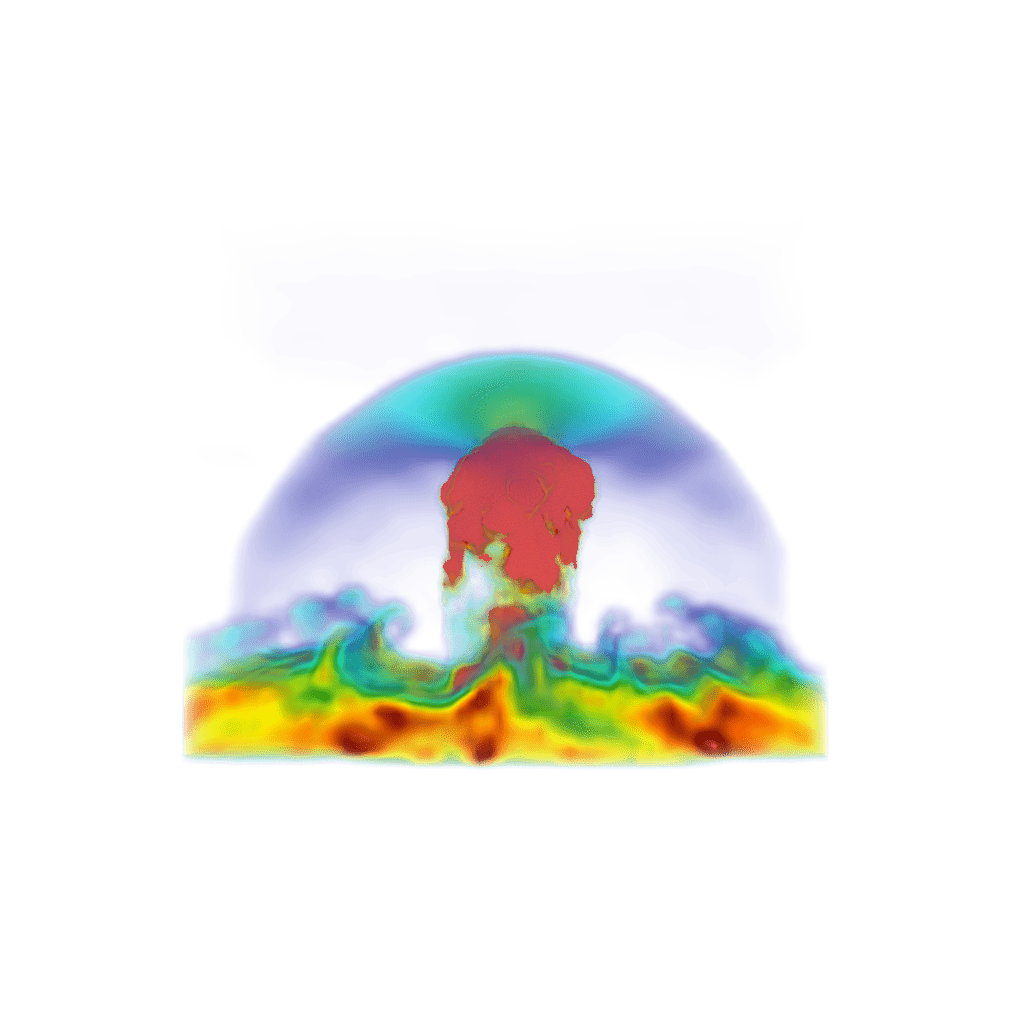}}
\endminipage
\rulesep
\minipage{0.24\textwidth}
{\includegraphics[width=\linewidth,trim={50 50 50 50}, clip, draft=false]{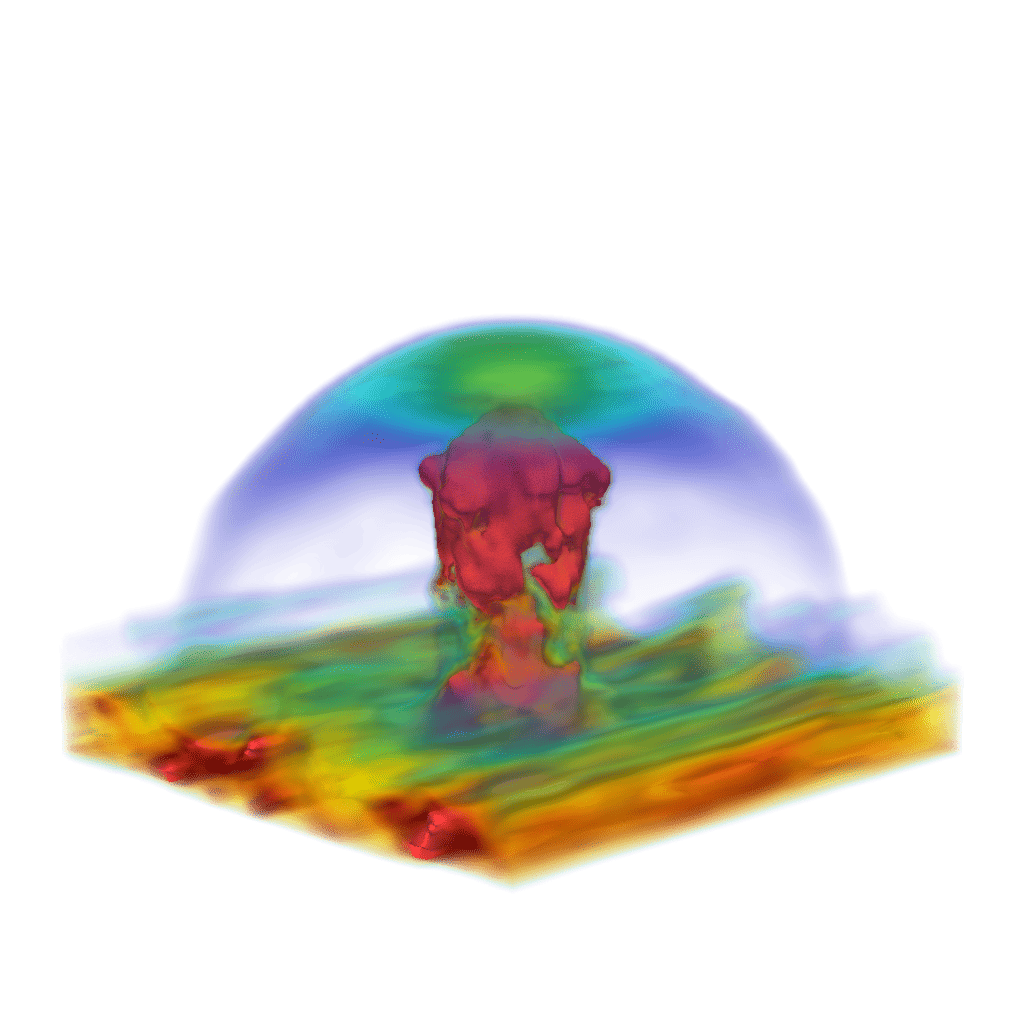}}
\endminipage
\minipage{0.24\textwidth}
{\includegraphics[width=\linewidth, trim={125 125 125 125}, clip, draft=false]{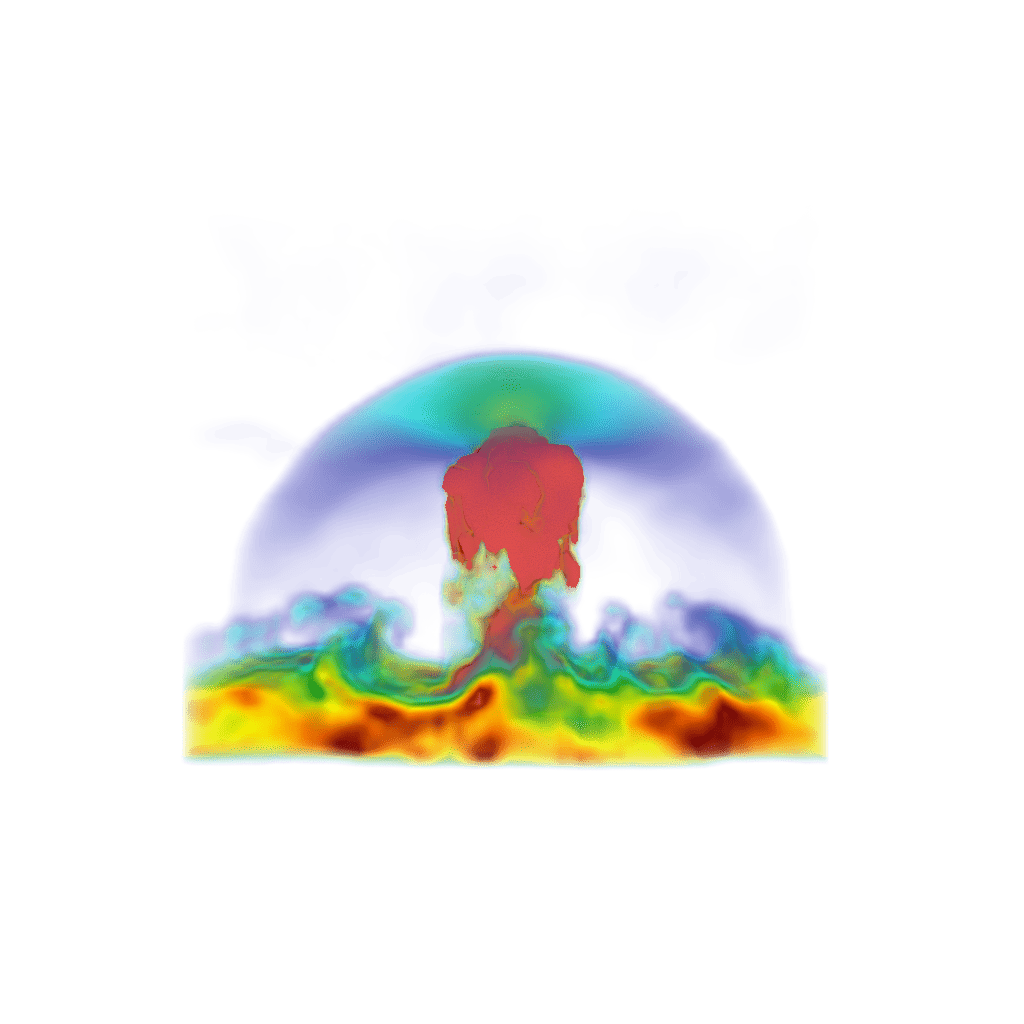}}
\endminipage
\hrulesep
\endminipage

\minipage{\linewidth}
\minipage{0.24\textwidth}
\includegraphics[width=\linewidth, clip,trim={50 50 50 50}, draft=false]{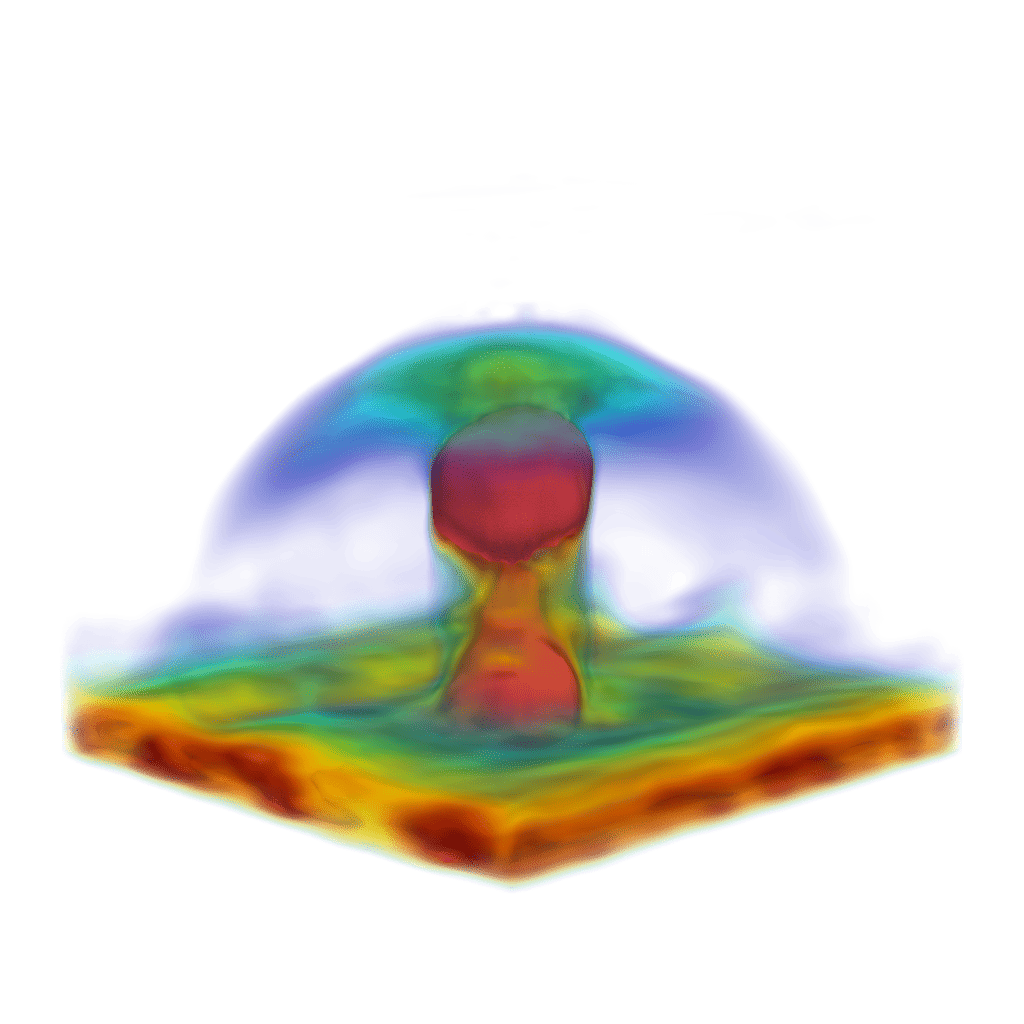}
\endminipage
\minipage{0.24\textwidth}
{\includegraphics[width=\linewidth, trim={125 125 125 125}, clip, draft=false]{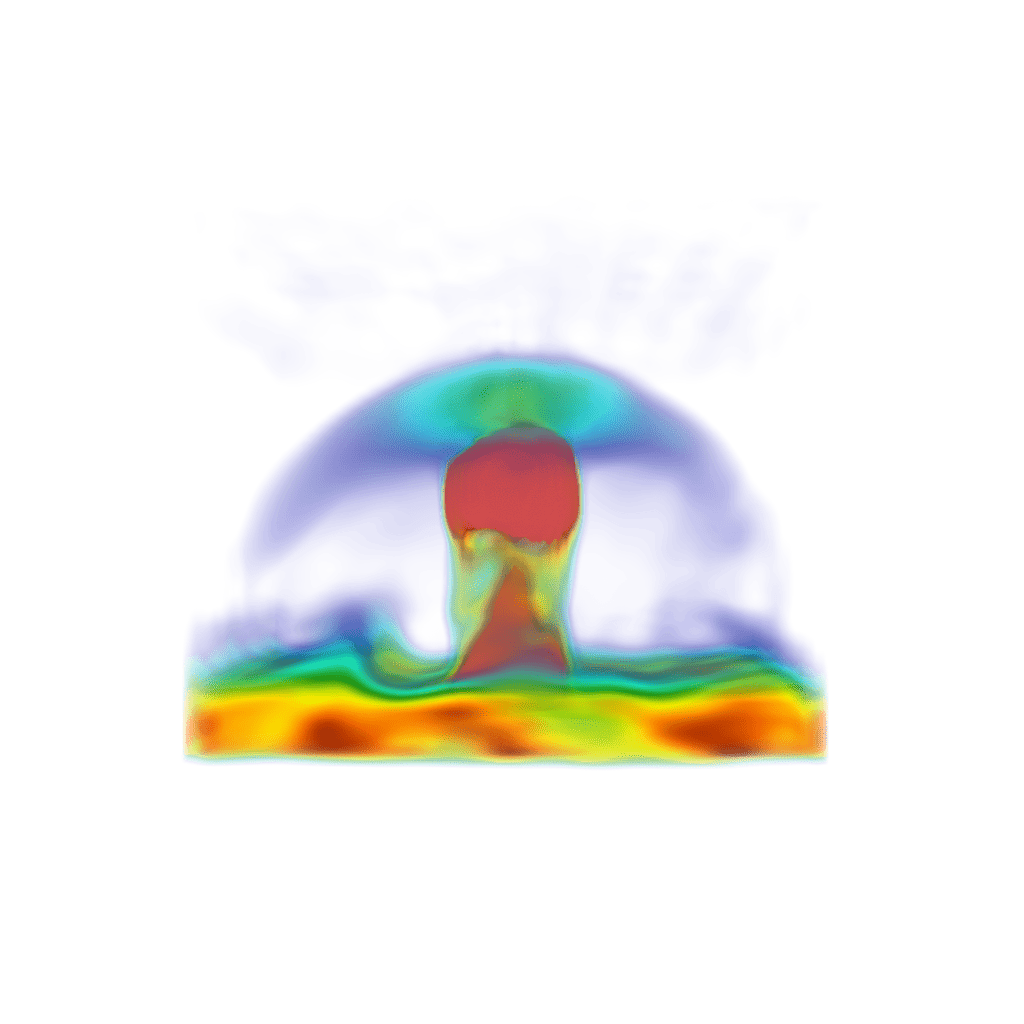}}
\endminipage
\rulesep
\minipage{0.24\textwidth}
{\includegraphics[width=\linewidth, trim={50 50 50 50}, clip, draft=false]{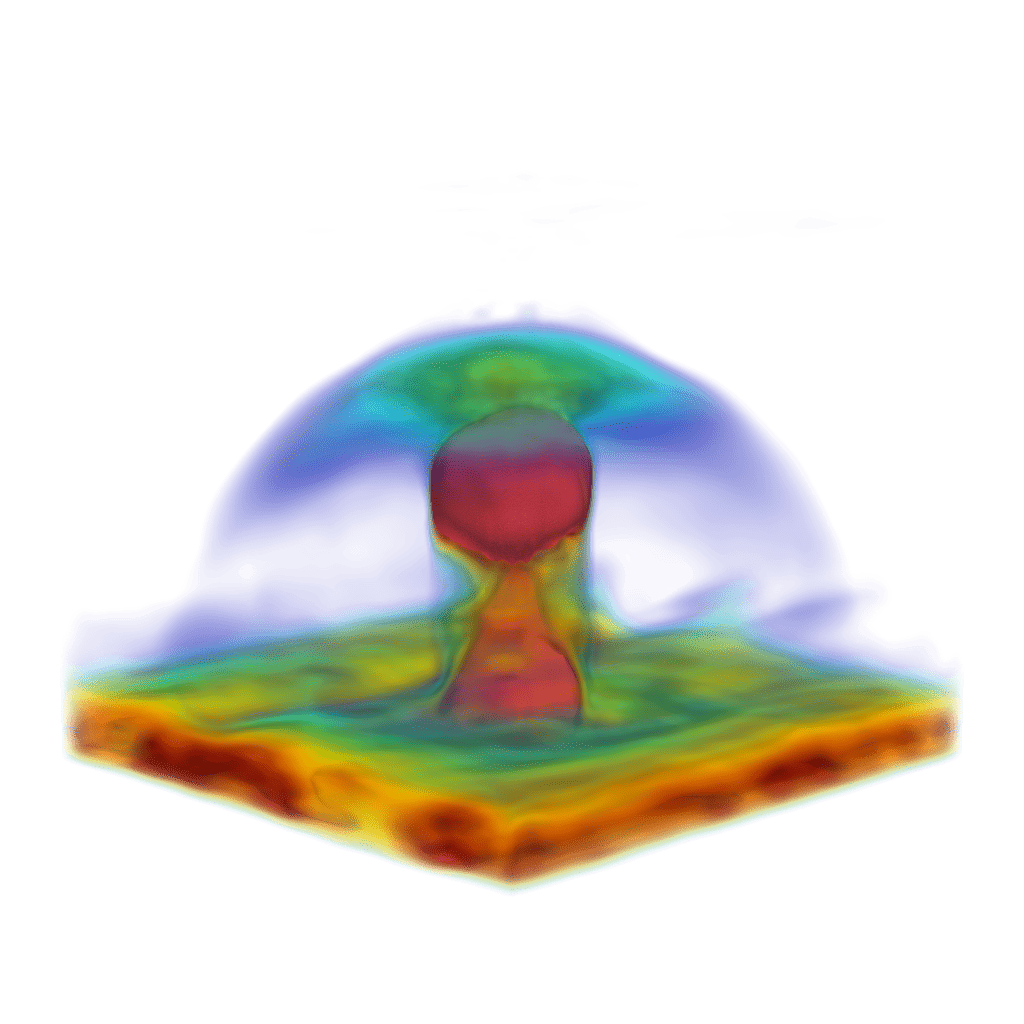}}
\endminipage
\minipage{0.24\textwidth}
{\includegraphics[width=\linewidth, trim={125 125 125 125}, clip, draft=false]{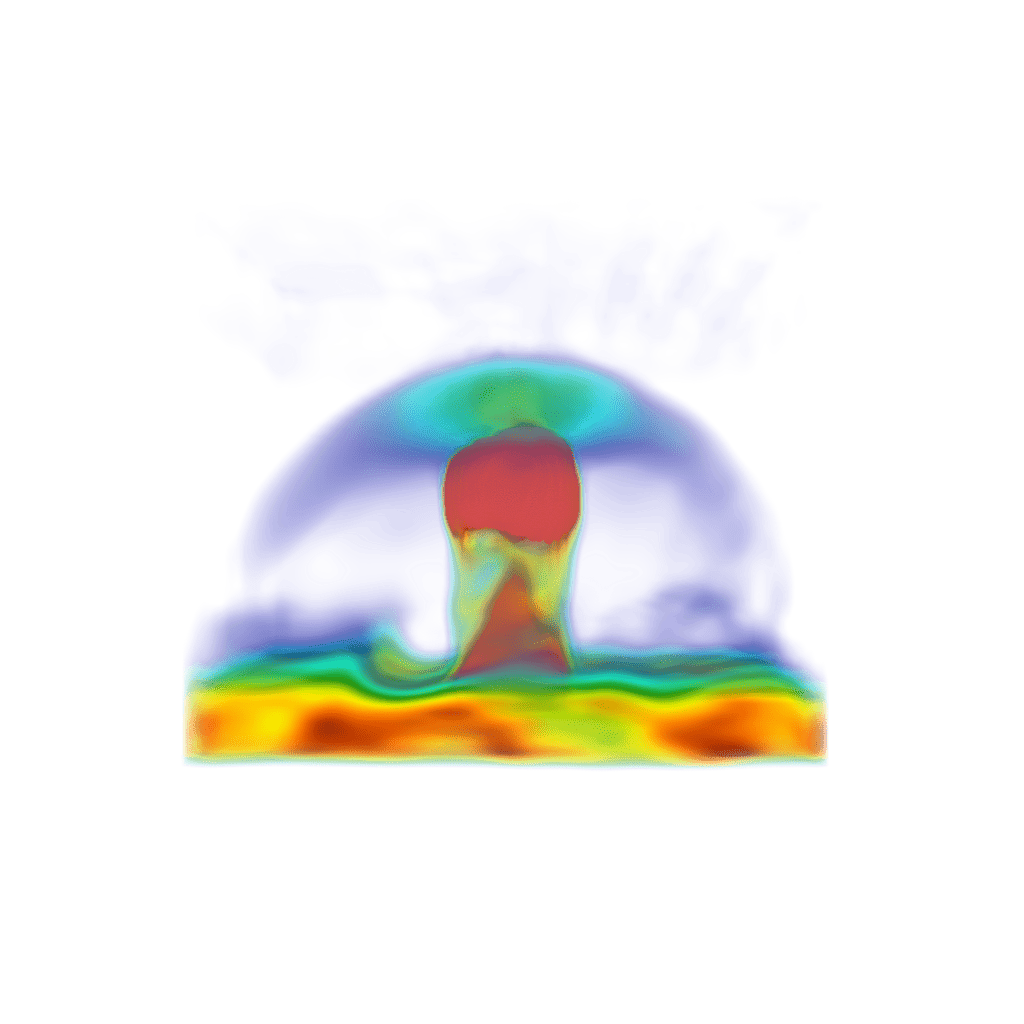}}
\endminipage
\endminipage
\caption{\textbf{Visualization from two different angles of the density at time $T=0.06$ for the cloud-shock interaction experiment, generated by the ground truth (Top Row), GenCFD (Middle Row) and C-FN0 (Bottom Row).} The colormap for all the figures ranges from $3.90$ (dark blue) to $6.35$ (dark red).}
\label{fig:s7}
\end{figure}

\begin{figure}[h!]
\begin{center}
\minipage{\linewidth}
\minipage{0.32\textwidth}
\includegraphics[width=\linewidth, clip,trim={50 50 50 50}]{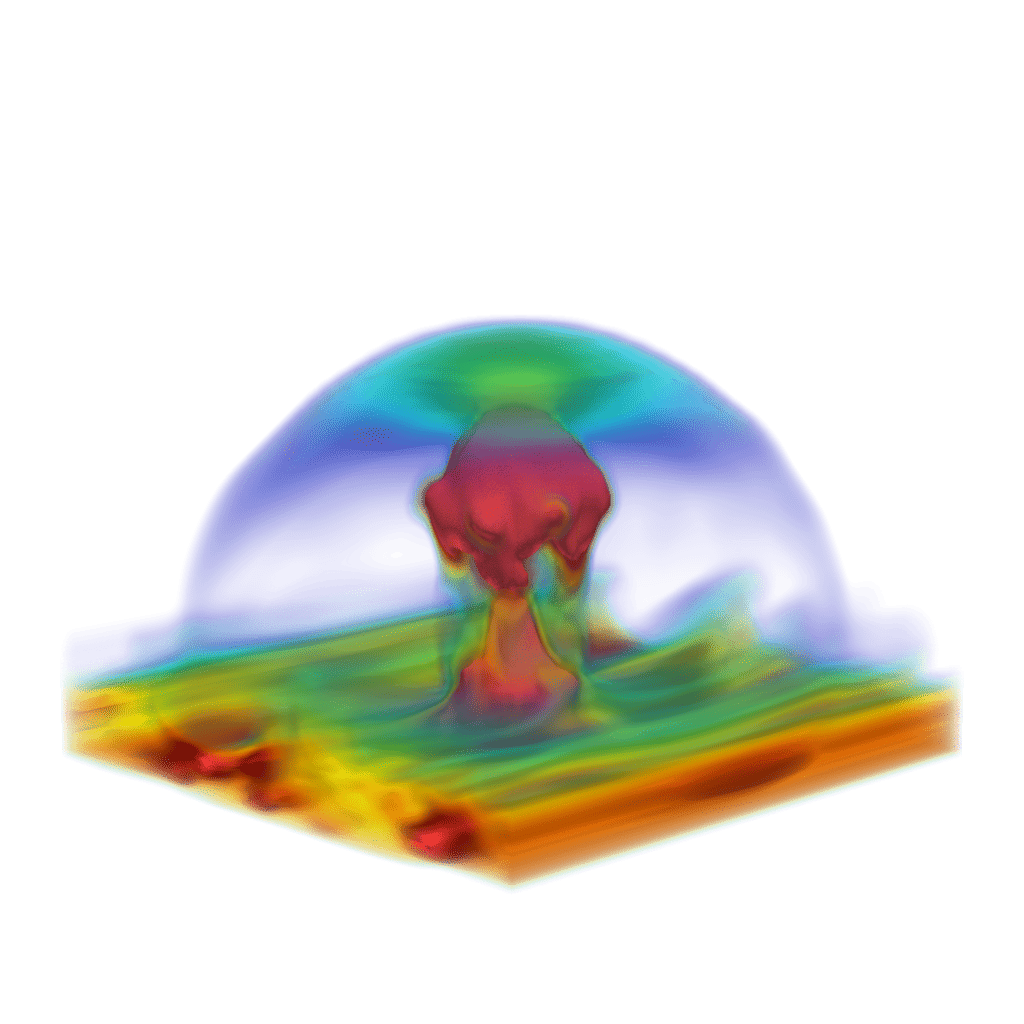}
\endminipage
\minipage{0.32\textwidth}
\includegraphics[width=\linewidth, clip,trim={123 125 125 125}]{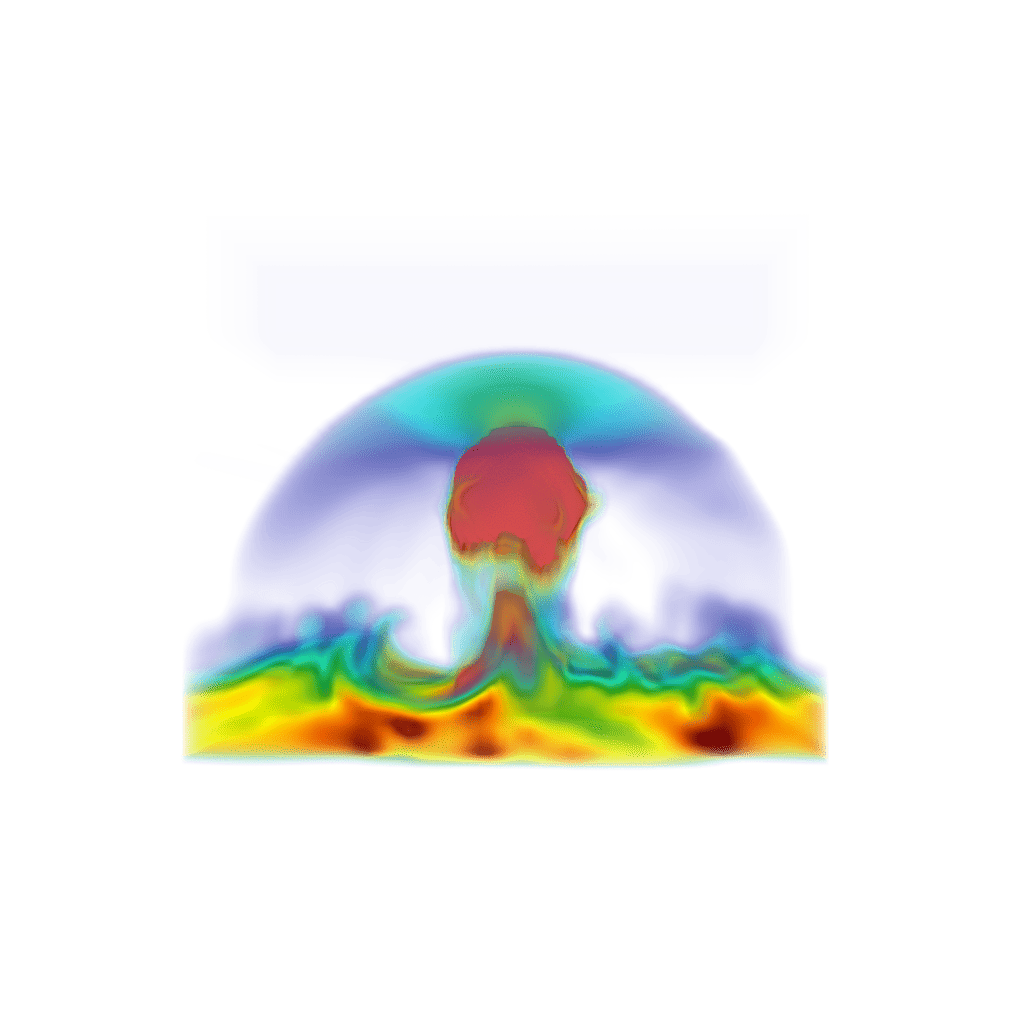}
\endminipage
\rulesep
\minipage{0.32\textwidth}
\includegraphics[width=\linewidth, clip,trim={50 50 50 50}]{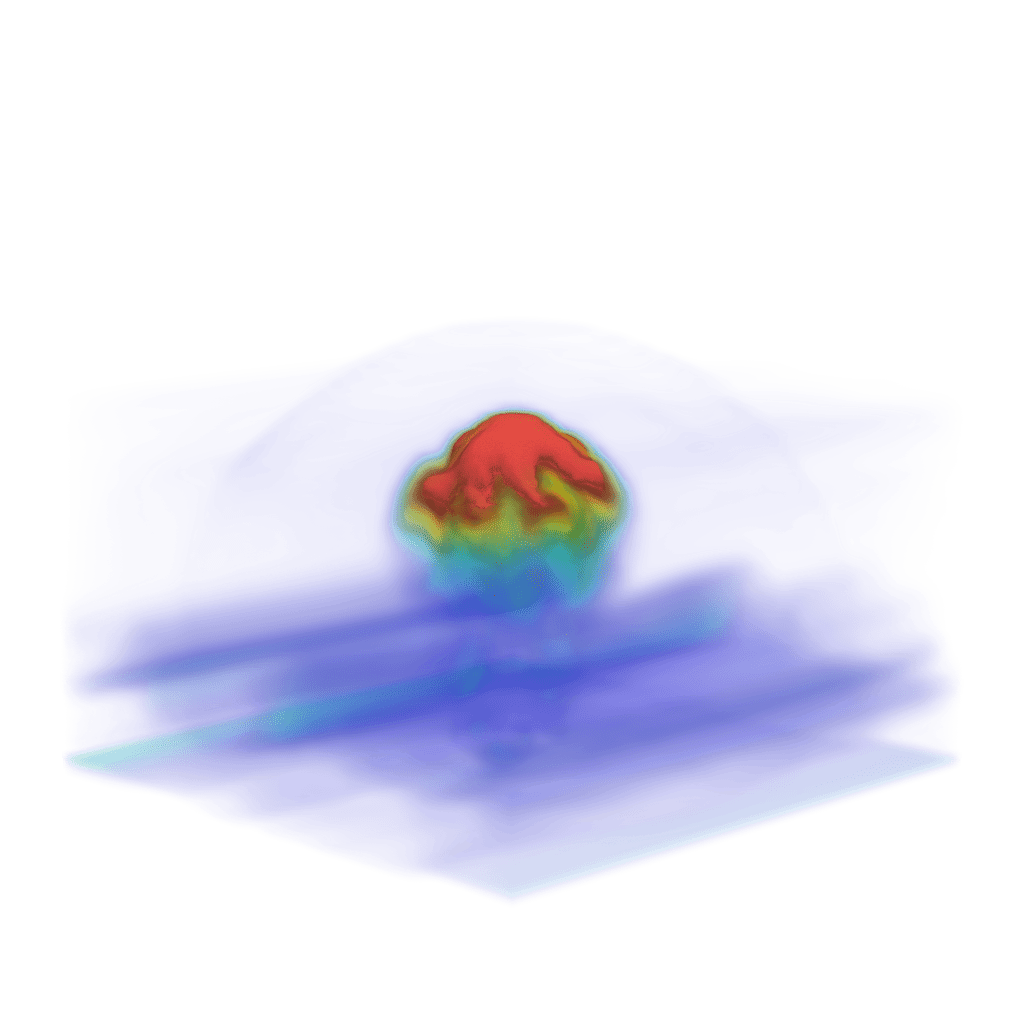}
\endminipage
\hrulesep
\endminipage

\minipage{\linewidth}
\minipage{0.32\textwidth}
\includegraphics[width=\linewidth, clip,trim={50 50 50 50}]
{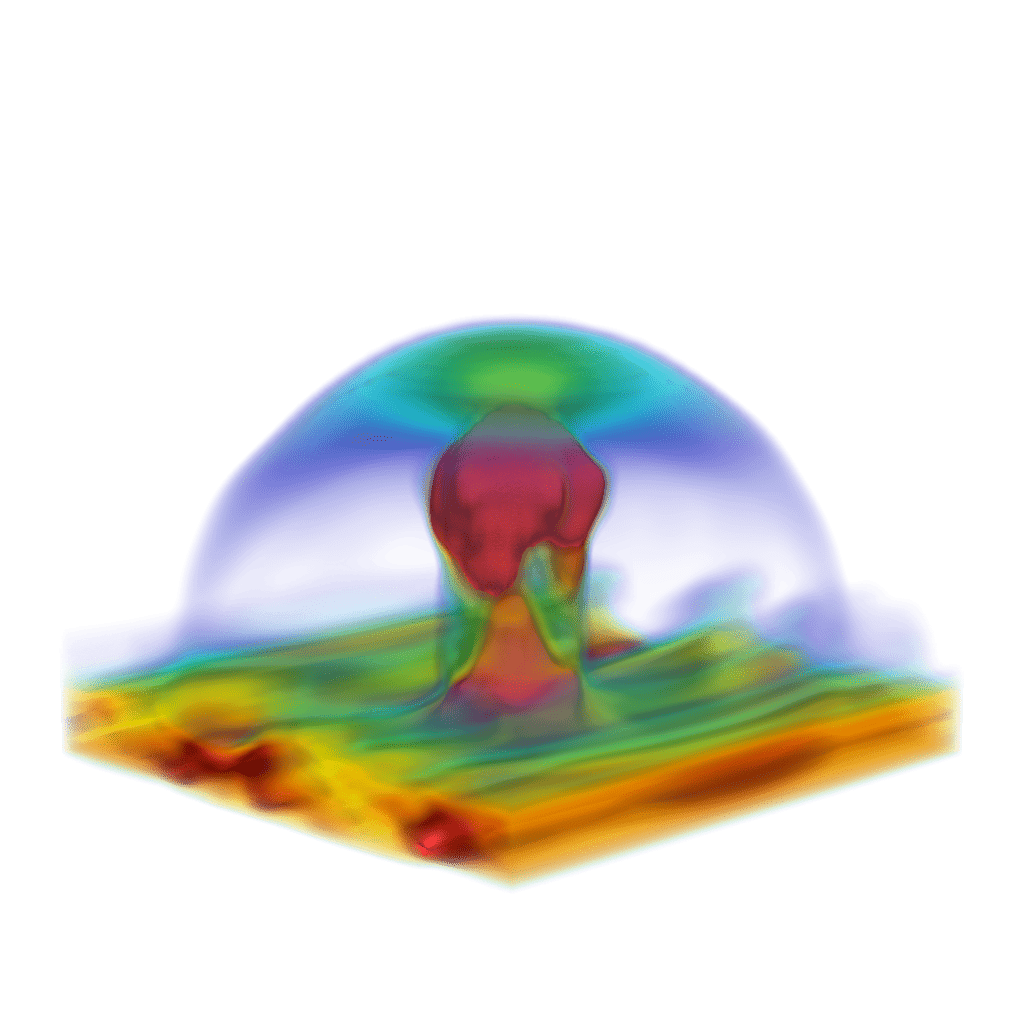}
\endminipage
\minipage{0.32\textwidth}
\includegraphics[width=\linewidth, clip,trim={123 125 125 125}]{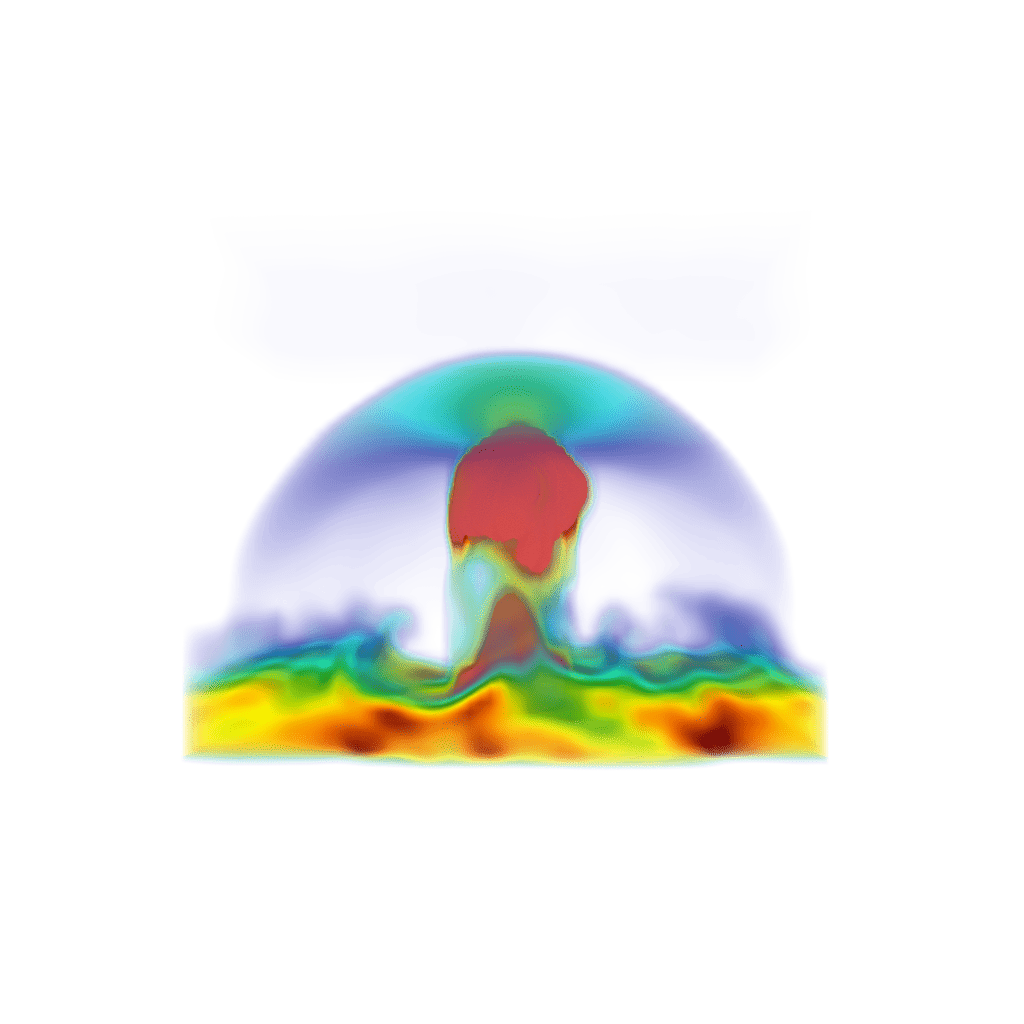}
\endminipage
\rulesep
\minipage{0.32\textwidth}
\includegraphics[width=\linewidth, clip,trim={50 50 50 50}]
{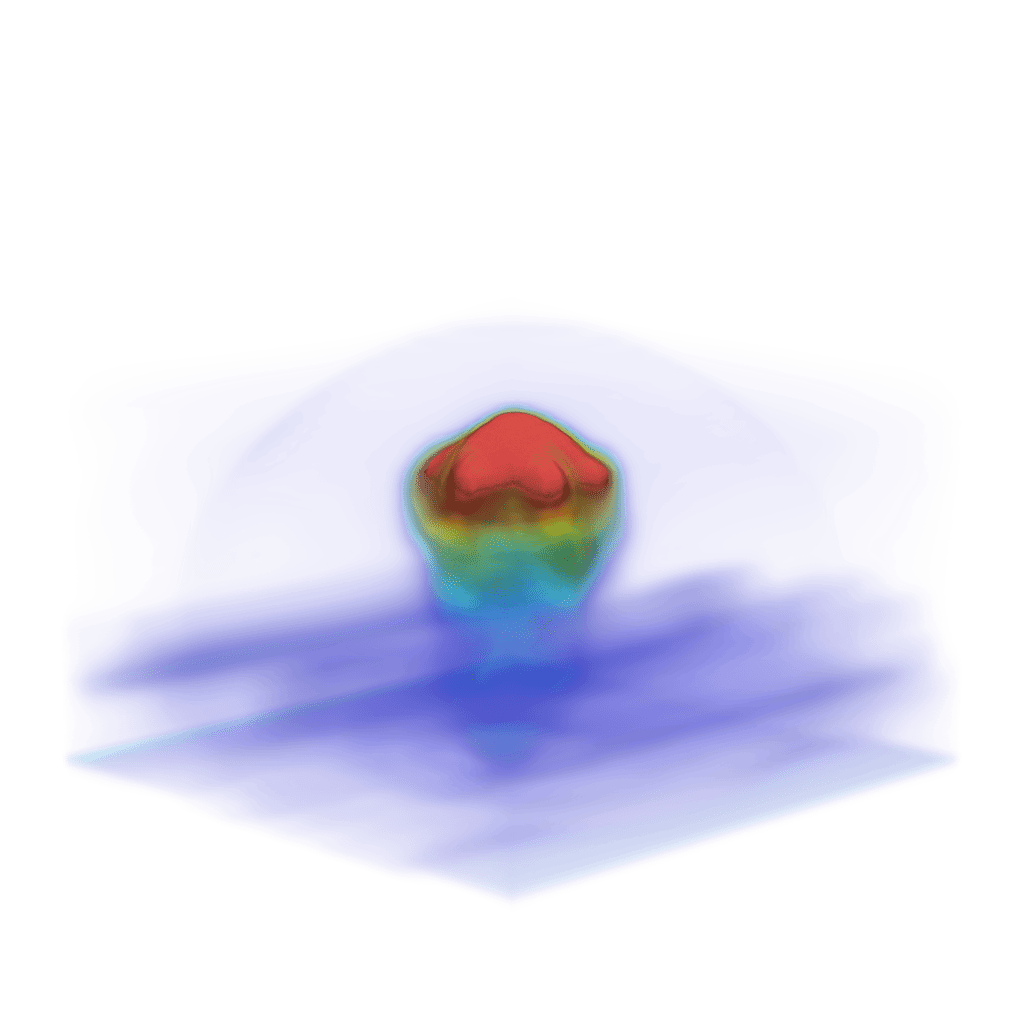}
\endminipage
\hrulesep
\endminipage

\minipage{\linewidth}
\minipage{0.32\textwidth}
\includegraphics[width=\linewidth, clip,trim={50 50 50 50}]
{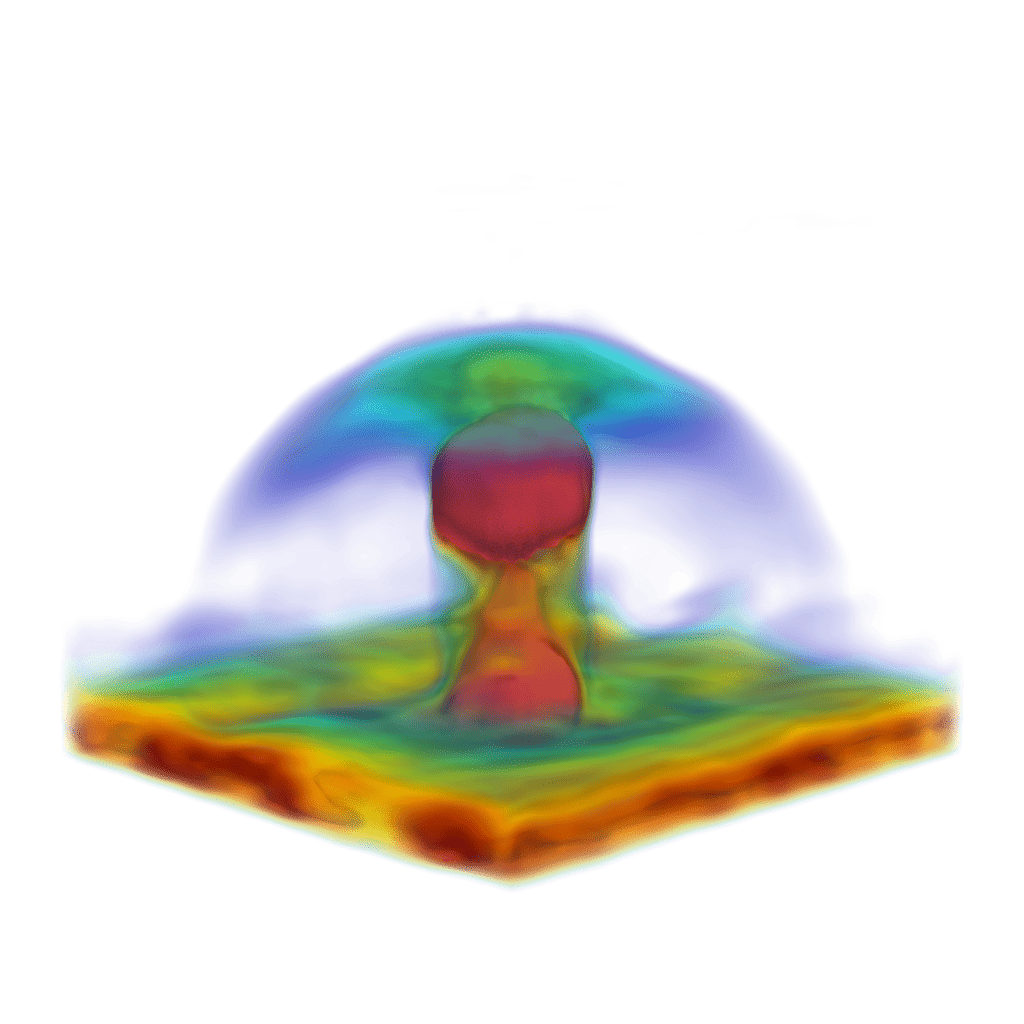}
\endminipage
\minipage{0.32\textwidth}
\includegraphics[width=\linewidth, clip,trim={125 125 125 125}]
{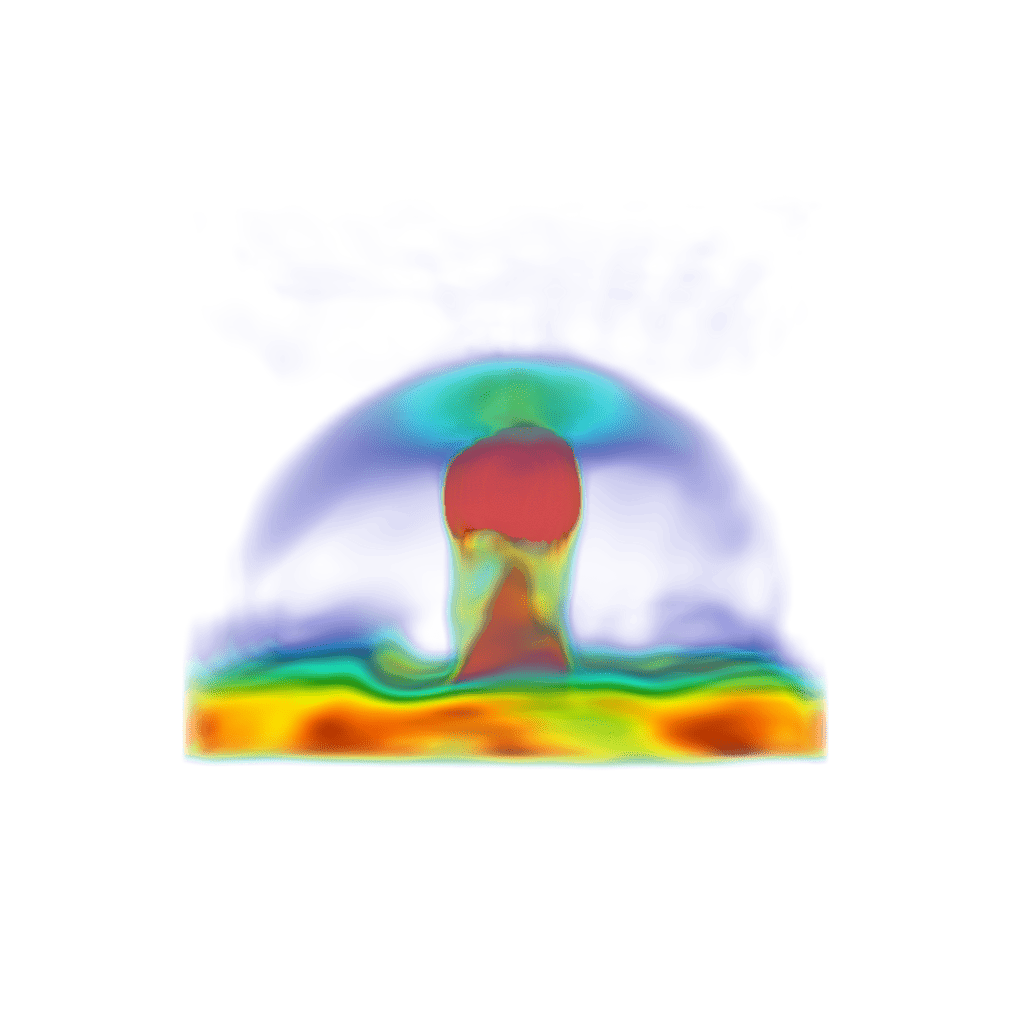}
\endminipage
\rulesep
\minipage{0.32\textwidth}
\includegraphics[width=\linewidth, clip,trim={50 50 50 50}]{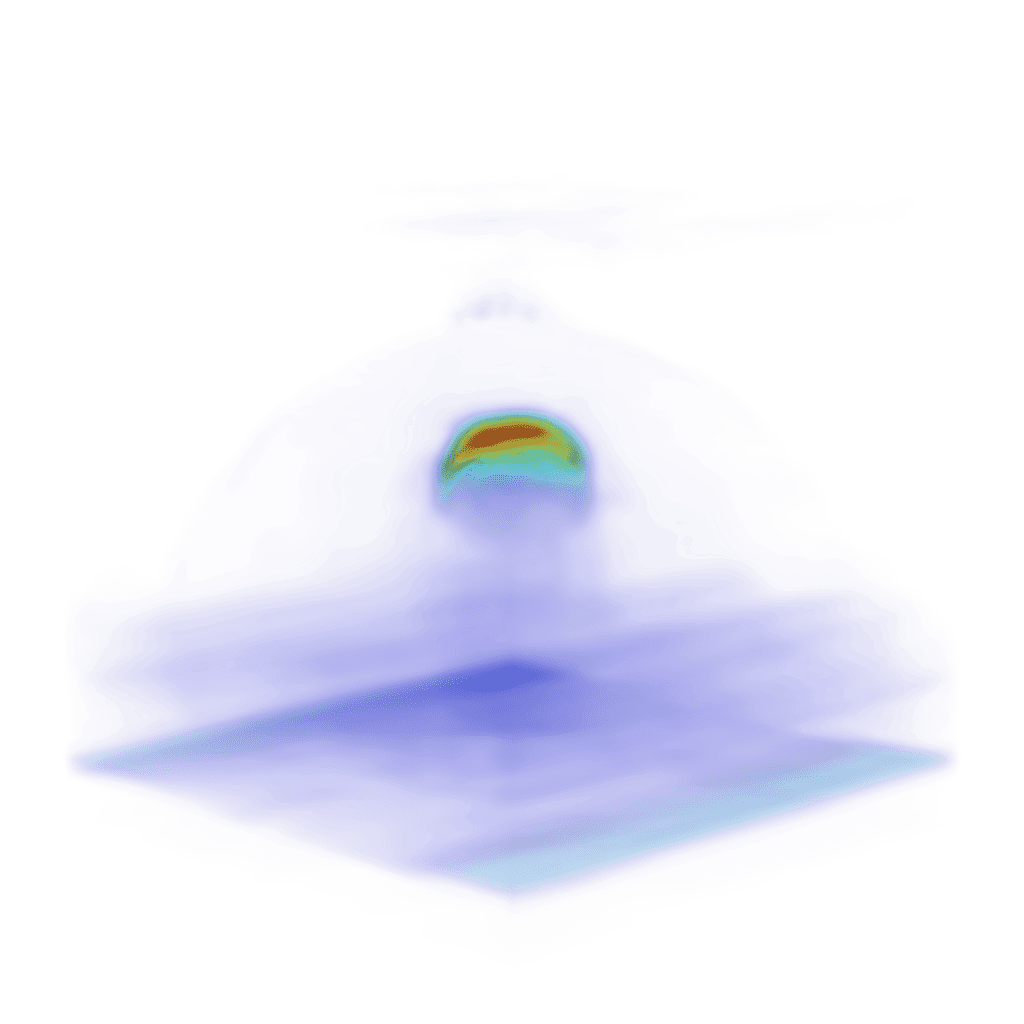}
\endminipage
\endminipage
\caption{\textbf{Mean (two different angles in Left and Center Column) and standard deviation (Right Column) of the density at time  $T=0.06$ for the cloud-shock interaction experiment with ground truth (Top Row), GenCFD (Middle Row) and C-FNO (Bottom Row).} The colormap for the figures representing the means ranges from $3.90$ (dark blue) to $6.35$ (dark red). The colormap for the figures representing the standard deviations ranges from $0.0075$ to $3.0$ for the top and middle rows, and from $0.0075$ to $2.0$ for the bottom row.}
\label{fig:s8}
\end{center}
\end{figure}

\begin{figure}[h!]
\minipage{\linewidth}
\minipage{0.25\textwidth}
\includegraphics[width=\linewidth]{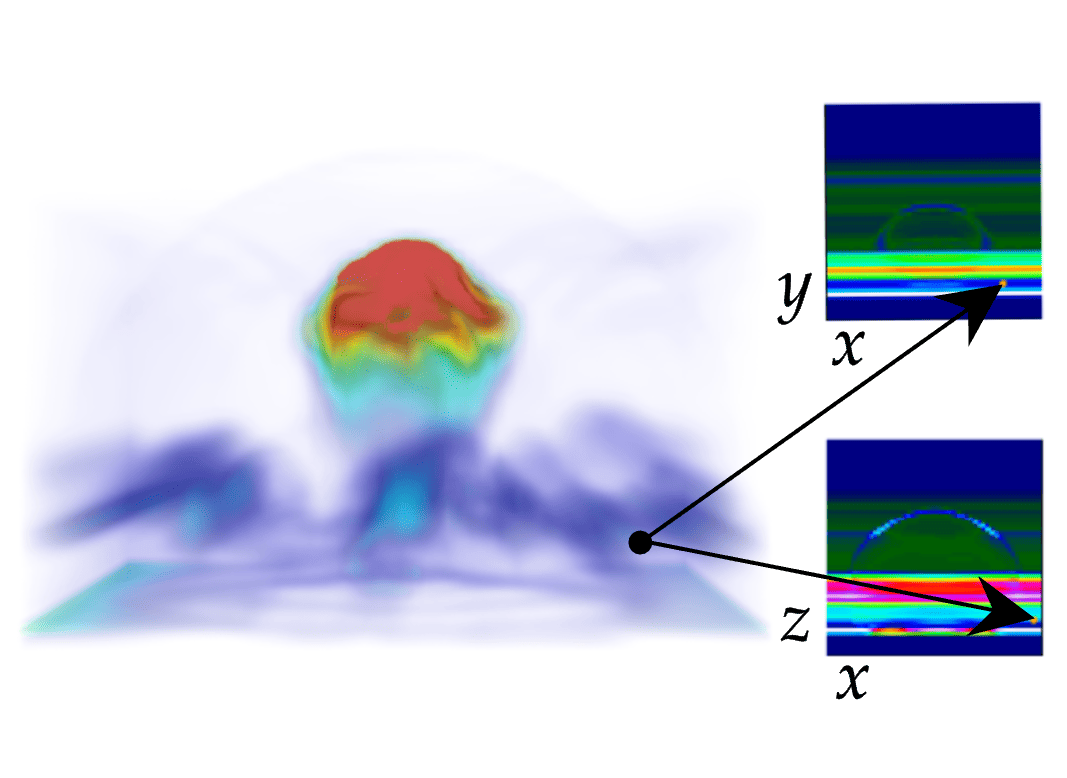}
\endminipage
\minipage{0.25\textwidth}
{\includegraphics[width=\linewidth]{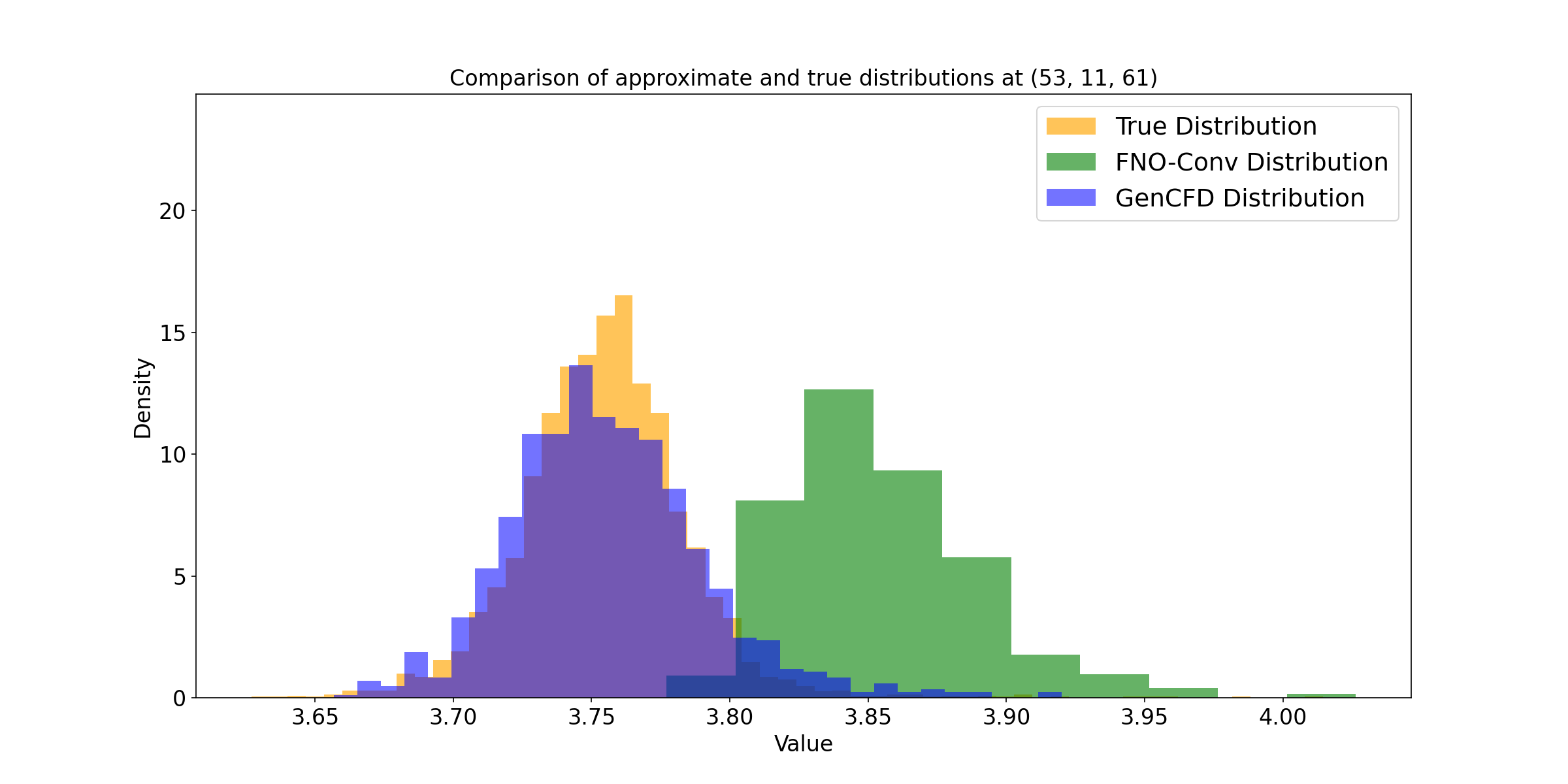}}
\endminipage
\minipage{0.25\textwidth}
{\includegraphics[width=\linewidth]{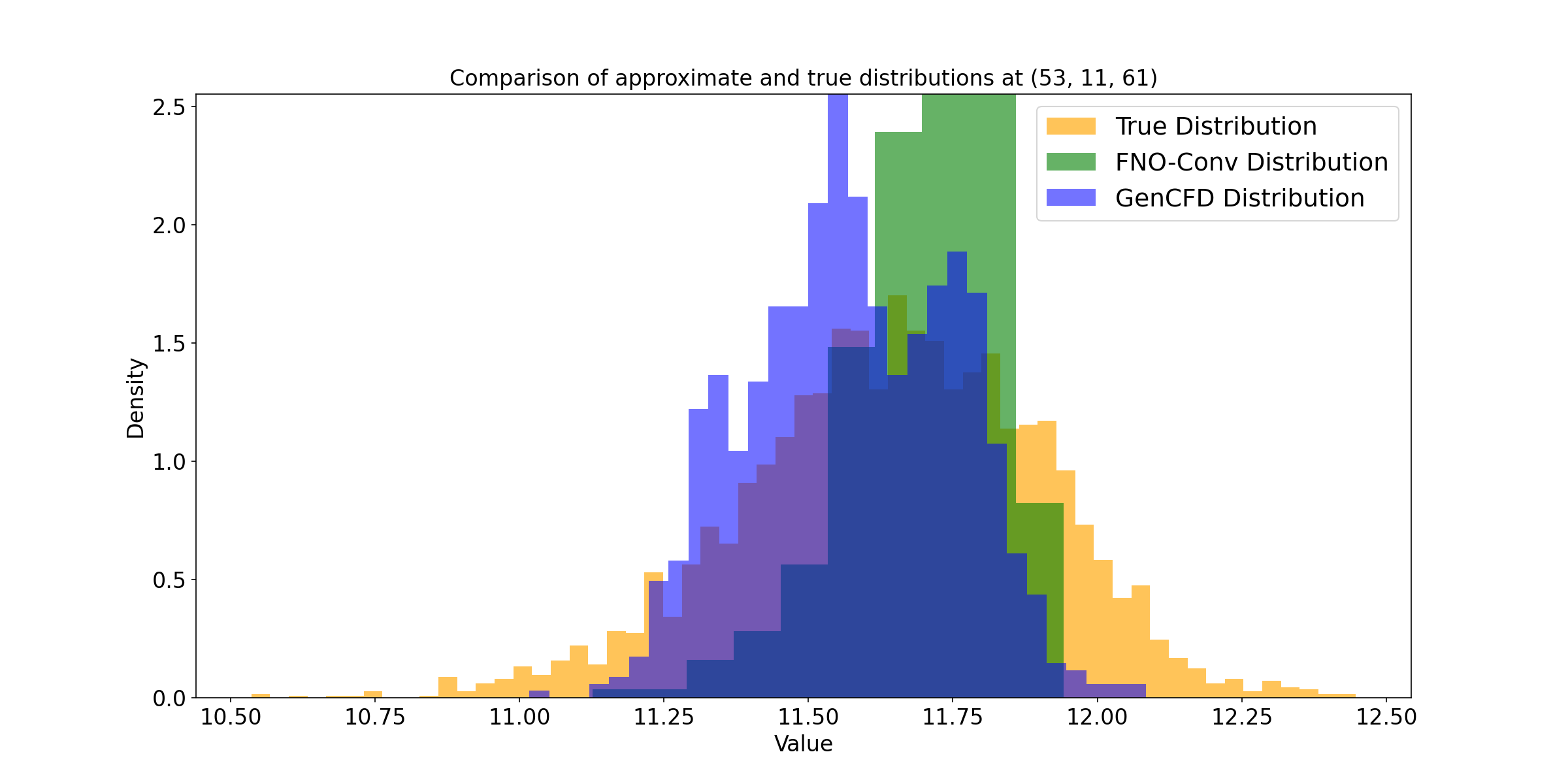}}
\endminipage
\minipage{0.25\textwidth}
{\includegraphics[width=\linewidth]{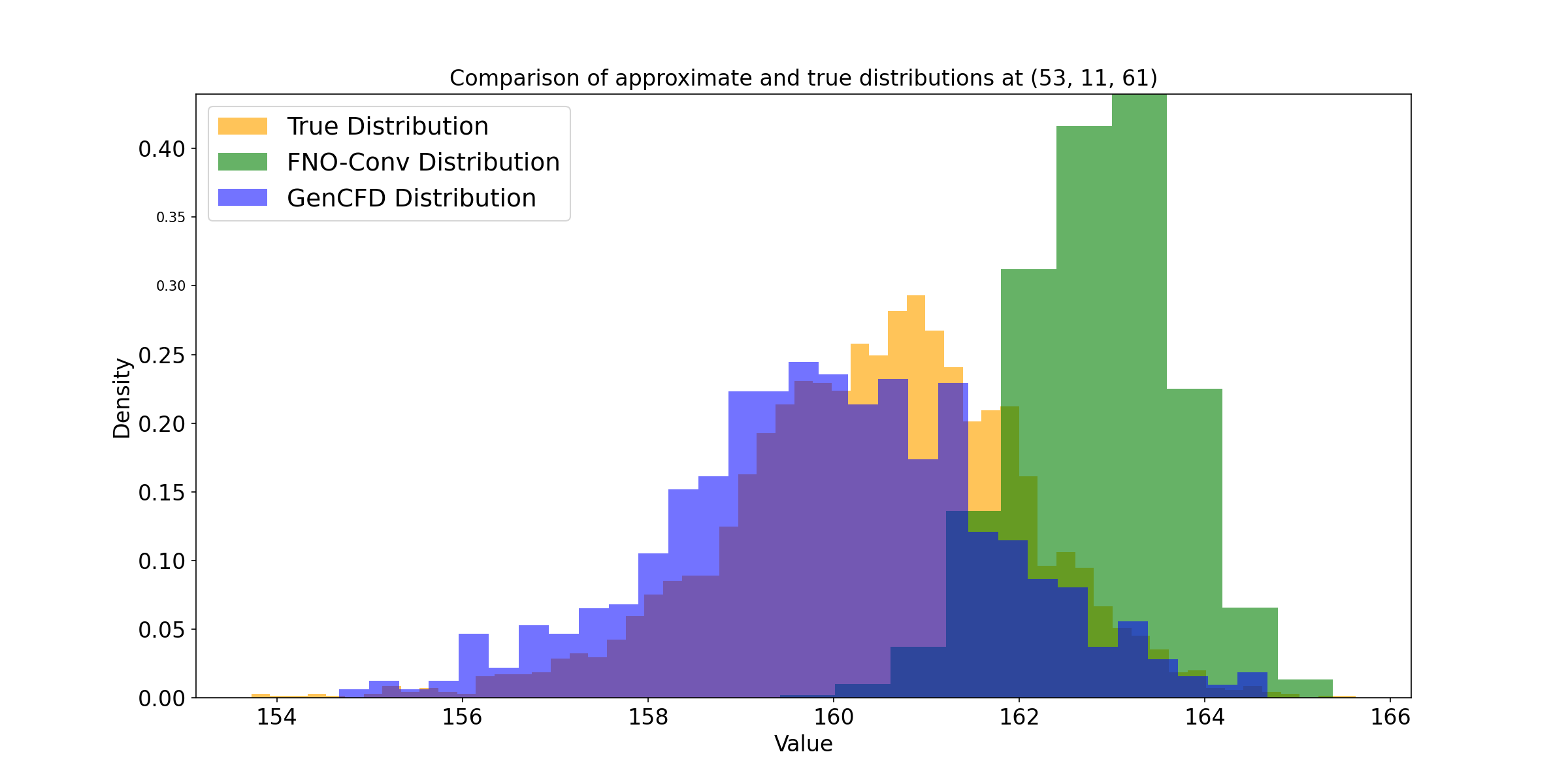}}
\endminipage
\endminipage

\minipage{\linewidth}
\minipage{0.25\textwidth}
\includegraphics[width=\linewidth]{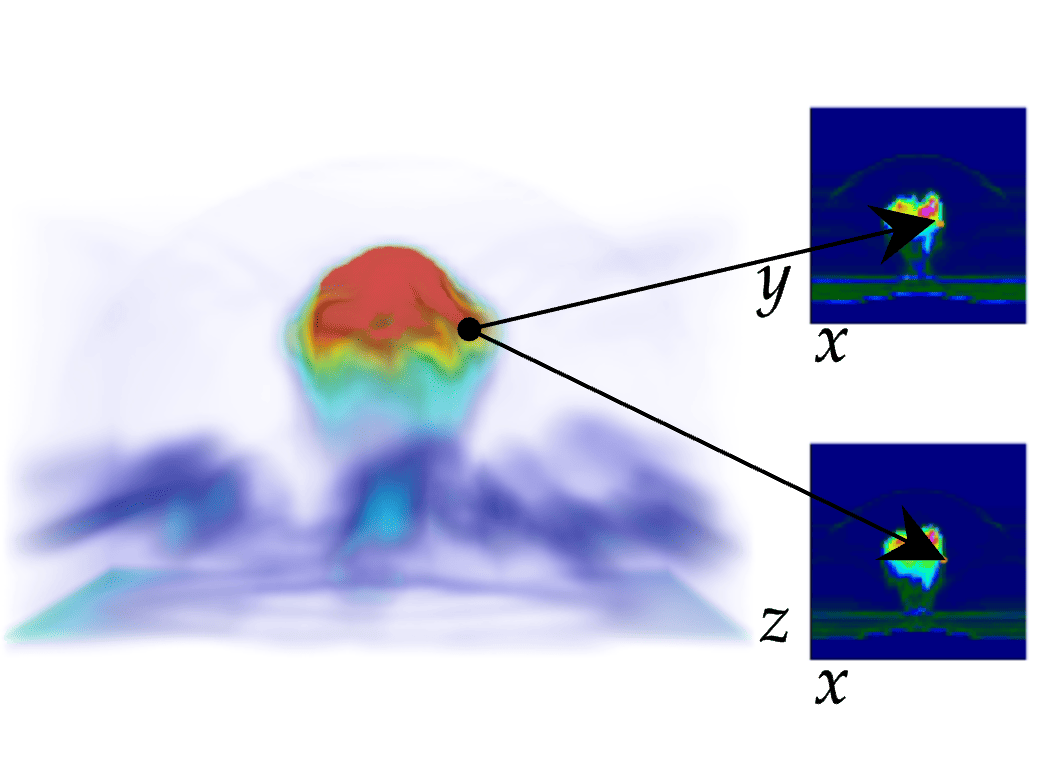}
\subcaption{Points}
\endminipage
\minipage{0.25\textwidth}
{\includegraphics[width=\linewidth]{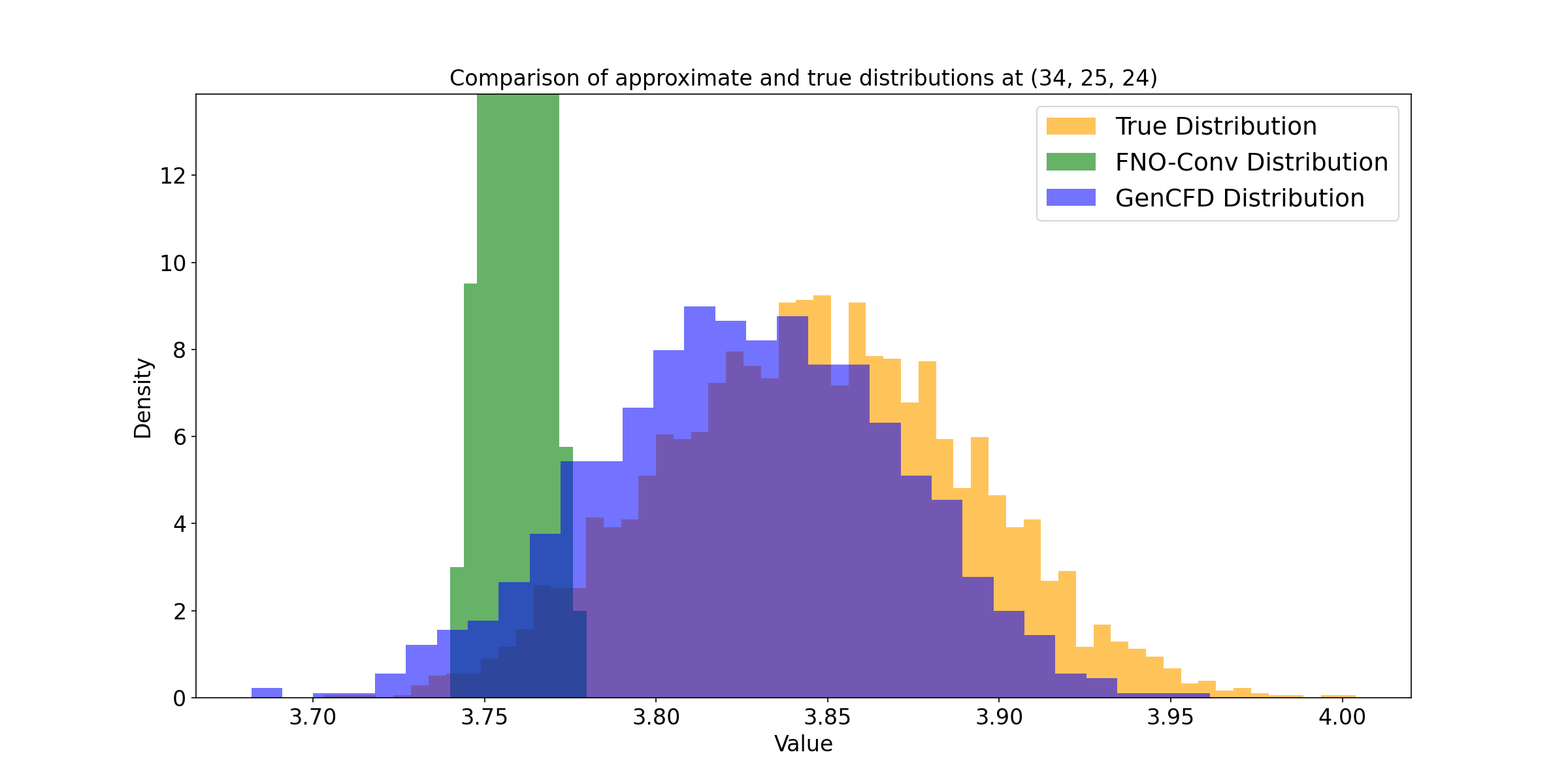}}
\subcaption{$\rho$}
\endminipage
\minipage{0.25\textwidth}
{\includegraphics[width=\linewidth]{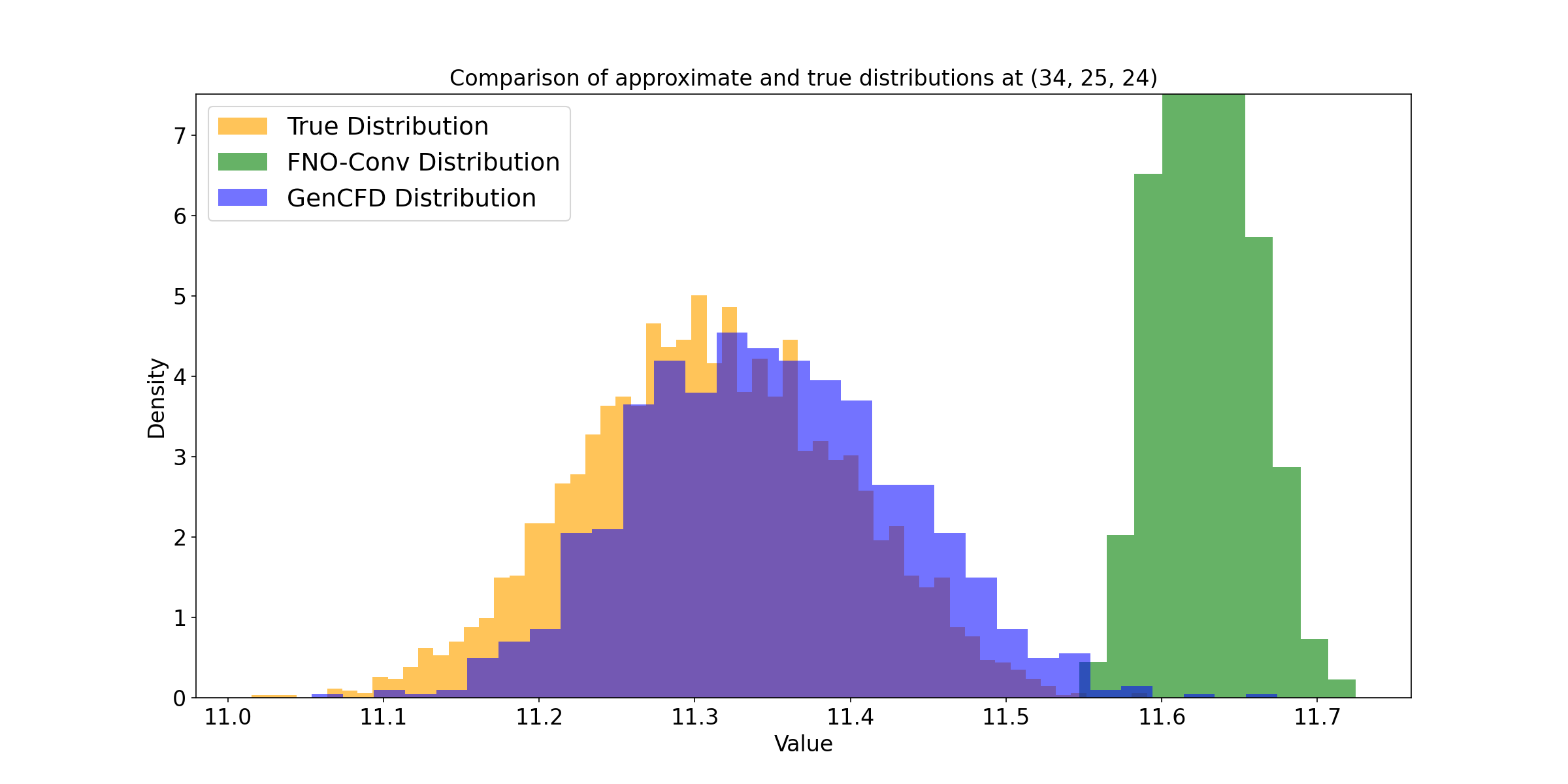}}
\subcaption{$u_x$}
\endminipage
\minipage{0.25\textwidth}
{\includegraphics[width=\linewidth]{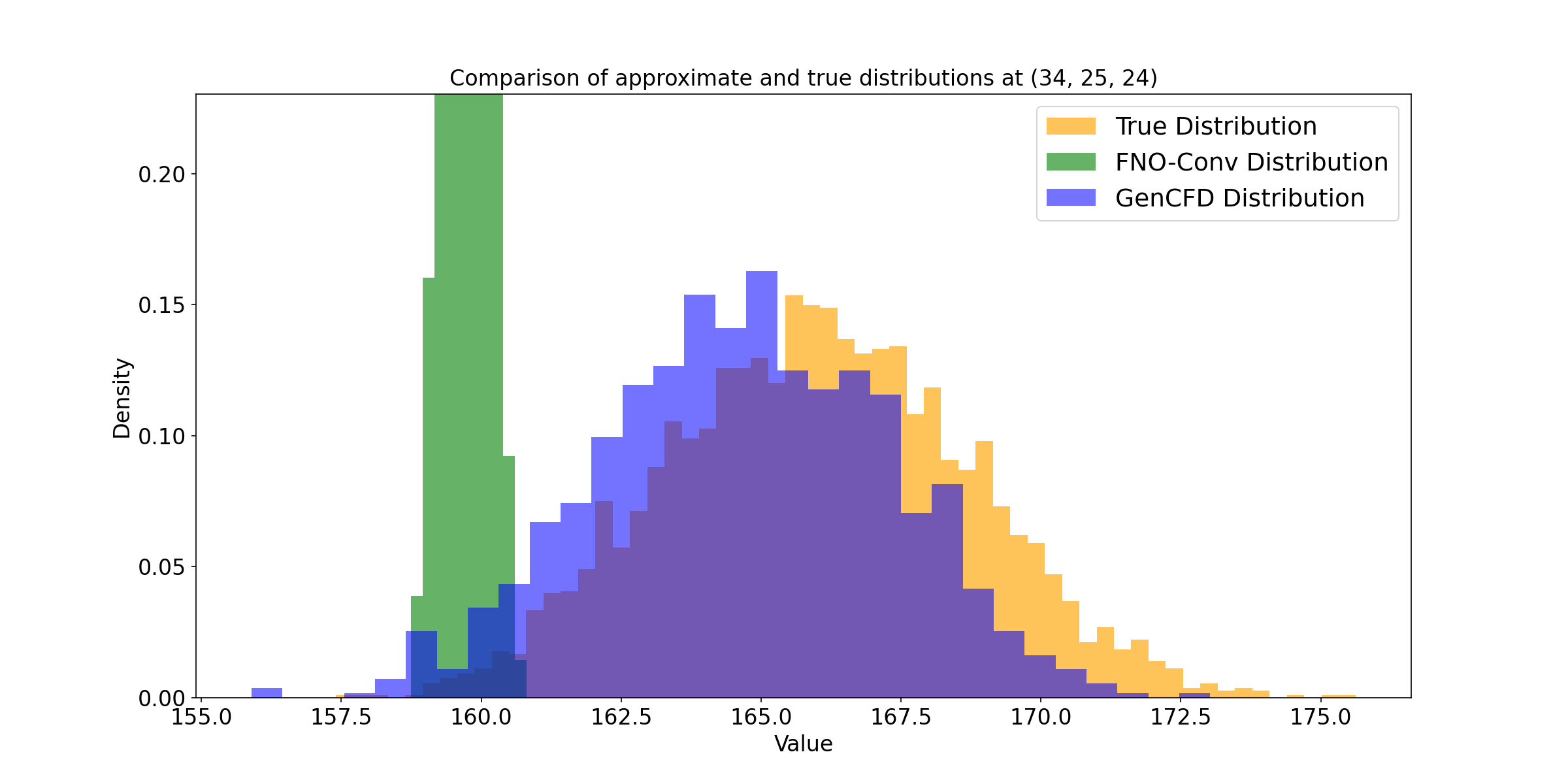}}
\subcaption{$p$}
\endminipage
\endminipage
\caption{\textbf{Point PDFs at two different points (left most column) of the density, $x$-velocity and pressure, at time $T=0.06$ of the cloud-shock interaction experiment, generated by the ground truth, GenCFD and C-FNO.}}
\label{fig:s9}
\end{figure}

\clearpage
\newpage
\begin{figure}[h!]
\centering
\minipage{\linewidth}
\minipage{0.48\textwidth}
  {\includegraphics[width=\linewidth]{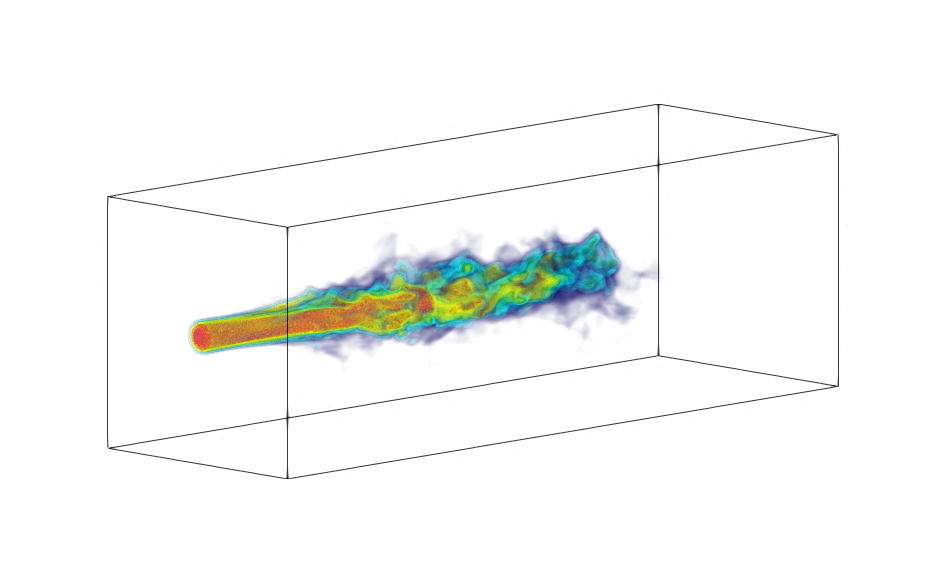}}
 \endminipage
\minipage{0.48\textwidth}
  {\includegraphics[width=\linewidth]{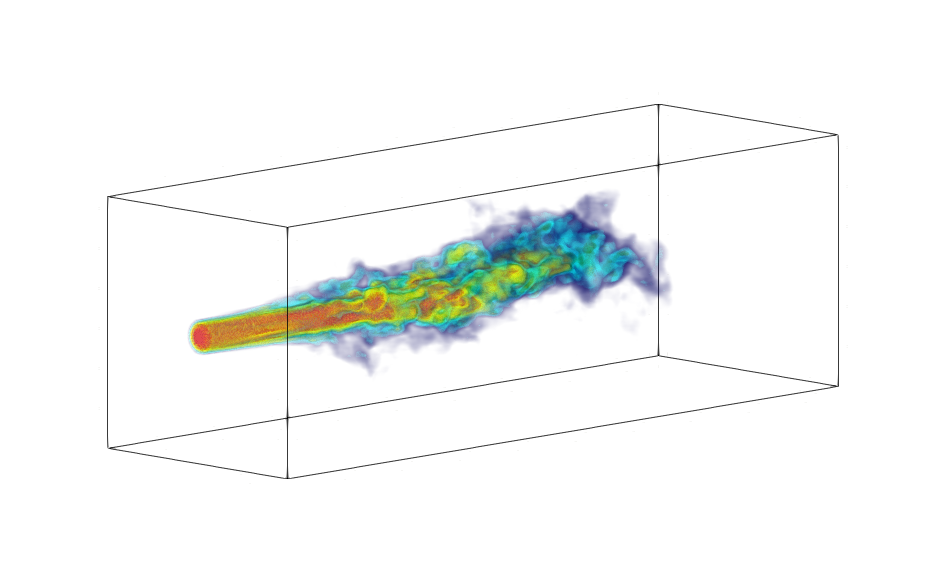}}
 \endminipage
 \endminipage

\minipage{\linewidth}
\minipage{0.48\textwidth}
  {\includegraphics[width=\linewidth]{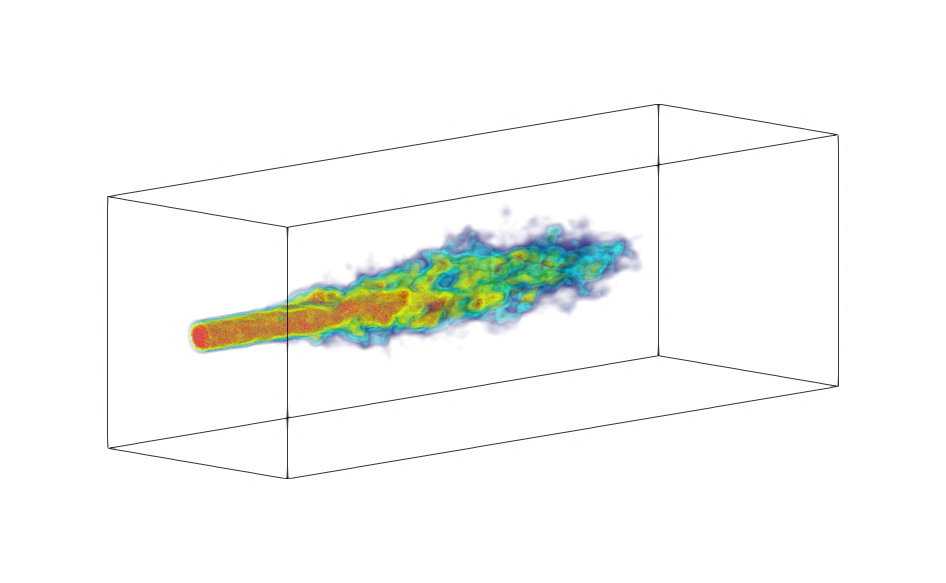}}
 \endminipage
\minipage{0.48\textwidth}
  {\includegraphics[width=\linewidth]{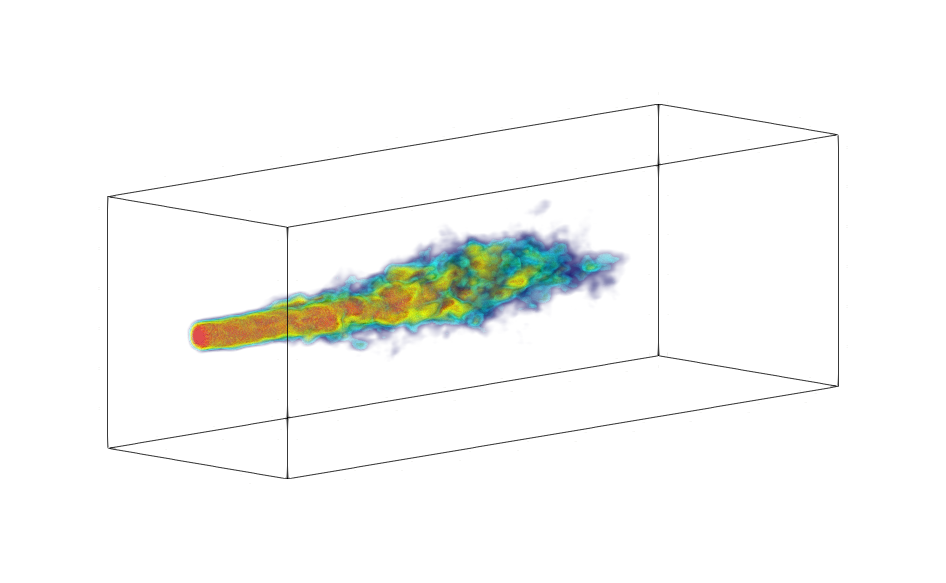}}
 \endminipage
 \endminipage

 \minipage{\linewidth}
\minipage{0.48\textwidth}
  {\includegraphics[width=\linewidth]{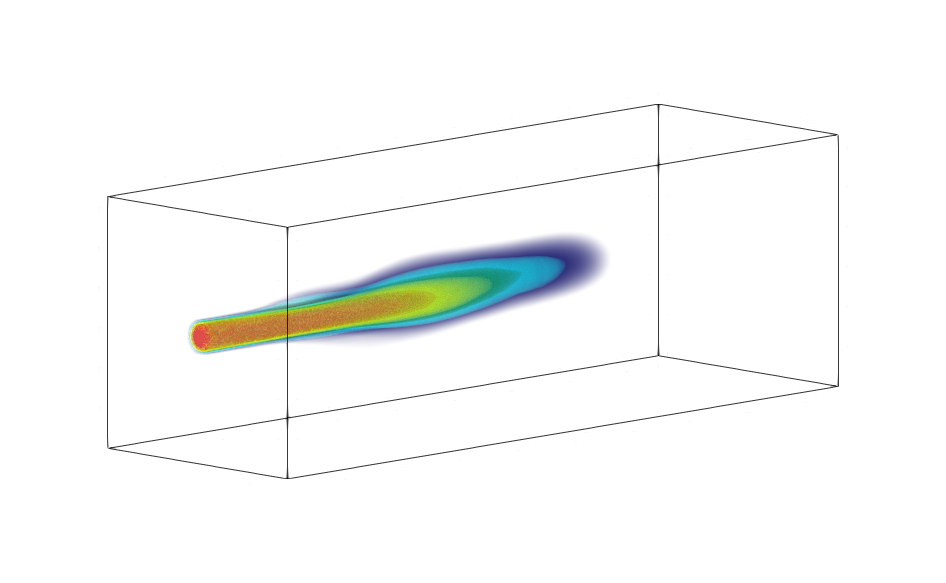}}
 \endminipage
\minipage{0.48\textwidth}
  {\includegraphics[width=\linewidth]{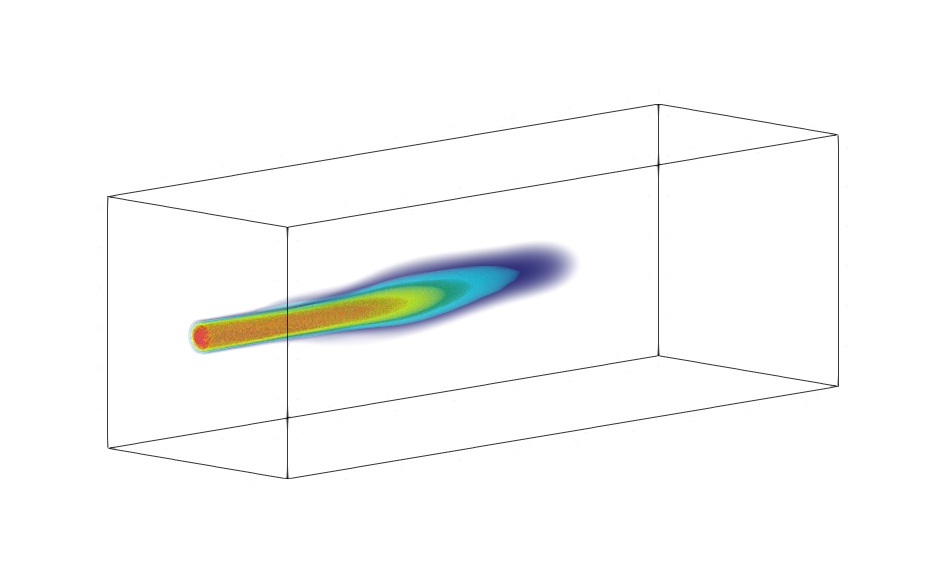}}
 \endminipage
 \endminipage

\caption{\textbf{Two random samples, for the same injection velocity, of the (pointwise) kinetic energy at time $T=130$ for the nozzle flow experiment}. Data generated with ground truth (Top Row), GenCFD (Middle Row) and UViT (Bottom Row).}
\label{fig:nozz1}
\end{figure}

\clearpage
\newpage
\begin{figure}[h!]
\centering
\minipage{\linewidth}
\minipage{0.48\textwidth}
  {\includegraphics[width=\linewidth]{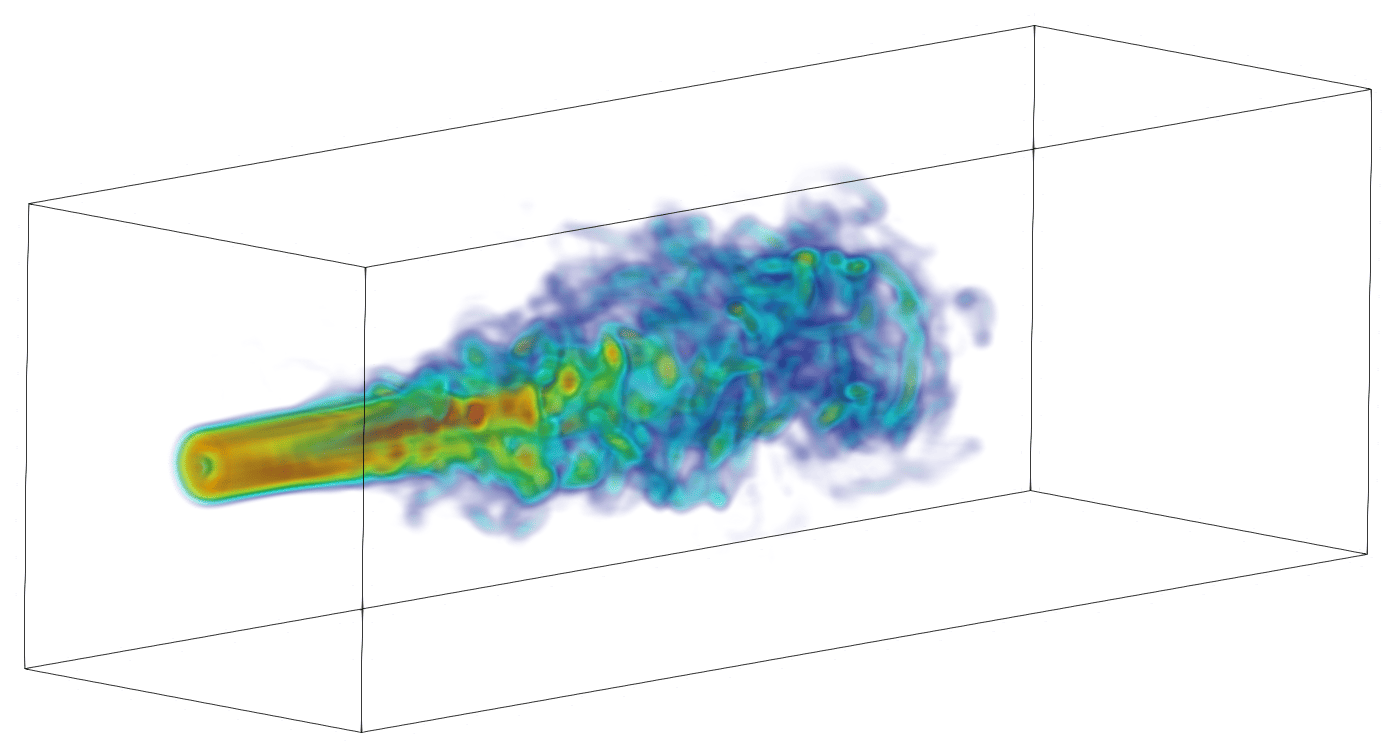}}
 \endminipage
\minipage{0.48\textwidth}
  {\includegraphics[width=\linewidth]{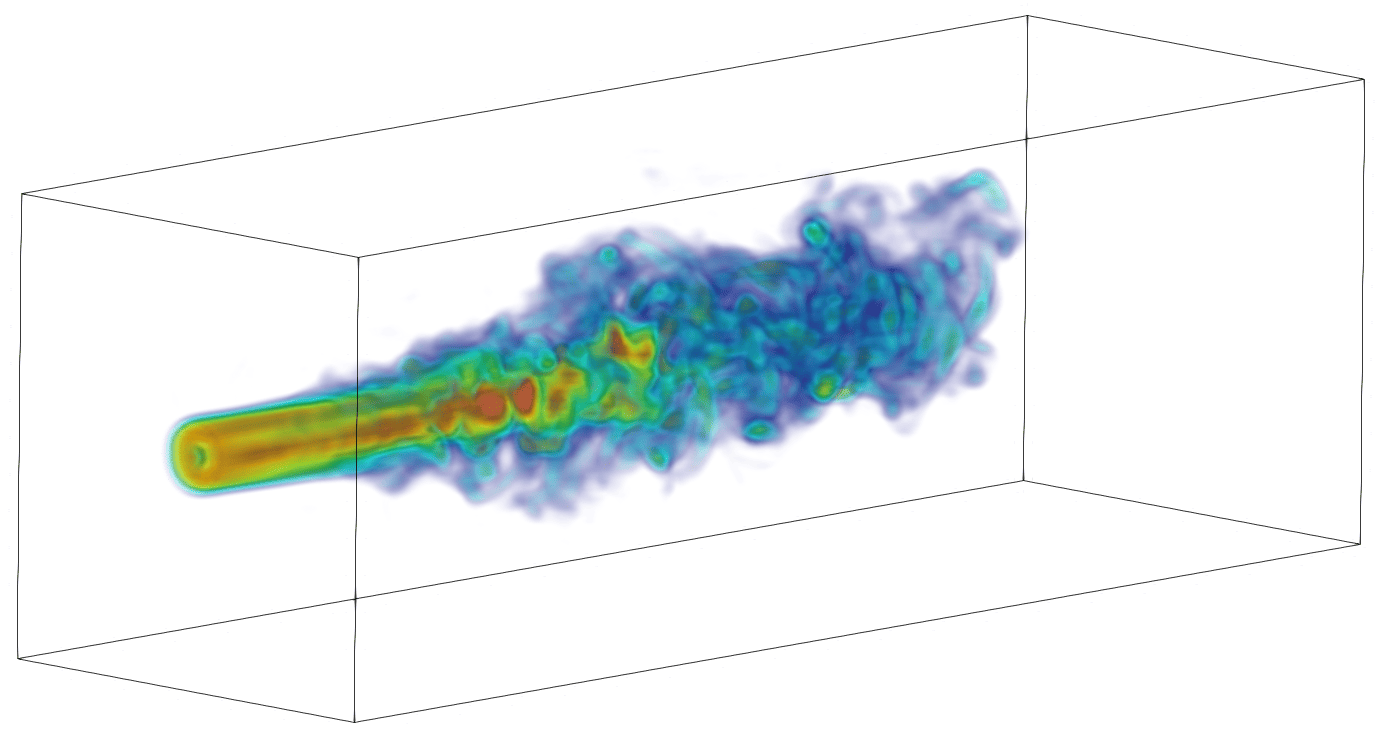}}
 \endminipage
 \endminipage

 \minipage{\linewidth}
\minipage{0.48\textwidth}
  {\includegraphics[width=\linewidth]{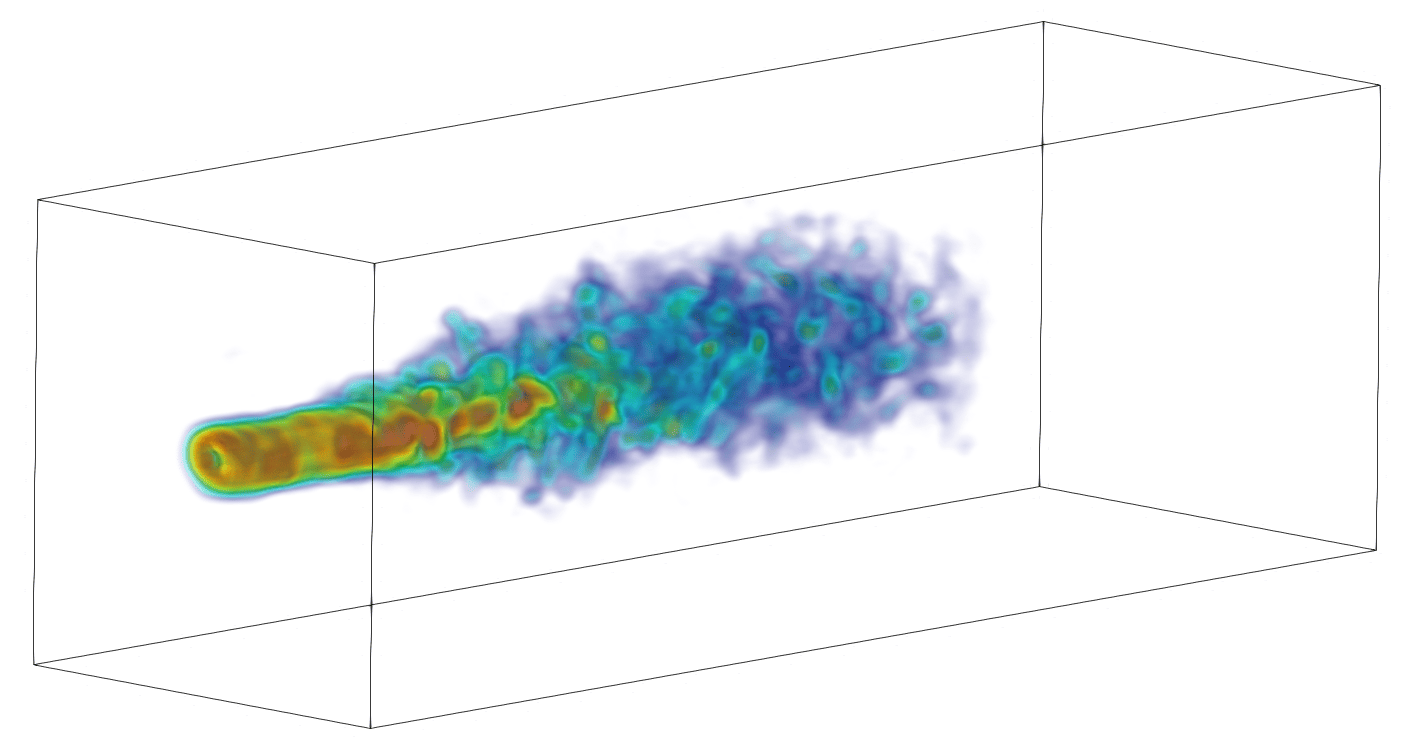}}
 \endminipage
\minipage{0.48\textwidth}
  {\includegraphics[width=\linewidth]{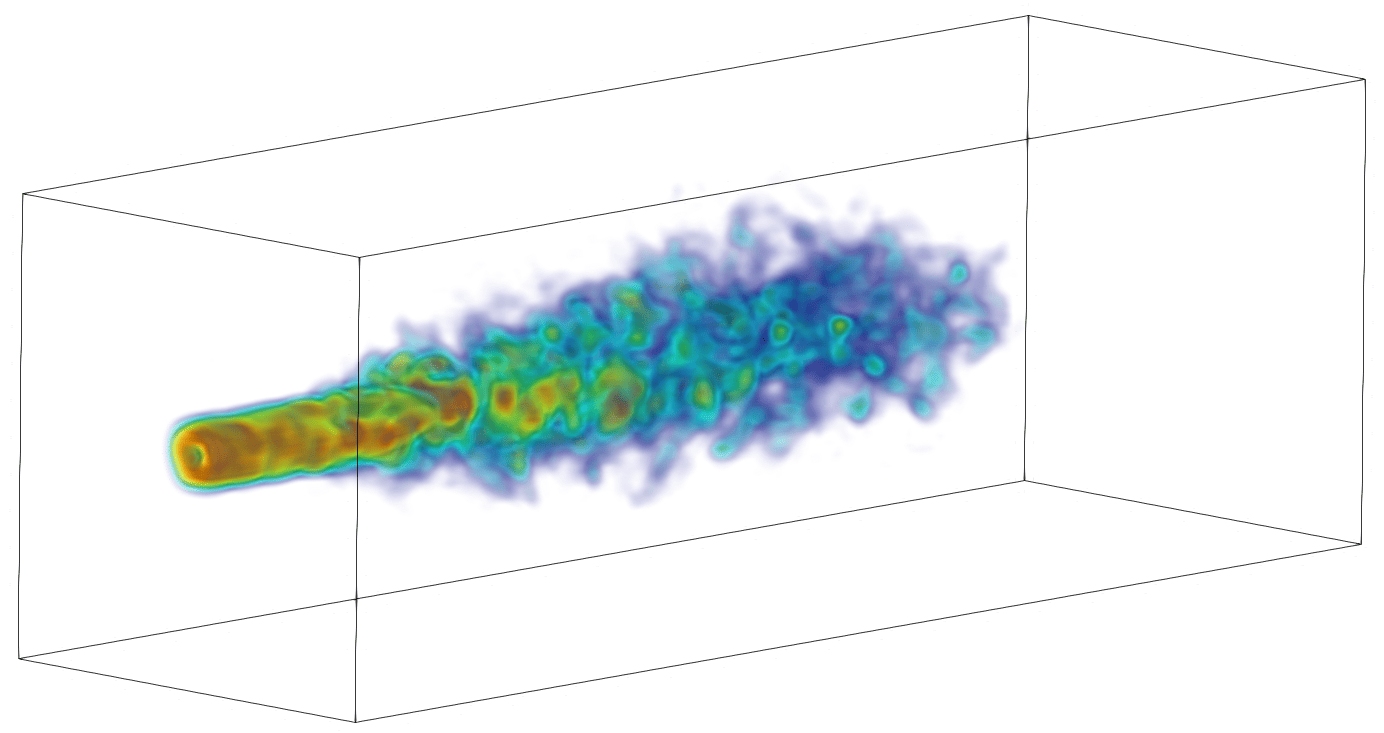}}
 \endminipage
 \endminipage

\minipage{\linewidth}
\minipage{0.48\textwidth}
  {\includegraphics[width=\linewidth]{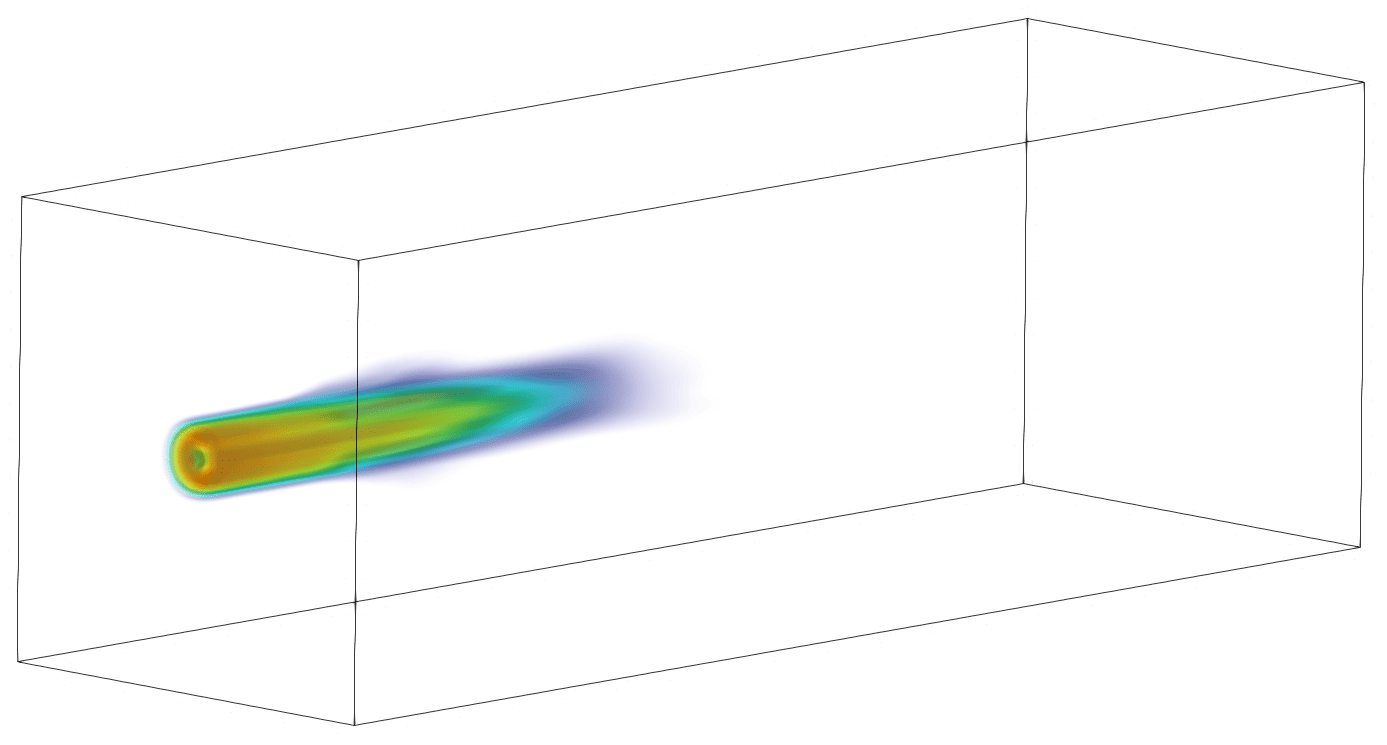}}
 \endminipage
\minipage{0.48\textwidth}
  {\includegraphics[width=\linewidth]{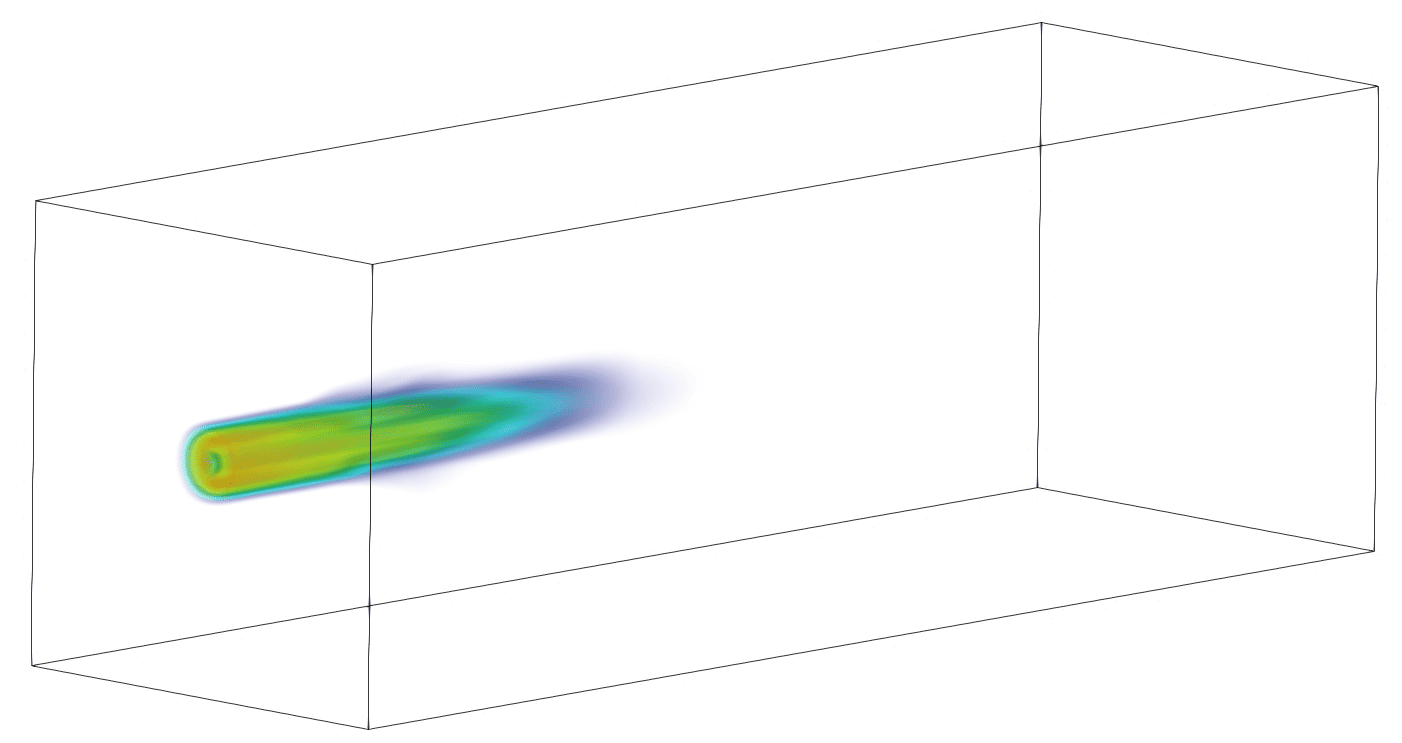}}
 \endminipage
 \endminipage

\caption{\textbf{Two random samples, for the same injection velocity, of the (pointwise) vorticity intensity at time $T=130$ for the nozzle flow experiment.} Data generated with ground truth (Top Row), GenCFD (Middle Row) and UViT (Bottom Row).}
\label{fig:nozz2}
\end{figure}
\clearpage
\newpage

\begin{figure}[h!]
\centering
\minipage{\linewidth}
\minipage{0.3\textwidth}
  {\includegraphics[width=\linewidth]{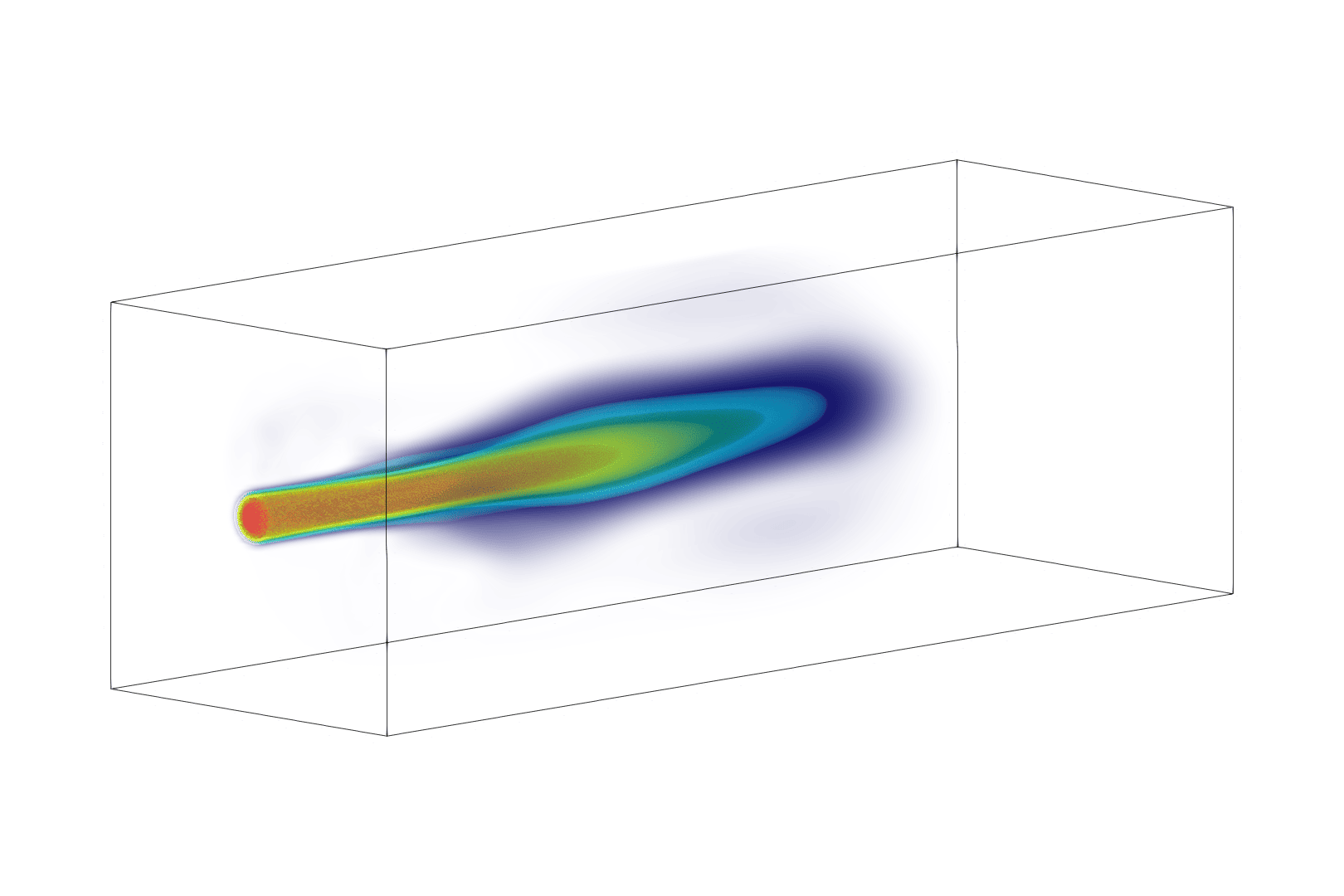}}
 \endminipage
\minipage{0.3\textwidth}
  {\includegraphics[width=\linewidth]{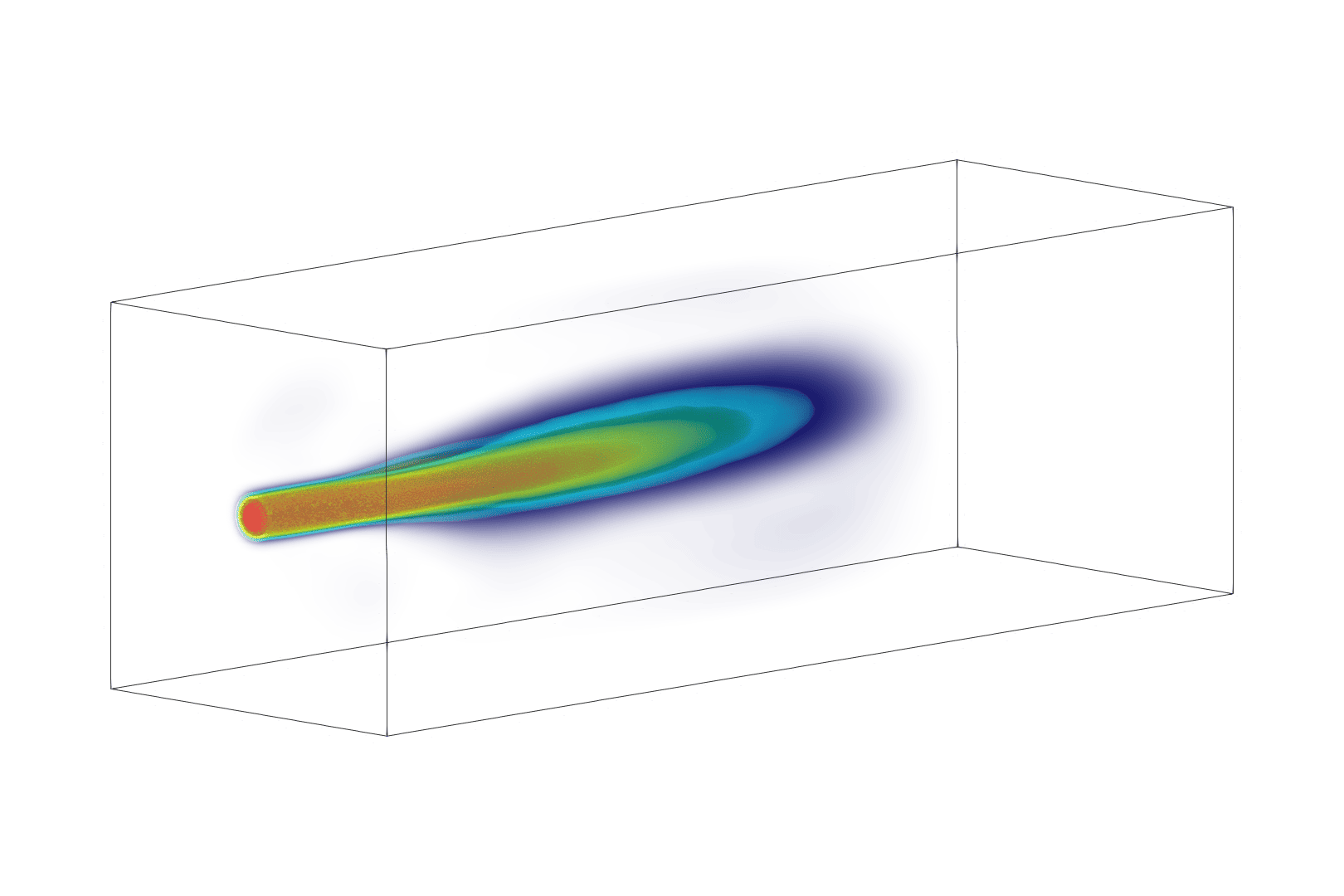}}
 \endminipage
\minipage{0.3\textwidth}
  {\includegraphics[width=\linewidth]{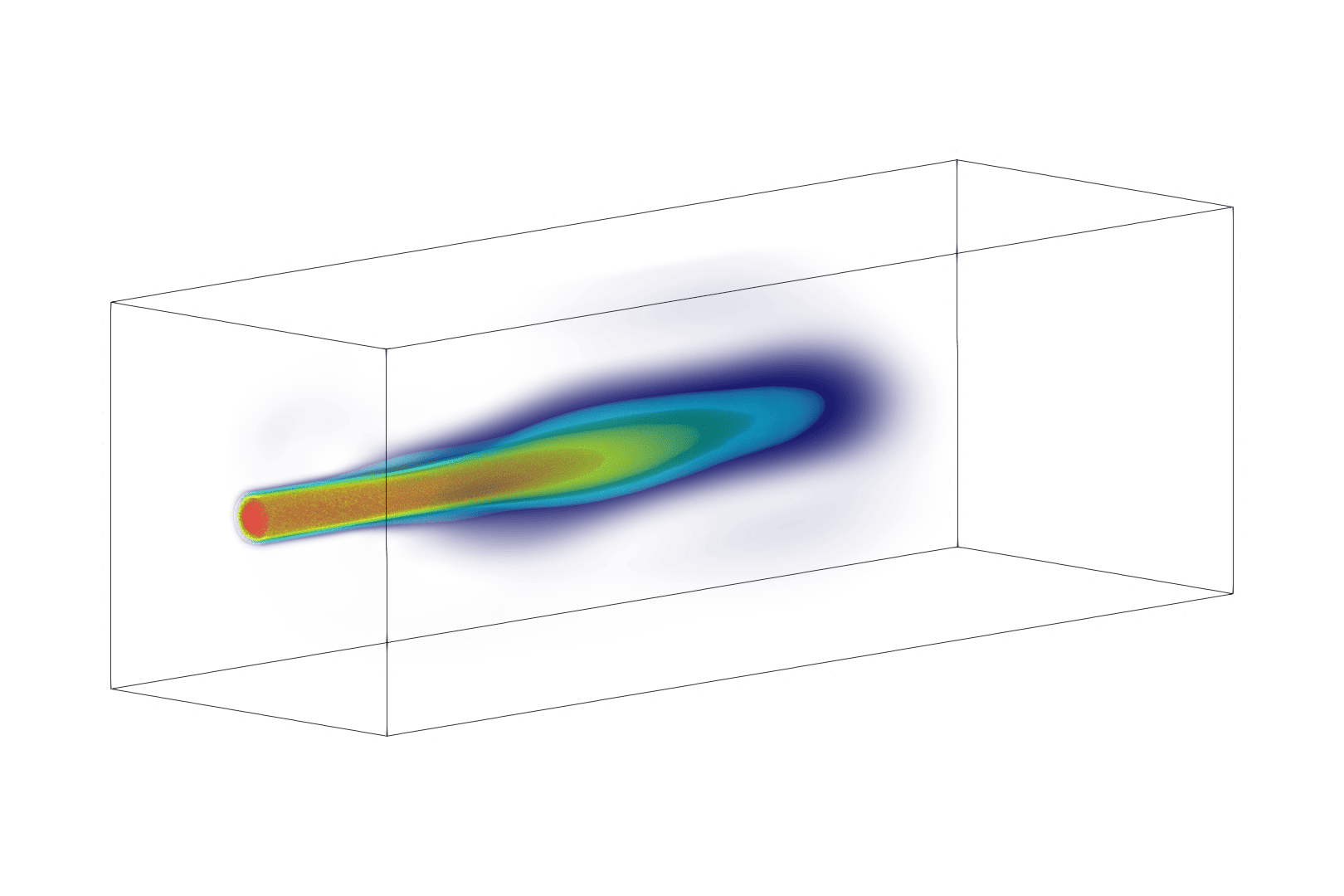}}
 \endminipage
\endminipage
\caption{\textbf{Mean of the pointwise kinetic energy at time $T=130$ for the nozzle flow experiment.} Ground truth (Left), GenCFD (Center) and UViT. (Right).}
\label{fig:nozz3}
\end{figure}

\begin{figure}[h!]
\centering
\minipage{\linewidth}
\minipage{0.3\textwidth}
  {\includegraphics[width=\linewidth]{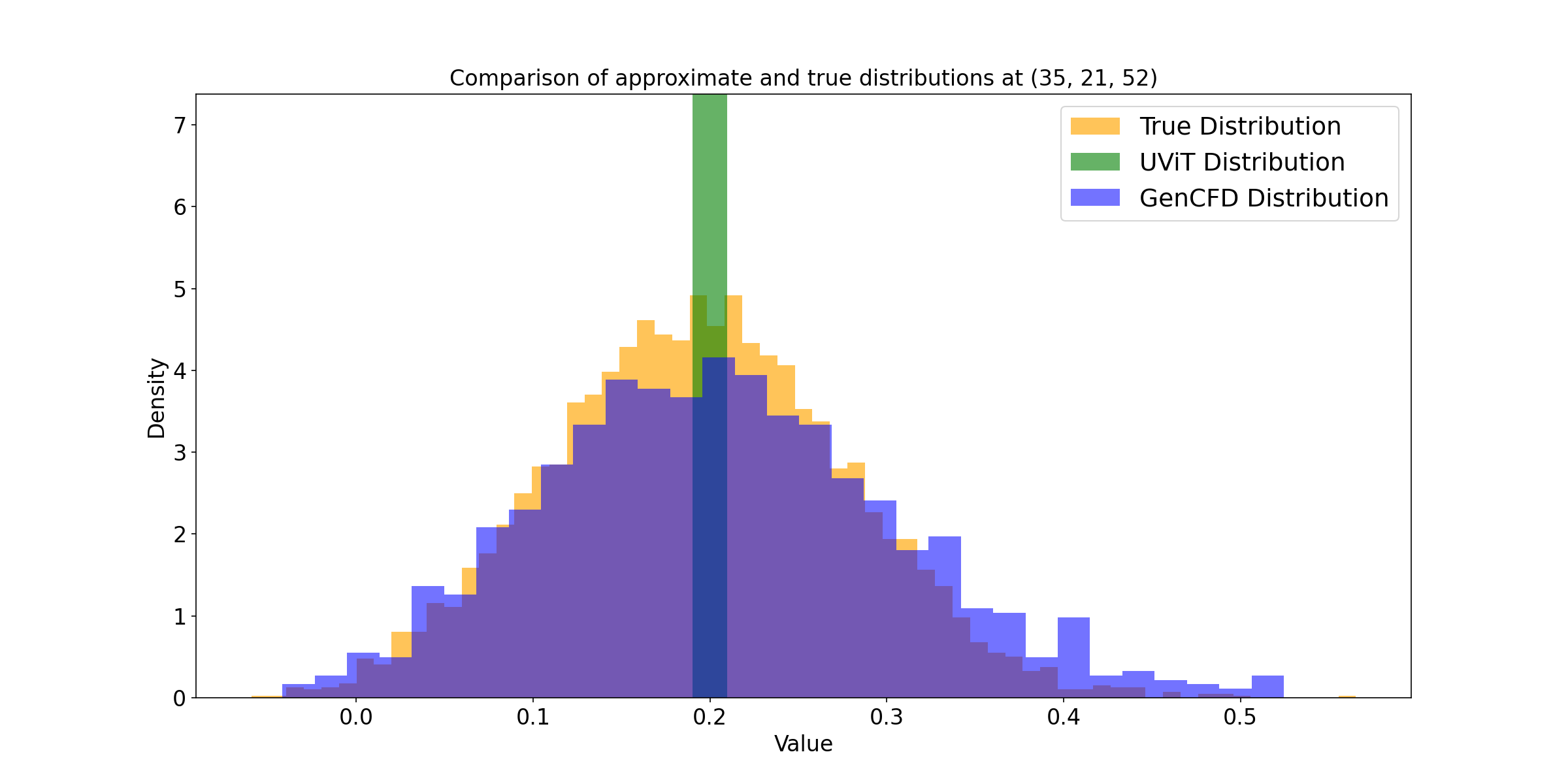}}
 \endminipage
\minipage{0.3\textwidth}
  {\includegraphics[width=\linewidth]{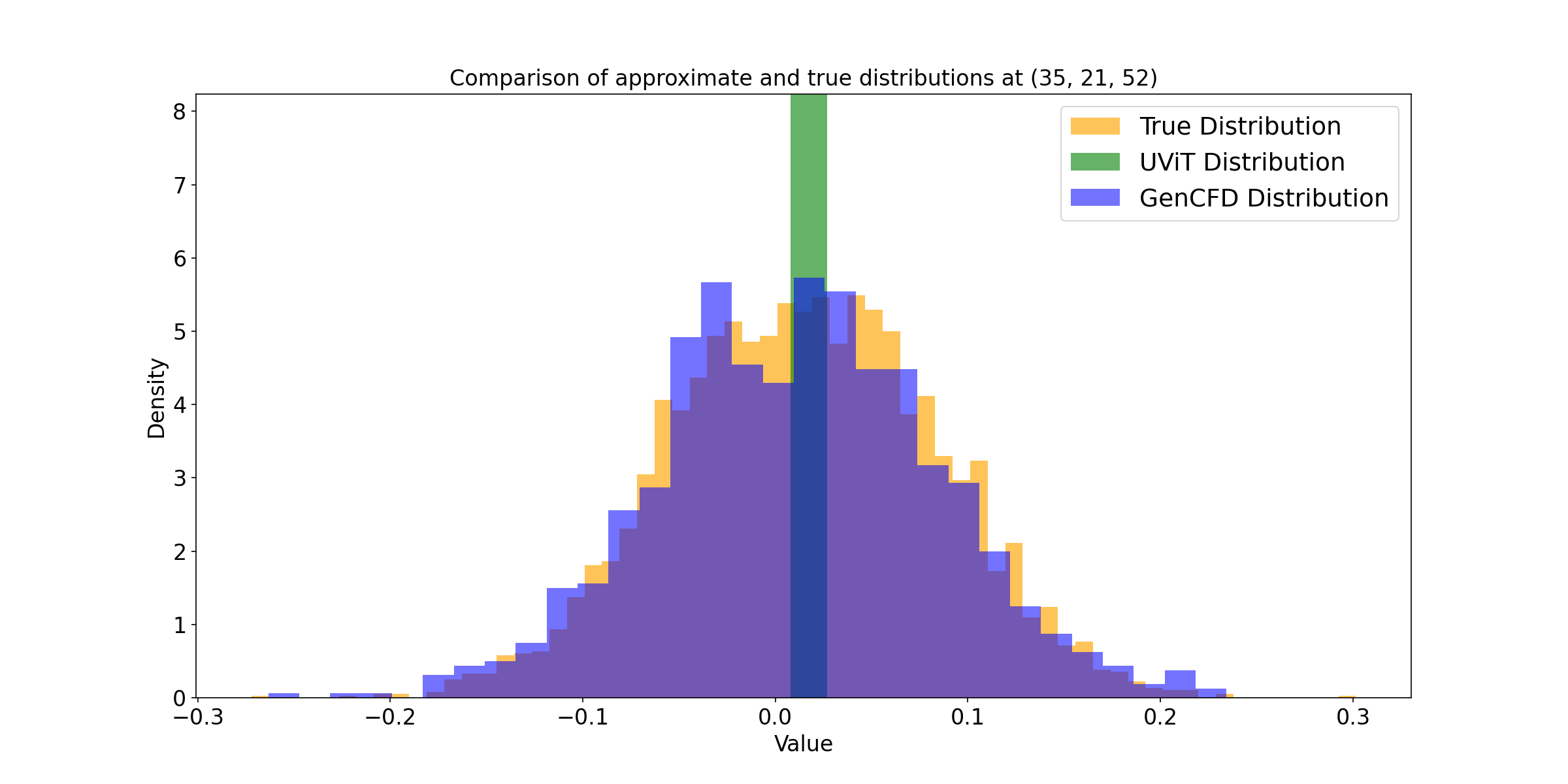}}
 \endminipage
\minipage{0.3\textwidth}
  {\includegraphics[width=\linewidth]{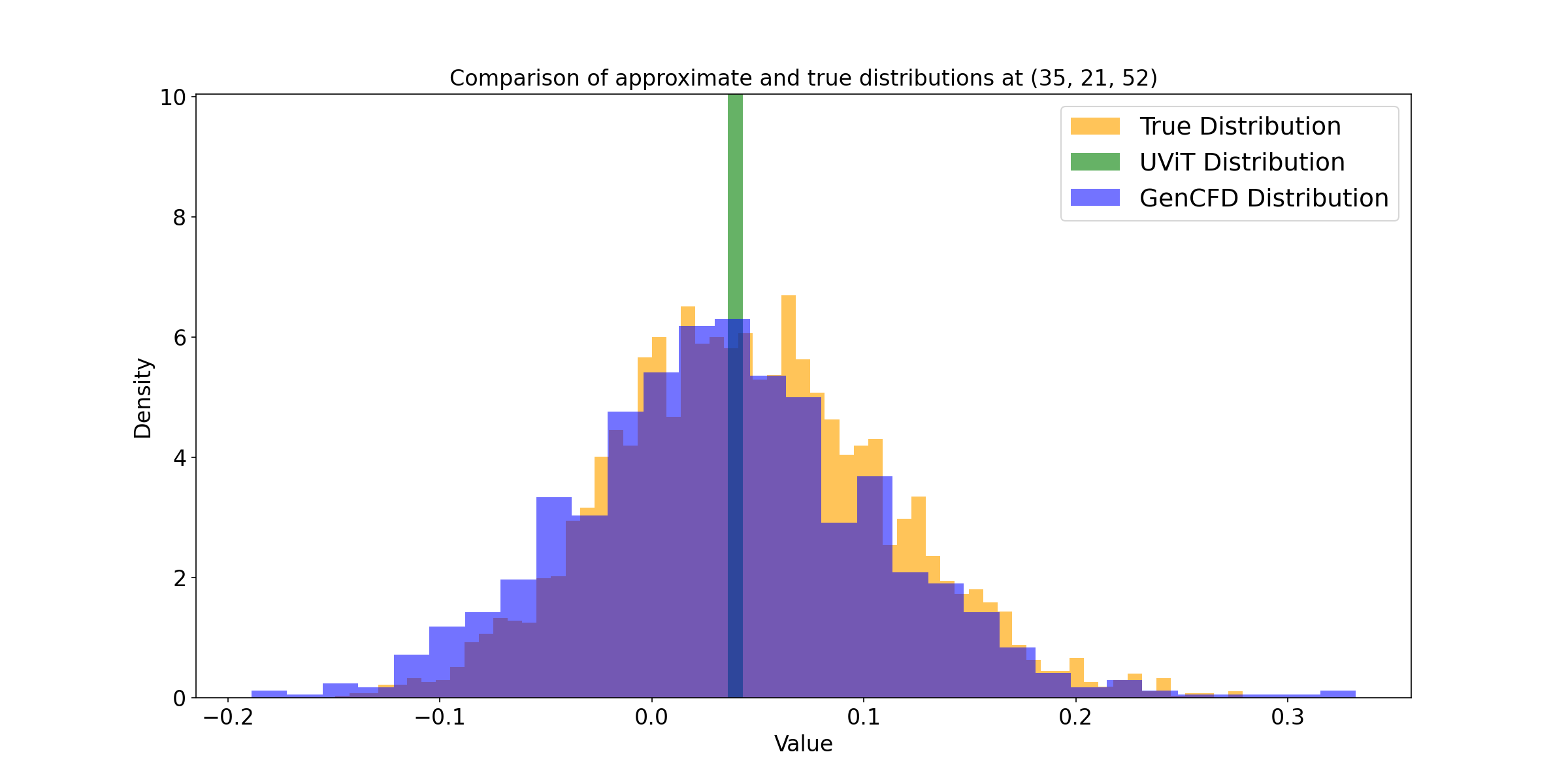}}
 \endminipage
\endminipage
\caption{\textbf{Point PDFs of the three velocity components at the spatial point $(0.547, 0.328, 0.271)$ at time $T=130$ for the nozzle flow experiment for Ground truth, GenCFD and UViT.}}
\label{fig:nozz4}
\end{figure}

\clearpage
\newpage
\begin{figure}[h!]
\centering
\minipage{\linewidth}
\centering
\minipage{0.35\textwidth}
  {\includegraphics[width=\linewidth]{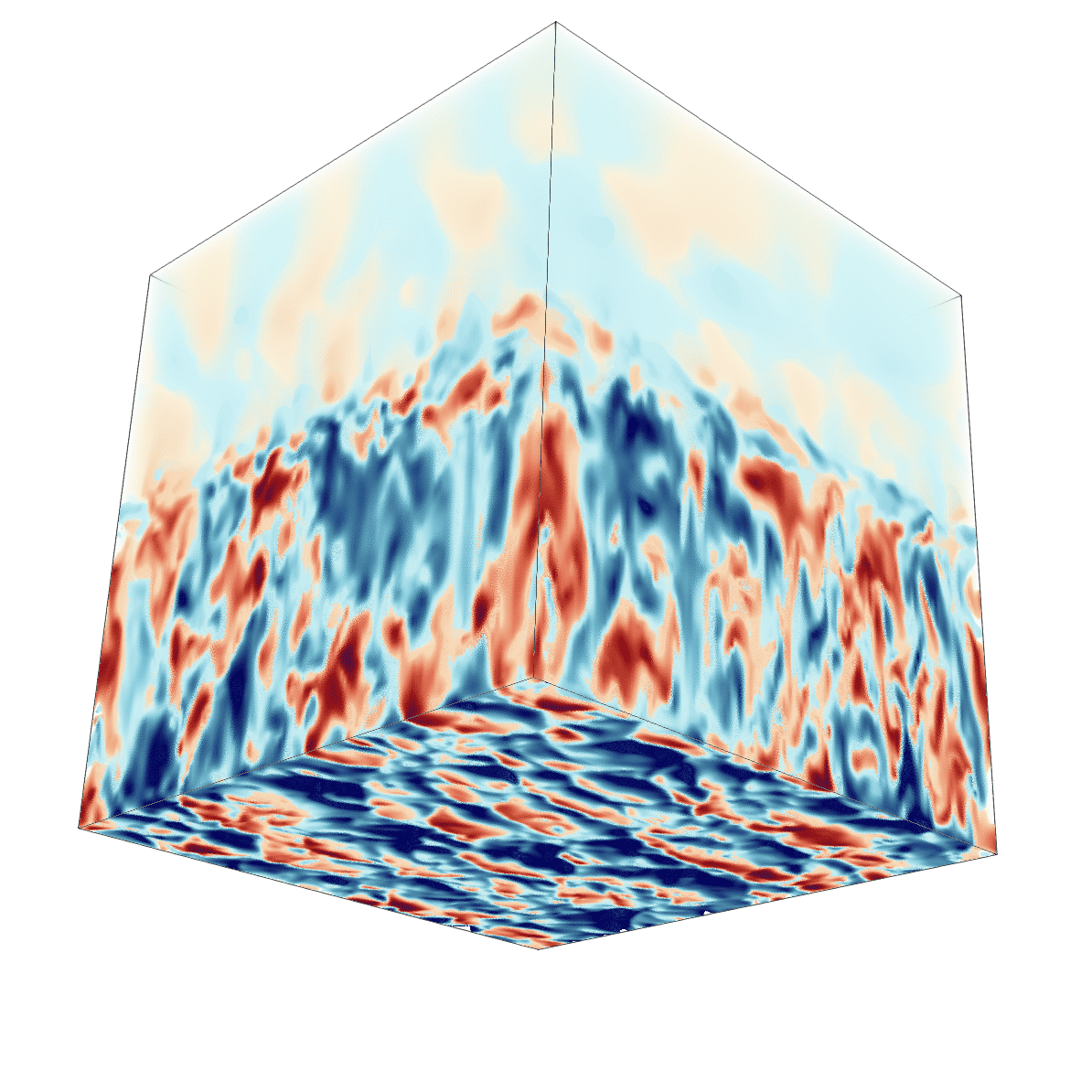}}
 \endminipage
\minipage{0.35\textwidth}
  {\includegraphics[width=\linewidth]{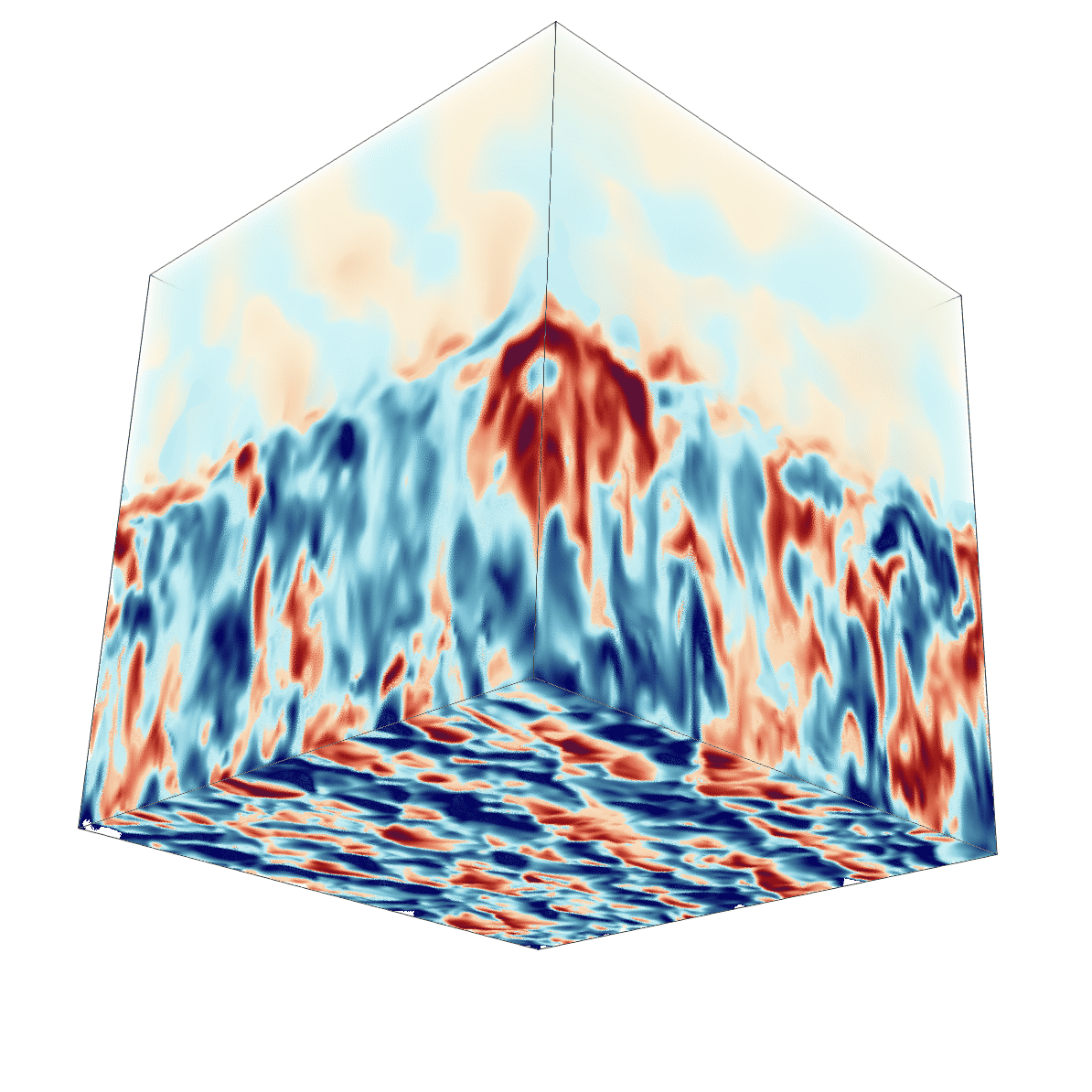}}
 \endminipage
 \endminipage

 \minipage{\linewidth}
 \centering
\minipage{0.35\textwidth}
  {\includegraphics[width=\linewidth]{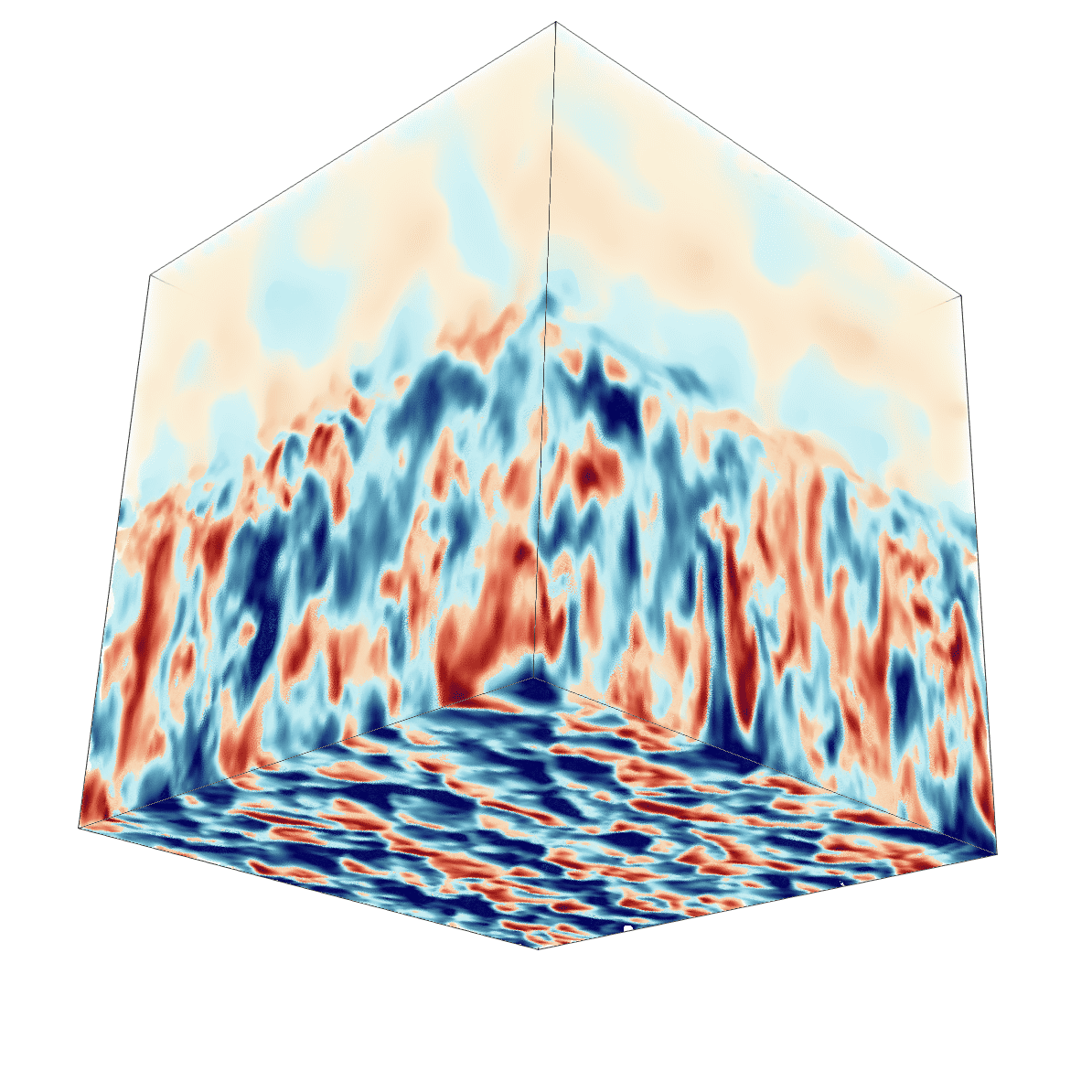}}
 \endminipage
\minipage{0.35\textwidth}
  {\includegraphics[width=\linewidth]{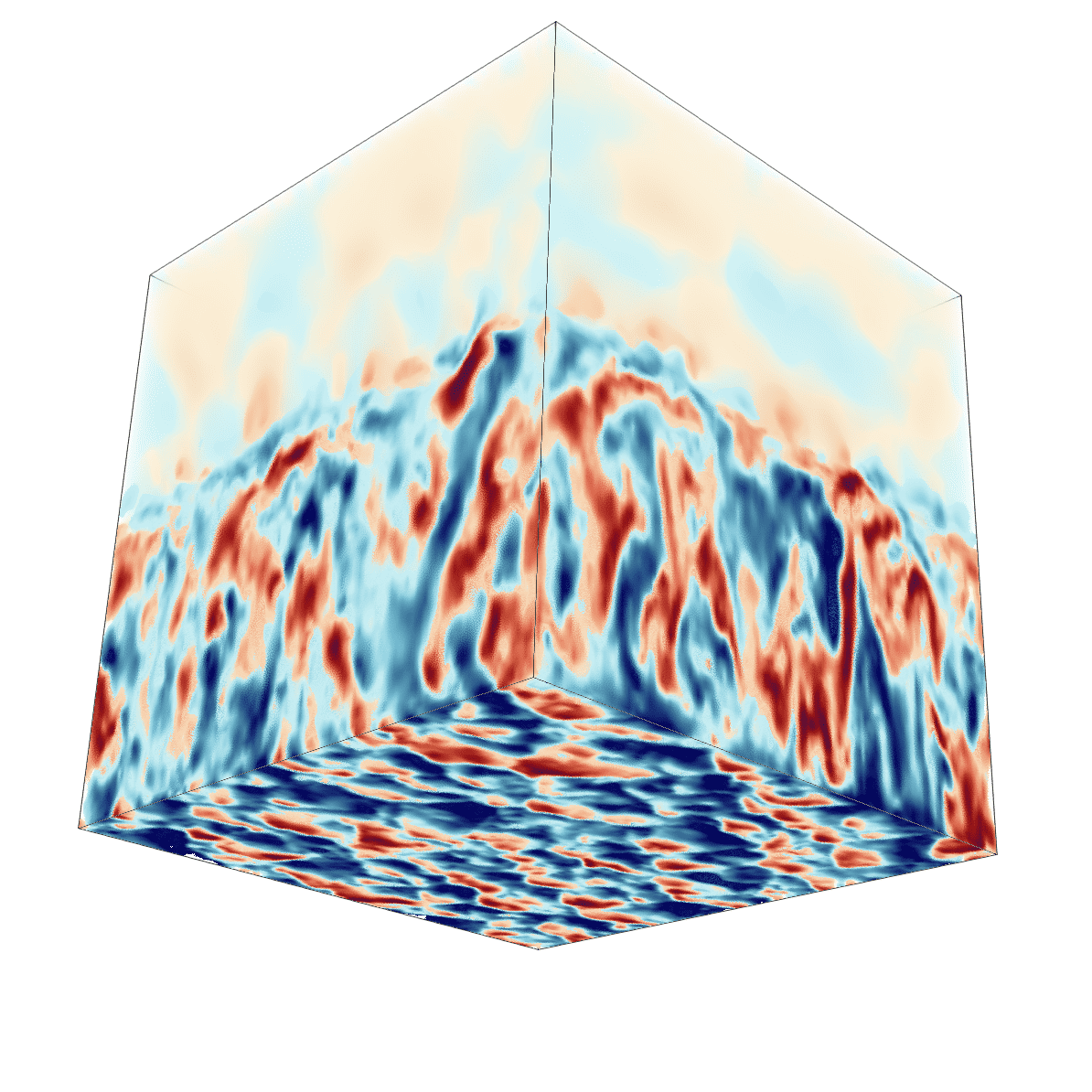}}
 \endminipage
 \endminipage

\minipage{\linewidth}
\centering
\minipage{0.35\textwidth}
  {\includegraphics[width=\linewidth]{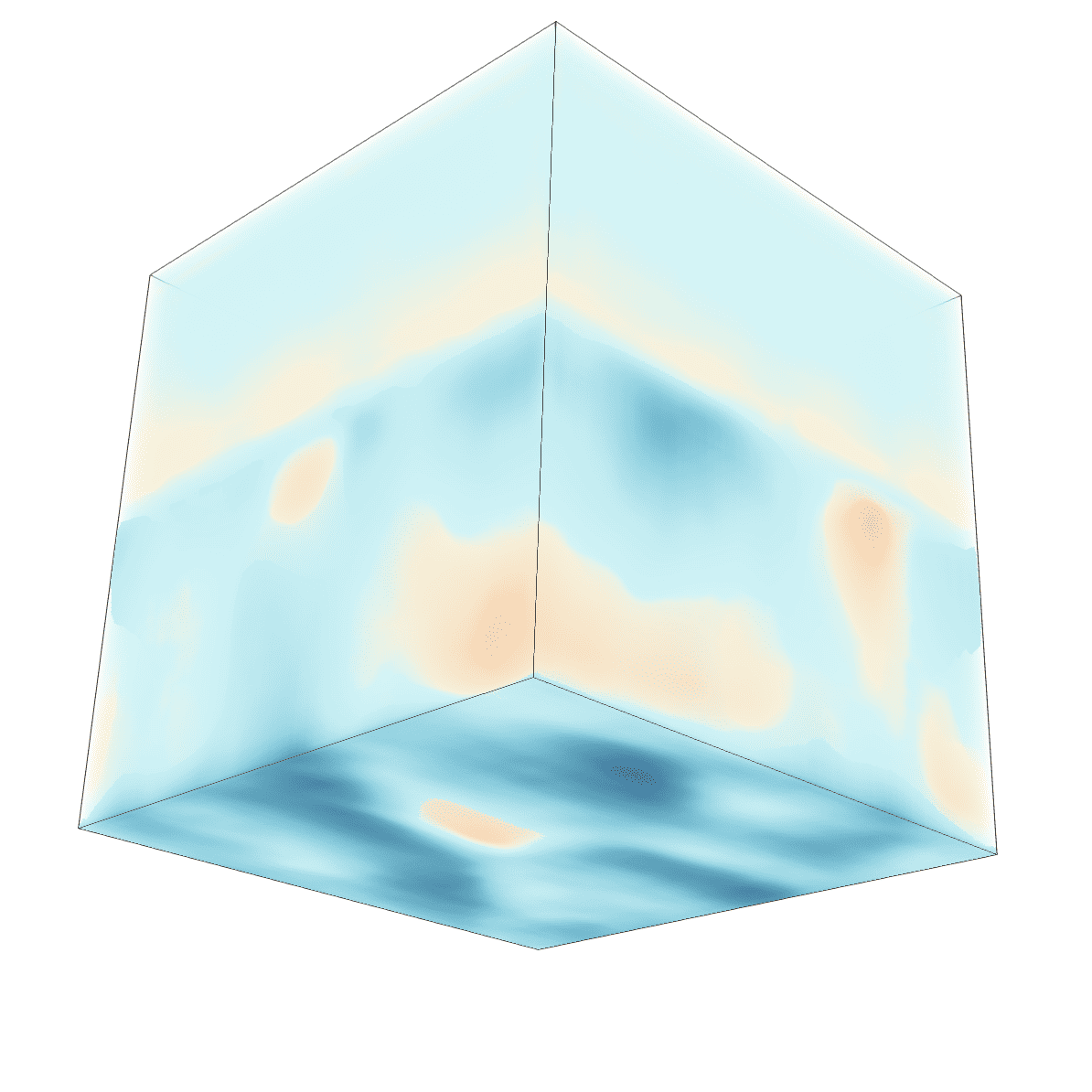}}
 \endminipage
\minipage{0.35\textwidth}
  {\includegraphics[width=\linewidth]{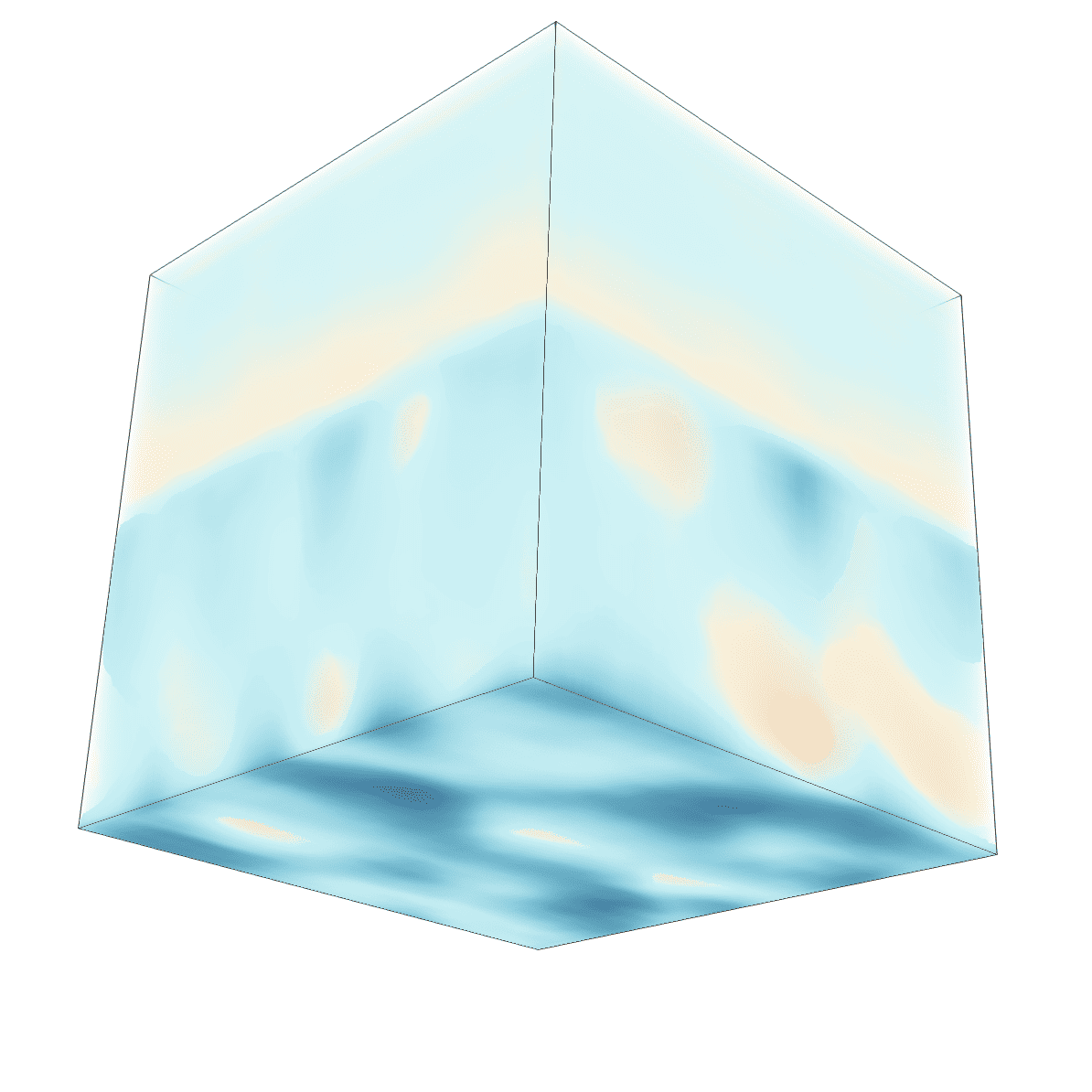}}
 \endminipage
 \endminipage
\caption{\textbf{Two random samples,  of the velocity component $u_x$ at time $T=2.4$ for the convective boundary layer experiment with ground truth (Top Row), GenCFD (Middle Row) and UViT (Bottom Row).}}
\label{fig:cbl1}
\end{figure}

\begin{figure}[h!]
\centering
\minipage{\linewidth}
\minipage{0.3\textwidth}
  {\includegraphics[width=\linewidth]{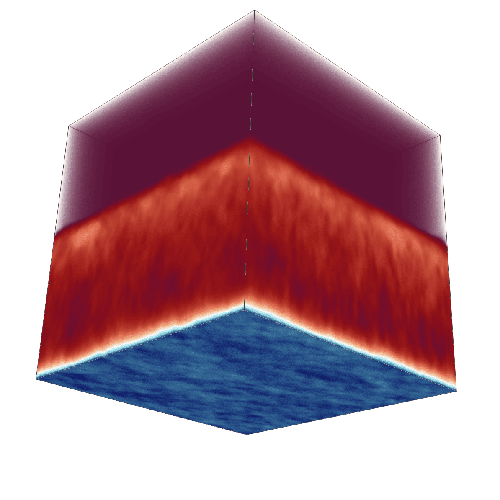}}
 \endminipage
\minipage{0.3\textwidth}
  {\includegraphics[width=\linewidth]{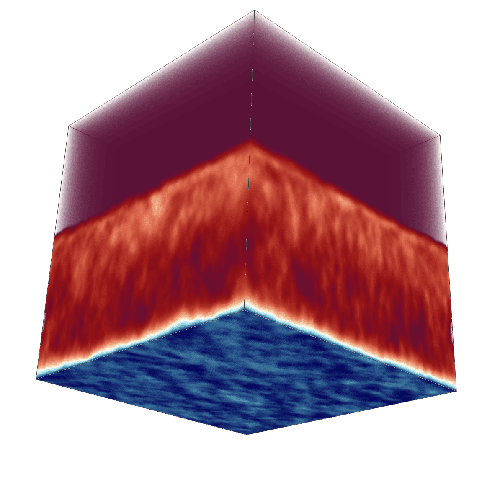}}
 \endminipage
\minipage{0.3\textwidth}
  {\includegraphics[width=\linewidth]{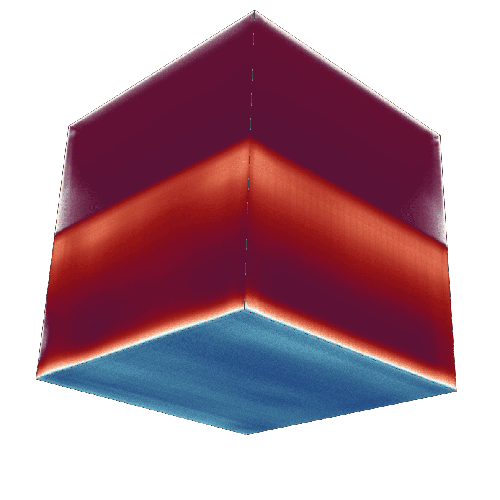}}
 \endminipage
\endminipage
\caption{\textbf{Mean of the velocity component $u_x$ at time $T=2.4$ for the convective boundary layer experiment with ground truth (Left), GenCFD (Center) and UViT (Right).}}
\label{fig:cbl2}
\end{figure}

\begin{figure}[h!]
\centering
\minipage{\linewidth}
\minipage{0.23\textwidth}
  {\includegraphics[width=\linewidth]{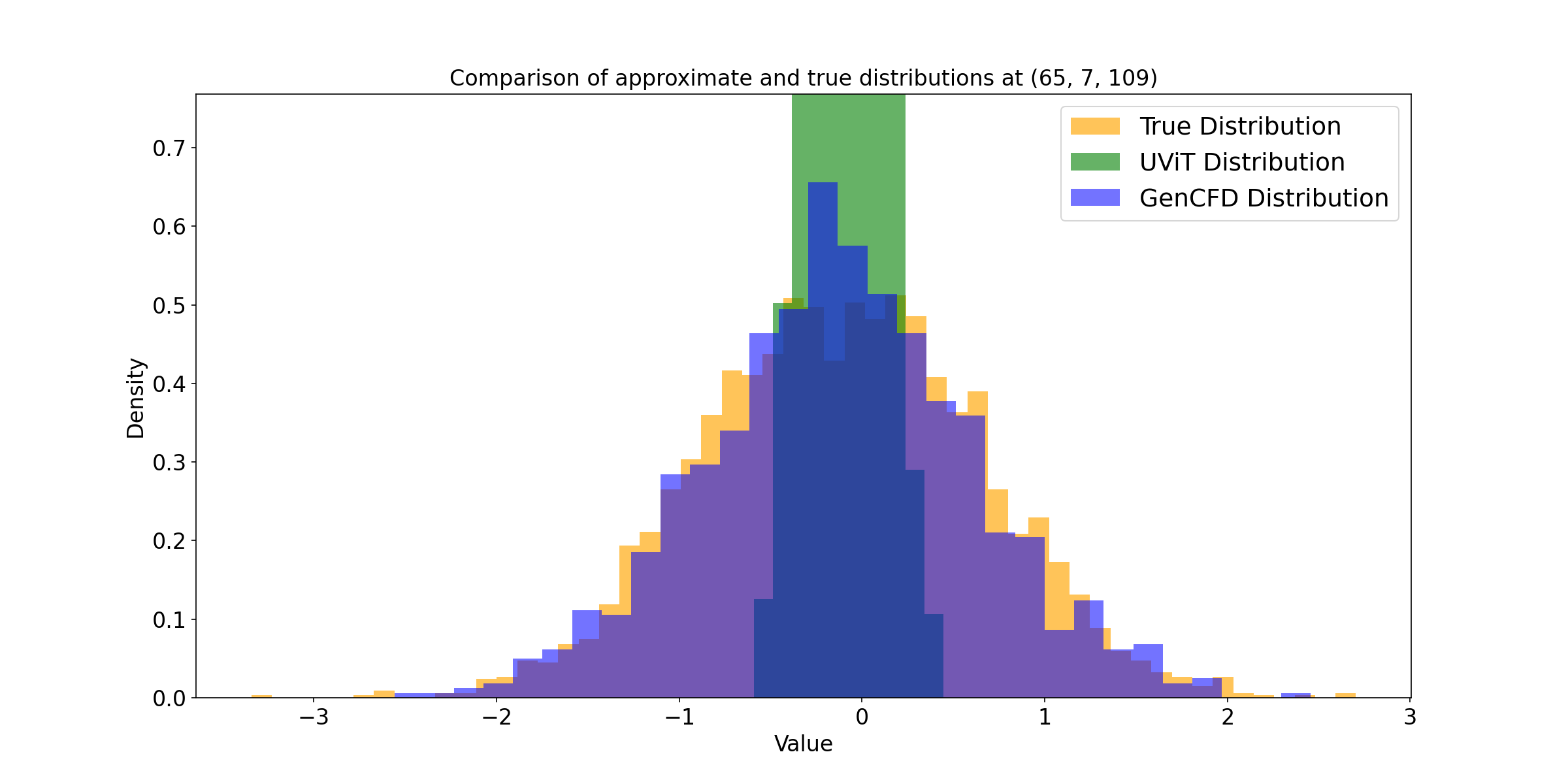}}
 \endminipage
\minipage{0.23\textwidth}
  {\includegraphics[width=\linewidth]{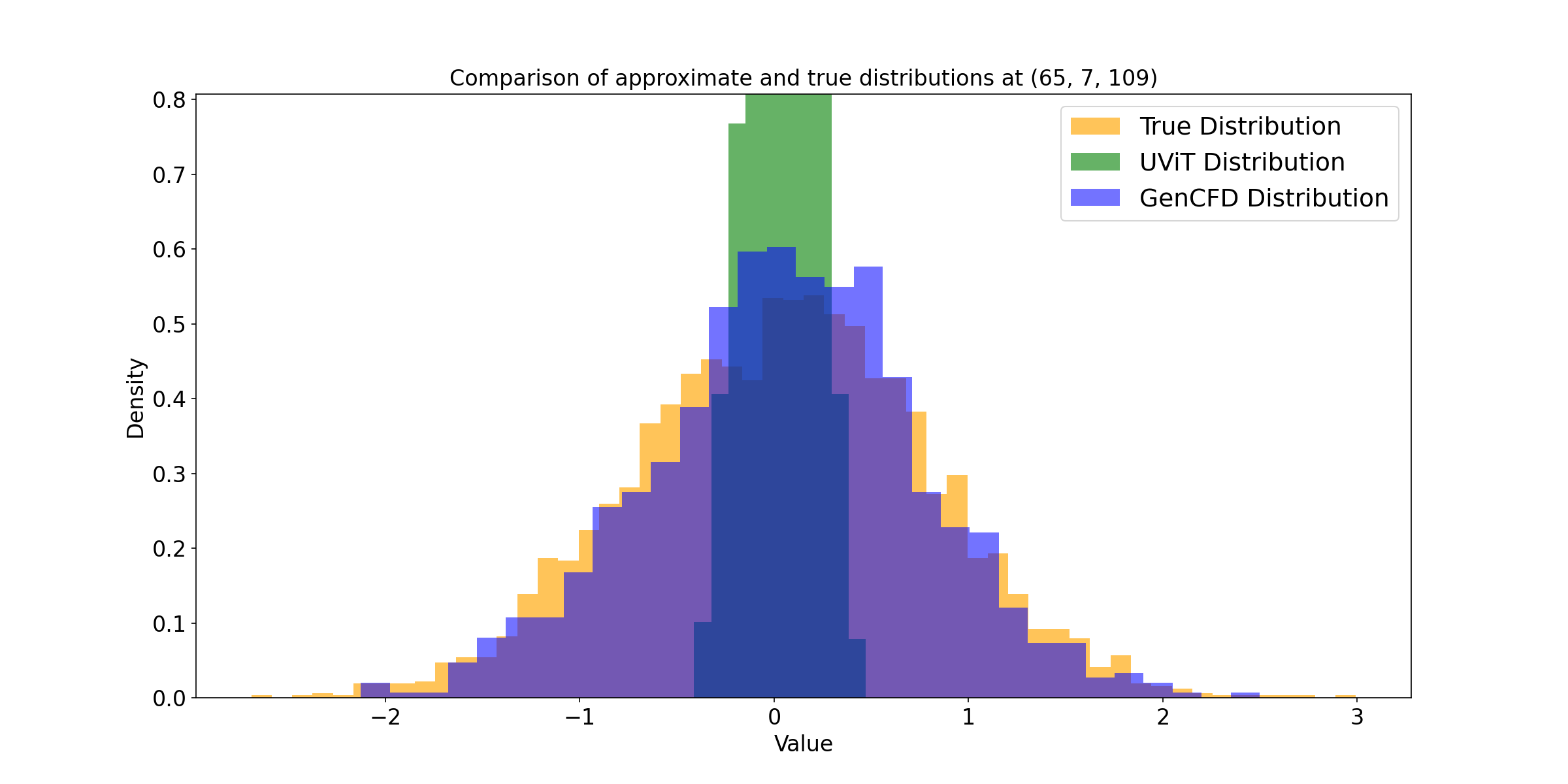}}
 \endminipage
\minipage{0.23\textwidth}
  {\includegraphics[width=\linewidth]{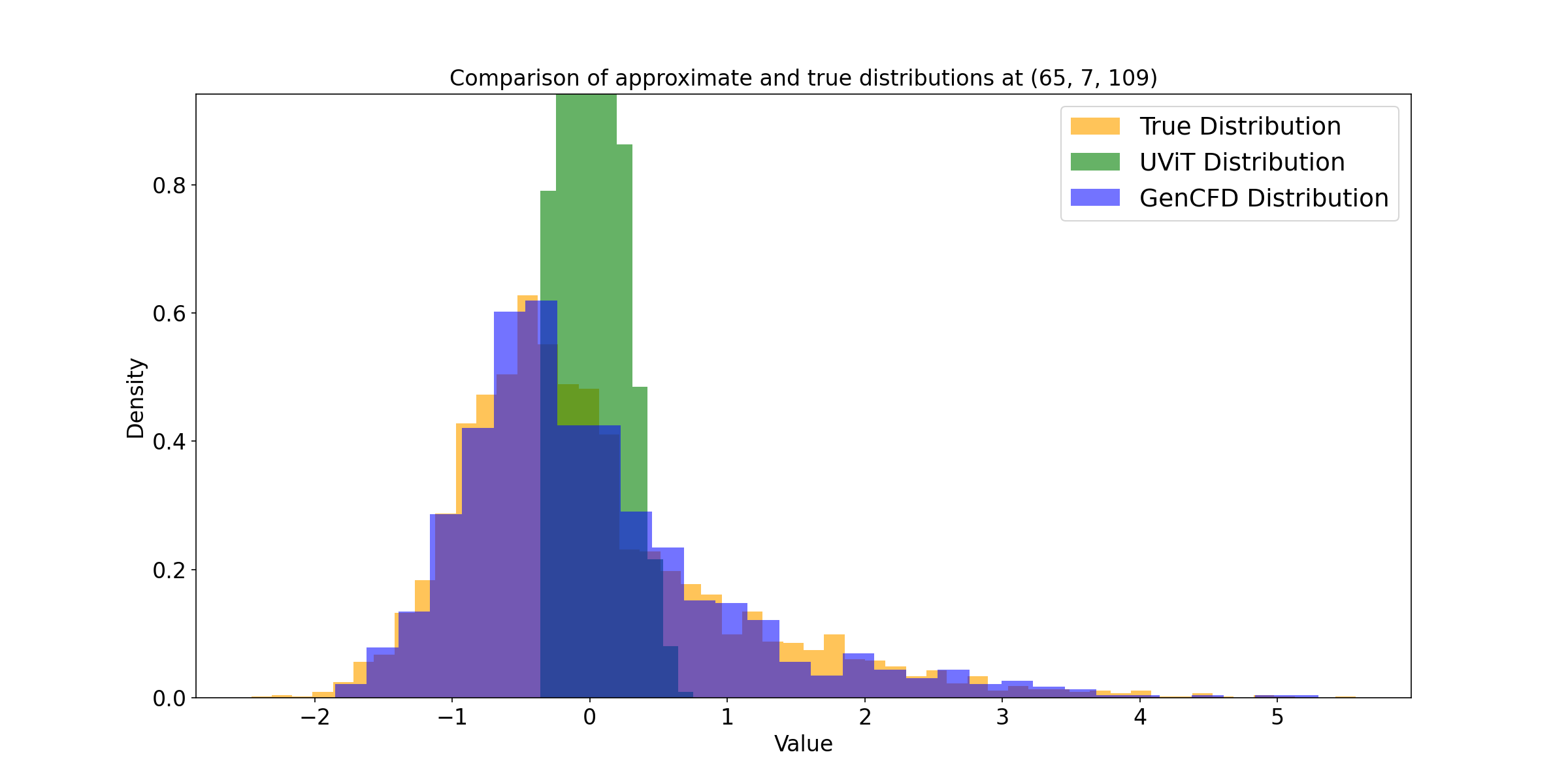}}
 \endminipage
\minipage{0.23\textwidth}
  {\includegraphics[width=\linewidth]{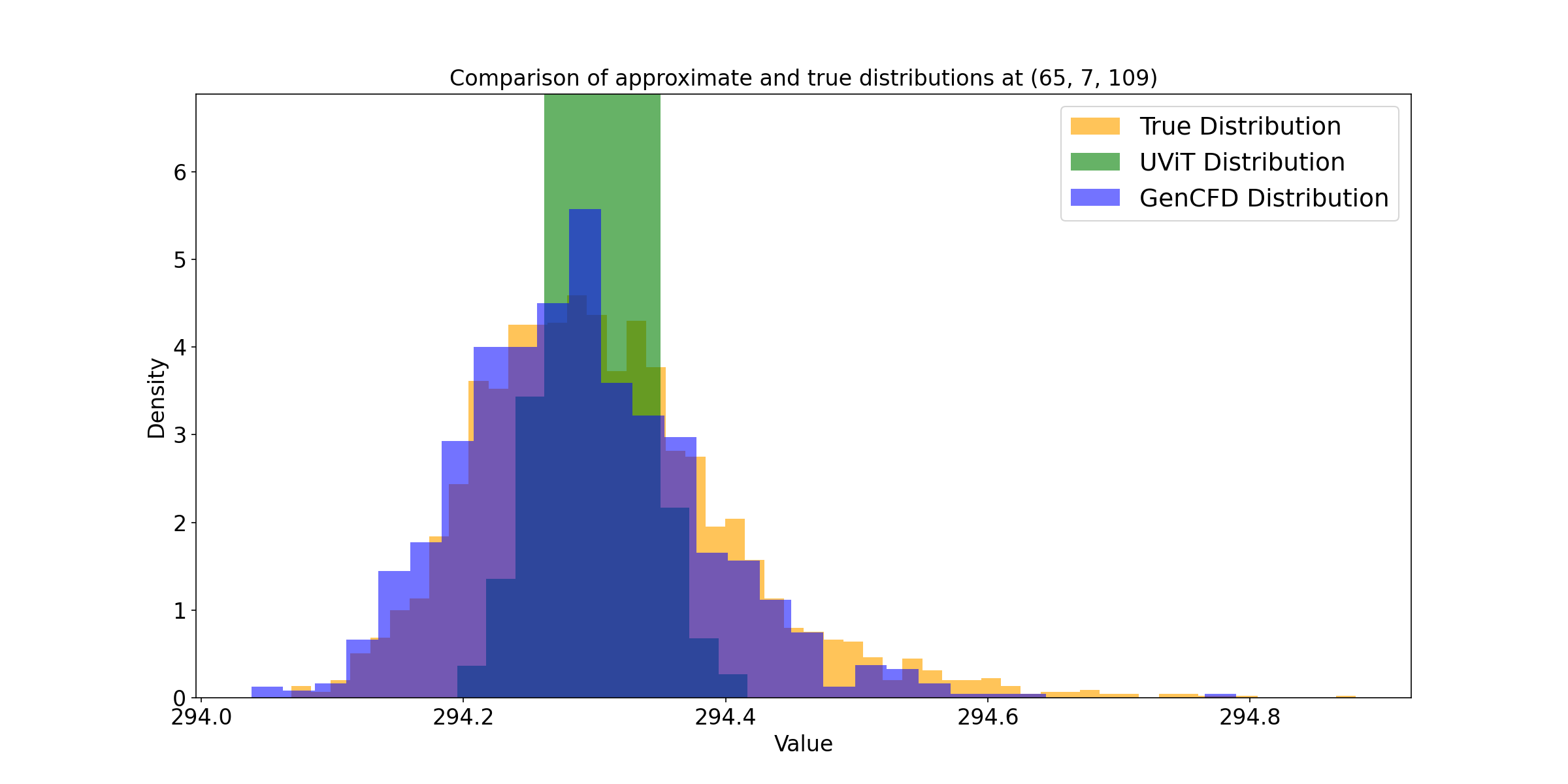}}
 \endminipage
 
\endminipage
\caption{\textbf{Point PDFs of the temperature (Left) and the three velocity components at the spatial point $(0.508, 0.055, 0.852)$ at time $T=2.4$ for the convective boundary layer experiment for ground truth, GenCFD and UViT.}}
\label{fig:cbl3}
\end{figure}

\begin{figure}[h!]
	\centering
    \includegraphics[height=4cm,trim=0 0 850 20, clip]{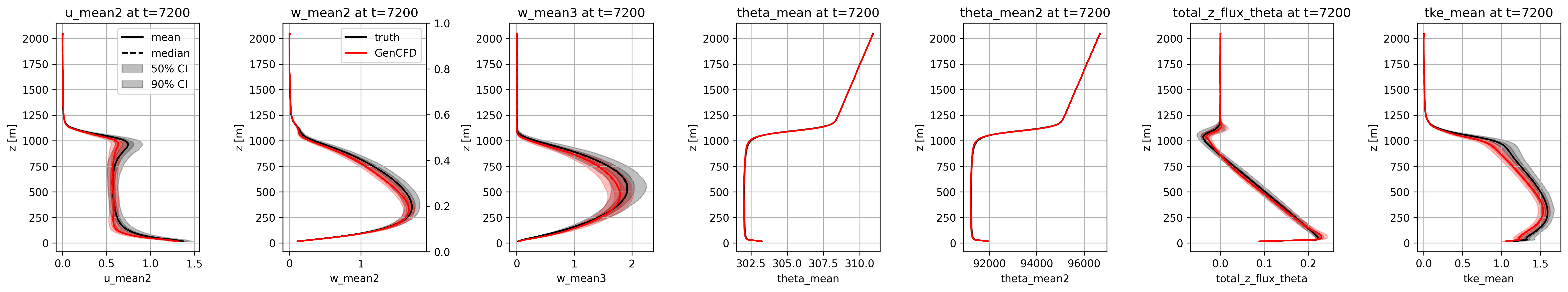} 
    \includegraphics[height=4cm,trim=850 0 0 20, clip]{Figures/Figs_CBL/MeanStats/fig-profiles_ensemble_Nens1000_compare_dirac0_gen.png}
    
    \caption{\textbf{Profiles (horizontal statistics) for the convective boundary layer experiment for GenCFD and ground truth at final simulation time $T=7200~\text{s}$.} From left to right and top to bottom: $x_1$-velocity variance, vertical velocity variance, vertical velocity skewness, potential temperature variance, vertical flux of potential temperature, and turbulent kinetic energy.}
    \label{fig:cbl4}
\end{figure}

\clearpage
\newpage

\begin{figure}[!t]
\minipage{\linewidth}
\minipage{0.33\textwidth}
  \includegraphics[width=\linewidth]{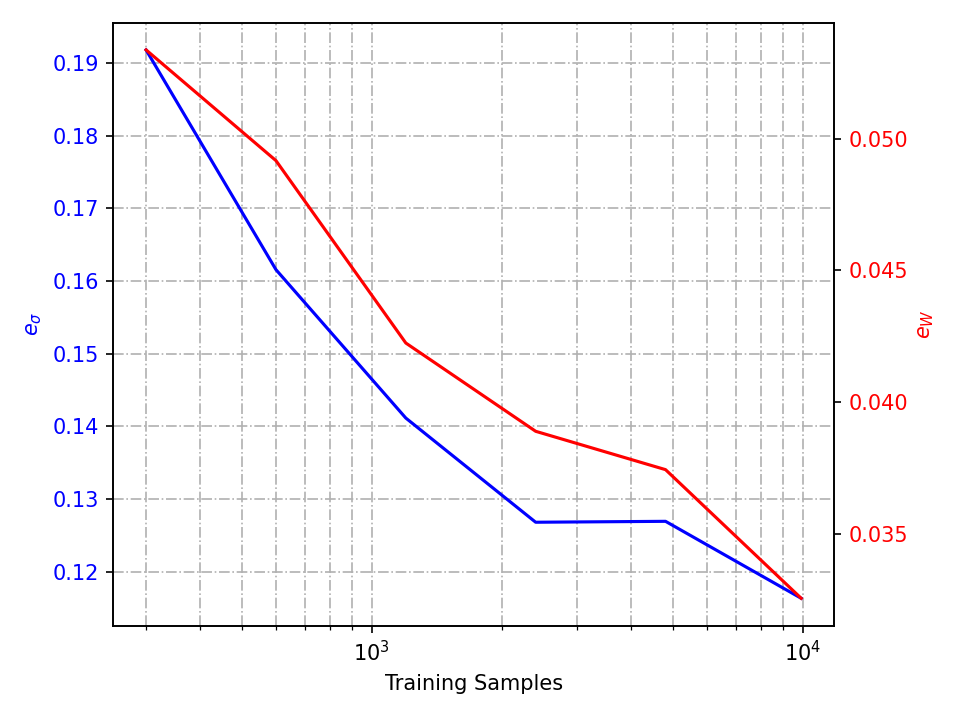}
  \subcaption{$u_x$}
\endminipage
\minipage{0.33\textwidth}
  {\includegraphics[width=\linewidth]{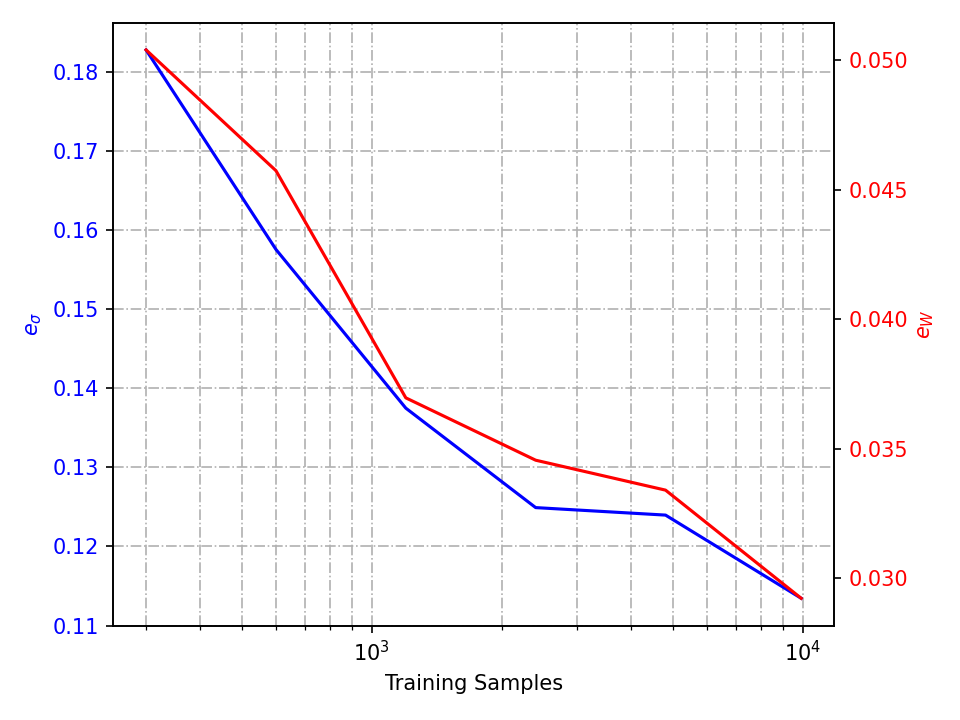}}
  \subcaption{$u_y$}
\endminipage
\minipage{0.33\textwidth}
  {\includegraphics[width=\linewidth]{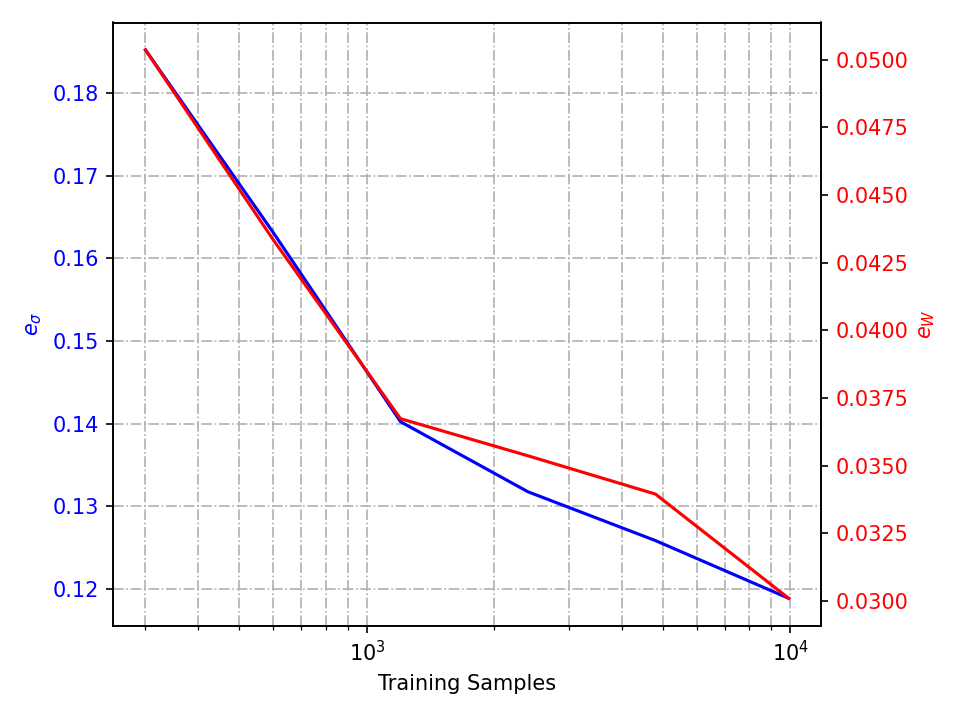}}
  \subcaption{$u_z$}
\endminipage
\endminipage
\caption{\textbf{Scaling of standard deviation and Wasserstein metric error, $e_\sigma$ and $e_W$, due to GenCFD, vs.\ number of training samples for the velocities $u_x$, $u_y$ and $u_z$ for the cylindrical shear flow.}}
\label{fig:8}
\end{figure}

\clearpage
\newpage

\begin{figure}[!t]
\centering
\minipage{\linewidth}
\minipage{0.45\textwidth}
  {\includegraphics[width=\linewidth]{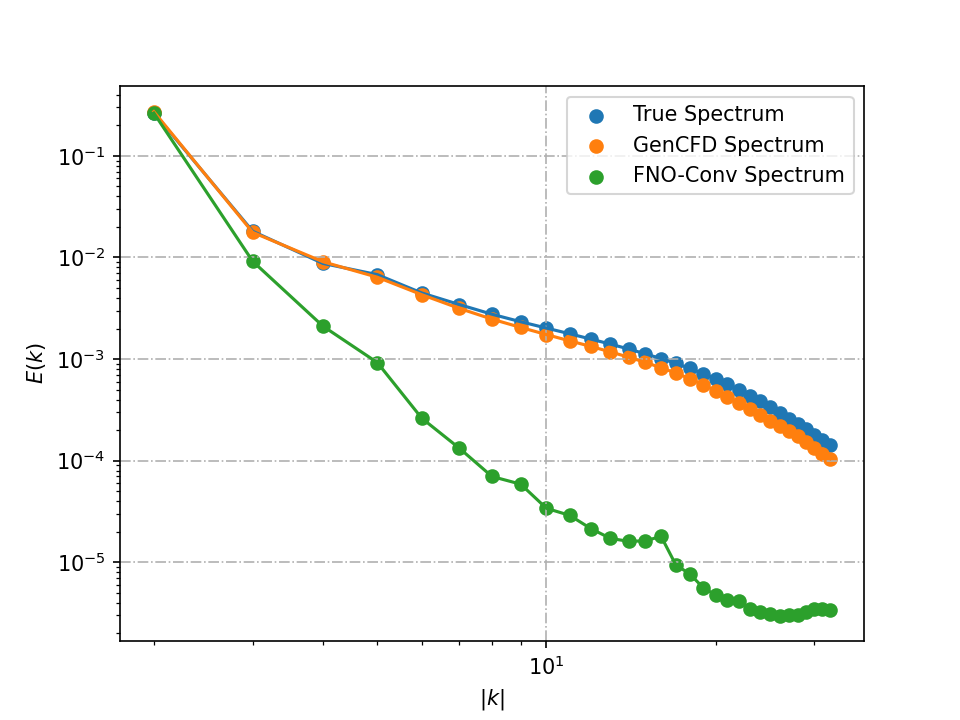}}
  \subcaption{Cylindrical shear flow}
\endminipage
\minipage{0.45\textwidth}
  {\includegraphics[width=\linewidth]{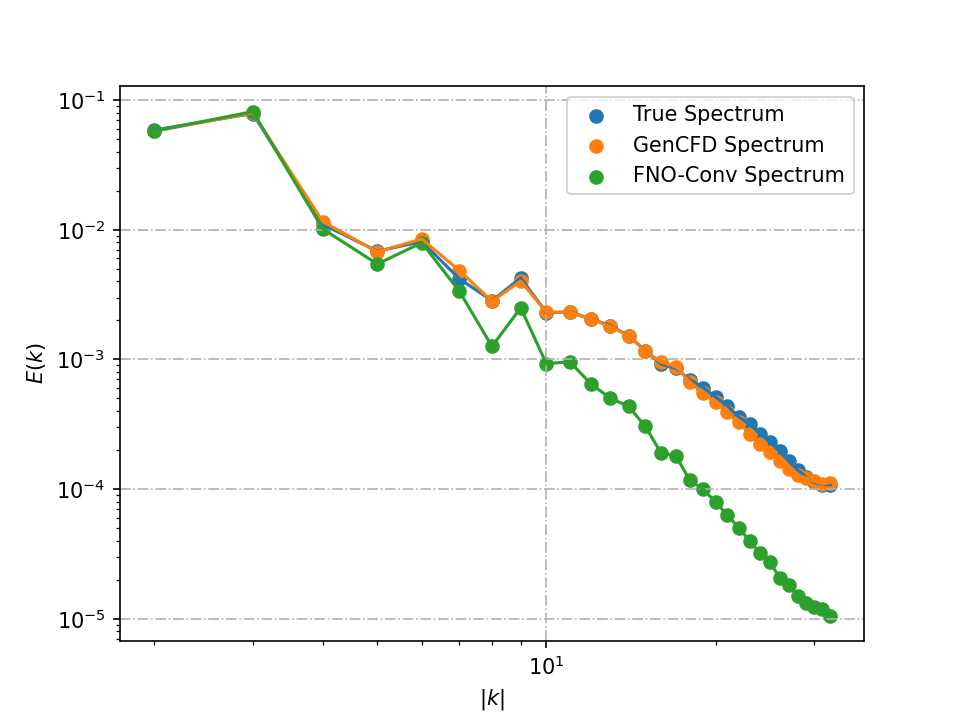}}
  \subcaption{Cloud-shock interaction}
\endminipage
\endminipage

\minipage{\linewidth}
\minipage{0.45\textwidth}
  {\includegraphics[width=\linewidth]{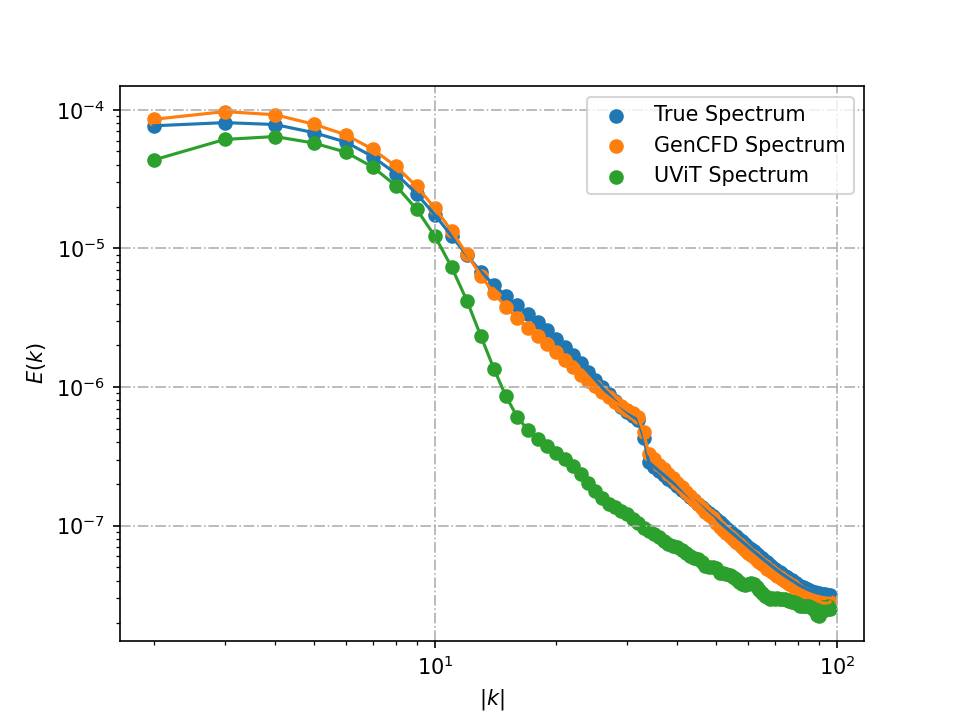}}
  \subcaption{Nozzle flow}
\endminipage
\minipage{0.45\textwidth}
  {\includegraphics[width=\linewidth]{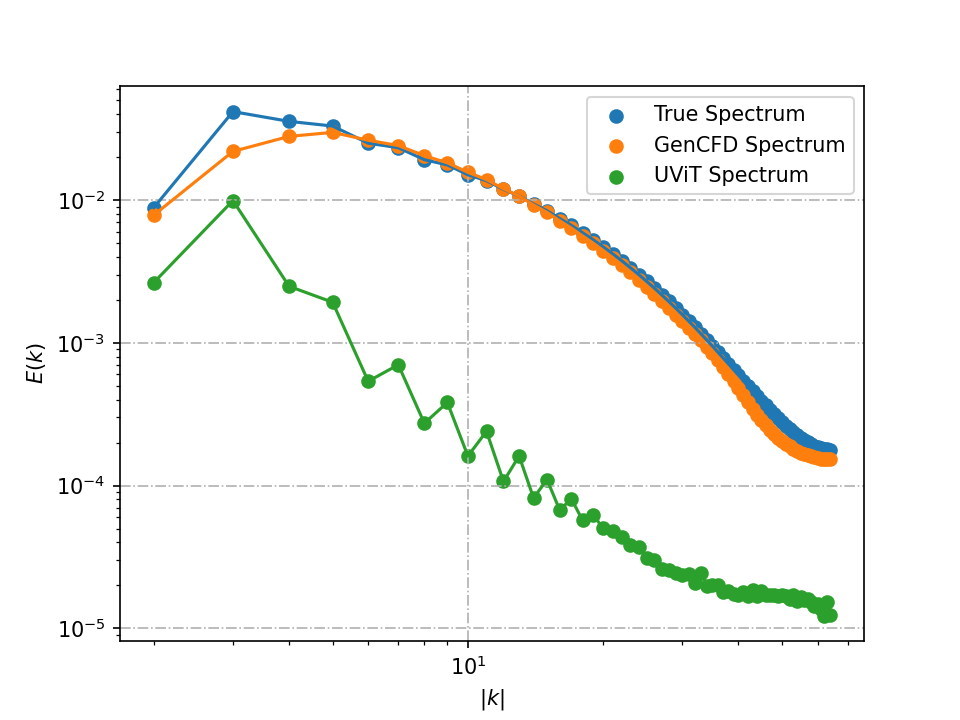}}
  \subcaption{Convective boundary layer}
\endminipage
\endminipage

\minipage{\linewidth}
\minipage{0.45\textwidth}
  {\includegraphics[width=\linewidth]{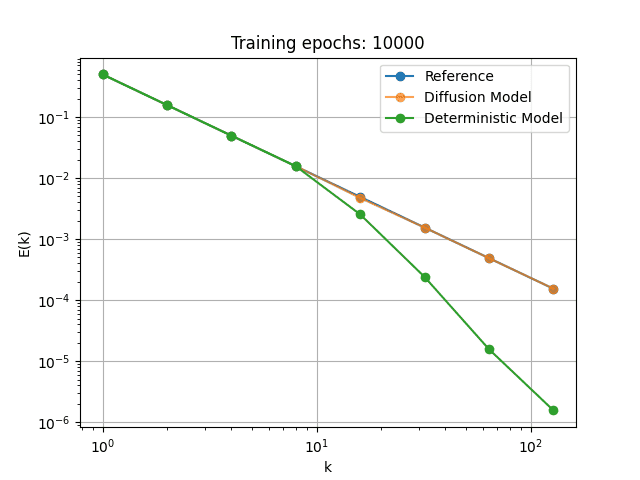}}
  \subcaption{Toy Problem 2}
\endminipage
\endminipage
\caption{\textbf{Energy spectra, generated by the ground truth, GenCFD and the best-performing baseline for 4 of the 3D Datasets and the spectrum of Toy Model $\# 2$.} Note that the spectrum shown here is the spectrum of the density for the cloud-shock interaction experiment. }
\label{fig:7}
\end{figure}

\begin{figure}[!t]
\minipage{\linewidth}
\minipage{0.33\textwidth}
\includegraphics[width=\linewidth, clip, draft=false]{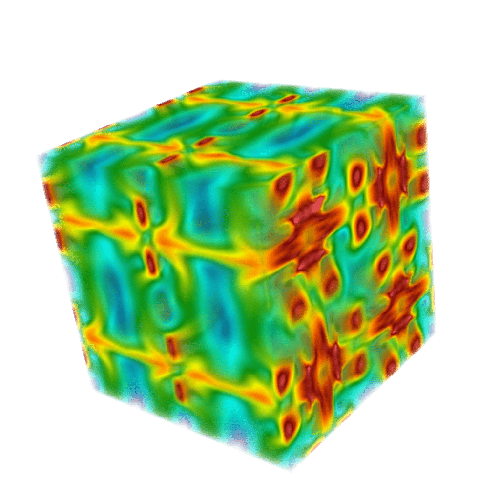}
\endminipage
\minipage{0.33\textwidth}
{\includegraphics[width=\linewidth, clip, draft=false]{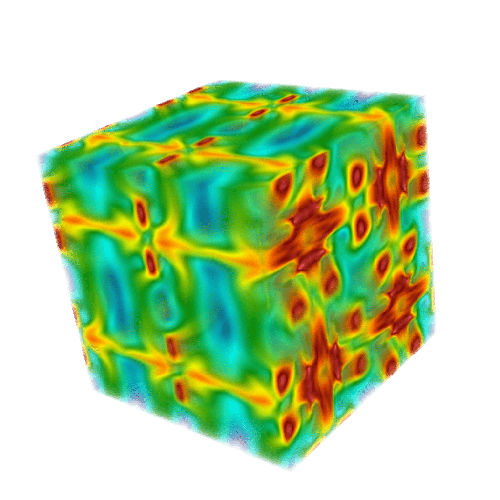}}
\endminipage
\minipage{0.33\textwidth}
{\includegraphics[width=\linewidth, clip, draft=false]{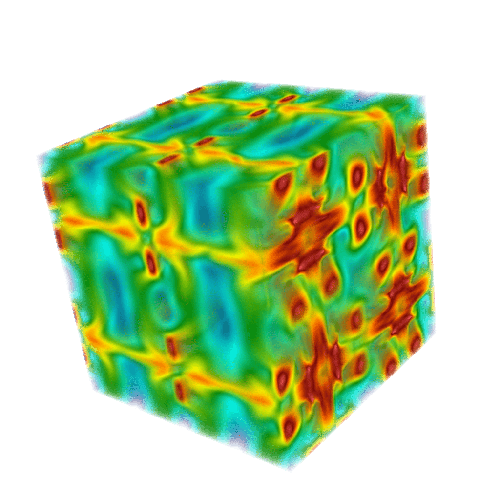}}
\endminipage
\endminipage

\minipage{\linewidth}
\minipage{0.33\textwidth}
\includegraphics[width=\linewidth, clip, draft=false]{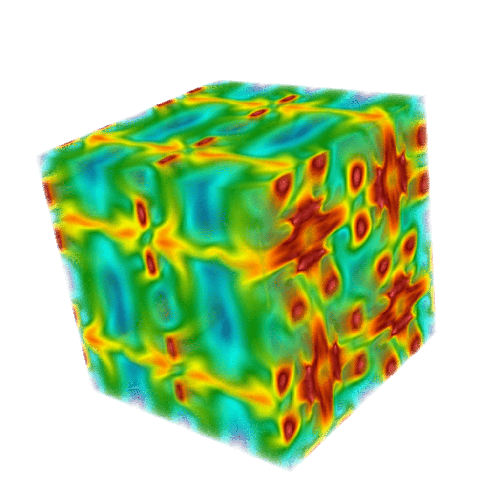}
\endminipage
\minipage{0.33\textwidth}
{\includegraphics[width=\linewidth, clip, draft=false]{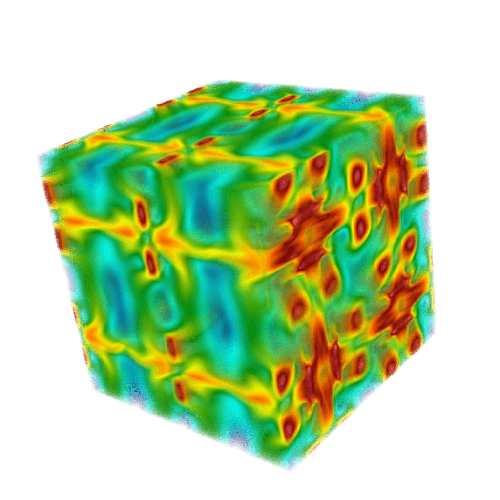}}
\endminipage
\minipage{0.33\textwidth}
{\includegraphics[width=\linewidth, clip, draft=false]{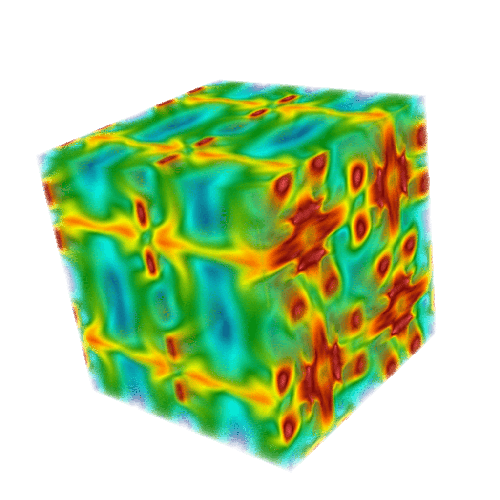}}
\endminipage
\endminipage

\minipage{\linewidth}
\minipage{0.33\textwidth}
\includegraphics[width=\linewidth, clip, draft=false]{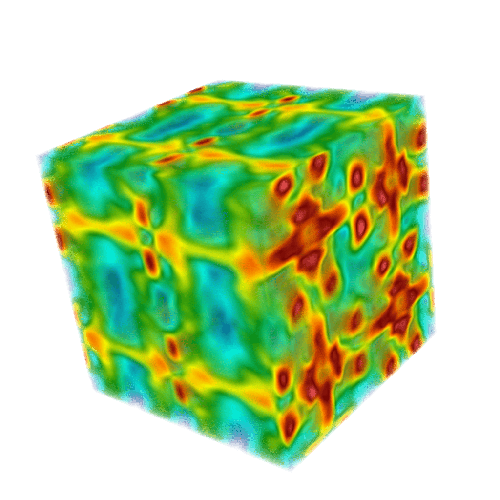}
\endminipage
\minipage{0.33\textwidth}
{\includegraphics[width=\linewidth, clip, draft=false]{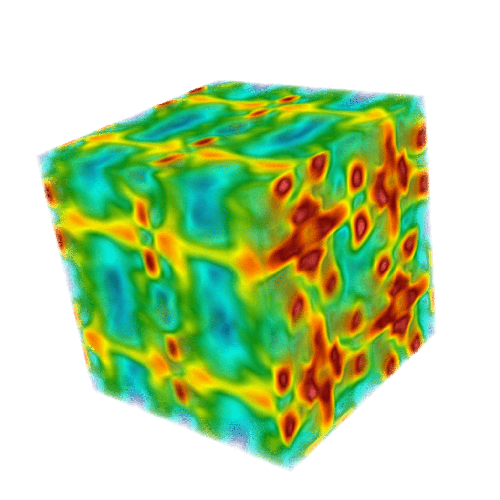}}
\endminipage
\minipage{0.33\textwidth}
{\includegraphics[width=\linewidth, clip, draft=false]{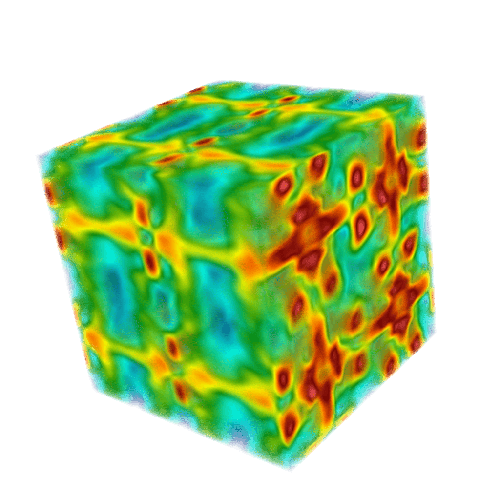}}
\endminipage
\endminipage
\caption{\textbf{Visualization of pointwise kinetic energy for 3 randomly generated samples for the three-dimensional Taylor--Green experiment at time $T=0.8$ with ground truth (top row), GenCFD (middle row) and C-FNO (bottom row).} The colormap for all the figures ranges from $0.0$ (dark blue) to $1.0$ (dark red).}
\label{fig:tg2}
\end{figure}

\begin{figure}[!t]
\minipage{\linewidth}
\minipage{0.33\textwidth}
\includegraphics[width=\linewidth, clip, draft=false]{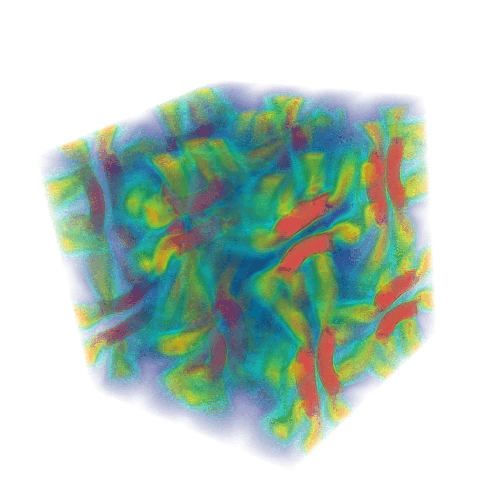}
\endminipage
\minipage{0.33\textwidth}
\includegraphics[width=\linewidth, clip, draft=false]{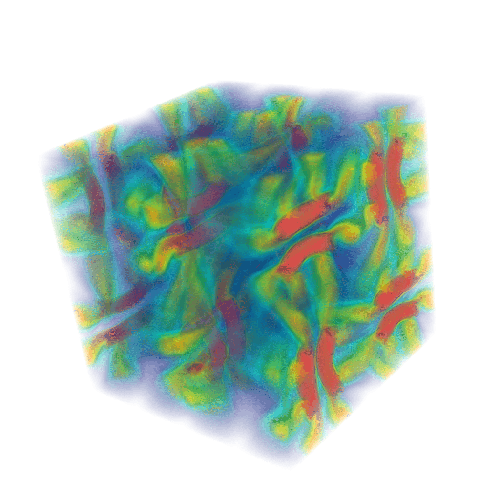}
\endminipage
\minipage{0.33\textwidth}
{\includegraphics[width=\linewidth, clip, draft=false]{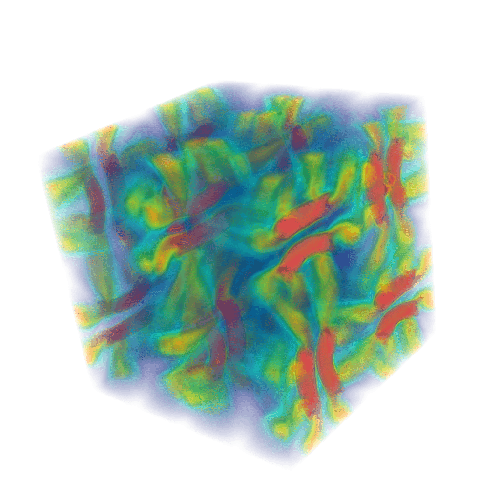}}
\endminipage
\endminipage

\minipage{\linewidth}
\minipage{0.33\textwidth}
\includegraphics[width=\linewidth, clip, draft=false]{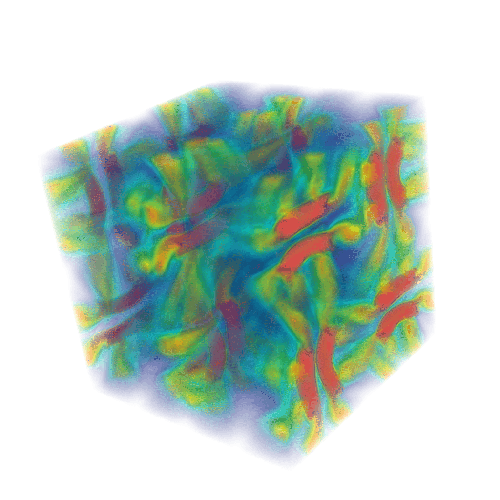}
\endminipage
\minipage{0.33\textwidth}
{\includegraphics[width=\linewidth, clip,  draft=false]{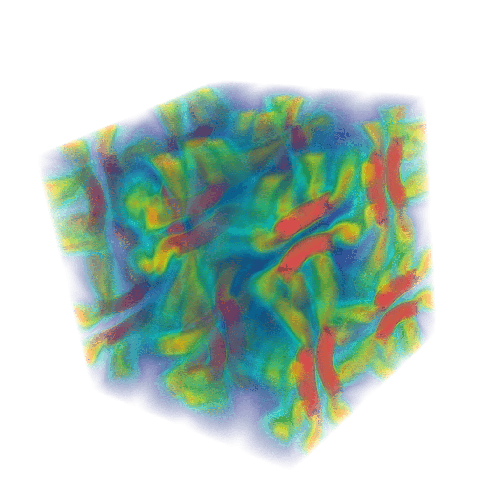}}
\endminipage
\minipage{0.33\textwidth}
{\includegraphics[width=\linewidth, clip,  draft=false]{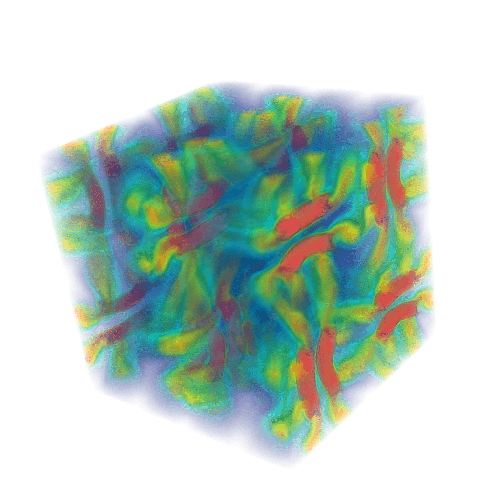}}
\endminipage
\endminipage

\minipage{\linewidth}
\minipage{0.33\textwidth}
\includegraphics[width=\linewidth, clip,  draft=false]{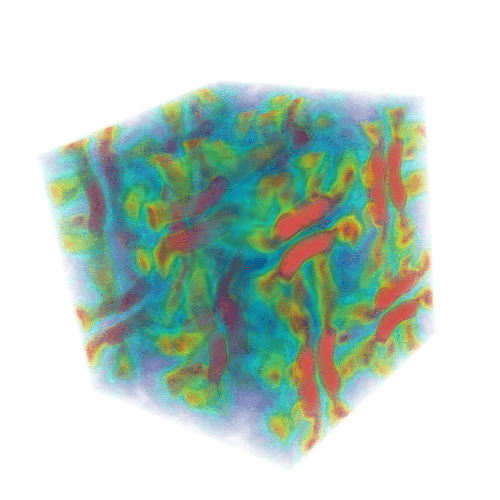}
\endminipage
\minipage{0.33\textwidth}
{\includegraphics[width=\linewidth, clip,  draft=false]{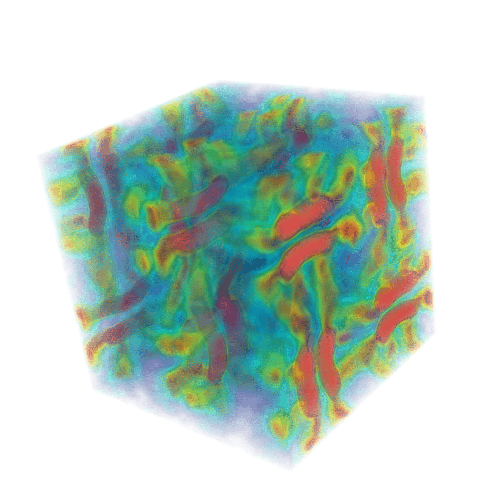}}
\endminipage
\minipage{0.33\textwidth}
{\includegraphics[width=\linewidth, clip,  draft=false]{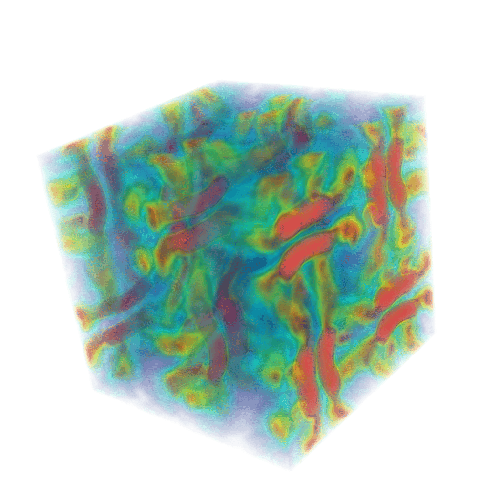}}
\endminipage
\endminipage
\caption{\textbf{Visualization of pointwise vorticity intensity for 3 randomly generated samples for the three-dimensional Taylor--Green experiment at time $T=2$ with ground truth (top row), GenCFD (middle row) and C-FNO (bottom row).} The colormap for the top and middle rows ranges from $10^{-4}$ (dark blue) to $40.0$ (dark red), whereas for the bottom row, it ranges from $10^{-4}$ to $35$.}
\label{fig:tgb2}
\end{figure}

\begin{figure}[!t]
\minipage{\linewidth}
\minipage{0.25\textwidth}
\includegraphics[width=\linewidth, clip]{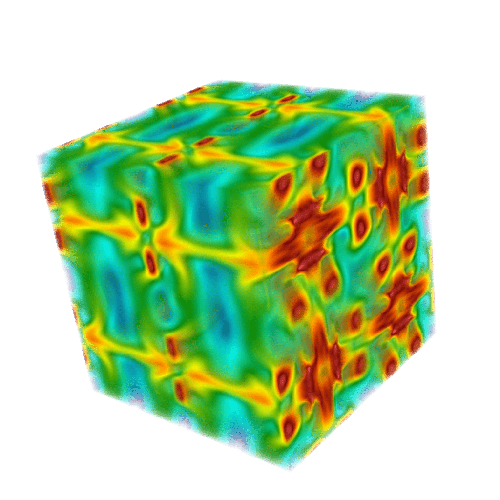}
\endminipage
\minipage{0.25\textwidth}
{\includegraphics[width=\linewidth, clip]{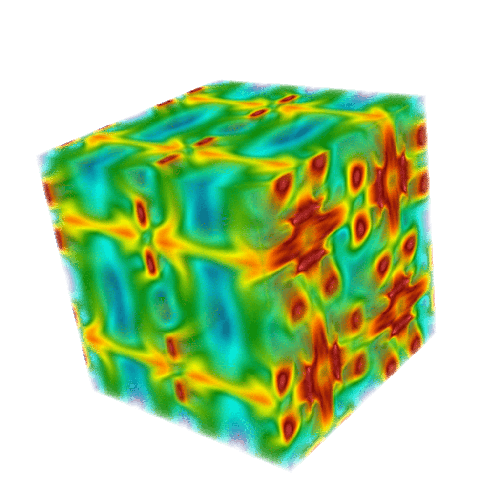}}
\endminipage
\minipage{0.25\textwidth}
{\includegraphics[width=\linewidth, clip]{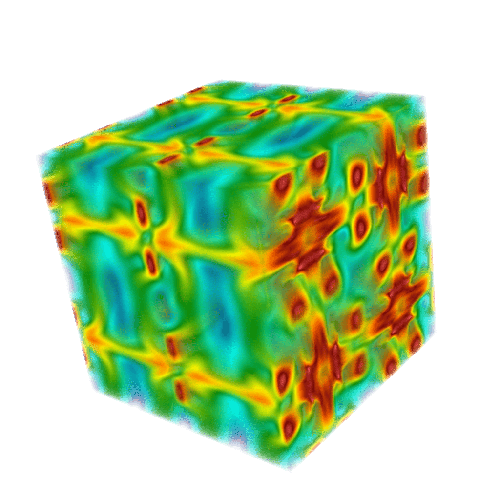}}
\endminipage
\minipage{0.25\textwidth}
{\includegraphics[width=\linewidth, clip]{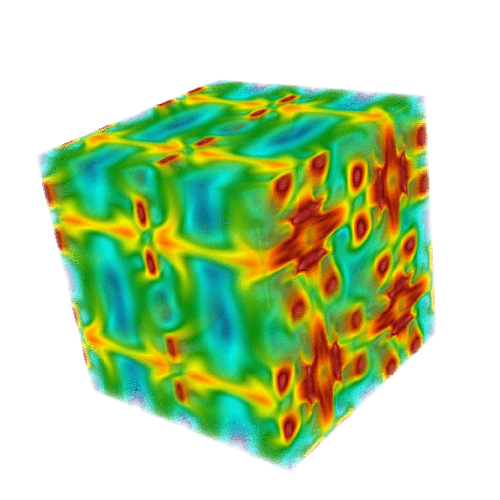}}
\endminipage
\endminipage

\minipage{\linewidth}
\minipage{0.25\textwidth}
\includegraphics[width=\linewidth, clip]{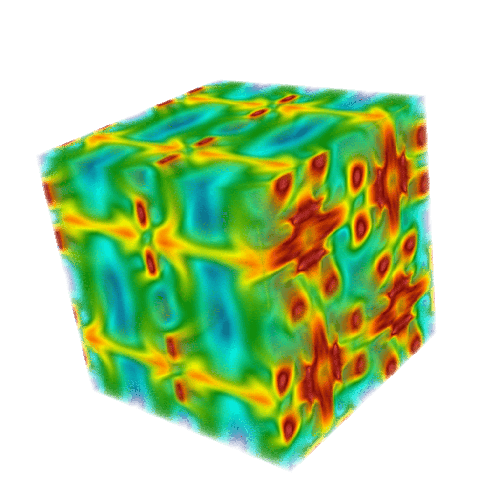}
\endminipage
\minipage{0.25\textwidth}
{\includegraphics[width=\linewidth, clip]{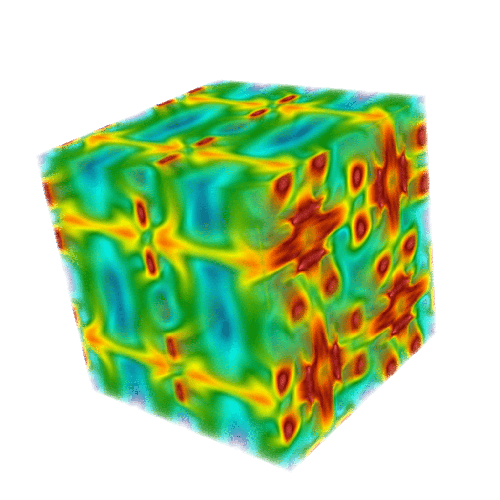}}
\endminipage
\minipage{0.25\textwidth}
{\includegraphics[width=\linewidth, clip]{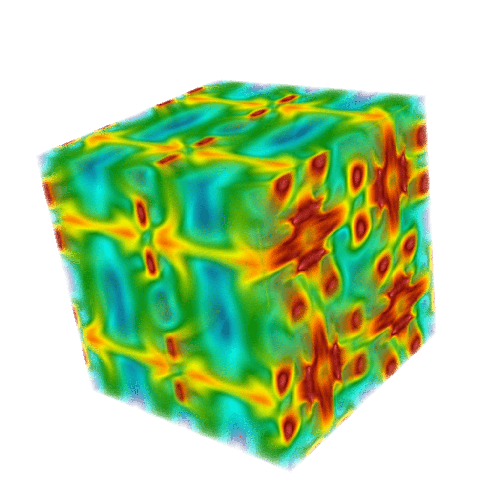}}
\endminipage
\minipage{0.25\textwidth}
{\includegraphics[width=\linewidth, clip]{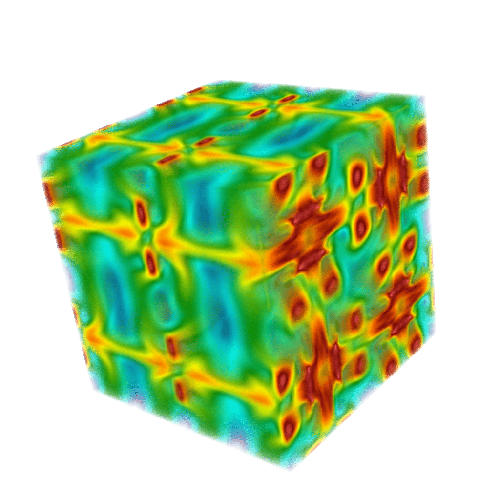}}
\endminipage
\endminipage

\minipage{\linewidth}
\minipage{0.25\textwidth}
\includegraphics[width=\linewidth, clip]{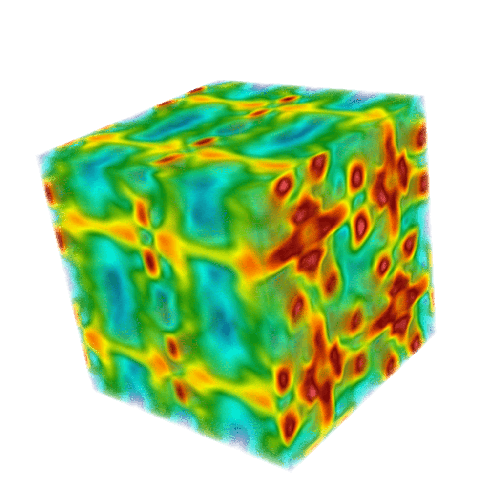}
\subcaption{$\bar{u}=\bar{u}^1$}
\endminipage
\minipage{0.25\textwidth}
{\includegraphics[width=\linewidth, clip]{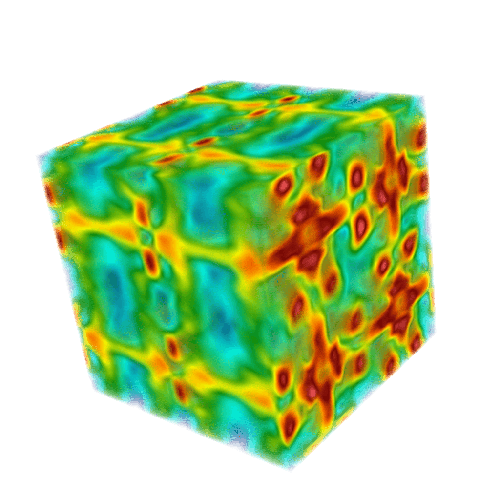}}
\subcaption{$\bar{u}=\bar{u}^2$}
\endminipage
\minipage{0.25\textwidth}
{\includegraphics[width=\linewidth, clip]{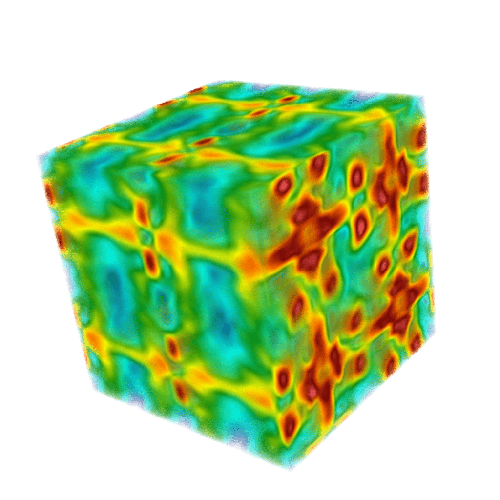}}
\subcaption{$\bar{u}=\bar{u}^3$}
\endminipage
\minipage{0.25\textwidth}
{\includegraphics[width=\linewidth, clip]{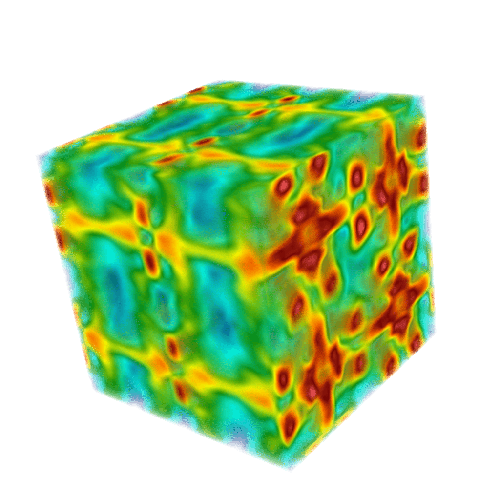}}
\subcaption{$\bar{u}=\bar{u}^4$}
\endminipage
\endminipage
\caption{\textbf{Visualization of the mean of the approximate statistical solution for the Taylor--Green experiment at time $T=2$, for four different initial distributions, generated by the ground truth (top row), GenCFD (middle row) and C-FNO (bottom row).} The colormap for all the figures ranges from $0.0$ (dark blue) to $1.0$ (dark red).}
\label{fig:tg_mean2}
\end{figure}

\begin{figure}[!t]
\minipage{\linewidth}
\minipage{0.25\textwidth}
\includegraphics[width=\linewidth, clip]{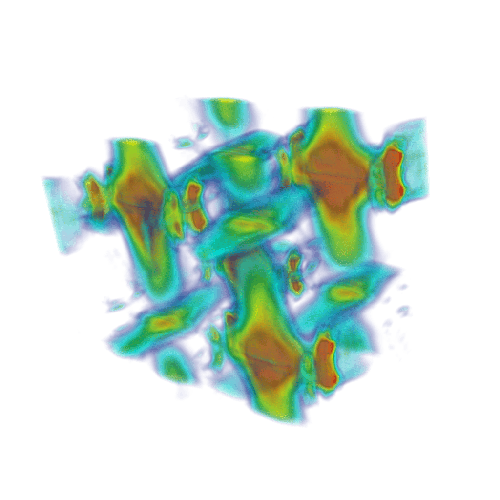}
\endminipage
\minipage{0.25\textwidth}
{\includegraphics[width=\linewidth, clip]{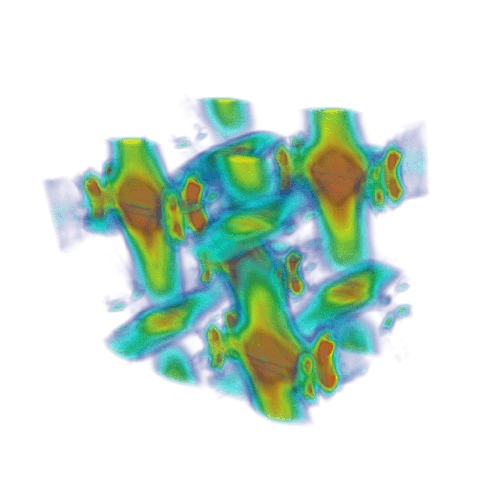}}
\endminipage
\minipage{0.25\textwidth}
{\includegraphics[width=\linewidth, clip]{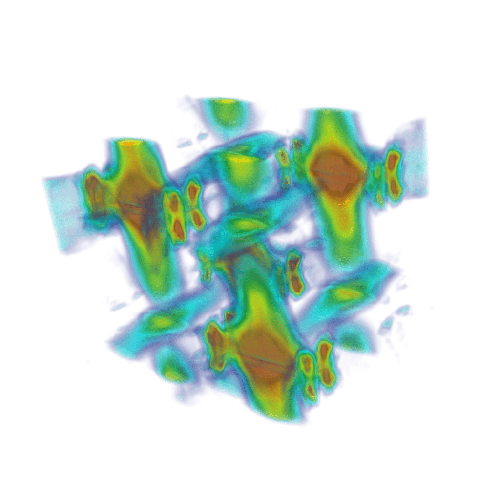}}
\endminipage
\minipage{0.25\textwidth}
{\includegraphics[width=\linewidth, clip]{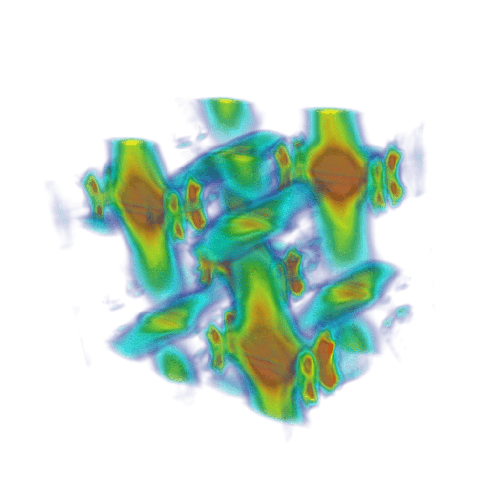}}
\endminipage
\endminipage

\minipage{\linewidth}
\minipage{0.25\textwidth}
\includegraphics[width=\linewidth, clip]{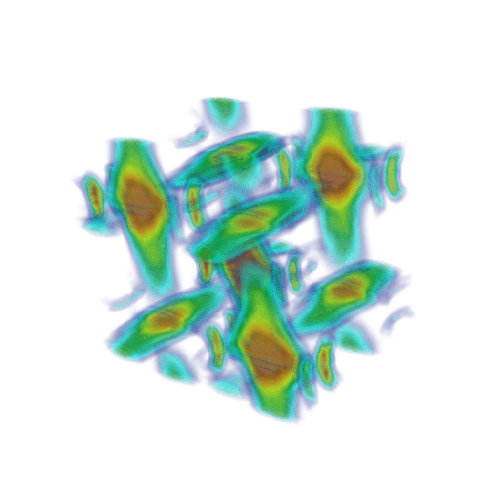}
\endminipage
\minipage{0.25\textwidth}
{\includegraphics[width=\linewidth, clip]{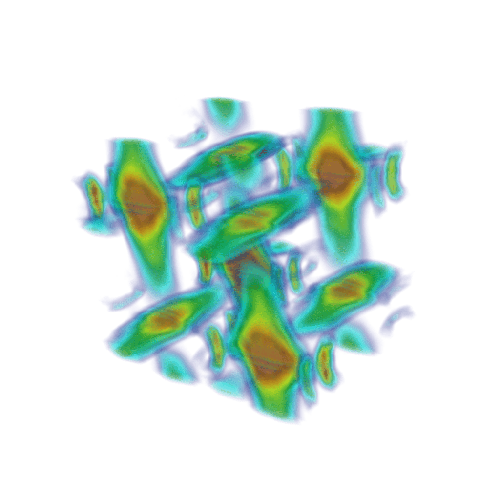}}
\endminipage
\minipage{0.25\textwidth}
{\includegraphics[width=\linewidth, clip]{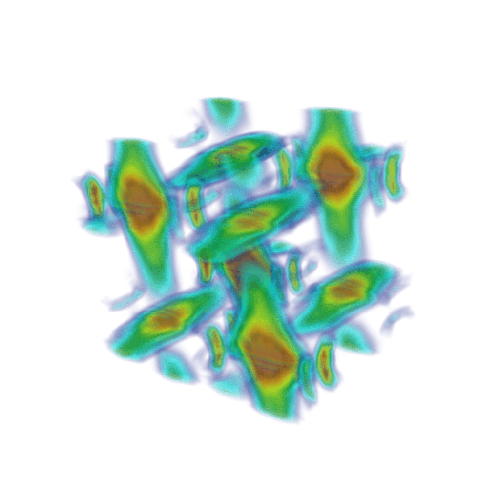}}
\endminipage
\minipage{0.25\textwidth}
{\includegraphics[width=\linewidth, clip]{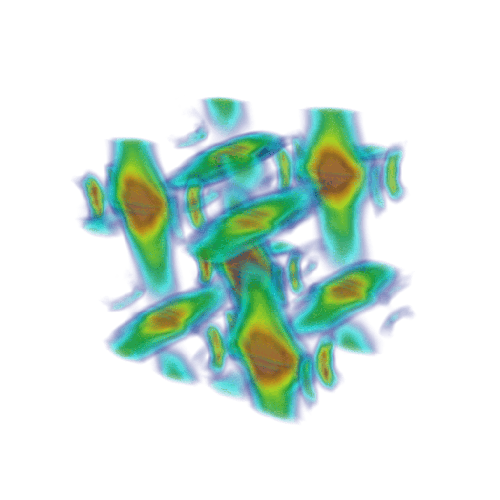}}
\endminipage
\endminipage

\minipage{\linewidth}
\minipage{0.25\textwidth}
\includegraphics[width=\linewidth, clip]{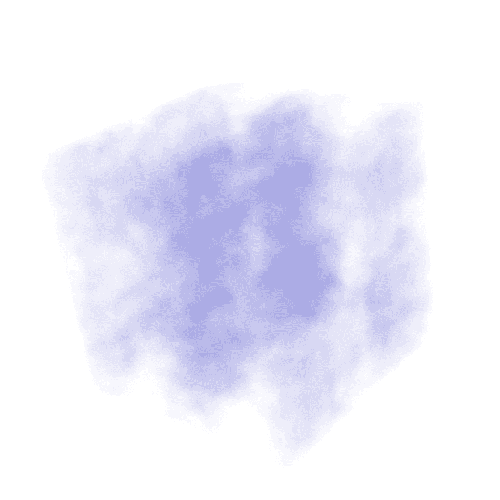}
\subcaption{$\bar{u}=\bar{u}^1$}
\endminipage
\minipage{0.25\textwidth}
{\includegraphics[width=\linewidth, clip]{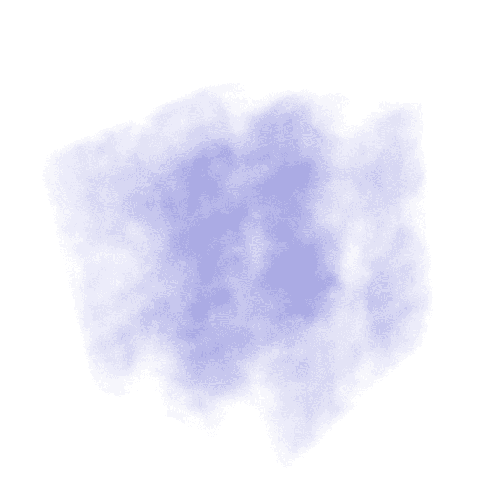}}
\subcaption{$\bar{u}=\bar{u}^2$}
\endminipage
\minipage{0.25\textwidth}
{\includegraphics[width=\linewidth, clip]{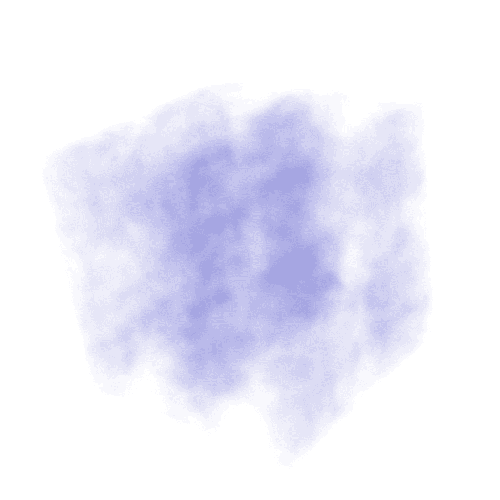}}
\subcaption{$\bar{u}=\bar{u}^3$}
\endminipage
\minipage{0.25\textwidth}
{\includegraphics[width=\linewidth, clip]{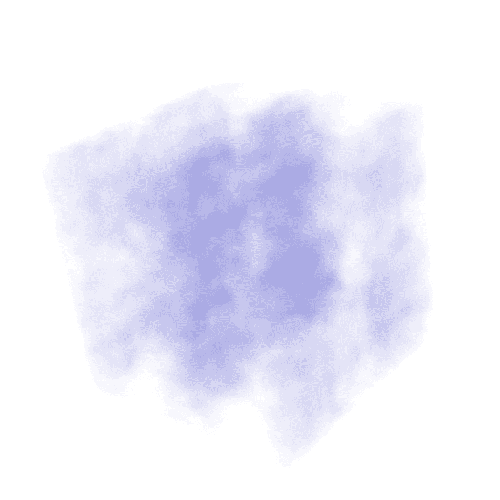}}
\subcaption{$\bar{u}=\bar{u}^4$}
\endminipage
\endminipage
\caption{\textbf{Visualization of the standard deviation of the (pointwise) kinetic energy for the Taylor--Green experiment at time $T=0.8$, for four different initial distributions, generated by the  ground truth (top row), GenCFD (middle row) and C-FNO (bottom row).} The colormap for the top and middle rows ranges from $0.05$ (dark blue) to $0.15$ (dark red), whereas for the bottom row, it ranges from $1.3\times 10^{-4}$ to $3.5\times 10^{-3}$.}
\label{fig:tg_std2}
\end{figure}

\begin{figure}[!t]
\minipage{\linewidth}
\centering
\minipage{0.3\textwidth}
\includegraphics[width=\linewidth, clip, trim=100 125 100 125, draft=false]{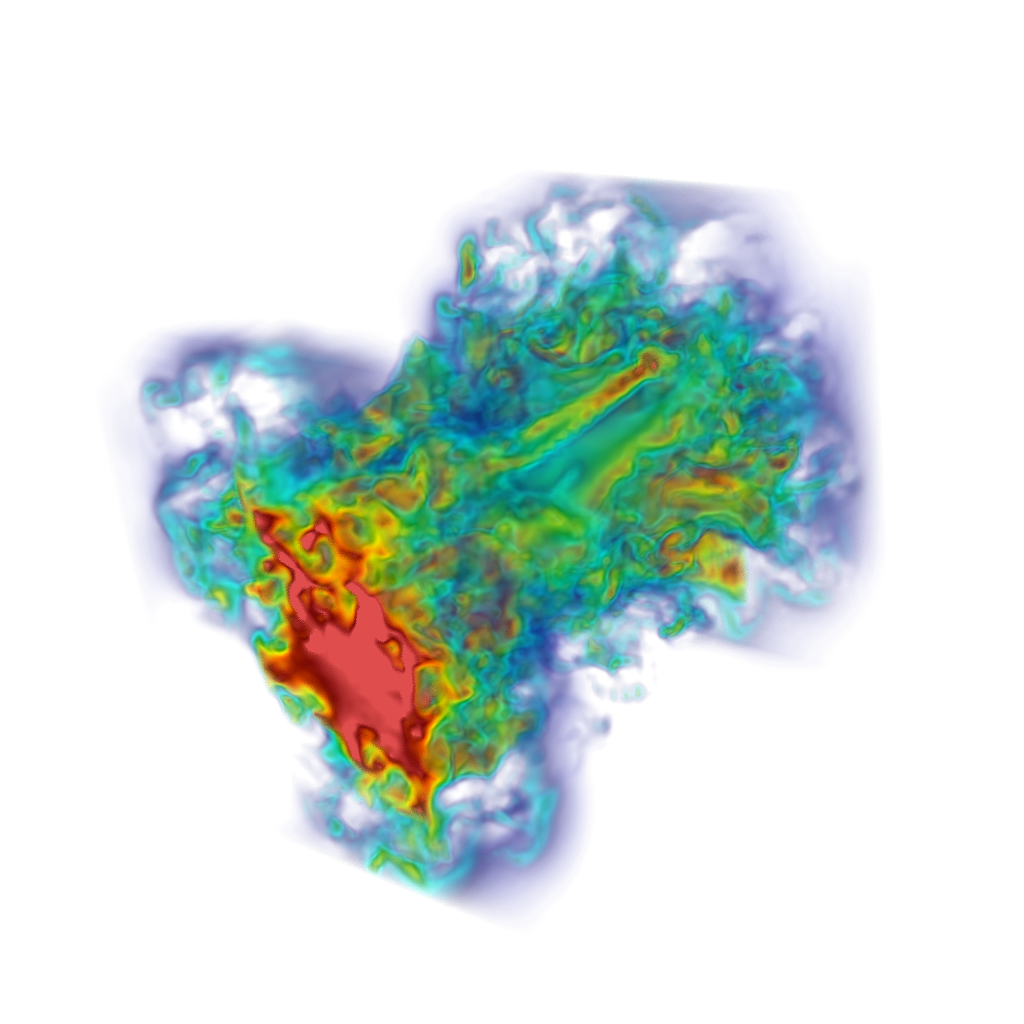}
\endminipage
\minipage{0.3\textwidth}
{\includegraphics[width=\linewidth, clip, trim=100 125 100 125, draft=false]{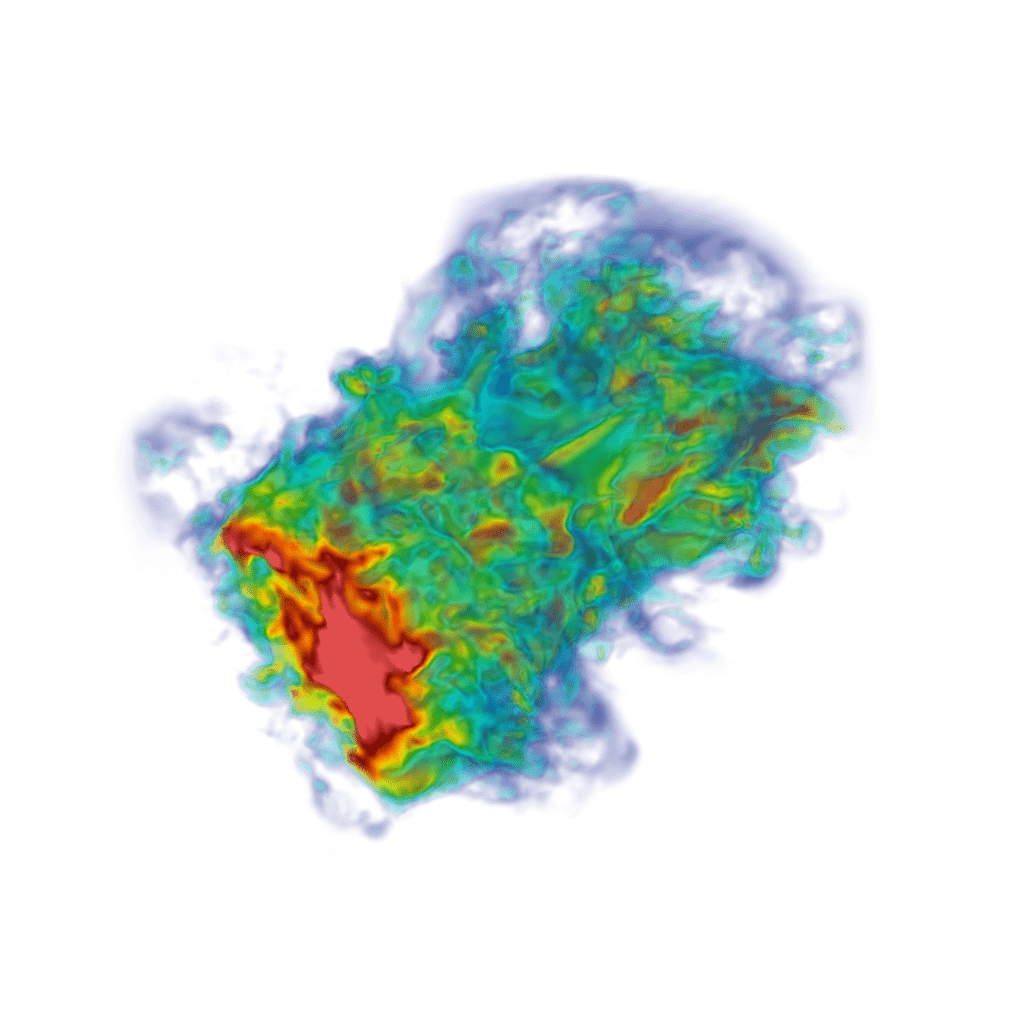}}
\endminipage
\minipage{0.3\textwidth}
{\includegraphics[width=\linewidth, clip, trim=100 125 100 125, draft=false]{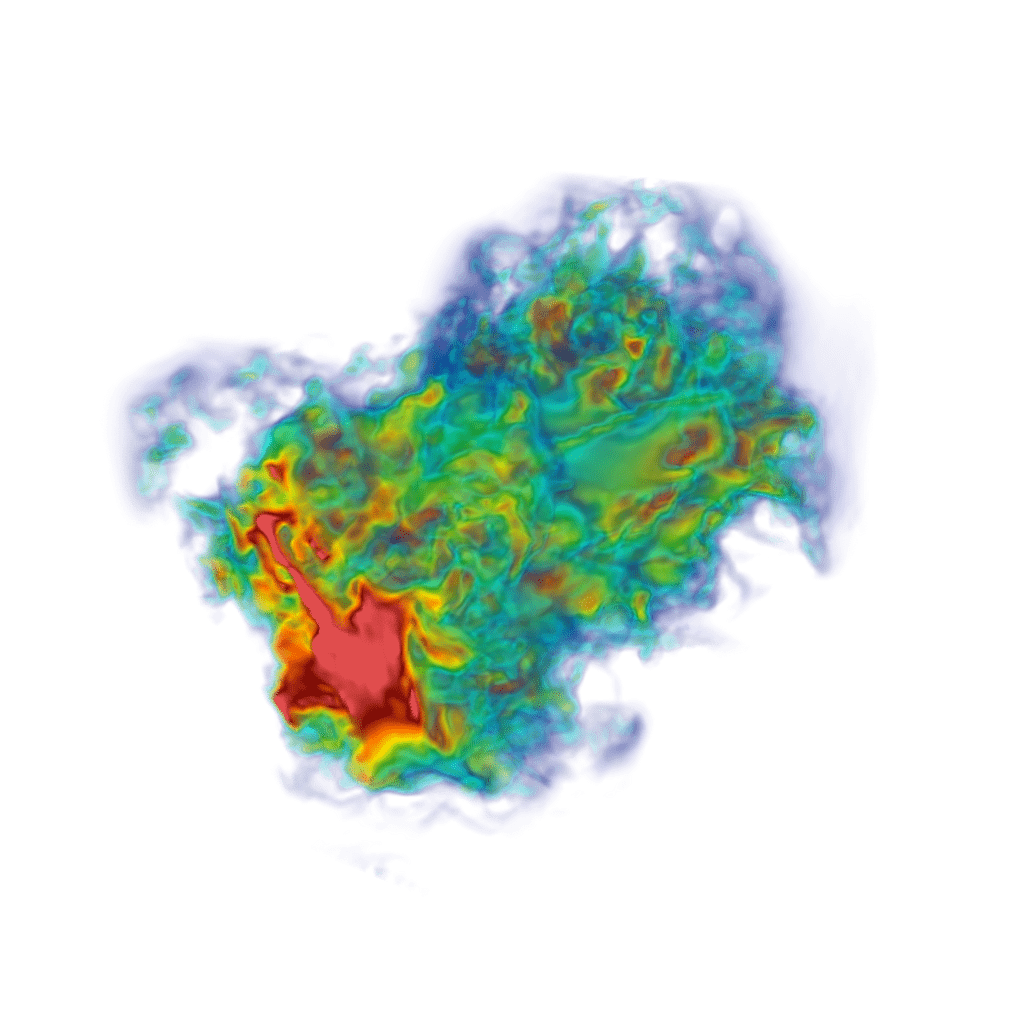}}
\endminipage
\endminipage

\minipage{\linewidth}
\centering
\minipage{0.3\textwidth}
\includegraphics[width=\linewidth, clip, trim=100 125 100 125, draft=false]{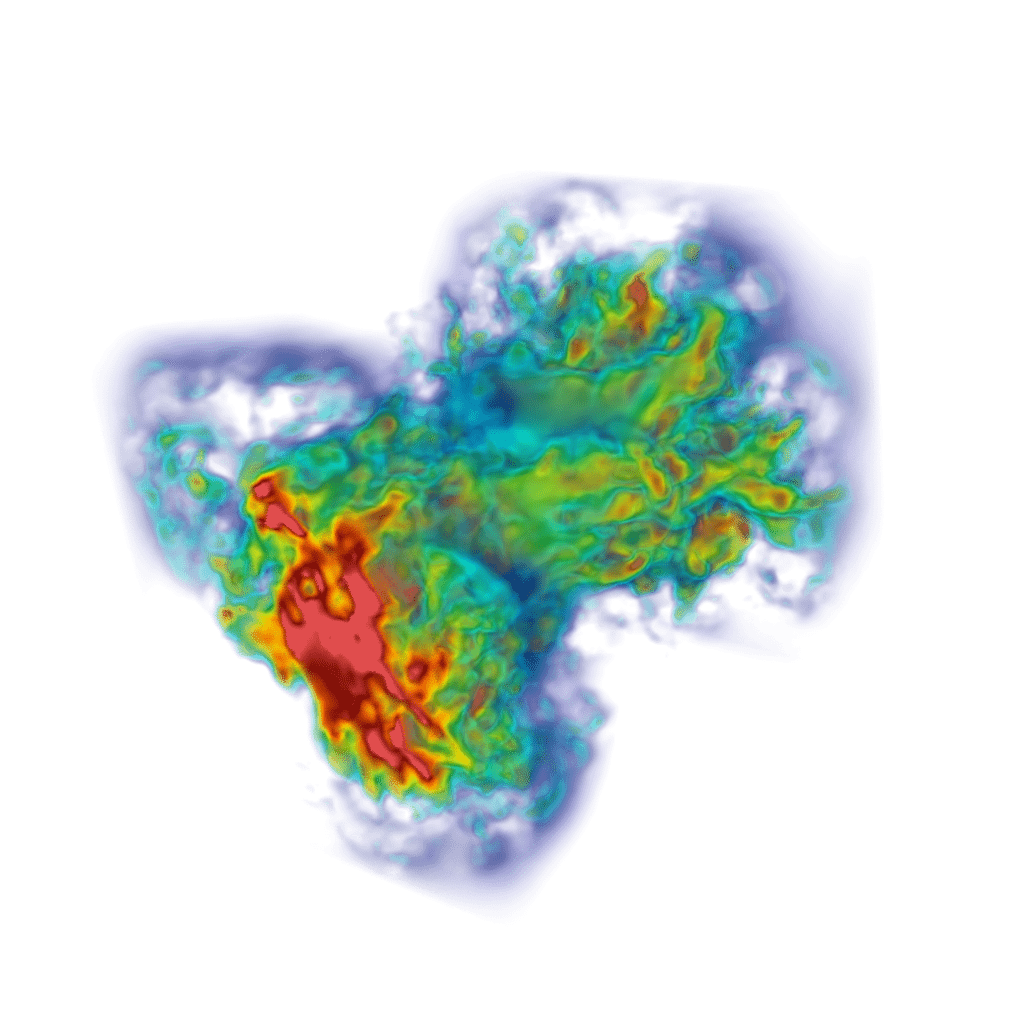}
\endminipage
\minipage{0.3\textwidth}
{\includegraphics[width=\linewidth, clip, trim=100 125 100 125, draft=false]{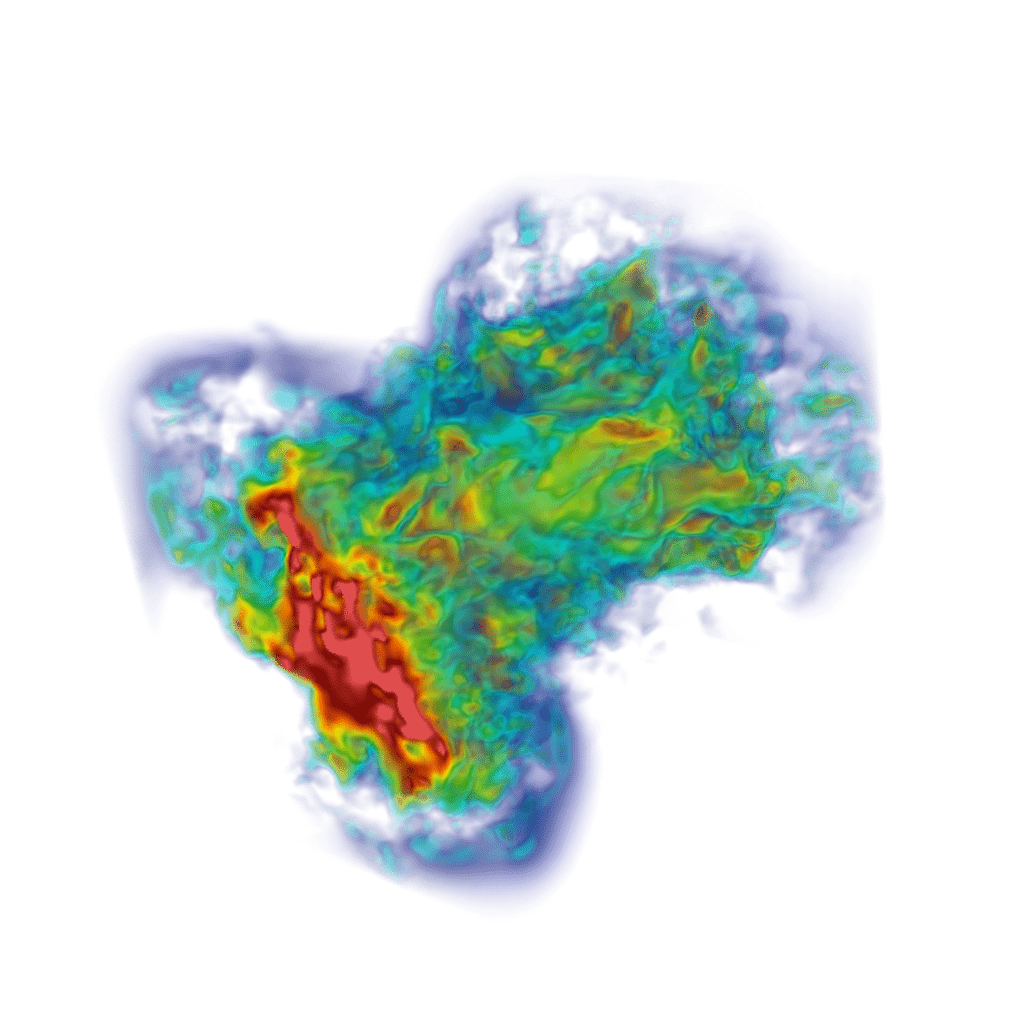}}
\endminipage
\minipage{0.3\textwidth}
{\includegraphics[width=\linewidth, clip, trim=100 125 100 125, draft=false]{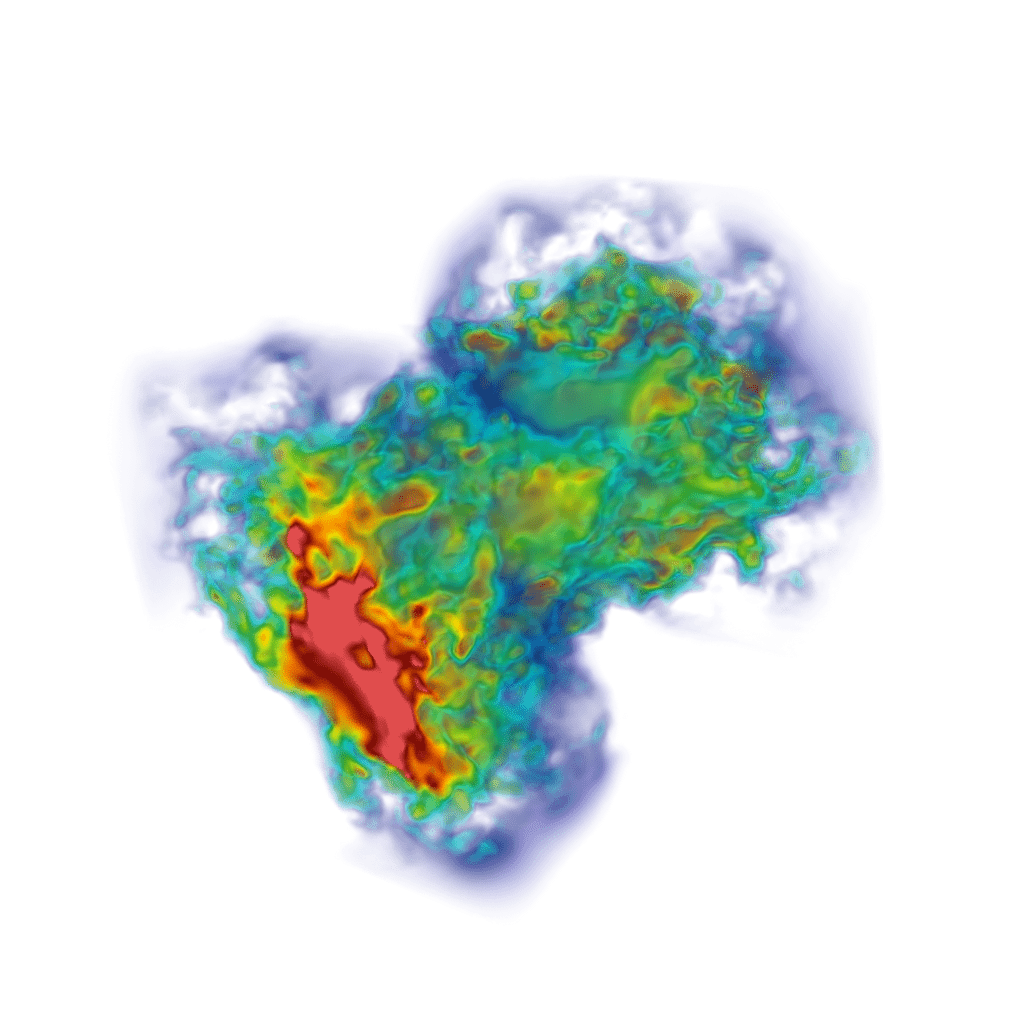}}
\endminipage
\subcaption{Ground truth (top) and GenCFD (bottom)}
\endminipage

\minipage{\linewidth}
\centering
\minipage{0.3\textwidth}
\includegraphics[width=\linewidth, clip, trim=100 125 100 125, draft=false]{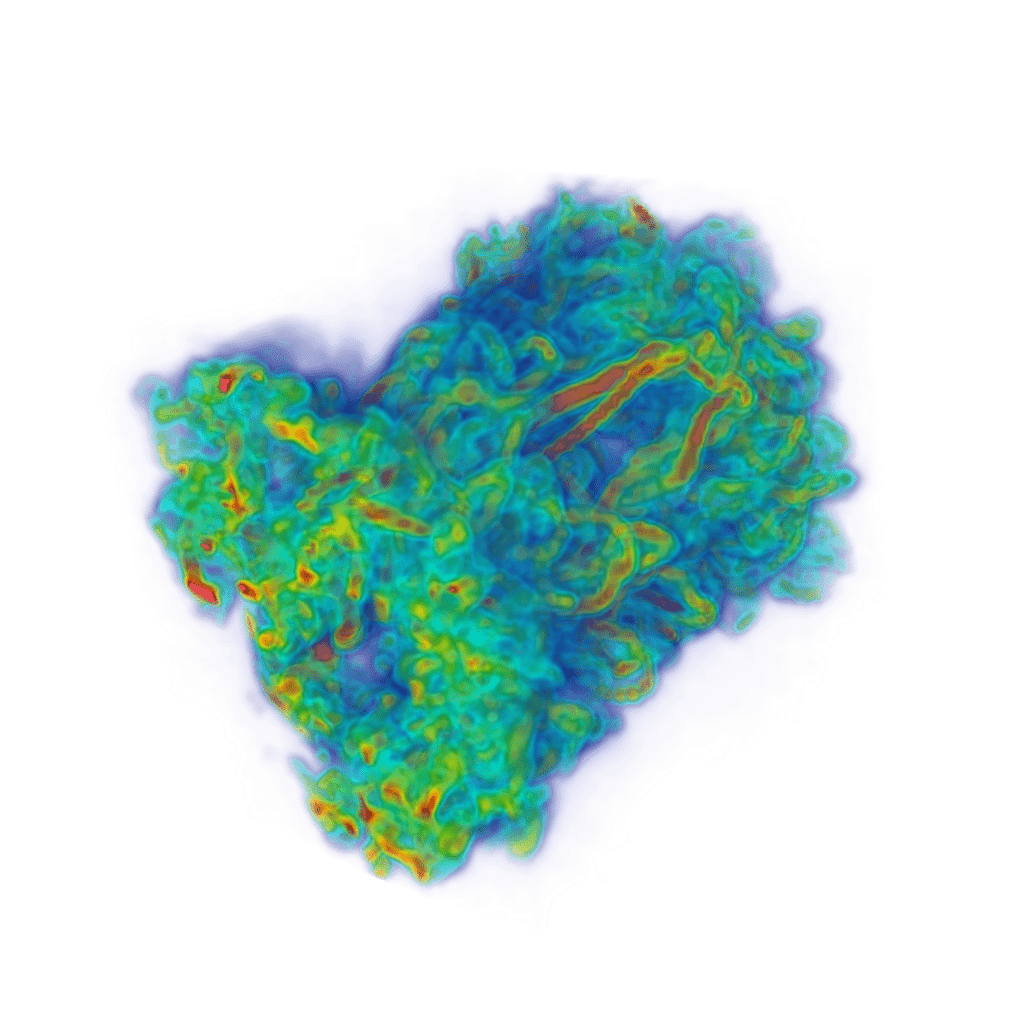}
\endminipage
\minipage{0.3\textwidth}
{\includegraphics[width=\linewidth, clip, trim=100 125 100 125, draft=false]{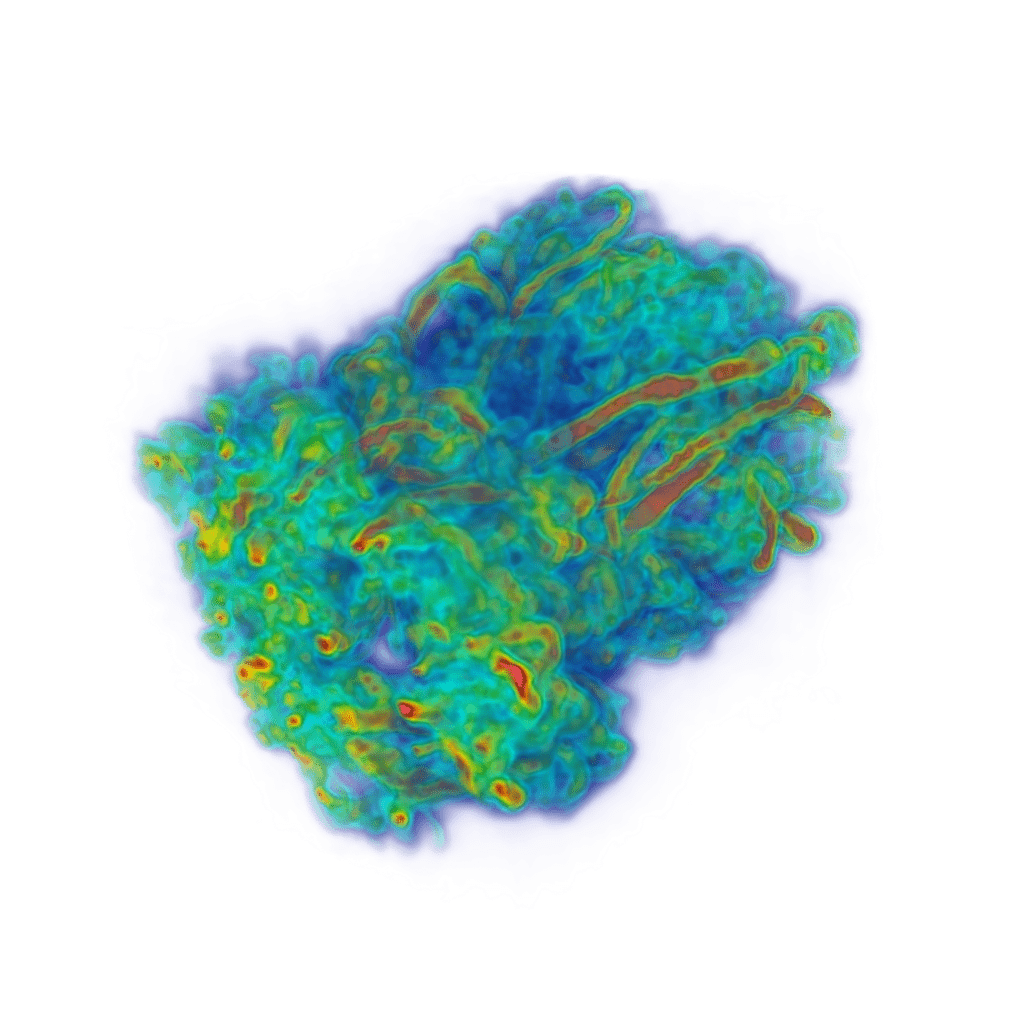}}
\endminipage
\minipage{0.3\textwidth}
{\includegraphics[width=\linewidth, clip, trim=100 125 100 125, draft=false]{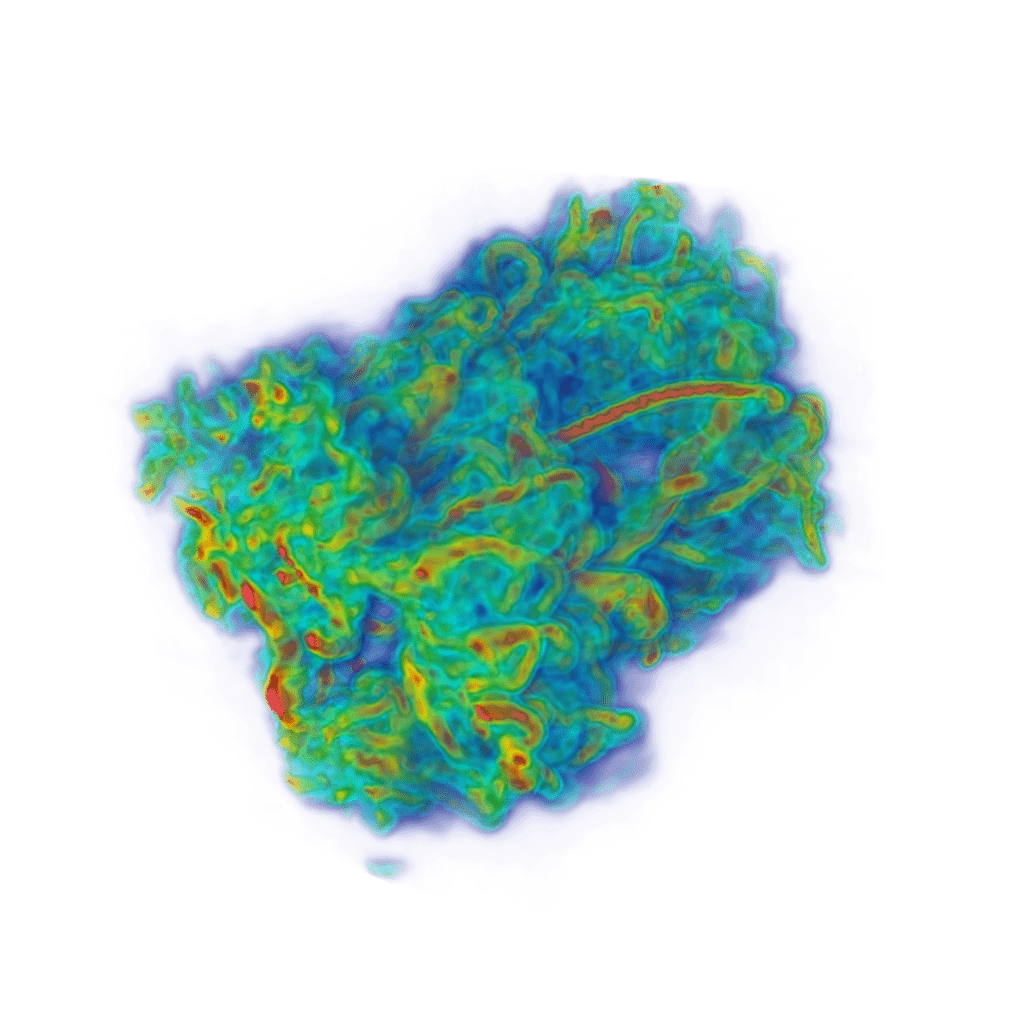}}
\endminipage

\endminipage

\minipage{\linewidth}
\centering
\minipage{0.3\textwidth}
\includegraphics[width=\linewidth, clip, trim=100 125 100 125, draft=false]{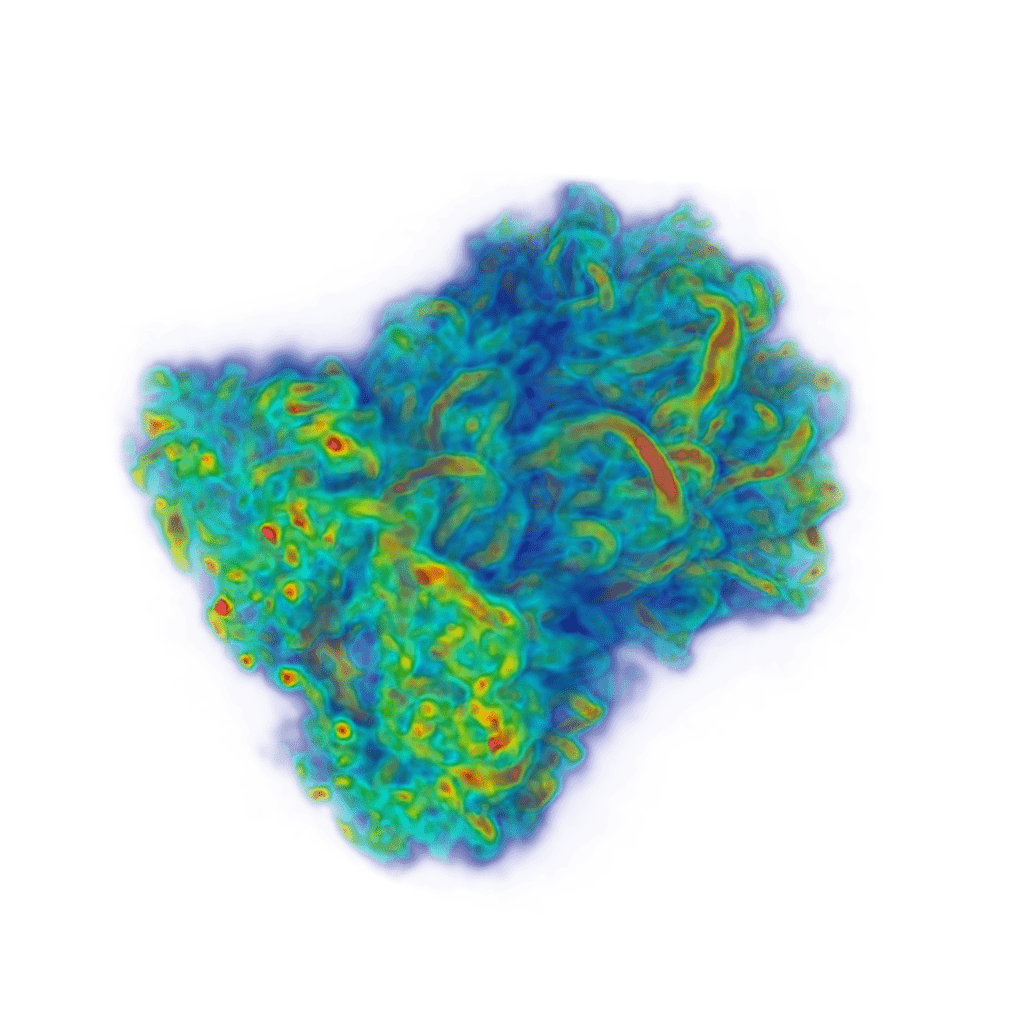}
\endminipage
\minipage{0.3\textwidth}
{\includegraphics[width=\linewidth, clip, trim=100 125 100 125, draft=false]{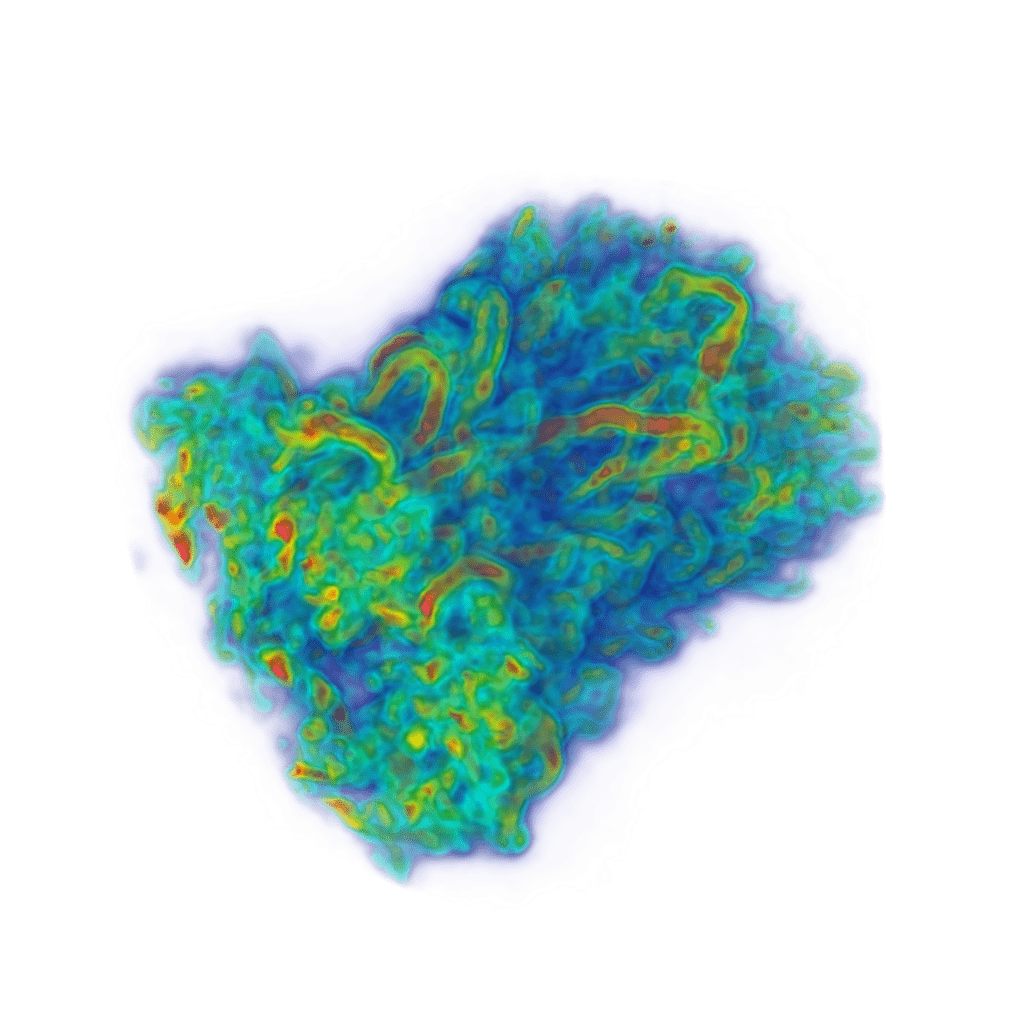}}
\endminipage
\minipage{0.3\textwidth}
{\includegraphics[width=\linewidth, clip, trim=100 125 100 125, draft=false]{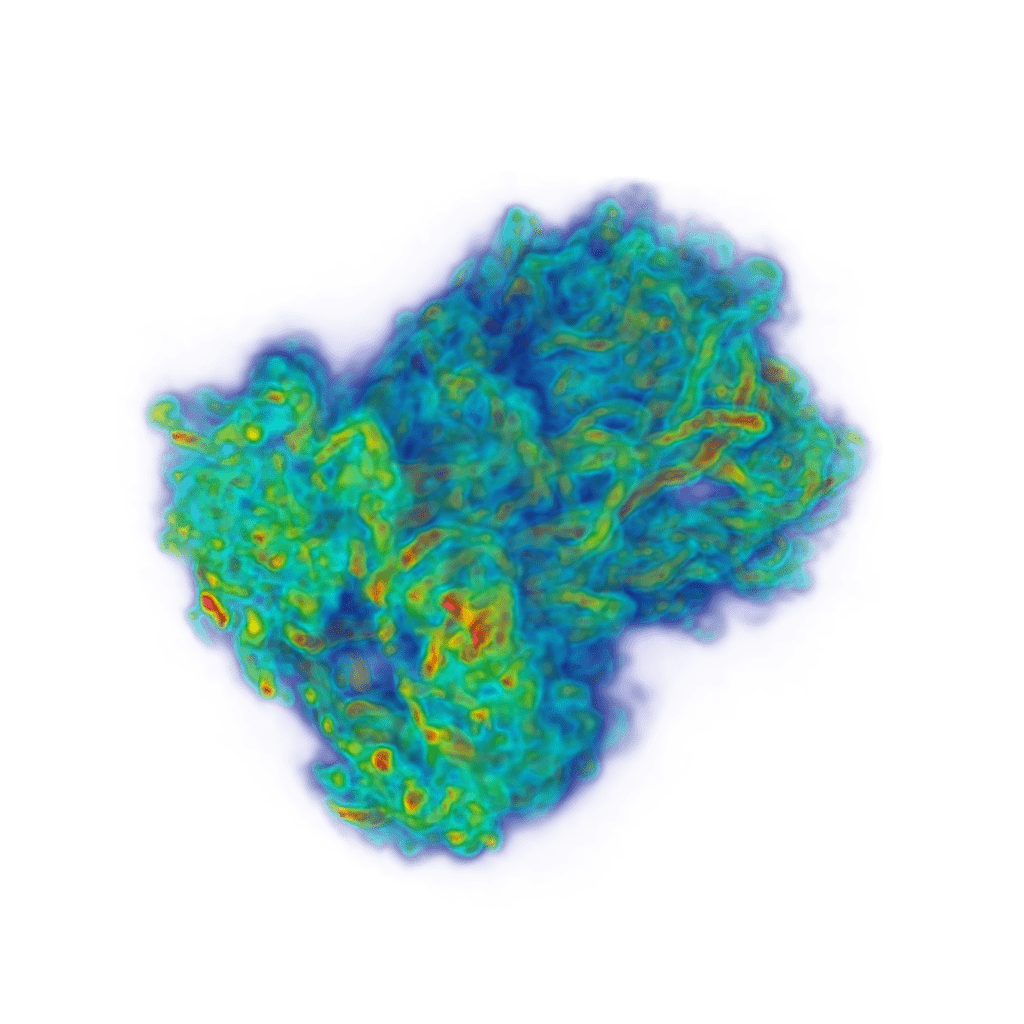}}
\endminipage
\subcaption{Ground truth (top) and GenCFD (bottom)}
\endminipage
\caption{\textbf{Visualization of pointwise kinetic energy (a) and vorticity (b) for 3 randomly generated samples for the three-dimensional cylindrical shear flow experiment at time $T=1$ for an initial condition different from the one presented in Fig.~\ref{fig:s2}.} Colormaps are identical to the ones used in Fig.~\ref{fig:s2}.}
\label{fig:sm1}
\end{figure}

\begin{figure}[!t]
\minipage{\linewidth}
\centering
\minipage{0.3\textwidth}
\includegraphics[width=\linewidth, clip, trim=75 100 75 100, draft=false]{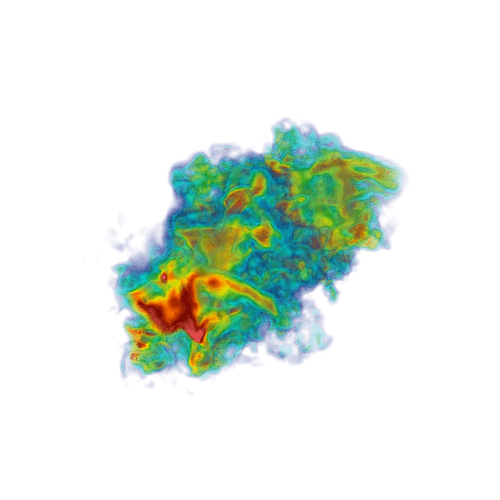}
\endminipage
\minipage{0.3\textwidth}
{\includegraphics[width=\linewidth, clip, trim=75 100 75 100, draft=false]{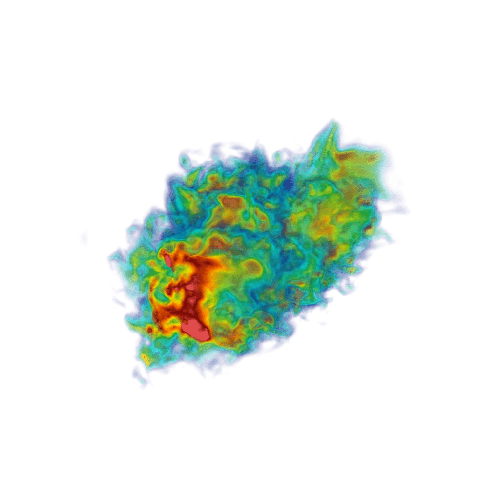}}
\endminipage
\minipage{0.3\textwidth}
{\includegraphics[width=\linewidth, clip, trim=75 100 75 100, draft=false]{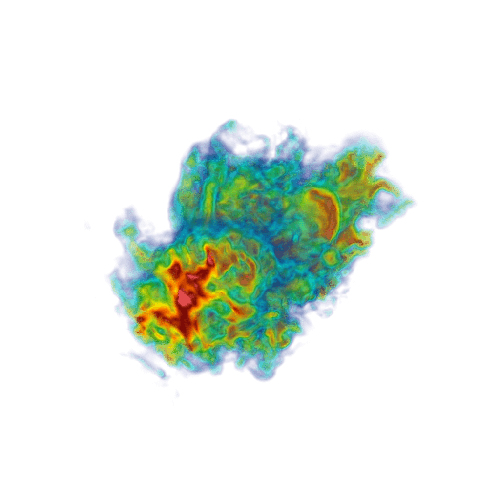}}
\endminipage
\endminipage

\minipage{\linewidth}
\centering
\minipage{0.3\textwidth}
\includegraphics[width=\linewidth, clip, trim=75 100 75 100, draft=false]{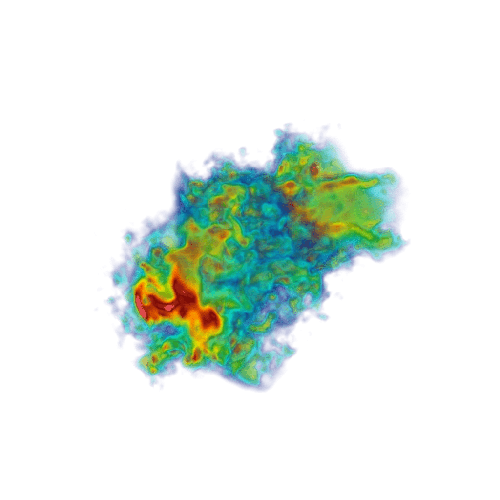}
\endminipage
\minipage{0.3\textwidth}
{\includegraphics[width=\linewidth, clip, trim=75 100 75 100, draft=false]{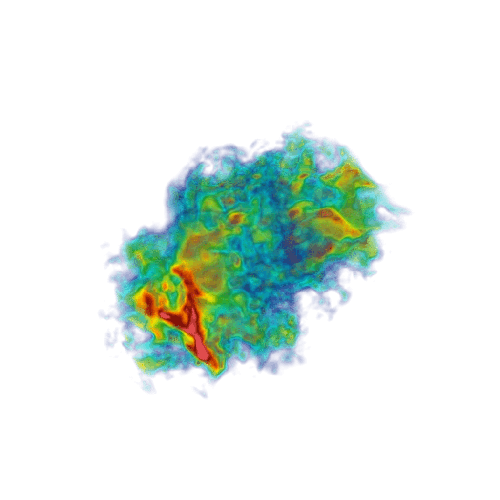}}
\endminipage
\minipage{0.3\textwidth}
{\includegraphics[width=\linewidth, clip, trim=75 100 75 100, draft=false]{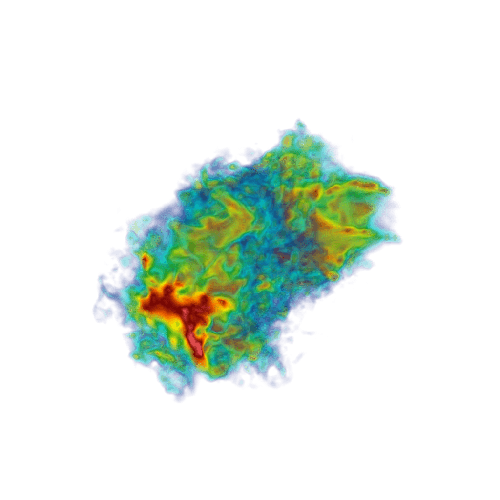}}
\endminipage
\subcaption{Ground truth (top) and GenCFD (bottom)}
\endminipage

\minipage{\linewidth}
\minipage{0.3\textwidth}
\includegraphics[width=\linewidth, clip, trim=60 80 60 80, draft=false]{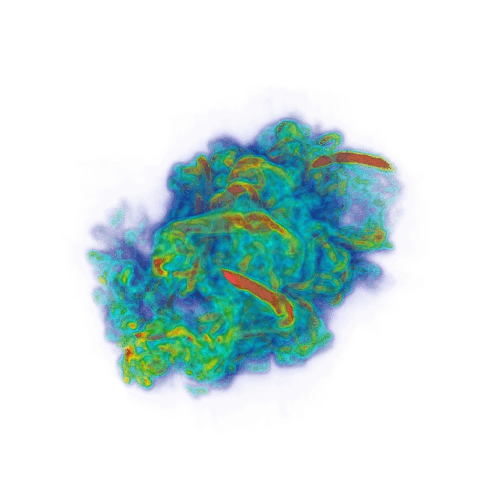}
\endminipage
\minipage{0.3\textwidth}
{\includegraphics[width=\linewidth, clip, trim=60 80 60 80, draft=false]{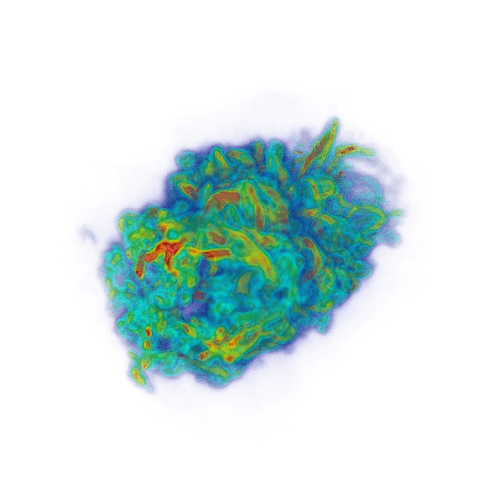}}
\endminipage
\minipage{0.3\textwidth}
{\includegraphics[width=\linewidth, clip, trim=60 80 60 80, draft=false]{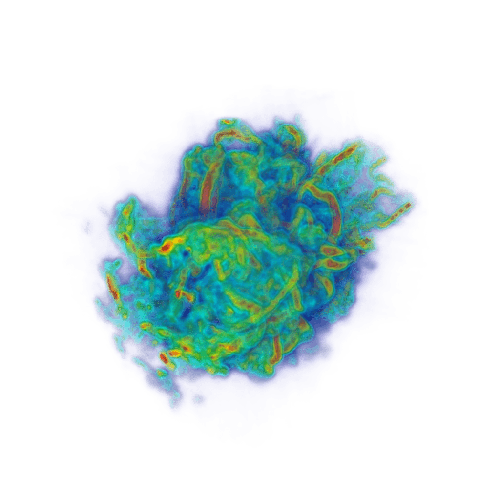}}
\endminipage

\endminipage

\minipage{\linewidth}
\centering
\minipage{0.3\textwidth}
\includegraphics[width=\linewidth, clip, trim=60 80 60 80, draft=false]{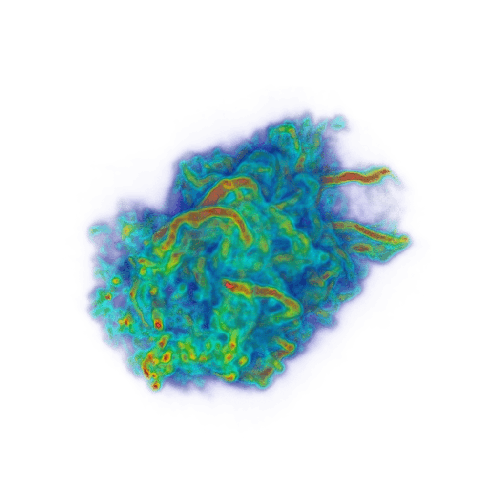}
\endminipage
\minipage{0.3\textwidth}
{\includegraphics[width=\linewidth, clip, trim=60 80 60 80, draft=false]{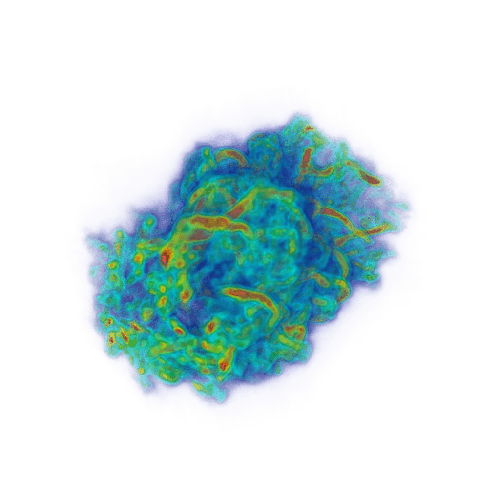}}
\endminipage
\minipage{0.3\textwidth}
{\includegraphics[width=\linewidth, clip, trim=60 80 60 80, draft=false]{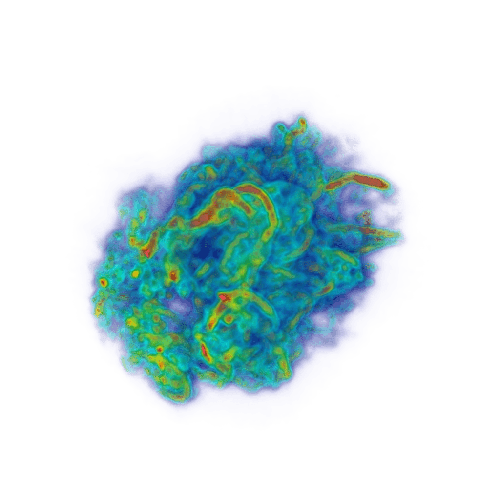}}
\endminipage
\subcaption{Ground truth (top) and GenCFD (bottom)}
\endminipage
\caption{\textbf{Visualization of pointwise kinetic energy (a) and vorticity (b) for 3 randomly generated samples for the three-dimensional cylindrical shear flow experiment at time $T=1$ for an initial condition different from Figs.~\ref{fig:s2} and \ref{fig:sm1}.} Colormaps are identical to the ones used in Fig.~\ref{fig:s2}.}
\label{fig:sm22}
\end{figure}

\begin{figure}[!t]
\minipage{\linewidth}
\centering
\minipage{0.3\textwidth}
\includegraphics[width=\linewidth, clip, trim=75 100 75 100, draft=false]{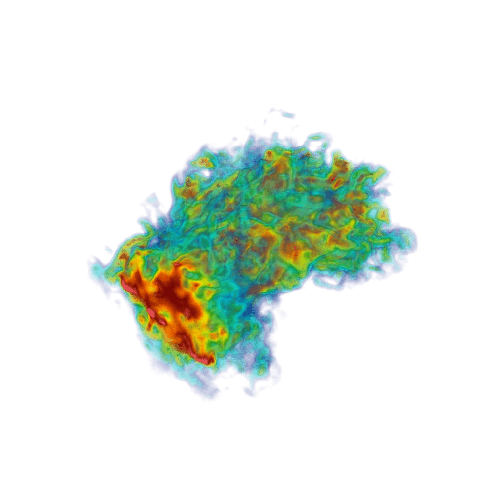}
\endminipage
\minipage{0.3\textwidth}
{\includegraphics[width=\linewidth, clip, trim=75 100 75 100, draft=false]{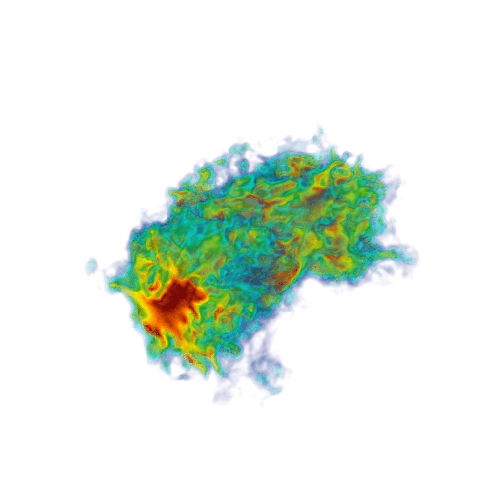}}
\endminipage
\minipage{0.3\textwidth}
{\includegraphics[width=\linewidth, clip, trim=75 100 75 100, draft=false]{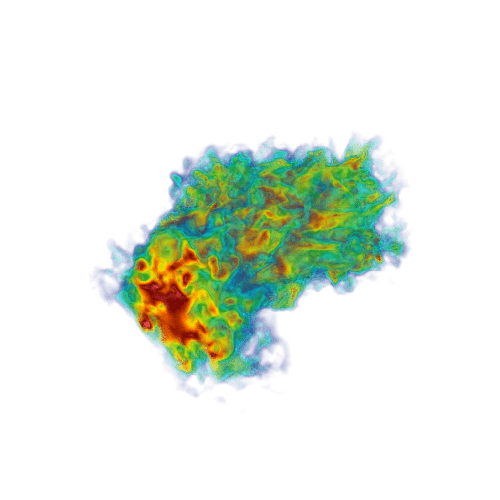}}
\endminipage
\endminipage

\minipage{\linewidth}
\centering
\minipage{0.3\textwidth}
\includegraphics[width=\linewidth, clip, trim=75 100 75 100, draft=false]{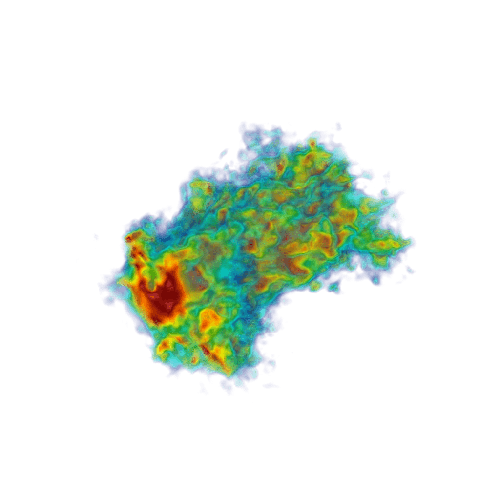}
\endminipage
\minipage{0.3\textwidth}
{\includegraphics[width=\linewidth, clip, trim=75 100 75 100, draft=false]{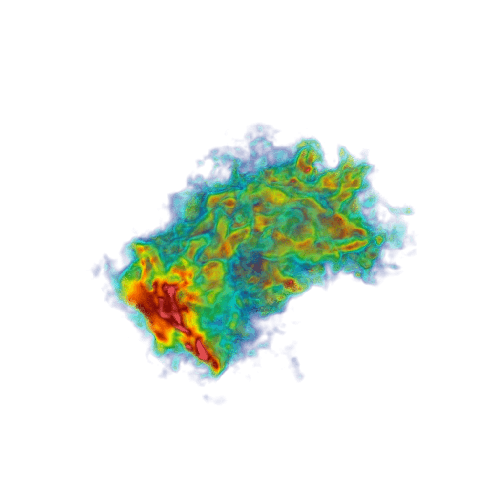}}
\endminipage
\minipage{0.3\textwidth}
{\includegraphics[width=\linewidth, clip, trim=75 100 75 100, draft=false]{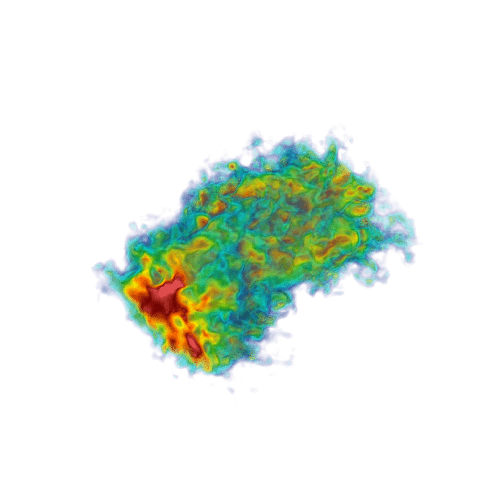}}
\endminipage
\subcaption{Ground truth (top) and GenCFD (bottom)}
\endminipage

\minipage{\linewidth}
\centering
\minipage{0.3\textwidth}
\includegraphics[width=\linewidth, clip, trim=60 80 60 80, draft=false]{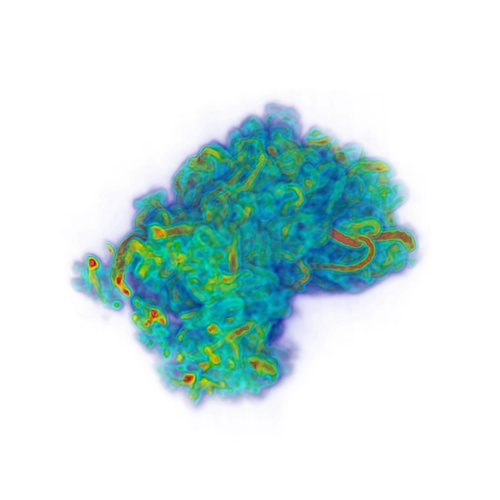}
\endminipage
\minipage{0.3\textwidth}
{\includegraphics[width=\linewidth, clip, trim=60 80 60 80, draft=false]{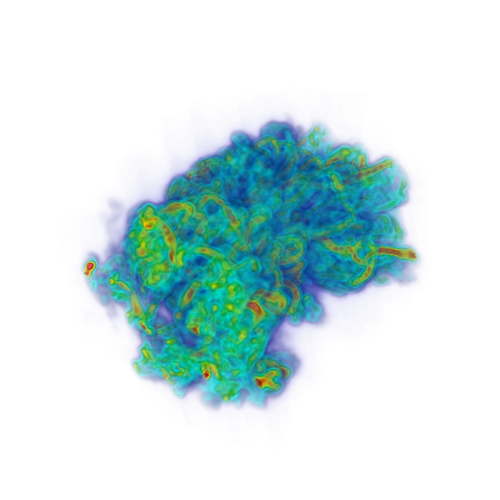}}
\endminipage
\minipage{0.3\textwidth}
{\includegraphics[width=\linewidth, clip, trim=60 80 60 80, draft=false]{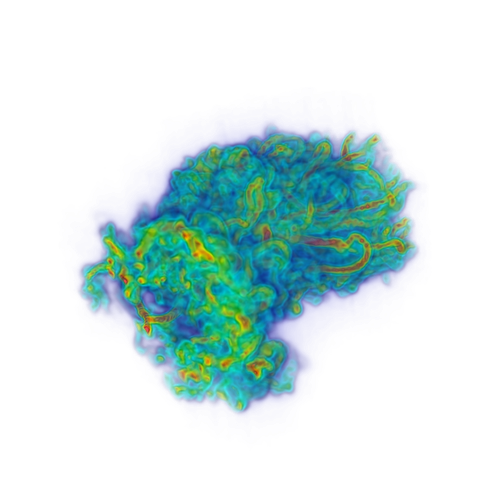}}
\endminipage

\endminipage

\minipage{\linewidth}
\centering
\minipage{0.3\textwidth}
\includegraphics[width=\linewidth, clip, trim=60 80 60 80, draft=false]{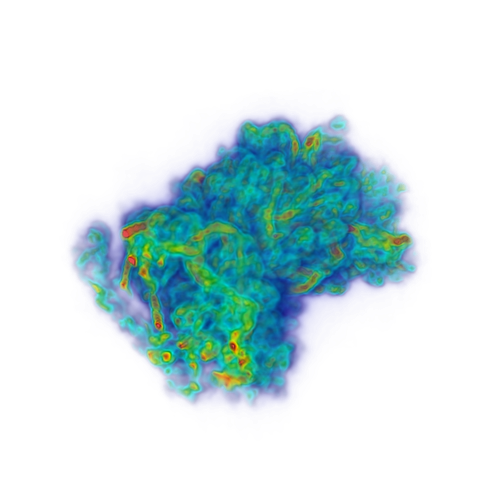}
\endminipage
\minipage{0.3\textwidth}
{\includegraphics[width=\linewidth, clip, trim=60 80 60 80, draft=false]{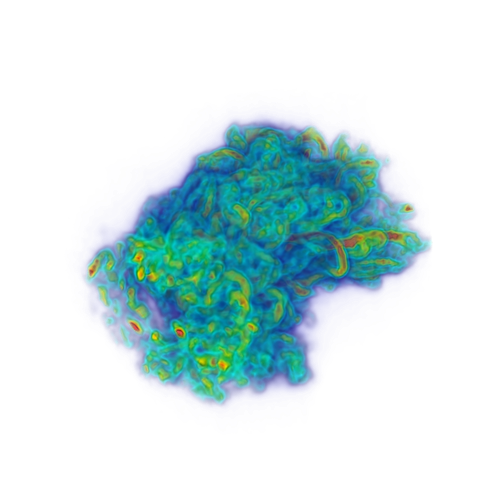}}
\endminipage
\minipage{0.3\textwidth}
{\includegraphics[width=\linewidth, clip, trim=60 80 60 80, draft=false]{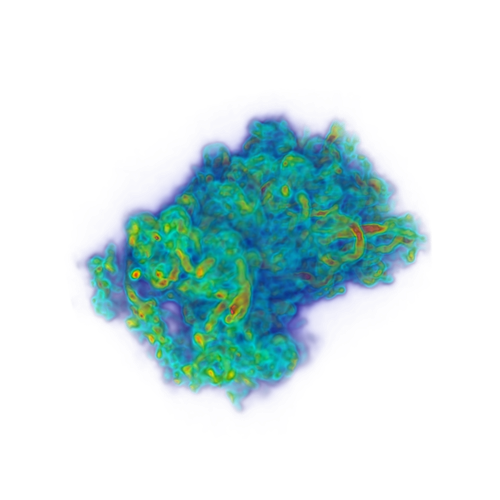}}
\endminipage
\subcaption{Ground truth (top) and GenCFD (bottom)}
\endminipage
\caption{\textbf{Visualization of pointwise kinetic energy (a) and vorticity (b) for 3 randomly generated Samples for the three-dimensional cylindrical shear flow experiment at time $T=1$ for an initial condition different from Figs.~\ref{fig:s2}, \ref{fig:sm1} and \ref{fig:sm22}.} Colormaps are identical to the ones used in Fig.~\ref{fig:s2}.}
\label{fig:sm3}
\end{figure}

\begin{figure}[!t]
\minipage{\linewidth}
\centering
\minipage{0.3\textwidth}
\includegraphics[width=\linewidth, clip, trim=100 125 100 125, draft=false]{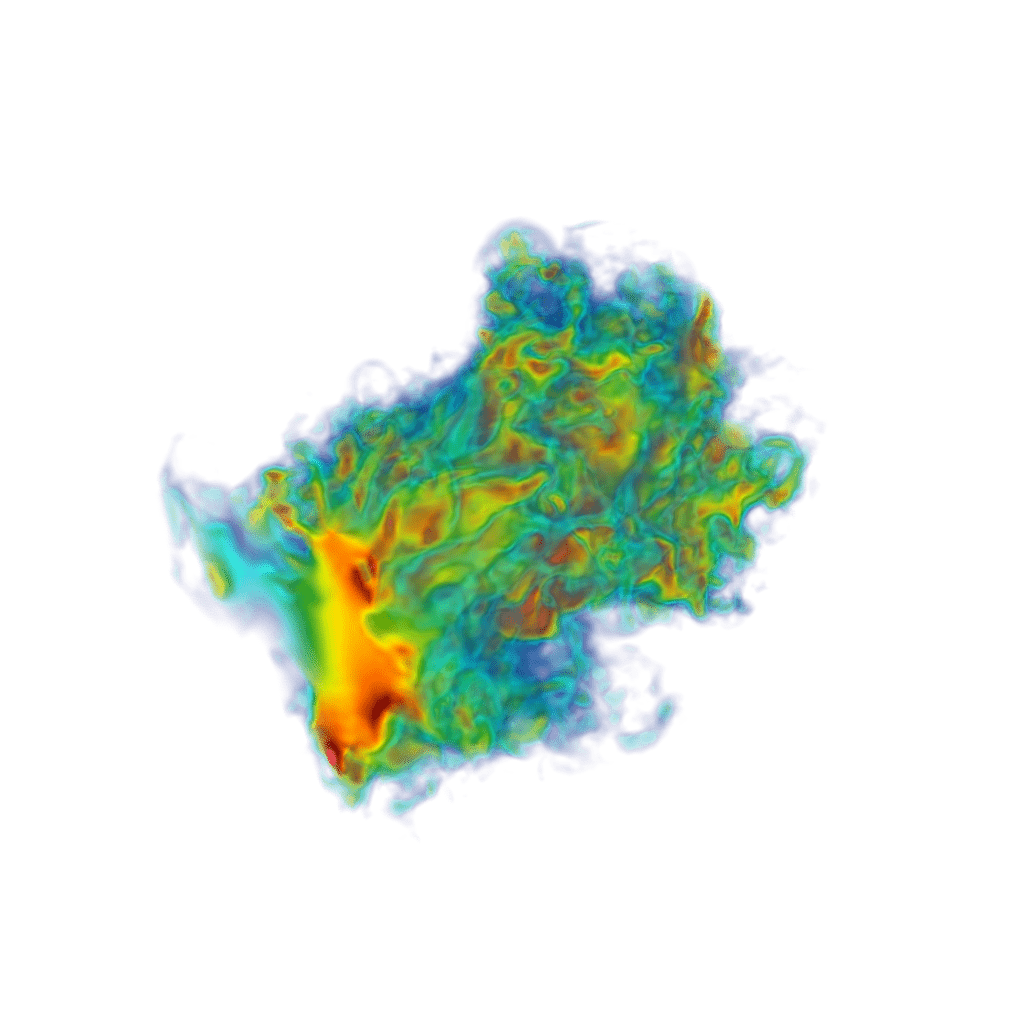}
\endminipage
\minipage{0.3\textwidth}
{\includegraphics[width=\linewidth, clip, trim=100 125 100 125, draft=false]{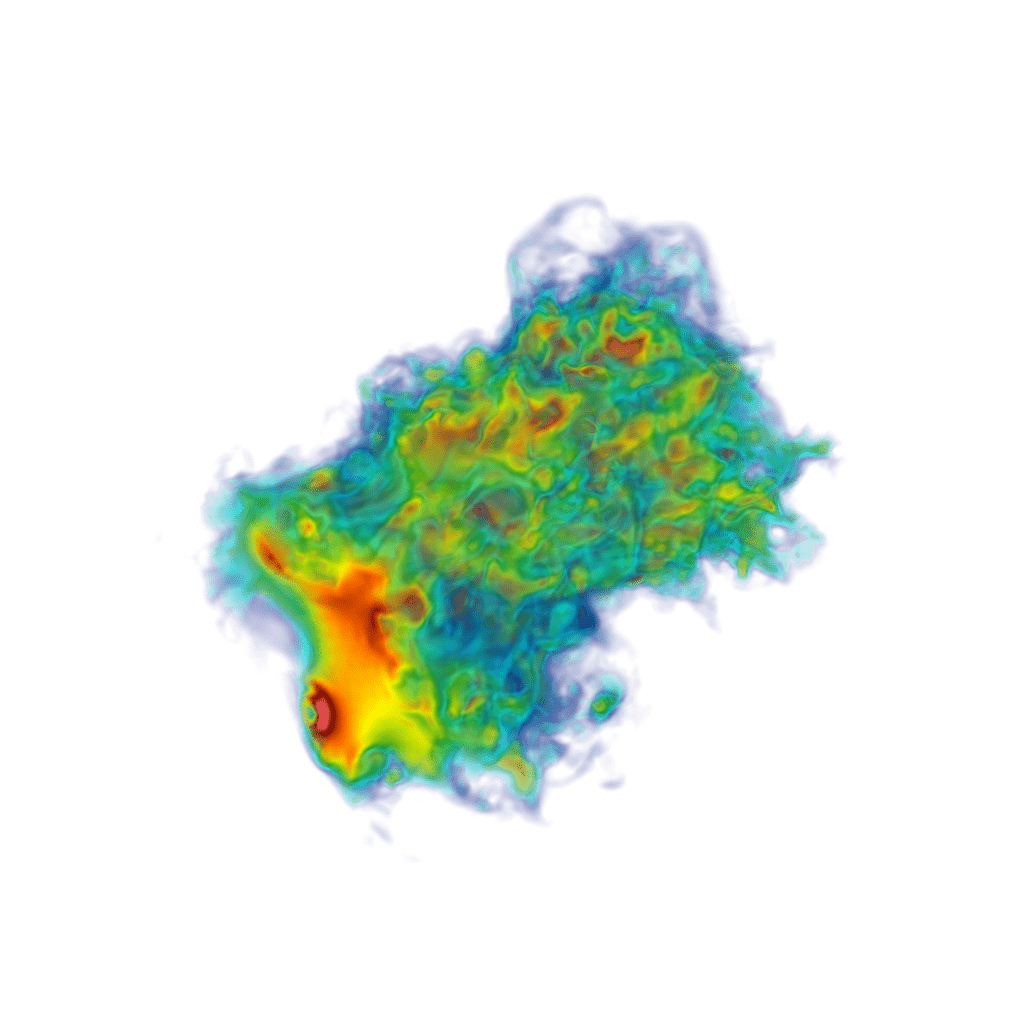}}
\endminipage
\minipage{0.3\textwidth}
{\includegraphics[width=\linewidth, clip, trim=100 125 100 125, draft=false]{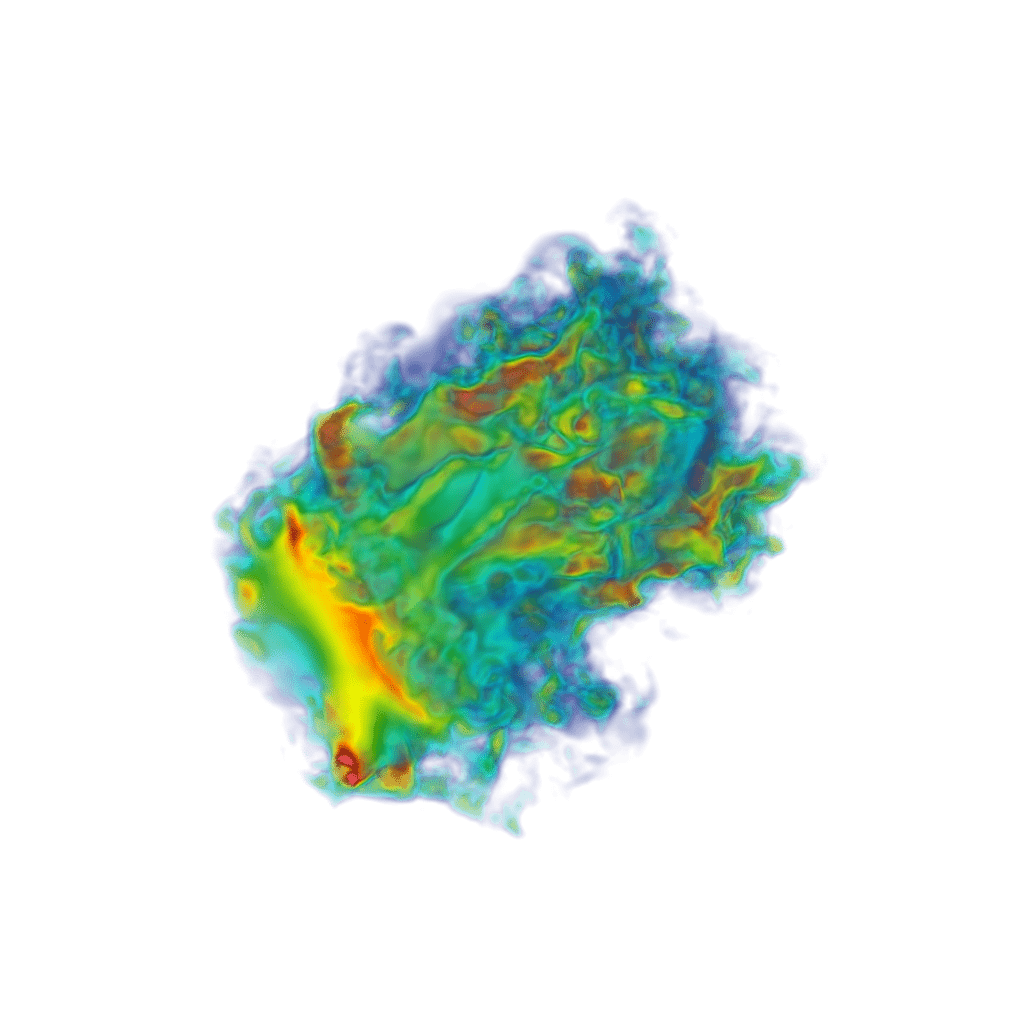}}
\endminipage
\endminipage

\minipage{\linewidth}
\centering
\minipage{0.3\textwidth}
\includegraphics[width=\linewidth, clip, trim=100 125 100 125, draft=false]{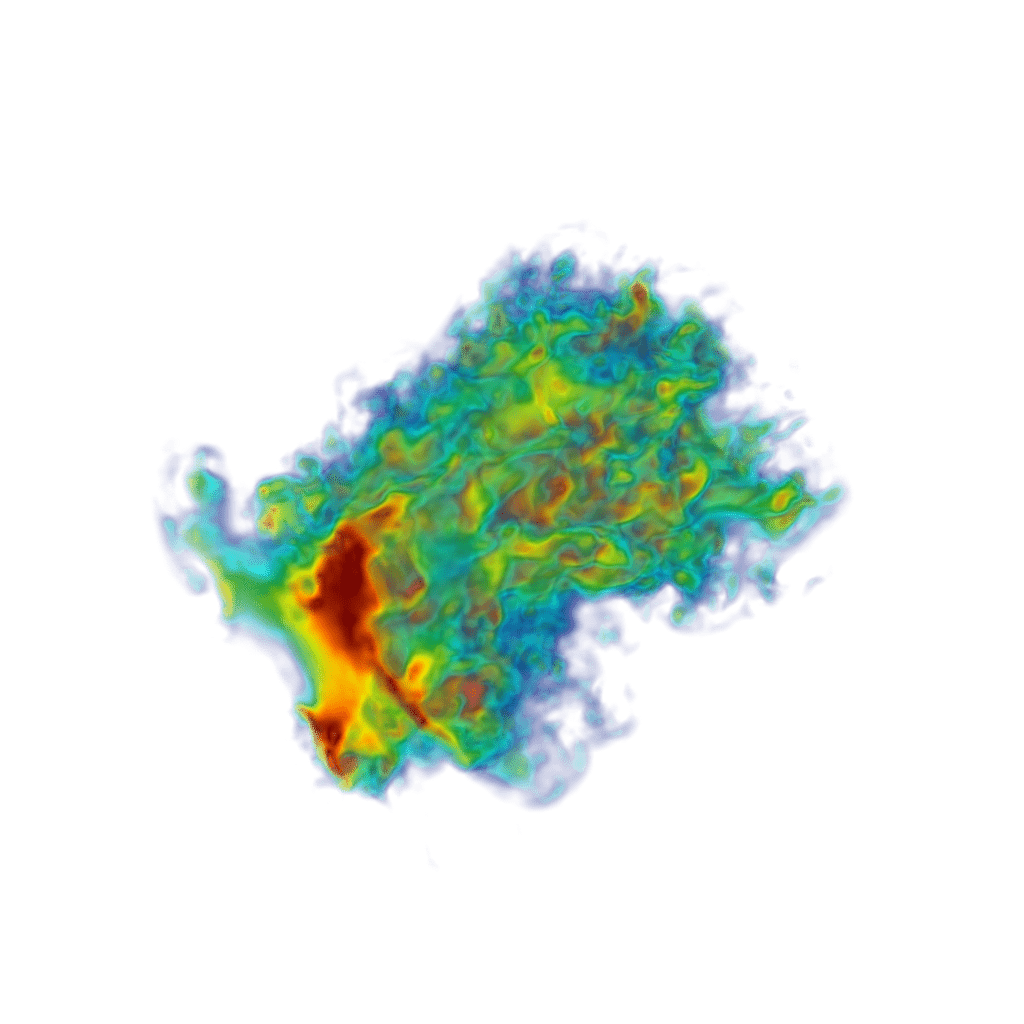}
\endminipage
\minipage{0.3\textwidth}
{\includegraphics[width=\linewidth, clip, trim=100 125 100 125, draft=false]{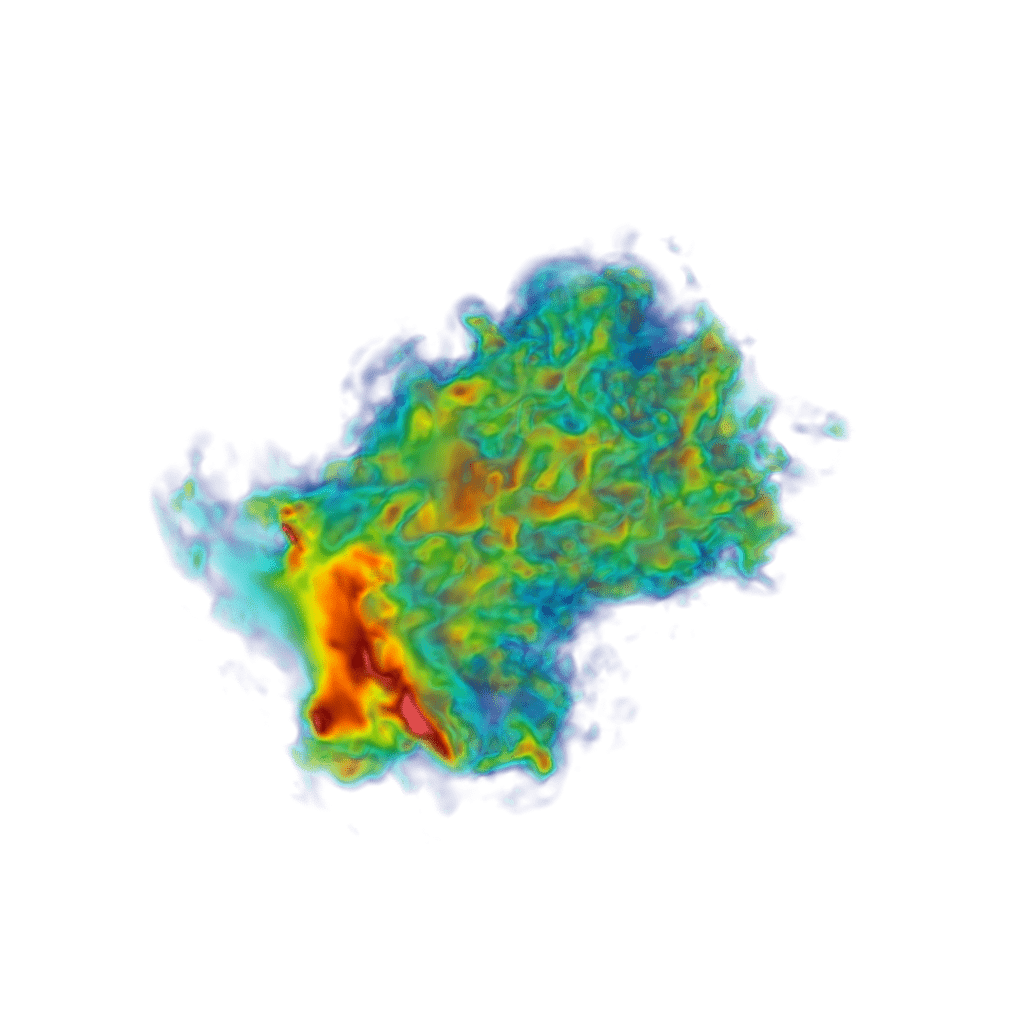}}
\endminipage
\minipage{0.3\textwidth}
{\includegraphics[width=\linewidth, clip, trim=100 125 100 125, draft=false]{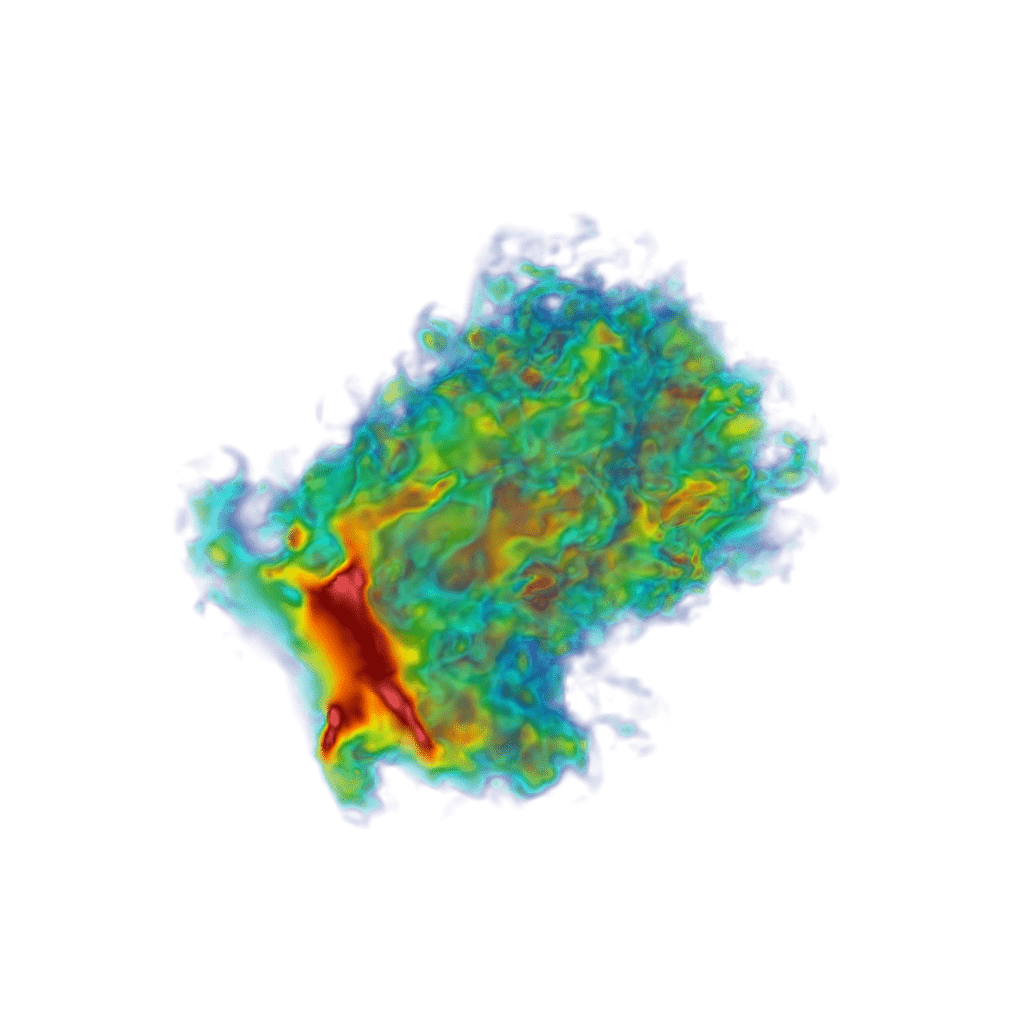}}
\endminipage
\subcaption{Ground truth (top) and GenCFD (bottom)}
\endminipage

\minipage{\linewidth}
\centering
\minipage{0.3\textwidth}
\includegraphics[width=\linewidth, clip, trim=100 125 100 125, draft=false]{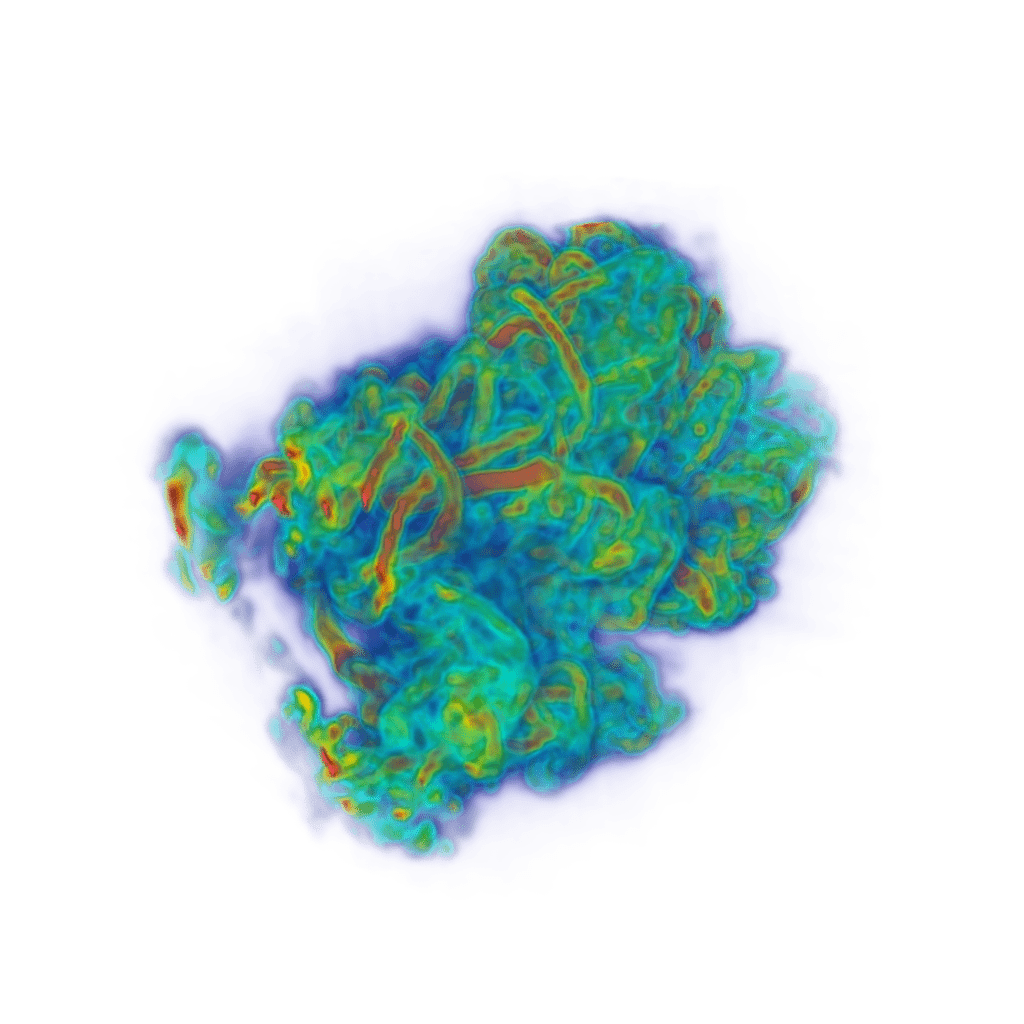}
\endminipage
\minipage{0.3\textwidth}
{\includegraphics[width=\linewidth, clip, trim=100 125 100 125, draft=false]{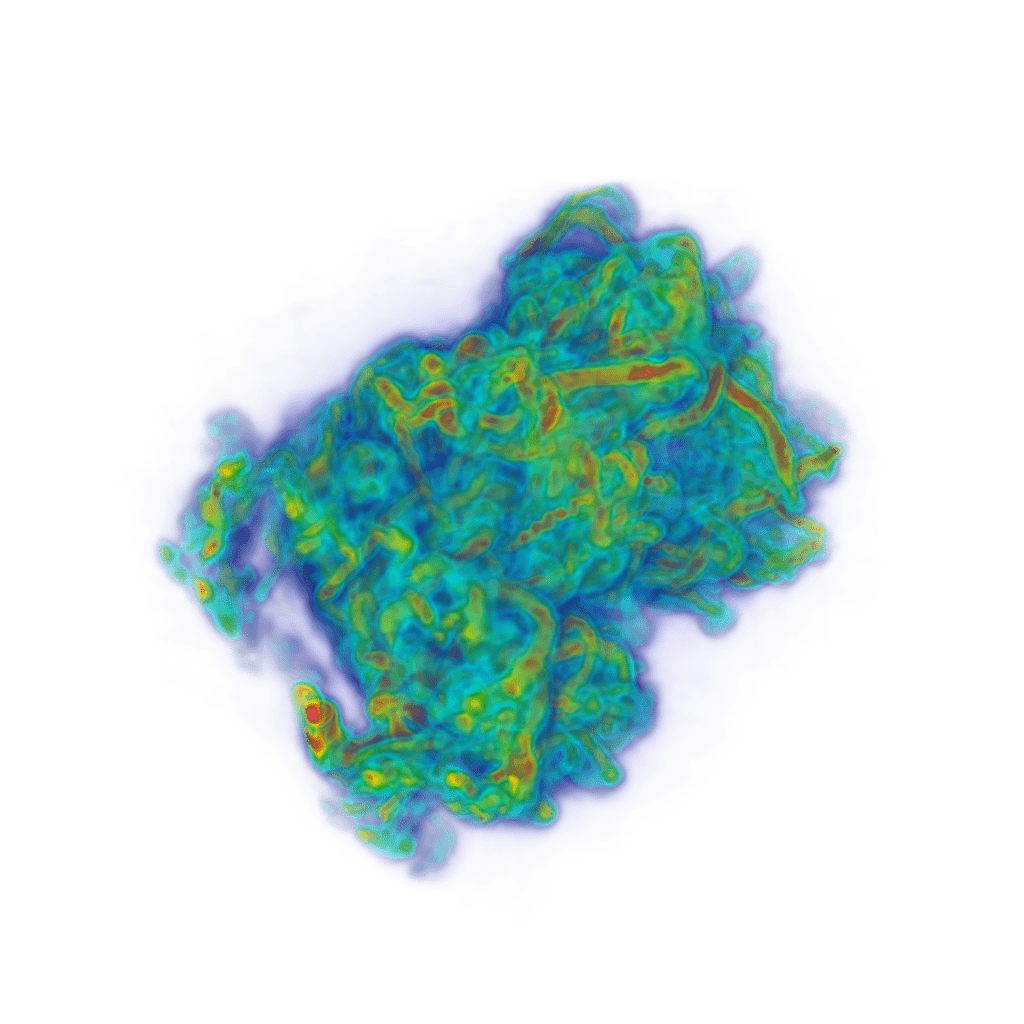}}
\endminipage
\minipage{0.3\textwidth}
{\includegraphics[width=\linewidth, clip, trim=100 125 100 125, draft=false]{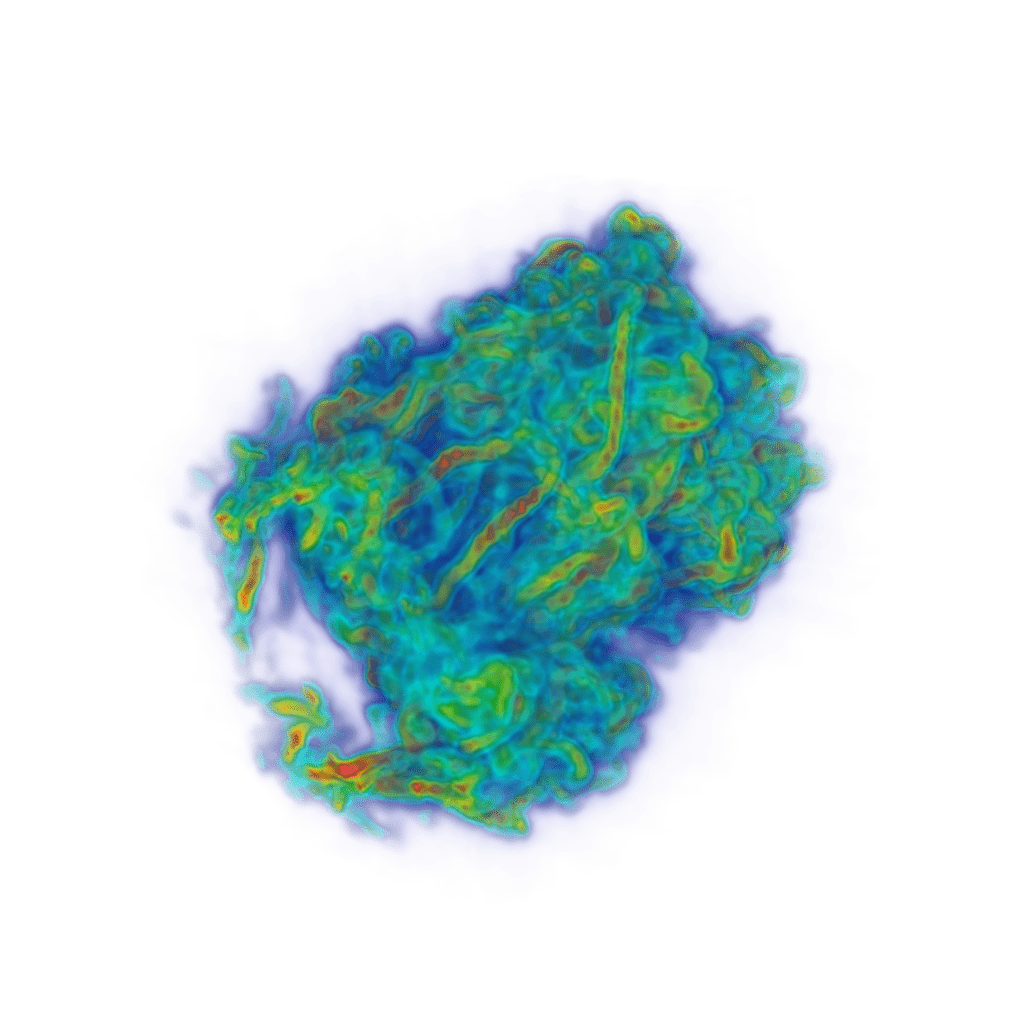}}
\endminipage

\endminipage

\minipage{\linewidth}
\centering
\minipage{0.3\textwidth}
\includegraphics[width=\linewidth, clip, trim=100 125 100 125, draft=false]{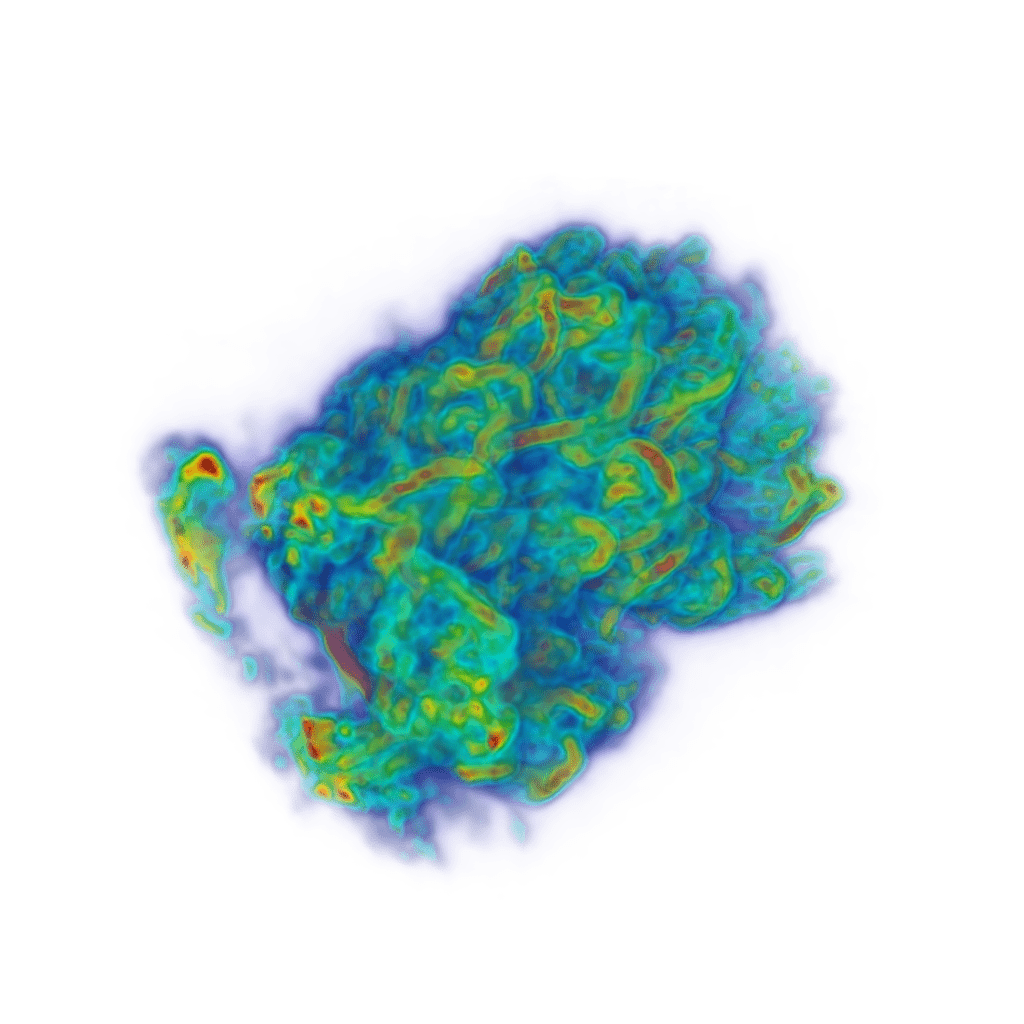}
\endminipage
\minipage{0.3\textwidth}
{\includegraphics[width=\linewidth, clip, trim=100 125 100 125, draft=false]{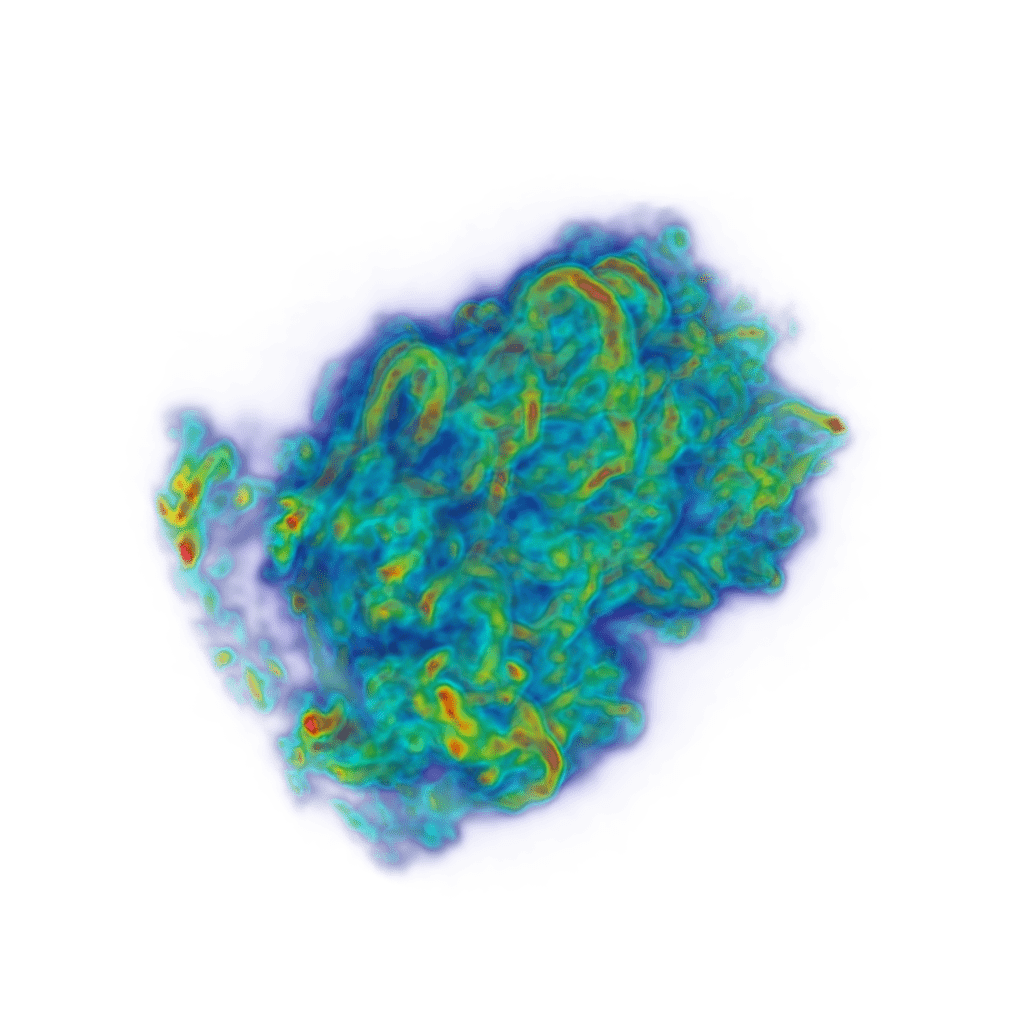}}
\endminipage
\minipage{0.3\textwidth}
{\includegraphics[width=\linewidth, clip, trim=100 125 100 125, draft=false]{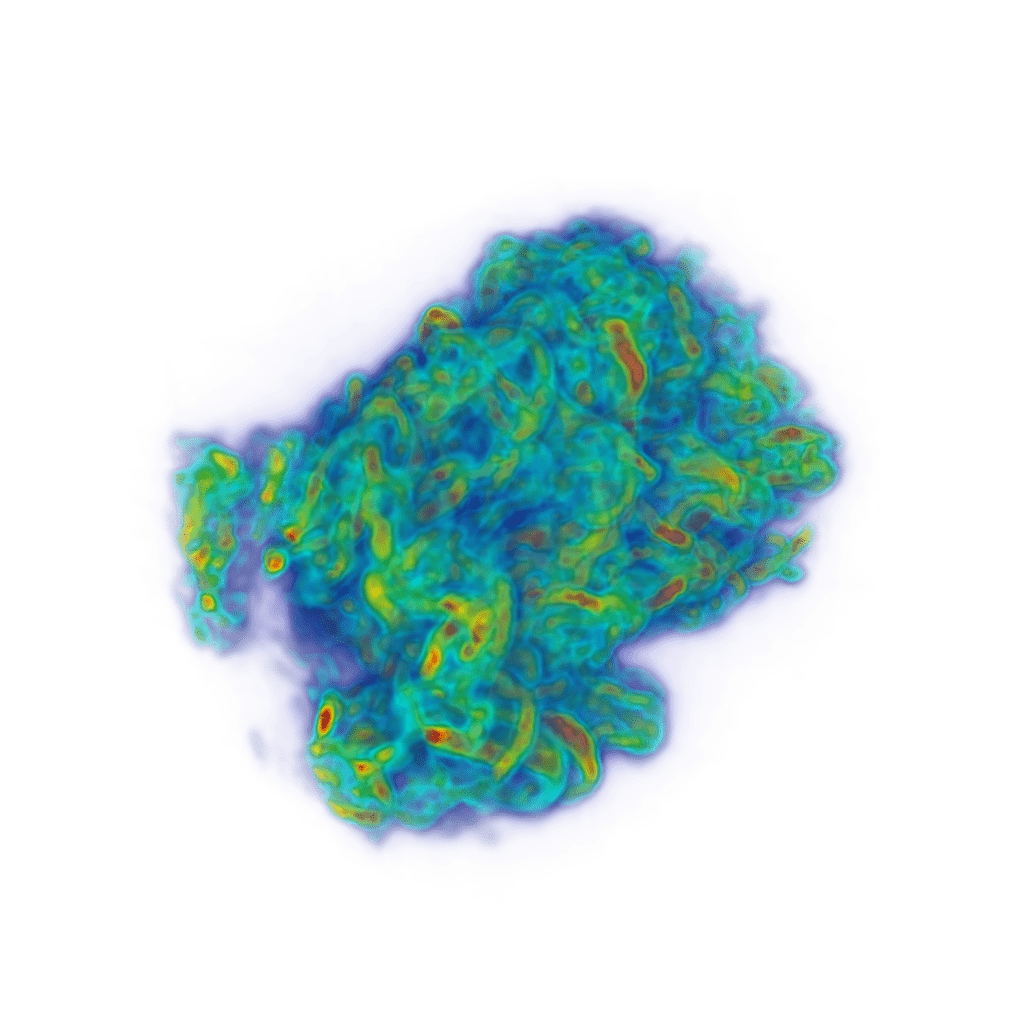}}
\endminipage
\subcaption{Ground truth (top) and GenCFD (bottom)}
\endminipage
\caption{\textbf{Visualization of the kinetic energy (a) of samples drawn from the ground truth and approximated conditional distribution $p(u|\bar{u}=\bar{u}^2)$ and corresponding computed vorticity intensity (b).} Colormaps are identical to the ones used in Fig.~\ref{fig:s2}.}
\label{fig:sm4}
\end{figure}

\clearpage
 \newpage

\begin{figure}[H]
\infobox{Ground truth \\ \phantom{a}}
\begin{subfigure}{.2\textwidth}
\includegraphics[width=\textwidth]{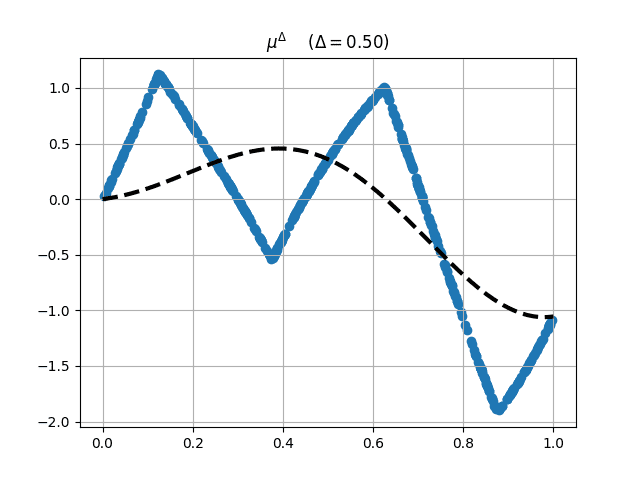}
\caption{$\Delta = 0.5$}
\end{subfigure}
\begin{subfigure}{.2\textwidth}
\includegraphics[width=\textwidth]{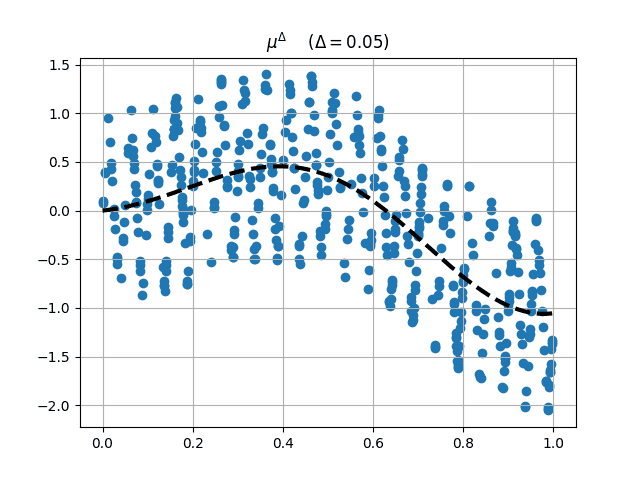}
\caption{$\Delta = 0.05$}
\end{subfigure}
\begin{subfigure}{.2\textwidth}
\includegraphics[width=\textwidth]{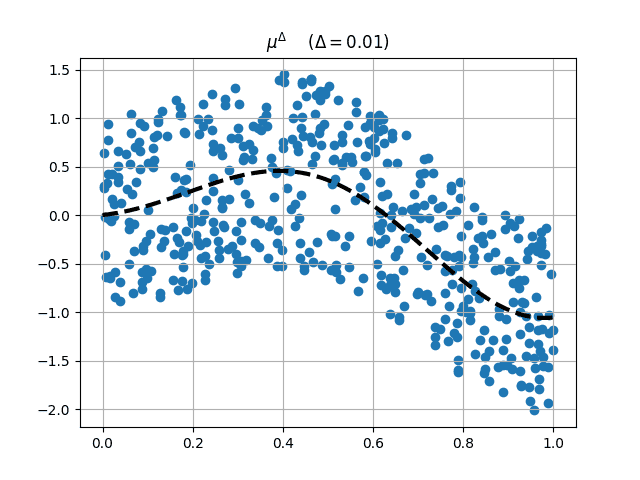}
\caption{$\Delta = 0.01$}
\end{subfigure}
\begin{subfigure}{.2\textwidth}
\includegraphics[width=\textwidth]{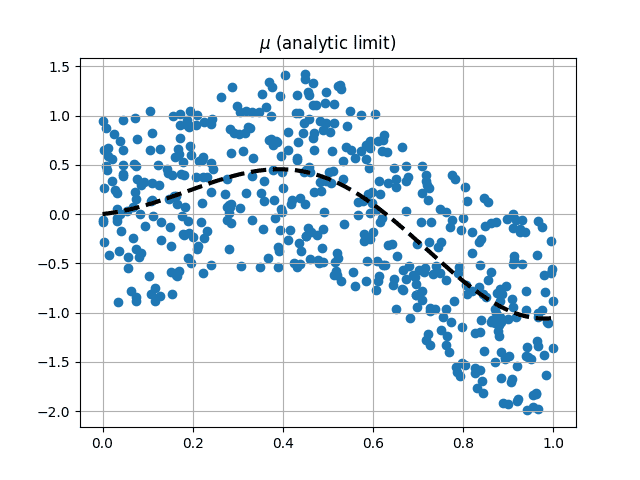}
\caption{Limit $\Delta \to 0$}
\end{subfigure} 
\\
\infobox{Deterministic \\ (500 ep.)}
\begin{subfigure}{0.2\textwidth}
    \includegraphics[width=\textwidth]{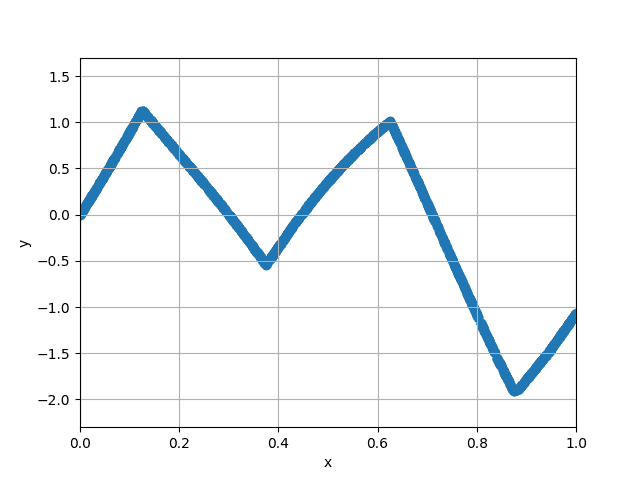}
    \caption{$\Delta=0.5$}
\end{subfigure}
\begin{subfigure}{0.2\textwidth}
    \includegraphics[width=\textwidth]{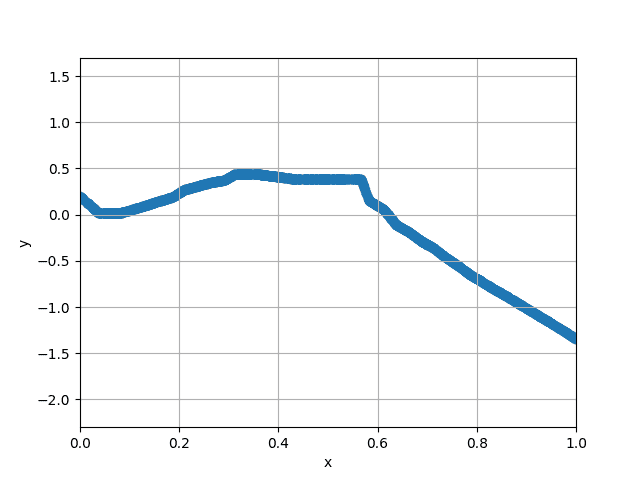}
    \caption{$\Delta=0.05$}
\end{subfigure}
\begin{subfigure}{0.2\textwidth}
    \includegraphics[width=\textwidth]{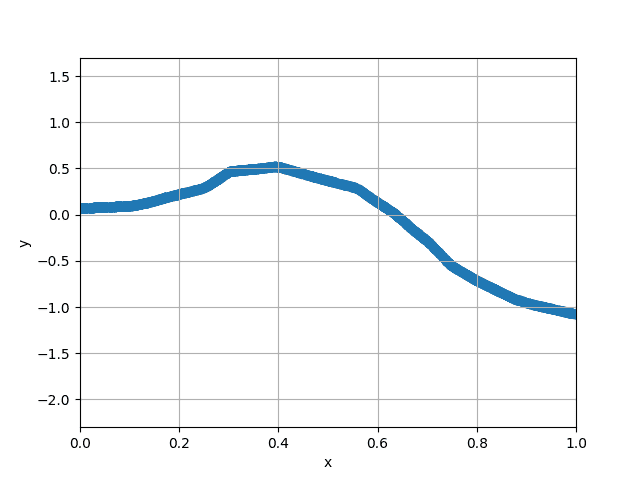}
    \caption{$\Delta=0.01$}
\end{subfigure}
\begin{subfigure}{0.2\textwidth}
    \includegraphics[width=\textwidth]{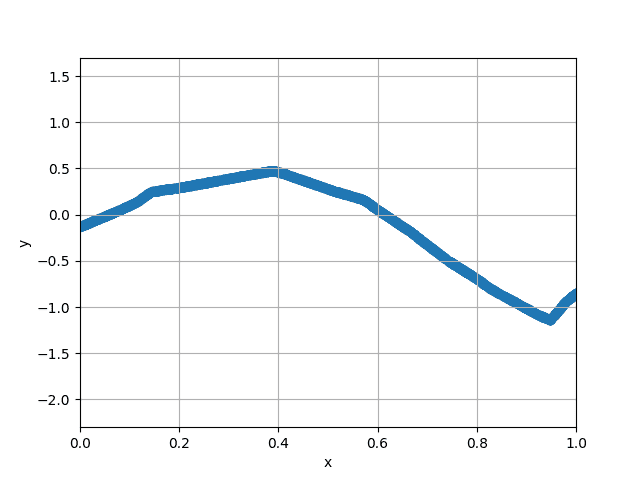}
    \caption{$\Delta=0.002$}
\end{subfigure} 
\\
\infobox{Deterministic \\ (10000 ep.)}
\begin{subfigure}{0.2\textwidth}
    \includegraphics[width=\textwidth]{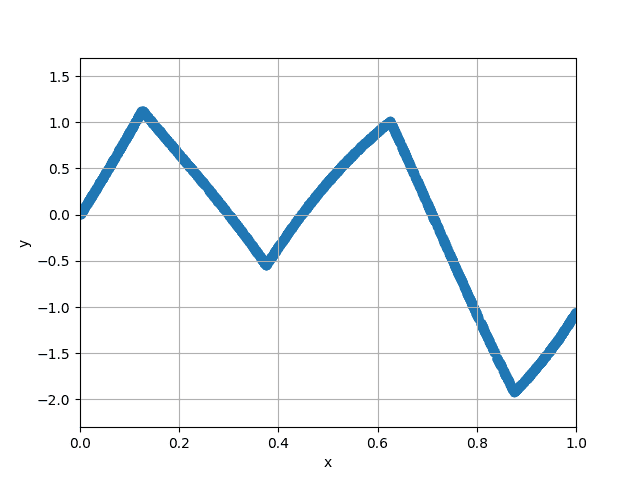}
    \caption{$\Delta=0.5$}
\end{subfigure}
\begin{subfigure}{0.2\textwidth}
    \includegraphics[width=\textwidth]{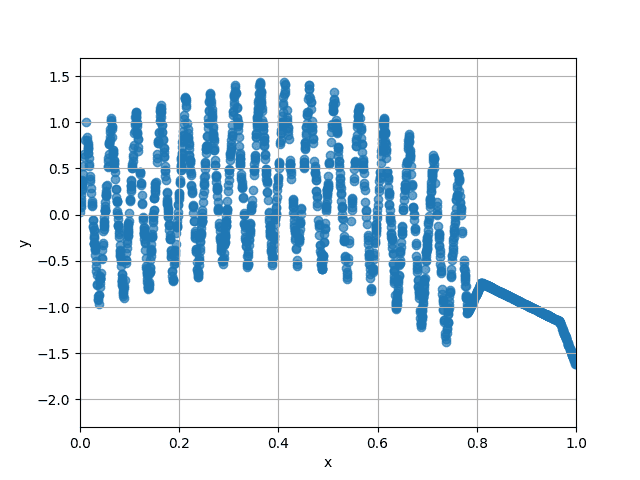}
    \caption{$\Delta=0.05$}
\end{subfigure}
\begin{subfigure}{0.2\textwidth}
    \includegraphics[width=\textwidth]{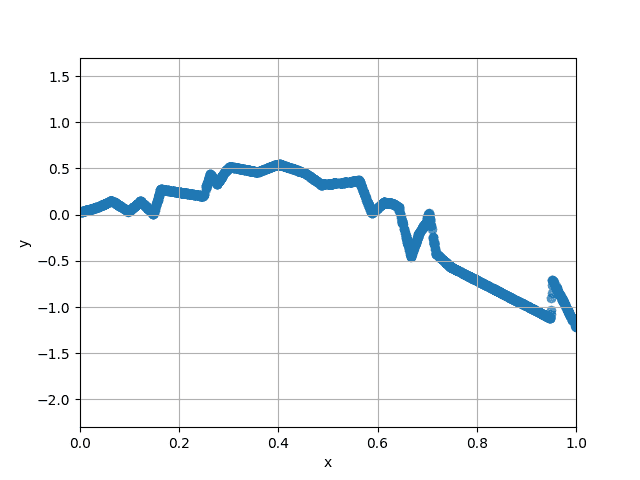}
    \caption{$\Delta=0.01$}
\end{subfigure}
\begin{subfigure}{0.2\textwidth}
    \includegraphics[width=\textwidth]{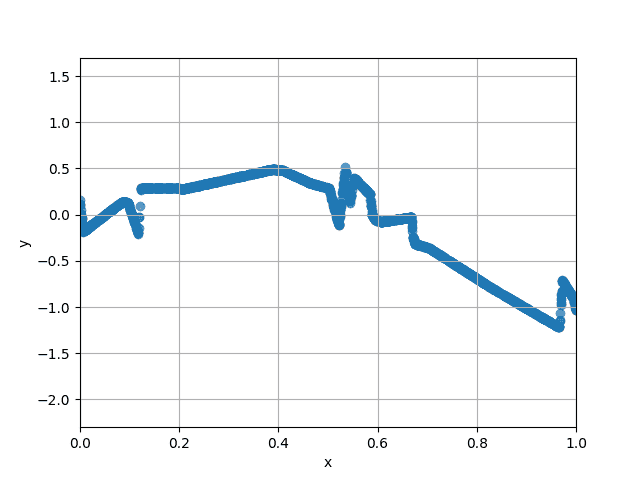}
    \caption{$\Delta=0.002$}
\end{subfigure} 
\\
\infobox{Diffusion \\ (500 ep.)}
\begin{subfigure}{0.2\textwidth}
    \includegraphics[width=\textwidth]{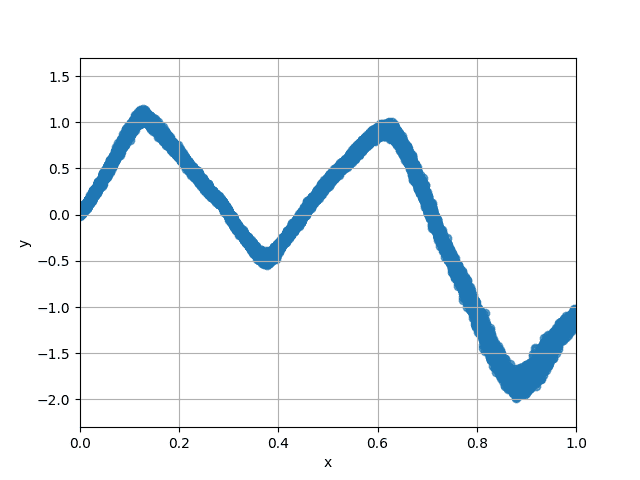}
    \caption{$\Delta=0.5$}
\end{subfigure}
\begin{subfigure}{0.2\textwidth}
    \includegraphics[width=\textwidth]{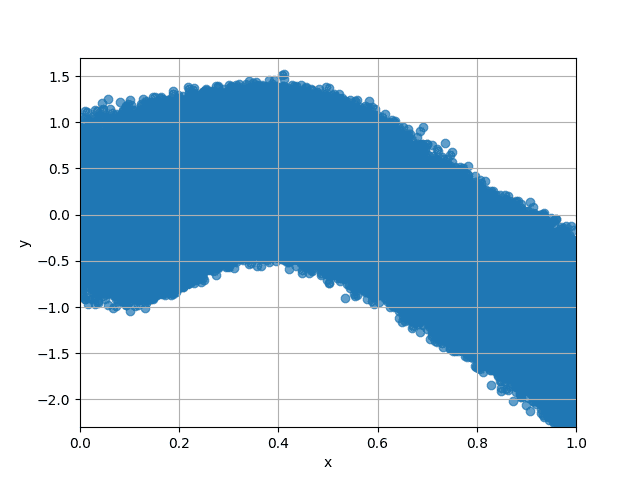}
    \caption{$\Delta=0.05$}
\end{subfigure}
\begin{subfigure}{0.2\textwidth}
    \includegraphics[width=\textwidth]{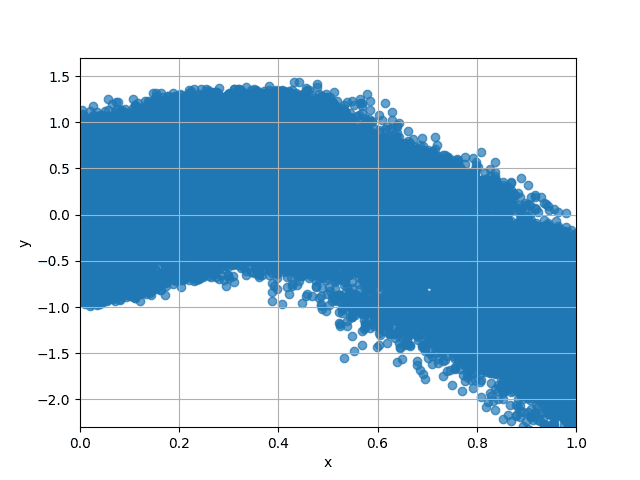}
    \caption{$\Delta=0.01$}
\end{subfigure}
\begin{subfigure}{0.2\textwidth}
    \includegraphics[width=\textwidth]{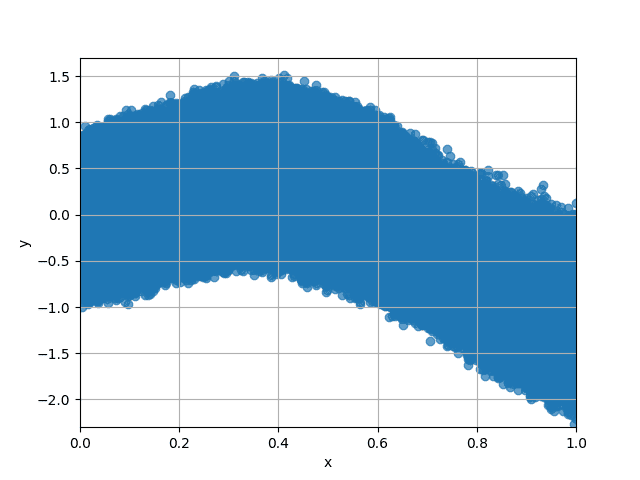}
    \caption{$\Delta=0.002$}
\end{subfigure} 
\\
\infobox{Diffusion \\ (10000 ep.)}
\begin{subfigure}{0.2\textwidth}
    \includegraphics[width=\textwidth]{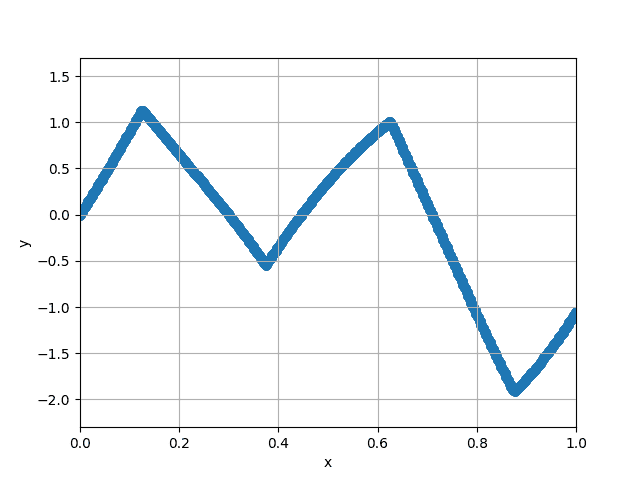}
    \caption{$\Delta=0.5$}
\end{subfigure}
\begin{subfigure}{0.2\textwidth}
    \includegraphics[width=\textwidth]{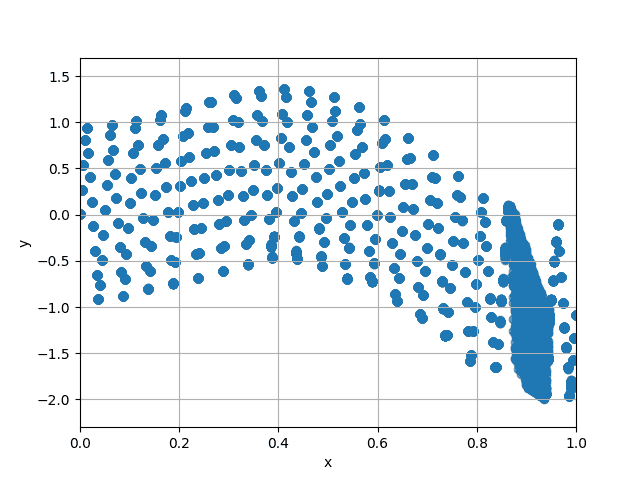}
    \caption{$\Delta=0.05$}
\end{subfigure}
\begin{subfigure}{0.2\textwidth}
    \includegraphics[width=\textwidth]{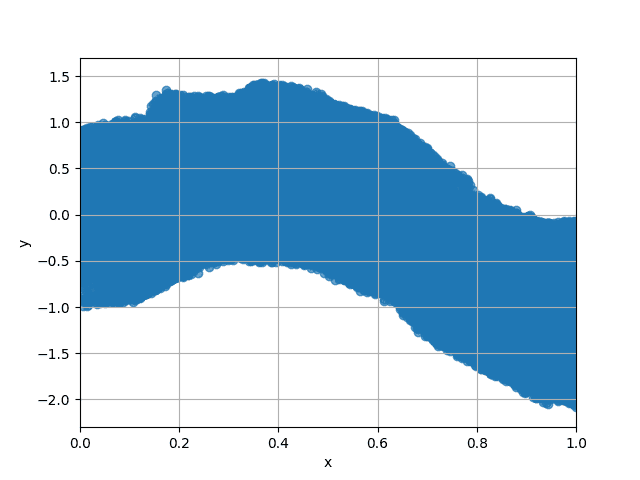}
    \caption{$\Delta=0.01$}
\end{subfigure}
\begin{subfigure}{0.2\textwidth}
    \includegraphics[width=\textwidth]{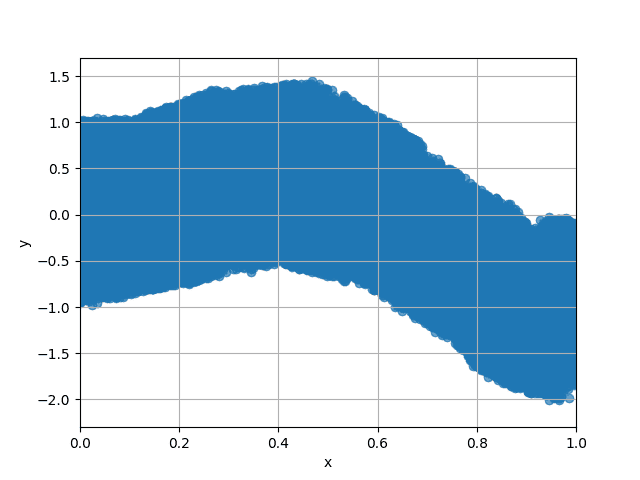}
    \caption{$\Delta=0.002$}
\end{subfigure}

\caption{\textbf{Results for Toy Model $\#1$ at different $\Delta$.} (\textbf{Row 1}): Ground truth, (\textbf{Row 2}): Deterministic ML model with $500$ epochs of Training, (\textbf{Row 3}): Deterministic ML model with $10000$ epochs, (\textbf{Row 4}): Diffusion model with $500$ epochs of training, (\textbf{Row 5}): Diffusion model with $10000$ epochs.}
\label{fig:15}
\end{figure}

\begin{figure}[H]
\centering
\kbox{$k=5$:}
\infobox{Ground truth}
\begin{subfigure}{0.15\textwidth}
    \includegraphics[width=\textwidth]{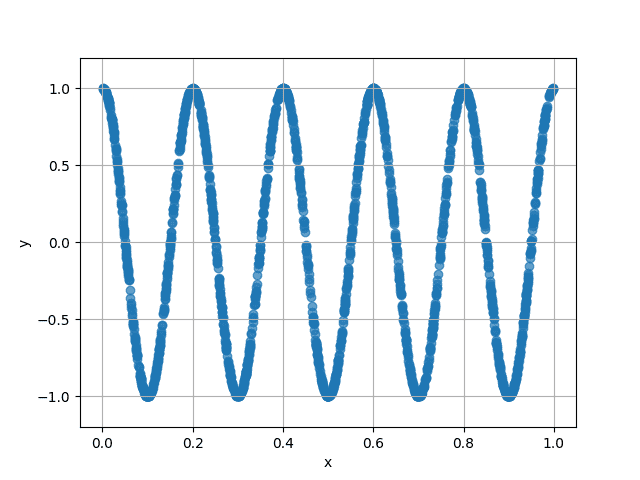}
    \caption{$(h,u_1)$}
\end{subfigure}
\begin{subfigure}{0.15\textwidth}
    \includegraphics[width=\textwidth]{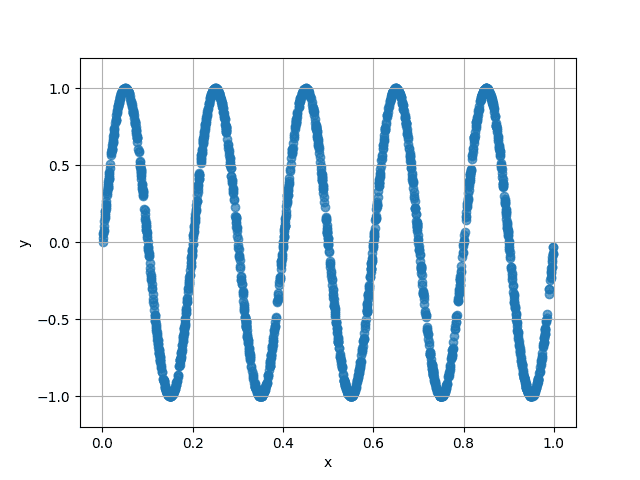}
    \caption{$(h,u_2)$}
\end{subfigure}
\begin{subfigure}{0.15\textwidth}
    \includegraphics[width=\textwidth]{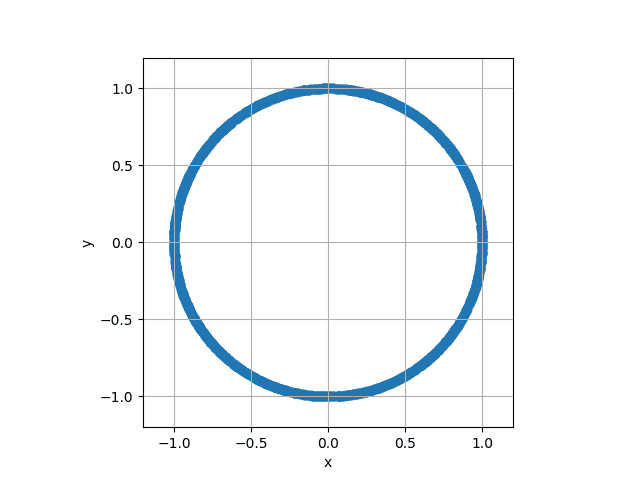}
    \caption{$(u_1,u_2)$}
\end{subfigure} 
\\
\kbox{}
\infobox{Deterministic}
\begin{subfigure}{0.15\textwidth}
    \includegraphics[width=\textwidth]{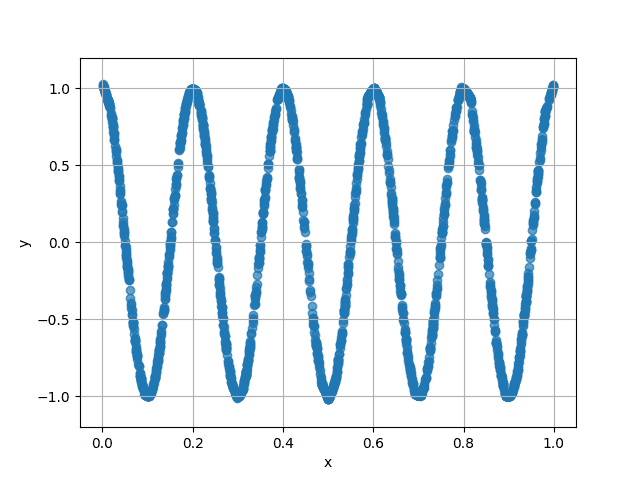}
    \caption{$(h,u_1)$}
\end{subfigure}
\begin{subfigure}{0.15\textwidth}
    \includegraphics[width=\textwidth]{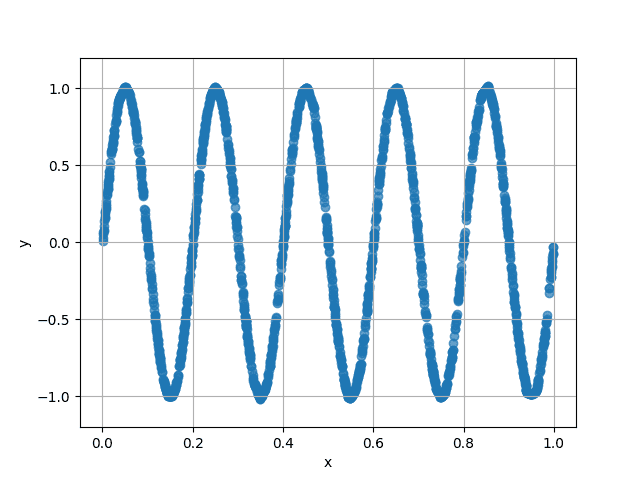}
    \caption{$(h,u_2)$}
\end{subfigure} 
\begin{subfigure}{0.15\textwidth}
    \includegraphics[width=\textwidth]{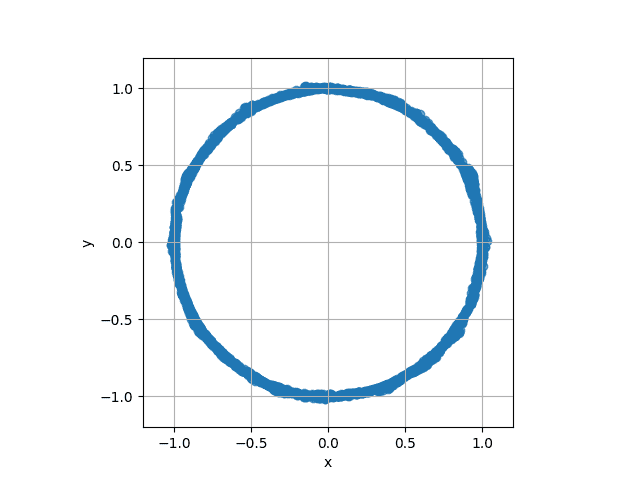}
    \caption{$(u_1,u_2)$}
\end{subfigure} \\

\kbox{}
\infobox{Diffusion}
\begin{subfigure}{0.15\textwidth}
    \includegraphics[width=\textwidth]{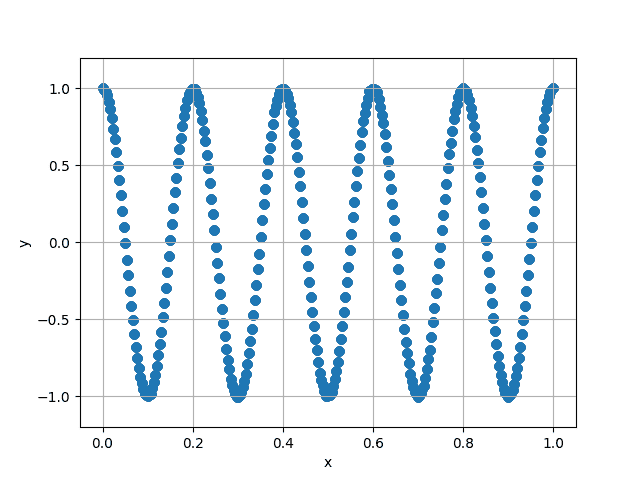}
    \caption{$(h,u_1)$}
\end{subfigure}
\begin{subfigure}{0.15\textwidth}
    \includegraphics[width=\textwidth]{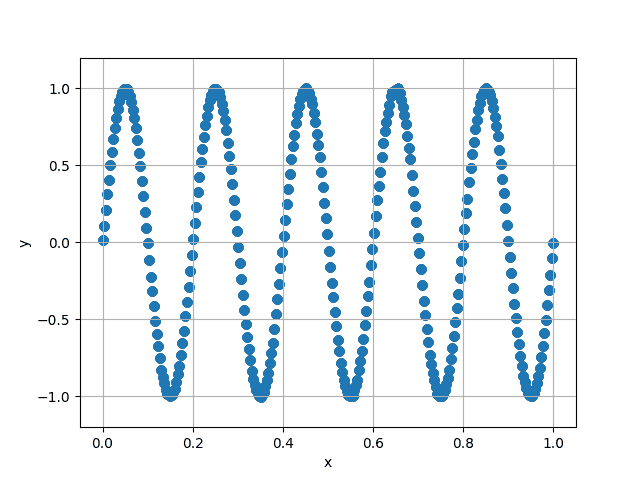}
    \caption{$(h,u_2)$}
\end{subfigure}
\begin{subfigure}{0.15\textwidth}
    \includegraphics[width=\textwidth]{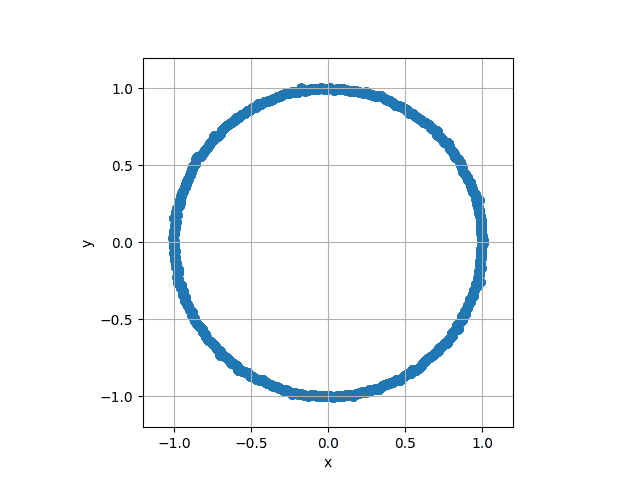}
    \caption{$(u_1,u_2)$}
\end{subfigure} \\

\rule{.85\textwidth}{0.5pt} 
\vspace{0.5em} 

\kbox{$k=30$:}
\infobox{Ground truth}
\begin{subfigure}{0.15\textwidth}
    \includegraphics[width=\textwidth]{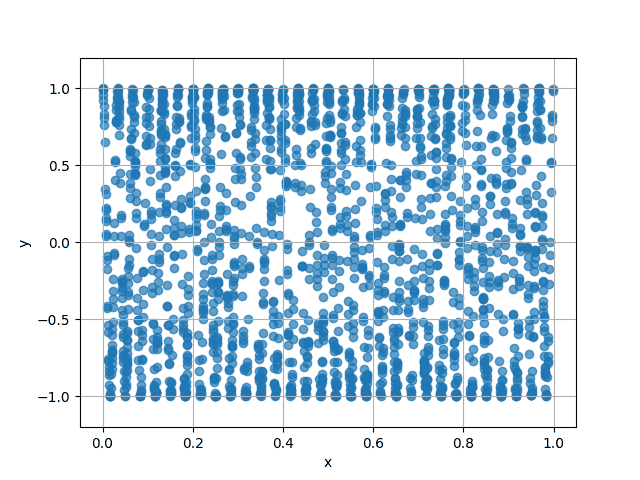}
    \caption{$(h,u_1)$}
\end{subfigure}
\begin{subfigure}{0.15\textwidth}
    \includegraphics[width=\textwidth]{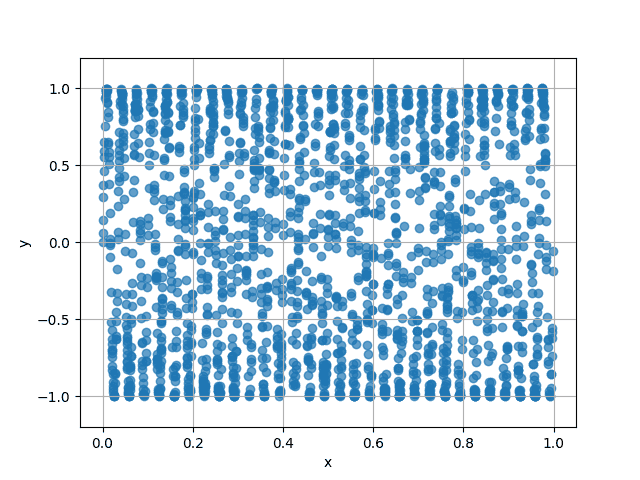}
    \caption{$(h,u_2)$}
\end{subfigure}
\begin{subfigure}{0.15\textwidth}
    \includegraphics[width=\textwidth]{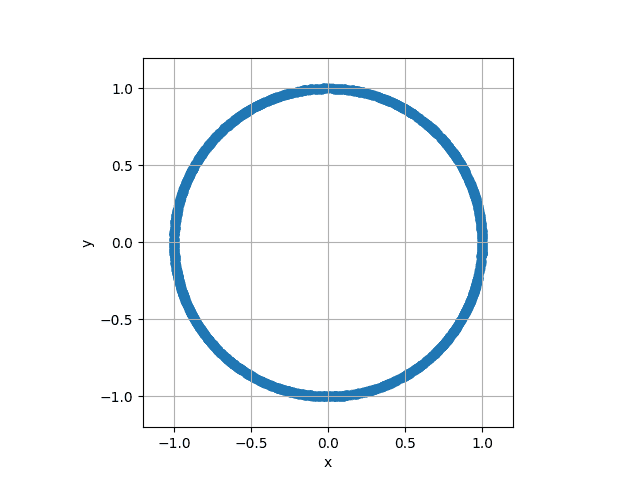}
    \caption{$(u_1,u_2)$}
\end{subfigure} 
\\

\kbox{}
\infobox{Deterministic}
\begin{subfigure}{0.15\textwidth}
    \includegraphics[width=\textwidth]{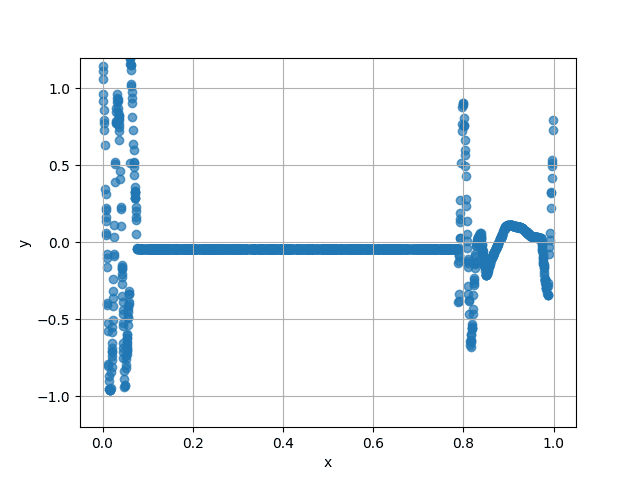}
    \caption{$(x,u_1)$}
\end{subfigure}
\begin{subfigure}{0.15\textwidth}
    \includegraphics[width=\textwidth]{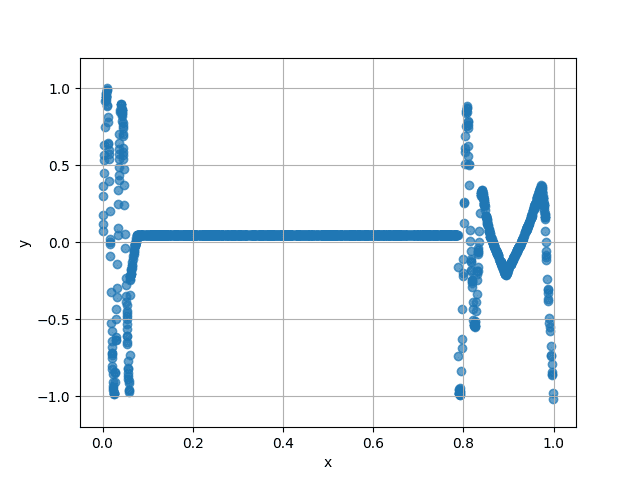}
    \caption{$(x,u_2)$}
\end{subfigure}
\begin{subfigure}{0.15\textwidth}
    \includegraphics[width=\textwidth]{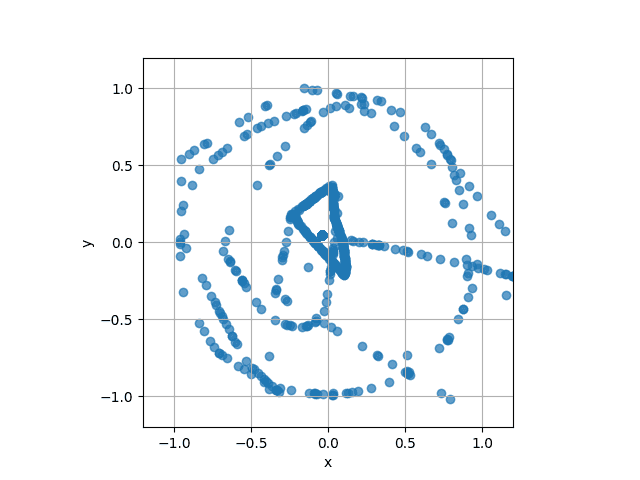}
    \caption{$(u_1,u_2)$}
\end{subfigure} 
\\

\kbox{}
\infobox{Diffusion}
\begin{subfigure}{0.15\textwidth}
    \includegraphics[width=\textwidth]{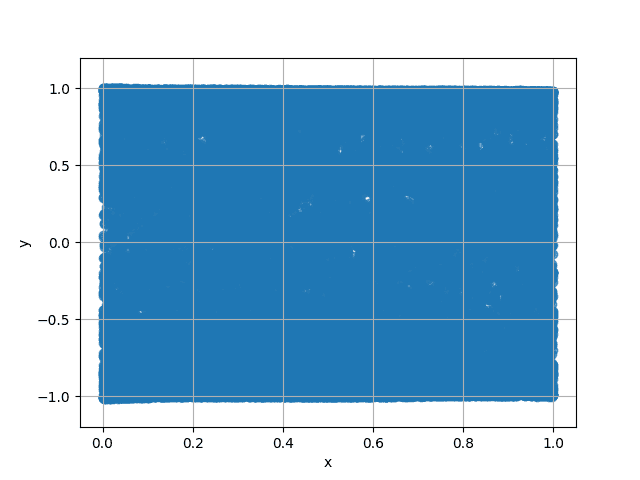}
    \caption{$(x,u_1)$}
\end{subfigure}
\begin{subfigure}{0.15\textwidth}
    \includegraphics[width=\textwidth]{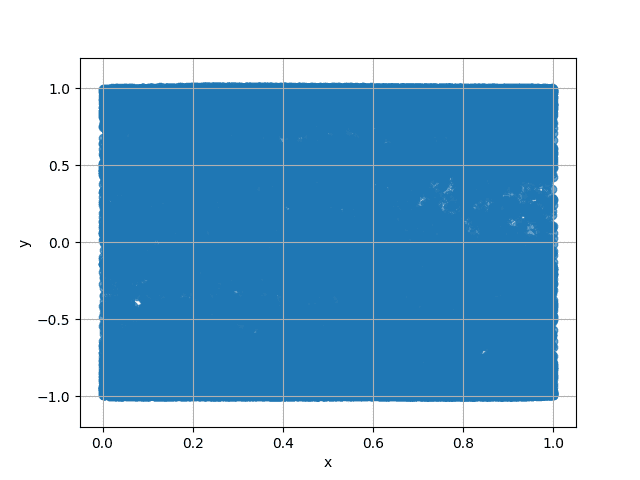}
    \caption{$(x,u_2)$}
\end{subfigure}
\begin{subfigure}{0.15\textwidth}
    \includegraphics[width=\textwidth]{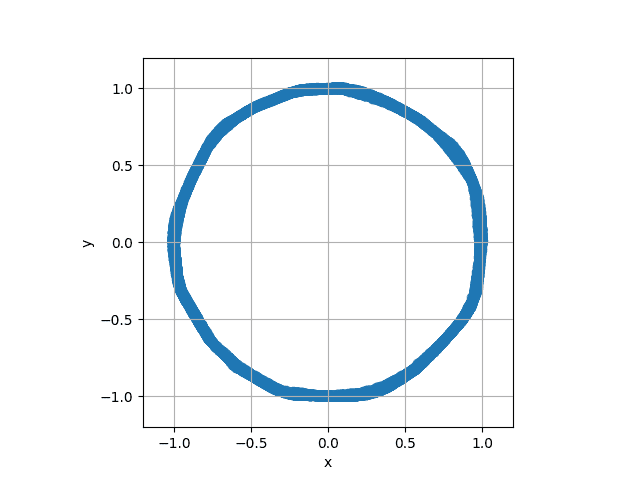}
    \caption{$(u_1,u_2)$}
\end{subfigure}
\caption{\textbf{Toy Model $\#2$ at two different frequencies $k=5$ (Top 3 rows) and $k=30$ (Bottom 3 rows).} Results with ground truth (Top), deterministic ML model (Middle) and diffusion model (Bottom) for each frequency.}
\label{fig:16}
\end{figure}


\newpage
\pagebreak
\clearpage
\bibliographystyle{plain}
\bibliography{biblio.bib} 
\pagebreak
\newpage
\section*{Acknowledgments}
This work was supported by a computing grant from the Swiss National Supercomputing Centre (CSCS) under project ID 1217 as well as a part of the Swiss AI Initiative under project ID a01 on Alps. The authors thank Dr. Emmanuel de B\'ezenac (INRIA, Paris) and Prof. Sebastian Schemm (U. Cambridge) for their inputs. 
S.\ S.\ gratefully acknowledges the support from KHYS at KIT through a ConYS grant and the computing time (NHR Project ID p0023756) made available on the high-performance computer HoreKa funded by the Ministry of Science, Research and the Arts Baden-W\"{u}rttemberg and by the Federal Ministry of Education and Research (Germany).


\end{document}